\documentclass[english]{article}

\usepackage{setspace}
\usepackage{amsmath,amsthm,bm}
\usepackage{amssymb,graphicx,mathtools,setspace,mdframed,verbatim, dsfont,pifont,soul,wrapfig}
\usepackage[mathscr]{euscript}
\usepackage{arydshln}
\usepackage{diy}
\usepackage{custom}
\usepackage{enumerate}
\usepackage[shortlabels]{enumitem}
\usepackage{titletoc}
\usepackage[page, header]{appendix}
\usepackage[hang]{footmisc}

\usepackage{cite}

\usepackage{geometry}
\usepackage{bold-extra}
\usepackage{nicefrac}
\usepackage[utf8]{inputenc} 
\usepackage[T1]{fontenc}    
\usepackage{hyperref}       

\usepackage{url}            
\usepackage{booktabs}       
\usepackage{amsfonts}       
\usepackage{nicefrac}       
\usepackage{microtype}      
\usepackage[dvipsnames]{xcolor}         
\usepackage{cleveref}
\usepackage{pifont}

\usepackage[shortlabels]{enumitem}
\usepackage[ruled,vlined]{algorithm2e}

\usepackage{subcaption}
\usepackage{titletoc}
\usepackage[page, header]{appendix}

\definecolor{myorange}{RGB}{245,156,74}
\hypersetup{
	colorlinks=true,
	linkcolor=blue,
	filecolor=blue,      
	urlcolor=black,
	citecolor=blue,
}

\crefname{definition}{Definition}{Definitions}
\Crefname{definition}{Definition}{Definitions}
\crefname{assumption}{Assumption}{Assumptions}
\Crefname{assumption}{Assumption}{Assumptions}
\crefname{equation}{Eq.}{Eqs.}
\Crefname{equation}{Equation}{Equations}
\crefname{section}{Section}{Sections}
\Crefname{section}{Section}{Sections}
\crefname{subsection}{Section}{Sections} 
\Crefname{subsection}{Section}{Sections}
\crefname{figure}{Fig.}{Figs.}
\Crefname{figure}{Figure}{Figures}
\crefname{table}{Table}{Tables}
\Crefname{table}{Table}{Tables}
\crefname{align}{Eq.}{Eqs.}
\Crefname{align}{Equation}{Equations}
\crefname{fact}{Fact}{Facts}
\Crefname{fact}{Fact}{Facts}
\crefname{induction}{Induction}{Inductions}
\Crefname{induction}{Induction}{Inductions}

\newtheorem{theorem}{Theorem}[section]
\newtheorem{lemma}{Lemma}[section]
\newtheorem{induction}{Induction}[section]
\newtheorem{corollary}{Corollary}[section]

\newtheorem{remark}{Remark}[section]
\theoremstyle{definition}
\newtheorem{definition}{Definition}[section]
\newtheorem{assumption}{Assumption}[section]

\newtheorem{fact}{Fact}[section]
\newtheorem{claim}{Claim}[section]

\makeatletter
\renewcommand\paragraph{\@startsection{paragraph}{4}{\z@}
  {0.25\baselineskip}
  {-0.6em}
  {\normalfont\normalsize\bfseries}}
\makeatother

\usepackage[most]{tcolorbox}

\definecolor{cqAccent}{HTML}{8AB4F8}
\colorlet{cqAccentFrame}{cqAccent!30}
\colorlet{cqBack}{cqAccent!10}
\colorlet{cqTitle}{cqAccent!25}
\definecolor{babyblueeyes}{rgb}{0.63, 0.79, 0.95}
\tcbset{
  enhanced,
  breakable,
  colback=cqBack,
  colframe=cqAccentFrame,
  boxrule=0.35pt,
  arc=2mm,
  drop shadow=black!6,
  before skip=10pt plus 2pt minus 1pt,
  after skip=10pt plus 2pt minus 1pt,
  cqTop/.style = {borderline north={0.6pt}{0pt}{cqAccent!45}}
}

\newtcolorbox{centralquestion}[1][]{
  left=7mm,
  right=7mm,
  top=2mm,
  bottom=2mm,
  fontupper=\itshape,
  coltitle=black,
  attach boxed title to top left={xshift=2mm, yshift*=-2mm},
  boxed title style={
    colback=cqTitle,
    colframe=cqAccentFrame,
    boxrule=0.35pt,
    arc=1mm,
    top=1pt, bottom=1pt, left=4pt, right=4pt
  },
  #1
}

\tcbset{
  myframe/.style={
    enhanced, breakable,
    colback=white, colframe=black, boxrule=0.4pt, arc=1pt,
    left=2mm, right=2mm, top=2mm, bottom=2mm,
    drop fuzzy shadow={opacity=0.35, color=black!50}
  }
}

\newtcolorbox{inlineheadbox}[2][]{%
  myframe,
  title={#2},
  fonttitle=\bfseries,
  boxed title style={
    frame empty, boxrule=0pt, size=minimal,
    interior style={fill=white},
    left=1mm, right=1mm, top=0pt, bottom=0pt
  },
  attach boxed title to top center={yshift=-\tcboxedtitleheight/2},
  #1
}

\allowdisplaybreaks

\usepackage{enumitem}
\usepackage{tikz}
\usetikzlibrary{tikzmark,arrows.meta,calc}
\tikzset{>=Stealth}
\newcommand{\hpad}{10mm}
\newcommand{\vpad}{9mm}
\newcommand{\eqstrut}{\vphantom{\displaystyle A}}
\newcommand{\mboxed}[1]{\boxed{#1\eqstrut}}

\title{Transformers Provably Learn Chain-of-Thought Reasoning \\with Length Generalization}

\author{
  \begin{tabular}{c}
  Yu Huang\footnote{Equal contribution, ordering determined by coin flip. This is the full version of a paper that was published at NeurIPS 2025.} \thanks{\hangindent=1.9em\hangafter=1  Department of Statistics and Data Science, Wharton School, University of Pennsylvania.  Email: \texttt{\{yuh42, yuxinc\}@wharton.upenn.edu} } \hspace{1.5em}
    Zixin Wen\footnotemark[1] \thanks{  Machine Learning Department, Carnegie Mellon University. Email: \texttt{\{zixinw@andrew, aarti@cs\}.cmu.edu} } \\[1em]
    Aarti Singh\footnotemark[3] \hspace{1.5em}
    Yuejie Chi\thanks{  Department of Statistics and Data Science, Yale University. Email: \texttt{yuejie.chi@yale.edu} } \hspace{1.5em}
    Yuxin Chen\footnotemark[2] \vspace{0.5em}
  \end{tabular}
}
\date{\today}

\begin{document}

\maketitle

\begin{abstract}
The ability to reason lies at the core of artificial intelligence (AI), and challenging problems usually call for deeper and longer reasoning to tackle. A crucial question about AI reasoning is whether models can extrapolate learned reasoning patterns to solve harder tasks with longer chain-of-thought (CoT). In this work, we present a theoretical analysis of transformers learning on synthetic state-tracking tasks with gradient descent. We mathematically prove how the \emph{algebraic structure} of state-tracking problems governs the degree of extrapolation of the learned CoT. Specifically, our theory characterizes the length generalization of transformers through the mechanism of \emph{attention concentration}, linking the retrieval robustness of the attention layer to the state-tracking task structure of long-context reasoning. Moreover, for transformers with limited reasoning length, we prove that a \emph{recursive self-training} scheme can progressively extend the range of solvable problem lengths. To our knowledge, we provide the first \emph{optimization guarantee} that constant-depth transformers provably learn \(\mathsf{NC}^1\)-complete problems with CoT, significantly going beyond prior art confined in \(\mathsf{TC}^0\), unless the widely held conjecture \(\mathsf{TC}^0 \neq \mathsf{NC}^1\) fails. Finally, we present a broad set of experiments supporting our theoretical results, confirming the length generalization behaviors and the mechanism of attention concentration.

\end{abstract}

\medskip


\medskip

\setcounter{tocdepth}{2}
\startcontents


\hypersetup{
citecolor=MidnightBlue,filecolor=blue,linkcolor=BrickRed,urlcolor=blue}
\section{Introduction}
\label{sec:intro}

Reasoning is central to artificial intelligence~\cite{chowdhery2023palm,brown2020language,chen2021evaluating,lewkowycz2022solving,gunasekar2023textbooks,touvron2023llama}.
Transformer-based~\cite{Vaswani2017AttentionIA} large language models (LLMs) achieve state-of-the-art results on complex reasoning tasks
via chain-of-thought (CoT) reasoning~\cite{wei2023chain,magister2022teaching,wang2022self,cobbe2021training,reynolds2021prompt,kojima2022large,madaan2022text,zhou2022least},
where the model generates intermediate steps before delivering a final answer. Recent frontier models such as OpenAI-o1~\cite{openai2024o1card} and DeepSeek-R1~\cite{deepseekai2025deepseekr1}
typically produce long CoT traces at inference time, often elicited via reinforcement learning and/or supervised fine-tuning (SFT)
that distills from longer chains~\cite{qin2024o1,min2024imitate}.
These advances have enabled improved performance on more challenging  problems~\cite{kumar2024training,xie2024monte,zhong2024evaluation,gao2025comparison,jones2025large},
yet the mechanisms and limitations underlying CoT reasoning remain poorly understood,  posing fundamental theoretical challenges.

Theoretical studies on transformers with CoT have recently advanced along two fronts: expressiveness~\cite{feng2023revealing, li2024chainthought,merrill2024expressive, chen2024theoretical, merrill2023parallelism} and statistical learnability~\cite{abbe2024fartransformers, wies2023subtask, sanford2024understanding, kim2025metastable, hu2024unveiling, prystawski2023think, li2023dissecting}.
On the expressiveness front, seminal work~\cite{feng2023revealing, li2024chainthought, merrill2024expressive} showed that constant-depth
transformers without CoT behave as shallow circuits and are restricted to express the circuit complexity class \(\mathsf{TC}^0\) of constant computation depth, whereas with \(O(L)\) CoT steps on inputs of length \(L\), they can express log-depth circuits in \(\mathsf{NC}^1\), which is a problem class conjectured to require inherently serial computation.\footnote{For background on circuit complexity and a detailed review of
expressiveness results, see \Cref{sec:problem} and \cite{vollmer1999introduction,arora2006computational}.}  These results reveal that CoT reasoning equips transformers with the expressive power to solve {\bf inherently sequential} problems. In stark contrast, there remains limited understanding of \emph{how} transformers acquire such reasoning abilities during training. Prior optimization analyses~\cite{kim2025transformersprovably,wen2025sparse,huang2025transformers,huang2025transformerslearn} were limited to simple, fully parallelizable tasks in \(\mathsf{TC}^0\) that do not require sequential reasoning. This leaves a substantial gap between what transformers can \emph{express} in principle and what they can \emph{learn} through training.

Another key question, motivated by the success of reasoning models, is whether large models can extrapolate their reasoning beyond the sequence lengths of the training data: a feature known as \textbf{length generalization}. To fully exploit the potential of long CoT reasoning, the model must harness \emph{long-context ability} \cite{kuratov2024babilong, yan2025inftythink, ling2025longreason, yang2025longer}, which can be severely affected by a phenomenon called \emph{context rot}: namely, a phenomenon in which model performance degrades as the number of tokens in the context increases~\cite{hsieh2024ruler, hong2025context, anthropic2025effective}. For reasoning tasks, empirical evidence on length generalization of transformers is mixed~\cite{zhang2022unveiling, liu2023transformers, jelassi2023length, zhou2023algorithms, zhou2024transformers, hou2024universal, xiao2025generalizing}, and several prior studies reported limited extrapolation despite strong in-distribution performance. Architectural choices, including positional encoding and attention variants, can influence generalization considerably~\cite{kazemnejad2023impact,sabbaghi2024explicitly,lee2025selfimproving,anil2022exploring,shaw2018self,raffel2020exploring,dziri2023faith}. On the theoretical front, previous work established existence or statistical guarantees, showing that transformers can, in principle, represent length-generalizing algorithms or achieve favorable sample complexity independent of the length of the CoT~\cite{marsden2024provable,ahuja2024provable,huang2024formal,golowich2025role,joshi2025theory}. Nonetheless, it remains unclear whether transformers can provably learn to length-generalize reliably when trained with gradient-based optimization algorithms.

In light of the aforementioned gaps in prior literature, the current paper seeks to make progress towards addressing the following two fundamental questions:
\vspace{-0.05in}

\begin{centralquestion}[title = {\bf Research Questions}]
\begin{enumerate}[itemsep=0pt,topsep=0pt,leftmargin=1em]
    \item \emph{Can transformers, trained via gradient descent (GD), learn CoT reasoning to solve problems requiring {\it inherently} sequential reasoning beyond \(\mathsf{TC}^0\)?}
    \item \emph{Can the learned reasoning ability {\it generalize} to problems that require longer CoTs beyond the lengths of training data?}
\end{enumerate}
\end{centralquestion}

To address these questions in a theoretically tractable manner, we analyze a minimally viable
transformer: a one-layer transformer block with softmax attention and a
feed-forward network (FFN), trained by GD
with \emph{no positional encoding} (NoPE).  We study
this model on synthetic \emph{state-tracking} tasks, namely LEGO~\cite{zhang2022unveiling}, which distill core LLM
skills such as entity tracking, game-state updates, and code evaluation
\cite{kim2023entity, merrill2024illusion}.  This setup is tractable for analysis, while capturing the mechanisms needed for step-by-step computation via
CoT.  We analyze the training dynamics with CoT on two
LEGO task families with distinct
\emph{algebraic} action structures: simply transitive group actions, and symmetry group actions.  By tracking attention patterns
throughout training, we demonstrate how reasoning capabilities emerge and how
length generalization is governed by the structural properties of these
tasks. More concretely, our main contributions are summarized as follows.
\begin{enumerate}[leftmargin=1.5em]
    \item 
    \textbf{Provable guarantee of learning CoT with length generalization in state-tracking.} We prove that for the LEGO state-tracking task, a one-layer NoPE transformer trained with GD provably solves constant-length problems using CoT reasoning. Moreover, the learned transformer directly generalizes to problems of substantially longer length, when the group actions in LEGO is \emph{simply transitive}. Conversely, for the canonical action of the symmetry group \(S_n\) on $\mathbb{Z}_n$, the learned transformer generalizes only up to constant-factor length. We identify an \textbf{attention concentration} mechanism that dictates step-wise retrieval depending on the task structure, leading to distinct length generalization behaviors.

    \item \textbf{Recursive self-training provably extends the solvable reasoning length.}
    When length generalization is limited (for example, under symmetry group actions), we introduce a \emph{self-training} curriculum that recursively trains the model on its own CoT traces, motivated by the empirical work \cite{lee2025selfimproving}. We prove that this scheme can bootstrap the solvable problem length up to maximal allowable length after sufficient rounds of self-training, thus offering a theoretical guarantee of recursive self-improvement.
    
    \item \textbf{Constant-depth transformers provably learn to solve problems beyond \(\mathsf{TC}^0\) via CoT.} 
    The first two results further establish that the model can learn the solution of state-tracking problems for non-solvable groups, which is \(\mathsf{NC}^1\)-complete and lies outside \(\mathsf{TC}^0\) unless the widely held conjecture \(\mathsf{TC}^0 \neq \mathsf{NC}^1\) in circuit complexity theory fails. Therefore, we provide \emph{the first optimization guarantee} showing that a one-layer transformer learns to solve reasoning tasks beyond \(\mathsf{TC}^0\) with CoT, matching the expressivity result of \cite{feng2023revealing, merrill2024expressive, li2024chainthought} for linear-CoT transformers.

    \item {\bf Empirical evidence on synthetic LEGO tasks that supports our theory.} We present a wide range of experiments based on our theoretical setup. Our results corroborate our predictions on length generalization for different group actions, showing a clear separation between the two algebraic structures in the theoretical setting. Moreover, we empirically demonstrate how recursive self-training effectively improves reasoning length. The attention concentration mechanism identified in our theory is also supported by our experimental findings. 

\end{enumerate}

\subsection{Overview of Theoretical Results}

To facilitate discussion, 
let us first briefly introduce the \emph{state-tracking problem}. Given a group \(\mathcal{G}\) acting on a state space \(\mathcal{Y}\), the goal is to compute the final state $y_L$ by applying a sequence of group actions $g_i \in \mathcal{G}$ starting from an initial state $y_0 \in \mathcal{Y}$: 
\begin{equation}
    y_0 \xrightarrow{g_1} y_1 \xrightarrow{g_2} \cdots \xrightarrow{g_{L'}} y_{L'} \cdots \xrightarrow{g_{L}} y_{L}. \notag
\end{equation}
This task naturally lends itself to CoT reasoning, where intermediate steps explicitly track how the state evolves under successive actions.

We consider two types of group action: simply transitive actions and symmetry group actions \cite{fraleigh2023first}. 
A simply transitive action is free and transitive; that is, there is {\em a unique} group element \(g \in \cG\) mapping any state to another. An example is a group acting on itself by the group composition. In contrast, in the symmetry case, the action is \emph{transitive but not free}: for example, consider the canonical action of $S_n$ on $\mathbb{Z}_n$ by permutations, where {\em many} group elements send a given number $i$ to $j$. In group theory, these two types of actions have different sizes of stabilizers, even when the group \(\cG\) is the same.

\paragraph{\underline{Result 1:} Provable CoT learning for state tracking.}
Our first result shows that transformers can learn to solve these two state-tracking problems via CoT. 
\begin{theorem}[Learning CoT, informal]
One-layer transformers,  trained via GD, can provably learn to solve state tracking problems for simply transitive and symmetry group actions via CoT reasoning.
\end{theorem}

Since state-tracking problems for symmetry group are $\mathsf{NC}^1$-complete \cite{barrington1986bounded}, this result provides the first optimization-based training guarantee for solving problems beyond $\mathsf{TC}^0$,
matching the linear-CoT expressiveness result of~\cite{li2024chainthought} with
provable learnability under gradient-based training. 
In comparison, recent theoretical progress only showed that transformers can provably learn tasks in  $\mathsf{TC}^0$ such as parity 
\cite{kim2025transformersprovably,wen2025sparse,huang2025transformerslearn} and linear regression~\cite{huang2025transformers} with CoT, which, however, are solvable by log-precision, constant-depth transformers without CoT~\cite{li2024chainthought,merrill2024expressive}.

\paragraph{\underline{Result 2:} Algebraic structure dictates length generalization.} Going beyond training, our next result unveils the qualitative difference in length generalization induced by the difference in the group action structure. More importantly, we provide a mechanistic understanding of the length generalization property learned by gradient-based optimization.

\begin{figure}[!t]
    \centering
    \begin{subfigure}[t]{0.48\textwidth}
        \centering
\includegraphics[width=.85\linewidth]{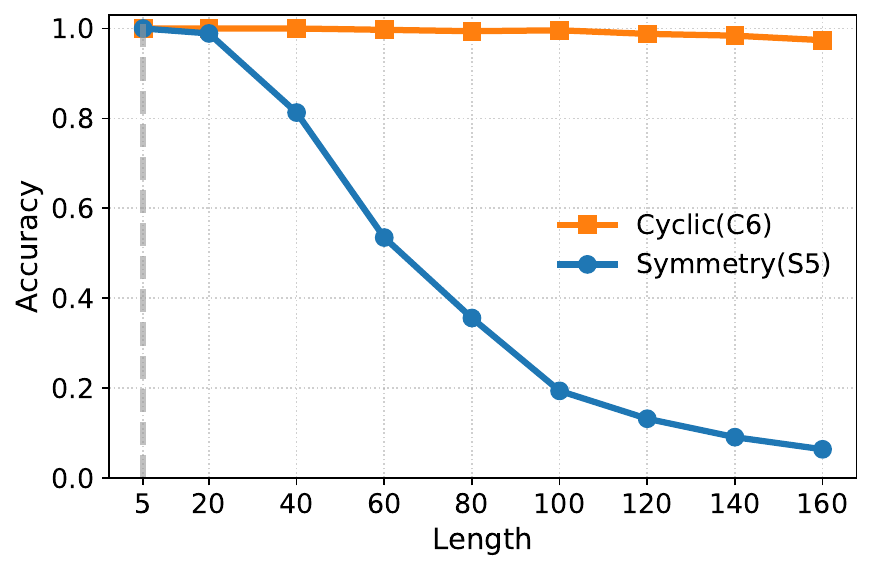}
        \caption{\centering Length generalization results of cyclic ($C_6$) vs. \\ symmetry ($S_5$) tasks.  }
        \label{fig:main-results-a}
    \end{subfigure}
    \hfill
    \begin{subfigure}[t]{0.48\textwidth}
        \centering
\includegraphics[width=.85\linewidth]{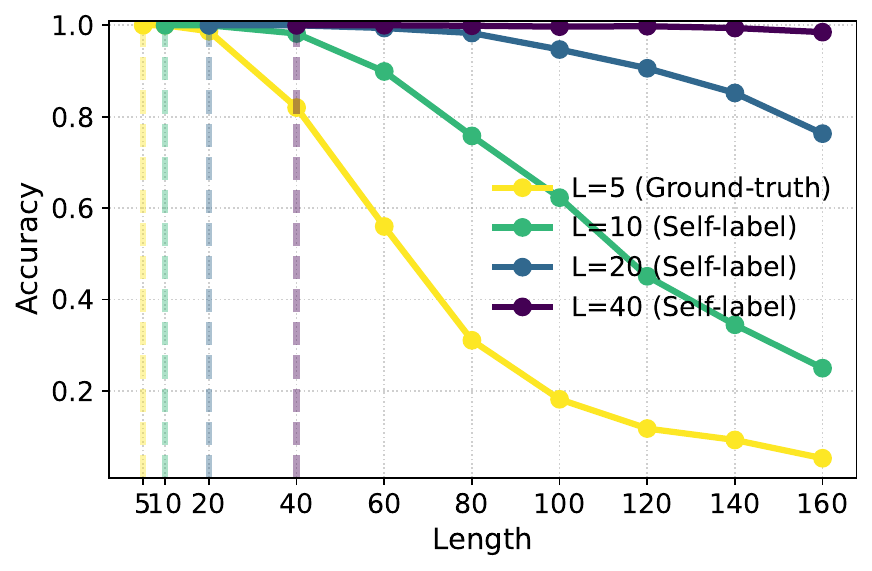}
        \caption{\centering Self-improvement results on symmetry tasks ($S_5$).}
        \label{fig:main-results-b}
    \end{subfigure}
    \caption{Empirical results of length generalization on LEGO tasks with different group actions. (a). Transformers length-generalize to solve significantly longer CoT tasks for simply transitive (cyclic) group (\Cref{thm:length-generalization}), while generalizing poorly for symmetry group tasks. (b). When direct length generalization falls short for symmetry actions, a recursive self-training scheme that train on the model's own longer CoT traces bootstraps the solvable problem length (\Cref{thm:length-gen-self-training}). The dashed lines indicate the training length.
    }
    \vspace{-0.35cm}
    \label{fig:main-results}
\end{figure}

\begin{theorem}[Length generalization, informal]\label{informal-thm-len-gen}
    The algebraic structure of the group actions in the state-tracking task dictates how far the reasoning length generalizes:
    \begin{itemize}[leftmargin=1.5em]
      \item \emph{For simply transitive actions:}
      standard CoT training on constant-length problems already yields generalization to problems of
      length \(d^{\,c^{*}}\), where \(0<c^{*}<1\) is a constant.
      \item \emph{For symmetry actions:}
      standard CoT training only generalizes to a constant factor of the training length. 
    \end{itemize}
    \end{theorem}

This provides a rigorous theory of length extrapolation for state-tracking: it characterizes when and how length generalization emerges for different reasoning problems. Formal statements are deferred to \Cref{sec:cyclic} (simply transitive actions) and \Cref{sec:symmetry} (symmetry group actions).

\vspace{-0.15cm}
\begin{inlineheadbox}{}
    Our theory uncovers the \textbf{mechanism of attention concentration} that explains the varying degrees of length generalization. At a high level, the performance of long-context reasoning depends on the level at which the attention layer focuses on task-relevant contexts. The algebraic structure of the group action dictates the level of attention concentration at convergence: simply transitive actions enable sharp concentration and strong length generalization, whereas symmetry actions introduce distractors that dilute attention focus, limiting length generalization to constant factors.
\end{inlineheadbox}

\paragraph{\underline{Result 3:} Self-training extends solvable problem length.} Although length generalization is limited for symmetry actions, recent empirical studies~\cite{singh2024human,gulcehre2023reinforced,lee2025selfimproving} have shown that neural networks can bootstrap their capability via \emph{self-improvement}. In this work, we prove that a one-layer transformer obtained via a self-training scheme similar to \cite{lee2025selfimproving}, can extend its solvable problem length. To the best of our knowledge, this offers the first {optimization} guarantee of self-improvement for transformer networks.

\begin{theorem}[Recursive self-improvement, informal]\label{informal-self-training}
A one-layer transformer, trained with a recursive self-training scheme, can self-improve: at each stage \(k> 1\), the model learning on the self-labeled reasoning traces of length \(2^k\), length generalizes to solve the next doubled length \(2^{k+1}\).
Consequently, after $\Theta(\log d)$ stages, the model can solve problems of length $d$,
the maximal length in our setting.
\end{theorem}

\begin{figure}[!t]
    \centering
    \begin{subfigure}[t]{0.45\textwidth}
        \centering
\includegraphics[width=0.9\linewidth]{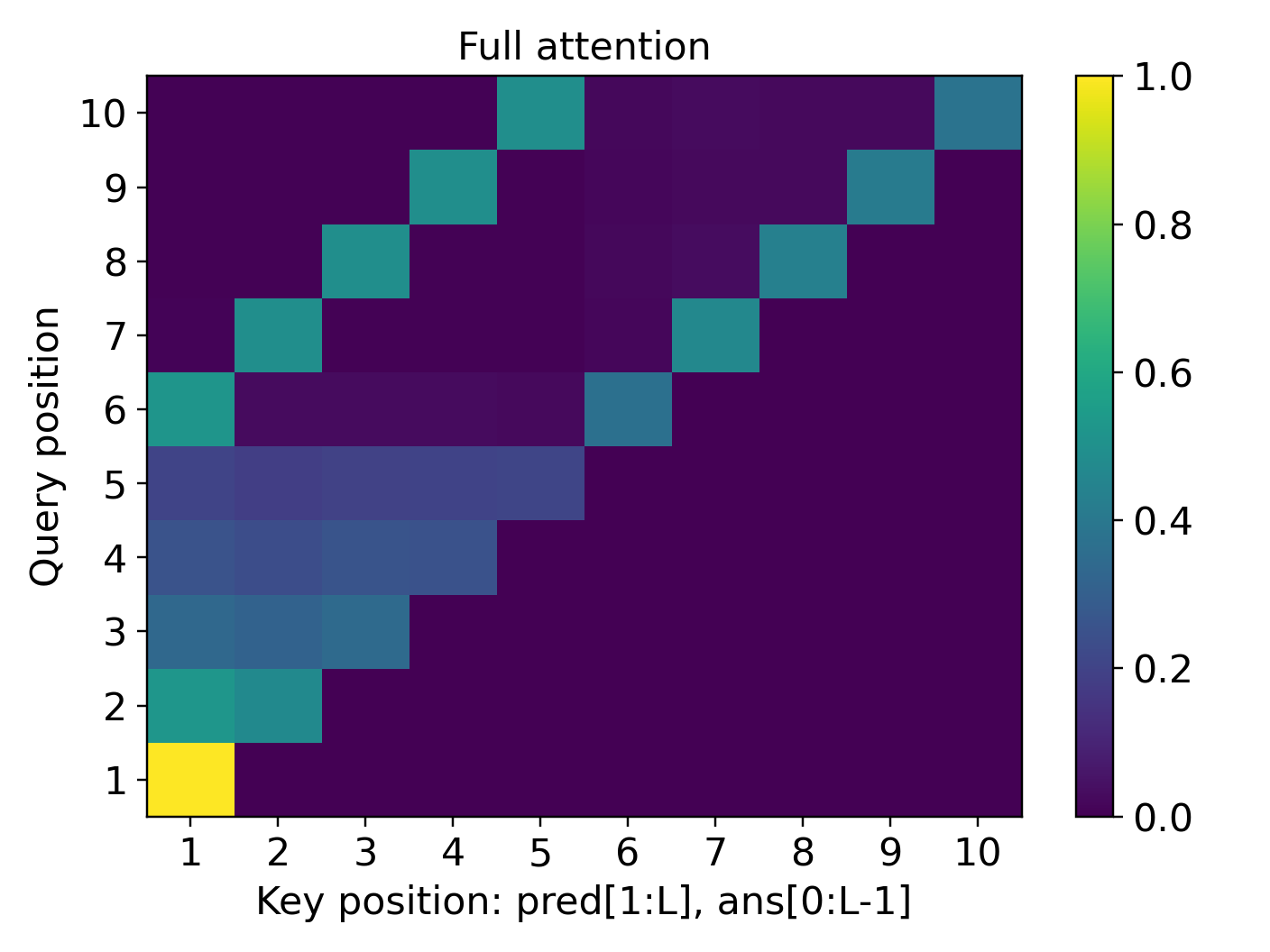}
        \caption{\centering  Cyclic ($C_6$) task }
        \label{fig:attn-con-1}
    \end{subfigure}
    \hfill
    \begin{subfigure}[t]{0.45\textwidth}
        \centering
\includegraphics[width=0.9\linewidth]{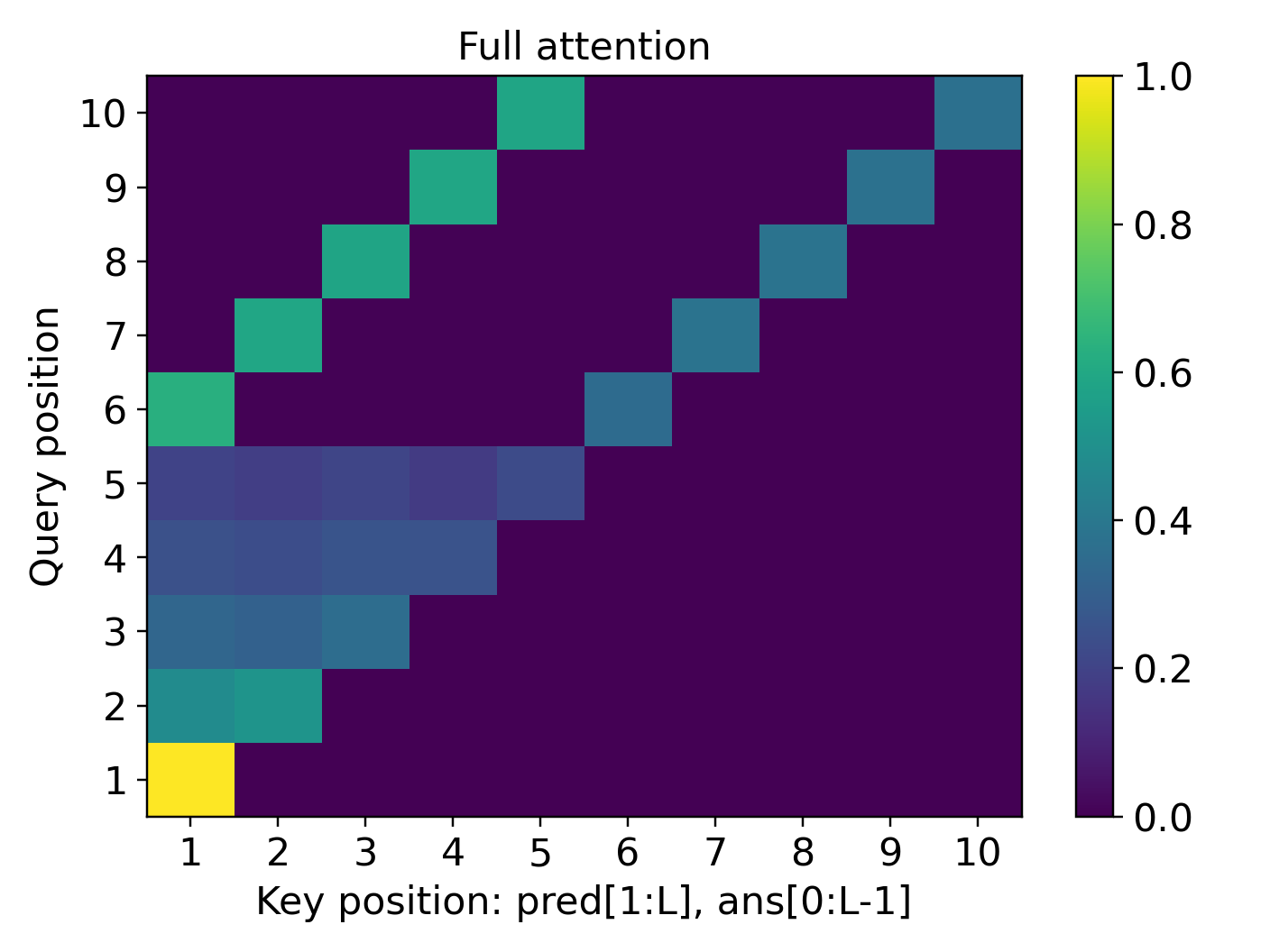}
        \caption{\centering Symmetry ($S_5$) task}
        \label{fig:attn-con-2}
    \end{subfigure}
    \caption{Attention concentration at convergence for LEGO task with length $L=5$. The heatmap places the query clause index on the $y$-axis (keys on the $x$-axis). For a task of length $L$, the LEGO sequence prior to the final
answer clause has length \(2L\);  we focus on query positions $L+1$ to $2L$, corresponding to answer clauses $Z_{\ans,0}$ to $Z_{\ans,L-1}$. Two diagonal bands in the upper region indicate attention concentrating on the answer clause $Z_{\ans,\ell}$ and the predicate clause $Z_{\pred,\ell+1}$ when
    the query is $Z_{\ans,\ell}$.
    }
    \label{fig:attn-con}
\end{figure}

Our work complements recent advances on length generalization in transformers. Prior studies~\cite{marsden2024provable, ahuja2024provable, huang2024formal, golowich2025role, joshi2025theory} primarily established statistical guarantees and \emph{non–gradient-based} learnability.
Notably, \cite{golowich2025role} identifies conditions under which specialized positional encodings and sparse contextual dependencies enable length extrapolation, provided the model fits the source length.
\cite{joshi2025theory} analyzed time-invariant autoregressive models (a fixed
next-token generator) and showed that the sample complexity can be independent of the CoT length but on the complexity of individual CoT steps. We provide a concrete setting that mirrors these works and, crucially,
show that gradient-based optimization \emph{actually finds} such solutions for transformer networks.

\subsection{Our Empirical Results at a Glance}
Empirically, we conduct a series of experiments aligned with our problem setup to test the
theory’s predictions, and observe strong agreements with our predicted behavior. In particular:
\begin{itemize}[leftmargin=1.5em]
  \item \textit{Length generalization comparison.}
  \Cref{fig:main-results-a} shows that training at a short length yields nearly perfect accuracy at much longer lengths for tasks with simply transitive actions, whereas
  tasks with symmetry actions exhibit only constant-factor length generalization. These findings are consistent with  \Cref{informal-thm-len-gen}.

  \item \textit{Effectiveness of recursive self-training.}
  In \Cref{fig:main-results-b}, a double-and-self-labeled curriculum for the
  task with symmetry actions consistently shifts the length–accuracy curve rightward; after
  several stages, performance matches the simply transitive case, validating
  \Cref{informal-self-training}.

  \item \textit{Attention concentration patterns.}
  At convergence, the attention heatmaps in \Cref{fig:attn-con} exhibit two clear
  diagonals with concentrated mass, confirming the attention-concentration mechanism
  predicted by our theory.
\end{itemize}
Detailed descriptions and discussions, along with additional experiments, are postponed to \Cref{sec:exp}.

\section{Background}\label{sec:problem}

Computational complexity \cite{arora2006computational} has been employed to characterize the power of neural networks \cite{godbeer1987computational,siegelmann1992computational}. 
Historically, \emph{circuit complexity} has been used extensively to study the power of neural networks \cite{shawe1992classes,parberry1994circuit,maass1996lower,vsima2003general,merrill2025exact,merrill2023parallelism,wang2025learning}, due to the structural resemblance of Boolean circuits and neural networks with threshold gates. In a nutshell, circuit complexity evaluates computation models by the size, depth, and gate types of Boolean circuits that implement them. Below we provide a brief introduction of the basics of circuit complexity.

\subsection{Circuits and Expressiveness}
A Boolean circuit is a finite acyclic network of logic gates that computes a
Boolean function on $\{0,1\}^n\to \{0,1\}$ for some fixed $n$. The gates with fan-in $0$ are the inputs, which are assigned one of the $n$ Boolean variables.  A \emph{circuit family} $\{C_n\}_{n\ge 1}$ computes a language by using $C_n$
on length-$n$ inputs. We measure the  \emph{size} of a circuit by the number of gates, and \emph{depth}  by the length of the
longest input-output path, respectively. For instance,  ``constant depth'' means $\operatorname{depth}(C_n)=O(1)$, while  ``polynomial size'' means $\operatorname{size}(C_n)=O(n^c)$ for some integer $c$. Standard circuit complexity classes are defined by restricting circuit depth, size, gate set and fan-in:
\begin{itemize}[itemsep=0pt,leftmargin=1.5em]
\item $\mathsf{TC}^0$ consists of \emph{constant-depth}, polynomial-size circuits with
unbounded fan-in $\{\mathsf{AND},\mathsf{OR},\mathsf{NOT}\}$ gates augmented with
threshold (e.g., $\mathsf{MAJORITY}$) gates. The circuit class $\mathsf{TC}^0$ captures exactly the complexity of integer multiplication and division, and sorting \cite{chandra1984constant}.
\item 
$\mathsf{NC}^1$ consists of \emph{log-depth}, polynomial-size circuits over
$\{\mathsf{AND},\mathsf{OR},\mathsf{NOT}\}$ with bounded fan-in. This class captures exactly the complexity of recognizing all regular languages.
\item $\mathsf{P}/\mathsf{poly}$ is the class of all languages computable by polynomial-size circuit families.
\end{itemize}
These classes satisfy the following relations:\footnote{
We intentionally omit uniformity conventions here as this section is only meant to
give readers a sense of the classes we refer to and their standard relationships.}
\[
\mathsf{TC}^0 \subseteq \mathsf{NC}^1 
\subseteq \mathsf{P}/\mathsf{poly}.
\]

\paragraph{Expressiveness gap of transformers via circuit complexity.} One can analyze the
expressivity of transformer neural nets through the lens of circuit complexity.  Prior work~\cite{merrill2022saturated,liu2023transformers, merrill2023parallelism, feng2023revealing,li2024chainthought} showed that the expressive upper bound of a vanilla constant-depth transformer is limited to $\mathsf{TC}^0$, the class of problems solvable by extremely shallow, highly parallel circuits.  In contrast, when equipped with CoT, recent work \cite{feng2023revealing, merrill2024expressive, li2024chainthought} proved that \(O(n)\) intermediate steps enable transformers to simulate $\mathsf{NC}^1$-complete language. Further, \cite{li2024chainthought} showed that $2$-layer transformers with polynomially many
CoT steps suffice to express arbitrary polynomial-size circuits, i.e., $\mathsf{P}/\mathsf{poly}$. Hence, under the widely believed conjecture that $\mathsf{TC}^0\neq \mathsf{NC}^1$ \cite{vollmer1999introduction}, CoT strictly extends the expressive power of transformers beyond $\mathsf{TC}^0$, allowing them to solve problems that inherently require super-constant (here \(O(\log n)\)) computation depth.  A standard example connecting algebraic structure with circuit complexity is the following \emph{word problem}, define by Dehn \cite{dehn1911unendliche, dehn1987infinite}.

\begin{definition}[Word problem for a group $G$]
Let $G$ be a group and let $e$ denote its identity.
For a word $w = g_1 \cdots g_k \in G^*$ (each $g_i \in G$), 
the word problem asks to decide whether $g_1 \circ \cdots \circ g_k=e$.

\end{definition}

\noindent Barrington~\cite{barrington1986bounded} proved that the above word problem is $\mathsf{NC}^1$-complete when the group is non-solvable.
\begin{theorem}[Barrington~\cite{barrington1986bounded}]
    The word problem of every finite non-solvable group is $\mathsf{NC}^1$-complete.  The canonical example is $S_n$ for $n\geq 5$, the symmetry group on $n$ elements that encodes the permutations. 
\end{theorem}

The word problem captures a fundamental reasoning task: state tracking in a finite world~\cite{merrill2024illusion,merrill2025little,Li2025HowDL}. Given an initial state and a sequence of transformations, the goal of the task is to compute the resulting state. State tracking underlies practical LLM abilities such as narrative entity tracking, chess move analysis, and code execution \cite{kim2023entity,merrill2024illusion}, while exhibiting a rich connection to the circuit complexity theory. Therefore, it has become a standard synthetic testbed for probing the reasoning abilities of language models, both theoretically~\cite{liu2023transformers, merrill2024illusion,merrill2025little} and empirically~\cite{zhang2022unveiling,liu2023transformers,Li2025HowDL}. Motivated by this connection, we study the state-tracking task with varying algebraic structures.

\subsection{LEGO for State Tracking}
We focus on a specific formulation of the state-tracking problem, LEGO (\emph{Learning Equality and Group Operations}) \cite{zhang2022unveiling}, which was originally proposed as a synthetic task to study the reasoning behavior of  transformers empirically. A typical LEGO instance in~\cite{zhang2022unveiling} takes the following form: 
\begin{equation*}
    \texttt{b = + a}, \quad \texttt{c = - b}, \quad \dots, \quad \texttt{t = - s},\quad \texttt{a = -1}, \quad \dots
\end{equation*}
Here, \texttt{a, b, c, \dots} are \textbf{variables}, each taking a \textbf{value} (or state) in $\{-1,+1\}$ in this example. Short expressions such as \texttt{b = + a} are \textbf{clauses}, where \texttt{= +} and \texttt{= -} denote \textbf{actions}: the action is applied to the \emph{right-hand-side} variable’s value to obtain the \emph{left-hand-side} variable’s value. For instance, from \texttt{b = + a} and \texttt{a = -1}, it follows that \texttt{b = -1}. Formally, the LEGO language is defined as follows:
 \begin{definition}[LEGO language~\cite{zhang2022unveiling}]\label{def:lego}
    Let $\cX,\cG,\cY$ be finite sets of variables, actions, and values, respectively, where each $g\in\cG$ is a map $g:\cY\to\cY$. The formal language $\lego(\cX,\cG,\cY)$ has alphabet $\cX\cup\cG\cup\cY\cup\{=,(,)\}$ and consists of two types of expressions (called \textbf{clauses}):
    \begin{enumerate}[label=(\arabic*), leftmargin=3em]
        \item \emph{Predicate clause} $x = g(x')$ specifies an action $g\in\cG$ linking variables $x,x'\in\cX$.
        \item \emph{Answer clause} $x = y$ assigns a value $y\in\cY$ to a variable $x\in\cX$.
    \end{enumerate}
            A canonical LEGO sentence of length $L$ with answer up to $L'$ concatenates predicate clauses $x_n = g_n(x_{n-1})$ for $n\in[L]$ and answer clauses $x_n = y_n$ for $n\in[L']$ with $L'\le L$:
    \begin{equation}\label{eqdef:lego-sentence-form}
        \underbrace{x_1 = g_1(x_0)\ \dots\dots\ x_L = g_L(x_{L-1})}_{\text{predicates}}\ \underbrace{x_0 = y_0\ \dots\dots\ x_{L'} = y_{L'}}_{\text{answers}},
    \end{equation}
    which describes the chain of transitions:
    \begin{equation}
        \underbrace{x_0 \xrightarrow{g_1} x_1 \xrightarrow{g_2} x_2 \xrightarrow{g_3} \cdots \xrightarrow{g_{L'}} x_{L'}}_{ \text{with answers $y_1,\dots,y_{L'}$ up to } L'} \cdots \xrightarrow{g_{L}} x_{L},\quad \text{starting with } x_0 = y_0. \notag
    \end{equation}
For semantic validity, any sentence containing a path $x_n = g_n(x_{n-1}), \dots, x_{n-k+1} = g_{n-k+1}(x_{n-k})$ for $k\in[n]$ must satisfy: 
$y_n \;=\; g_n\circ g_{n-1}\circ\cdots\circ g_{n-k+1}(y_{n-k}).$ 
    \end{definition}
    In the LEGO language, predicate clauses encode transformations, while answer clauses encode observed states. Thus, solving a LEGO sentence is exactly state tracking: compose the listed
actions along the path and propagate the observed states to predict the next
answer consistent with the composition.  
\section{Problem Formulation}\label{sec:setup}

As introduced in Section~\ref{sec:problem}, the current paper employs the LEGO framework to investigate the reasoning
capabilities of transformers. To set the stage, this section presents precise mathematical formulations of the problems to be studied in this paper. Before proceeding, we introduce the following notation: 
\begin{itemize}[itemsep=0pt,leftmargin=1.5em]
\item {\em Vocabulary.} Define the {vocabulary} as $\cV \coloneqq \cX \cup \cG \cup \cY \cup \{\blank\}$, where the \emph{blank} token $\blank$ is a null symbol indicating the absence of other tokens.
\item {\em Vocabulary size.} Let $d \coloneqq |\cV|$ denote the (finite) vocabulary size.
\end{itemize}
To facilitate theoretical analysis, we concentrate on the asymptotic regime where $d\to\infty$. Note that \(|\cX|\), \(|\cG|\), \(|\cY|\) may depend on
\(d\), and we shall specify any required scaling assumptions as needed.

\begin{assumption}[Asymptotic regime]\label{assump:asymptotic-regime}
    For a language $\lego(\cX,\cG,\cY)$ defined in \Cref{def:lego}, we consider the asymptotic regime where both the vocabulary size $d$ and the number of variables $|\cX|$ tend to infinity. Assume $|\cG| \le \log^{C_0} d$ for some constant $C_0 \in [1,100)$, and hence $\cY$ and $\cG$ are much smaller than $\cX$ in size.
\end{assumption}
\subsection{Data Distribution}
We begin by specifying how LEGO clauses are tokenized, followed by a definition of the LEGO distribution.
\begin{definition}[LEGO encoding]\label{def:lego-encoding}
Each LEGO clause from \Cref{def:lego} is encoded as a fixed-length, 5-token tuple
$Z \in \cV^5$. More specifically, 
\begin{itemize}[itemsep=0pt,leftmargin=1.5em]
    \item For each predicate clause $x = g(x')$, set
    $Z_{\pred} \coloneqq (x,\, g,\, x',\, \blank,\, \blank) \in \cV^5$;
    \item For each answer clause $x = y$, set
    $Z_{\ans} \coloneqq (\blank,\, \blank,\, \blank,\, x,\, y) \in \cV^5$.
\end{itemize}
\end{definition}

With \Cref{def:lego-encoding} in place, we can encode a LEGO sentence \eqref{eqdef:lego-sentence-form} to a sequence \(Z^{L, L'}\):
\begin{equation}\label{eqdef:lego-sequence-encoding}
    Z^{L, L'} = (Z_{\pred,1},\dots,Z_{\pred,L},Z_{\ans,0},\dots,Z_{\ans,L'}),
\end{equation}
where  \(Z_{\pred,k}\) and \(Z_{\ans,k}\) represent the $k$-th predicate and answer clauses, respectively. For convenience, we denote \(\cI^{L,L'} = \{(\pred,\ell)\}_{\ell\in[L]} \cup \{(\ans,\ell)\}^{L'}_{\ell=0}\) as the index set associated with the clauses in \(Z^{L,L'}\). If \(L = L'\), we  write the sequence simply as \(Z^{L}\) and the corresponding index set as \(\cI^L\).

To feed LEGO tokens into a neural network, we first map each symbol to an
integer index (\emph{tokenization}), and then map indices to continuous
vectors via a learned table (\emph{embedding}).
The following definitions formalize these two steps.

\begin{definition}[Tokenization and token embedding]\label{def:token-embedding}
    Each $v \in \cV$ is assigned a unique index
    $\tau(v) \in \{1,\ldots,d\}$.
    Denote by $e_i \in \mathbb{R}^d$ the embedding vector associated with
    index $i$, and write $e_v \equiv e_{\tau(v)}$ for convenience.
    The blank token is assigned the zero vector,
    $e_{\tau(\blank)} = \mathbf{0}_d \in \mathbb{R}^d$.
    For technical simplicity, we assume that
    $\{\,e_v : v \in \cV \setminus \{\blank\}\,\}$ forms an orthonormal set
    in $\mathbb{R}^d$ (note that this assumption can be relaxed to a well-conditioned
    embedding matrix without affecting our results).
    \end{definition}

Equipped with \Cref{def:token-embedding}, we can transform the LEGO sequence encoding into vector embeddings that can be used as inputs to neural networks.

\begin{definition}[Embedding of LEGO sentences]\label{def:clause-representation}
    Let \(d_c := 5d\) be the clause embedding dimension. We define an operator \(\mathsf{Embed}: \cV^5 \to \RR^{d_c}\) that maps a clause to embedding by 
    \begin{align*}
        \Zb = \mathsf{Embed}(Z) \coloneqq (e_{v_1}, e_{v_2},\dots, e_{v_5}) \in \RR^{d_c},\quad \textrm{for clause } Z = (v_1, v_2,\dots, v_5) \in \cV^5.
    \end{align*}
    Specifically, a LEGO sentence $Z^{L,L'}$ defined in \eqref{eqdef:lego-sequence-encoding} is embedded as
\begin{equation*}
    \Zb^{L,L'} = (\Zb_{\pred,1},\dots,\Zb_{\pred,L},\Zb_{\ans,0},\dots,\Zb_{\ans,L'}) \in \RR^{d_c \times (L + L' + 1)}
\end{equation*} 
where each column $\Zb_{\pred,\ell}\in\RR^{d_c}$ (resp.~$\Zb_{\ans,\ell}$) is the embedding of clause $Z_{\pred,\ell}$ (resp.~$Z_{\ans,\ell}$). When \(L' = L\), we simply write \(\Zb^L \equiv \Zb^{L,L'}\) for simplicity.
\end{definition}

Next, we describe the distribution that governs the generation of a LEGO sentence.
\begin{assumption}[LEGO distribution \( \cD^L, \cD^{L,L'}\)]\label{assump:lego-data-distribution}
    Consider \(\lego(\cX,\cG,\cY)\) as defined in \Cref{def:lego}, and let \(L\) denote the sequence length. We assume that the distribution \(\cD^{L}\) of length-$L$ LEGO sentences satisfies the following properties.
    \begin{enumerate}[itemindent=0pt,leftmargin=1.5em]
        \item All LEGO sentences \(Z^{L} \sim \cD^{L}\) of the form \eqref{eqdef:lego-sentence-form} are encoded by \Cref{def:lego-encoding} into representation \eqref{eqdef:lego-sequence-encoding}.
        \item The variables \(x_0, x_1, \dots, x_L \in \cX\) are sampled uniformly at random from \(\cX\) without replacement.
        \item The first value \(y_0 \in \cY\) is chosen uniformly at random from \(\cY\).
        \item The actions \(g_1, g_2,\dots, g_L \in \cG\) are sampled uniformly at random from \(\cG\) with replacement.
        \item The intermediate values \(y_1,y_2,\dots, y_{L}\) are computed recursively by \(y_{i} = g_i(y_{i-1})\).
    \end{enumerate}
    For any \(L' < L\), we define the truncated distribution \(\cD^{L,L'}\) of sequences \(Z^{L,L'}\)  containing all the predicates and the first \(L'+1\) many answer clauses, where \(Z^{L,L'}\) is obtained by first sampling \(Z^{L}\sim\cD^L\) and then removing the answer clauses \(Z_{\ans,\ell}\) for all \( \ell > L'\).
\end{assumption}

As can be easily seen, sequences sampled from \(\cD^L\) or \(\cD^{L,L'}\) correspond to valid LEGO sentences as defined in \Cref{def:lego}. With a slight abuse of notation, we write \(\Zb^{L,L'} \sim \cD^{L,L'} \)  to indicate that  \(\Zb^{L,L'}\) is the embedding of a sentence \(Z^{L,L'}\) drawn from \(\cD^{L,L'}\).

\subsection{Transformer Architecture}\label{sec:network-architecture}

In this subsection, we introduce the transformer architecture investigated in this paper. Towards this end, 
we first introduce a smoothed activation function that will be used in
our network.
\newcommand{\srelu}{\mathbf{sReLU}}
\begin{definition}[Smooth ReLU]\label{def:smooth-relu}
    Define a continuously differentiable variant of the ReLU activation function~\cite{allen2020towards,huang2022modality} as follows:
    \[
      \srelu(x)
      \;:=\;
      \begin{cases}
        \displaystyle \frac{\varrho}{q},
          & x \le -\varrho,\\[6pt]
        \displaystyle \frac{x^{q}}{\varrho^{\,q-1} q},
          & x \in (-\varrho,\varrho],\\[8pt]
        x - \varrho\!\left(1 - \frac{1}{q}\right),
          & x > \varrho,
      \end{cases}
    \]
    where $q$ and $\varrho$ are some design parameters. Here and throughout, we choose $q=O(1)$ to be a large, even integer,  and take 
$\varrho=\Theta\!\bigl(1/\polylog(d)\bigr)$.
    \end{definition}

\paragraph{Transformer architecture.} 
In this work, we focus on an autoregressive transformer~\cite{Vaswani2017AttentionIA} whose block
consists of a softmax attention layer followed by a position-wise feed-forward network, as described below. 

\begin{itemize}[itemsep=0pt,leftmargin=1.5em]
\item 
{\em Attention layer.}
Given LEGO sentence embeddings \(\Zb^{L,L'}\) and indices \(\bj,\bk \in \cI^{L,L'}\), the attention from clause \(\Zb_{\bj}\) to clause \(\Zb_{\bk}\) is defined through the softmax operator as
\[
\attn_{\bj \to \bk}(\Qb,\Zb^{L,L'}) \;\coloneqq \;
\frac{\exp(\Zb_{\bj}^{\top}\Qb\,\Zb_{\bk})}
{\sum\nolimits_{\br \in \cI^{L,L'}} \exp(\Zb_{\bj}^{\top}\Qb\,\Zb_{\br})}.
\]
Since the model is autoregressive, a standard causal mask is applied to ensure that the latest (answer) token attends only to \emph{preceding} tokens (including itself). The attention output is
\[
\mathrm{Attention}(\Qb,\Zb^{L,L'})
\;\coloneqq\;
\sum\nolimits_{\bk \in \cI^{L,L'}} \attn_{\ans,L' \to \bk}(\Qb,\Zb^{L,L'}) \cdot \Zb_{\bk}.
\]
Note that in the standard formulation of transformers, the score takes the form
\(\Zb_{\bj}^{\top}\Wb^{Q\top}\Wb^{K}\Zb_{\bk}\) instead for a ``query'' parameter matrix $\Wb^Q$ and a ``key'' parameter matrix $\Wb^K$.
Here, we fold \(\Wb^{Q\top}\Wb^{K}\) into a single  parameter matrix \(\Qb\), resulting in a score
\(\Zb_{\bj}^{\top}\Qb\Zb_{\bk}\), which can be viewed as an equivalent
reparameterization that simplifies analysis without changing expressivity~\cite{huang2023context,huang2025a,yang2025multi}.

\item {\em Feed-forward network (FFN).}
The FFN with parameter \(\Wb \in \RR^{5\times d\times m \times d_c}\) is defined by
\[
\ffn_{i,j}(\Wb,\Xb)
\;\coloneqq\;
\sum\nolimits_{r\in[m]} \mathbf{sReLU}\!\Big(\!\dbrack{\Wb_{i,j,r}, \Xb} + b_{i,j,r}\Big),
\quad \forall\, i\in[5],\, j\in[d],
\]
where \(\Xb\) denotes the input,  \(\Wb_{i,j,r}\in\RR^{d_c}\) are neuron weights, \(m\) indicates the number of neurons, and $b_{i,j,r}\in \mathbb{R}$ is some {\it fixed} bias.
\end{itemize}

\noindent With the above layers defined, we are ready to introduce the transformer model used in this work, which is summarized as follows. 
\begin{definition}[Transformer language model]\label{def:transformer-arch}
    Assume that our learner neural network \(F\) is a one-layer decoder transformer block composed of an attention layer with no positional encoding (NoPE) as well as an FFN layer: \(F = \ffn\circ\mathrm{Attention}\), with parameters \(\Wb \in \RR^{5\times d\times m \times d_c}\) and \( \Qb\in\RR^{d_c\times d_c}\). Formally, for the \(i\)-th token position and the \(j\)-th vocabulary index, we have
    \begin{align}
    \Big[F_{i}\big(\Zb^{L,L'}\big)\Big]_j \coloneqq \ffn_{i,j}\Big(\Wb,\mathrm{Attention}\big(\Qb,\Zb^{L,L'}\big)\Big) \in \RR, \quad \forall i\in[5], j\in[d]. \label{eq-def-F-main-1}
    \end{align}
    We interpret \(F(\Zb^\ell)\) as five logit vectors
\(\{F_{i}(\Zb^\ell)\}_{i=1}^{5}\subset\RR^{d}\),
each parameterizing the distribution of the \(i\)-th token of the next clause.
Recall that \(\cV\) is the vocabulary with \(|\cV|=d\) and
\(\tau:\cV\to[d]\) denotes the index map. Given an encoded LEGO sequence \(Z^\ell=(Z_1,\ldots,Z_\ell)\) with embedding
\(\Zb^\ell\), the model’s predictive distribution for the \(i\)-th token of the
\((\ell+1)\)-th clause is the following softmax:
    \begin{equation}\label{eqdef:p_F-distribution}
        p_{F_i}\big(Z_{\ell+1, i} = v \mid Z_1,\dots,Z_\ell\big) = \frac{e^{[F_{i}(\Zb^\ell)]_{\tau(v)}}}{\sum_{j \in [d]} e^{[F_{i}(\Zb^\ell)]_{j}}},\quad \forall v \in \cV.
    \end{equation}
    Now we can sample the next clause \(Z_{\ell+1}\) by sampling from the product distribution 
    $$Z_{\ell+1} = (Z_{\ell+1,1},\dots,Z_{\ell+1,5}) \sim \bigotimes_{i=1}^5 p_{F_i} \eqqcolon p_F,$$ in an autoregressive manner.
\end{definition}

\subsection{LEGO Task via CoT Reasoning and Training Objective}

With the LEGO distribution and the transformer model in place, we now formalize the task within the LEGO framework.
We view solving a length-$L$ LEGO problem as generating a sequence of CoT steps, 
where each step $\ell$ produces an intermediate result used to predict $x_{\ell} = y_{\ell}$, ultimately leading to the final solution $x_{L} = y_{L}$. 
To be precise, we define the following reasoning task.

\begin{definition}[Reasoning tasks $\cT^L$]\label{def:reasoning-task}
We define a family $\{\cT^{L}\}_{L \in \NN^+}$ that captures the ability to solve
sequential reasoning problems.
For each $L \ge 1$, task $\cT^{L}$ measures the model's \emph{accuracy along the chain}
from step $1$ to step $L$:
\begin{equation}\label{eq:acc-L-def}
\mathrm{Acc}_L(F)
= \frac{1}{L}\sum_{0\le L' < L}
\E_{Z^{L}\sim \cD^{L}}
\Big[\E_{ \hat{Z}_{\ans,L'+1} \sim p_F(\cdot \mid Z^{L,L'}) }
\big[\,\1\{\hat{Z}_{\ans,L'+1} = Z_{\ans,L'+1}\}\,\big]\Big],
\end{equation}
where $p_F$ is induced by the model $F$ from \eqref{eqdef:p_F-distribution}.  

\end{definition}

 Clearly, $\Acc_L(F)\in[0,1]$. We say that task $\cT^{L}$ is solved if $\Acc_L(F)\approx 1$. As $L$ grows, $\{\cT^{L}\}_{L \in \NN^+}$ poses increasingly difficult state tracking challenges.  At step $L'$, the model conditions on the partial transcript $Z^{L,L'}$ and
predicts the next answer $Z_{\ans,L'+1}$. To enforce step-by-step generation of the CoT trace, we define the following
training objective.

\begin{definition}[Next clause loss]\label{def:next-clause-loss}
The training objective for $\cT^{L}$ with $L \ge 1$ is the \emph{next clause} loss
\begin{subequations}\label{eq:next-clause-loss}
    \begin{align}
        \Loss^{L}(F) &\coloneqq \sum_{1 \le L' \le L} \Loss^{L,L'}(F), \label{eq:next-clause-loss-def}\\
        \text{where}\quad
        \Loss^{L,L'}(F) &\coloneqq
        \E_{Z^{L,L'}\sim \cD^{L,L'}}\big[ -\log p_{F}(Z_{\ans,L'} \mid Z^{L,L'-1}) \big].
        \end{align}
    \end{subequations}
    Here, $p_F$ is induced by the model $F$ from \eqref{eqdef:p_F-distribution}. 
We also define the per-token loss
$$\Loss^{L,L'}_i(F) \coloneqq
\E_{Z^{L,L'}\sim \cD^{L,L'}} \big[-\log p_{F_i}(Z_{\ans,L',i} \mid Z^{L,L'-1}) \big].$$
\end{definition}
\noindent This is a teacher forcing style CoT objective: the model is given the ground truth answers for all previous steps and is trained to match the next step's answer~\cite{ho2022large, kim2025transformersprovably}.

Further, we adopt the following initialization for training.

\begin{assumption}[Initialization]\label{assump:init}
    Let \(F\) be the transformer network in \Cref{def:transformer-arch} with parameters \(\Wb\) and \(\Qb\). The attention parameter is initialized to zero:
    \(\Qb^{(0)}=\mathbf{0}_{d_c\times d_c}\).
    The FFN weights are initialized randomly and independently as
    \(\Wb_{i,j,r}^{(0)} \sim \mathcal{N}(\mathbf{0},\,\sigma_0^{2}\mathbf{I}_d)\)
    with \(\sigma_0 = d^{-1/2}\).
    The biases are not trained and instead fixed at
    \(b_{i,j,r} = \sigma_0 \log d\) for all \(i,j,r\), chosen to keep most sReLU
    units active at initialization.
    All random draws are independent across indices. 
\end{assumption}

\paragraph{Additional Technical Assumptions.} Additionally, we introduce a couple of technical assumptions to be used throughout the proof. To control rare large deviations in the logits during training-time analysis,
  we adopt a bounded-output assumption as stated below.

  \begin{assumption}[Logit clipping]\label{assumption:output-bound}
  There exists $B = C_B \log d$, for some sufficiently large constant $C_B>0$, such
  that each coordinate of the raw model output $F_i$ is clipped from above:
  \[
    [F_i]_j \;\leftarrow\; \min\{[F_i]_j,\, B\}
    \qquad \text{for all } i,j .
  \]
  \end{assumption}
\noindent This coordinate-wise clipping is a technical device to control large-deviation
  tails and simplify the analysis of the dynamics; $B$ can be chosen large enough to
  avoid interfering with the regimes we study.
  

Moreover, to simplify the analysis of attention dynamics, we impose a fixed block-sparsity
pattern on the attention parameter \(\Qb\).
  \begin{assumption}[Block-sparse attention matrix]\label{assump-Q-structure}
    Let \(\Qb = [\Qb_{p,q}]_{p,q\in[5]} \in \RR^{5d\times 5d}\) be partitioned into
    \(5\times 5\) blocks with \(\Qb_{p,q}\in\RR^{d\times d}\).
    We assume that
    \[
    \Qb_{p,q} \equiv \mathbf{0}_{d\times d}
    \quad \text{for all } (p,q)\notin\{(4,3),(4,4)\},
    \]
    i.e., only the blocks \((4,3)\) and \((4,4)\) are trainable.  
    \end{assumption}
\noindent This block-sparsity pattern, which zeros out most inter-token attention, is standard in
    recent theoretical analyses of transformer training dynamics~\cite{huang2023context,huang2025a,yang2025multi,yang2024context,cheng2025transformers}.
    Importantly, although sparse at the \(5\times 5\) token level, the two retained
    blocks \((4,3)\) and \((4,4)\) are fully \emph{dense} \(d\times d\) matrices
    trained without constraints, leaving \(2d^{2}\) free parameters and thus a
    substantive, non-trivial learning problem.

\section{Learning CoT on Simply Transitive Actions}\label{sec:cyclic}

In this section, we present our first main results for the case where the group action \(\cG\) on \(\cY\) is simply transitive, which is isomorphic to the action of the cyclic group \(C_n\) on \(\ZZ_n\).

\begin{algorithm}[t]
    \caption{Curriculum training for simply transitive actions}
    \label{alg:cot-transitive-training}
    \KwIn{Model $F^{(0)}$ with parameters $(\Wb^{(0)}, \Qb^{(0)})$; Learning rate $\eta$; Stage snapshots \(T_1\), \(T_2\).
    }
    
    \BlankLine
    \textbf{Stage $1$:} Learning one-step reasoning (\(\mathcal{T}^{1}\)) \;
    \Indp
    \For(\tcp*[f]{\small Update the FFN parameter $\Wb$}){$t = 1$ \KwTo $T_1$}{
       $\Wb^{(t)}\gets \Wb^{(t-1)}-\eta \nabla_{\Wb}\Loss^{1}(F^{(t-1)}) $ \;
       $\Qb^{(t)}\equiv \Qb^{(t-1)}$\;
    }
    \Indm
    
    \BlankLine
    \textbf{Stage $2$:} Learning two-step reasoning for length extension (\(\mathcal{T}^{2}\)) \;
    \Indp
    \For(\tcp*[f]{  \small Update the attention parameter $\Qb$}){$t = T_1+1$ \KwTo $T_2$}{
      $\Qb^{(t)}\gets \Qb^{(t-1)}-\eta \nabla_{\Qb}\sum_{\ell=1}^2\Loss^{2,\ell}_{5}(F^{(t-1)})$ \;
      $\Wb^{(t)}\equiv \Wb^{(t-1)}$\;
    }
    \Indm
    \KwOut{ Model $F^{(T_1+T_2)}$.}
    \BlankLine
    \end{algorithm}

\begin{assumption}[Simply transitive group action]\label{assump:structure-1}
    Let \(\cY = \{0,1,\dots,n_y-1\}\) with
\(n_y \in [\,\Omega(\log\log d),\) \( \log d\,]\).
Assume the group \(\cG\) acts simply transitively on \(\cY\): the action
is transitive and for any \(y_1,y_2\in\cY\) there exists a unique
\(g\in\cG\) such that \(g\cdot y_1 = y_2\).
Equivalently, for any fixed \(y\in\cY\) the map \(g\mapsto g\cdot y\) is a
bijection from \(\cG\) to \(\cY\), so \(|\cG|=|\cY|=n_y\).
\end{assumption}

Our first main result demonstrates that, for such a simple task the model obtained via Algorithm~\ref{alg:cot-transitive-training} for short-chain tasks $\cT^1$ and $\cT^2$, successfully generalizes to significantly longer tasks. 

\begin{theorem}\label{thm:length-generalization}
Under Assumptions~\ref{assump:asymptotic-regime}, \ref{assump:lego-data-distribution}, \ref{assump:init},  \ref{assumption:output-bound},  \ref{assump-Q-structure} and \ref{assump:structure-1}, for some constant $0<c^{*}<1$, $n_y<m\ll \log^2 d$, the transformer model \(F^{(T_1 + T_2)}\) obtained by Algorithm~\ref{alg:cot-transitive-training} with learning rate \(\eta = \frac{1}{\poly(d)}\),  and stage 1 and 2 iteration $T_1=\tilde{O}\Big(\frac{d}{\eta (\sigma_0)^{q-2}}\Big)$, $T_2=\tilde{O}\Big(\frac{\poly (d)}{\eta \sigma_0 }\Big)$ satisfies 
\begin{enumerate}[itemindent=0pt,leftmargin=2em]
    \item {\bf Direct short-to-long length generalization:}
    \begin{align}\label{eq:direct-length-generalization}
        \mathrm{Acc}_L\!\Bigl(F^{(T_1 + T_2)}\Bigr) \;\geq\; 1 - \frac{1}{\poly(d)}, \text{ for every \(L \leq O(d^{c^{*}})\), }
    \end{align}

i.e., \(F^{(T_1 + T_2)}\), which is trained for task \(\cT^1\) and \(\cT^2\), generalizes to solve the tasks \(\cT^{\ell}, \ell \leq L\). 
\item {\bf Attention concentration:}\footnote{For readability, we abbreviate \(\attn_{\kk\to\kk'}(\Qb^{(t)},\Zb)\) by \(\attn^{(t)}_{\kk\to\kk'}\), omitting explicit dependence when immaterial.}
given $Z^{2,\ell}$ with $\ell\in\{0,1\}$, we have
\begin{align}\label{eq:attention-concentration-cyc}
    \attn^{{(T_1 + T_2)}}_{\ans,\ell \to \pred,\ell+1}+\attn^{{(T_1 + T_2)}}_{\ans,\ell \to \ans,\ell} \geq 1 - O\Big(\frac{1}{d^{c^{*}}}\Big).
    \end{align}
\end{enumerate}
\end{theorem}

\begin{figure}[t]
    \centering
    \vspace{2em}
\begin{align*}
  &\text{\bf  Retrieve (Attention): }   \qquad {\mboxed{{x_1} = g_1(x_0)}},\ \tikzmarknode{first}{
     \mboxed{ x_2 = \textcolor{RedOrange}{g_2}(x_{1})}}, 
    \quad   
      {\mboxed{x_0 = y_0}},\
      \tikzmarknode{last} {\underbrace{\mboxed{{{x_{1}}} = 
      \textcolor{RedOrange}{y_1}
      }}_{\textcolor{NavyBlue}{\text{query}}}}
       ,\
      \mboxed{x_{2} = {\underline{?}}}\\
      &\text{\bf  Apply group action (FFN): }  \qquad \qquad \qquad \qquad y_2 = g_2\cdot y_1
  \end{align*}

\begin{tikzpicture}[remember picture,overlay,line cap=round,line join=round]
\draw[->,thick,shorten >=2pt,shorten <=2pt]
(last.north)
.. controls ($(last.north)+(0,\vpad)$)
       and ($(first.north)+(0,\vpad)$)
.. node[midway, below, sloped, inner sep=1pt]
{{\scriptsize $\attn_{\ans,1 \to \pred,2}$}}
(first.north);

\draw[->,thick,shorten >=2pt,shorten <=2pt]
(last.north)
.. controls ($(last.north)+(\hpad,0)$)
       and ($(last.north)+(\hpad,\vpad)$)
.. ($(last.north)+(0,\vpad)$)
.. controls ($(last.north)+(-\hpad,\vpad)$)
       and ($(last.north)+(-\hpad,0)$)
       .. node[pos=0, above, sloped,  inner  sep=1pt]
       { {\scriptsize $\attn_{\ans,1 \to \ans,1}$}}
        (last.north);
\end{tikzpicture}
\caption{Illustration of how the model solves the LEGO task: given $Z^{2,1}$, the goal is to predict  $y_2$.}
\label{fig:toy-example}

\end{figure}
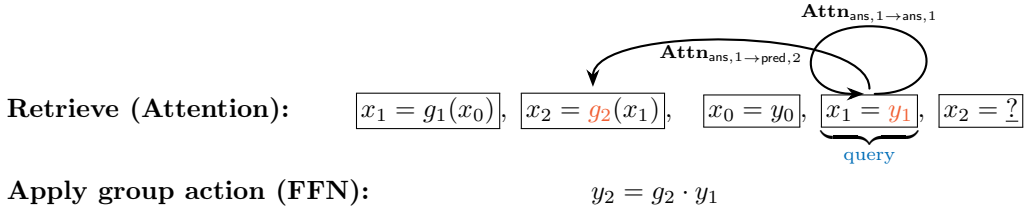

\paragraph{{Mechanism of CoT for state-tracking.}}
Given the current sequence \(Z^{L,\ell}\) with intermediate steps up $\ell$, predicting the next state \(y_{\ell+1}=g_{\ell+1}(y_\ell)\) requires two steps:
\begin{itemize}
    \item[(i)] \textbf{Retrieve} the correct action \(g_{\ell+1}\) from the context clause \(Z_{\pred,\ell+1}\) and the current state \(y_\ell\) from the answer clause \(Z_{\ans,\ell}\); 
    \item[(ii)] \textbf{Apply the group action}, that is, compute the next state \(y_{\ell+1} =  g_{\ell+1}(y_\ell)\).
\end{itemize}
It is well established that attention can implement content-based retrieval~\cite{elhage2021mathematical}, and that FFN can represent the
group operation~\cite{lu2017expressive} (see \Cref{fig:toy-example} for an example). Algorithm~\ref{alg:cot-transitive-training} decouples learning the
attention (retrieval) and FFN (action) components to simplify the analysis. For task \(\cT^1\), the transcript \(Z^{1,0}\) contains only the two relevant
clauses, \(Z_{\pred,1}\) and \(Z_{\ans,0}\), without useless contents. Fixed uniform attention ($\Qb$ initialized to be zero in \Cref{assump:init}) therefore suffices to expose both clauses to the FFN, and we optimize the FFN to learn the group operation. The high accuracy for \(\cT^1\) in \Cref{thm:length-generalization} indicates that the FFN has indeed \emph{learned} to apply the operation correctly. For task \(\cT^2\), with the FFN already trained, the attention layer only needs to learn to route the correct context to the FFN input. The attention concentration result in \eqref{eq:attention-concentration-cyc} confirms that the learned routing pattern is correct. \Cref{fig:attn-con,fig:attn-con-perm} also demonstrate the attention concentration patterns empirically.

\paragraph{How does attention concentration induce strong length
generalization?} As we increase the chain length in \(\cT^{\ell}\) for \(\ell>2\), the FFN layer
remains largely insensitive to input length since the learned group action is location-invariant.
By contrast, the attention layer is affected: more \emph{irrelevant} clauses
appear, so retrieval must scan over longer contexts, which risks diluting attention on
the relevant clause. 
\Cref{thm:length-generalization} guarantees that training on short chains
already yields attention concentration with error \(O(d^{-c^{*}})\). This
“purity” allows the model to tolerate dilution and maintain high attention on
the relevant clauses for chain lengths up to \(O(d^{c^{*}})\). Technically, this concentration arises because the query matrix \(\Q\)
learns to locate the same variable \(x_{\ell}\) that appears simultaneously:
the third token of the context clause \(Z_{\pred,\ell+1}\) and the fourth token
of the answer clause \(Z_{\ans,\ell}\). This co-occurrence furnishes a strong,
consistent signal that enables robust retrieval across longer chains, which will be elaborated in the proof overview in \Cref{sec:overview-attention}.

\section{Learning CoT on Symmetry Group Actions on $\mathbb{Z}_n$}\label{sec:symmetry}

We now turn to the case where the action group \(\cG\) is isomorphic to the symmetry group, under \Cref{assump:structure-2}. In this case, the problem is  $\mathsf{NC}^1$-complete for $n_y\geq 5$.

\begin{assumption}[Symmetry group actions]\label{assump:structure-2}
    Let \(\cY = \{0,1,\dots,n_y-1\}\).
We set \(\cG = \textbf{Sym}(\cY)\), the symmetry group of all permutations of
\(\cY\), so that \(|\cG| = n_y!\).
We let \(\cG\) act on \(\cY\) in the natural way and write \(g\cdot y\) (or
\(g(y)\)) for the image of \(y\in\cY\) under \(g\in\cG\).
For example, when \(\cY=\{0,1,2\}\) and \(g\) swaps \(0\) and \(1\), then
\(g(0)=1\), \(g(1)=0\), and \(g(2)=2\).
We assume \(n_y = \Theta\!\big(\tfrac{\log\log d}{\log\log\log d}\big)\), and
hence \(|\cG| = n_y! = \polylog d\).

\end{assumption}

\begin{algorithm}[t]
\caption{Recursive self-training for symmetry actions}
\label{alg:cot-symmetry-training}
\KwIn{Model $F^{(0)}$ with parameters $(\Wb^{(0)}, \Qb^{(0)})$; Learning rate $\eta$; 
Error degree $\mathsf{E}_1>0$ (constant) ; 
$\tau_1,\tau_2=\tilde{O}(\frac{\poly d}{\eta})$; Total Stage $K$.}

\BlankLine
\textbf{Stage $1.1$:} Train FFN for one-step reasoning (\(\mathcal{T}^{1}\)) \;
\Indp
\For(\tcp*[f]{\small Update the FFN parameter $\Wb$}){$t = 1$ \KwTo $\tau_1$}{
   $\Wb^{(t)}\gets \Wb^{(t-1)}-\eta \nabla_{\Wb}\Loss^{1}(F^{(t-1)}) $ \;
   $\Qb^{(t)}\equiv \Qb^{(t-1)}$\;
}
\Indm

\BlankLine
\textbf{Stage $1.2$:} Train attention for length extension (\(\mathcal{T}^{2}\)) \;
\Indp
\For(\tcp*[f]{  \small Update the attention parameter $\Qb$ }){$t = \tau_1+1$ \KwTo $\tau_1+\tau_2$}{
  $\Qb^{(t)}\gets \Qb^{(t-1)}-\eta \nabla_{\Qb}\Loss^{2,2}_{5}(F^{(t-1)})$ \;
  $\Wb^{(t)}\equiv \Wb^{(t-1)}$\;
}
${T}_{1}\gets t$\;

\Indm
\BlankLine
\textbf{Till Stage $K$:} Recursive self-train for length extension \;
\Indp
\For(\tcp*[f]{{\small Stage $k$ to solve \(\mathcal{T}^{2^k}\)}}){$k = 2$ \KwTo $K$}{
    $L \gets 2^{k}$, $\tilde{F}^{(k)} \gets F^{(T_{k-1})} $\;
    \While(\tcp*[f]{  \small Update the attention parameter $\Qb$ }){$\Loss^{L,2}_{\tilde{F}^{(k)}, 5}\bigl(F^{(t-1)}\bigr)>\frac{1}{d^{\mathsf{E}_1}}$}{
        $t\gets t+1$\;
        $\Qb^{(t)}
           \;\gets\;
           \Qb^{(t-1)}
           -\eta\,
           \nabla_{\Qb}\,
           \Loss^{L,2}_{\tilde{F}^{(k)}, 5}\bigl(F^{(t-1)}\bigr)$\;
        $\Wb^{(t)} \equiv \Wb^{(t-1)}$\;
    }
     ${T}_{k}\gets t$\;
}
\Indm
\KwOut{ Models $\{F^{({T}_k)}\}_{k=1}^{K}$.}
\BlankLine

\end{algorithm}

\paragraph{Challenges of learning CoT for the symmetry task.}  When the model tries to retrieve the group element \(g_{\ell+1}\) from the
correct context clause \(Z_{\pred,\ell}\), there can be other context clauses
whose group elements also send \(y_{\ell}\) to \(y_{\ell+1}\); we call these
\emph{distractor} clauses.
In the symmetry case on \(\cY\), each pair
\((i,j)\) admits \((n_y-1)!\) elements mapping \(i\) to \(j\), so the fraction
of distractors is substantial.
By contrast, in the simply transitive setting each pair has a unique element,
so distractors are unlikely and can be ignored. Attending to distractors still produces the correct next answer, so training
may converge with \textbf{insufficient attention concentration}.
This weaker concentration makes the attention layer less robust to dilution in
longer contexts.  
Hence, for this harder setting, directly proving \(d^{\Omega(1)}\) length CoT
generalization from constant-length training is difficult.

\paragraph{Self-improvement for reasoning length extension.} Recent empirical studies~\cite{singh2024human, gulcehre2023reinforced, lee2025selfimproving} show that \emph{length generalization} can be bootstrapped via model \textit{self-improvement}: models training on their own output can bootstrap their capability to solve longer problems. In particular, the work \cite{lee2025selfimproving} motivates a \textbf{recursive self-training} scheme for the symmetry task. To perform recursive self-training, we adopt the greedy language model as data annotator: the greedy language model \(\widehat{p}_F\) induced by the network \(F\) is defined by
\begin{equation}
    \widehat{p}_{F}(Z_{\ans,L'+1}|Z^{L,L'}) = \begin{cases}
        1, & \text{if } Z_{\ans,L'+1} = \argmax_{Z} p_F(Z|Z^{L,L'}), \\
        0, & \text{otherwise.}
    \end{cases}
\end{equation}

\noindent Now we can define the self-annotated LEGO data distribution:
\begin{definition}[Bootstrapped LEGO distribution]\label{def:bootstrap-lego-distribution}
    We define \(\cD_{F}^{L,L'}\) as the LEGO distribution in \Cref{assump:lego-data-distribution} except that the answers \(Z_{\ans, \ell}, 1 \leq \ell \leq L'\) is given recursively by sampling the prediction \(Z_{\ans, \ell} \sim \widehat{p}_{F}(\cdot|Z^{L,\ell-1}), 1\leq \ell\leq L'\) from the greedy language model \(\wh{p}_F\).
\end{definition}

\begin{definition}[Self-training loss]\label{def:self-training-loss}
Given a (fixed) model $\tilde{F}$ and length $L$, The self-training next-clause-prediction loss is defined by replacing \(\cD^{L,L'}\) with \(\cD_{F}^{L,L'}\)  (\Cref{def:bootstrap-lego-distribution}) in \eqref{eq:next-clause-loss}:
\begin{subequations}\label{eq-self-loss}  
    \begin{align}
        \Loss^{L,L'}_{\tilde{F}}(F)
        \triangleq \E_{Z^{L,L'}\sim \cD_{\tilde{F}}^{L,L'}}\left[ -\log p_{F}(Z_{\ans,L',i} \mid Z^{L,L'-1}) \right],\\
      \Loss^{L,L'}_{\tilde{F}, i} = \E_{Z^{L,L'}\sim \cD_{\tilde{F}}^{L,L'}}[-\log p_{F_i}(Z_{\ans,L',i} \mid Z^{L,L'-1})]  \quad \text{ for } i\in[5].   
        \end{align}
        \end{subequations}

\end{definition}

We now present our main results, establishing that a recursive self-training scheme can provably bootstrap the reasoning length for the symmetry LEGO task.

\begin{theorem}\label{thm:length-gen-self-training}
Assume the distribution \(\cD^{L}\) induced from \(\lego(\cX,\cG,\cY)\) satisfies \Cref{assump:asymptotic-regime}, \ref{assump:lego-data-distribution} and \ref{assump:structure-2}, and assume the transformer network satisfies Assumption~\ref{assump:init}, \ref{assumption:output-bound}, \ref{assump-Q-structure}, and $n_y<m\ll \log^2 d$. Then for any $1 \leq k < \log_2|\cX|$, the transformer $F^{(T_{k})} $ trained via Algorithm~\ref{alg:cot-symmetry-training} up to length \(L_{k} = 2^{k}\) and \(T_k = O(\frac{\poly(d)}{\eta})\) satisfies:
\begin{enumerate}[itemindent=0pt,leftmargin=2em]
    \item {\bf Constant-factor length generalization:} \(F^{(T_k)}\) is able to solve \(\cT^{L_{k+1}}\) with \({L_{k+1}}=2^{k+1}\)
    \begin{align}\label{eq:constant-factor-length-generalization}
        \mathrm{Acc}_{L_{k+1}}\Bigl(F^{(T_{k})}\Bigr) = 1 - \frac{1}{\poly(d)}.
    \end{align}

\item {\bf Attention concentration:} given $Z^{L_{k},\ell}$ with $\ell\in\{0,\dots,L_{k}-1\}$, we have
\begin{align}\label{eq:attention-concentration-self}
    \attn^{{(T_k)}}_{\ans,\ell \to \pred,\ell+1}+\attn^{{(T_k)}}_{\ans,\ell \to \ans,\ell} \geq 1 - \tilde{c},
    \end{align}
where $\tilde{c}$ is some sufficiently small constant (smaller than $0.01$).
\end{enumerate}

\end{theorem}

    At convergence for the current task \(\cT^{L_k}\) (at time
    \(T_k\), the loss
    has fallen below \({1}/{d^{\mathsf{E}_1}}\) in Algorithm~\ref{alg:cot-symmetry-training}), %
    \eqref{eq:attention-concentration-self} confirms that attention concentration
    is still insufficient. Nevertheless,
    \Cref{thm:length-gen-self-training} shows that while this level of
    concentration cannot withstand the dilution from much longer contexts, it is
    sufficient for doubling the length. Consequently, a model trained
    progressively on \(\cT^{L}\) for \(L=1,2,\dots,2^{k}\) generalizes to the
    more challenging task of length \(2^{k+1}\), yielding the following corollary.

\begin{corollary}[Self-improvement for $|\cX|$-length reasoning]\label{cor:length-generalization-self}
    Under the same assumptions as \Cref{thm:length-gen-self-training}, letting  $K = \Theta(\log d)$, for any length $L\leq |\cX|$, the model $F^{(T_{K})}$ trained via Algorithm~\ref{alg:cot-symmetry-training} achieves
$$
\mathrm{Acc}_{L}\!\Bigl(F^{(T_{K})}\Bigr) \geq 1 - \frac{1}{\poly(d)}.
$$

\end{corollary}

\paragraph{Significance of the result.} Note that \(\{x_{\ell}\}_{\ell=0}^{L}\) is sampled from \(\cX\) \emph{without
replacement} (\Cref{assump:lego-data-distribution}), the longest feasible
chain scales with the variable size: \(L{+}1 \le |\cX| = \Theta(d)\).
Thus our guarantee attains the best possible length in this setting.
\Cref{cor:length-generalization-self} also demonstrates that the transformer can be trained to solve a
task beyond \(\mathsf{TC}^0\) with linear-step CoT, matching the expressivity
result of \cite{li2024chainthought}.\footnote{There are a few caveats.  For example, we do not analyze
an embedding dimension logarithmic in the problem length.  We believe our
techniques can be extended to cover this setting.}
While prior empirical work reports self-improvement in practice, theoretical
guarantees, especially for transformers and length generalization, have been scarce~\cite{huang2024self,sun2025theoretical,song2024mind}.
\Cref{thm:length-gen-self-training} provides, to our knowledge, the first
rigorous evidence that transformers can \emph{bootstrap} their reasoning via
self-training without additional supervision.

\section{Proof Overview}
\label{sec:overview-proof}

In this section, we outline the proof ideas for the main theorem.
Our training schemes in Algorithms~\ref{alg:cot-transitive-training} and
\ref{alg:cot-symmetry-training} alternate between two phases.
First, we train the FFN parameters $\Wb$ to solve the one-step task $\cT^1$.
Then, holding $\Wb$ fixed, we train the attention parameters $\Qb$ to solve
$\cT^2$ and, for symmetry group actions, recursively $\cT^{2^k}$.
This mirrors the division of labor in our setting (\Cref{fig:toy-example}): the FFN learns the
\emph{local update rule}, while the attention layer learns to \emph{route and
compose} these updates by locating relevant context over long sequences.

Guided by this picture, our proof overview proceeds in two parts: (1) learning the one-step mechanism for LEGO ($\cT^1$): in-context variable
retrieval (Section~\ref{sec-overview-in-context}) and group operations
(Section~\ref{sec-overview-group-operations}).
(2) learning the attention layer: direct short-to-long generalization on $\cT^2$
under simply transitive actions (Section~\ref{sec-overview-simply-transitive}),
and recursive length generalization via self-training on $\cT^{2^k}$ under
symmetry group actions (Section~\ref{sec-overview-symmetry}).

\paragraph{Notations}  Let us first define a few notations to facilitate the presentation of the proofs.
For each \(i\in[5]\), \(j\in[d]\), \(r\in[m]\),  define
\begin{align}
&\Lambda_{i,j,r}\!\left(\Zb^{L,\ell-1}\right)
\;\triangleq\;
\sum_{\kk\in\mathcal{I}^{L,\ell-1}}
\attn_{{\ans,\ell-1}\rightarrow \kk}\cdot
\big\langle \Wb_{i,j,r},\,\Zb_{\kk}\big\rangle
\;+\;b_{i,j,r}.
\label{eq-def-Lambda-main}
\end{align}
The quantity \(\Lambda_{i,j,r}\) is the FFN pre-activation,
i.e., the input to \(\mathbf{sReLU}\), for token position \(i\), vocabulary
index \(j\), and hidden unit \(r\). According to \eqref{eq-def-F-main-1}, given \(\Zb^{L,\ell-1}\), the model’s
output at token position \(i\) and vocabulary index \(j\) is
\begin{align}
\bigg[F_{i}\!\left(\Zb^{L,\ell-1}\right)\bigg]_j
\;=\;
\sum_{r\in[m]}\,\mathbf{sReLU}\big(\Lambda_{i,j,r}
\big(\Zb^{L,\ell-1}\big)\big).
\label{eq-def-F-main-2}
\end{align}
We denote $\Wb_i \triangleq \{\Wb_{i,j,r}\}_ 
{j\in[d],,r\in[m]}$ as the FFN parameters
associated with token position $i$.
Each $\Wb_{i,j,r}\in\RR^{5d}$ is written as a vertical concatenation
$\Wb_{i,j,r}=[\Wb_{i,j,r,1};\dots;\Wb_{i,j,r,5}]$, where
$\Wb_{i,j,r,i'}\in\RR^{d}$ for $i'\in[5]$; these five blocks align with the five
token-type inputs.

\subsection{Learning One-Step Reasoning}
For solving task $\cT^1$, in stage 1 (stage 1.1 for symmetry group task), we train $\Wb$ via the full loss $\Loss^{1}$, which contains the prediction loss across five tokens in the answer clause and $\Wb_i$ is responsible for predicting the $i$-th token in the answer clause. Thus, there are three different types of prediction tasks here: the $\blank$ tokens (tokens 1, 2, 3), the correct variable $x_1$ (token 4) and the action update $y_1=g_1(y_0)$ (token 5). Notice that the $\blank$ tokens are deterministic, the learning task is straightforward. We therefore focus on the learning dynamics for the 4th and 5th tokens, which involve in-context retrieval $(\Wb_4)$ and one-step group action $(\Wb_5)$. Across these stage, the attention layer is fixed and keeps as uniform attention due to the zero initialization of the attention matrix, i.e., the input for the FFN layer is $\frac{1}{2} \Zb_{\pred,1}+\frac{1}{2} \Zb_{\ans,0}$.
\subsubsection{Learning In-Context Retrieval of Variables}
\label{sec-overview-in-context}
Given the input \(\tfrac{1}{2}\Zb_{\pred,1}+\tfrac{1}{2}\Zb_{\ans,0}\) for the FFN layer,
the goal is to predict the fourth token in the answer clause \(\Zb_{\ans,1}\),
which is the variable \(x_1\).
Intuitively, the network should retrieve the occurrence of \(x_1\) in the first
predicate token \(\Zb_{\pred,1}\) and copy it to the target position.   Specifically, the FFN pre-activation for predicting the fourth token to be $j$ is 
\begin{align*}
\Lambda_{4,j,r}\!\left(\Zb^{1,0}\right)
&=\tfrac{1}{2}\,\big\langle \Wb_{4,j,r},\,\Zb_{\pred,1}\big\rangle
  +\tfrac{1}{2}\,\big\langle \Wb_{4,j,r},\,\Zb_{\ans,0}\big\rangle
  + b_{4,j,r}.
\end{align*}
Using the 5-vector decomposition of $\Wb_{4,j,r}$, and the fact that the embedding of $\blank$ tokens are zero vectors (recall \(e_{\tau(\cdot)}\) denote the token embedding vectors), we further obtain:
\begin{align}
  \Lambda_{4,j,r}\!\left(\Zb^{1,0}\right)
    &=\tfrac{1}{2}\Big(
  \textcolor{BrickRed}{\langle \Wb_{4,j,r,1},e_{\tau(x_1)}\rangle}
  +\langle \Wb_{4,j,r,2},e_{\tau(g_1)}\rangle
  +\langle \Wb_{4,j,r,3},e_{\tau(x_0)}\rangle
  \Big) \label{eq-def-Lambda-4-x1}\\
  &\quad+\tfrac{1}{2}\Big(
  \langle \Wb_{4,j,r,4},e_{\tau(x_0)}\rangle
  +\langle \Wb_{4,j,r,5},e_{\tau(y_0)}\rangle
  \Big)
  +b_{4,j,r}. \notag
  \end{align}
Therefore, the main idea of our analysis is to track the training dynamics of $\langle \Wb_{4,j,r,p}, e_{s'} \rangle$ for $j,s'\!\in[d]$, $p\!\in[5]$,
and $r\!\in[m]$. Letting $s\in \tau(\X)$ be an embedding index for a variable, our analysis shows 
the diagonal correlations  $\langle \Wb_{4,s,r,1}, e_{s} \rangle$ receive strictly larger updates than all other  $\langle \Wb_{4,j,r,p}, e_{s'} \rangle$. This occurs because \(x_1\) co-occurs simultaneously as the first input token
at \(\Zb_{\pred,1}\) and as the supervised target at \(\Zb_{\ans,1}\),
which amplifies the gradient on
\(\langle \Wb_{4,s,r,1},e_s\rangle\) when $s=\tau(x_1)$.  In contrast, non-target coordinates
(off-diagonals, wrong variables, value tokens, and group-action tokens)
incur negligible gradients and remain \(o(1)\) throughout training.
Hence the correct variable's diagonal signal becomes order-wise larger, and the \emph{active diagonal mass}
\(\sum_{r}\langle \Wb_{4,s,r,1},e_s\rangle\)
dominately grows until the end of training. 

Given the dominance above, the learned weights align so that
\(\Wb_{4,s,r,1}\) points toward \(e_s\) for $s\in\tau(\X)$.
Substituting this alignment into \eqref{eq-def-Lambda-4-x1} shows that,
when \(j=\tau(x_1)\), the red term
\(\langle \Wb_{4,\tau(x_1),r,1},e_{\tau(x_1)}\rangle\) contributes dominantly to
\(\Lambda_{4,\tau(x_1),r}(\Zb^{1,0})\),
thereby realizing the intended in-context retrieval:
the model copies \(x_1\) from the frist position of \(\Zb_{\pred,1}\) to the fourth position of
\(\Zb_{\ans,1}\).

\subsubsection{Learning the Group Actions}\label{sec-overview-group-operations}

For task \(\cT^1\), the FFN input \(\tfrac{1}{2}\Zb_{\pred,1}+\tfrac{1}{2}\Zb_{\ans,0}\)
already contains the current value \(y_0\) and the action \(g_1\), with no distracting
information.
Accordingly, when predicting the next value \(y_1\), the role of \(\Wb_5\) is to
correctly apply the action \(g_1\) to the current value \(y_0\).
In this section, we sketch how the model learns to implement simply transitive group
actions, and we briefly discuss the symmetry case.

We first introduce notation to explain what the model should learn.

\begin{definition}[Combinations, simply transitive actions]
Assuming the group \(\cG\) follows \Cref{assump:structure-1}, for each class index \(j\in\tau(\cY)\), define the \emph{combinations}
\[
\Phi:=\bigcup\nolimits_{j\in\tau(\cY)}\Phi^\star_j, \quad \text{where } \Phi^\star_j:=\{(g',y')\in\cG\times\cY:\ \tau(g'(y'))=j\}.
\]
$|\Phi|=n_y^2$ for the simply transitive case. We call \(\phi=(g,y)\in\Phi^\star_j\) a \emph{combination} for predicting \(j=\tau(g(y))\). Hence, the goal of the model is to correctly identify all of these  
\(\phi=(g,y)\in\Phi\).
\end{definition}

Analogously to \eqref{eq-def-Lambda-4-x1}, the pre-activation at the fifth token
for class \(j\) and neuron \(r\) is
\[
\Lambda_{5,j,r}(\Zb^{1,0})
=\tfrac{1}{2}\,\dbrack{\Wb_{5,j,r,2},e_{\tau(g_1)}}
+\tfrac{1}{2}\,\dbrack{\Wb_{5,j,r,5},e_{\tau(y_0)}}+\text{other terms},
\]
where slot ``2'' reads the action token \(g_1\) and slot ``5'' reads the current
value token \(y_0\).
We then define the following feature-magnitude notation:
\begin{definition}[\(V\)-Notations]\label{def:V-notations}
  Given \(\phi=(g,y)\in\Phi\), for the \(r\)-th neuron in the \(j\)-th class with
  \(j\in\tau(\cY)\) and \(r\in[m]\), i.e., \(\Wb_{5,j,r}\in\RR^{5d}\), we define
  \[
  V_{j,r}(g):=\dbrack{\Wb_{5,j,r,2},e_{g}},\quad
  V_{j,r}(y):=\dbrack{\Wb_{5,j,r,5},e_{y}},
  \]
  and the composite feature magnitude
  \[
  V_{j,r}(\phi):=\tfrac{1}{2}\big(V_{j,r}(g)+V_{j,r}(y)\big).
  \]
  \end{definition}
  \noindent Then \(V_{j,r}(\phi)\) is exactly the contribution of the input pair \((g,y)\) to
  \(\Lambda_{5,j,r}(\Zb^{1,0})\) for predicting \(j\).
  Notice that for each class \(j\), there are \(m\) associated neurons
  \(\{\Wb_{5,j,r}\}_{r=1}^m\), and 
  we index neurons by the pair \((j,r)\) to emphasize that neuron \(r\) is
  specific to class \(j\); an index \(r\) alone has no cross-class meaning in
  this context.

The main proof idea is to track \(V_{j,r}(\phi)\) throughout training and show
that it amplifies the correct correlations across the combinations \(\Phi\) while suppressing
the incorrect ones.
To make this concrete and to clarify the roles of different neurons, we
introduce the following neuron–feature index set:
\[
\Psi:=\{(j,r,\phi)\mid j\in\tau(\cY),\ r\in[m],\ \phi\in\Phi\}.
\]
Here, \((j,r)\) again refers to neuron \(r\) for class \(j\), as in
\Cref{def:V-notations}.
With this notation we can write \(V_{j,r}(\phi)\) as \(V_{\psi}\) for
\(\psi=(j,r,\phi)\in\Psi\).

Our proof shows that, given \(j\), for each \(\phi\in\Phi^\star_j\), there is
exactly one neuron \(r\) in \(\Wb_{5,j,r}\) that is activated, denoted
\(r_{g\cdot y}\), to learn  \(\phi\); that is,
\(V_{j,r_{g\cdot y}}(\phi)\) will grow to a large value. Therefore, in total \(n_y^2\) distinct neurons will be activated to learn all combinations, i.e.,
\( \{ (j,r_{g\cdot y},(g,y)), \forall j\in\tau(\cY),\ (g,y)\in\Phi^\star_j \}\). The magnitude of remaining $V_{\psi}$ with non-activated neurons will stay close to the initialization.
Our analysis shows that learning follows an \emph{implicit curriculum} induced by the magnitude of features \(V_{\psi}\) at initialization:
\[
\psi\prec\psi'\iff V_{\psi}^{(0)}\ge V_{\psi'}^{(0)},\quad\forall\psi,\psi'\in\Psi.
\]
Items on the left under this ordering are learned first, and those on the right
are learned later.  We denote by \(\Sigma^{\star}\) the \emph{learning curriculum}: the ordered set
of neuron–feature indices \(\{(j, r_{g\cdot y}, (g,y))\}\) identified above,
equipped with this order.

Then, for each \(\psi\in\Sigma^{\star}\), the learning process follows its associated
ordering, and when it is the turn of \(\psi=(j,r_{g\cdot y},(g,y))\), the
learning process is mainly characterized by the following two phases:
\begin{itemize}[itemsep=1pt,leftmargin=1.5em]
  \item {\bf Phase I:} Emergence of the feature \(V_{\psi}\) among other
features.
During this phase, \(V_{\psi}\) grows faster than any \(\psi'\neq\psi\) with
\(\psi\prec\psi'\), while not affecting the already learned predecessors
\(\psi'\in\Sigma^{\star}\) with \(\psi'\prec\psi\), with growth rate
\begin{align}\label{eq:phase-1-tpm-overview}
V_{\psi}^{(t+1)}\ge V_{\psi}^{(t)}+\tOmega(\eta)\cdot \big(b_{j,r}
+V_{\psi}^{(t)}\pm o(\mu)\big)^{q-1}.
\end{align}
This form permits the application of the tensor power method (TPM)
\cite{allen2020towards} and explains the ordering induced by the magnitude of
\(V_{\psi}\) at initialization, since TPM implies that a slightly larger
initial value leads to dramatically faster growth.
By the end of Phase I, for any other feature \(\psi'\neq\psi\) with
\(\psi\prec\psi'\), the growth of
that feature is capped at \(\tO(\sigma_0)\), its initial magnitude.
  \item {\bf Phase II:} Growth of \(V_{\psi}\) and cancellation of incorrect
  features. After Phase I, the target feature \(V_{\psi}\) already has a relatively large
  magnitude; however, features in the set
  \(\{\psi'=(j,r_{g\cdot y},\phi')\in\Psi\mid
  \phi'=(g',y)\ \text{or}\ \phi'=(g,y')\}\), namely, the features in the same
  neuron \((j,r_{g\cdot y})\) that share exactly one component of
  \(\phi=(g,y)\), may also grow as the shared component increases.
  Note that \(j\) is not the correct label for such \(\phi'\), and
  we call them \textbf{confounding features}.
  The key characterization in this phase is that, due to a stationarity property
  of the gradients, although the dynamics are coupled, the wrong half of the
  confounding features grows in the negative direction while the correct half
  continues to grow, and ultimately the confounding feature cancels out.  
\end{itemize}
We show that for each \(\psi\in\Sigma^{\star}\) that should be learned, it
retains its structure after its own learning process (i.e., at the end of
Phase II) and persists through the final convergence while the other features
in \(\Sigma^{\star}\) are learned.
Specifically, we have the following properties at the end of training:
\begin{subequations}
  \label{eq:properties-at-the-end-of-training-cyc}
\begin{align}
&V_{j,r_{g\cdot y}}(g),\ V_{j,r_{g\cdot y}}(y)\ \approx\ B,\\
&V_{j,r_{g\cdot y}}(g'),\ V_{j,r_{g\cdot y}}(y')\ \approx\ -B,\\
&V_{j,r_{g\cdot y}}(g)+V_{j,r_{g\cdot y}}(y')\ \le\ o(1)\ \text{ and }\
V_{j,r_{g\cdot y}}(g')+V_{j,r_{g\cdot y}}(y)\ \le\ o(1)
\quad \forall\ g'\neq g,\ y'\neq y. 
\end{align}
\end{subequations}
\paragraph{How does this structure perform group actions?}
For an input pair \((g,y)\), let the correct answer be \(j=\tau(g(y))\).
Then the pre-activation for predicting $j$ is around \(B\) at neuron \(r_{g\cdot y}\), with all
other neurons for \(j\) near the initial value \(\tilde{O}(\sigma_0)\).
Thus the model output at \(j\) is around \(B\).
For all other predictions \(j'\), there exist
\(\phi_1=(g',y)\) and \(\phi_2=(g,y')\) such that
\(j'=\tau(g'(y))=\tau(g(y'))\).
By cancellation of the incorrect features, we have
\(V_{j',r_{g'\cdot y}}(\phi_1),\ V_{j',r_{g\cdot y'}}(\phi_2)\ \le\ o(1)\).
Since \(B=\Theta(\log d)\), this implies that the logit on the correct
prediction \(j\) is very close to \(1\).

  \paragraph{Symmetry Group Actions.}
  The proof strategy for symmetry group actions mirrors the simply transitive case
  through the emergence, refinement, and convergence phases.
  However, symmetry actions create richer interactions because multiple group
  elements can map the same \(y\) to the same \(j\), which requires more nuanced
  control of the training dynamics and leads to different learned feature
  structures.
  To illustrate this pattern, we slightly modify the definition of
  \textbf{combinations} used in the simply transitive case by introducing the
  notion of the \textbf{fiber} of a value.
 
  \begin{definition}[Combinations, Symmetry Actions]
  Assuming the group \(\cG\) follows \Cref{assump:structure-2}, define
    \(\fiber_{j,y}:=\{g\in\cG\mid \tau(g(y))=j\}\).\footnote{We use the notion of
    fiber to denote the left cosets \(gG_y\subset\cG\), \(g\in\cG\), where
    \(y\in\cY\) and \(G_y=\{g\in\cG\mid g(y)=y\}\) is the stabilizer.
    Since the fiber of the orbit map \(f_y(g)=g(y)\) is the preimage of \(f_y\),
    i.e., \(f_y^{-1}(y')\) with \(y'=\tau^{-1}(j)\), \(j\in\tau(\cY)\), this is
    exactly the set \(\fiber_{j,y}\).} 
    This allows us to define the combinations as
    \[
    \Phi=\{\varphi_{j,y}\mid j\in\tau(\cY),\ y\in\cY\},\quad
    \text{ where }\ \varphi_{j,y}=\fiber_{j,y}\times\{y\}.
    \]
    Moreover, \(|\varphi_{j,y}|=(n_y-1)!\).
    We continue to call \(\phi=(g,y)\in\cG\times\cY\) a combination, whereas an
    element of \(\Phi\) is now a subset of combinations \(\varphi_{j,y}\) (there
    are \(n_y^2\) such \(\varphi_{j,y}\) in total), which includes all pairs
    \((g,y)\) such that \(g\) sends \(y\) to \(j\), and which reduces to the single
    pair \((g,y)\) in the simply transitive case.
    \end{definition}  
    Based on these notions, the main difference from the simply transitive case is
    that previously the basic learning unit is each pair \((g,y)\in\Phi\), whereas
    now it is the subset \(\varphi_{j,y}\) of combinations. All combinations \(\phi\) within \(\varphi_{j,y}\) are captured by the same and
    unique neuron, denoted \(r_{j,y}\), in the sense that the feature magnitude
    \(V_{j,r_{j,y}}(\phi)\) is large for any \(\phi\in\varphi_{j,y}\).
    Accordingly, the learning curriculum \(\Sigma^{\star}\) is now based on the
    initial value of the ensemble of feature magnitudes:
    \(\frac{1}{|\varphi_{j,y}|}\sum_{\phi\in\varphi_{j,y}} V_{j,r_{j,y}}(\phi)\).
    Finally, at convergence, we have the following imbalance of feature
    magnitudes:
    \begin{subequations}
      \begin{align}
      &V_{j,r_{j,y}}(g)\ \approx\ 2B,\qquad
        V_{j,r_{j,y}}(y)\ \approx\ \frac{2B}{n_y}
        \quad \forall\, g\in \fiber_{j,y};\\
      &V_{j,r_{j,y}}(g')\ \approx\ \frac{2B}{n_y},\qquad
        V_{j,r_{j,y}}(y')\ \approx\ -\,2B;\\
      &V_{j,r_{j,y}}(g)+V_{j,r_{j,y}}(y')\ \le\ o(1)\ \text{ and }\
        V_{j,r_{j,y}}(g')+V_{j,r_{j,y}}(y)\ \le\ o(1)
        \quad \forall\, g'\notin\fiber_{j,y},\ y'\neq y.
      \end{align}
      \end{subequations}

\subsubsection{What Changes for Longer Tasks?}
\paragraph{FFN is length invariant.}
The feed-forward network (with weights \(\Wb\)) only acts on the \emph{attended
linear combination} at the current output clause,
rather than scanning the sequence to retrieve information.
Retrieval from the context is delegated to the attention layer.
Hence, as sequence length grows and positions shift,
the FFN computation remains the same mapping on its local input,
making it length- and position-invariant. On the other hand, uniform attention becomes increasingly diluted as irrelevant clauses accumulate,
so selective  attention and learning attention layer are required.

\paragraph{Desired attention patterns for predicting each tokens.} To predict the fourth token, the model only needs the variable
\(x_{\ell+1}\) from the predicate clause \(\Zb_{\pred,\ell+1}\).
If the attention mass on \(\Zb_{\pred,\ell+1}\) does not vanish with \(L\),
the same pattern learned in \(\cT^1\) with \(\Wb_4\) applies directly. Predicting the fifth token is harder because it requires \emph{two} retrievals:
(i) the group action from \(\Zb_{\pred,\ell+1}\), and
(ii) the current value from \(\Zb_{\ans,\ell}\),
followed by applying the update via \(\Wb_5\).
Thus robustness over length hinges on maintaining both attention links. If the attention layer is robust enough to support the fifth-token prediction,
then the fourth-token prediction follows as a special case, since it needs only
one of the two attention links to persist.
Accordingly, in what follows we focus on optimizing the loss for the fifth
token; achieving high accuracy there is effectively equivalent to solving the
entire task.

\subsection{Learning the Attention Layer}\label{sec:overview-attention}

Successful training on $\cT^{1}$ shows that the model has learned the one-step
update $g\circ y$. Turning to task $\cT^{2}$, the 
remaining difficulty is \emph{routing}: directing attention to the appropriate
locations. Thus we train the attention matrix $\Qb$ to learn the routing pattern and keep $\Wb$ fixed. In this part, we will show how training induces the  
\textbf{attention concentration} pattern highlighted in \Cref{thm:length-generalization,thm:length-gen-self-training}: given an input $\Zb^{L,\ell-1}$, the
attention mass concentrates on
$\attn_{\ans,\ell-1\to\pred,\ell}$ and
$\attn_{\ans,\ell-1\to\ans,\ell-1}$. 

We quantify routing quality via the \emph{attention concentration degree}
\begin{align}
  \epsilon^{L,\ell}_{\mathsf{attn}}\!\left(\Zb^{L,\ell-1}\right)
  \;=\;
  1
  - \attn_{\ans,\ell-1\to\pred,\ell}\!\left(\Zb^{L,\ell-1}\right)
  - \attn_{\ans,\ell-1\to\ans,\ell-1}\!\left(\Zb^{L,\ell-1}\right),
  \label{eq-attention-concentrate}
\end{align}
which measures the fraction of attention mass \emph{not} placed on the two key
clauses, and the \emph{attention gap}
\begin{align}
  \Delta^{L,\ell}\!\left(\Zb^{L,\ell-1}\right)
  \;=\;
  \Big|
  \attn_{\ans,\ell-1\to\pred,\ell}\!\left(\Zb^{L,\ell-1}\right)
  - \attn_{\ans,\ell-1\to\ans,\ell-1}\!\left(\Zb^{L,\ell-1}\right)
  \Big|,
  \label{eq:attention-gap}
\end{align}
which captures how balanced the two target attentions are.
Thus, effective routing corresponds to small $\epsilon^{L,\ell}_{\mathsf{attn}}$
(high concentration) and small $\Delta^{L,\ell}$ (good balance).

Under \Cref{assump-Q-structure}, for an input $\Zb^{L,\ell-1}$ the
(unnormalized) attention score from   $\Zb_{\ans,\ell-1}$ to
a clause $\Zb_{\kk}$ decomposes as
\[
  \Zb_{\ans,\ell-1}^{\top}\Qb\Zb_{\kk}
  \;=\;
  \Zb_{\ans,\ell-1,4}^{\top}\Qb_{4,3}\Zb_{\kk,3}
  \;+\;
  \Zb_{\ans,\ell-1,4}^{\top}\Qb_{4,4}\Zb_{\kk,4},
\]
where $\Zb_{\kk}=[\Zb_{\kk,1},\ldots,\Zb_{\kk,5}]$ with
$\Zb_{\kk,i}\in\RR^d$.
By design, in the clause embeddings the \emph{fourth} token of a \emph{predicate}
clause and the \emph{third} token of an \emph{answer} clause are $\blank$. Consequently, we further have
\begin{align}
 & \Zb_{\ans,\ell-1}^{\top}\Qb\Zb_{\pred,\ell'} =  \Zb_{\ans,\ell-1,4}^{\top}\Qb_{4,3}\Zb_{\pred,\ell',3},\qquad \Zb_{\ans,\ell-1}^{\top}\Qb\Zb_{\ans,\ell'} =  \Zb_{\ans,\ell-1,4}^{\top}\Qb_{4,4}\Zb_{\ans,\ell',4}, \label{eq:Q-role}
\end{align}
which means $\Qb_{4,3}$ governs attention to predicate clauses, while
$\Qb_{4,4}$ governs attention to answer clauses. An immediate observation from \eqref{eq:Q-role} is that the desired allocation can be realized by growing the diagonal entries $[\Qb_{4,3}]_{s,s}$ and $[\Qb_{4,4}]_{s,s}$\footnote{$[\Ab]_{s,s'}$ denotes the entry of the matrix $\Ab$ at the $s$-th row and $s'$-th column} for
$s\in\tau(\cX)$ since $x_{\ell-1}$ appears as the third token in $\Zb_{\pred,\ell}$ and as the fourth token in $\Zb_{\ans,\ell-1}$ simultaneously. Our analysis of tracking the attention dynamics will show that these diagonal coordinates indeed receive asymptotically larger gradient
magnitudes than all other entries due to strong co-occurrence signal, and the training dynamics are dominated by their growth. See \Cref{fig:Q-example} for an illustration.

\begin{figure}[t]
  \centering
\includegraphics[width=.85\textwidth]{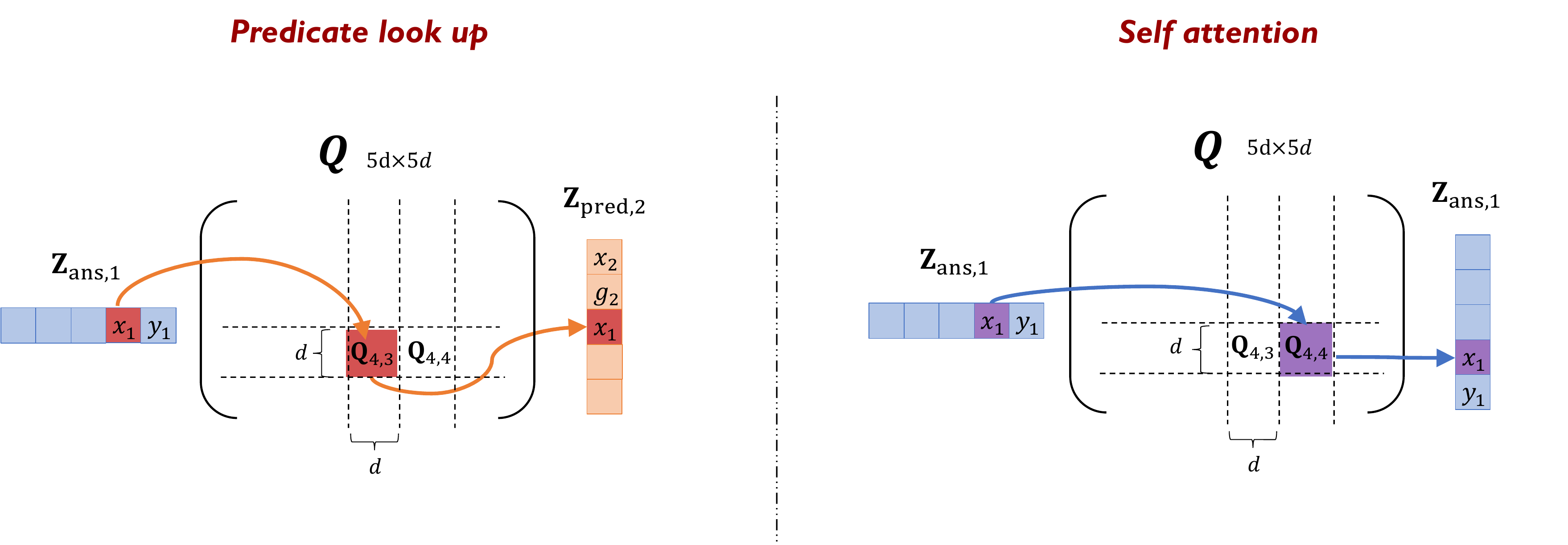}
  \caption{The illustration of how different components of the attention matrix $\Qb$ are used to route the attention to the appropriate locations. The query clause is $\Zb_{\ans,1}$ and the goal is to retrieve the correct action $g_2$ from $\Zb_{\pred,2}$  and value $y_1$ from the current answer clause $\Zb_{\ans,1}$. $\big[\Qb_{4,p}\big]_{s,s}$ will grow and dominate the learning dynamics for $p\in\{3,4\}$ and $s\in\tau(\cX)$. Thus, in this example, large $\big[\Qb_{4,3}\big]_{\tau(x_1),\tau(x_1)}$ indicates the large attention to the predicate clause $\Zb_{\pred,2}$ and large $\big[\Qb_{4,4}\big]_{\tau(x_1),\tau(x_1)}$ indicates the large self-attention to the answer clause $\Zb_{\ans,1}$.}
  \label{fig:Q-example}
\end{figure}

For notational simplicity, we will refer to the relevant diagonal entries
$[\Qb_{4,3}]_{s,s}$ and $[\Qb_{4,4}]_{s,s}$ simply as $\Qb_{4,3}$ and
$\Qb_{4,4}$ below. The remainder of the proof quantifies how the growth of $\Qb_{4,3}$ and
$\Qb_{4,4}$ simultaneously drives the concentration degree
$\epsilon^{L,\ell}_{\mathsf{attn}}$ toward zero and controls the gap
$\Delta^{L,\ell}$, ensuring that the FFN consistently receives
$(g_\ell,\,y_{\ell-1})$ and thus outputs the correct $y_\ell$.

\subsubsection{Simply Transitive Actions}\label{sec-overview-simply-transitive}
For the simply transitive case, we analyze the gradient contribution at position
$i=5$ on task $\cT^2$, i.e., the loss $\sum_{\ell=1}^{2}\Loss_{5}^{2,\ell}$.
We show that for $\cT^2$ the \emph{attention concentration degree}
$\epsilon^{2,\ell}_{\mathsf{attn}}$ (for $\ell\in[2]$) can be reduced below
$O(1/d^{c^{*}})$ for some constant $0<c^{*}<1$, indicating highly focused mass on the relevant clauses.
When irrelevant entries of $\Qb$ are small, we also have
$\epsilon^{2,1}_{\mathsf{attn}} \le \epsilon^{2,2}_{\mathsf{attn}}$, since the
number of irrelevant clauses doubles from $\ell=1$ to $\ell=2$; hence we focus
on controlling $\epsilon^{2,2}_{\mathsf{attn}}$.

\begin{itemize}[itemsep=1pt,leftmargin=1.5em]
    \item \textbf{Stage 2.1: Growth of an initial gap.}
          Early in training, attention is close to uniform, so given
          $\Zb^{2,\ell-1}$ we have the approximations
          \[
            \Lambda_{5,j,r}(\Zb^{2,0})
            \;\approx\; \tfrac13\,V_{j,r}(g_{1})
            + \tfrac13\,V_{j,r}(g_{2})
            + \tfrac13\,V_{j,r}(y_{0})
            \;,
          \]
          \[
            \Lambda_{5,j,r}(\Zb^{2,1})
            \;\approx\; \tfrac14\,V_{j,r}(g_{1})
            + \tfrac14\,V_{j,r}(g_{2})
            + \tfrac14\,V_{j,r}(y_{0})
            + \tfrac14\,V_{j,r}(y_{1})
            \;.
          \]
          By the cancellation at the convergence in \eqref{eq:properties-at-the-end-of-training-cyc}, for $\ell=2$ all
          $\Lambda$ lie in the small smoothed regime, whereas for $\ell=1$ we
          obtain a correct logit for $y_1=g_1(y_0)$ and a spurious logit for
          $g_2(y_0)$ of magnitude about $B/3$.
          Consequently, $-\nabla_{\Qb}\Loss^{2,2}_5$ is negligible, while
          $-\nabla_{\Qb}\Loss^{2,1}_5$ is comparatively large and drives
          $\Qb_{4,3}$ to grow faster than $\Qb_{4,4}$ (increasing $\Qb_{4,4}$
          would also amplify the wrong prediction $\tau(g_2(y_0))$ and thus not
          reduce $\Loss^{2,1}_5$).
          An $\Omega(1/\log d)$ gap emerges between the diagonals
          $[\Qb_{4,3}]_{s,s}$ and $[\Qb_{4,4}]_{s,s}$, yielding an early routing
          advantage toward right predicate clause.
  
    \item \textbf{Stage 2.2: Joint growth with a controlled gap.}
          As $\Qb_{4,3}$ increases, the weight $\attn_{\ans,1\to\pred,2}$ becomes large, moving
          $\Lambda_{5,\tau(g_2(y_1)),r_{g_2\cdot y_1}}$ and
          $\Lambda_{5,\tau(g_2(y_0)),r_{g_2\cdot y_0}}$ for $\ell=2$ into the
          linear regime.
          Gradients from $\ell=2$ then dominate, and $\Qb_{4,4}$ starts to grow
          to separate the correct $y_2=g_2(y_1)$ from the incorrect
          $\tau(g_2(y_0))$.
          Throughout, the gap between $\Qb_{4,3}$ and $\Qb_{4,4}$ stays within
          $[\Omega(1/\log d),\,O(1)]$, so the attention gap satisfies
          $\Delta^{2,2}=\Omega(1/\log d)$.
  
    \item \textbf{Stage 2.3: Convergence and gap reduction.}
          Continued joint growth of $\Qb_{4,3}$ and $\Qb_{4,4}$ concentrates
          attention near its ideal limit, making $\epsilon^{2,2}_{\mathsf{attn}}$
          small.
          We show $\Delta^{2,2}$ cannot remain above $o(1)$ for long; otherwise an
          incorrect logit  $\logit_{5,\tau(g_2(y_0))}$ would acquire a
          stronger gradient and force $\Qb_{4,4}$ to outpace $\Qb_{4,3}$, which
          contradicts stability. Therefore, at convergence we have
(i) $\epsilon^{2,2}_{\mathsf{attn}} = O(d^{-c^*})$ for some $c^*\in(0,1)$ and
(ii) $\Delta^{2,2} = o(1)$; hence both $\Qb_{4,3}$ and $\Qb_{4,4}$ equal
$C\log d \pm o(1)$, while all other entries of the attention matrix remain
close to their initial values, where $C>0$ depends on $c^*$.
  \end{itemize}

\paragraph{Strong length generalization.}
The key is that attention concentrates cleanly while remaining stably balanced.
For \(\cT^L\), we obtain
\[
  \epsilon^{L,\ell}
  \;\le\;
  \frac{O(1)\cdot L}{\,O(1)\cdot L + 1/\epsilon^{2,2}_{\mathsf{attn}}\,}
  \;\le\;
  \frac{1}{1+\Omega(d^{c^{*}}/L)}.
\]
Hence the model tolerates \(O(d^{c^*})\) irrelevant clauses: in particular,
if \(L=o(d^{c^{*}})\) then \(\epsilon^{L,\ell}=o(1)\); and if
\(L=\Theta(d^{c^{*}})\) then
\(\epsilon^{L,\ell}\le \frac{1}{1+\Omega(1)}\), which can be made small by choosing the
proportionality constant in \(L=\Theta(d^{c^{*}})\) appropriately. 
Moreover, \(\Delta^{L,\ell}\le \Delta^{2,2}\le o(1)\), so the two target
attention masses remain balanced as \(L\) grows, preventing errors from large
imbalance (e.g., predicting \(\tau(g_2(y_0))\) in Stage~2.3). Consequently,
the correct logit satisfies $1 - \logit_{5,\tau(g_{\ell+1}(y_\ell))}
  \;\le\; \frac{1}{\poly(d)}$, 
so \(\cT^L\) is solved with accuracy \(1 - 1/\poly(d)\).

\subsubsection{Symmetry Group Actions}\label{sec-overview-symmetry}
We now turn to symmetry group tasks $\cT^{L}$ and analyze GD updates with
respect to the per-token loss $\Loss_{5}^{L,2}$ (i.e., predicting the value
token in $\Zb_{\ans,2}$ from $\Zb^{L,1}$).

\paragraph{The case $L=2$.}
The high-level picture mirrors the simply transitive case, but because multiple
group elements can map a given $y$ to the same $j$, the learned $V_{j,r}(g)$ and $V_{j,r}(y)$ structures are
unbalanced.  Moreover, as we discussed in the hardness part for the symmetry group in \Cref{sec:symmetry}, there will be a non-negligible proportions of distractor clauses.  These make it harder to keep a tight balance between $\Qb_{4,3}$ and
$\Qb_{4,4}$ as in the simply transitive case.
Nevertheless, we prove that both $\Qb_{4,3}$ and $\Qb_{4,4}$ grow, and the
attention gap $\Delta^{2,2}$ is controlled by a feedback mechanism: if
$\Delta^{2,2}$ exceeds a small fixed threshold (in either direction), some
incorrect logit receives a stronger gradient, which pushes the system back
toward balance.
Consequently, after sufficient training:
(i) $\epsilon^{2,2}_{\mathsf{attn}}\le C_1$;
(ii) $\Delta^{2,2}\le C_2$,
for sufficiently small constants $C_1,C_2$; and since $B=\Theta(\log d)$, we
still obtain $\Loss^{2,2}_5\le 1/\poly(d)$.

\paragraph{Recursive learning for $\cT^{2^k}$, $k\ge 2$.}
Because $\epsilon^{2^{k-1},2}_{\mathsf{attn}}$ is already a small constant,
initialization for $\cT^{2^k}$ satisfies
$\epsilon^{2^k,2}_{\mathsf{attn}}\le 2\,\epsilon^{2^{k-1},2}_{\mathsf{attn}}$
(still small), and $\Delta^{2^k,2}\le \Delta^{2^{k-1},2}$.
Thus the attention pattern remains close to that in $\cT^{2^{k-1}}$, and
$\Zb^{2^k,1}$ follows a bootstrapped LEGO distribution generated by the greedy
model $\hat{p}_{F^{(T_{k-1})}}$, which coincides with the original LEGO source.
In particular, $y_1=g_1(y_0)$ and $y_2=g_2(y_1)$ are correct.
We can therefore reuse the convergence analysis from $\cT^{2^{k-1}}$ to show
that both $\epsilon^{2^k,2}_{\mathsf{attn}}$ and $\Delta^{2^k,2}$ decrease to a
small constant, yielding stable, inductive concentration across recursive
reasoning depths.

\subsection{Significance of the Proof}\label{sec:overview-significance}
Our proof techniques are inspired by recent advances in understanding the
dynamics of feature learning in neural networks~\cite{allen2021forward,allen2022feature,wen2021toward,wen2022mechanism,huang2022modality,jelassi2022vision,huang2023context,huang2025a},
which show how gradient-based training induces useful internal patterns and
representations.
Building on these ideas, we analyze the dynamics of the FFN and attention
layers to capture how transformers gradually acquire length-generalizable
reasoning through CoT training.
We conclude this section by summarizing the technical significance of our
analysis for learning in the FFN and attention layers.

\paragraph{Learning the group actions.}
Most existing optimization-based analyses of CoT~\cite{huang2025transformers,wen2025sparse,kim2025transformersprovably, huang2025transformerslearn} hard-code the group action, 
e.g., parity, as a fixed FFN or an almost-linear map.
In contrast, we \emph{train} a nonlinear FFN end-to-end to perform the group
action at each step, with parity as a special \(n=2\) instance of our
simply transitive framework. This formulation is strictly more general and substantially more challenging
than prior setups. Our proof shows that the model learns not only the basic action for simply transitive groups but also more complex actions for richer actions that is transitive but not free. Especially, we precisely characterizing how task-relevant features emerge and spurious features are suppressed during the process. The technique we used in the proof for learning the FFN on discrete combinations of data is of independent interest beyond analyzing CoT.

\paragraph{Learning the attention patterns.}
Prior work on transformer training dynamics has established the
\emph{attention concentration} principle as a key mechanism for solving
various tasks, e.g., in-context learning~\cite{huang2023context},
self-supervised learning~\cite{huang2025a}, and graph
learning~\cite{nichani2024transformers}.
In those settings, only a single token needs to be retrieved, so the attention
matrix is learned to be diagonal (after a suitable change of basis), and a
single pattern suffices.
In contrast, our task requires retrieving two different types of clauses, so
the model must learn two distinct components (blocks) in the attention matrix,
\(\Qb_{4,3}\) and \(\Qb_{4,4}\).
As we show in \Cref{sec:overview-attention}, beyond growing these blocks, we
must maintain a delicate balance between them.
This balance is crucial for length generalization across different group
actions, introduces substantial technical challenges, and our proof provides
a fine-grained control of it.

\section{Experiments}\label{sec:exp}

In this section, we conduct synthetic experiments to verify our theoretical claim.

\paragraph{General Setup.}  We adopt an experimental setting closely aligned with our theoretical setup.
Data strictly follow the LEGO distribution in the five-token-per-clause format
defined in \Cref{def:lego-encoding} and \Cref{assump:lego-data-distribution}.
The simply transitive task uses the cyclic group of order $6$ (denoted $C_6$), and
the symmetry task uses the symmetry group on five elements (denoted $S_5$).
The network is a one-layer, decoder-only transformer with two attention heads and
a FFN block.
Training optimizes the next-clause loss in \eqref{eq:next-clause-loss-def} using
Adam~\cite{Kingma2014AdamAM} with a learning rate of $1$e-$4$.
We train for 300 epochs to ensure the training loss approaches zero and the model
converges. For evaluation, to measure the transformer's ability to solve CoT reasoning tasks,
we eschew the teacher-forced accuracy in \eqref{eq:acc-L-def}, which conditions on
the ground-truth answer prefix, and instead evaluate final-answer accuracy after
autoregressively generating all intermediate answer steps without teacher forcing.
This better mimics CoT reasoning and is more challenging because errors can
accumulate during self-rollout.
For computational efficiency, we report the probability that the model predicts the value of all
answer steps correctly, following \cite{li2024chainthought}.

\paragraph{Length generalization for different group actions.}
\Cref{fig:main-results-a} shows that when trained at the short length $L=5$, the model
exhibits strong length generalization, achieving near-perfect accuracy at much
longer lengths, for example, up to $L=160$ on the simply transitive task, which
corroborates \Cref{thm:length-generalization}.
By contrast, on the harder symmetry task the generalization is weaker, yielding
only constant-factor extensions as expected.
\begin{figure}[t]
    \centering
    \begin{subfigure}[t]{0.3\textwidth}
        \centering
\includegraphics[width=\linewidth]{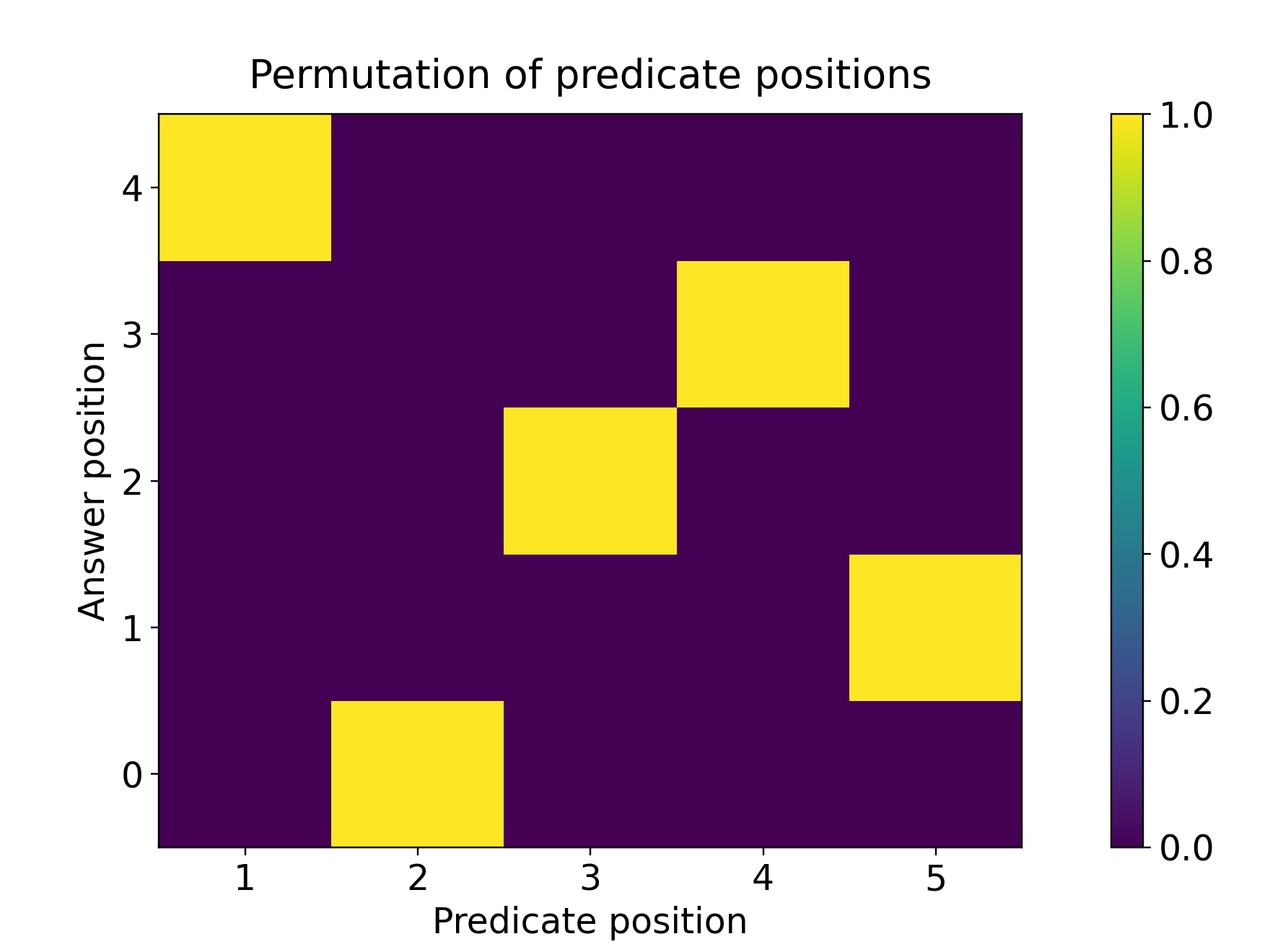}
    \end{subfigure}
    \hfill
    \begin{subfigure}[t]{0.3\textwidth}
        \centering
\includegraphics[width=\linewidth]{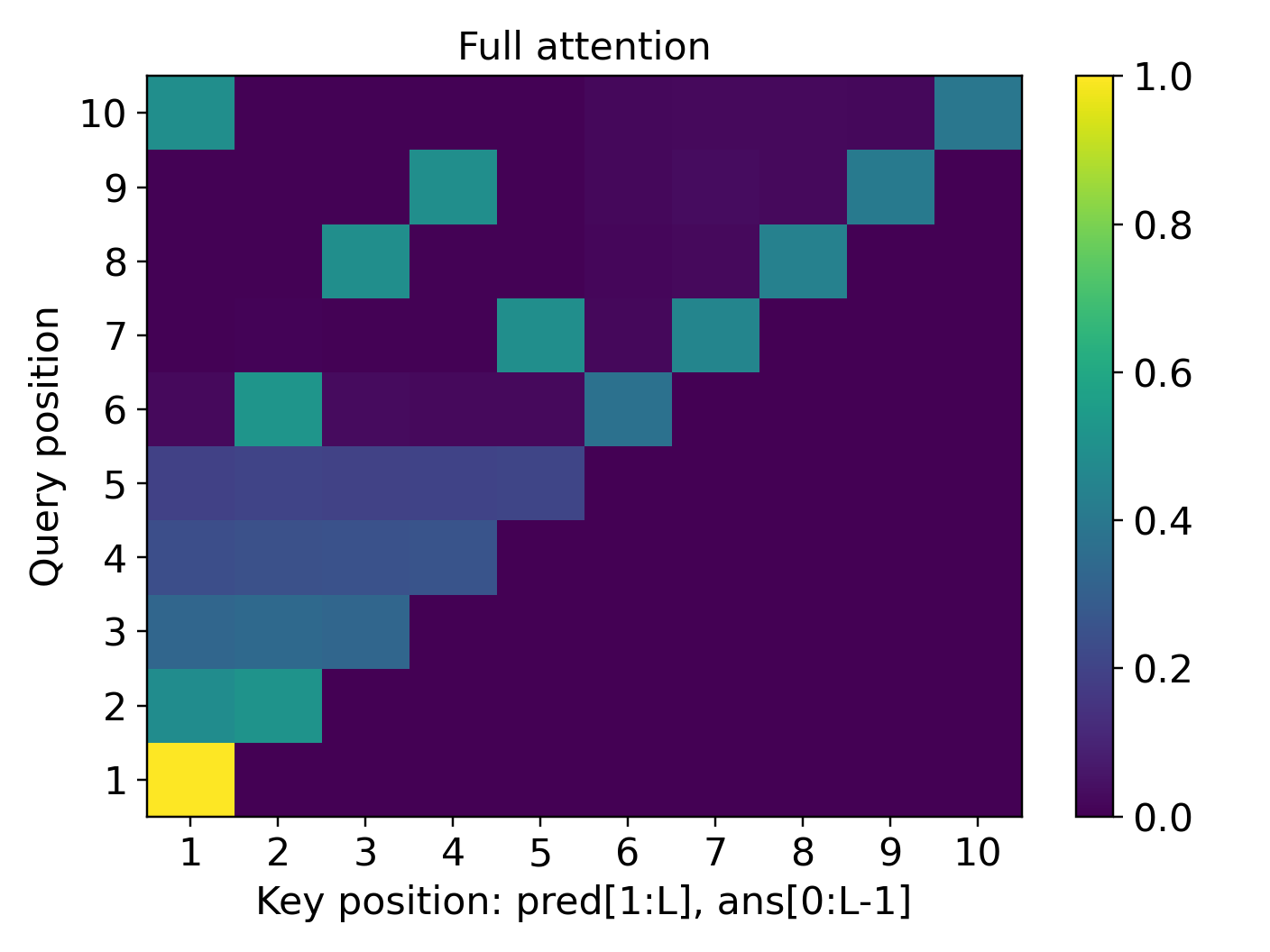}
    \end{subfigure}
    \hfill
    \begin{subfigure}[t]{0.3\textwidth}
        \centering
\includegraphics[width=\linewidth]{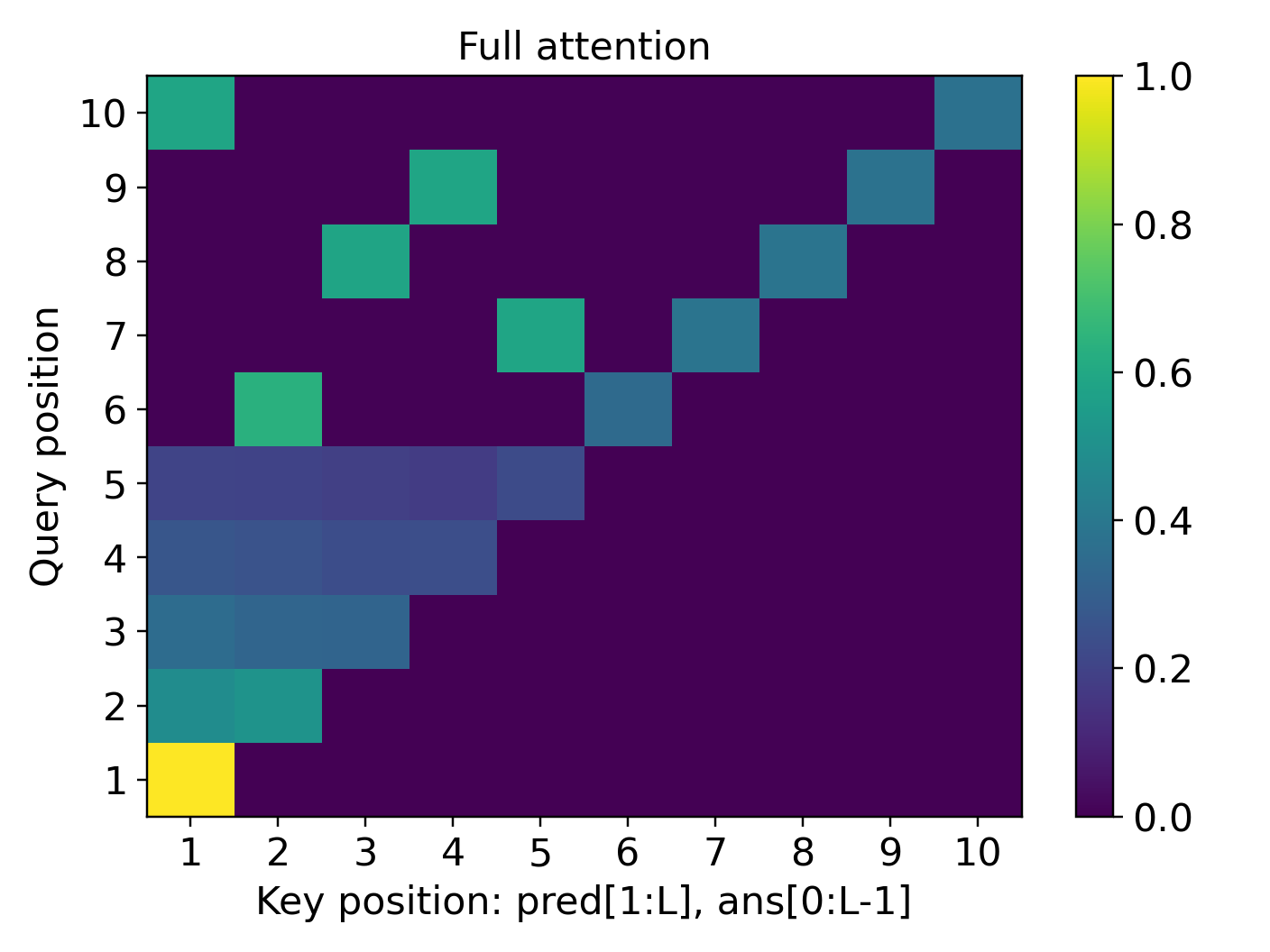}
    \end{subfigure}
    \vspace{0.3cm}
    \begin{subfigure}[t]{0.3\textwidth}
        \centering
\includegraphics[width=\linewidth]{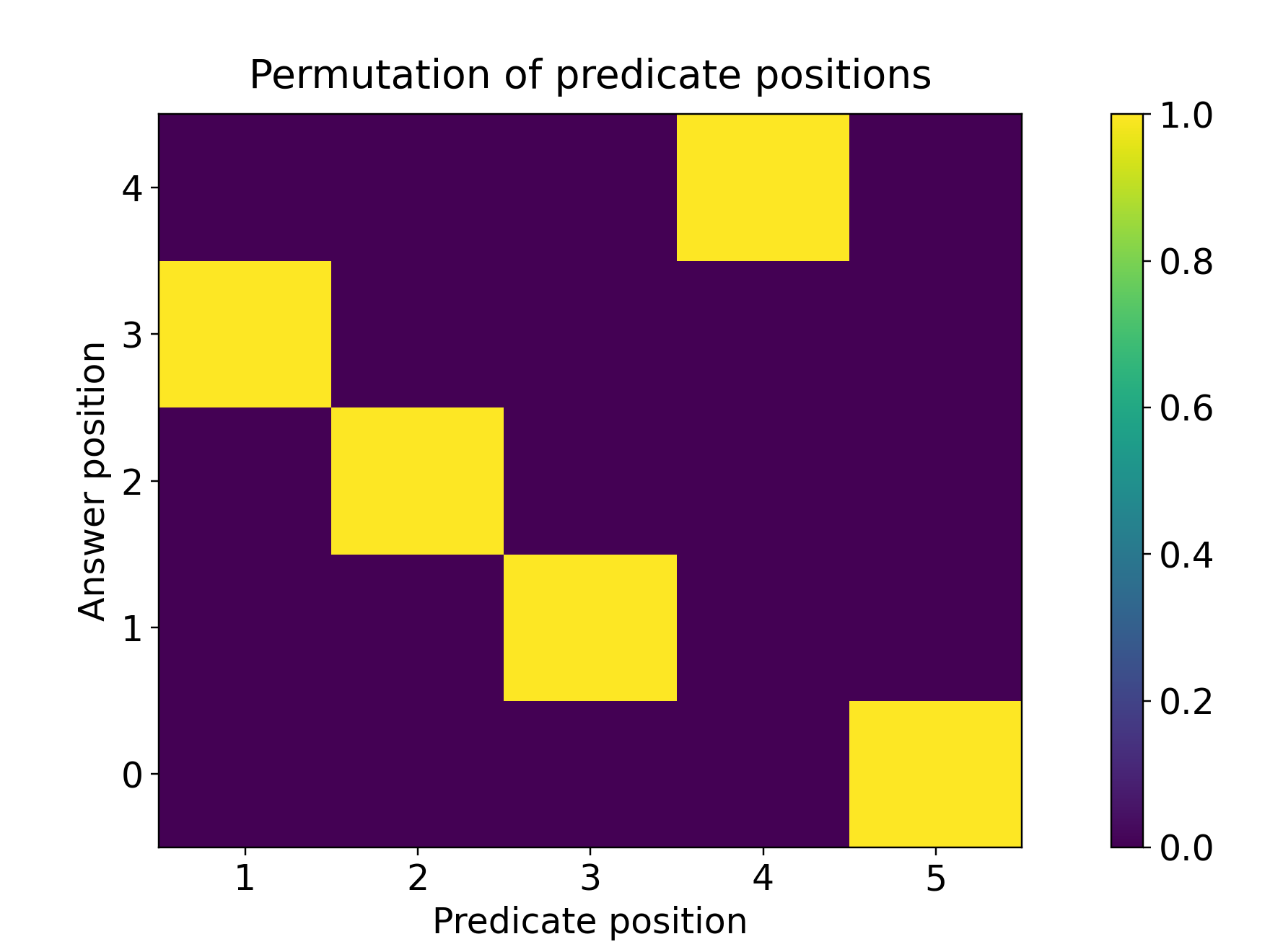}
\caption{\centering  Ground-truth permutation }
    \end{subfigure}
    \hfill
    \begin{subfigure}[t]{0.3\textwidth}
        \centering
\includegraphics[width=\linewidth]{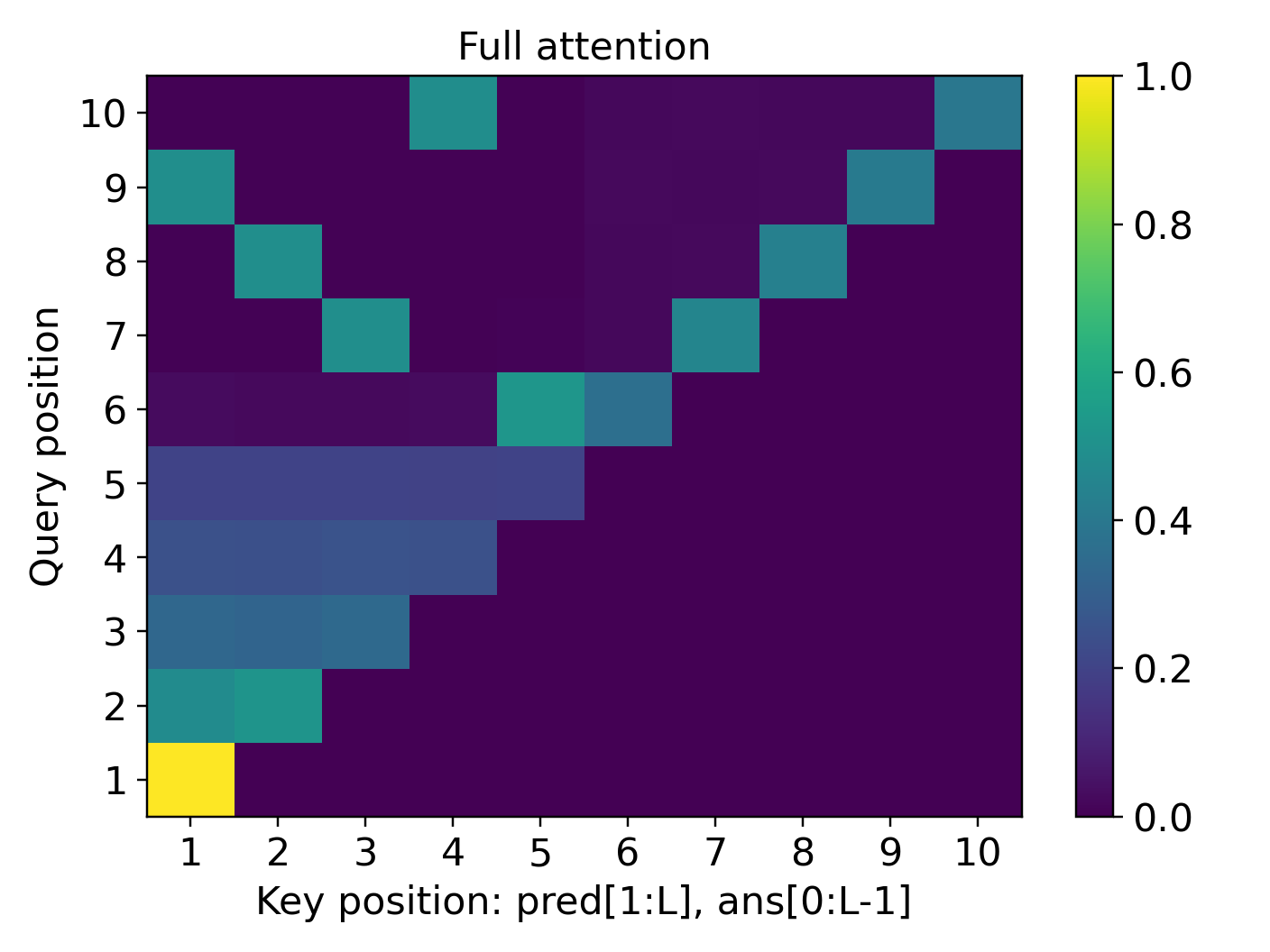}
\caption{\centering $C_6$}
    \end{subfigure}
    \hfill
    \begin{subfigure}[t]{0.3\textwidth}
        \centering
\includegraphics[width=\linewidth]{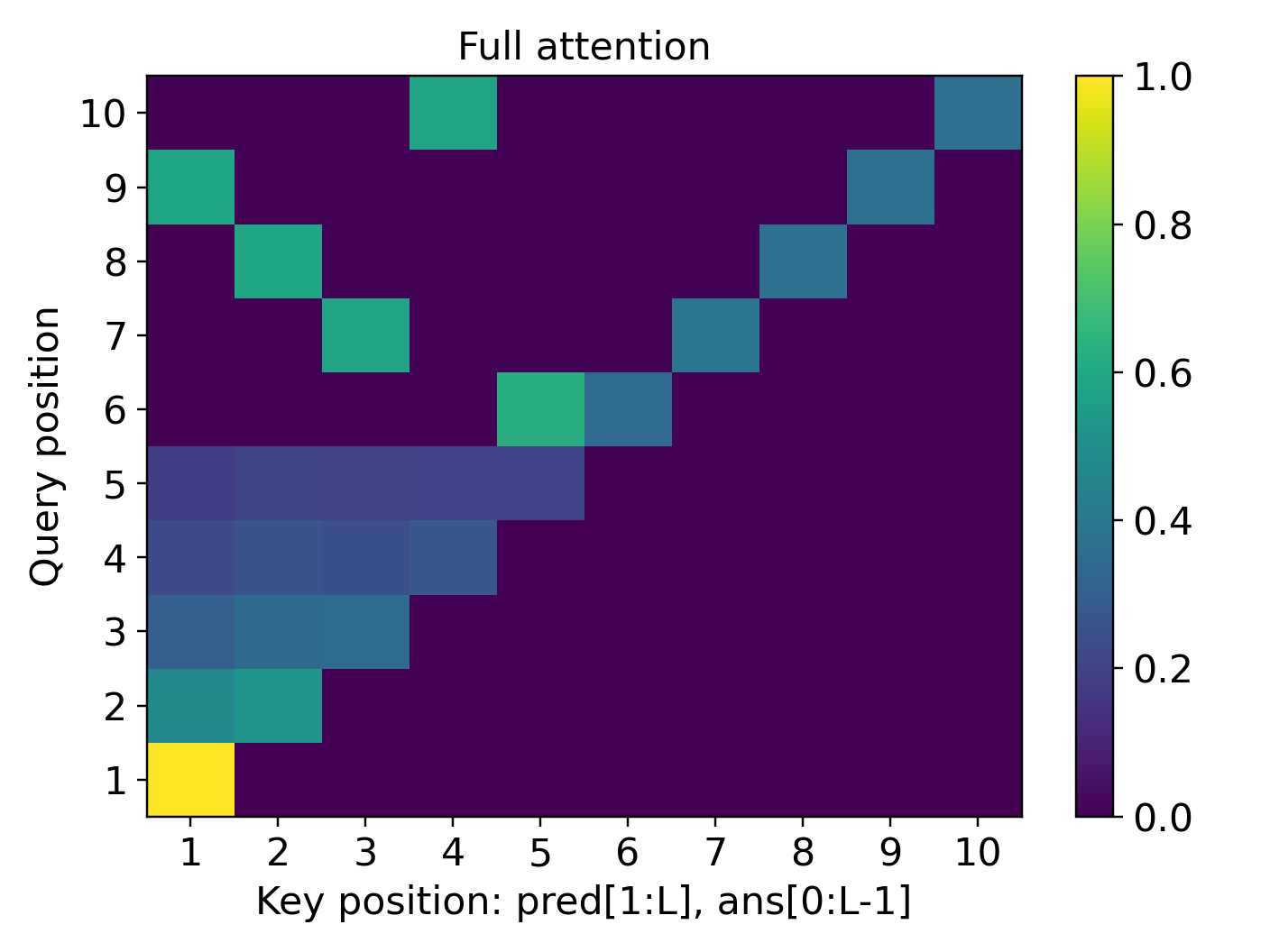}
\caption{\centering $S_5$}
    \end{subfigure}
    \caption {Attention patterns of the same trained
model as \Cref{fig:attn-con}, evaluated  with randomly permuted
predicate-clause positions. Column (a) gives the ground-truth permutation. Column (b) and Column (c) show the attention heatmaps for the simply transitive and symmetry tasks, respectively.}
    \label{fig:attn-con-perm}
\end{figure}
\begin{figure}[t]
  \centering
  \begin{subfigure}[t]{\textwidth}
    \centering
    \begin{minipage}[b]{0.53\textwidth}
      \centering
      \includegraphics[width=\linewidth]{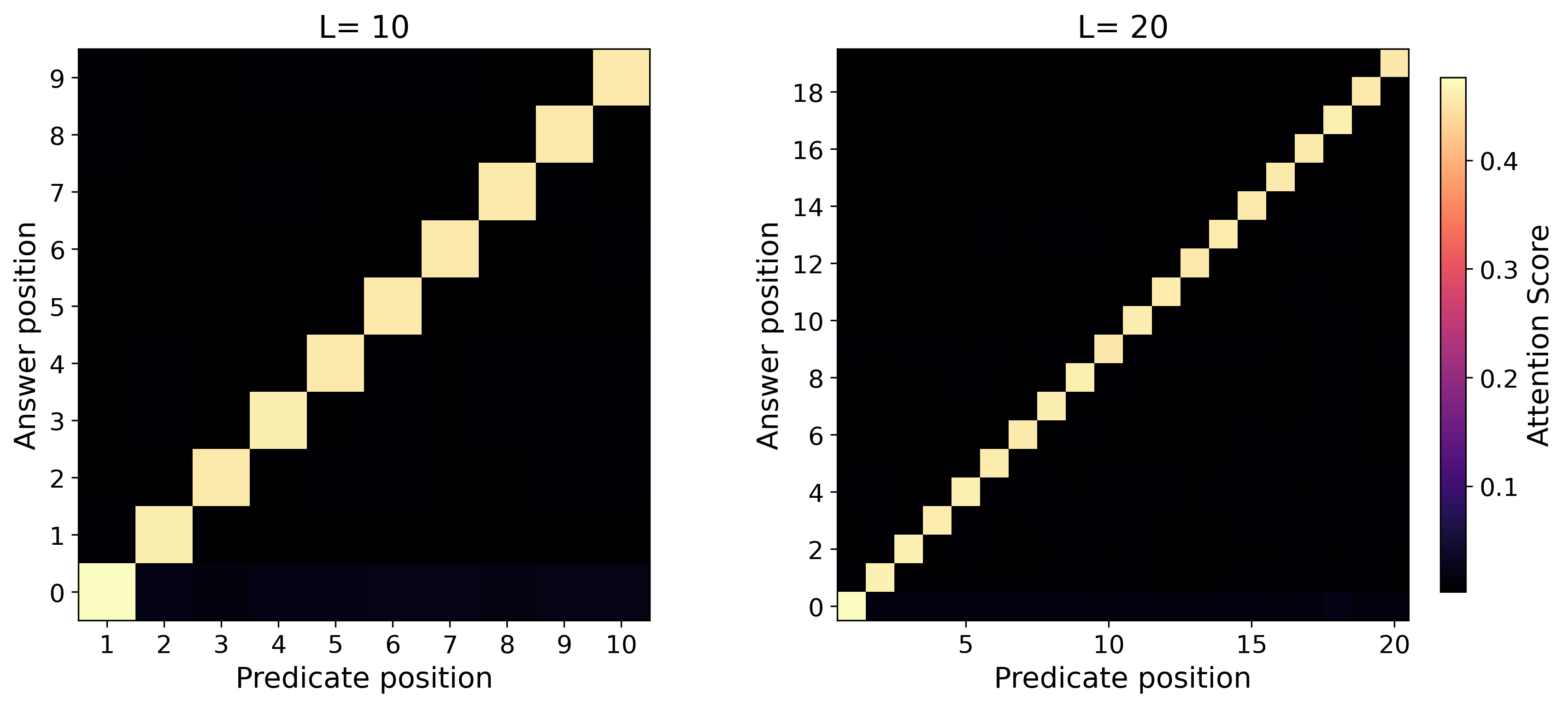}
    \end{minipage}
    \begin{minipage}[b]{0.4\textwidth}
      \centering
      \includegraphics[width=.9\linewidth]{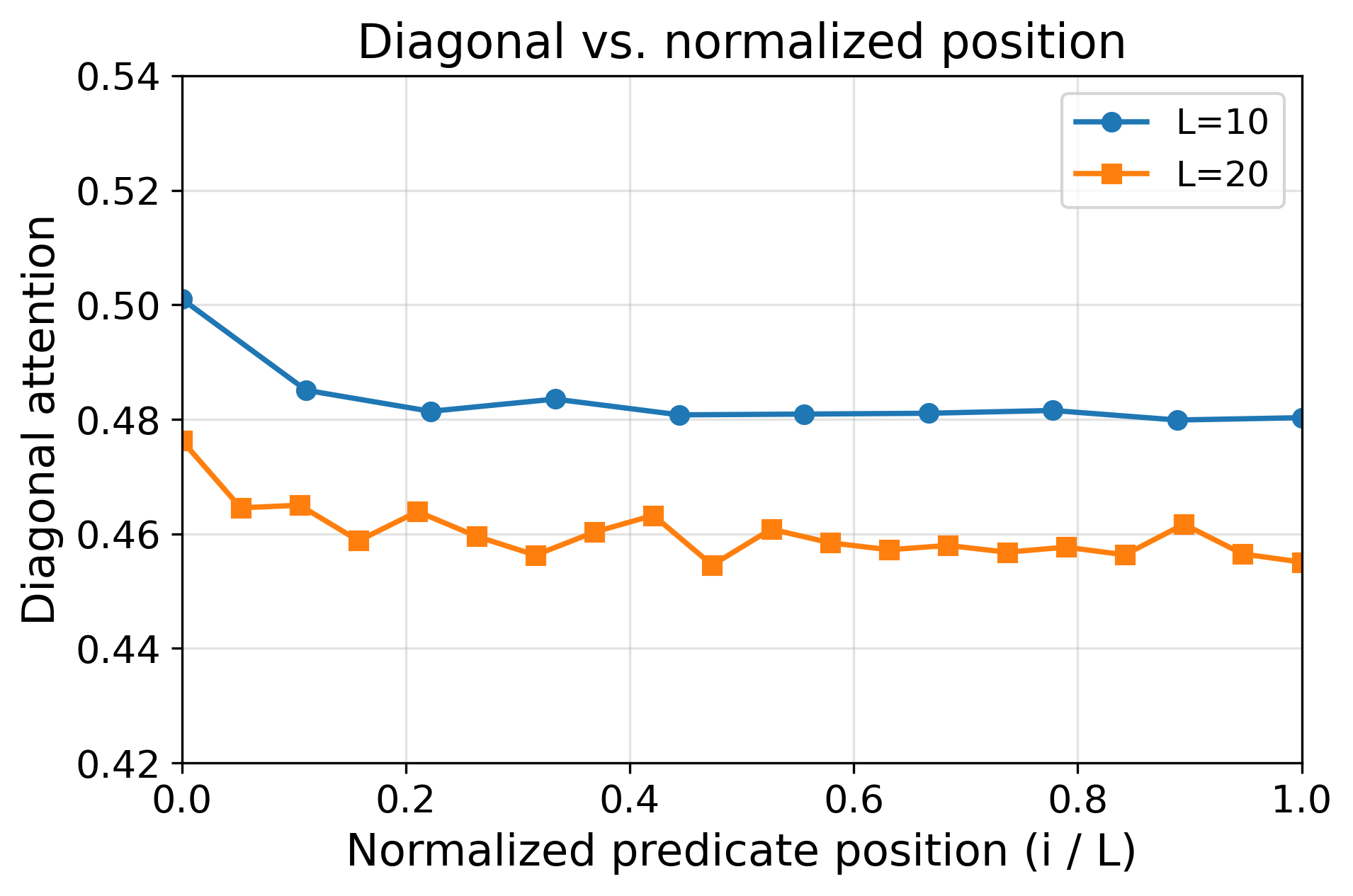}
    \end{minipage}
    \caption{\centering $C_6$}
    \label{fig:attn-con-diff-row1}
  \end{subfigure}

  \vspace{0.5em}

  \begin{subfigure}[t]{\textwidth}
    \centering
    \begin{minipage}[b]{0.53\textwidth}
      \centering
      \includegraphics[width=\linewidth]{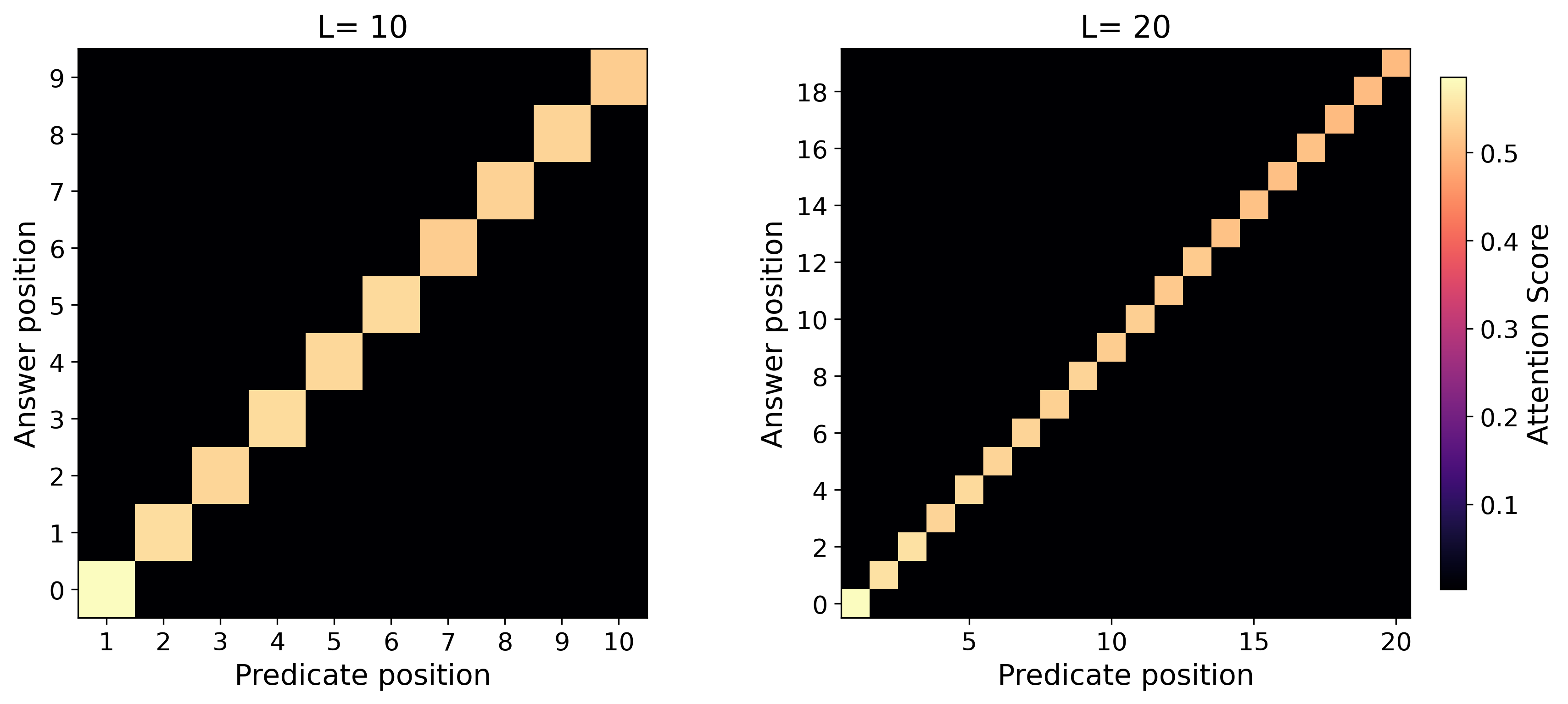}
    \end{minipage}
    \begin{minipage}[b]{0.4\textwidth}
      \centering
      \includegraphics[width=.9\linewidth]{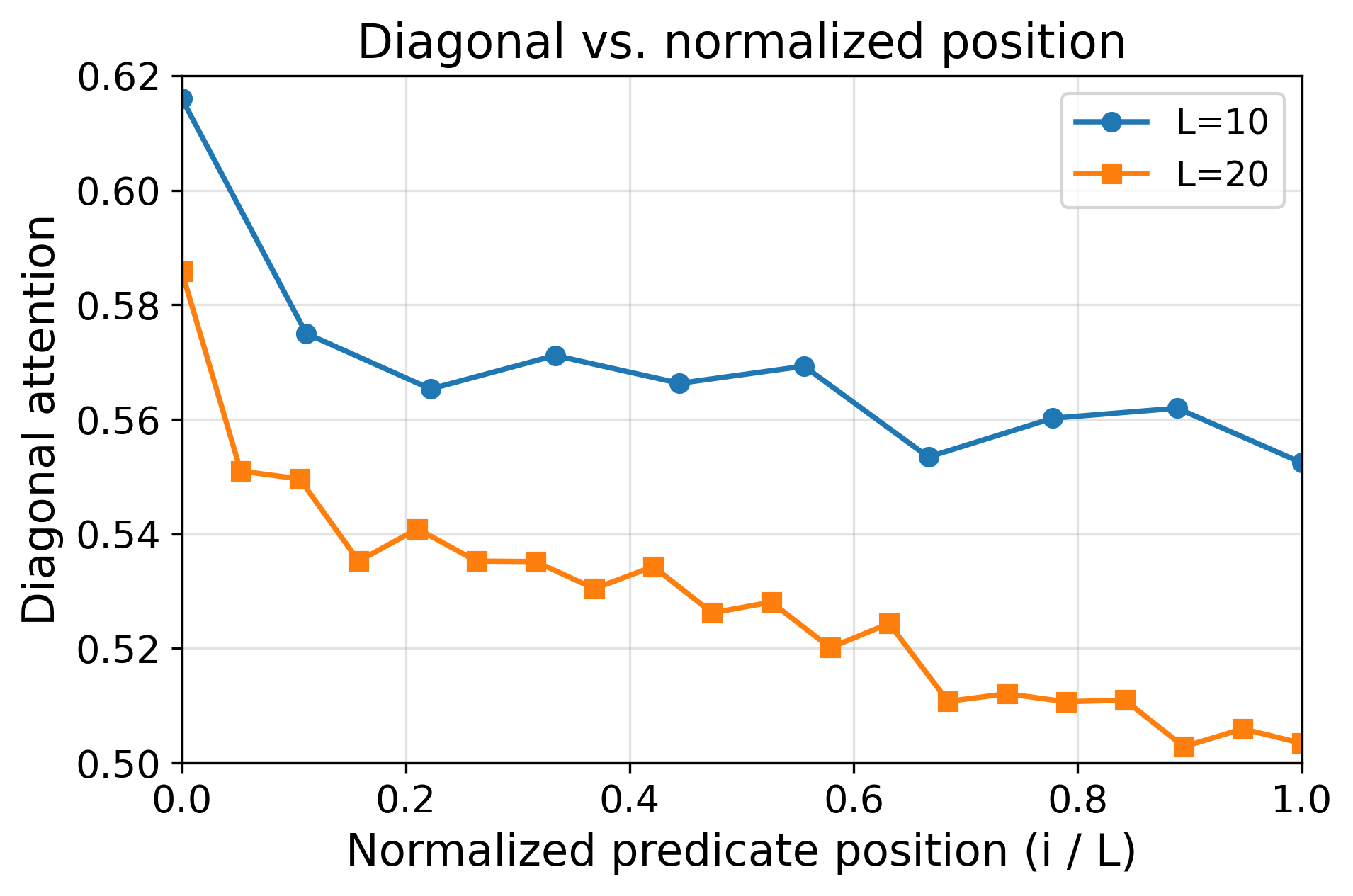}
    \end{minipage}
    \caption{\centering $S_5$}
    \label{fig:attn-con-diff-row2}
  \end{subfigure}

  \caption{
    Predicate-clause attention for models trained with LEGO task of length
    $L=5$ and tested with $L=10,20$.
    We extract the heatmap of the attention from answer queries to predicate
    clauses (the upper-left part of the full attention heatmap as in
    \Cref{fig:attn-con}), and line-plot the diagonal attention against the normalized predicate position.
    At longer contexts, the symmetry group task displays a visible drop in
    attention to the target predicate $Z_{\pred,\ell+1}$ at later predicate
    locations (larger $\ell$; note the absolute query index $L{+}\ell$ also
    increases).
    By contrast, the cyclic group task maintains more consistent attention across
    lengths and positions.
  }
  \label{fig:attn-con-diff}
\end{figure}

\paragraph{Self-training improves reasoning length.}  To extend the solvable length on the harder symmetry task, we adopt a recursive
self-training curriculum inspired by \cite{lee2025selfimproving}.
We first train at $L=5$ with ground-truth answer supervision until convergence.
At the next stage, we double the length to $L=10$, use the $L{=}5$ model to
greedily generate answer traces (self-labels) for the $L{=}10$ data, and retrain
on these pseudo-labels.
We repeat this doubling process for three stages ($L=5\to10\to20\to40$),
so that the final model is trained on self-labeled data at $L=40$.
\Cref{fig:main-results-b} reports the length-generalization curve for each stage, with
the dashed line (matching the curve’s color) marking that stage’s training length.
Across stages, we observe at least constant-factor generalization beyond the
training length; after multiple rounds of self-training, the model achieves nearly
perfect accuracy at lengths far exceeding the initial $L=5$, ultimately matching
the simply transitive task’s performance up to $L=160$.
These results validate \Cref{thm:length-gen-self-training} and demonstrate the
effectiveness of recursive self-training for extending reasoning length.

\begin{figure}[t]
  \centering
  \begin{subfigure}[t]{\textwidth}
    \centering
    \begin{minipage}[b]{0.53\textwidth}
      \centering
      \includegraphics[width=\linewidth]{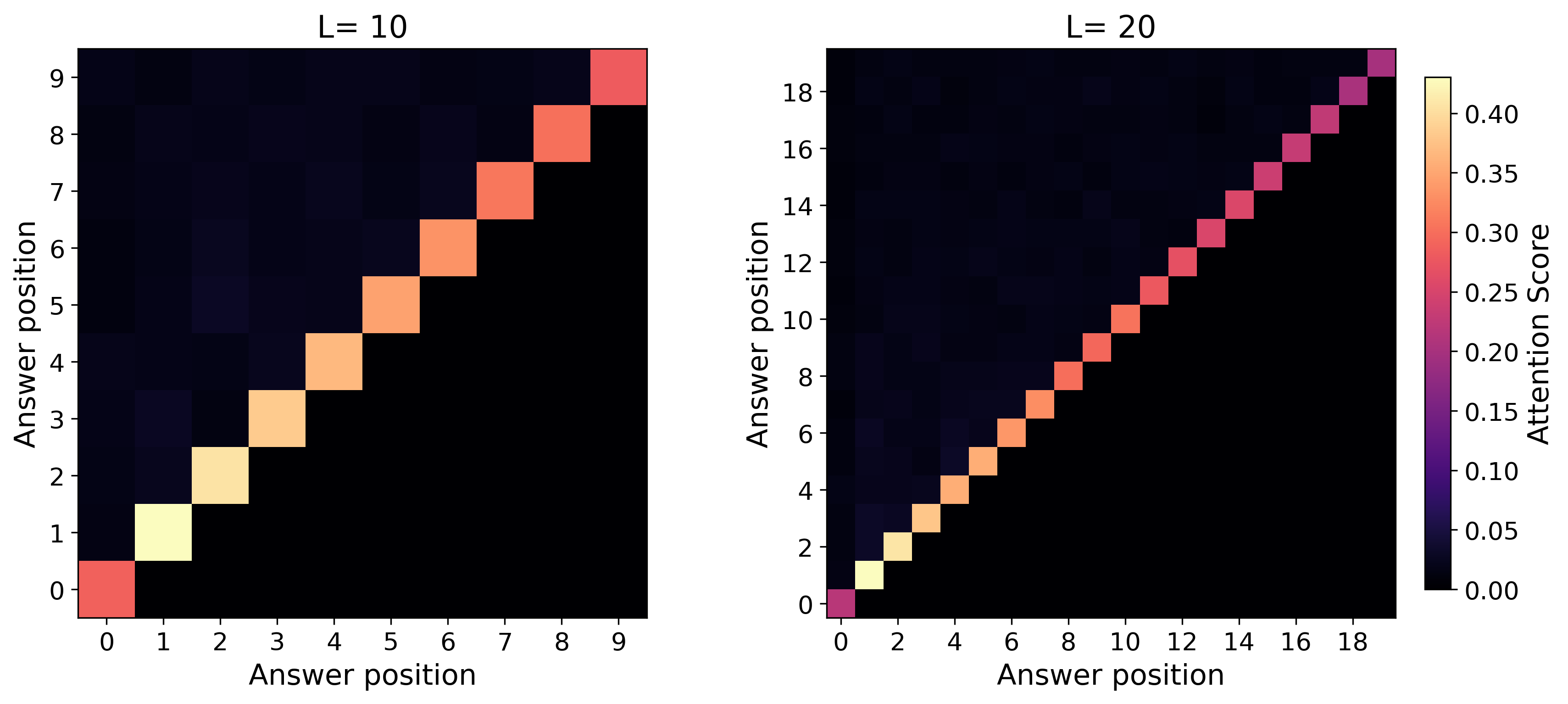}
    \end{minipage}
    \begin{minipage}[b]{0.4\textwidth}
      \centering
      \includegraphics[width=.9\linewidth]{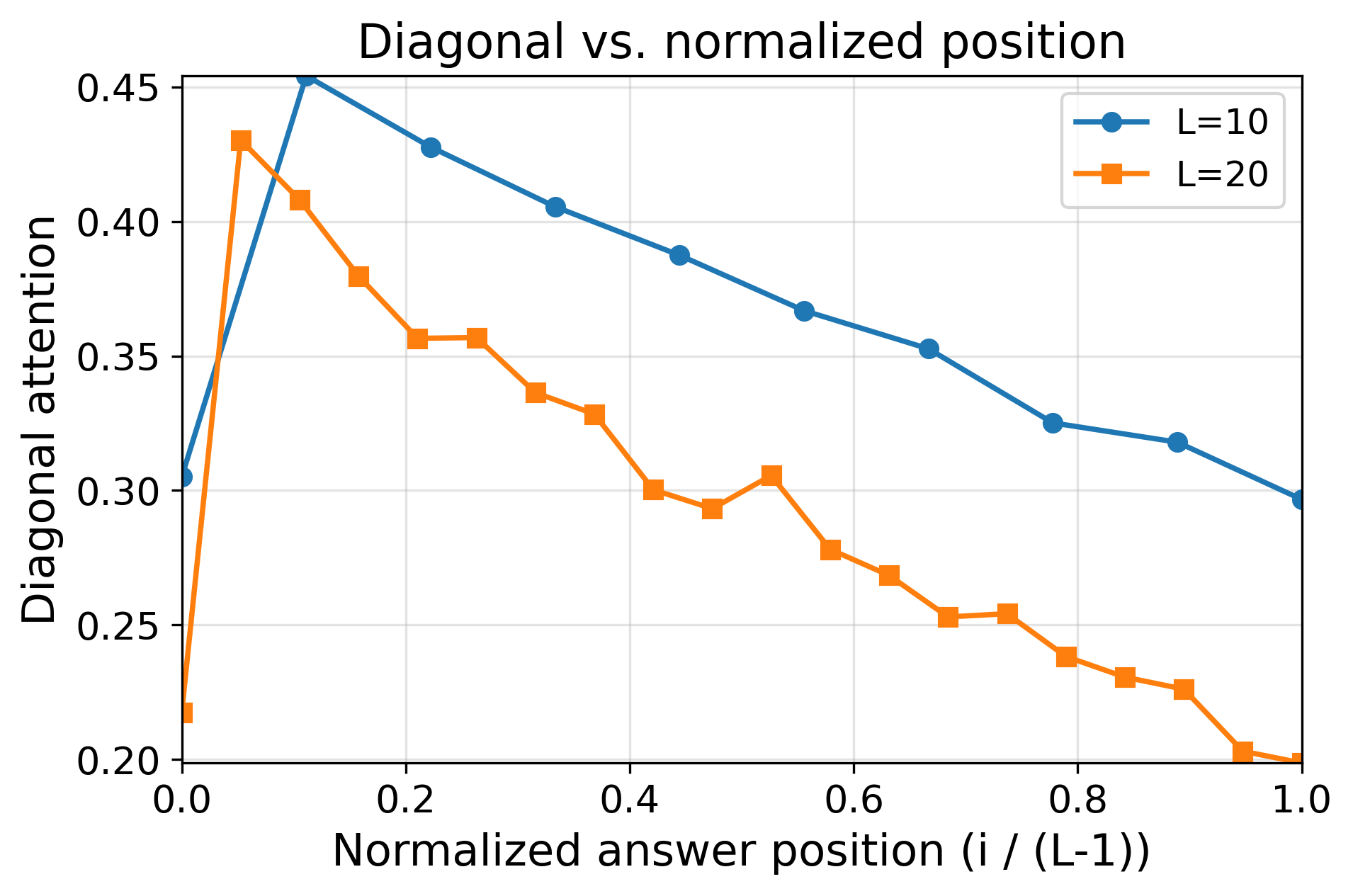}
    \end{minipage}
    \caption{\centering $C_6$}
    \label{fig:attn-con-diff-row1-ans}
  \end{subfigure}

  \vspace{0.5em}

  \begin{subfigure}[t]{\textwidth}
    \centering
    \begin{minipage}[b]{0.53\textwidth}
      \centering
      \includegraphics[width=\linewidth]{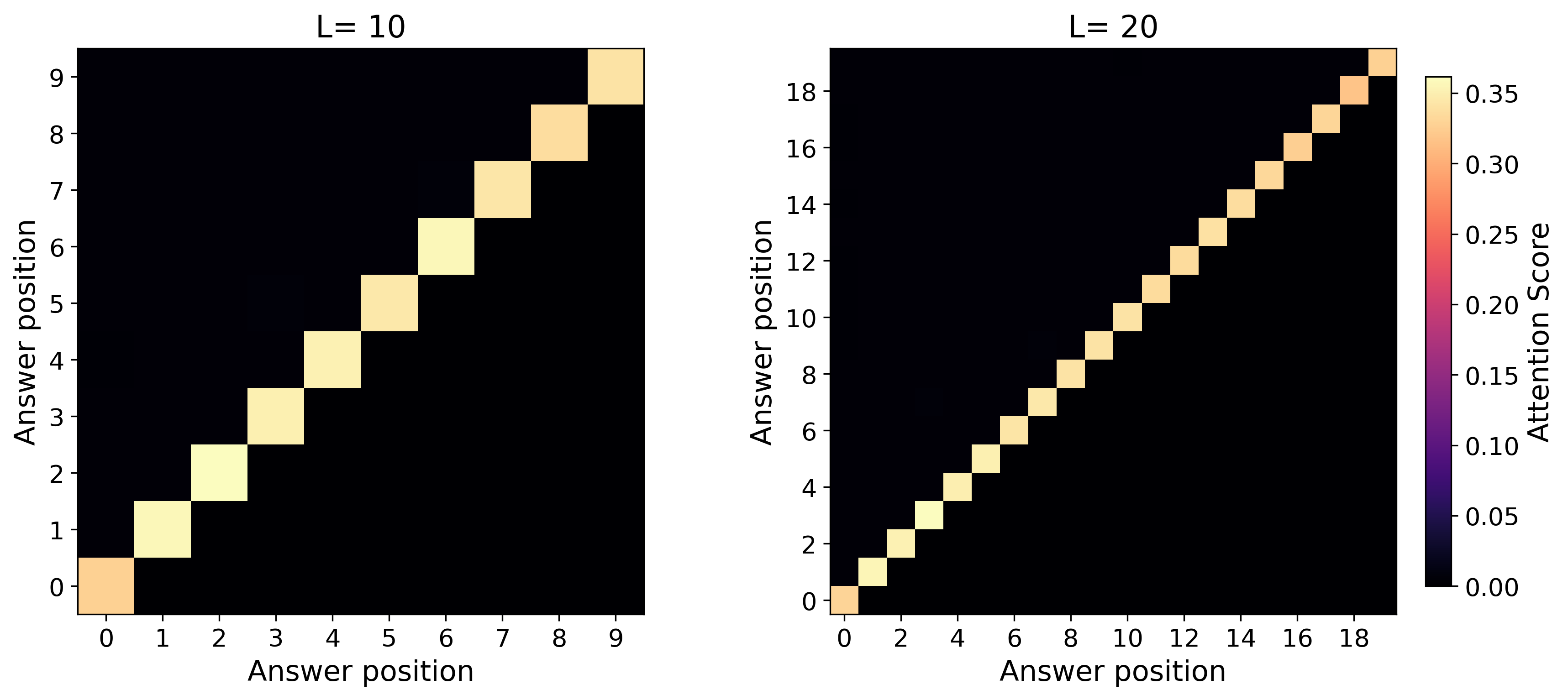}
    \end{minipage}
    \begin{minipage}[b]{0.4\textwidth}
      \centering
      \includegraphics[width=.9\linewidth]{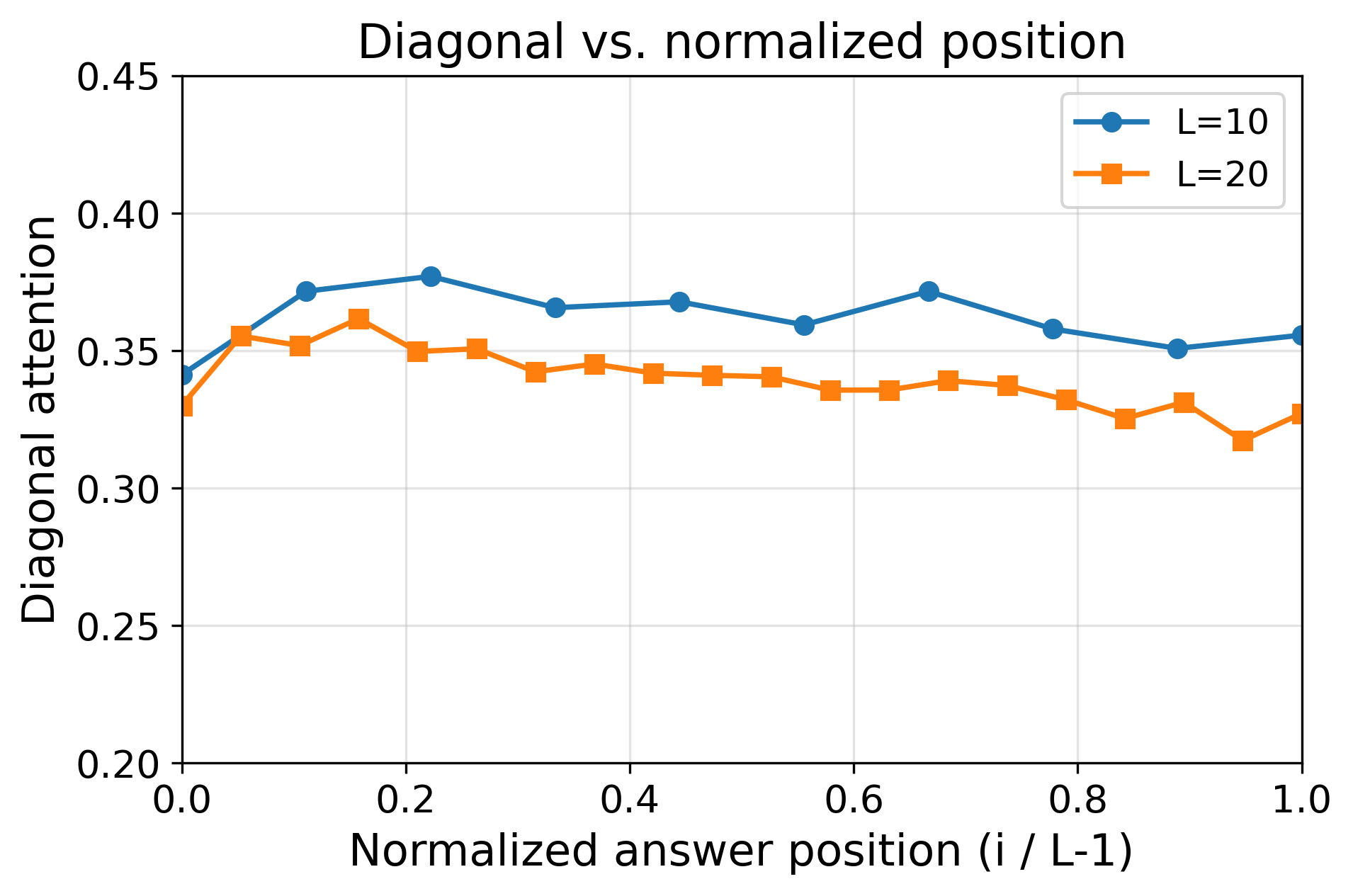}
    \end{minipage}
    \caption{\centering $S_5$}
    \label{fig:attn-con-diff-row2-ans}
  \end{subfigure}

  \caption{
    Answer-clause attention for models trained with LEGO task of length
    $L=5$ and tested with $L=10,20$.
    We extract the heatmap of the attention from answer queries to answer
    clauses (the upper-right part of the full attention heatmap, as in
    \Cref{fig:attn-con}), and line-plot the diagonal attention against the normalized answer position. In contrast to the predicate attention in \Cref{fig:attn-con-diff}, 
    at longer contexts, the symmetry group task preserves a stable attention for the answer clause, while the answer attention pattern of the cyclic group task is sensitive to the lengths and positions. Nevertheless, the sharp decay of attention scores for the cyclic group task did not result in a visible performance drop for length generalization, as seen in \Cref{fig:main-results}.
  }
  \label{fig:attn-con-diff-ans}
\end{figure}

\paragraph{Attention concentration.} We visualize attention heatmaps at convergence for models trained at length $L=5$ on both simply transitive and symmetry group tasks.
Each attention matrix is averaged across the two heads and over 100 independently
sampled LEGO sequences. Note that for a task of length \(L\), the LEGO sequence prior to the final
answer clause has length \(2L\). We focus on the query positions \(L{+}1,\ldots,2L\), which correspond to the
answer clauses \(Z_{\ans,0},\ldots,Z_{\ans,L-1}\) for predicting all outputs
\(y_1,\ldots,y_L\). 
\begin{itemize}
    \item In \Cref{fig:attn-con}, the two diagonal structures in the upper region indicate attention concentrating on the answer clause $Z_{\ans,\ell}$ and on the predicate clause $Z_{\pred,\ell+1}$ when the query is $Z_{\ans,\ell}$, validating the attention concentration principle highlighted by \Cref{thm:length-generalization} and \Cref{thm:length-gen-self-training}.
    \item  To ablate positional bias, we test on the trained model from
\Cref{fig:attn-con} and only change the input format by randomly permuting the
locations of predicate clauses.
Results in \Cref{fig:attn-con-perm} show that retrieval is keyed to the shared
variable rather than absolute position.
    \item \Cref{fig:attn-con-diff,fig:attn-con-diff-ans} probe out-of-length (OOL) behavior
through attention patterns at test lengths $L\in\{10,20\}$. As the line plot indicates, the attention scores for both the target predicate
and the answer clauses decrease as the task length increases from $10$ to
$20$.
This pattern matches the attention dilution predicted by our theory of length
generalization.  In \Cref{fig:attn-con-diff}, for the symmetry group task, attention to the
correct predicate $Z_{\pred,\ell+1}$ decays with position, becoming weaker at
later predicates (larger $\ell$, hence a larger absolute index $L{+}\ell$).
By contrast, the cyclic group task shows no comparable decay, with attention
remaining fairly uniform across lengths and positions.
In \Cref{fig:attn-con-diff-ans}, the pattern reverses: attention over answer
clauses is more stable for the symmetry group, whereas the cyclic case is more
sensitive to length and location.
We conjecture that length generalization in the LEGO task hinges on a robust
predicate–attention pattern.
Accordingly, the observed length-dependent decay of attention on predicate
clauses in the symmetry task indicates a non-robust learned attention
mechanism, which may underlie its poorer length generalization.
\end{itemize}

\section{Conclusions and Discussions}\label{sec:discuss}

In this paper, we have theoretically analyzed how length-generalizable reasoning ability emerges during gradient-descent training of transformers on synthetic CoT tasks. To the best of our knowledge, our results provide the first optimization theory for learning CoT with length generalization guarantees. 
For tasks associated with simply transitive actions, we have proven that transformers can directly generalize from short constant-length chains to substantially longer tasks. For tasks associated with symmetry actions, we have established convergence guarantees showing that transformers can bootstrap their CoT length through recursive self-training. By addressing inherently sequential tasks beyond $\mathsf{TC}^0$, our optimization theory bridges the gap with  known expressive power of transformers with CoT, while in the meantime rigorously validating the effectiveness of self-improvement training observed empirically.
We wrap up this paper with a few discussions.

\paragraph{NoPE, recency bias, and length generalization.} Empirically, standard positional embeddings often hinder length extrapolation, while NoPE has been advantageous considering its length-generalization performance~\cite{kazemnejad2023impact,Allenzhu2025-canon}. Prior works~\cite{kim2025transformersprovably,wen2025sparse,huang2025transformerslearn} typically adopt fixed positional encodings, which tie the learned computation to the training horizon. Intuitively, positional embeddings inject location-specific biases that favor local neighborhoods,  making longer inputs harder.  In our setting, length generalization is possible because the model retrieves
relevant information from long, unordered contexts by \emph{content} (variables)
rather than by position, as discussed in \Cref{sec:overview-attention} and
verified empirically in \Cref{fig:attn-con-perm}.
This provides concrete architectural guidance for practice and identifies positional embeddings as a plausible cause of observed failures to generalize to unseen lengths.
\paragraph{Local aggregation of tokens is beneficial.}
While NoPE is appealing for length generalization, it lacks  position-aware horizontal mixing within a layer, which often hurts empirical performance. In practice, many strong models adopt RoPE (Rotary Positional Encoding)~\cite{su2024roformer,peng2023yarn,ding2024longrope},
and recent seminal work~\cite{Allenzhu2025-canon} introduces the \emph{Cannon} layer, a short
convolutional window that adds local residual links across neighboring tokens, to
inject locality. Integrating Cannon layers substantially boosts the reasoning depth of NoPE
transformers in synthetic experiments. Our clause-based data can be interpreted as mimicking the
\emph{output} of a Cannon layer: locality is built directly into the data
structure, enabling one-layer transformers to learn induction heads, which otherwise require at least a two-layer transformer to express.
\paragraph{Context rot.}  Context rot is a phenomenon that the performance of LLMs drops as their input context grows even when the task is unchanged, which has been widely noted in practice~\cite{hong2025context}. A prevailing empirical view is that longer contexts introduce many \emph{distractors} and other irrelevant tokens,
forcing the relevant signal to compete with them for attention and thereby weakening
retrieval. Our attention concentration perspective characterizes this competition at a
fine-grained level: we show that as irrelevant clauses accumulate, attention mass
is diluted away from the correct clause, reducing performance at extended reasoning lengths. This offers a simple theoretical perspective on the origin of context rot and why long-context training hasn't been as effective in eliminating the problem. We expect the insight from our analysis to extend to richer tasks and inspire practical mitigation strategies, e.g., \emph{context engineering}~\cite{mei2025survey}.

\section*{Acknowledgments}

The work of Z.~Wen is supported in part by NSF DMS-2134080 and DMS-2134133. Y.~Chen is supported in part by the Alfred P.~Sloan Research Fellowship, 
the NSF grants IIS-2218713 and IIS-2218773, the ONR grants N00014-22-1-2354 and N00014-25-1-2344, the Wharton AI \& Analytics Initiative's AI Research Fund, 
and the Amazon Research Award. Z.~Wen thanks Yuanzhi Li for initial contributions and helpful discussions on this project.

\bibliographystyle{alphaabbr}
\bibliography{cot}

\newcommand{\etalchar}[1]{$^{#1}$}
\begin{thebibliography}{WWS{\etalchar{+}}22b}

\bibitem[AB06]{arora2006computational}
S.~Arora and B.~Barak.
\newblock {\em Computational Complexity: A Modern Approach}.
\newblock Cambridge University Press, 2006.

\bibitem[ABL{\etalchar{+}}24]{abbe2024fartransformers}
E.~Abbe, S.~Bengio, A.~Lotfi, C.~Sandon, and O.~Saremi.
\newblock How far can transformers reason? the globality barrier and inductive scratchpad.
\newblock {\em Advances in Neural Information Processing Systems}, 37:27850--27895, 2024.

\bibitem[{All}25]{Allenzhu2025-canon}
Z.~{Allen-Zhu}.
\newblock {Physics of Language Models: Part 4.1, Architecture Design and the Magic of Canon Layers}.
\newblock {\em SSRN Electronic Journal}, May 2025.
\newblock \url{https://ssrn.com/abstract=5240330}.

\bibitem[AM24]{ahuja2024provable}
K.~Ahuja and A.~Mansouri.
\newblock On provable length and compositional generalization.
\newblock {\em arXiv preprint arXiv:2402.04875}, 2024.

\bibitem[AWA{\etalchar{+}}22]{anil2022exploring}
C.~Anil, Y.~Wu, A.~Andreassen, A.~Lewkowycz, V.~Misra, V.~Ramasesh, A.~Slone, G.~Gur-Ari, E.~Dyer, and B.~Neyshabur.
\newblock Exploring length generalization in large language models.
\newblock {\em Advances in Neural Information Processing Systems}, 35:38546--38556, 2022.

\bibitem[AZL20]{allen2020towards}
Z.~Allen-Zhu and Y.~Li.
\newblock Towards understanding ensemble, knowledge distillation and self-distillation in deep learning.
\newblock {\em arXiv preprint arXiv:2012.09816}, 2020.

\bibitem[AZL21]{allen2021forward}
Z.~Allen-Zhu and Y.~Li.
\newblock Forward super-resolution: How can {GANs} learn hierarchical generative models for real-world distributions.
\newblock {\em arXiv preprint arXiv:2106.02619}, 2021.

\bibitem[AZL22]{allen2022feature}
Z.~Allen-Zhu and Y.~Li.
\newblock Feature purification: How adversarial training performs robust deep learning.
\newblock In {\em 2021 IEEE 62nd Annual Symposium on Foundations of Computer Science (FOCS)}, pages 977--988. IEEE, 2022.

\bibitem[Bar86]{barrington1986bounded}
D.~A. Barrington.
\newblock Bounded-width polynomial-size branching programs recognize exactly those languages in nc.
\newblock In {\em Proceedings of the eighteenth annual ACM symposium on Theory of computing}, pages 1--5, 1986.

\bibitem[BMR{\etalchar{+}}20]{brown2020language}
T.~Brown, B.~Mann, N.~Ryder, M.~Subbiah, J.~D. Kaplan, P.~Dhariwal, A.~Neelakantan, P.~Shyam, G.~Sastry, and A.~Askell.
\newblock Language models are few-shot learners.
\newblock {\em Advances in neural information processing systems}, 33:1877--1901, 2020.

\bibitem[CHX{\etalchar{+}}25]{cheng2025transformers}
Y.~Cheng, Y.~Huang, Z.~Xiong, Y.~Liang, and V.~Y. Tan.
\newblock Transformers provably learn directed acyclic graphs via kernel-guided mutual information.
\newblock {\em arXiv preprint arXiv:2510.25542}, 2025.

\bibitem[CKB{\etalchar{+}}21]{cobbe2021training}
K.~Cobbe, V.~Kosaraju, M.~Bavarian, M.~Chen, H.~Jun, L.~Kaiser, M.~Plappert, J.~Tworek, J.~Hilton, and R.~Nakano.
\newblock Training verifiers to solve math word problems.
\newblock {\em arXiv preprint arXiv:2110.14168}, 2021.

\bibitem[CND{\etalchar{+}}23]{chowdhery2023palm}
A.~Chowdhery, S.~Narang, J.~Devlin, M.~Bosma, G.~Mishra, A.~Roberts, P.~Barham, H.~W. Chung, C.~Sutton, and S.~Gehrmann.
\newblock {PaLM}: Scaling language modeling with pathways.
\newblock {\em Journal of Machine Learning Research}, 24(240):1--113, 2023.

\bibitem[CPW24]{chen2024theoretical}
L.~Chen, B.~Peng, and H.~Wu.
\newblock Theoretical limitations of multi-layer transformer.
\newblock {\em arXiv preprint arXiv:2412.02975}, 2024.

\bibitem[CSV84]{chandra1984constant}
A.~K. Chandra, L.~Stockmeyer, and U.~Vishkin.
\newblock Constant depth reducibility.
\newblock {\em SIAM Journal on Computing}, 13(2):423--439, 1984.

\bibitem[CTJ{\etalchar{+}}21]{chen2021evaluating}
M.~Chen, J.~Tworek, H.~Jun, Q.~Yuan, H.~P. D.~O. Pinto, J.~Kaplan, H.~Edwards, Y.~Burda, N.~Joseph, and G.~Brockman.
\newblock Evaluating large language models trained on code.
\newblock {\em arXiv preprint arXiv:2107.03374}, 2021.

\bibitem[DA25]{deepseekai2025deepseekr1}
DeepSeek-AI.
\newblock {DeepSeek-R1}: Incentivizing reasoning capability in {LLMs} via reinforcement learning.
\newblock {\em ArXiv}, abs/2501.12948, 2025.

\bibitem[Deh11]{dehn1911unendliche}
M.~Dehn.
\newblock {\"U}ber unendliche diskontinuierliche gruppen.
\newblock {\em Mathematische Annalen}, 71(1):116--144, 1911.

\bibitem[Deh87]{dehn1987infinite}
M.~Dehn.
\newblock On infinite discontinuous groups.
\newblock In {\em Papers on Group Theory and Topology}, pages 133--178. Springer, 1987.

\bibitem[DLS{\etalchar{+}}23]{dziri2023faith}
N.~Dziri, X.~Lu, M.~Sclar, X.~L. Li, L.~Jiang, B.~Y. Lin, S.~Welleck, P.~West, C.~Bhagavatula, and R.~Le~Bras.
\newblock Faith and fate: Limits of transformers on compositionality.
\newblock {\em Advances in Neural Information Processing Systems}, 36:70293--70332, 2023.

\bibitem[DZZ{\etalchar{+}}24]{ding2024longrope}
Y.~Ding, L.~L. Zhang, C.~Zhang, Y.~Xu, N.~Shang, J.~Xu, F.~Yang, and M.~Yang.
\newblock Longrope: Extending llm context window beyond 2 million tokens.
\newblock {\em arXiv preprint arXiv:2402.13753}, 2024.

\bibitem[ENO{\etalchar{+}}21]{elhage2021mathematical}
N.~Elhage, N.~Nanda, C.~Olsson, T.~Henighan, N.~Joseph, B.~Mann, A.~Askell, Y.~Bai, A.~Chen, T.~Conerly, N.~DasSarma, D.~Drain, D.~Ganguli, Z.~Hatfield-Dodds, D.~Hernandez, A.~Jones, J.~Kernion, L.~Lovitt, K.~Ndousse, D.~Amodei, T.~Brown, J.~Clark, J.~Kaplan, S.~McCandlish, and C.~Olah.
\newblock A mathematical framework for transformer circuits.
\newblock {\em Transformer Circuits Thread}, 2021.
\newblock https://transformer-circuits.pub/2021/framework/index.html.

\bibitem[Fra23]{fraleigh2023first}
J.~B. Fraleigh.
\newblock {\em A first course in abstract algebra (8th Edition)}.
\newblock Pearson India, 2023.

\bibitem[FZG{\etalchar{+}}23]{feng2023revealing}
G.~Feng, B.~Zhang, Y.~Gu, H.~Ye, D.~He, and L.~Wang.
\newblock Towards revealing the mystery behind chain of thought: a theoretical perspective.
\newblock {\em Advances in Neural Information Processing Systems}, 36:70757--70798, 2023.

\bibitem[GJB{\etalchar{+}}25]{golowich2025role}
N.~Golowich, S.~Jelassi, D.~Brandfonbrener, S.~M. Kakade, and E.~Malach.
\newblock The role of sparsity for length generalization in transformers.
\newblock {\em arXiv preprint arXiv:2502.16792}, 2025.

\bibitem[GJKM25]{gao2025comparison}
T.~Gao, J.~Jin, Z.~T. Ke, and G.~Moryoussef.
\newblock A comparison of {DeepSeek} and other {LLMs}.
\newblock {\em arXiv preprint arXiv:2502.03688}, 2025.

\bibitem[GLL87]{godbeer1987computational}
G.~H. Godbeer, J.~Lipscomb, and M.~G. Luby.
\newblock {\em On the computational complexity of finding stable state vectors in connectionist models (Hopfield nets)}.
\newblock Department of Computer Science, University of Toronto, 1987.

\bibitem[GPS{\etalchar{+}}23]{gulcehre2023reinforced}
C.~Gulcehre, T.~L. Paine, S.~Srinivasan, K.~Konyushkova, L.~Weerts, A.~Sharma, A.~Siddhant, A.~Ahern, M.~Wang, and C.~Gu.
\newblock Reinforced self-training (rest) for language modeling.
\newblock {\em arXiv preprint arXiv:2308.08998}, 2023.

\bibitem[GZA{\etalchar{+}}23]{gunasekar2023textbooks}
S.~Gunasekar, Y.~Zhang, J.~Aneja, C.~C.~T. Mendes, A.~Del~Giorno, S.~Gopi, M.~Javaheripi, P.~Kauffmann, G.~de~Rosa, and O.~Saarikivi.
\newblock Textbooks are all you need.
\newblock {\em arXiv preprint arXiv:2306.11644}, 2023.

\bibitem[HBF{\etalchar{+}}24]{huang2024self}
A.~Huang, A.~Block, D.~J. Foster, D.~Rohatgi, C.~Zhang, M.~Simchowitz, J.~T. Ash, and A.~Krishnamurthy.
\newblock Self-improvement in language models: The sharpening mechanism.
\newblock {\em arXiv preprint arXiv:2412.01951}, 2024.

\bibitem[HBK{\etalchar{+}}24]{hou2024universal}
K.~Hou, D.~Brandfonbrener, S.~Kakade, S.~Jelassi, and E.~Malach.
\newblock Universal length generalization with turing programs.
\newblock {\em arXiv preprint arXiv:2407.03310}, 2024.

\bibitem[HCL24]{huang2023context}
Y.~Huang, Y.~Cheng, and Y.~Liang.
\newblock In-context convergence of transformers.
\newblock {\em International Conference on Machine Learning}, pages 19660--19722, 2024.

\bibitem[HLY25]{huang2025transformerslearn}
R.~Huang, Y.~Liang, and J.~Yang.
\newblock How transformers learn regular language recognition: A theoretical study on training dynamics and implicit bias.
\newblock {\em arXiv preprint arXiv:2505.00926}, 2025.

\bibitem[HLZ{\etalchar{+}}22]{huang2022modality}
Y.~Huang, J.~Lin, C.~Zhou, H.~Yang, and L.~Huang.
\newblock Modality competition: What makes joint training of multi-modal network fail in deep learning?(provably).
\newblock In {\em International conference on machine learning}, pages 9226--9259. PMLR, 2022.

\bibitem[HSK{\etalchar{+}}24]{hsieh2024ruler}
C.-P. Hsieh, S.~Sun, S.~Kriman, S.~Acharya, D.~Rekesh, F.~Jia, Y.~Zhang, and B.~Ginsburg.
\newblock Ruler: What's the real context size of your long-context language models?
\newblock {\em arXiv preprint arXiv:2404.06654}, 2024.

\bibitem[HSY22]{ho2022large}
N.~Ho, L.~Schmid, and S.-Y. Yun.
\newblock Large language models are reasoning teachers.
\newblock {\em arXiv preprint arXiv:2212.10071}, 2022.

\bibitem[HTH25]{hong2025context}
K.~Hong, A.~Troynikov, and J.~Huber.
\newblock Context rot: How increasing input tokens impacts {LLM} performance.
\newblock Technical report, Chroma, July 2025.

\bibitem[HWCL25]{huang2025a}
Y.~Huang, Z.~Wen, Y.~Chi, and Y.~Liang.
\newblock A theoretical analysis of self-supervised learning for vision transformers.
\newblock In {\em The Thirteenth International Conference on Learning Representations}, 2025.

\bibitem[HWL25]{huang2025transformers}
J.~Huang, Z.~Wang, and J.~D. Lee.
\newblock Transformers learn to implement multi-step gradient descent with chain of thought.
\newblock {\em arXiv preprint arXiv:2502.21212}, 2025.

\bibitem[HYB{\etalchar{+}}24]{huang2024formal}
X.~Huang, A.~Yang, S.~Bhattamishra, Y.~Sarrof, A.~Krebs, H.~Zhou, P.~Nakkiran, and M.~Hahn.
\newblock A formal framework for understanding length generalization in transformers.
\newblock {\em arXiv preprint arXiv:2410.02140}, 2024.

\bibitem[HZCY24]{hu2024unveiling}
X.~Hu, F.~Zhang, S.~Chen, and Z.~Yang.
\newblock Unveiling the statistical foundations of chain-of-thought prompting methods.
\newblock {\em arXiv preprint arXiv:2408.14511}, 2024.

\bibitem[JB25]{jones2025large}
C.~R. Jones and B.~K. Bergen.
\newblock Large language models pass the turing test.
\newblock {\em arXiv preprint arXiv:2503.23674}, 2025.

\bibitem[JdDE{\etalchar{+}}23]{jelassi2023length}
S.~Jelassi, S.~d'Ascoli, C.~Domingo-Enrich, Y.~Wu, Y.~Li, and F.~Charton.
\newblock Length generalization in arithmetic transformers.
\newblock {\em arXiv preprint arXiv:2306.15400}, 2023.

\bibitem[JSL22]{jelassi2022vision}
S.~Jelassi, M.~Sander, and Y.~Li.
\newblock Vision transformers provably learn spatial structure.
\newblock {\em Advances in Neural Information Processing Systems}, 35:37822--37836, 2022.

\bibitem[JVB{\etalchar{+}}25]{joshi2025theory}
N.~Joshi, G.~Vardi, A.~Block, S.~Goel, Z.~Li, T.~Misiakiewicz, and N.~Srebro.
\newblock A theory of learning with autoregressive chain of thought.
\newblock {\em arXiv preprint arXiv:2503.07932}, 2025.

\bibitem[KB14]{Kingma2014AdamAM}
D.~P. Kingma and J.~Ba.
\newblock Adam: A method for stochastic optimization.
\newblock {\em CoRR}, abs/1412.6980, 2014.

\bibitem[KBA{\etalchar{+}}24]{kuratov2024babilong}
Y.~Kuratov, A.~Bulatov, P.~Anokhin, I.~Rodkin, D.~Sorokin, A.~Sorokin, and M.~Burtsev.
\newblock Babilong: Testing the limits of llms with long context reasoning-in-a-haystack.
\newblock {\em Advances in Neural Information Processing Systems}, 37:106519--106554, 2024.

\bibitem[KGR{\etalchar{+}}22]{kojima2022large}
T.~Kojima, S.~S. Gu, M.~Reid, Y.~Matsuo, and Y.~Iwasawa.
\newblock Large language models are zero-shot reasoners.
\newblock {\em Advances in neural information processing systems}, 35:22199--22213, 2022.

\bibitem[KPNR{\etalchar{+}}23]{kazemnejad2023impact}
A.~Kazemnejad, I.~Padhi, K.~Natesan~Ramamurthy, P.~Das, and S.~Reddy.
\newblock The impact of positional encoding on length generalization in transformers.
\newblock {\em Advances in Neural Information Processing Systems}, 36:24892--24928, 2023.

\bibitem[KS23]{kim2023entity}
N.~Kim and S.~Schuster.
\newblock Entity tracking in language models.
\newblock {\em arXiv preprint arXiv:2305.02363}, 2023.

\bibitem[KS24]{kim2025transformersprovably}
J.~Kim and T.~Suzuki.
\newblock Transformers provably solve parity efficiently with chain of thought.
\newblock {\em arXiv preprint arXiv:2410.08633}, 2024.

\bibitem[KWLS25]{kim2025metastable}
J.~Kim, D.~Wu, J.~Lee, and T.~Suzuki.
\newblock Metastable dynamics of chain-of-thought reasoning: Provable benefits of search, rl and distillation.
\newblock {\em arXiv preprint arXiv:2502.01694}, 2025.

\bibitem[KZA{\etalchar{+}}24]{kumar2024training}
A.~Kumar, V.~Zhuang, R.~Agarwal, Y.~Su, J.~D. Co-Reyes, A.~Singh, K.~Baumli, S.~Iqbal, C.~Bishop, and R.~Roelofs.
\newblock Training language models to self-correct via reinforcement learning.
\newblock {\em arXiv preprint arXiv:2409.12917}, 2024.

\bibitem[LAD{\etalchar{+}}22]{lewkowycz2022solving}
A.~Lewkowycz, A.~Andreassen, D.~Dohan, E.~Dyer, H.~Michalewski, V.~Ramasesh, A.~Slone, C.~Anil, I.~Schlag, and T.~Gutman-Solo.
\newblock Solving quantitative reasoning problems with language models.
\newblock {\em Advances in Neural Information Processing Systems}, 35:3843--3857, 2022.

\bibitem[LAG{\etalchar{+}}22]{liu2023transformers}
B.~Liu, J.~T. Ash, S.~Goel, A.~Krishnamurthy, and C.~Zhang.
\newblock Transformers learn shortcuts to automata.
\newblock {\em arXiv preprint arXiv:2210.10749}, 2022.

\bibitem[LCS{\etalchar{+}}25]{lee2025selfimproving}
N.~Lee, Z.~Cai, A.~Schwarzschild, K.~Lee, and D.~Papailiopoulos.
\newblock Self-improving transformers overcome easy-to-hard and length generalization challenges.
\newblock {\em arXiv preprint arXiv:2502.01612}, 2025.

\bibitem[LGA25]{Li2025HowDL}
B.~Z. Li, Z.~C. Guo, and J.~Andreas.
\newblock (how) do language models track state?
\newblock {\em ArXiv}, abs/2503.02854, 2025.

\bibitem[LLY{\etalchar{+}}25]{ling2025longreason}
Z.~Ling, K.~Liu, K.~Yan, Y.~Yang, W.~Lin, T.-H. Fan, L.~Shen, Z.~Du, and J.~Chen.
\newblock Longreason: A synthetic long-context reasoning benchmark via context expansion.
\newblock {\em arXiv preprint arXiv:2501.15089}, 2025.

\bibitem[LLZM24]{li2024chainthought}
Z.~Li, H.~Liu, D.~Zhou, and T.~Ma.
\newblock Chain of thought empowers transformers to solve inherently serial problems.
\newblock {\em arXiv preprint arXiv:2402.12875}, 1, 2024.

\bibitem[LPW{\etalchar{+}}17]{lu2017expressive}
Z.~Lu, H.~Pu, F.~Wang, Z.~Hu, and L.~Wang.
\newblock The expressive power of neural networks: A view from the width.
\newblock {\em Advances in neural information processing systems}, 30, 2017.

\bibitem[LSG{\etalchar{+}}23]{li2023dissecting}
Y.~Li, K.~Sreenivasan, A.~Giannou, D.~Papailiopoulos, and S.~Oymak.
\newblock Dissecting chain-of-thought: Compositionality through in-context filtering and learning.
\newblock {\em Advances in Neural Information Processing Systems}, 36:22021--22046, 2023.

\bibitem[Maa96]{maass1996lower}
W.~Maass.
\newblock Lower bounds for the computational power of networks of spiking neurons.
\newblock {\em Neural computation}, 8(1):1--40, 1996.

\bibitem[MCJ{\etalchar{+}}24]{min2024imitate}
Y.~Min, Z.~Chen, J.~Jiang, J.~Chen, J.~Deng, Y.~Hu, Y.~Tang, J.~Wang, X.~Cheng, and H.~Song.
\newblock Imitate, explore, and self-improve: A reproduction report on slow-thinking reasoning systems.
\newblock {\em arXiv preprint arXiv:2412.09413}, 2024.

\bibitem[MDA{\etalchar{+}}24]{marsden2024provable}
A.~Marsden, E.~Dogariu, N.~Agarwal, X.~Chen, D.~Suo, and E.~Hazan.
\newblock Provable length generalization in sequence prediction via spectral filtering.
\newblock {\em arXiv preprint arXiv:2411.01035}, 2024.

\bibitem[MMA{\etalchar{+}}22]{magister2022teaching}
L.~C. Magister, J.~Mallinson, J.~Adamek, E.~Malmi, and A.~Severyn.
\newblock Teaching small language models to reason.
\newblock {\em arXiv preprint arXiv:2212.08410}, 2022.

\bibitem[MPS24]{merrill2024illusion}
W.~Merrill, J.~Petty, and A.~Sabharwal.
\newblock The illusion of state in state-space models.
\newblock {\em arXiv preprint arXiv:2404.08819}, 2024.

\bibitem[MS23a]{merrill2024expressive}
W.~Merrill and A.~Sabharwal.
\newblock The expressive power of transformers with chain of thought.
\newblock {\em arXiv preprint arXiv:2310.07923}, 2023.

\bibitem[MS23b]{merrill2023parallelism}
W.~Merrill and A.~Sabharwal.
\newblock The parallelism tradeoff: Limitations of log-precision transformers.
\newblock {\em Transactions of the Association for Computational Linguistics}, 11:531--545, 2023.

\bibitem[MS25a]{merrill2025exact}
W.~Merrill and A.~Sabharwal.
\newblock Exact expressive power of transformers with padding.
\newblock {\em arXiv preprint arXiv:2505.18948}, 2025.

\bibitem[MS25b]{merrill2025little}
W.~Merrill and A.~Sabharwal.
\newblock A little depth goes a long way: The expressive power of log-depth transformers.
\newblock {\em arXiv preprint arXiv:2503.03961}, 2025.

\bibitem[MSS22]{merrill2022saturated}
W.~Merrill, A.~Sabharwal, and N.~A. Smith.
\newblock Saturated transformers are constant-depth threshold circuits.
\newblock {\em Transactions of the Association for Computational Linguistics}, 10:843--856, 2022.

\bibitem[MY22]{madaan2022text}
A.~Madaan and A.~Yazdanbakhsh.
\newblock Text and patterns: For effective chain of thought, it takes two to tango.
\newblock {\em arXiv preprint arXiv:2209.07686}, 2022.

\bibitem[MYG{\etalchar{+}}25]{mei2025survey}
L.~Mei, J.~Yao, Y.~Ge, Y.~Wang, B.~Bi, Y.~Cai, J.~Liu, M.~Li, Z.-Z. Li, and D.~Zhang.
\newblock A survey of context engineering for large language models.
\newblock {\em arXiv preprint arXiv:2507.13334}, 2025.

\bibitem[NDL24]{nichani2024transformers}
E.~Nichani, A.~Damian, and J.~D. Lee.
\newblock How transformers learn causal structure with gradient descent.
\newblock {\em arXiv preprint arXiv:2402.14735}, 2024.

\bibitem[Ope24]{openai2024o1card}
OpenAI.
\newblock {OpenAI} o1 system card.
\newblock {\em ArXiv}, abs/2412.16720, 2024.

\bibitem[Par94]{parberry1994circuit}
I.~Parberry.
\newblock {\em Circuit complexity and neural networks}.
\newblock MIT press, 1994.

\bibitem[PLG23]{prystawski2023think}
B.~Prystawski, M.~Li, and N.~Goodman.
\newblock Why think step by step? reasoning emerges from the locality of experience.
\newblock {\em Advances in Neural Information Processing Systems}, 36:70926--70947, 2023.

\bibitem[PQFS23]{peng2023yarn}
B.~Peng, J.~Quesnelle, H.~Fan, and E.~Shippole.
\newblock Yarn: Efficient context window extension of large language models.
\newblock {\em arXiv preprint arXiv:2309.00071}, 2023.

\bibitem[QLZ{\etalchar{+}}24]{qin2024o1}
Y.~Qin, X.~Li, H.~Zou, Y.~Liu, S.~Xia, Z.~Huang, Y.~Ye, W.~Yuan, H.~Liu, and Y.~Li.
\newblock O1 replication journey: A strategic progress report--part 1.
\newblock {\em arXiv preprint arXiv:2410.18982}, 2024.

\bibitem[Rin16]{Rinaldo2016Lecture27}
A.~Rinaldo.
\newblock 36-755: Advanced statistical theory, lecture 27.
\newblock Lecture notes, December 5 2016.
\newblock Scribed by Xiao Hui Tai.

\bibitem[RM21]{reynolds2021prompt}
L.~Reynolds and K.~McDonell.
\newblock Prompt programming for large language models: Beyond the few-shot paradigm.
\newblock In {\em Extended abstracts of the 2021 CHI conference on human factors in computing systems}, pages 1--7, 2021.

\bibitem[RS98]{raab1998balls}
M.~Raab and A.~Steger.
\newblock “balls into bins”—a simple and tight analysis.
\newblock In {\em International Workshop on Randomization and Approximation Techniques in Computer Science}, pages 159--170. Springer, 1998.

\bibitem[RSR{\etalchar{+}}20]{raffel2020exploring}
C.~Raffel, N.~Shazeer, A.~Roberts, K.~Lee, S.~Narang, M.~Matena, Y.~Zhou, W.~Li, and P.~J. Liu.
\newblock Exploring the limits of transfer learning with a unified text-to-text transformer.
\newblock {\em Journal of machine learning research}, 21(140):1--67, 2020.

\bibitem[SAL{\etalchar{+}}24]{su2024roformer}
J.~Su, M.~Ahmed, Y.~Lu, S.~Pan, W.~Bo, and Y.~Liu.
\newblock Roformer: Enhanced transformer with rotary position embedding.
\newblock {\em Neurocomputing}, 568:127063, 2024.

\bibitem[SCRA{\etalchar{+}}23]{singh2024human}
A.~Singh, J.~D. Co-Reyes, R.~Agarwal, A.~Anand, P.~Patil, X.~Garcia, P.~J. Liu, J.~Harrison, J.~Lee, and K.~Xu.
\newblock Beyond human data: Scaling self-training for problem-solving with language models.
\newblock {\em arXiv preprint arXiv:2312.06585}, 2023.

\bibitem[SFH{\etalchar{+}}24]{sanford2024understanding}
C.~Sanford, B.~Fatemi, E.~Hall, A.~Tsitsulin, M.~Kazemi, J.~Halcrow, B.~Perozzi, and V.~Mirrokni.
\newblock Understanding transformer reasoning capabilities via graph algorithms.
\newblock {\em Advances in Neural Information Processing Systems}, 37:78320--78370, 2024.

\bibitem[SLZT25]{sun2025theoretical}
Y.~Sun, Y.~Liang, Z.~Zhang, and J.~Teng.
\newblock Theoretical modeling of llm self-improvement training dynamics through solver-verifier gap.
\newblock {\em arXiv preprint arXiv:2507.00075}, 2025.

\bibitem[{\v{S}}O03]{vsima2003general}
J.~{\v{S}}{\'\i}ma and P.~Orponen.
\newblock General-purpose computation with neural networks: A survey of complexity theoretic results.
\newblock {\em Neural Computation}, 15(12):2727--2778, 2003.

\bibitem[SPHG24]{sabbaghi2024explicitly}
M.~Sabbaghi, G.~Pappas, H.~Hassani, and S.~Goel.
\newblock Explicitly encoding structural symmetry is key to length generalization in arithmetic tasks.
\newblock {\em arXiv preprint arXiv:2406.01895}, 2024.

\bibitem[SS92]{siegelmann1992computational}
H.~T. Siegelmann and E.~D. Sontag.
\newblock On the computational power of neural nets.
\newblock In {\em Proceedings of the fifth annual workshop on Computational learning theory}, pages 440--449, 1992.

\bibitem[STAK92]{shawe1992classes}
J.~S. Shawe-Taylor, M.~H. Anthony, and W.~Kern.
\newblock Classes of feedforward neural networks and their circuit complexity.
\newblock {\em Neural networks}, 5(6):971--977, 1992.

\bibitem[SUV18]{shaw2018self}
P.~Shaw, J.~Uszkoreit, and A.~Vaswani.
\newblock Self-attention with relative position representations.
\newblock {\em arXiv preprint arXiv:1803.02155}, 2018.

\bibitem[SZE{\etalchar{+}}24]{song2024mind}
Y.~Song, H.~Zhang, C.~Eisenach, S.~Kakade, D.~Foster, and U.~Ghai.
\newblock Mind the gap: Examining the self-improvement capabilities of large language models.
\newblock {\em arXiv preprint arXiv:2412.02674}, 2024.

\bibitem[tea25]{anthropic2025effective}
A.~A.~A. team.
\newblock Effective context engineering for ai agents.
\newblock Technical report, Anthropic, 2025.

\bibitem[TLI{\etalchar{+}}23]{touvron2023llama}
H.~Touvron, T.~Lavril, G.~Izacard, X.~Martinet, M.-A. Lachaux, T.~Lacroix, B.~Rozi{\`e}re, N.~Goyal, E.~Hambro, and F.~Azhar.
\newblock Llama: Open and efficient foundation language models.
\newblock {\em arXiv preprint arXiv:2302.13971}, 2023.

\bibitem[Vol99]{vollmer1999introduction}
H.~Vollmer.
\newblock {\em Introduction to circuit complexity: a uniform approach}.
\newblock Springer Science \& Business Media, 1999.

\bibitem[VSP{\etalchar{+}}17]{Vaswani2017AttentionIA}
A.~Vaswani, N.~M. Shazeer, N.~Parmar, J.~Uszkoreit, L.~Jones, A.~N. Gomez, L.~Kaiser, and I.~Polosukhin.
\newblock Attention is all you need.
\newblock In {\em Neural Information Processing Systems}, 2017.

\bibitem[WL21]{wen2021toward}
Z.~Wen and Y.~Li.
\newblock Toward understanding the feature learning process of self-supervised contrastive learning.
\newblock In {\em International Conference on Machine Learning}, pages 11112--11122. PMLR, 2021.

\bibitem[WL22]{wen2022mechanism}
Z.~Wen and Y.~Li.
\newblock The mechanism of prediction head in non-contrastive self-supervised learning.
\newblock {\em Advances in Neural Information Processing Systems}, 35:24794--24809, 2022.

\bibitem[WLS22]{wies2023subtask}
N.~Wies, Y.~Levine, and A.~Shashua.
\newblock Sub-task decomposition enables learning in sequence to sequence tasks.
\newblock {\em arXiv preprint arXiv:2204.02892}, 2022.

\bibitem[WNB{\etalchar{+}}25]{wang2025learning}
Z.~Wang, E.~Nichani, A.~Bietti, A.~Damian, D.~Hsu, J.~D. Lee, and D.~Wu.
\newblock Learning compositional functions with transformers from easy-to-hard data.
\newblock {\em arXiv preprint arXiv:2505.23683}, 2025.

\bibitem[WWS{\etalchar{+}}22a]{wang2022self}
X.~Wang, J.~Wei, D.~Schuurmans, Q.~Le, E.~Chi, S.~Narang, A.~Chowdhery, and D.~Zhou.
\newblock Self-consistency improves chain of thought reasoning in language models.
\newblock {\em arXiv preprint arXiv:2203.11171}, 2022.

\bibitem[WWS{\etalchar{+}}22b]{wei2023chain}
J.~Wei, X.~Wang, D.~Schuurmans, M.~Bosma, F.~Xia, E.~Chi, Q.~V. Le, and D.~Zhou.
\newblock Chain-of-thought prompting elicits reasoning in large language models.
\newblock {\em Advances in neural information processing systems}, 35:24824--24837, 2022.

\bibitem[WZLZ24]{wen2025sparse}
K.~Wen, H.~Zhang, H.~Lin, and J.~Zhang.
\newblock From sparse dependence to sparse attention: unveiling how chain-of-thought enhances transformer sample efficiency.
\newblock {\em arXiv preprint arXiv:2410.05459}, 2024.

\bibitem[XGZ{\etalchar{+}}24]{xie2024monte}
Y.~Xie, A.~Goyal, W.~Zheng, M.-Y. Kan, T.~P. Lillicrap, K.~Kawaguchi, and M.~Shieh.
\newblock Monte carlo tree search boosts reasoning via iterative preference learning.
\newblock {\em arXiv preprint arXiv:2405.00451}, 2024.

\bibitem[XL25]{xiao2025generalizing}
C.~Xiao and B.~Liu.
\newblock Generalizing reasoning problems to longer lengths.
\newblock In {\em The Thirteenth International Conference on Learning Representations}, 2025.

\bibitem[YHLC24]{yang2024context}
T.~Yang, Y.~Huang, Y.~Liang, and Y.~Chi.
\newblock In-context learning with representations: Contextual generalization of trained transformers.
\newblock {\em Advances in Neural Information Processing Systems}, 37:85867--85898, 2024.

\bibitem[YHLC25]{yang2025multi}
T.~Yang, Y.~Huang, Y.~Liang, and Y.~Chi.
\newblock Multi-head transformers provably learn symbolic multi-step reasoning via gradient descent.
\newblock {\em arXiv preprint arXiv:2508.08222}, 2025.

\bibitem[YLJ{\etalchar{+}}25]{yang2025longer}
W.~Yang, Z.~Liu, H.~Jin, Q.~Yin, V.~Chaudhary, and X.~Han.
\newblock Longer context, deeper thinking: Uncovering the role of long-context ability in reasoning.
\newblock {\em arXiv preprint arXiv:2505.17315}, 2025.

\bibitem[YSL{\etalchar{+}}25]{yan2025inftythink}
Y.~Yan, Y.~Shen, Y.~Liu, J.~Jiang, M.~Zhang, J.~Shao, and Y.~Zhuang.
\newblock Inftythink: Breaking the length limits of long-context reasoning in large language models.
\newblock {\em arXiv preprint arXiv:2503.06692}, 2025.

\bibitem[ZAC{\etalchar{+}}24]{zhou2024transformers}
Y.~Zhou, U.~Alon, X.~Chen, X.~Wang, R.~Agarwal, and D.~Zhou.
\newblock Transformers can achieve length generalization but not robustly.
\newblock {\em arXiv preprint arXiv:2402.09371}, 2024.

\bibitem[ZBB{\etalchar{+}}22]{zhang2022unveiling}
Y.~Zhang, A.~Backurs, S.~Bubeck, R.~Eldan, S.~Gunasekar, and T.~Wagner.
\newblock Unveiling transformers with lego: a synthetic reasoning task.
\newblock {\em arXiv preprint arXiv:2206.04301}, 2022.

\bibitem[ZBL{\etalchar{+}}23]{zhou2023algorithms}
H.~Zhou, A.~Bradley, E.~Littwin, N.~Razin, O.~Saremi, J.~Susskind, S.~Bengio, and P.~Nakkiran.
\newblock What algorithms can transformers learn? a study in length generalization.
\newblock {\em arXiv preprint arXiv:2310.16028}, 2023.

\bibitem[ZLP{\etalchar{+}}24]{zhong2024evaluation}
T.~Zhong, Z.~Liu, Y.~Pan, Y.~Zhang, Y.~Zhou, S.~Liang, Z.~Wu, Y.~Lyu, P.~Shu, and X.~Yu.
\newblock Evaluation of {OpenAI} o1: Opportunities and challenges of {AGI}.
\newblock {\em arXiv preprint arXiv:2409.18486}, 2024.

\bibitem[ZSH{\etalchar{+}}22]{zhou2022least}
D.~Zhou, N.~Sch{\"a}rli, L.~Hou, J.~Wei, N.~Scales, X.~Wang, D.~Schuurmans, C.~Cui, O.~Bousquet, and Q.~Le.
\newblock Least-to-most prompting enables complex reasoning in large language models.
\newblock {\em arXiv preprint arXiv:2205.10625}, 2022.

\end{thebibliography}
\stopcontents

\newpage

\appendix

\allowdisplaybreaks
\newpage

\begin{center}
 \LARGE  \bf Appendix: Complete Proofs
\end{center}

\startcontents[sections]
{
\hypersetup{linkcolor=blue}
\printcontents[sections]{l}{1}{\setcounter{tocdepth}{2}}
}

\section{Learning In-Context Retrieval of Variables}\label{sec:icl-retrieval}

In this section, we focus on the learning process of $\Wb_4$ for the task $\cT^1$. For this task, $\Wb_4$ should predict the $4$-th token in the answer clause $Z_{\ans,1}$, which is the variable $x_1$. We will show that, the network learns to  retrieve the target variable $x_1$ from the first token of the first predicate clause $Z_{\pred,1}$ to make an accurate prediction. Throughout the rest of the proof, we will omit the subscript for the expectation $\mathbb{E}$, when the context is clear.

\subsection{Preliminaries}
First we define some notations for the presentation of gradients.
\paragraph{Notations for gradient expressions}
For each \(i\in[5],\ell\in [L],j\in[d],r\in [m]\), we denote
\begin{align}
    & \Ecal_{i,j}(\Zb^{L,\ell-1}) \triangleq \1_{\Z_{\ans,\ell,i} = e_j} - \mathbf{logit}_{i,j}(F,\Z^{L,\ell-1}),\label{eq-def-Ecal-icl}\\&\Lambda_{i, j,r}(\Zb^{L,\ell-1})\triangleq\sum_{\mathbf{k} \in \mathcal{I}^{L, \ell-1}} \attn_{{\ans,\ell-1} \rightarrow \kk}\cdot\big\langle \Wb_{i,j, r}, \mathbf{Z}_{\mathbf{k}}\big\rangle +b_{i,j,r} . \label{eq-def-Lambda-icl}
 \end{align}
    where \(\mathbf{logit}_{i,j}(F,\Z^{{L,\ell-1}})\) are defined as 
    \begin{align*}
        \mathbf{logit}_{i,j}\big(F,\Z^{{L,\ell-1}}\big) : = \frac{e^{F_{i,j}\big({\Z^{L,\ell-1}}\big)}}{\sum_{j^{\prime}\in[d]} e^{F_{i,j^{\prime}}\big({\Z^{L,\ell-1}}\big)}}.
    \end{align*}

\begin{fact}\label{fact-icl-gd} 
For any $i\in [5]$, $j\in[d]$, $r\in [m]$, we have the following gradient expression:
    \begin{align*}
        -\nabla_{ \Wb_{i,j,r}}\Loss^{L} &= \mathbb{E}\Big[\sum_{\ell=1}^L \Ecal_{i,j}(\Zb^{L,\ell-1})\ReLU^{\prime}\big(\Lambda_{i, j,r}(\Zb^{L,\ell-1})\big)\sum_{\kk\in\cI^{L,\ell-1}}\attn_{{\ans,\ell-1}\to {\kk}}\Zb_\kk\Big]. 
    \end{align*}
\end{fact}
\noindent For simplicity of notation, we will henceforth denote $\Lambda_{i, j, r}(\Zb^{L, \ell - 1})$ by $\Lambda_{i, j, r}$ and $\Ecal_{i,j}(\Zb^{L, \ell - 1})$ by $\Ecal_{i,j}$ when the context is clear.

Given $\Z^{L}$, we use $\hat{\X}^{L}$ to denote the appeared variables in the context clauses, i.e. $\hat{\X}^{L} = \{x_0,x_1,\dots,x_L\}$. We write $\hat{\X}^{L}$ as $\hat{\X}$ for simplicity. Throughout this section, we write $[F_{i}]_j$ as $F_{i,j}$ for simplicity.

\subsection{Induction Hypothesis}

\paragraph{Proof Plan.} 
The main idea is to track the  dynamics of different types of  weights $\Wb_{4,j,r,p}$. Specifically, we prove that for each $j\in\tau(\X)$, there exists certain neurons $r\in[m]$ such that the corresponding weights $\Wb_{4,j,r,1}$ grow significantly along the direction $e_{j}$, while all others remain small. Specifically, we proceed in four steps:
\begin{enumerate}
    \item 
    For $j \in \tau(\mathcal{X})$, define
    \[
    \Gamma^{(t)}_{4,j} \triangleq \max_{r \in [m]} \langle \Wb^{(t)}_{4,j,r,1}, e_j\rangle + \sigma_0 \log d,
    \]
    to track the maximal activation associated with retrieving the correct variable token.

    \item \textbf{Establish rapid growth (early phase).}  
    Let $\Lambda^{-} = \Theta(1/m)$, and define the hitting time
    \[
    T_{1,j} \triangleq \min\{t>0 : \Gamma^{(t)}_{4,j} \geq \Lambda^{-}\}.
    \]
    We show that for iterations $t \leq T_1 = \Theta(d\sigma_0^{q-2}/\eta)$, the diagonal weights $\langle \Wb^{(t)}_{4,j,r,1}, e_j\rangle$ grow rapidly, causing $\Gamma^{(t)}_{4,j}$ to enter a linear growth regime. Simultaneously, the model confidently identifies the correct variable, indicated by $1 - \logit_{4,\tau(x_1)}^{(t)} = 1 - o(1)$.

    \item \textbf{Convergence via dominant neurons (late phase).}  
    For each $j\in\tau(\X)$, define active neuron sets and their total activation as:
    \[
    \cA_{4,j}^{(t)} \triangleq \{r \in [m] : \langle \Wb^{(t)}_{4,j,r,1}, e_j\rangle \geq \varrho\log d\}, \quad \Phi_{4,j}^{(t)}\triangleq\sum_{r\in\cA_{4,j}^{(t)}}\langle \Wb^{(t)}_{4,j,r,1}, e_j\rangle.
    \]
    For iterations $t > T_1$, we analyze the refined dynamics, proving that the total diagonal activation $\Phi^{(t)}_{4,j^\star}$, for the weakest activated variable $j^\star$, eventually grows to $\Theta(\log d)$, ensuring successful learning.

    \item \textbf{Bounding non-target correlations.}  
    We finally show by induction that all other correlations $\langle \Wb^{(t)}_{4,j,r,p}, e_s\rangle$, including group actions, value tokens, off-diagonal tokens, and non-target variables, remain negligible throughout the training process. 
\end{enumerate}
Our proof begins by positing an induction hypothesis expected to hold
throughout training. We then analyze the dynamics under this hypothesis and
show that it remains valid along the entire training trajectory, establishing
the claim at convergence.

\begin{induction}\label{indutction-1.1.1}
   For  $t\leq T=\frac{\poly d}{\eta}$, all of the following holds:
    \begin{enumerate}[(a).]
        \item for $j\in\tau(\X)$, $\tilde{\Omega}(\sigma_0)\leq  \langle \Wb^{(t)}_{4,j,r,1}, e_{j}\rangle+\mu \leq \tilde{O}\big(1\big)$, where $\langle \Wb^{(t)}_{4,j,r,1}, e_{j}\rangle$ is non-decreasing; 
        \item for $j\in\tau(\X)$, $g\in\cG$$$\big|\langle\Wb^{(t)}_{4,j,r,2}, e_{\tau(g)}\rangle\big|\leq \tilde{O}(\sigma_0)+ O\big(\frac{1}{|\cG|}\big)     \max\Big\{\langle \Wb^{(t)}_{4,j,r,1}, e_{j}\rangle, \min_{r^{\prime}\in\cA_{4,j^*}^{(t)}}\langle \Wb^{(t)}_{4,j^*,r^{\prime},1}, e_{j^*}\rangle\Big\};$$ 
        \item for $j\in\tau(\X)$, $y\in\cY$ $$\big|\langle\Wb^{(t)}_{4,j,r,5}, e_{\tau(y)}\rangle\big|\leq \tilde{O}(\sigma_0)+ O\big(\frac{1}{|\cY|}\big) \max\Big\{\langle \Wb^{(t)}_{4,j,r,1}, e_{j}\rangle, \min_{r^{\prime}\in\cA_{4,j^*}^{(t)}}\langle \Wb^{(t)}_{4,j^*,r^{\prime},1}, e_{j^*}\rangle\Big\};$$ 
        \item else, $\big|\langle \Wb^{(t)}_{4,j,r,p}, e_{j'}\rangle\big| \leq \tilde{O}(\sigma_0)$ for any other $j, j'\in[d]$.
    \end{enumerate}
\end{induction}

\begin{claim}\label{clm-lambda}
     If \Cref{indutction-1.1.1} holds at iteration $t$, then for an input $\Zb^{1,0}$, we have 
     \begin{enumerate}
         \item if $j=\tau(x_1)$, $$\Lambda^{(t)}_{4,j,r}=\frac{1}{2}\langle \Wb^{(t)}_{4,j,r,2}, e_{j}\rangle+\frac{1}{2}\langle\Wb^{(t)}_{4,j,r,2}, e_{\tau(g_1)}\rangle+\frac{1}{2}\langle\Wb^{(t)}_{4,j,r,5}, e_{\tau(y_0)}\rangle+ \frac{5}{2}\mu+\tilde{O}(\sigma_0);$$
         \item else if $j\in\tau(\X\setminus\{x_1\}) $, 
         $$\Lambda^{(t)}_{4,j,r}=\frac{1}{2}\langle\Wb^{(t)}_{4,j,r,2}, e_{\tau(g_1)}\rangle+\frac{1}{2}\langle\Wb^{(t)}_{4,j,r,5}, e_{\tau(y_0)}\rangle+ \frac{5}{2}\mu+\tilde{O}(\sigma_0);$$
         \item otherwise, $0 \leq \Lambda^{(t)}_{4,j,r}\leq  \frac{5}{2}\mu+\tilde{O}(\sigma_0)$.
     \end{enumerate}
\end{claim}
\begin{claim}\label{clm-logit}
    If \Cref{indutction-1.1.1} holds at iteration $t$, then for an input $\Zb^{1,0}$, 
    \begin{enumerate}
        \item if $j=\tau(x_1)$,   $\logit^{(t)}_{4,j}=\frac{e^{O(\Phi^{(t)}_{4,j})}}{e^{O(\Phi^{(t)}_{4,j})}+d}$;
        \item otherwise,  $\logit^{(t)}_{4,j}=O\big(\frac{1}{d}\big)\Big(1-\logit^{(t)}_{\tau(x_1)}\Big)$.
    \end{enumerate}
\end{claim}
\begin{proof}
If $j=\tau(x_1)$, by \Cref{indutction-1.1.1} and \Cref{clm-lambda},  we have 
      \begin{align*}
      0 \leq   F_{4,j}^{(t)}(\Zb^{1,0}) &\leq  \sum_{r\in [m]} [\Lambda^{(t)}_{4,j,r}]^{+}\leq \big(\Phi^{(t)}_{4,j}+O(\frac{\max\{\Phi^{(t)}_{4,j}, \Phi^{(t)}_{4,j^*}\}}{|\cG|})\big)+\tilde{O}(\sigma_0)+O(m\varrho\log d)\\
      &=\big(\Phi^{(t)}_{4,j}+O(\frac{\Phi^{(t)}_{4,j}}{|\cG|})\big)+\tilde{O}(\sigma_0)+O(\frac{1}{\polylog d}).
    \end{align*}
For $j\in\tau(\X)\not=\tau(x_1)$, $F_{4,j}^{(t)}(\Zb^{1,0})\leq \tilde{O}(\sigma_0)+O(\frac{\max\{\Phi^{(t)}_{4,j}, \Phi^{(t)}_{4,j^*}\}}{|\cG|})$; else $F_{4,j}^{(t)}(\Zb^{1,0})\leq \tilde{O}(\sigma_0)$. 
Combining them together, we complete the proof.
\end{proof}
\subsection{Gradient Lemma}
Starting with the gradient computation from \Cref{fact-icl-gd}: 
            \begin{align*}
         -\nabla_{ \Wb_{4,j,r,p}}\Loss^{1} &=\frac{1}{2} \mathbb{E}\Big[ \Ecal_{4,j}\ReLU^{\prime}\big(\Lambda_{4, j,r}\big)\sum_{\kk\in\cI^{1,0}}\Zb_{\kk, p}\Big],
    \end{align*}
we first consider the gradient for $j\in\tau(\X)$
\begin{lemma}\label{lem-grad-1}
    For $j\in\tau(\X)$, we have
    \begin{enumerate}[(a)]
        \item for \( \Wb_{4,j,r,1}\),  $s\in\tau(\X)$
        \begin{enumerate}[(1)]
            \item if $s=j$, 
         $\langle-\nabla_{ \Wb^{(t)}_{4,j,r, 1}}\Loss^{1} , e_{s}\rangle 
         =\frac{1}{2} \mathbb{E}\Big[ (1-\logit^{(t)}_{4,j})\ReLU^{\prime}\big(\Lambda^{(t)}_{4, j,r}\big)\1_{\tau(x_1)=j}  \Big] $;
    \item $s\not=j$, $\langle-\nabla_{ \Wb^{(t)}_{4,j,r, 1}}\Loss^{1} , e_{s}\rangle =\frac{1}{2} \mathbb{E}\Big[ 
        -\logit^{(t)}_{4,j}\ReLU^{\prime}\big(\Lambda^{(t)}_{4, j,r}\big)\1_{\tau(x_1)=s}\Big]$.
        \end{enumerate}
        \item for \( \Wb_{4,j,r,2}\), $s=\tau(g)$ for $g\in\cG$
        \begin{align*}
         \langle-\nabla_{ \Wb^{(t)}_{4,j,r, 2}}\Loss^{1} , e_{s}\rangle &=\frac{1}{2} \mathbb{E}\Big[ (1-\logit^{(t)}_{4,j})\ReLU^{\prime}\big(\Lambda^{(t)}_{4, j,r}\big)\1_{\tau(x_1)=j,g_1=g} \\
         &~~~~~~~~~~~~~~~~-\logit^{(t)}_{4,j}\ReLU^{\prime}\big(\Lambda^{(t)}_{4, j,r}\big)\1_{\tau(x_1)\not=j, g_1=g}\Big].
    \end{align*}
    \item for \( \Wb_{4,j,r,p}\) with $p\in\{3,4\}$, $s\in\tau(\X)$
    \begin{enumerate}[(1)]
        \item $s=j$, $\langle-\nabla_{ \Wb_{4,j,r, 3}}\Loss^{1} , e_{j}\rangle =\frac{1}{2} \mathbb{E}\Big[ 
        -\logit^{(t)}_{4,j}\ReLU^{\prime}\big(\Lambda^{(t)}_{4, j,r}\big)\1_{\tau(x_0)=j}\Big]$;
        \item $s\not= j $
        \begin{align*}
         \langle-\nabla_{ \Wb_{4,j,r,3}}\Loss^{1} , e_{s}\rangle   
         =&\frac{1}{2} \mathbb{E}\Big[ (1-\logit^{(t)}_{4,j})\ReLU^{\prime}\big(\Lambda^{(t)}_{4,j,r}\big)\1_{\tau(x_0)=s, \tau(x_1)=j}\\
         &~~~~~~~~~~~~~~~~~-\logit^{(t)}_{4,j}\ReLU^{\prime}\big(\Lambda^{(t)}_{4, j,r}\big)\1_{\tau(x_0)=s, j\not\in\tau(\hat{X})}\Big].
    \end{align*}
    \end{enumerate}
   \item  for \( \Wb_{4,j,r,5}\), $s=\tau(y)$ for $g\in\cY$
        \begin{align*}
         \langle-\nabla_{ \Wb^{(t)}_{4,j,r, 5}}\Loss^{1} , e_{s}\rangle 
         =&\frac{1}{2} \mathbb{E}\Big[ (1-\logit^{(t)}_{4,j})\ReLU^{\prime}\big(\Lambda^{(t)}_{4, j,r}\big)\1_{\tau(x_1)=j,y_0=y}\\&~~~~~~~~~~~~~~~~~~ -\logit^{(t)}_{4,j}\ReLU^{\prime}\big(\Lambda^{(t)}_{4, j,r}\big)\1_{\tau(x_1)\not=j, y_0=y}\Big].
    \end{align*}
    \end{enumerate}
\end{lemma}
Move on to $j\not\in\tau(\X)$, we can obtain
\begin{lemma}\label{lem-grad-2}
    For $j\not\in\tau(\X)$, we have 
        \begin{enumerate}[(a)]
        \item for \( \Wb_{4,j,r,1}\),  $s\in\tau(\X)$,  $$\langle-\nabla_{ \Wb^{(t)}_{4,j,r, 1}}\Loss^{1} , e_{s}\rangle =\frac{1}{2} \mathbb{E}\Big[ 
        -\logit^{(t)}_{4,j}\ReLU^{\prime}\big(\Lambda^{(t)}_{4, j,r}\big)\1_{\tau(x_1)=s}\Big].$$

        \item for \( \Wb_{4,j,r,2}\), $s=\tau(g)$ for $g\in\cG$,
       $$\langle-\nabla_{ \Wb^{(t)}_{4,j,r, 2}}\Loss^{1} , e_{s}\rangle =\frac{1}{2} \mathbb{E}\Big[ -\logit^{(t)}_{4,j}\ReLU^{\prime}\big(\Lambda^{(t)}_{4, j,r}\big)\1_{ g_1=g}\Big].$$
    \item for \( \Wb_{4,j,r,p}\) with $p\in\{3,4\}$, $s\in\tau(\X)$, $$\langle-\nabla_{ \Wb^{(t)}_{4,j,r, p}}\Loss^{1} , e_{j}\rangle =\frac{1}{2} \mathbb{E}\Big[ 
        -\logit^{(t)}_{4,j}\ReLU^{\prime}\big(\Lambda^{(t)}_{4, j,r}\big)\1_{\tau(x_0)=s}\Big].$$

   \item  for \( \Wb_{4,j,r,5}\), $s=\tau(y)$ for $g\in\cY$
        \begin{align*}
         &\langle-\nabla_{ \Wb^{(t)}_{4,j,r, 5}}\Loss^{1} , e_{s}\rangle =\frac{1}{2} \mathbb{E}\Big[-\logit^{(t)}_{4,j}\ReLU^{\prime}\big(\Lambda^{(t)}_{4, j,r}\big)\1_{ y_0=y}\Big].
    \end{align*}
    \end{enumerate}
\end{lemma}

\subsection{Growth of Gamma}
    \begin{lemma}[Growth]\label{lem-1.1.1-growth}
        Given $j\in\tau(\X)$, suppose \Cref{indutction-1.1.1} holds at  iteration $t$, when $\Phi_{4,j}^{(t)}\leq 0.01\log d$ or $\Gamma_{4,j}^{(t)}\leq \frac{0.01 \log d}{m}$, then it satisfies
        \begin{align*}
         \Gamma_{4,j}^{(t+1)} =  \Gamma_{4,j}^{(t)}+\Theta\big(\frac{\eta}{d}\big) \ReLU^{\prime}(\Gamma_{4,j}^{(t)}).
        \end{align*}
    \end{lemma}

    \begin{proof}
  By  \Cref{lem-grad-1}, we have 
                            \begin{align*}
         &\langle-\nabla_{ \Wb^{(t)}_{4,j,r, 1}}\Loss^{1} , e_{j}\rangle =\frac{1}{2} \mathbb{E}\Big[ (1-\logit^{(t)}_{4,j})\ReLU^{\prime}\big(\Lambda^{(t)}_{4, j,r}\big)\1_{\tau(x_1)=j}  \Big]. 
    \end{align*}
    By \Cref{clm-logit}, when $\Phi_{4,j}^{(t)}\leq 0.01\log d$, $\logit^{(t)}_{4,j}=\frac{O(e^{0.01\log d})}{O(e^{0.01\log d})+d}
    \ll 1$ when $j=\tau(x_1)$; and combing with the fact that the event $\{\tau(x_1)=j\}$ happens with probability $\frac{1}{|\cX|}$, we complete the proof.
    \end{proof}
    \Cref{lem-1.1.1-growth}, combined with the growth of the tensor power method, immediately gives the following corollary. 
    \begin{lemma}\label{lem-time}
         Suppose \Cref{indutction-1.1.1} holds for all iterations. Define threshold  $\Lambda^{-}=\Theta(\frac{1}{m})$. Let $T_{1,j}$ be the first iteration so that $\Gamma_{4,j}^{(t)}\geq \Lambda^{-}$, and $T_1\stackrel{\text{def}}{=}\Theta(\frac{d}{\eta\sigma_0^{q-2}})$. Then we have $T_{1}\geq T_{1,j}$ for every $j\in \tau(\X)$, i.e., for $t\geq T_{1}$, it satisfies $\Gamma_{4,j}^{(t)}\geq \Lambda^{-}$.
    \end{lemma}
\begin{lemma}[Upper bound]\label{lem-phi}
    Suppose \Cref{indutction-1.1.1} holds for all iterations $< t$, we have $\Phi_{4,j}^{(t)}\leq \tilde{O}(1)$, for $j\in\tau(\X)$. 
\end{lemma}
    \begin{proof}
    We only need to consider the time $t\geq T_1$.  
    Notice that the gradient descent update in \Cref{lem-grad-1} gives
                            \begin{align*}
         &\langle-\nabla_{ \Wb^{(t)}_{4,j,r, 1}}\Loss^{1} , e_{j}\rangle 
         =\frac{1}{2} \mathbb{E}\Big[ (1-\logit^{(t)}_{4,j})\ReLU^{\prime}\big(\Lambda^{(t)}_{4, j,r}\big)\1_{\tau(x_1)=j}  \Big] 
    \end{align*}
    Therefore, for sufficiently small $\eta$, we have
    \begin{align*}
        \Phi_{4,j}^{(t+1)}&= \Phi_{4,j}^{(t)}+\sum_{r\in\cA_{4,j}^{(t)}}\frac{\eta}{2} \mathbb{E}\Big[ (1-\logit^{(t)}_{4,j})\ReLU^{\prime}\big(\Lambda^{(t)}_{4, j,r}\big)\1_{\tau(x_1)=j}  \Big]+O(\varrho\log d)\cdot |\cA_{4,j}^{(t+1)}\setminus\cA_{4,j}^{(t)}| \\
        &=\Phi_{4,j}^{(t)}+\sum_{r\in\cA_{4,j}^{(t)}}\frac{\eta}{2} \mathbb{E}\Big[ (1-\logit^{(t)}_{4,j})\ReLU^{\prime}\big(\Lambda^{(t)}_{4, j,r}\big)\1_{\tau(x_1)=j}  \Big]+\frac{1}{\polylog d}.
    \end{align*}
    When there exists $\tilde{T}$, s.t., $ \max _{j\in\tau(\X)}\Phi_{4,j}^{(\tilde{T})}>\Omega(\log^{1.5}d)$, by \Cref{indutction-1.1.1} and \Cref{clm-lambda}, given an input sequence $\Zb^{1,0}$ with $\tau(x_1)=\tilde{j}= \arg\max _{j\in\tau(\X)}\Phi_{4,j}^{(\tilde{T})}$, we have   
    \begin{align*}
        F_{4,j}^{(\tilde{T})}(\Zb^{1,0}) \geq  \sum_{r\in\cA_{4,\tilde{j}}^{(t)}} \Lambda^{(\tilde{T})}_{4,\tilde{j},r}\geq \bigg(1-O\Big(\frac{1}{|\cG|}\Big)\bigg)\Phi^{(\tilde{T})}_{4,\tilde{j}}-\tilde{O}(\sigma_0)>\Omega(\log^{1.5}d).
    \end{align*}
    Following the similar analysis as \Cref{clm-logit}, $ F_{4,j^{\prime}}^{(\tilde{T})}(\Zb^{1,0})\leq O(\frac{\Phi^{(\tilde{T})}_{4,j^{\prime}}}{|\cG|})$ for other $j^{\prime}\in\tau(\X)$, and $F_{4,j^{\prime}}^{(\tilde{T})}(\Zb)\leq o(1)$ for $j^{\prime}\notin\tau(\X)$, which implies $1-\logit_{4,j}^{(\tilde{T})}=e^{-\Omega(\log^{1.5}d)}$. Therefore, we derive that for $t\in [\tilde{T}+1, \frac{\poly d}{\eta})$,
    \begin{align*}
        \Phi_{4,j}^{(t)}\leq \Phi^{(\tilde{T})}_{4,j}+\tilde{O}(\poly d\cdot e^{-\Omega(\log^{1.5}d)})+ O(\rho\log d)\cdot m,
    \end{align*}
which implies $\Phi_{4,\tilde{j}}^{(t)}\leq O(\log^{1.5}d)$ since {$\varrho\ll\frac{1}{m\log d}$}. 
    \end{proof}

\subsection{Group and Value Correlations Are Not Large}
                   \begin{lemma}\label{lem-1.1.1-group}
                   Suppose \Cref{indutction-1.1.1} holds for all iterations $<t$, then for any $j\in\tau(\X)$ and  $s=\tau(g), g\in\cG$, we have 
      \begin{align*}
          \big|\langle\Wb^{(t)}_{4,j,r,2}, e_{\tau(g)}\rangle\big|\leq \tilde{O}(\sigma_0)+ O\big(\frac{1}{|\cG|}\big)     \max\Big\{\langle \Wb^{(t)}_{4,j,r,1}, e_{j}\rangle, \min_{r^{\prime}\in\cA_{4,j^*}^{(t)}}\langle \Wb^{(t)}_{4,j^*,r^{\prime},1}, e_{j^*}\rangle\Big\}
      \end{align*}
  \end{lemma}
\begin{proof}
 By  \Cref{lem-grad-1}, we have 
             \begin{align*}
         &\langle-\nabla_{ \Wb^{(t)}_{4,j,r, 2}}\Loss^{1} , e_{s}\rangle  \\
         &=\frac{1}{2} \mathbb{E}\Big[ (1-\logit^{(t)}_{4,j})\ReLU^{\prime}\big(\Lambda^{(t)}_{4, j,r}\big)\1_{\tau(x_1)=j,g_1=g} -\logit^{(t)}_{4,j}\ReLU^{\prime}\big(\Lambda^{(t)}_{4, j,r}\big)\1_{\tau(x_1)\not=j, g_1=g}\Big].
    \end{align*}
    Clearly, the positive gradient can be upper-bounded by 
    $O\big(\frac{1}{|\cG|}\langle-\nabla_{ \Wb^{(t)}_{4,j,r, 1}}\Loss^{1} , e_{j}\rangle\big)$. Moreover, for the negative gradient, by \Cref{clm-lambda}, we have a naive bound 
\begin{align*}
   \ReLU^{\prime}\big(\Lambda^{(t)}_{4, j,r}\big)\con_{\tau(x_1)\not=j,g_1=g}\leq  O\big(1\big)\ReLU^{\prime}\big(\Lambda^{(t)}_{4, j,r}\big)\big|_{\tau(x_1)=j, g_1=g}.
\end{align*}
When $t\leq T_{1}$, by \Cref{clm-logit}, we have $1-\logit^{(t)}_{4,j}\mid_{j=\tau(x_1)}\geq \Omega(1)$ and  $\logit^{(t)}_{4,j}\mid_{j\not=\tau(x_1)}\leq O(\frac{1}{d})$, which implies 
\begin{align*}
\mathbb{E}\Big[\logit^{(t)}_{4,j}\ReLU^{\prime}\big(\Lambda^{(t)}_{4, j,r}\big)\1_{\tau(x_1)\not=j, g_1=g}\Big]\leq O(\frac{1}{|\cG|})\langle-\nabla_{ \Wb^{(t)}_{4,j,r, 1}}\Loss^{1} , e_{j}\rangle.
    \end{align*}
Therefore, for $t\leq T_{1}$, we have
\begin{align*}
          \big|\langle\Wb^{(t)}_{4,j,r,2}, e_{\tau(g)}\rangle\big|\leq \tilde{O}(\sigma_0)+ O\Big(\frac{1}{|\cG|}\Big)    \langle \Wb^{(t)}_{4,j,r,1}, e_{j}\rangle.
      \end{align*}
For $t\geq T_{1}$, notice that by \Cref{lem-time}, $\cA_{4, j^{\prime}}^{(t)}\not=\emptyset$ for $j^{\prime}\in\tau(\X)$, thus for $r^{\prime}\in \cA_{4, j^{*}}^{(t)}$
    \begin{align*}
   \ReLU^{\prime}\big(\Lambda^{(t)}_{4, j,r}\big)\con_{\tau(x_1)\not=j,g_1=g}\leq   \ReLU^{\prime}\big(\Lambda^{(t)}_{4, j^{*},r^{\prime}}\big)\big|_{\tau(x_1)=j^{*}, g_1=g}.
\end{align*}
Furthermore,  $\logit^{(t)}_{4,j}\mid_{j\not=\tau(x_1)}\leq O(\frac{1}{d})(1-\logit^{(t)}_{4, j^{*}}\mid_{j^{*}=\tau(x_1)} )$, which implies 
\begin{align*}
\mathbb{E}\Big[\logit^{(t)}_{4,j}\ReLU^{\prime}\big(\Lambda^{(t)}_{4, j,r}\big)\1_{\tau(x_1)\not=j, g_1=g}\Big]\leq O(\frac{1}{|\cG|})\langle-\nabla_{ \Wb_{4,j^{*},r^{\prime}, 1}}\Loss^{1} , e_{j^{*}}\rangle.
    \end{align*}
    Due to the arbitrary of $r^{\prime}$, we can conclude that 
    \begin{align*}
          \big|\langle\Wb^{(t)}_{4,j,r,2}, e_{\tau(g)}\rangle\big|\leq \tilde{O}(\sigma_0)+ O\Big(\frac{1}{|\cG|}\Big)    \min_{r^{\prime}\in\cA_{4,j^*}^{(t)}}\langle \Wb^{(t)}_{4,j^*,r^{\prime},1}, e_{j^*}\rangle.
      \end{align*}
\end{proof}

 \begin{lemma}\label{lem-1.1.1-value}
                   Suppose \Cref{indutction-1.1.1} holds for all iterations $<t$, then for any $j\in\tau(\X)$ and  $s=\tau(y), y\in\cY$, we have 
      \begin{align*}
          \big|\langle\Wb^{(t)}_{4,j,r,5}, e_{\tau(y)}\rangle\big|\leq \tilde{O}(\sigma_0)+ O\big(\frac{1}{|\cY|}\big)     \max\Big\{\langle \Wb^{(t)}_{4,j,r,1}, e_{j}\rangle, \min_{r^{\prime}\in\cA_{4,j^*}^{(t)}}\langle \Wb^{(t)}_{4,j^*,r^{\prime},1}, e_{j^*}\rangle\Big\}.
      \end{align*}
  \end{lemma}
\begin{proof}
    The proof is similar to \Cref{lem-1.1.1-group}, and we omit the details here.
\end{proof}
\subsection{Off-diagonal Correlations Are Small}
        \begin{lemma}[off-diagonal bound] \label{lem-off-diagonal}
        Given $j\in\tau(\X)$, suppose \Cref{indutction-1.1.1} holds at all iterations $<t$, for  $s\in\tau(\X)\not=j$
  \begin{align*}
         \big|\langle\Wb^{(t)}_{4,j,r,1}, e_{s}\rangle\big| \leq \tilde{O}(\sigma_0).
  \end{align*}
    \end{lemma}
\begin{proof} 
By \Cref{lem-grad-1}
            \begin{align*}
         &\langle-\nabla_{ \Wb^{(t)}_{4,j,r, 1}}\Loss^{1} , e_{s}\rangle =\frac{1}{2} \mathbb{E}\Big[ 
        -\logit^{(t)}_{4,j}\ReLU^{\prime}\big(\Lambda^{(t)}_{4, j,r}\big)\1_{\tau(x_1)=s}\Big].
    \end{align*}
Notice that by \Cref{clm-lambda},
\begin{align*}
   \ReLU^{\prime}\big(\Lambda^{(t)}_{4, j,r}\big)\con_{\tau(x_1)=s}\leq  O\big(1\big)\ReLU^{\prime}\big(\Lambda^{(t)}_{4, s,r}\big)\big|_{\tau(x_1)=s},
\end{align*}
combined with \Cref{clm-logit}, $\logit_{4,j}\leq O(\frac{1}{d})(1-\logit_{4,s}^{(t)})$ when $s=\tau(x_1)$, thus 
\begin{align*}
\mathbb{E}\Big[\logit^{(t)}_{4,j}\ReLU^{\prime}\big(\Lambda^{(t)}_{4, j,r}\big)\1_{\tau(x_1)=s}\Big]&\leq \mathbb{E}\Big[O(\frac{1}{d})(1-\logit^{(t)}_{4,s})\ReLU^{\prime}\big(\Lambda^{(t)}_{4,s,r}\big)\1_{\tau(x_1)=s}\Big]\\
&\leq O(\frac{1}{d})\langle-\nabla_{ \Wb_{4,s,r, 1}}\Loss, e_{s}\rangle.
    \end{align*}
    From \Cref{indutction-1.1.1}, we have  $$\big|\langle\Wb^{(t)}_{4,j,r,1}, e_{s}\rangle\big| \leq O(\frac{1}{d})\big|\langle\Wb^{(t)}_{4,s,r,1}, e_{s}\rangle\big|+\tilde{O}(\sigma_0)\leq \tilde{O}(\frac{1}{d})+\tilde{O}(\sigma_0)=\tilde{O}(\sigma_0).$$
\end{proof}
    \begin{lemma} Given $j\in\tau(\X)$, suppose \Cref{indutction-1.1.1} holds at all iterations $<t$, we have 
  \begin{align*}
          \big|\langle{ \Wb^{(t)}_{4,j,r, p}}\Loss^{1} , e_{s}\rangle\big| \leq \tilde{O}(\sigma_0),\quad \text{ for $p\in\{3,4\}$ and all $s\in\tau(\X)$}
  \end{align*}
    \end{lemma}
    \begin{proof}
        When $s=j$, we have 
        \begin{align*}
         \langle-\nabla_{ \Wb_{4,j,r,p}}\Loss^{1} , e_{j}\rangle  &=\frac{1}{2} \mathbb{E}\Big[ -\logit^{(t)}_{4,j}\ReLU^{\prime}\big(\Lambda^{(t)}_{4,j,r}\big)\1_{\tau(x_0)=j} \Big]\\
         &=\frac{1}{2} \mathbb{E}\Big[ -\logit^{(t)}_{4,j}\ReLU^{\prime}\big(\Lambda^{(t)}_{4,j,r}\big)\sum_{s\not=j}\1_{\tau(x_0)=j, \tau(x_1)=s} \Big].
    \end{align*}
    Therefore, we can bound the above gradient in the similar way as the off-diagonal case, and obtain 
    $$\big|\langle\Wb^{(t)}_{4,j,r,p}, e_{j}\rangle\big| \leq O(\frac{1}{d})\max_{s\in\tau(\X)} \big|\langle\Wb^{(t)}_{4,s,r,1}, e_{s}\rangle\big|+\tilde{O}(\sigma_0)\leq \tilde{O}(\sigma_0).$$
    When $s\not=j$,
         \begin{align*}
         &\langle-\nabla_{ \Wb_{4,j,r,p}}\Loss^{1} , e_{s}\rangle   \\
         &=\frac{1}{2} \mathbb{E}\Big[ (1-\logit^{(t)}_{4,j})\ReLU^{\prime}\big(\Lambda^{(t)}_{4,j,r}\big)\1_{\tau(x_0)=s, \tau(x_1)=j} -\logit^{(t)}_{4,j}\ReLU^{\prime}\big(\Lambda^{(t)}_{4, j,r}\big)\1_{\tau(x_0)=s, j\not\in\tau(\hat{X})}\Big].
    \end{align*}
    Noticing that $\{\tau(x_0)=s, \tau(x_1)=j\}$ happens with probability $\frac{1}{|\cX|(|\X|-1)}$, thus the positive gradient can be upper bounded by $O(\frac{1}{d})\cdot |\langle-\nabla_{ \Wb_{4,j,r,1}}\Loss^{1} , e_{j}\rangle|$. Furthermore,  the negative part can be upper bounded in the similar way as previous off-diagonal negative gradient. Putting it together, we complete the proof.
    \end{proof}
    \subsection{Non-target Correlations Are Negligible}
    \begin{lemma}\label{lem-nontarget}
Suppose \Cref{indutction-1.1.1} holds at all iterations $<t$, for $j^{\prime}\not\in\tau(\X)$, for $p\in [5]$ and $s\in [d]$
  \begin{align*}
         \big|\langle\Wb^{(t)}_{4,j^{\prime},r,p}, e_{s}\rangle\big| \leq \tilde{O}(\sigma_0).
  \end{align*}
    \end{lemma}
    \begin{proof}
   By \Cref{lem-grad-2}, $\langle\Wb^{(t)}_{4,j^{\prime},r,p}, e_{s}\rangle$ for $p\in\{1,3,4\}$ and $s\in\tau(\X)$ can be bounded in the similar way previous off-diagonal negative gradient.
   
We can observe that for $j^{\prime}\not\in\tau(\X)$, all the non-zero gradient on the different directions are negative gradient, which implies $\langle\Wb^{(t)}_{4,j^{\prime},r,p}, e_{s}\rangle \leq \langle\Wb^{(0)}_{4,j^{\prime},r,p}, e_{s}\rangle=\tilde{O}(\sigma_0) $. Moreover,  $\Lambda_{4,j^{\prime},r}^{(t)}\leq \tilde{O}(\sigma_0)$ is also non-increasing.  

For $s=\tau(g), g\in\G$, whenever $\langle\Wb^{(t)}_{4,j^{\prime},r,2}, e_{s}\rangle$ reaches $-3\mu$, we have $\varLambda_{4,j^{\prime},r}^{(t)}\big|_{g_1=g}\leq -3\mu+\frac{5}{2}\mu+\tilde{O}(\sigma_0)\leq 0$, and thus $\langle-\nabla_{ \Wb^{(t)}_{4,j,r, 2}}\Loss^{1} , e_{s}\rangle =\frac{1}{2} \mathbb{E}\Big[ -\logit^{(t)}_{4,j}\ReLU^{\prime}\big(\Lambda^{(t)}_{4, j,r}\big)\1_{ g_1=g}\Big]=0$, which implies $\langle\Wb^{(t)}_{4,j^{\prime},r,2}, e_{s}\rangle\geq -3\mu$. Hence,  
$|\langle\Wb^{(t)}_{4,j^{\prime},r,2}, e_{s}\rangle|\leq \tilde{O}(\sigma_0)$. Following the similar argument, we can prove the result for $\langle\Wb^{(t)}_{4,j^{\prime},r,5}, e_{s}\rangle$ for $s\in\tau(\Y)$. 
    \end{proof}
   \subsection{Convergence} 
   \begin{lemma}\label{lem-induct}
       For  $|\cG|\geq |\cY|\geq \Omega(\frac{\log\log d}{\log\log\log d})$,  $\polylog d\geq m\geq |\cY|$, $\varrho \ll \frac{1}{m\log d}$ and sufficiently small $\eta\leq \frac{1}{\poly d}$, \Cref{indutction-1.1.1} holds for all iterations $t\leq T=\frac{\poly d}{\eta}$. 
   \end{lemma}
   \begin{proof}
          Putting the results in \Cref{lem-phi,lem-1.1.1-group,lem-1.1.1-value,lem-off-diagonal,lem-nontarget}, we can directly establish the results in \Cref{indutction-1.1.1}.
   \end{proof}
\begin{lemma}[Convergence]
    For sufficiently large $T_{1} \leq t=\frac{\poly d}{\eta}$, we have
 \begin{enumerate}[(a)]
 \item Objective convergence: $\Loss^{1} \leq \frac{1}{\poly d}$;
 \item Successful learning of diagonal feature: $\Phi_{4,j}^{(t)}\geq \Omega(\log d)$ for any $j\in\tau(\X)$.
 \end{enumerate}
\end{lemma}
\begin{proof}
    Assuming for some sufficiently large constant $n>0$,  $\mathbb{E}[(1-\logit_{4,j^{*}}^{(t)})\mid {\tau(x_1)=j^{*}}]\geq \Omega(\frac{1}{d^n})$ for $t\in (T_{1}, T_{1}+\frac{d^{n+1} \log^2 d}{\eta}]$ then by \Cref{lem-grad-1}, we have 
    \begin{align*}
        \Gamma_{4,j^{(*)}}^{( T_{1}+\frac{d^2\log^2 d}{\eta})}\geq \Omega\Big(\frac{\eta}{d^{n+1}}\Big)\cdot \frac{d^{n+1}\log^2 d}{\eta} +\Gamma_{4,j^{(*)}}^{(t)}\geq \Omega(\log^2d),
    \end{align*}
    which contradicts with $ \Gamma_{4,j^{(*)}}^{(t)}\leq \Phi_{4,j^{(*)}}^{(t)} \leq O(\log^{1.5}d)=\tilde{O}(1)$ in the polynomial time. This implies after sufficiently large iteration $t$, we must have $\mathbb{E}[(1-\logit_{4,j}^{(t)})\mid {\tau(x_1)=j}]\leq O(\frac{1}{d^n})$ for $j\in\tau(\X)$. Hence
    \begin{align*}
        \Loss^{1} =\E[-\log \logit_{4,\tau(x_1)}^{(t)}] &=\sum_{j\in\tau(\X)}\mathbb{E}[-\log \logit_{4,j}^{(t)}\1_{\tau(x_1)=j}]\\
        &\leq \sum_{j\in\tau(\X)}\mathbb{E}[O(1)\big(1-\logit_{4,j}^{(t)}\big)\1_{\tau(x_1)=j}] \tag{$\logit_{4,j}^{(t)}$ is very close to $1$}\\
        &\leq  O\Big(\frac{1}{\poly d}\Big).
    \end{align*}
    By \Cref{clm-logit}, at the time of convergence, we must have  $\Phi_{4,j}^{(t)}\geq \Omega(\log d)$. 
\end{proof}

\section{Learning Simply Transitive Actions}\label{appendix:learning-simply-transitive-actions}



\subsection{Preliminaries}

\begin{definition}[Combinations \(\phi, \Phi, \Phi^\dagger\)]\label{def:feature-combinations-simply}
    Let \(\cG\), \(\cY\) be defined as in \cref{def:lego}. For any pair \((g, y) \in \cG \times \cY\), we call \(\phi = (g,y)\) a \textbf{combination}. We write \(\phi_1 = g\) and \(\phi_2 = y\) to denote the corresponding components. We further define the set of \textbf{correct combinations} \(\Phi^\star\) by
    \begin{equation}\label{eqdef:combination-set-simply}
        \Phi_j^\star = \{\phi = (g,y) \mid j = \tau(\alpha(g,y))\},\ \forall j \in \tau(\cY); \qquad \Phi = \bigcup_{j\in \tau(\cY)} \Phi_j^\star
    \end{equation}
    The combinations \(\phi \in \Phi_j^\star\) are the ones that can be composed to get \(\tilde{y} = \tau^{-1}(j) \in \cY\). We also define the set of \textbf{incorrect combinations} \(\Phi^\dagger_j\) for each \(j\in\tau(\cY)\) to be
    \begin{equation}\label{eqdef:confounding-combinations-simply}
        \Phi^\dagger_j = \{\phi \in \Phi\setminus \Phi_j^\star\}
    \end{equation}
    For simply transitive action \(\alpha: \cG \times \cY \to \cY\), any combination \(\phi \in \Phi^\dagger_j\) is a \emph{incorrect} pair of \(g\) and \(y\). That is, it holds that any \(y' \neq y \in \cY\) and \(g' \neq g \in \cG\) satisfy \(\alpha(g,y') = \alpha(y,g') \neq \tau^{-1}(j)\).
\end{definition}

For each combination, we define the following neural-features:

\begin{definition}[combination features \(\psi, \Psi, \Psi^{\star}, \Psi^\dagger\)]
    Given a combination \(\phi = (g,y) \in \Phi\), token index \(j\in\tau(\cY)\) and neuron index \(r \in [m]\), we define
    \begin{align*}
        V_{j,r}(g) = \vbrack{\Wb_{5,j,r,2}, e_g}, \quad V_{j,r}(y) = \vbrack{\Wb_{5,j,r,5}, e_y}, \quad V_{j,r}(\phi) := \frac{1}{2}(V_{j,r}(g) + V_{j,r}(y))
    \end{align*}
    and we call \(V_{j,r}(\phi)\) the \emph{features for combination} \(\phi\) in neuron \((j,r) \in \tau(\cY)\times [m]\). We write \(\psi = (j,r,\phi)\) and \(V_\psi \equiv V_{j,r}(\phi)\) to make the notation concise. Similar to above, we further define 
    \[\Psi := \{\psi = (j,r, \phi) \times \tau(\cY)\times[m]\times \Phi\}\] 
    and \(\Psi^\star\) that contains the desirable features for each class \(j \in \tau(\cY)\):
    \begin{align}\label{eqdef:combination-features-set-simply}
        \Psi^{\star}_{j,r} = \{\psi = (j,r,\phi) \mid \phi \in \Phi_j^\star\},\quad \Psi^{\star} = \bigcup_{(j,r)\in \tau(\cY)\times[m]} \Psi_{j,r}
    \end{align}
    and \(\Psi^\dagger\) that contains the incorrect feature combinations:
    \begin{align}\label{eqdef:incorrect-combination-features-set-simply}
        \Psi^\dagger_{j,r} = \{\psi = (j,r,\phi) \mid \phi \in \Phi\setminus \Phi_j^\star\},\quad \Psi^\dagger = \bigcup_{(j,r)\in \tau(\cY)\times[m]}\wh{\Psi}_{j,r}
    \end{align}
\end{definition}

\paragraph{Events of combination appearance.} Let us first define some useful notations. For a combination \(\phi = (g,y)\), we write \(\cH_\phi\) to denote the event when \(\phi\) appears in the sequence \(\Zb^1\):
\[\cH_\phi := \{(g_1, y_0) = \phi\}, \qquad \cH_{\phi_1} = \cH_{g} = \{g_1 = g\},\qquad \cH_{\phi_2} = \cH_y = \{y_0 = y\}.\] 
Further, we write
\[\cH^\dagger_{\phi,1} = \{g_1 \neq g, y_0 = y\},\quad \cH^\dagger_{\phi,2} = \{g_1 = g, y_0 \neq y\},\quad \cH^\dagger_{\phi} = \cH^\dagger_{\phi,1} \cup \cH^\dagger_{\phi,2} \]
to denote the event where \(\phi\) did not appear but its group element or value is the same.

Finally we define a notion called \emph{learning curriculum}, that sits at the center of our proof.
\begin{definition}[Learning curriculum]\label{def:learning-curriculum-simply}
    We define an order on the set \(\Sigma\), defined by the following process: Let \(\Sigma_0 = \Psi\). At each \(i\in[n_y^2]\), we choose 
    \begin{align*}
        \psi_{i} = \arg\max_{\psi\in\Sigma_{i}} V_\psi^{(0)}
    \end{align*}
    Let's write \(\psi_{i} = (j,r,\phi)\) where \(\phi = (g,y)\), then we define \(\Sigma_{i+1}\) by excluding the following features:
    \begin{enumerate}[1.]
        \item Exclude the confusing combinations in neuron \((j,r)\);
        \[
            \Sigma_{\psi_{i}}^{\dagger, 1} \equiv \Sigma_i^{\dagger, 1} = \{\psi' = (j,r,\phi') \in \Sigma_{i}, \phi' = (g',y')\mid (g' = g) \textbf{ XOR } (y' = y)\}
        \] 
        \item Exclude the unselected combinations in neuron \((j,r)\);
        \[
            \Sigma_{\psi_{i}}^{\dagger, 2} \equiv \Sigma_i^{\dagger, 2} = \{ \psi' = (j, r, \phi') \in \Sigma_{i}, \phi' = (g',y')\mid g'\neq g, y'\neq y\}
        \]
        \item Exclude the feature indices of the same combination in other neurons \((j,r'), r'\neq r\);
        \[
            \Sigma_{\psi_{i}}^{\dagger, 3} \equiv \Sigma_i^{\dagger, 3} = \{ \psi' = (j, r', \phi) \in \Sigma_{i}, \forall r'\neq r \in [m]\}.
        \] 
    \end{enumerate}
    Which returns \(\Sigma_{i+1} = \Sigma_i \setminus \{\psi = (j,r',\phi') \in \Sigma_i \mid \text{either } r'= r \text{ or } \phi' = \phi \}\). Iterate over the whole set \(\Psi\), we will arive at \(\Sigma_{n_y^2} = \varnothing\). This gives rise to an ordered sequence \(\Sigma^{\star}\) and a set \(\Sigma^{\dagger}\):
    \begin{align*}
        \Sigma^{\star} := (\psi_1, \psi_2, \dots, \psi_{n_y^2}); \qquad \Sigma^\dagger := \bigcup\nolimits_{\psi \in \Sigma^\star}\Sigma^\dagger_\psi,\qquad \Sigma^{\dagger}_{\psi} := \Sigma_\psi^{\dagger, 1}\cup \Sigma_\psi^{\dagger, 2}\cup \Sigma_\psi^{\dagger, 3}
    \end{align*}
    By our construction \(V_{\psi_i}^{(0)} \geq V_{\psi_{i+1}}^{(0)} \).We write \(\psi \prec_{\Sigma} \psi'\) to denote that \(\psi\) is ahead of \(\psi'\) in \( \Sigma^{\star}\).
\end{definition}
Intuitively, \(\Sigma^{\star}\) encodes the order at which the features \(V_\psi\) grow in magnitude and \(\Sigma^\dagger\) contains all the unlearned features during the process.

Throughout the analysis of the FFN layer, we further make the following assumptions.

\begin{assumption}\label{assump:no-learning-x}
    For \((j,r,p) \in [d]\times[m]\times[5]\), we fix \(\Wb_{5,j,r,p}^{(t)} \equiv \Wb_{5,j,r,p}^{(0)}\) for \(p \in \{1,3,4\}\) at initialization for simplicity of proof.
\end{assumption}

\begin{assumption}\label{assump:modify-relu}
We use a modified smoothed ReLU as our activation function. This technical assumption is to avoid many pathologies appearing in the learning dynamics.
\begin{equation*}
    \mathbf{sReLU}(x) := \begin{cases} \varpi B &  x \leq -B \\ -\varpi x &  x \in (-B,-\varpi] \\ \frac{1}{2}x^2 & x \in (-\varpi, 0] ; \\ x^q/(\varrho^{q-1}q) & x \in (0, \varrho] ; \\ x-\varrho\left(1-\frac{1}{q}\right)  &  x \in  (\varrho, B] \\ B -\varrho\left(1-\frac{1}{q}\right)  &  x > B \end{cases}
\end{equation*}
where \(q  = O(1)\) is a large even integer and \(\varrho = \Theta(1/\polylog(d)), \varpi \in (d^2 \mu^{q-1}, \frac{\lambda}{d^{q/3}})\), \(B = C\log d, C=\Theta(1) \in (5,\frac{q-1}{3})\) and \(\lambda = \frac{d-1}{d-1+e^B}\).
\end{assumption}

\subsubsection{Theorem Statement}

We will try to prove the following theorem:
\begin{theorem}[Learning Simply Transitive Actions]\label{thm:mlp-simply-transitive}
    Suppose \(F^{(t)}\) is returned by Algorithm~\ref{alg:cot-transitive-training} at \(t=T_1\), and let \(\delta_1 = d^{c_1}\mu, \delta_2 = \varpi^{\frac{1}{q-1}}\), it holds that with probability \(\geq 1-o(1)\):
    \begin{enumerate}[A.]
        \item For all \(\psi \in \Sigma^{\star}\), we have \(V_\psi^{(T_1)} = B \pm O(\delta_1)\);
        \item For any \(\psi \in \Sigma^{\dagger,1}\), \(|V_{\psi}^{(T_1)}| \leq \tO(\delta_2)\);
        \item For any \(\psi \in \Sigma^{\dagger,2}\), \(V_{\psi}^{(T_1)} \leq -B + O(\delta_1)\);
        \item For any \(\psi \in \Sigma^{\dagger,3}\), \(V_{\psi}^{(T_1)} \leq O(\delta_1)\);
        \item For \(\psi = (j,r, (g,y)) \in \Sigma^\star\), \(|V_{j,r}^{(t)}(g) - V_{j,r}^{(t)}(y)| \leq \tO(\delta_2)\).
    \end{enumerate}
\end{theorem}

In fact, we can decompose the statement of \cref{thm:mlp-simply-transitive} into the following statement for every feature \(\psi \in \Sigma^\star\).
\begin{definition}[Feature Shape]\label{def:feature-shape-simply}
    Let \(\delta = (\delta_1, \delta_2)\) be a tuple of error parameters. Let \(\psi \in \Sigma^\star\) be a feature in the learning curriculum. We say the feature \(\psi\) reached \emph{feature shape} \(\cF_{\psi}(\delta)\) with error \(\delta\) if:
    \begin{enumerate}[1.]
        \item \(\cF_{\psi,1}(\delta_1)\): \(V_\psi^{(t)} \geq B - O(\delta_1)\);
        \item \(\cF_{\psi,2}(\delta_2)\): For any \(\psi' \in \Sigma_\psi^{\dagger, 1}\), it holds that \(|V_{\psi'}^{(t)}| \leq \tO(\delta_2)\);
        \item \(\cF_{\psi,3}(\delta_1)\): For any \(\psi' \in \Sigma_\psi^{\dagger, 2}\), it holds that \(V_{\psi'}^{(t)} \leq -B \pm O(\delta_1)\);
        \item \(\cF_{\psi,4}(\delta_1)\): For any feature \(\psi' \in \Sigma_\psi^{\dagger,3}\), it holds that \(V_{\psi'}^{(t)} \leq O(\delta_1)\).
        \item \(\cF_{\psi,5}(\delta_2)\): Writing \(\psi = (j,r,(g,y))\), then it holds that \(|V_{j,r}^{(t)}(g) - V_{j,r}^{(t)}(y)| \leq \tO(\delta_2)\);
    \end{enumerate}
\end{definition}

Obviously if we proved the condition \(\cF_\psi(\delta)\) of all \(\psi \in \Sigma^\star\) are reached for \(\delta = (d^{c_1}\mu, \varpi^{\frac{1}{q-1}})\) at \(T_1\), then \cref{thm:mlp-simply-transitive} is proven. We do this by following a induction process sequentially for each \(\psi\in\Sigma^\star\), and provide a guarantee that all features in \(\Sigma^\star\) reach the claimed convergence condition together at \(t = T_1\).


\subsubsection{Induction Hypotheses and Phase Decomposition}

To prove \cref{thm:mlp-simply-transitive}, or equivalently, to prove the convergence condition in \cref{def:feature-shape-simply} is reached for every feature \(\psi \in \Sigma^\star\), we shall charaterize the dynamics of each feature and prove that they eventually arrive at the desired shape.

\paragraph{Phase Decomposition.} Let us define the following timestamps for different phases of learning each \(\psi = (j,r,\phi) \in \Psi \cup \wh{\Psi}\). Let \(c_1 = \frac{1}{1000}\) be a small constant:
\begin{enumerate}[(a)]
    \item Phase I: \(t \in [0, T_{\psi}^1]\), where \(T_{\psi}^1 := \min\{t \geq 0 \mid V_\psi^{(t)}\geq d^{c_1}\mu\}\);
    \item Phase II.1: \(t \in (T_{\psi}^1, T_{\psi}^{2,1}]\), where \(T_{\psi}^{2,1} := \min\{t \geq 0 \mid V_\psi^{(t)}\geq \frac{1}{2}\log d\}\);
    \item Phase II.2: \(t \in (T_{\psi}^{2,1}, T_{\psi}^{2,2}]\), where \(T_{\psi}^{2,2} := \min\{t \geq 0 \mid \E[\logit_{5,j}^{(t)}\1_{\cH_\phi}] \geq 1 - \frac{1}{\sqrt{d}}\}\);
    \item Phase II.3: \(t \in (T_{\psi}^{2,2}, T_{\psi}^{2}]\), where \(T_\psi^2 = \max \{T_\psi^{2,3}, T_\psi^{2,2} + \Omega(\frac{1}{\eta (d^{c_1}\mu)^{q-2} })\}\) and \(T_\psi^{2,3} \) is defined as:
    \[
        T_\psi^{2,3} := \min\Big\{t \geq T_{\psi}^{2,2} \mid \cF_\psi(\delta_1,\delta_2)\text{ holds, where } \delta_1 = d^{c_1}\sigma_0, \delta_2 = (\frac{n_y^2d\varpi}{\lambda})^{\frac{1}{q-1}}\Big\}
    \]
    \item Phase III.1: \(t\in(T^{2}_{\psi}, T_{1,1}]\), where \(T_{1,1} = T_{\psi_{n_y^2}}^2\) and \(\psi_{n_y^2}\) is the last feature in \(\Sigma^\star\). This is the convergence phase where the rest of the features are learned and the feature converge to a perfect shape.
    \item Phase III.2: \(t \in (T_{1,1}, T_1]\), the end phase where all features are in perfect shape and have stablized.
\end{enumerate}

Now we state the following induction hypotheses which naturally results in \cref{thm:mlp-simply-transitive}.

\begin{induction}[Simply Transitive Actions, All Phases]\label{induction:mlp-simply-transitive}
    For all \(t \leq T_1\), the following holds:
    \begin{enumerate}[(a)]
        \item \(V_{\psi}^{(t)} + b_{i,j,r}\geq \Omega(\mu)\).
        \item \(T_{\psi'}^2 < T_{\psi}^1\) if \(\psi'\prec_\Sigma \psi\) for \(\psi, \psi' \in \Sigma^\star\), i.e., the intervals \(\{[T_{\psi}^1, T_{\psi}^2]\}_{\psi \in \Sigma^\star}\) are non-overlapping;
    \end{enumerate}
\end{induction}

In order to prove \cref{induction:mlp-simply-transitive,thm:mlp-simply-transitive}, we introduce the following feature based induction hypothesis to prove that \(\cF_\psi\) defined in \cref{def:feature-shape-simply} holds for some error parameters at \(t =T_{\psi}^2\).

\begin{induction}[Induction for Individual Feature]\label{induction:individual-feature-simply}
    Let \(\psi = (j,r,\phi)\in \Sigma^\star, \phi = (g,y)\), at \(t \leq T_1\), the following holds:
    \begin{enumerate}[(a)]
        \item At \(t \leq T_{\psi}^1\), any \(g\neq g', y'\neq y\) has \(V_{j,r}(g'), V_{j,r}(y') \leq \tO(\mu)\).
        \item During \(t \leq T_\psi^1\), there are at most \( \tO(\sqrt{d}/\eta)\) iterations where \(\logit^{(t)} \geq \frac{1}{\sqrt{d}}\) conditioned on \(\cH_\phi\).
        \item Let \(\psi' \prec \psi \in \Sigma^\star\), the feature shape \(\cF_{\psi'}(d^{c_1}\sigma_0, (\frac{n_y^2 d\varpi}{\lambda})^{\frac{1}{q-1}})\) holds throughout \(t \in [T_{\psi}^1, T_1]\);
        \item Let \(\psi' \succ \psi \in \Sigma^\star\), then \(V_{\psi}^{(t)} \leq \tO(\sigma_0)\) throughout \(t \in [0, T_{\psi}^2]\).
    \end{enumerate}
\end{induction}

\subsubsection{Technical Calculations}

Let us recall some facts about gradient computation for \(\Wb\).

\begin{fact}[gradient computation]\label{fact:grad-computation-simply}
    The gradient with respect to \(\Wb_{i,j,r,p}\) for \(i = 5\), \(j \in [d]\), \(r \in [m]\), \(p \in [5]\) and \(v \in \cV\) is 
    \begin{equation}\label{eqdef:grad-expression-Z5-simply}
        \vbrack{-\nabla_{ \Wb_{5,j,r,p}}\Loss, e_v} = \frac{1}{2} \mathbb{E}\Big[ \Ecal_{5,j}\ReLU^{\prime}\big(\Lambda_{5, j,r}\big)\sum_{\kk\in\cI^{1,0}}\vbrack{\Zb_{\kk, p}, e_v}\Big]
    \end{equation}
    where \(\cE_{5,j} = \1_{v=Z_{\ans,1,p}} - \logit_{5,j}\) is the loss derivatives. We compute more precise expressions when \(j\), \(p\) and \(v\) varies:
    \begin{enumerate}[(a)]
        \item when \(p = 2\), for \(g \in \cG\), let \(\phi = (g,y) \in \Phi^\star_j\) be the combination with \(\phi_1 = g\), we have for all \(j\in\tau(\cY), r\in[m]\),
        \begin{align}
            \vbrack{-\nabla_{\Wb_{5,j,r,2}}\Loss, e_g} = &\frac{1}{2}\E\Big[(1 - \logit_{5,j})\ReLU'(\Lambda_{5,j,r})\1_{\cH_\phi}\Big] \nonumber\\
            & \qquad - \frac{1}{2}\E\Big[\logit_{5,j}\ReLU'(\Lambda_{5,j,r})\1_{\cH^\dagger_{\phi,2}}\Big] \label{eq:grad-computation-g-simply}
        \end{align}
        \item when \(p = 5\), for \(y \in \cY\), let \(\phi = (g,y) \in \Phi^\star_j\) be the combination with \(\phi_2 = y\),  we have for all \(j\in\tau(\cY), r\in[m]\),
        \begin{align}
            \vbrack{-\nabla_{\Wb_{5,j,r,5}}\Loss, e_y} = &\frac{1}{2}\E\Big[(1 - \logit_{5,j})\ReLU'(\Lambda_{5,j,r})\1_{\cH_\phi}\Big] \nonumber \\
            & \qquad - \frac{1}{2}\E\Big[\logit_{5,j}\ReLU'(\Lambda_{5,j,r})\1_{\cH^\dagger_{\phi,1}}\Big] \label{eq:grad-computation-y-simply}
        \end{align}
        \item For any combination of \(p\), \(v\) not included above, we have due to our assumptions of the distribution and update rule:
        \begin{align*}
            \vbrack{-\nabla_{\Wb_{5,j,r,p}}\Loss, e_v} \equiv 0
        \end{align*}
    \end{enumerate}
\end{fact}

\begin{definition}[\(\Gamma\)-notations]
    To simplify gradient expression, we define the following notation: let \(\psi = (j,r,\phi) \in \Psi\), we write
    \begin{align}\label{eqdef:Gamma-notation}
        &\Gamma^{+,(t)}_{j,r,\phi} = \Gamma^{+,(t)}_\psi := \E\Big[(1 - \logit_{5,j})\ReLU'(\Lambda_{5,j,r})\1_{\cH_\phi}\Big] \nonumber\\
        &\Gamma^{-,(t)}_{j,r,\phi} = \Gamma^{-,(t)}_\psi := \E\Big[\logit_{5,j}\ReLU'(\Lambda_{5,j,r})\1_{\cH^\dagger_\phi}\Big]
    \end{align}
    we can further define \(\Gamma\) for each \(g \in \cG\) and \(y\in\cY\): Let \(v \in \cG\cup\cY\), we define
    \begin{align}\label{eqdef:Gamma-notation-gy}
        &\Gamma^{+,(t)}_{j,r,v} = \Gamma^{+,(t)}_{\psi,v} := \E\Big[(1 - \logit_{5,j})\ReLU'(\Lambda_{5,j,r})\1_{\cH_v}\Big] \nonumber\\
        &\Gamma^{-,(t)}_{j,r,v} = \Gamma^{-,(t)}_{\psi, v} := \E\Big[\logit_{5,j}\ReLU'(\Lambda_{5,j,r})\1_{\cH^\dagger_v}\Big]
    \end{align}
\end{definition}

This allows us to define the following gradient condition:
\begin{definition}[gradient condition]\label{def:gradient-condition}
    At \(t \leq T_1\), letting \(\psi = (j,r,\phi) \in \Psi\), we write \(\cC_{\psi}(\delta)\) for \(\delta > 0\) to denote that the gradient update for \(V_{\psi}\) is smaller than \(\delta\), formally:
    \begin{align*}
        & \cC_{j,r,v}(\delta) \text{ holds } \implies |\Gamma_{j,r,v}^{+,(t)} - \Gamma_{j,r,v}^{-, (t)}|  \leq \delta \\
        & \cC_{j,r,v}^+(\delta)\text{ holds } \implies |\Gamma_{j,r,v}^{+,(t)}|  \leq \delta \\
        & \cC_{j,r,v}^-(\delta)\text{ holds } \implies |\Gamma_{j,r,v}^{-,(t)}|  \leq \delta
    \end{align*}
    We further write 
    \begin{align*}
        &\cC_{\psi}(\delta)\text{ holds } \implies \cC_{j,r,g}(\delta) \wedge \cC_{j,r,y}(\delta) \text{ holds } \\
        &\cC_{\psi}^+(\delta)\text{ holds } \implies \cC_{j,r,g}^+(\delta) \wedge \cC_{j,r,y}^+(\delta) \text{ holds } \\
        &\cC_{\psi}^{-}(\delta)\text{ holds } \implies \cC_{j,r,g}^{-}(\delta) \wedge \cC_{j,r,y}^{-}(\delta) \text{ holds } 
    \end{align*}
\end{definition}

Now we proceed to characterize the lower bound of logits
\begin{fact}[logit lower bound]\label{fact:logit-lower-bound}
    We list some basic facts about the architecture \cref{def:transformer-arch} here.
    \begin{enumerate}
        \item For any \(Z^1 \in \supp(\cD^1) \), we have 
        \[\lambda := \min_{i,j,F,Z}(1 - \logit_{i,j}(F,Z)) = \frac{d-1}{d-1 + e^B}\]
        \item Suppose for some \(i\in[5]\), \(F_{i,j} = o(1)\) for all \(j \notin \tau(\cY)\), then 
        \[\min_{i,j,F,Z}\logit_{i,j}(F,Z) = \frac{1}{(1+o(1))d + n_ye^B} = O(\frac{\lambda}{n_y d})\]
    \end{enumerate}
\end{fact}

\subsection{Phase I: Feature Emergence and Competition}

We shall prove the following properties at initialization.
\begin{fact}[initialization range]\label{fact:range-parameter-simply}
    At \(t = 0\), the following holds with probability \(\geq 1 - o(1)\):
    \begin{enumerate}[(a)]
        \item \(|\vbrack{\Wb_{5,j,r,p}^{(0)}, e_v}| \leq O(\sigma_0\sqrt{\log d})\) For all \(j,r,p\) and \(v \in \cV\);
        \item \(V_{\psi}^{(0)} \geq V_{\psi'}^{(0)} + \gamma\), where \(\gamma = \Omega(\frac{\sigma_0}{n_y^4m^2\log d})\), for any pair \(\psi \prec \psi' \in \Sigma^\star\) in Def.~\ref{def:learning-curriculum-simply};
    \end{enumerate}
\end{fact}

\begin{proof}
    \cref{fact:range-parameter-simply}a can easily verified from our initialization \(\Wb_{i,j,r,p}^{(0)} \sim \cN(0, \sigma_0I_d)\). For \cref{fact:range-parameter-simply}b, we give a straightforward proof that every pair of \(V_{\psi}^{(0)}, \psi\in\Sigma^\star\) has a gap of \(\frac{\sigma_0}{n_y^4m^2 \log d}\). First note that \(V_{\psi}^{(0)}\) of different \(\psi \in \Sigma^\star\) are independent and identically distributed on the randomness of \(\Wb^{(0)}\), due to the orthogonality of embeddings \(e_v, v \in \cV\). Then, by the basic property of a Gaussian variable (notice that \(V_{\psi}^{(0)} - V_{\psi'}^{(0)}\) is also Gaussian with variance \(2\sigma_0\)) (all though different pairs could be dependent), we have with probability \(1 - \frac{1}{n_y^4m^2\log d}\) that their gap is at least \(\gtrsim \frac{\sigma_0}{n_y^4 m^2 \log d}\) for each pair. Then by a union bound over \( O(m^2 n_y^4)\)-many all possible pairs we can conclude the proof.
\end{proof}

We give another characterization of the initial activations.
\begin{fact}[activation magnitude]\label{fact:activation-magnitude-simply}
    At \(t = 0\), with high probability it holds that 
    \[\Lambda_{5,j,r}^{(0)} = (1 + O(\frac{1}{\sqrt{\log d}}))\mu = \Theta(\sigma \log d) \ll \varrho\]
    and thus \(\ReLU(\Lambda_{5,j,r}) = \Theta(\frac{1}{q}\mu^q)\), \(\ReLU'(\Lambda_{5,j,r}) = \Theta(\mu^{q-1})\) and \(F_{5,j} = \Theta(\frac{m}{q}\mu^q)\).
\end{fact}

\begin{proof}
    Combining \cref{def:smooth-relu,def:transformer-arch,assump:init,fact:range-parameter-simply} gives the fact.
\end{proof}

We establish some properties in Phase I.
\begin{lemma}[Key properties in Phase I]\label{lem:phase-I-simply}
    Let \(\psi = (j,r,(g,y)) \in \Sigma^\star\), and assume \cref{induction:mlp-simply-transitive,induction:individual-feature-simply} holds at \(t \leq T_{\psi}
    ^1\), then
    \begin{enumerate}[(a)]
        \item \(|V_{j,r}^{(t)}(v)| \leq O(\mu/\sqrt{\log d})\) if \(v\neq g, v\neq y\);
        \item \(\Lambda_{5,j,r}^{(t)}(\Zb) \geq 0\) for \(\Zb \in \cH_\phi\);
        \item \(\logit_{5,j} = O(1/d)\) whenver \(\cH_\phi^\dagger\) happens.
    \end{enumerate}
\end{lemma}

\begin{proof}
    \cref{lem:phase-I-simply}a is from both the initialization of weights \cref{fact:range-parameter-simply} and \cref{induction:individual-feature-simply}a. \cref{lem:phase-I-simply}b is from \cref{lem:phase-I-simply}a and \cref{induction:mlp-simply-transitive}a. \cref{lem:phase-I-simply}c is due to both \cref{induction:mlp-simply-transitive}b and \cref{induction:mlp-simply-transitive}b. 
\end{proof}

\subsubsection{Competition between feature combinations.}

In this phase we compare the growth of different features using a proxy.

\begin{lemma}[Approximating gradient with proxy]\label{lem:grad-proxy-simply}
    Assume \cref{induction:mlp-simply-transitive}. For a fixed \(\psi=(j,r,\phi)\in\Sigma^{\star}\), define the auxiliary sequence
    \[
        \wt{V}_{\psi}^{(0)} := V_{\psi}^{(0)},\qquad
        \wt{V}_{\psi}^{(t+1)} := \wt{V}_{\psi}^{(t)} + \eta\,\E\Big[\ReLU'(\wt{\Lambda}_{5,j,r}^{(t)})\,\1_{\cH_\phi}\Big],
    \]
    where in \(\wt{\Lambda}_{5,j,r}^{(t)}\) the weight \(V_{\psi}\) is replaced by \(\wt{V}_{\psi}\) (all other weights are unchanged). Then for any feature \(\psi \succ \psi_i\) or \(\psi\in\Sigma_i\setminus\{\psi_i\}\) and all \(t\le T^{1}_{\psi}\), \(|V_{\psi}^{(t)}-\wt{V}_{\psi}^{(t)}\big|\ \le\ O(\wt{V}_{\psi}^{(t)} / d^{1/2 - 3c_1})\).
\end{lemma}

\begin{proof}
    We need to prove the result by induction. Firstly we notice that the newly defined proxy sequence \(\wt{V}_{\psi}^{(t)}\) is monotonically increasing. Moreover, as \(\wt{V}_{\psi}^{(t)}\) is initialized the same as \(V_\psi^{(t)}\), we can use \cref{lem:tpm-in-expectation,lem:bernstein-u-stats} to argue its growth by taking 
    \begin{align*}
        x_t = \wt{V}_{\psi}^{(t)} + b_{i,j,r},\quad \xi = \sum_{p = 3,4}\vbrack{\Wb_{5,j,r,p}, e_{x_0}} + \vbrack{\Wb_{5,j,r,1}, e_{x_1}},\quad x_{t+1} \geq x_t + \eta \frac{1}{n_y^2}\E[(x_t + \mu)^{q-1}]
    \end{align*}
    And \cref{lem:tpm-in-expectation,lem:TPM} guarantee \(\wt{V}_{\psi}^{(t)} \geq d^{c_1}\mu\) after \(\tO(\frac{1}{\eta \mu^{q-2}})\) steps.
    Now assume the bound of difference between \(V_{\psi}^{(t)}\) and \(\wt{V}_{\psi}^{(t)}\) holds at a \(t \leq T_{\psi}^1\).
    Since \(t\le T^1_{\psi}\), we have \(V_{\psi}^{(t)}\le d^{c_1}\sigma_0\) by definition of \(T^1_{\psi}\). By \cref{fact:grad-computation-simply}, the exact update of \(V_{\psi}\) is
    \begin{align*}
        V_{\psi}^{(t+1)}
        &= V_{\psi}^{(t)} + \eta \Big(\E[(1-\logit^{(t)}_{5,j})\ReLU'(\Lambda^{(t)}_{5,j,r})\1_{\cH_\phi}] - \tfrac{1}{2}\,\E[\logit^{(t)}_{5,j}\ReLU'(\Lambda^{(t)}_{5,j,r})\1_{\cH^{\dagger}_{\phi}}]\Big).
    \end{align*}
    Compare the above with the definition of \(\wt{V}_\psi^{(t)} \), we first we notice that \(\wt{V}_\psi^{(t)} \) is dominating \(V_\psi^{(t)}\) for all iterations \(t\leq T_{\psi}^1\). Then, we can bound the difference between the two sequences by
    \begin{align*}
        &\big|\wt{V}_{\psi}^{(t+1)}-V_{\psi}^{(t+1)}\big| \\
        \leq \ & \big|\wt{V}_{\psi}^{(t)}-V_{\psi}^{(t)}\big| + \eta \Big|\E\big[\ReLU'(\wt{\Lambda}^{(t)}_{5,j,r})-\ReLU'(\Lambda^{(t)}_{5,j,r})\1_{\cH_\phi}\big]\Big| \\
        &\quad + \eta \Big|\E\big[\logit^{(t)}_{5,j}\ReLU'(\Lambda^{(t)}_{5,j,r})\1_{\cH_\phi}\big]\Big| + \frac{\eta}{2}\Big|\E\big[\logit^{(t)}_{5,j} \ReLU'(\Lambda^{(t)}_{5,j,r})\1_{\cH^{\dagger}_{\phi}}\big]\Big|
    \end{align*}
    We bound the three error terms for \(t\le T^1_{\psi}\):
    \begin{itemize}
        \item[A.] Activation perturbation: by the smoothness of \(\ReLU'\) on \([0,\varrho]\) and that both pre-activations are in this range during Phase I, we have
        \[
            \Big|\E[(\ReLU'(\wt{\Lambda}_{5,j,r}^{(t)})-\ReLU'(\Lambda_{5,j,r}^{(t)}))\1_{\cH_\phi}]\Big| \leq \,|\wt{V}_{\psi}^{(t)}-V_{\psi}^{(t)}|\E[(\tilde{\Lambda}_{5,j,r}^{(t)})^{q-2}\1_{\cH_\phi}]
        \]
        where we have used the fact that \(\wt{\Lambda}_{5,j,r} \geq \Lambda_{5,j,r}\) because of the dominance of \(\wt{V}_\psi\).
        
        \item[B.] Outside a set of iterations \(\cA_\phi = \{t \leq T_{\psi}^1\mid \logit_{5,j}^{(t)} \geq \frac{1}{\sqrt{d}} \text{ conditioned on } \cH_\phi\}\)
        whose cardinality is no more than \(\tO(d^{\frac12+c_1}/\eta)\) by \cref{induction:individual-feature-simply}a, we have \(\logit^{(t)}_{5,j}\le d^{-1/2}\). Hence for \(t\notin\cA_\phi\)
        \[
            \Big|\E[\logit^{(t)}_{5,j}\ReLU'(\Lambda^{(t)}_{5,j,r})\1_{\cH_\phi}]\Big|\ \le\ O(d^{-1/2}) \E[\ReLU'(\wt\Lambda^{(t)}_{5,j,r})\1_{\cH_\phi}]\
        \]
        Now we can sum over the iterations to get 
        \begin{align*}
            &\left(\sum_{s\leq t, s \in \cA_\phi} +  \sum_{s\leq t, s \notin \cA_\phi} \right)\eta \Big|\E\big[\logit^{(s)}_{5,j}\ReLU'(\Lambda^{(s)}_{5,j,r})\1_{\cH_\phi}\big]\Big| \\
            \leq \ & O(\frac{d^{\frac{1}{2} + c_1}}{\eta}) \cdot \eta (d^{c_1}\mu)^{q-1}  + O(d^{-1/2})(\wt{V}_\psi^{(t)} - \wt{V}_\psi^{(0)}) \\
            \leq \ & O(\mu / d^{1/2 - c_1})
        \end{align*}
        \item[C.] Confusing events: before \(V_\psi^{(t)}\geq \frac{1}{2}\log d\), we have \(\logit_{5,j} = O(1/d)\) in Phase I by \cref{lem:phase-I-simply}. So we get
        \[
            \sum_{s\leq t}\eta\E[\logit^{(t)}_{5,j}\,|\ReLU'(\Lambda^{(t)}_{5,j,r})|\1_{\cH^{\dagger}_{\phi}}]\ \le\ O(\frac{1}{\eta \mu^{q-2}})\eta \Big(\frac{(d^{c_1}\mu)^{q-1}}{d}\Big) \ll O(\mu/ \sqrt{d})
        \]
    \end{itemize}
    So we have that the accumulated errors from the last two terms are smaller than \(O(\mu / d^{1/2 - c_1})\). This allow us to bound the difference by using the more naive approach: we can assume the difference started with \(O(\mu / d^{1/2 - c_1})\) and now we only need to bound the following sequence for \(t\leq T_{\psi}^1\)
    \begin{align*}
        \delta_{t+1} = \delta_t + O(\eta\delta_t)\E\bigg[\frac{1}{\wt\Lambda_{5,j,r}^{(t)}}\ReLU'(\wt\Lambda_{5,j,r}^{(t)})\1_{\cH_\phi}\bigg] ,\quad \delta_0 \leq O(\mu / d^{1/2 - c_1})
    \end{align*}
    Now let \(\delta' > \eta\) be some small parameter, we can do some slicing of time \(t_0, t_1, \dots, t_n\) where \(t_i = \min\{t\geq 0 \mid \wt{V}_\psi^{(t)} \geq (1 + \delta')^i\mu \}\) and \((1 + \delta')^n\mu \leq \wt{V}_\psi^{(t+1)} \leq (1 + \delta')^{n+1}\mu\). This produce the following bound:
    \begin{align*}
        & \log\delta_{t+1} - \log\delta_0 \\
        = \ & \sum_{1\leq i \leq n}\sum_{s\in[t_i,t_{i-1}]} \log \left(1 + O(\eta)\E\bigg[\frac{1}{\wt\Lambda_{5,j,r}^{(s)}}\ReLU'(\wt\Lambda_{5,j,r}^{(s)})\1_{\cH_\phi}\bigg] \right) \\
        \leq \ & \sum_{1\leq i \leq n}\sum_{s\in[t_i,t_{i-1}]} O(\eta)\E\bigg[\frac{1}{\wt\Lambda_{5,j,r}^{(s)}}\ReLU'(\wt\Lambda_{5,j,r}^{(s)})\1_{\cH_\phi}\bigg] \tag{\(\log(1 + x) \leq x\) when \( x> 0\)}\\
        \leq \ &  \sum_{1\leq i \leq n} O\bigg(\frac{1}{\min_{s\in [t_i, t_{i-1}], \Zb\in \cH_\phi} \tilde{\Lambda}_{5,j,r}^{(s)}(\Zb)}\bigg) \sum_{s\in[t_i,t_{i-1}]}\eta\Big|\E[\ReLU'(\tilde{\Lambda}_{5,j,r}^{(s)})\1_{\cH_\phi}]\Big| \\
        \leq \ & \sum_{1\leq i \leq n} O\bigg(\frac{1}{(1 + \delta')^i \mu}\bigg)(\wt{V}_{\psi}^{(t_{i+1})} - \wt{V}_{\psi}^{(t_{i})}) \\
        \leq \ & \frac{\delta'\log(\wt{V}_\psi^{(t)}/\mu)}{\log(1 + \delta')} \\
        \leq \ & 2\log(\wt{V}_\psi^{(t)}/\mu) \tag{\(\log(1 + \delta')^{\frac{1}{\delta'}} < 2\) if \(\delta'\) is small enough}
    \end{align*}
    So when \(\wt{V}_\psi^{(t)} \in [d^{c_1}\mu, 2d^{c_1}\mu]\), it holds that \(\delta_t \lesssim d^{2c_1}\delta_0 \lesssim \mu/d^{1/2-3c_1}\). This proves the claim.
\end{proof}

\begin{lemma}[competition]\label{lem:competition-simply}
    Assuming \cref{induction:mlp-simply-transitive}, and let \(\psi \in \Sigma^\star\), we shall have for any feature \(\psi' \succ \psi \in \Sigma^\star\) or \(\psi' \in \Sigma_\psi \setminus \{\psi\}\) that \(V_{\psi'}^{(T_{\psi}^1)} \leq \tO(\sigma_0) \).
\end{lemma}

\begin{proof}
    We can approximate the sequence \(V_{\psi_i}\) using \(\wt{V}_{\psi_i}^{(t)}\) and straightforwardly use \cref{lem:tpm-in-expectation} to compare \(\psi_i\) with any \(\psi' \succ \psi_i \in \Sigma^\star\). In fact, let's compared the time they reach \(V_\psi^{(t)} \geq d^{c_1}\sigma_0\). By \cref{lem:grad-proxy-simply}, we can just compare \(\wt{V}_{\psi}^{(t)}\) before \(t = T_{\psi}^1\), now since by \cref{fact:range-parameter-simply} we know \(\wt{V}_{\psi}^{(0)} \geq \wt{V}_{\psi'}^{(0)} + \gamma\), we can use \cref{lem:tpm-in-expectation} to get that the first time when \(\wt{V}_{\psi_i}^{(t)} \geq d^{c_1}\sigma_0\), we have \(\wt{V}_{\psi'}^{(t)} = \tO(\sigma_0)\). This combined with \cref{lem:grad-proxy-simply} concludes the result. Moreover, we can also obtain bounds on all \(\psi \in \Psi\) by discussing different \(\psi\):
    \begin{itemize}
        \item For all \(\phi' \in \Phi^\star\) such that there exists \(\psi' = (j',r',\phi') \prec \psi_i\), \(\phi'\) is learned. That means any \(\psi\) contains \(\phi'\)) has gradient too small because of \ref{induction:individual-feature-simply}c and therefore \(\E[\cE_{5,j}\1_{\cH_{\phi'}}] \leq \tO(\lambda |\Gamma_{\psi}^{(t)}|)\) and therefore wouldn't be learned.
        \item All the rest of \(\phi' \in \Phi^\star\) has their corresponding \(\psi' \succ \psi \in \Sigma^\star\) and therefore lost in the competition and has small growth.
        \item All other \(\psi \in \Psi \setminus \cup_{i'\leq i}\Sigma_{i'}\) lost the competition as well for the same reason. 
    \end{itemize}
    This concludes the proof.
\end{proof}

\begin{proof}[Proof of \Cref{induction:mlp-simply-transitive,induction:individual-feature-simply} in Phase I]
    Let \(\psi_i \in \Sigma^\star\). Assume \cref{induction:mlp-simply-transitive} holds for some \(t \leq T_{\psi_{i-1}}^3\), then we shall prove that it continues to hold for \(t \in [T_{\psi_{i-1}}^3, T_{\psi_i}^1]\).
    \begin{itemize}
        \item \cref{induction:mlp-simply-transitive}a: This is simple as all \(V_\psi\) for \(\psi \succ \psi_{i}\) has updates approximated by \(\wt{V}_\psi^{(t)}\), which is monotonically increasing, and has minimal update by \cref{lem:grad-proxy-simply}.
        \item \cref{induction:mlp-simply-transitive}b: We proved in \cref{lem:competition-simply} that whenever \(\psi \prec \psi' \in \Sigma^\star\), \(T_{\psi}^1 \leq T_{\psi'}^1\). This does not violate \cref{induction:mlp-simply-transitive}b so the non-overlapping is preserved.
        \item \cref{induction:individual-feature-simply}a: This is proven in \cref{lem:competition-simply}.
        \item \cref{induction:individual-feature-simply}b-c: They will be proven when we analyze the end phase of feature learning. They are the corollaries of the end phase guarantee of the previous features.
        \item \cref{induction:individual-feature-simply}c: 
        \item \cref{induction:mlp-simply-transitive}e: We did not violate \cref{induction:mlp-simply-transitive}e either.
        \item \cref{induction:individual-feature-simply}a-d will be proved in Phase III where we argue the end phase of each feature learned.
    \end{itemize}
\end{proof}

\subsection{Phase II: Feature Growth and Cancellations}

Below we present the induction hypothesis for phase II.
\begin{induction}[Phase II]\label{induction:phase-II-simply}
    For all \(\psi = (j,r,\phi) \in \Sigma^\star\) and \(t \in [T_{\psi}^1, T_{\psi}^{2}]\), the following holds
    \begin{enumerate}[(a)]
        \item For \(t \in [T_{\psi}^1, T_{\psi}^{2,1}]\), \(\cE_{5,j}^{(t)} \geq 1 - \frac{1}{\sqrt{d}}\) conditioned on \(\cH_\phi\);
        \item For \(t \leq T_{\psi}^{2,1}\), for all \(j' \neq j\), we have \(\logit_{5,j'}^{(t)} \leq O(\sqrt{d})\cdot \logit_{5,j'}^{(T_{\psi}^1)}\).
        \item For any \(\phi' = (g',y') \neq \phi\), \(V_{j,r}^{(t)}(g'), V_{j,r}^{(t)}(y') \in (-B - O(1), \tO(\mu))\);
        \item For \(t \leq T_{\psi}^{2,2}\), the number of iterations where \(V_\psi^{(t+1)} - V_\psi^{(t)} < 0\) is bounded by \(\tO(1/d^{c_1}\eta)\).
    \end{enumerate}
\end{induction}

\subsubsection{Technical Lemmas}

\begin{lemma}[gradient estimation]\label{lem:grad-estimate-phase-2-simply}
    Assuming \cref{induction:mlp-simply-transitive,induction:phase-II-simply}, let \(\psi = (j,r,\phi) \in \Sigma^\star, \phi = (g,y)\) and \(t \in [T_{\psi}^{2,1}, T_\psi^2]\), then 
    \begin{enumerate}[(a)]
        \item At any \(t \geq T_{\psi}^{2,1}\), if \(V_{\psi}^{(t)} \in (\frac{1}{3}\log d,\frac{2}{3}\log d)\), then \(\Gamma_{j,r,\phi}^{(t)} \geq \Omega(1)\).
        \item Any \(g' \neq g \in \cG, y'\neq y \in \cY\) has the following gradient approximation
        \begin{align*}
            \Gamma_{j,r,g'}^{(t)}  &= - \E\Big[\logit_{5,j}^{(t)}\ReLU'(\Lambda_{5,j,r}^{(t)})\1_{\cH_{(g',y)}}] \pm O(\varpi/n_y) \\
            \Gamma_{j,r,y'}^{(t)}  &= - \E\Big[\logit_{5,j}^{(t)}\ReLU'(\Lambda_{5,j,r}^{(t)})\1_{\cH_{(g,y')}}] \pm O(\varpi/n_y)
        \end{align*}
        \item let \(\phi' = (g',y') \neq \phi \in \Phi^\star_j\) if \(V_{j,r}(g,y) \geq B - o(1)\) and \(V_{j,r}(g',y') < -B-2\mu\), we have 
        \begin{align*}
            \Gamma_{j,r,g'}^{(t)} & = \sum_{\tilde{y}\in\cY, V_{j,r}(g',\tilde{y}) \in (-B, -\varpi - \mu) }\Omega(\frac{\lambda \varpi}{d n_y^3})  -  \E\Big[\logit_{5,j}^{(t)}\ReLU'(\Lambda_{5,j,r}^{(t)})\1_{\cH_{(g',y)}}\Big]
        \end{align*}
        and similarly for \(\Gamma_{j,r,y'}^{(t)}\).
    \end{enumerate}
\end{lemma}

\begin{proof}
    Firstly (a) is simple to prove, because when \(V_{\psi}^{(t)}\in (\frac{1}{3}\log d,\frac{2}{3}\log d)\), we have \(\logit_{5,j}\1_{\cH_{\phi}^\dagger} \leq O(\frac{1}{d^{c_1}})\), which is from the calculation \(F_{5,j} \leq \frac{2}{3}\log d + O(1)\) guaranteed by \cref{induction:individual-feature-simply,induction:phase-II-simply}. For (b), we can estimate the gradient update of \( V_{j,r}^{(t)}(g')\) where \(g' \neq g\). In fact
    \begin{align*}
        -\nabla_{V_{j,r}(g')}\Loss^{(t)}  & = \Gamma_{j,r,g'}^{(t)} =  \Gamma_{j,r,g'}^{+,(t)} - \Gamma_{j,r,g'}^{-,(t)} \\
        & = \Gamma_{j,r,g'}^{+,(t)} - \sum_{y'\in\cY}\E\Big[\logit_{5,j}^{(t)}\ReLU'(\Lambda_{5,j,r}^{(t)})\1_{\cH_{(g',y')}}\Big] \\
        & = - \E\Big[\logit_{5,j}^{(t)}\ReLU'(\Lambda_{5,j,r}^{(t)})\1_{\cH_{(g',y)}}] \pm \max\{\tO(\mu^{q-1}), \varpi\}
    \end{align*}
    where the last one is because for \(y'\neq y\), the activation gradient \(\ReLU'(\Lambda_{5,j,r}^{(t)})\) is either positive and bounded by \(\tO(\mu^{q-1})\) or negative and bounded by \(\varpi\) while \(\logit_{5,j}\1_{\cH_{\phi'}} \leq 1\). Similarly we can also get the same for \(V_{j,r}(y')\) where \(y'\neq y\):
    \begin{align*}
        -\nabla_{V_{j,r}(y')}\Loss^{(t)} & = \Gamma_{j,r,y'}^{(t)} = - \E\Big[\logit_{5,j}^{(t)}\ReLU'(\Lambda_{5,j,r}^{(t)})\1_{\cH_{(g',y)}}] \pm O(\varpi/n_y^2)
    \end{align*}
    To prove (c), notice that when \(V_{j,r}(g',y') < -B - 2\mu\), the part of \(\Gamma_{j,r,g'}^{+,(t)} = 0\) due to \(\Lambda_{5,j,r}\) exceeding the boundary \(-B\). We can group different \(\tilde{y}\neq y\in \cY\) by their feature magnitude and obtain the lower bound. The factor of \(\frac{\lambda}{dn_y}\) is due to the logit lower bound from \cref{fact:logit-lower-bound}.
\end{proof}

\begin{lemma}[Feature Magnitude]\label{lem:feature-magnitude-phase-2-simply}
    Assume \cref{induction:mlp-simply-transitive,induction:phase-II-simply} holds, then for any \(\psi \in \Sigma^\star\), 
    \begin{enumerate}[(a)]
        \item at \(t = T_{\psi}^{2,1}\), we have \(V_{j,r}^{(t)}(g) = \frac{1}{4}\log d \pm o(1)\), \(V_{j,r}^{(t)}(y) = \frac{1}{4}\log d \pm o(1)\);
        \item when \(\cC_\psi^+(\delta)\) hold for some \(\delta \leq O(\lambda/d^{c_1})\) at some \(t \geq T_{\psi}^{2,1}\), we have \(V_{\psi}^{(t)} \geq B - O(d^{c_1}\mu)\), that is, \(\cF_{\psi,1}(d^{c_1}\mu)\) holds.
    \end{enumerate}

\end{lemma}

\begin{proof}
    Let's proceed the proof one by one.
    \begin{itemize}
        \item Proof of (a) :This can be computed by comparing   the gradient of \(V_{j,r}^{(t)}(g)\) which is \(\Gamma_{j,r,g}^{(t)} = \Gamma_{j,r,g}^{+,(t)} - \Gamma_{j,r,g}^{-,(t)}\) and the gradient of \(V_{j,r}^{(t)}(y)\) which is \(\Gamma_{j,r,g}^{(t)} = \Gamma_{j,r,y}^{+,(t)} - \Gamma_{j,r,y}^{-,(t)}\). Now since 
        \begin{align*}
            \Gamma_{j,r,g}^{+,(t)}  = \Gamma_{j,r,y}^{+,(t)} \geq (1 - O(\frac{1}{\sqrt{d}}))\ReLU'(V_{j,r}^{(t)}(g) + V_{j,r}^{(t)}(y) + O(\mu))
        \end{align*}
        and that
        \[\Gamma_{j,r,g}^{-,(t)} = O(\frac{1}{\sqrt{d}}\ReLU'(V_{j,r}^{(t)}(g) + O(\mu))|),\quad \Gamma_{j,r,y}^{-,(t)} = O(\frac{1}{\sqrt{d}}\ReLU'(V_{j,r}^{(t)}(y) + O(\mu))|)\]
        one could see that \( \sum_{t\leq T_{\psi}^{2,1}} |\Gamma_{j,r,g}^{(t)}  - \Gamma_{j,r,y}^{(t)} | \leq \tO(\frac{1}{\sqrt{d}})\) and thus (a) holds.
        \item Proof of (b): from \cref{lem:grad-estimate-phase-2-simply}a we know that when \(\cC_\psi^+(\delta)\) holds for some small \(\delta = O(\frac{\lambda}{d^{c_1}})\) it is not because \(V_{\psi}^{(t)}\) dropped below \(\frac{1}{2}\log d\), thus when \cref{induction:mlp-simply-transitive,induction:phase-II-simply} holds, it can only be because \(V_{\psi}^{(t)} \geq B - O(d^{c_1}\mu)\), otherwise 
        \begin{align*}
            \Gamma_{\psi}^{+,(t)} = \E[(1 - \logit_{5,j}^{(t)})\1_{\cH_\phi}] \geq \lambda \Pr(\cH_\phi) \gg \frac{\lambda}{d^{c_1}} \tag{if \(V_{\psi}^{(t)} \leq B - O(d^{c_1}\mu)\)}
        \end{align*}
        which concludes the proof.
    \end{itemize}
    
\end{proof}

\begin{lemma}[Gradient Stationarity]\label{lem:grad-stationarity-simply}
    Let \(\psi \in \Sigma^\star\), assume \cref{induction:mlp-simply-transitive,induction:individual-feature-simply} holds at \(t \in [T_{\psi}^1, T_{\psi}^2]\), and choose \(\delta \gg n_y\varpi\), then we have the following guarantees: 
    \begin{itemize}
        \item[(a)] Define \( \cB_\delta^+ = \{t \in [T_{\psi}^{2,1}, T_\psi^2] \mid \cC^{+}_\psi(\delta) \text{ doesn't hold }\}\), then \(|\cB_\delta^+| \leq O(\frac{n_y\log^2d}{\eta\delta})\);
        \item[(b)] Define \( \cB_\delta^- = \{t \in [T_{\psi}^{2,1},T_\psi^2] \mid \cC^{-}_\psi(\delta) \text{ doesn't hold }\}|\), then \(|\cB_\delta^-| \leq O(\frac{n_y\log^3d}{\eta\delta})\);
        \item[(c)] Finally, at \(t = T_\psi^2\), the feature shape \(\cF_{\psi}(\delta_1,\delta_2)\) holds with \(\delta_1 = d^{c_1}\mu\) and \(\delta_2 = (\frac{n_y^3 d \varpi}{\lambda})^{\frac{1}{q-1}}\).
    \end{itemize}
\end{lemma}

\begin{proof}
    Let us start with (a). First by writing \(\psi = (j,r,\phi)\), we define the following quantity
    \begin{equation}\label{eqdef:core-quantity-simply}
        \Upsilon^{(t)} := V_{j,r}^{(t)}(\phi) - \sum_{\phi'\in \Phi^\star_j\setminus\{\phi\}}V_{j,r}^{(t)}(\phi')
    \end{equation}
    which sits at the central of our proof of counter-argument.Now we have for any \(T \geq T_{\psi}^{2,1}\)
    \begin{align*}
        \Upsilon^{(t)} & = V_{j,r}^{(T)}(\phi) - \sum_{\phi'\in \Phi^\star_j\setminus\{\phi\}}V_{j,r}^{(T)}(\phi') \\
        & = V_{j,r}^{(T_{\psi}^{2,1})}(\phi) - \sum_{\phi'\in \Phi^\star_j\setminus\{\phi\}}V_{j,r}^{(T_{\psi}^{2,1})}(\phi')  + \sum_{s\in [T_{\psi}^{2,1}, t)} \eta\Gamma_{j,r,\phi}^{(t)} - \sum_{\phi'\in \Phi^\star_j\setminus\{\phi\}}\sum_{s\in [T_{\psi}^{2,1}, t)} \eta\Gamma_{j,r,\phi'}^{(t)}
    \end{align*}
    We can further rewrite the \(\Gamma\)s on the RHS to 
    \begin{align}
        &\sum_{t\in [T_{\psi}^{2,1}, T)} \eta\Bigg(\Gamma_{j,r,\phi}^{(t)} - \sum_{\phi'\in \Phi^\star_j\setminus\{\phi\}}\eta\Gamma_{j,r,\phi'}^{(t)} \Bigg) \nonumber \\
        = \ &\sum_{t\in [T_{\psi}^{2,1}, T)}  \eta\Bigg(\Gamma_{j,r,\phi}^{+, (t)} - \Gamma_{j,r,\phi}^{-, (t)} - \sum_{\phi'\neq \phi \in \Phi_j^\star}\Gamma_{j,r,\phi'}^{+,(t)} + \sum_{\phi'\neq \phi \in \Phi_j^\star}\Gamma_{j,r,\phi'}^{-,(t)} \Bigg) \label{eq:expansion-core-quantity-simply}
    \end{align}
    Now we can start to estimate the RHS. Define \(P_1 = \{\phi'\neq \phi\in \Phi^\star\mid \phi'\text{ learned}\}\)\footnote{By \(\phi'\) is learned we mean there is a feature \(\psi'\prec \psi \in \Sigma^\star\) such that \(\psi' = (j,r,\phi')\).} and \(P_2 = \{\phi'\neq \phi\in \Phi^\star\mid \phi'\text{ not learned}\}\). The sum \( \sum_{\phi'\neq \phi \in \Phi_j^\star}\Gamma_{j,r,\phi'}^{+,(t)}\) can be decomposed and bounded by 
    \begin{align*}
        &- O(\eta n_y\varpi) \leq \sum_{\phi' \in P_1}\eta\Gamma_{j,r,\phi}^{+,(t)} +  \sum_{\phi'\in P_2}\eta\Gamma_{j,r,\phi}^{+,(t)}\leq  \tO(\eta \lambda\mu^{q-2})
    \end{align*}
    and then by the same methods of bounding the gradients, we have that 
    \begin{align*}
        &\quad \sum_{\phi'\neq \phi \in \Phi_j^\star}\Gamma_{j,r,\phi'}^{-,(t)}  \\
        &= \sum_{\phi' = (g',y') \in \Phi_j^\star \setminus \{\phi\}} \E\Bigg[\logit_{5,j}^{(t)}\ReLU'(\Lambda_{5,j,r}^{(t)})\Big(\sum_{g''\neq g' \in \cG}\1_{\cH_{(g'',y')}}+ \sum_{y''\neq y' \in \cY}\1_{\cH_{(g',y'')}}\Big)\Bigg] \\
        & \stackrel{\ding{173}}{}= \sum_{g'\neq g} \E[\logit_{5,j}^{(t)}\ReLU'(\Lambda_{5,j,r}^{(t)})\1_{\cH_{(g',y)}}] + \sum_{y'\neq y} \E[\logit_{5,j}^{(t)}\ReLU'(\Lambda_{5,j,r}^{(t)})\1_{\cH_{(g,y')} }] \pm O(\varpi) \\
        & = \Gamma_{j,r,\phi}^{-,(t)} \pm O(\varpi)
    \end{align*}
    where \ding{172} is because for any incorrect feature combination \(\phi' = (g',y') \notin \Phi^\star_j\) and does not share a component with \(\phi\), the activation of \(\Lambda_{5,j,r}\) conditioned on \(\cH_{\phi'}\) is within \([-B, \tO(\mu^{q-1})]\) and thus the negative gradient is upper bounded by \(\varpi/n_y^2\) and lower bounded by \(-\tO(\mu^{q-1})\). Inserting this back to \cref{eq:expansion-core-quantity-simply}, we shall have 
    \begin{align*}
        \sum_{t\in [T_{\psi}^{2,1}, T)} \eta\Bigg(\Gamma_{j,r,\phi}^{(t)} - \sum_{\phi'\in \Phi^\star_j\setminus\{\phi\}}\eta\Gamma_{j,r,\phi'}^{(t)} \Bigg)
        = \ & \sum_{t\in [T_{\psi}^{2,1}, T)}\eta\Gamma_{j,r,\phi}^{+, (t)} \pm O((T - T_{\psi}^{2,1})\eta \varpi  )
    \end{align*}
    Now we have arrived at the desired estimate of the update, finally, we have the following counter-argument: if \(\cC_\psi^+(\delta)\) does not hold for \(\Omega(\frac{n_y\log^2d}{\eta\delta})\) iterations, letting \(T-1\) be the last iteration which \(\cC_\psi^+(\delta)\) does not hold, we have 
    \begin{align*}
        \Upsilon^{(T)} &\geq \sum_{t\in [T_{\psi}^{2,1}, T)}\eta\Gamma_{j,r,\phi}^{+, (t)} \\
        &\geq  \sum_{t\in [T_{\psi}^{2,1}, T]\setminus B_\delta^+ }\eta\delta +  \Omega((T - T_{\psi}^{2,1}) \eta \varpi) \tag{For \(t\in B_\delta^+\), \(\Gamma_{j,r,\phi}^{+, (t)} \geq 0\)} \\
        & \geq \sum_{t\in [T_{\psi}^{2,1}, T]\setminus B_\delta^+ }\eta\delta - O( \frac{n_y\log^2d}{\eta\delta}\eta \varpi ) \\
        & \geq \Omega(\frac{n_y\log^2d}{\eta\delta}) \eta \delta  \tag{by choosing \(\delta \gg \varpi\)}\\
        & \geq \Omega(n_y\log^2 d)
    \end{align*}
    which is impossible because the gradient would have vanished when \(\Upsilon^{(T)} \geq n_y \log^{1.5} d\) as that would mean \(\exists \phi \in \Phi_j^\star\) such that \(|V_{j,r}^{(t)} (\phi)| \geq \log^2 d \) and it is forbidden in our setting, which proves the claim via contradiction.
    
    \textbf{Proof of \cref{lem:grad-stationarity-simply}b}. Actually, by applying (a) combined with \cref{lem:feature-magnitude-phase-2-simply}b, we know that for some \(T \geq T_{\psi}^{2,1} + \tOmega(\frac{1}{\eta \lambda})\) it holds \(V_{\psi}^{(T)} \geq B - O(d^{c_1}\mu)\). From this iteration forward we shall prove that \(V_{\psi}^{(t)}\) stays above \(B - O(d^{c_1}\mu)\) and can help. In fact, at \(t = T\), \(V_{j,r}^{(t)}(V) \in (-O(B), \tO(\mu))\) for \(v \in \cG\cup\cY \setminus\{g,y\}\) by \cref{induction:phase-II-simply}. Say \(v = g'\neq g\), then due to \cref{lem:grad-estimate-phase-2-simply}, its gradient is approximated by
    \begin{align}\label{eq:grad-g-prime-phase2-simply}
        - \E\Big[\logit_{5,j}^{(t)}\ReLU'(\Lambda_{5,j,r}^{(t)})\1_{\cH_{(g',y)}}] \pm O(\varpi/n_y)
    \end{align}
    Let \(\delta \geq \varpi\), we will find an iteration \(t \leq T_{\psi}^{2,1} + O(\frac{n_y \log^3 d}{\eta\delta})\) that satisfy \(\cC_{\psi}^{-}(\delta)\). By applying (a), we first get that \(\cC_\psi^+(\delta)\) holds for no more than \(O(\frac{n_y \log^2 d}{\eta\delta})\) iters after \(T_{\psi}^{2,1}\) and before \(T_{\psi}^{2}\).
    
    Let's assume this is the case, that \(\cC_{\psi}^{-}(\delta)\) doesn't hold for more than \(\Omega(\frac{n_y \log^3 d}{\eta\delta})\) iterations.
    In this case, it must be that the gradient \(\Gamma_{j,r,\phi}^{-,(t)} < -\delta\) is negative, since when the gradient is positive it has an upper bound \(O(\varpi/n_y) = o(\delta)\) and therefore cannot violate the condition \(\cC_\psi^+(\delta)\). Therefore the updates \(\sum_{s\leq t}\Gamma_{j,r,\phi}^{-,(s)}\) must accumulate to lower than 
    \[
        V_{j,r}^{(t)}(\phi) \leq \Upsilon^{(t)} -\Omega(\frac{n_y \log^3 d}{\eta\delta}) \times \eta\delta = - \Omega(n_y \log^3 d)
    \] 
    which is impossible in our setting. This proved the result for any \(\delta \gg n_y\varpi\). 
    
    \textbf{Proof of \cref{lem:grad-stationarity-simply}c}. Let \(\delta = n_y^2\varpi\), which is covered by \cref{lem:grad-stationarity-simply}a-b, we can find an iteration where \(|\Gamma_{j,r,\phi'}^{(t)}| \leq \delta\)  \(\phi'\in \Phi^\star_j\setminus\{\phi\}\). This implies we have found a \(t \) such that
    \begin{align*}
        |\E[\logit_{5,j}^{(t)}\ReLU'(\Lambda_{5,j,r})\1_{\cH_{\phi'}}]| \leq n_y^2\varpi \implies |\ReLU'(V_{j,r}^{(t)}(\phi') + O(\mu))| \leq O(\frac{n_y^3 d \varpi}{\lambda}) \tag{We lower bounded the logit by \cref{fact:logit-lower-bound}} 
    \end{align*}
    Thus we arrive at \(|V_{j,r}(\phi')| \leq O((\frac{n_y^3 d \varpi}{\lambda})^{\frac{1}{q-1}})\)
    Remember in the proof of b, \(V_{\psi}^{(t)} = \frac{1}{2}( V_{j,r}^{(t)}(g) + V_{j,r}^{(t)}(y)) \geq B - O(d^{c_1}\mu)\) at \(t \geq T + \tOmega(\frac{d^{c_1}}{\eta \lambda})\). So we have for any \(\phi' = (g',y')\in \Phi^\star_j\setminus\{\phi\}\),
    \[
        \frac{1}{2}(V_{j,r}^{(t)}(g') + V_{j,r}^{(t)}(y')) \leq - \frac{1}{2}(V_{j,r}^{(t)}(g) + V_{j,r}^{(t)}(y) ) + O(\frac{1}{d^{1/3q}}) \leq -B + O((\frac{n_y^3 d \varpi}{\lambda})^{\frac{1}{q-1}})
    \]
    Now we have verified that given \(\delta = n_y^2\varpi\), there is a iteration \(T_{\psi}^{2,3}\) which gives the near-perfect feature shape \(\cF_\psi(\delta_1,\delta_2)\) with \(\delta_1 = d^{c_1}\mu, \delta_2 = (\frac{n_y^3 d \varpi}{\lambda})^{\frac{1}{q-1}}\). We shall prove that this feature shape will hold until \(t = T_{2}\).
    
    For \(t \geq T_{\psi}^{2,3} \), we verify \(\cF_\psi(\delta_1,\delta_2)\) in \cref{def:feature-shape-simply} one by one:
    \begin{itemize}
        \item[1.] \(\cF_{\psi,1}(d^{c_1}\mu)\): By \cref{lem:grad-stationarity-simply}a we shall find a step after \(t \geq T_{\psi}^{2,1} + O(\frac{d^{c_1}}{\eta \lambda})\) at which \(\cF_{\psi,1}(d^{c_1}\mu)\) is satisfied. In fact, the majority of the steps after \(T_{\psi}^{2,1} + O(\frac{d^{c_1}}{\eta \lambda})\) will satisfy this. We shall prove that combined with the above feature shape guarantee, \(\cF_{\psi,1}(d^{c_1}\mu)\) will hold true until \(t = T_{1}\).
        \item[2.] \(\cF_{\psi,2}((\frac{n_y^3 d \varpi}{\lambda})^{\frac{1}{q-1}})\): From the above argument that applies \cref{lem:grad-stationarity-simply}b, we shall find a time step where \(\cF_{\psi,1}(d^{c_1}\mu)\) and \(\cF_{\psi,2}(d^{-1/3q})\) hold. Now we discuss the possible values of the gradient \(\Gamma_{\psi'}^{-,(t)}, \psi' \in \Sigma_\psi^{\dagger,1}\). In fact, let \(\phi' = (g',y') \in \Phi^\star_j \setminus \{\phi\}\), we have
        \begin{itemize}
            \item Regime A: If \(V_{j,r}(\phi') < -B - 2\mu\), then there is no negative gradient for \(V_{j,r}(g')\) and \(V_{j,r}(y')\) unless \(V_{j,r}(g',y)\) or \(V_{j,r}(g,y')\) is positive. In fact, only one of the feature can be positive to have negative gradient, and in that case the gradient of \(\Gamma_{j,r,v}^{(t)}, v \in \{g',y'\}\) is pointing to the convergence direction \(|V_{j,r}(g',y)| \to 0\).
            \item Regime B: If \(V_{j,r}(\phi') \geq -B - 2\mu\), then there is a possibility of negative gradient for \(V_{j,r}(g')\) and \(V_{j,r}(y')\) even when \(V_{j,r}(g',y)\) or \(V_{j,r}(g,y')\) is negative, In this case the feature \(V_{j,r}(g',y)\) or \(V_{j,r}(g,y')\) might not go the desired direction to converge to zero, but given a \(\delta > n_y^2\varpi\), \(\Gamma_{j,r,(g',y)}^{-, (t)}\) or \(\Gamma_{j,r,(g,y')}^{-, (t)}\) cannot stay above \(\delta\), otherwise the term will dominate and revert the gradient direction. On the other hand, when one side has \(V_{j,r}(g,y') < - \Omega((\frac{n_y^3 d \varpi}{\lambda})^{\frac{1}{q-1}})\), we have 
            \begin{align*}
                V_{j,r}(\phi') \leq - V_{j,r}(\phi) - \Omega((\frac{n_y^3 d \varpi}{\lambda})^{\frac{1}{q-1}}) \leq -B - 2\mu
            \end{align*}
            which is back in regime A and therefore will converge back.
        \end{itemize}
        Induction over the above two regimes for all \(t \in [T_{\psi}^{2,3}, T_\psi^2]\) showed that the feature will indeed satisfy \(\cF_{\psi,2}(\delta)\) for all \(\delta > n_y^2\varpi\) at \(t\).
        \item[3.] \(\cF_{\psi, 3}((\frac{n_y^3 d \varpi}{\lambda})^{\frac{1}{q-1}})\): We need to use both \(\cF_{\psi, 1}\) and \(\cF_{\psi, 2}\) to get the desired result. In fact, we have for any \(\psi' = (j,r,(g',y')) \in \Sigma^{\dagger,2}\), 
        \begin{align*}
            V_{\psi'}^{(t)} = V_{j,r}^{(t)}(g') + V_{j,r}^{(t)}(y') + V_{j,r}^{(t)}(g) + V_{j,r}^{(t)}(y) - V_{j,r}^{(t)}(\phi) = - B + O((\frac{n_y^3 d \varpi}{\lambda})^{\frac{1}{q-1}})
        \end{align*}
        because \(V_{j,r}^{(t)}(g) + V_{j,r}^{(t)}(y')\) and \(V_{j,r}^{(t)}(g') + V_{j,r}^{(t)}(y)\) are bounded due to \(\cF_{\psi, 2}\).
        \item[4.] \(\cF_{\psi,4}(d^{c_1}\mu)\): The number of iterations \([T_{\psi}^1, T_{\psi}^2]\) is no more than \(\tO(\frac{1}{\eta (d^{c_1}\mu)^{q-2}})\), so by using \cref{lem:TPM} along with the induction from the result at \(T_{\psi}^1\) shall conclude the proof.
        \item[5.] \(\cF_{\psi,5}(((\frac{n_y^3 d \varpi}{\lambda})^{\frac{1}{q-1}}))\): By using \cref{lem:symmetry-of-features-simply} and \(\cF_{\psi,2}\) after \(t\geq T_{\psi}^{2,3}\), we have the result.
    \end{itemize}
    This proved (c).
\end{proof}

\begin{lemma}[Symmetry of Features]\label{lem:symmetry-of-features-simply}
    Let \(\psi = (j,r,(g,y)) \in \Sigma^\star\) and \(t \geq T_{\psi}^{2,3}\), we have that if \(|V_{\psi'}|\leq \delta\) for all \(\psi' \in \Sigma^{\dagger,1}\), then
    \begin{align*}
        |V_{j,r}^{(t)}(g) - V_{j,r}^{(t)}(y)| \leq O(\delta)
    \end{align*}
\end{lemma}

\begin{proof}
    For a feature \(\psi = (j,r, \phi) \in \Sigma^\star, \phi = (g,y)\), their update can be represented as:
    \begin{align*}
        V_{j,r}^{(t)}(g) & = U_{g,y} - \sum_{y'\neq y} R_{g,y'} + V_{j,r}^{(0)}(g)\\
        V_{j,r}^{(t)}(y) & = U_{g,y} - \sum_{g'\neq g} R_{g',y} + V_{j,r}^{(0)}(y)
    \end{align*}
    where \(U_{g,y} = \sum_{s\leq t} \eta \Gamma_{j,r,\phi}^{+,(s)}\) and \(R_{g,y'}\) is
    \begin{align*}
        R_{g,y'} = \sum_{s\leq t}\eta\E[\logit_{5,j}^{(s)} \ReLU'(\Lambda_{5,j,r}^{(s)} )\1_{\cH_{g',y}}]
    \end{align*}
    
    Therefore:
    \begin{align*}
        \sum_g  V_{j,r}^{(t)}(g) = \sum_{\phi \in \Phi^\star_j} U_{\phi} - \sum_g\sum_{y'\neq y} R_{g,y'} = \sum_{\phi \in \Phi^\star_j} U_{\phi} - \sum_{\phi^\dagger \in \Phi_j^\dagger} R_{\phi^\dagger} + V_{j,r}^{(0)}(g)\\
        \sum_y  V_{j,r}^{(t)}(y) = \sum_{\phi \in \Phi^\star_j} U_{\phi} - \sum_y\sum_{g'\neq g} R_{g,y'} = \sum_{\phi \in \Phi^\star_j} U_{\phi} - \sum_{\phi^\dagger \in \Phi_j^\dagger} R_{\phi^\dagger} + V_{j,r}^{(0)}(y)
    \end{align*}
    Thus 
    \begin{align*}
        \sum_{g'}  V_{j,r}^{(t)}(g') = \sum_{y'}  V_{j,r}^{(t)}(y') \pm O(\mu)
    \end{align*}
    Moreover, we have assumed
    \begin{align*}
        |V_{j,r}^{(t)}(g) + V_{j,r}^{(t)}(y')| = O(\delta), \forall y'\neq y, \quad |V_{j,r}^{(t)}(y) + V_{j,r}^{(t)}(g')| = O(\delta), \forall g'\neq g
    \end{align*}
    Thus 
    \begin{align*}
        \sum_{g'}  V_{j,r}^{(t)}(g') = V_{j,r}^{(t)}(g)  - (n_y-1)V_{j,r}^{(t)}(y)  \pm O(n_y\delta) = V_{j,r}^{(t)}(y)  - (n_y-1)V_{j,r}^{(t)}(g)  \pm O(n_y\delta)
    \end{align*}
    Therefore
    \begin{align*}
        V_{j,r}^{(t)}(g) = V_{j,r}^{(t)}(y) + O(\delta) + O(\mu)
    \end{align*}
    inserting \(\delta = O((\frac{n_y^2d\varpi}{\lambda})^{\frac{1}{q-1}})\) which is from \cref{lem:grad-stationarity-simply} proves the claim.
\end{proof}

\begin{lemma}[Arrival time estimates]\label{lem:arrival-time-estimate-simply}
    Assuming \cref{induction:mlp-simply-transitive,induction:phase-II-simply} for \(\psi \in \Sigma^\star\), we have
    \begin{enumerate}[(a)]
        \item \(T_{\psi}^{2,1} - T_{\psi}^1 \leq \tO(1/\eta (d^{c_1}\sigma_0)^{q-2})\);
        \item \(T_{\psi}^{2,2} - T_{\psi}^{2,1} \leq O(\frac{d^{\frac{1}{2}+c_1}}{\eta})\);
        \item \(T_{\psi}^{2,3} - T_{\psi}^{2,2}\leq \tO(\frac{1}{\eta\varpi})\);
        \item \(T_{\psi}^{2} - T_{\psi}^{1}\leq \tO(\frac{1}{\eta (d^{2c_1}\mu)^{q-2}})\)
    \end{enumerate}
    
\end{lemma}

\begin{proof}
    \begin{itemize}
        \item Proof of (a): Once \(V_{\psi}^{(t)} \geq d^{c_1}\sigma_0\), we can apply \cref{lem:TPM} to compute the number of iteration needed to reach \(V_\psi^{(t)} \geq \frac{1}{2}\log d\). In fact, by \cref{induction:phase-II-simply}a-b we know that 
        \begin{align*}
            V_\psi^{(t+1)} \geq V_\psi^{(t)} + \eta (1 - O(\frac{1}{\sqrt{d}}) ) \ReLU'( V_\psi^{(t)} + O(\mu)) \geq  V_\psi^{(t)} + \Omega(\eta(V_\psi^{(t)} )^{q-1})
        \end{align*}
        Since \(V_\psi^{(T_{\psi}^1)}\geq d^{c_1}\sigma_0\), we have by \cref{lem:TPM} that \(T_{\psi}^{2,1} - T_{\psi}^1 \leq O(\frac{1}{\eta(d^{c_1}\sigma_0)^{q-2}})\).
        \item Proof of (b): Simply apply \cref{lem:grad-stationarity-simply} with \(\delta = \frac{1}{d^{\frac{1}{2}+c_1}}\);
        \item Proof of (c): Again by applying \cref{lem:grad-stationarity-simply} with \(\delta = n_y^2\varpi\);
        \item Proof of (d): Follow the definition of \(T_{\psi}^2\).
    \end{itemize}
    
\end{proof}

\begin{proof}[Proof of \cref{induction:phase-II-simply}]
    The proof of time steps are all in \cref{lem:arrival-time-estimate-simply} and \cref{lem:grad-stationarity-simply}. We do not repeat here.
\end{proof}

\subsection{Phase III: Convergence}

We present the induction hypothesis in this phase.

\begin{induction}[Phase III, final]\label{induction:phase-III-simply}
    Let \(\psi = (j,r,\phi) \in \Sigma^\star\), for \(t \in [T_{\psi}^{2}, T_1]\), the following holds:
    \begin{enumerate}[(a)]
        \item At \(t \in [T_\psi^2, T_{1,1}]\), we have \(\cF_{\psi}(\delta_1, \delta_2)\) holds with \(\delta_1 = d^{c_1}\mu\), \(\delta_2 = (\frac{n_y^2d\varpi}{\lambda})^{\frac{1}{q-1}}\);
        \item At \(t \in [T_{\psi}^3, T_1]\), we have \(\cF_{\psi}(\delta_1, \delta_2)\) holds with \(\delta_1 = d^{c_1}\mu\), \(\delta_2 = (n_y^2\varpi)^{\frac{1}{q-1}}\).
    \end{enumerate}
\end{induction}

We need a lemma to describe the logit distribution when all features are learned.
\begin{lemma}[Logit shape at convergence]\label{lem:logit-shape-at-convergence}
    Assuming \cref{induction:mlp-simply-transitive,induction:phase-III-simply}. Let \(\psi \in \Sigma^\star\) be the last of \(\Sigma^\star\), then at \(t = T_{\psi}^2\), the followings hold:
    \begin{enumerate}
        \item For \(\phi \in \Phi^\star_j\), \(\logit_{5,j}^{(t)} \geq 1 - O(\lambda) \) conditioned on \(\cH_\phi\).
        \item For \(\phi \notin \Phi^\star_j\), \(\logit_{5,j}^{(t)} \leq O(\lambda /d) \) conditioned on \(\cH_\phi\).
    \end{enumerate}
\end{lemma}

\begin{proof}
    Since \(\psi\) is the last feature in the learning curriculum \(\Sigma^\star\), all feature at \(t = T_{\psi}^2\) has satisfied \(\cF_{\psi}(\delta_1,\delta_2)\) for \(\delta_1 = d^{c_1}\mu\), \(\delta_2 = (\frac{n_y^2d\varpi}{\lambda})^{\frac{1}{q-1}}\). Thus we can bound the logit by 
    \begin{align*}
        \logit_{5,j}^{(t)} = \frac{e^{F^{(t)}_{5,j}(\Zb) }}{e^{F^{(t)}_{5,j}(\Zb) }+ d} = \frac{e^{B - o(1)}}{e^{B - o(1)} + d} = 1 - O(\lambda)\tag{conditioned on \(\Zb \in \cH_\phi, \phi \in \Phi^\star_j\)}
    \end{align*}
    and 
    \begin{align*}
        \logit_{5,j}^{(t)} = \frac{e^{F^{(t)}_{5,j}(\Zb) }}{e^{F^{(t)}_{5,j}(\Zb) }+ d} = \frac{e^{o(1)}}{e^{B - o(1)} + d-1 + e^{o(1)}} =  O(\frac{\lambda}{d})\tag{conditioned on \(\Zb \in \cH_\phi, \phi \in \Phi^\star_j\)}
    \end{align*}
    which is the desired result.
\end{proof}

\begin{lemma}[Gradient Bounds, Phase III]\label{lem:grad-bound-phase-3-simply}
    Assuming \cref{induction:mlp-simply-transitive,induction:phase-III-simply} holds, and let \(\psi = (j,r,\phi)\in\Sigma^\star\), then at \(t\in [T_{1,1}, T_1]\), the followings hold
    \begin{enumerate}[(a)]
        \item \(\Gamma_{\psi}^{+,(t)}\geq \Omega(\lambda)\) when \(V_\psi^{(t)} \leq B - d^{c_1}\mu\);
        \item \(|\Gamma_{j,r,v}^{-,(t)}| \leq O(\frac{\lambda \varpi}{d n_y}) \) for any \(v \in \cG\cup\cY\);
        \item For \(\phi'\in \Phi^\star_j \setminus \{\phi\}\), we have
        \begin{align*}
            \Gamma_{j,r,\phi'}^{+,(t)} = 
            \begin{cases}
                -\Theta(\lambda\varpi/n_y^2), & \text{if } V_{j,r,\phi'}^{(t)} \in [-B + 2\mu, -\varpi] \\
                \in [-\Theta(\lambda\varpi/n_y^2), -\Theta(\lambda\varpi/dn_y^2)] \cup \{0\}, & \text{if } V_{j,r,\phi'}^{(t)} \in [-B - 2\mu, -B + 2\mu]
            \end{cases}
        \end{align*}
    \end{enumerate}
\end{lemma}

\begin{proof}
    \cref{lem:grad-bound-phase-3-simply}a is simple when \cref{lem:logit-shape-at-convergence}a holds. \cref{lem:grad-bound-phase-3-simply}b is basically redoing the calculations in \cref{lem:grad-estimate-phase-2-simply}. \cref{lem:grad-bound-phase-3-simply}c is because when \(V_{j,r,\phi'}^{(t)} \geq -B + 2\mu\), we have 
    \begin{align*}
        \E[\cE_{5,j}^{(t)}\ReLU'(\Lambda_{5,j,r})\1_{\cH_\phi'}] &= \E[(1 - \logit_{5,j}^{(t)})\ReLU'(\Lambda_{5,j,r}^{(t)})\1_{\cH_\phi'}] \\
        & \geq \lambda \ReLU'(V_{j,r,\phi'}^{(t)} + \mu \pm o(\mu)) \Pr(\cH_{\phi'})\\
        & \geq -\Theta(\lambda\varpi/n_y^2)
    \end{align*}
    And the rest is simply the property of \(\ReLU'\) at the negative boundary when some activations of \(x_0, x_1\) crosses the boundary.
\end{proof}

\subsubsection{Proof of Induction in Phase III and \cref{thm:mlp-simply-transitive}}

\begin{proof}[Proof of \cref{induction:phase-III-simply,induction:mlp-simply-transitive,induction:individual-feature-simply}]
    We prove the induction by reusing the proof of \cref{lem:grad-stationarity-simply}. 
    \begin{itemize}
        \item Proof of \cref{induction:phase-III-simply}a: The same as in the proof of \cref{lem:grad-stationarity-simply}c, with longer time. We do not use any property of \(T_{\psi}^3\) so no additional handling is needed.
        \item Proof of \cref{induction:phase-III-simply}b: When all feature is learned, that is, after \(t = T_{\psi}^2\) for the last feature \(\psi \in \Sigma^\star\), we have perfect feature shape and therefore have much better logit shape as described by \cref{lem:logit-shape-at-convergence}. We can reuse the proof of \cref{lem:grad-stationarity-simply}a-b to get the following result: for any \(\delta = \omega(\frac{\lambda\varpi}{d n_y})\), we can guarantee that \(\cC_\psi(\delta)\) doesn't hold for more than \(\tO(\frac{1}{\eta \delta})\) iterations. And beyond that, we can just reuse the argument in the proof of \cref{lem:grad-stationarity-simply}c to obtain the desired result.
        \item Proof of \cref{induction:mlp-simply-transitive}a-b: They are correct in Phase III because there is no violation.
        \item Proof of \cref{induction:individual-feature-simply}a-b: They are proven in previous phases.
        \item Proof of \cref{induction:individual-feature-simply}c: For predecessor feature \(\psi' \prec \psi \in \Sigma^\star\), we proved in this phase that they maintain feature shape until \(T_{\psi'}^3\), and improved their feature at \(T_{\psi'}^3\).
        \item Proof of \cref{induction:individual-feature-simply}d: This is proven in previous phases.
    \end{itemize}
\end{proof}

\begin{proof}[Proof of \cref{thm:mlp-simply-transitive}]
    Since all feature in \(\Sigma^\star\) have shape \(\cF_{\psi}(\delta_1, \delta_2)\) for \(\delta_1 = d^{c_1}\mu, \delta_2 = \varpi^{\frac{1}{q-1}}\) at the end of \cref{induction:phase-III-simply}, we have proven \cref{thm:mlp-simply-transitive}.
\end{proof}

\subsection{Auxiliary Technical Tools}

First we need a Bernstein inequality for U-statistics
\begin{lemma}[concentration inequality for pseudo-U-statistics]\label{lem:bernstein-u-stats}
    Let \(x_1,\dots, x_n\) be different symbols, and let \(m \ll n\) be such that \( n \equiv 0 \pmod n \). Suppose for some function \(h\) with \(|h| \leq M\) the random variables \(h(x_{i_1},x_{i_2},\dots,x_{i_m}) \) and \(h(x_{i'_1},x_{i'_2},\dots,x_{i'_m})\) are independent and identically distributed as long as \(\{x_{i_1},x_{i_2},\dots,x_{i_m}\}\cap \{x_{i'_1},x_{i'_2},\dots,x_{i'_m}\} = \emptyset \), then the pseudo-U-statistic
    \begin{align*}
        U_{m,n} = \frac{1}{\binom{n}{m}}\sum_{0\leq i_1<i_2<\cdots<i_m \leq n} h(x_{i_1},x_{i_2},\dots,x_{i_m})
    \end{align*}
    satisfies \(\Pr( |U_{n,m} - \E[U_{n,m}] | \geq t ) \leq e^{ - \frac{ nt^2}{mM^2}}\)
\end{lemma}

\begin{proof}
    The proof is the same as in \cite{Rinaldo2016Lecture27}.
\end{proof}

We present two lemmas related to the tensor power method. 

\begin{lemma}[TPM in expectation]\label{lem:tpm-in-expectation}
    Let \(p > 3\) be an integer, and \(\mu, \sigma >0\) satisfying \(\mu \gg \sigma\sqrt{\log d}\). Now suppose \(\xi \sim \cN(0,\sigma)\) is a Gaussian variable and there are two sequences \(x_t, y_t\) defined by the following update rules:
    \begin{itemize}
        \item \(x_{t+1} \geq x_t + \eta C_t \E_{\xi}[(\xi + x_t)^{p}]\) for some \(C_t =\Theta(1)\);
        \item \(y_{t+1} \leq y_t + \eta C_t \E_{\xi}[(\xi + y_t)^{p}]\) for some \(S=\Theta(1)\).
    \end{itemize}
    Then if \(x_0, y_0 \geq \mu, x_0 - y_0 \geq \varepsilon = \Omega(\frac{1}{\polylog(d)})\), for every \(t \geq 0\) such that \(x_t \in [ d^{0.01}x_0, O(1)]\), we shall have \(y_t \leq O(\polylog(d))\).
\end{lemma}

\begin{proof}
    Let \(\delta > 0\) be such that \(\delta \gg \frac{\mu}{\sigma}\) and \(\delta = o(1)\). Let \(T_i, i\geq 0\) be defined as the first time when \(x_t \geq (1+\delta)^i x_0\), and \(b := \min \{i |(1 + \delta)^i x_0 \geq A\}\). Now we can compute a growth lower bound for \(\varepsilon_t := x_t - y_t\) at each step:
    \begin{align*}
        \varepsilon_{T_{i+1}} & \geq \varepsilon_{T_i} + \sum_{t \in [T_{i}, T_{i+1})}\eta C_t \E_{\xi}[(\xi + x_t)^{p} - (\xi + y_t)^{p}] \\
        & \geq \varepsilon_{T_{i}} + \varepsilon_{T_{i}} \sum_{t \in [T_{i}, T_{i+1})} \eta C_t \E_{\xi}[(\xi + x_t)^{p-1}] \\
        & \geq \varepsilon_{T_{i}} + \varepsilon_{T_{i}} \sum_{t \in [T_{i}, T_{i+1})} \eta C_t \E_{\xi}[(\xi + x_t)^{p}] \cdot \frac{1}{(1 + \delta)x_{T_{i+1}}} \\
        & \geq \varepsilon_{T_{i}} + \varepsilon_{T_{i}} \delta (1 + \delta)^{i} x_0 \frac{1}{(1 + \delta)^{i+2}x_0 } \\
        & \geq (1 + \frac{\delta}{(1 + \delta)^2})  \varepsilon_{T_{i}}
    \end{align*}
    Therefore let \(i' \) be such that \((1 + \delta)^{i'} \geq \frac{\mu}{\varepsilon}\), we have that 
    \begin{align}
        \varepsilon_{T_{i'}} \geq (1 + \frac{\delta}{(1 + \delta)^2})^{i'} \varepsilon_0 \geq (1 - o(1)) \mu \gg \sigma\sqrt{\log d}
    \end{align}
    At \(t = T_{i'}\) we have \(y_{T_{i'}} \leq x_{T_{i'}} \leq (1 + o(1))\frac{\mu}{\varepsilon} x_0\). After \(t = T_{i'}\), we can construct a surrogate sequence \(\tilde{y}_t \geq y_t\) for \(t \geq T_{i'}\) by initializing \(\tilde{y}_{T_{i'}} = y_{T_{i'}} + \frac{\mu}{2}\), and updates by
    \begin{align*}
        \tilde{y}_{t+1} & = \tilde{y}_t + \eta C_t\tilde{y}_t^p
    \end{align*}
    This update guarantees that \(\tilde{y}_t \geq y_t\) because \(\tilde{y}_t^p \geq \E_\xi[(\xi+y_t)^p]\) at every step. Now by applying \cref{coro:TPM} to the sequence of \(x_t, \tilde{y}_t\) using the fact that \(x_{T_{i'}} \geq \tilde{y}_{T_{i'}} + \frac{\mu}{2} = \tilde{y}_{T_{i'}}(1 + \frac{1}{\polylog(d)})\), we can get the desired result for \(\tilde{y}_t\) and thus \(y_t\) at \(t = T_{b}\), i.e, when \(x_t \geq A\).
\end{proof}

\begin{lemma}[TPM, adapted from \cite{allen2020towards}]\label{lem:TPM}
    Consider an increasing sequence \(x_t \geq 0\) defined by \(x_{t+1} = x_t + \eta C_t x_{t}^{q-1}\) for some integer \(q \geq 3\) and \(C_t=\Theta(1) > 0\), 
    then we have for every $A>x_0$, for every \(\delta > 0\), and every \(\eta \in (0,1)\):
    \begin{align*}
        \sum_{t\geq 0, x_t\leq A}\eta C_t &\geq \left(\frac{\delta(1+\delta)^{-1}}{(1+\delta)^{q-2} - 1}\left(1 - \left(\frac{(1+\delta)x_0}{A}\right)^{q-2}\right) - \frac{O(\eta A^{q-1})}{x_0}\frac{\log(A/x_0)}{\log(1+\delta)}\right)\cdot\frac{1}{x_0^{q-2}} \\
        \sum_{t\geq 0, x_t\leq A}\eta C_t &\leq \left(\frac{(1+\delta)^{q-2}}{q-2} + \frac{O(\eta A^{q-1})}{x_0}\frac{\log(A/x_0)}{\log(1+\delta)}\right)\cdot\frac{1}{x_0^{q-2}} 
    \end{align*}
\end{lemma}
This lemma has a corollary:
\begin{corollary}[TPM, from \cite{allen2020towards}]\label{coro:TPM}
    Let \(q\geq 3\) be a constant and \(x_0, y_0 = o(1)\) and \(A =O(1)\). Let \(\{x_t, y_t\}_{t\geq 0}\) be two positive sequences updated as 
    \begin{itemize}
        \item \(x_{t+1} = x_t + \eta C_t x_t^{q-1}\) for some \(C_t =\Theta(1)\);
        \item \(y_{t+1} = y_t + \eta S C_t y_t^{q-1}\) for some \(S=\Theta(1)\).
    \end{itemize}
    Suppose \(x_0 \geq y_0 S^{\frac{1}{q-2 }}(1 + \frac{1}{\polylog (d)})\), letting $T_{x}$ be the first iteration s.t., $x_t\geq A$,
    then $y_{T_{x}} \leq \widetilde{O}(y_0).$ 
\end{corollary}

\section{Learning Symmetry Group Actions}\label{appendix:learning-symmetry-actions}
\paragraph{Roadmap.} We begin by introducing notation and defining fibers, followed by a description of the learning curriculum and the key appearance events. We then break the training process into clearly defined phases and prove the main results for each phase: In Phase I, features emerge and compete; In Phase II, features grow and mutually cancel, with precise tracking of when different outcomes occur; In Phase III, features converge to a desired shape.

Let us restate the action structure of \(\cG\) and \(\cY\) here.

\begin{assumption}[Assumption \ref{assump:structure-2}, restated]\label{assump:structure-2-restated}
    Let \(\lego(\cX,\cG,\cY)\) be the LEGO language, where \(|\cY| = n_Y\)\footnote{We use \(n_y\) or \(n_Y\) interchangeably in the proof below.} and \(\cG\) is a group acting on on \(\cY\). We assume the (left) group action \(\alpha: \cG\times \cY \to \cY\) is transitive but not free. In particular, we assume \(n_Y = |\cY| \in [\omega(1), O(\log d)]\) and the group \(\cG\) satisfy \(n_G = |\cG| = \Theta(\polylog (d)) > \frac{1}{\varrho}\) for all \(y\in\cY\). 
\end{assumption}

\begin{remark}
    When the action \(\alpha: \cG\times \cY\to \cY\) is transitive, there is only one orbit \(\cY = \{g\cdot y \mid g\in \cG\} \simeq \cG/G_y\) for all \(y \in \cY\). By orbit-stabilizer theorem, the stabilizer\(G_y\) at any point \(y\) has the same cardinality. 
\end{remark}

\begin{remark}
    The canonical action of a symmetry group \(S_n\) on \(\ZZ_n\) satisfy \cref{assump:structure-2-restated} with the choice \(n_y = \Theta(\frac{\log \log d}{\log\log\log d})\) that keeps \(n_y! = O(\polylog (d))\).
\end{remark}

\subsection{Preliminaries}
We work with fiber-indexed features. We first define fibers and super-combinations \(\varphi\), then we define neuron features by its index \(\psi=(j,r,\varphi)\), and finally set the curriculum and appearance events in this notation.

\begin{definition}[Fiber of values]\label{def:fiber-of-values}
    Assuming \(\cG\) follows \Cref{assump:structure-2-restated}. For each \(j\in\tau(\cY)\) and \(y \in \cY\), define the \textbf{fiber} of \(j\) at \(y\) as:
    \begin{align*}
        \fiber_{j,y} := \{g \in \cG \mid \tau(g\cdot y) = j\}
    \end{align*}
    which collects all group elements that send \(y\) to \(y' = \tau^{-1}(j)\).
\end{definition}

\begin{remark}
    We denote \(\fiber_{j,y}\) to be the set that transport \(y\) to \(y' = \tau^{-1}(j)\) because it is the pre-image of the predictor map \(\phi_y(\cdot): g \mapsto \tau(g\cdot y)\) over \(y\in\cY\). Algebraically the fibers are the left-cosets \(\{gG_y\}_{g \in \cG}\), where \(G_y := \{g\in\cG\mid g\cdot y = y\}\) is the stablizer subgroup. Moreover, we have the following fact from the orbit-stablizer theorem.
\end{remark}

\begin{fact}[cardinality of the fibers]
    For every \(j \in \tau(\cY)\) and \(y \in \cY\), \(|\fiber_{j,y}| = |\cG| / |\cY| = n_G/n_Y\) from the orbit-stablizer theorem.
\end{fact}

\begin{definition}[Super-combinations]\label{def:super-combinations-symmetry}
    Let \((j,y)\in\tau(\cY)\times\cY\), we define the \textbf{super-combinations}, which is a set of combinations:
    \[
        \varphi_{j,y}:=\fiber_{j,y}\times\{y\} \subseteq \cG\times\cY, \quad \varphi_{j,y}^1 = \fiber_{j,y}, \quad \varphi_{j,y}^2 = \{y\}
    \]
    Let \(\Phi:= \cup_{j\in\tau(\cY),\ y\in\cY}\varphi_{j,y}\) be the full set. For a fixed class \(j \in \tau(\cY)\), we denote the \textbf{correct} set to be \(\Phi_j^\star:=\cup_{y \in \cY}\varphi_{j,y}\), and the \textbf{incorrect} set \(\Phi_j^\dagger:=\cup_{j'\neq j, y \in \cY}\varphi_{j,y}\).
    A base pair \(\phi=(g,y)\) belongs to \(\varphi_{j,y}\) if and only if \(g\in\fiber_{j,y}\). To simplify the proof, we also define for each \(\varphi_{j,y} \subset \Phi\), the confounding combinations
    \begin{align*}
        \Phi_{\varphi}^\dagger = \bigcup_{j'\neq j \textbf{ XOR } y'\neq y}\varphi_{j',y'}
    \end{align*}
    That is, the set of combinations that intersect with exactly one component with \(\varphi_{j,y}\).
\end{definition}

\begin{definition}[Neuron feature indices]\label{def:psi-varphi-symmetry}
    Let us index neuron features by \(\psi=(j,r,\varphi)\) with \(\varphi\in\Phi\) and \((j,r)\in\tau(\cY)\times[m]\), and define
    \[
        \Psi^{\star}_{j,r}:=\{(j,r,\varphi):\varphi\in\Phi_j^\star\},\quad \Psi^{\dagger}_{j,r}:=\{(j,r,\varphi):\varphi\in\Phi_j^\dagger\}.
    \]
    And we further define 
    \[
        \Psi:=\tau(\cY)\times[m]\times\Phi,\quad \Psi^\star = \bigcup_{(j,r)\in\tau(\cY)\times[m]}\Psi^{\star}_{j,r},\quad \Psi^\dagger = \bigcup_{(j,r)\in\tau(\cY)\times[m]}\Psi^{\dagger}_{j,r}
    \]
    The set \(\Psi\) contains all the neuron indices we care about, and \(\Psi^\star\) and \(\Psi^\dagger\) contain the correct and incorrect indices for class \(j\).
\end{definition}

\paragraph{\(V\)-Notations} We call a pair \(\phi=(g,y)\) with \(g\in\cG\) and \(y\in\cY\) a \emph{base pair}, or following the naming of \cref{appendix:learning-simply-transitive-actions}, a \emph{base combination}. For neuron \((j,r)\), we define
\[
    V_{j,r}(g):=\vbrack{\Wb_{5,j,r,2},e_g},\quad V_{j,r}(y):=\vbrack{\Wb_{5,j,r,5},e_y},\quad V_{j,r}(\phi):=\tfrac{1}{2}\big(V_{j,r}(g)+V_{j,r}(y)\big).
\]
For a fixed value \(y\) and class \(j\), the fiber \(\fiber_{j,y}\) collects all permutations sending \(y\) to \(\tau^{-1}(j)\). For \(\psi=(j,r,\varphi_{j,y})\), we define:
\[
    V_{\psi}:=\frac{1}{n_G/n_Y}\sum_{g\in\fiber_{j,y}} V_{j,r}(g,y),\quad \overline{V}_{\psi}:=\max_{g\in\fiber_{j,y}} V_{j,r}(g,y),\quad \underline{V}_{\psi}:=\min_{g\in\fiber_{j,y}} V_{j,r}(g,y).
\]
For compactness of notation, let us also write \(V_{j,r}(x) = \sum_{p=3,4}\vbrack{\Wb_{5,j,r,1}, e_{x_0}} + \vbrack{\Wb_{5,j,r,1}, e_{x_1}}\) which is a random variable depending on the randomness of \(x_0, x_1\).

The definition of learning curriculum for the symmetry group actions is essentially the same with the simply transitive case, with a slight difference of using the fiber-combination \(\varphi\) instead of the base combination \(\phi\). We repeat here for reference.

\begin{definition}[Learning curriculum, symmetry group actions]\label{def:learning-curriculum-symmetry}
    We generate an feature index set \(\Sigma^\star\) over fiber-indexed features. Letting \(\Sigma_0 = \Psi\), at each \(i\in[n_y^2]\), we choose
    \[
        \psi_i = \arg\max_{\psi=(j,r,\varphi)\in\Sigma_i,\ \varphi \in \Phi_j^\star} V^{(0)}_{\psi}.
    \]
    Write \(\psi_i=(j,r,\varphi)\), we define the next iteration set \(\Sigma_{i+1}\) by excluding:
    \begin{enumerate}[(1)]
        \item The confusing combinations for \(\varphi\) in the same neuron: 
        \[
            \Sigma_{\psi_i}^{\dagger,1} \equiv \Sigma_i^{\dagger,1}=\{(j,r,\varphi')\in\Sigma_i, \varphi' \in \Phi_\varphi^\dagger \};
        \]
        \item All other combinations in the same neuron \((j,r)\): 
        \[
            \Sigma_{\psi_i}^{\dagger,2} \equiv \Sigma_i^{\dagger,2}=\{(j,r,\varphi') \mid \varphi' \in \Phi \setminus (\Phi_{\varphi}^\dagger\cup\{\varphi\})\};
        \]
        \item The same combination \(\varphi_{j,y}\) in other neurons of the same class \(j,r', r'\neq r\): 
        \[
            \Sigma_{\psi_i}^{\dagger,3} \equiv \Sigma_i^{\dagger,3}=\{(j,r',\varphi_{j,y})\in\Sigma_i: r'\ne r\}.
        \]
    \end{enumerate}
    The next iteration is given by 
    \[\Sigma_{i+1}=\Sigma_i\setminus(\Sigma_{\psi_i}^{\dagger,1}\cup\Sigma_{\psi_i}^{\dagger,2}\cup\Sigma_{\psi_i}^{\dagger,3}\cup\{\psi_i\})
    \]
    Eventually we obtain 
    \(\Sigma^{\star}=(\psi_1,\ldots,\psi_{n_y^2})\) and \(\Sigma^{\dagger}:=\bigcup_{\psi \in \Sigma^\star}\Sigma_{\psi}^{\dagger}\). We write \(\psi\prec_{\Sigma}\psi'\) (or simply \(\psi\prec \psi'\)) if \(\psi\) precedes \(\psi'\) in \(\Sigma^\star\). More generally, we also write \(\psi \prec\psi'\) for \(\psi,\psi' \in \Psi\) if \(\psi \in \Sigma^\star\) and \(\psi' \in \cup_{\psi'' \succ \psi}\Sigma^\dagger_{\psi''}\) is from the .
\end{definition}

\paragraph{Events of combination appearance.}
These events specify whether a specific base combination \(\phi=(g,y)\) appears, or only one of its component matches.
\begin{definition}[Events of appearance]\label{def:appearance-events-symmetry}
    For \(\phi=(g,y) \in \varphi_{j,y}\) for some \(j \in \tau(\cY)\) and \(y \in \cY\), we define
    \[
        \cH_\phi := \{g_1=g,\ y_0=y\},\quad \cH_g := \{g_1=g\},\quad \cH_y := \{y_0=y\}, \quad
    \]
    For the mismatched events \(\cH^\dagger\), the definition is a little different from the simply transitive actions: since \(\varphi_{j,y} = \fiber_{j,y}\times \{y\}\) and \(g \in \fiber_{j,y}\), we write
    \[
        \cH^{\dagger}_{\varphi}(g):=\{g_1 = g, y_0 \neq y\},\quad \cH^{\dagger}_{\varphi}(y):=\{g_1 \notin \fiber_{j,y},\ y_0=y\},
    \]
    When considering the event of all the base pair \(\phi \in \varphi_{j,y}\), we also define the union events
    \[
        \cH_{\varphi}:=\bigcup_{\phi\in\varphi}\cH_\phi,\quad \cH^{\dagger}_{\varphi}:=\cH^{\dagger}_{\varphi}(y) \cup \Big(\bigcup\nolimits_{\phi = (g,y)\in\varphi}\cH^{\dagger}_{\varphi}(g)\Big).
    \]
\end{definition}

\subsubsection{Theorem Statement}

Now we present the theorem for learning symmetric group actions.

\begin{theorem}[Learning Symmetric Group Actions]\label{thm:learning-symmetric-actions}
    Assuming \cref{assump:structure-2}, and let \(F^{(t)}\) be obtained by Algorithm~\ref{alg:cot-symmetry-training} at \(t = T_1\), then the loss for the \(5\)th-token is minimized, i.e., \(\Loss_5^{(T_1)} \leq \frac{1}{\poly(d)}\), and for each \(\psi = (j,r,\varphi) \in \Sigma^\star\), the followings hold:
    \begin{enumerate}[A.]
        \item For any \(\phi \in \varphi\), we have \(V_{j,r}^{(T_1)} (\phi) \in [B - O(d^{c_1}\mu), B + O(\mu)]\);
        \item For any \(\psi' = (j,r,\varphi') \in \Sigma^{\dagger, 1}\) and \(\phi' \in \varphi'\), we have \(|V^{(T_1)}_{j,r}(\phi')| \leq \tO(\varpi^{\frac{1}{q-1}})\). Or more concretely:
        \[
            \frac{1}{2}\Big|V_{j,r}^{(T_1)}(g') + V_{j,r}^{(T_1)}(y')\Big| \leq \tO(\varpi^{\frac{1}{q-1}}),\quad \forall (g',y') \in \varphi'
        \]
        \item For any \(\phi = (g,y) \in \varphi\) the following relation holds:
        \[
           \Big| V^{(T_1)}_{j,r}(g) - C_\alpha V_{j,r}^{(T_1)} (y) \Big| \leq \tO(\varpi^{\frac{1}{q-1}})
        \]
        where \(C_\alpha = \frac{1 + n_G(n_Y - 1)/n_Y }{(n_Y-1) + n_G/n_Y} = \Theta(n_Y)\) due to our choice of \(n_G, n_Y\) in \cref{assump:structure-2-restated}.
        \item For any \(\psi' \in \Sigma^{\dagger,3}\), it holds that \(\overline{V}_{\psi'}^{(t)} = \tO(\mu)\).
    \end{enumerate}
\end{theorem}

The assertions of \Cref{thm:learning-symmetric-actions} described the \emph{feature shapes} of the neurons, which we formally defined in the proof of learning simply transitive actions (\Cref{def:feature-combinations-simply}). The main difference between the results here and \Cref{thm:mlp-simply-transitive} is that \Cref{thm:learning-symmetric-actions}C described a different symmetry of features without the factor \(\Theta(n_y)\). 

Let us follow the proof in \Cref{appendix:learning-simply-transitive-actions} and the feature shape for symmetry group actions.

\begin{definition}[Feature shape]\label{def:feature-shape-symmetry}
    Let \(\psi = (j,r,\varphi) \in \Psi\), and let \(\delta = (\delta_1, \delta_2)\) be the error parameters. We say the feature \(\psi\) has shape \(\cF_\psi(\delta)\) if the followings are satisfied:
    \begin{enumerate}[1.]
        \item \(\cF_{\psi,1}(\delta_1)\): \(\underline V_\psi^{(t)} \geq B - O(\delta_1)\) and \(\overline{V}_\psi^{(t)} \leq B + O(\delta_1)\);
        \item \(\cF_{\psi,2}(\delta_2)\): For any \(\psi' = (j,r,\varphi') \in \Sigma_\psi^{\dagger, 1}\), it holds that \(|V_{j,r}^{(t)}(\phi')| \leq O(\delta_2)\) for all \(\phi' \in \varphi'\);
        \item \(\cF_{\psi,3}(\delta_2)\): \(|V_{j,r}^{(t)}(g) - C_{\cY} V_{j,r}^{(t)}(y)| \leq O(\delta_2)\) for any \(\phi = (g,y) \in \varphi\);
        \item \(\cF_{\psi,4}\): For any feature \(\psi' \in \Sigma_\psi^{\dagger,3}\), it holds that \(\overline{V}_{\psi'}^{(t)} \leq \tO(\mu)\).
    \end{enumerate}
\end{definition}

Indeed, we aim to prove that \(\cF_\psi(\delta)\) holds for suitable parameter \(\delta\) at the end of training, for every \(\psi \in \Sigma^\star\). And this is indeed what the theorem states.

\begin{remark}
    A major difference between \Cref{def:feature-shape-symmetry} and \Cref{def:feature-shape-simply} is that now the conditions involve not just one combination of weights, but all combinations \(\phi \in \varphi\) given a combination \(\varphi \in \Phi\). This creates some significant difficulties for the proof.
\end{remark}

\subsubsection{Facts and Basic Calculations}

We now define the gradient notation and conditions.

\begin{definition}[gradient notation]\label{def:gamma-notation-symmetry}
    Let \(\psi=(j,r,\varphi)\), we write \(\Gamma^{+}\) and \(\Gamma^{-}\) to denote the different terms in the gradient computation. For any base combination \(\phi = (g,y) \in \varphi_{j,y}\) where \(\varphi_{j,y} \in \Phi\), we write
    \begin{align*}
        &\Gamma^{+,(t)}_{j,r,\phi} := \E\big[(1 - \logit_{5,j})\ReLU'(\Lambda_{5,j,r})\1_{\cH_\phi}\big],\quad
        \Gamma^{-,(t)}_{j,r,\phi} := \E\big[\logit_{5,j}\ReLU'(\Lambda_{5,j,r})\1_{\cH_\phi}\big].
    \end{align*}
    When we consider the gradient of \(V_{j,r}(g)\) or \(V_{j,r}(y)\) for individual tokens \(g\in\cG\) and \(y \in \cY\), we use the following notations: let \(\phi = (g,y) \in \varphi_{j,y}\), then we define
    \begin{align*}
        &\Gamma^{+,(t)}_{j,r,g} := \E\big[\logit_{5,j}\ReLU'(\Lambda_{5,j,r})\1_{\cH_{\phi}}\big] \\
        &\Gamma^{-,(t)}_{j,r,g} := \E\big[\logit_{5,j}\ReLU'(\Lambda_{5,j,r})\1_{\cH^{\dagger}_{\varphi}(g)}\big] \\
        &\Gamma^{-,(t)}_{j,r,y} := \E\big[\logit_{5,j}\ReLU'(\Lambda_{5,j,r})\1_{\cH^{\dagger}_{\varphi}(y)}\big].
    \end{align*}
    Finally, we sum up over \(\phi \in \varphi\) for the gradient of combination \(\psi = (j,r,\varphi)\):
    \begin{align*}
        &\Gamma^{+,(t)}_{\psi}:= \E\big[(1 - \logit_{5,j})\ReLU'(\Lambda_{5,j,r})\1_{\cH_{\varphi}}\big],\\
        &\Gamma^{-,(t)}_{\psi}:= \E\big[\logit_{5,j}\ReLU'(\Lambda_{5,j,r})\1_{\cH^{\dagger}_{\varphi}}\big].
    \end{align*}
\end{definition}

Let \(\psi = (j,r, \varphi)\) where \(\varphi_1 = \fiber_{j,y}\) One can check that the defined \(\Gamma^{+,(t)}_{\psi}\) is connected to the base pair \(\Gamma\) by the following relation:
\begin{align}
    \Gamma^{+,(t)}_{\psi} = \sum_{g \in \fiber_{j,y}}\Gamma^{+,(t)}_{j,r,g},\quad \Gamma^{-,(t)}_{\psi} = \Gamma^{-,(t)}_{j,r,y} + \sum_{g \in \fiber_{j,y}}\Gamma^{-,(t)}_{j,r,g}
\end{align}

\begin{definition}[gradient condition]\label{def:gradient-condition-symmetry}
    At \(t\le T_1\) and \(\delta>0\), write for \(\psi=(j,r,\varphi)\)
    \begin{align*}
        &\cC_{\psi}(\delta):\ |\Gamma^{+,(t)}_{\psi} - \Gamma^{-,(t)}_{\psi}|\le \delta,\qquad
        \cC^{+}_{\psi}(\delta):\ |\Gamma^{+,(t)}_{\psi}|\le \delta,\qquad
        \cC^{-}_{\psi}(\delta):\ |\Gamma^{-,(t)}_{\psi}|\le \delta.
    \end{align*}
\end{definition}
These conditions control the magnitude of the gradient of the feature at iteration \(t\). Next we compute the expressions of the gradient using

\begin{fact}[gradient expressions]\label{fact:grad-expression-symmetry}
    Equipped with \Cref{def:gradient-condition-symmetry}, the gradient with respect to \(\Wb_{i,j,r,p}\) for \(i = 5\), \(j \in [d]\), \(r \in [m]\), \(p \in [5]\) and \(v \in \cG\cup \cY\) can be computed by
    \begin{enumerate}[(a)]
        \item when \(j\in \tau(\cY), r\in[m], p = 5\), for \(g \in \cG\), let \(\phi = (g,y) \in \varphi_{j,y}\) be the combination with \(\phi_1 = g\), we have
        \begin{align}
            \vbrack{-\nabla_{\Wb_{5,j,r,2}}\Loss, e_g} = &\frac{1}{2} (\Gamma_{j,r,\phi}^{+,(t)} - \Gamma_{j,r,g}^{-,(t)}) \label{eq:grad-computation-g-symmetry}
        \end{align}
        \item when \(j\in \tau(\cY), r\in[m], p = 5\), for \(y \in \cY\), let \(\varphi_{j,y} \in \Phi^\star_j\) be the combination, we have
        \begin{align}
            \vbrack{-\nabla_{\Wb_{5,j,r,5}}\Loss, e_y} = &\frac{1}{2} (\Gamma_{\psi}^{+,(t)} - \Gamma_{j,r,y}^{-,(t)}) \label{eq:grad-computation-y-symmetry}
        \end{align}
        \item For any combination of \(p\), \(v\) not included above, we have due to our assumptions of the distribution and update rule:
        \begin{align*}
            \vbrack{-\nabla_{\Wb_{5,j,r,p}}\Loss, e_v} \equiv 0
        \end{align*}
    \end{enumerate}
\end{fact}


\subsubsection{Induction Hypothesis and Training Phases}

We define the following intermediate time-steps:

\begin{definition}[Phase decomposition]\label{def:phase-decomposition-symmetry}
Define timestamps for \(\psi\) using the feature:
    \begin{enumerate}[(a)]
        \item Phase I: \(t\in[0, T^1_{\psi}]\), where \(T^1_{\psi}:=\min\{t\geq 0: \underline{V}^{(t)}_{\psi}\geq d^{c_1}\mu\}\).
        \item Phase II.1: \(t\in(T^1_{\psi}, T^{2,1}_{\psi}]\), where \(T^{2,1}_{\psi}:=\min\{t\geq 0: V^{(t)}_{\psi}\geq \tfrac{1}{2}\log d\}\).
        \item Phase II.2: \(t\in(T^{2,1}_{\psi}, T^{2,2}_{\psi}]\), where \(T^{2,2}_{\psi}:=\min\{t\geq 0: \E[\cE_{5,j}^{(t)}\1_{\cH_{\varphi}}]\le d^{-1/2}\}\).
        \item Phase II.3: \(t\in(T^{2,2}_{\psi}, T^{2}_{\psi}]\), where \(T^{2}_{\psi} = \min\{T^{2,3}_{\psi}, T_{2,1} + O(\frac{1}{\eta (d^{c_1}\mu)^{q-2}})\}\), and \(T^{2,3}_{\psi}\) is defined by
        \[
            T^{2,3}_{\psi}:=\min\{t\geq 0: \cF_\psi(\delta_1,\delta_2) \text{ holds, where } \delta_1 = d^{c_1}\mu, \delta_2 = \tO((d\varpi/\lambda)^{\frac{1}{q-1}})\}
        \]
        \item Phase III.1: \(t\in(T^{2}_{\psi}, T_{1,1}]\), where \(T_{1,1} = T_{\psi_{n_y^2}}^2\) and \(\psi_{n_y^2}\) is the last feature in \(\Sigma^\star\).
        \item Phase III.2: \(t \in (T_{1,1}, T_1]\), the end of the training where the feature shape has stabilized
    \end{enumerate}
\end{definition}

We give several induction hypotheses, each characterizing different aspects of the process.

\begin{induction}[learning symmetric group actions]\label{induction:symmetric-group-actions}
    Let \(t \leq T_1\). The following holds:
    \begin{enumerate}[(a)]
        \item \(\underline{V}_{\psi}^{(t)} + b_{i,j,r} \geq \Omega(\mu)\) for all \(\psi\in\Sigma^{\star}\).
        \item \(T_{\psi}^2 < T_{\psi'}^1\) if \(\psi\prec\psi' \in \Sigma^\star\). Thus the intervals \(\{[T^{1}_{\psi}, T^{2}_{\psi}]\}_{\psi \in \Sigma^\star}\) are non-overlapping.
        \item Let \(\psi' \prec \psi \in \Sigma^\star\), then the feature shape \(\cF_{\psi'}(d^{c_1}\sigma_0, \varpi^{\frac{1}{q-1}})\) holds throughout \(t \in [T_{\psi}^1, T_1]\);
        \item For any \(\psi = (j,r,\varphi) \in \Sigma^\star\) There are at most \( \tO(\sqrt{d}/\eta)\) iterations where \(\logit^{(t)} \geq \frac{1}{\sqrt{d}}\) conditioned on \(\cH_\varphi\) in \(t \in [0, T_{\psi}^1]\).
    \end{enumerate}
\end{induction}

\paragraph{Interpretations of \cref{induction:symmetric-group-actions}.} We explain what the inductions hypothesis mean here:
\begin{enumerate}[(a)]
    \item This property simply asserts that the \emph{good} features never shrink below a certain level which is crucial for the optimization to work.
    \item This property describes the \emph{curriculum} nature of \(\Sigma^\star\): each feature \(V_\psi\) grows in different time intervals within \(0\leq t \leq T_1\). 
    \item This property asserts that the feature \(\psi'\) previous to \(\psi\) in the curriculum \(\Sigma^\star\) has stable and good feature shapes, which helps with learning the current feature.
    \item This property means there are very few iteration where the gradient of \(V_{j,r}(\phi), \phi \in \varphi\) is disturbed before \(T_{j,r,\varphi}^1\) is reached.
\end{enumerate}



\subsection{Phase I: Emergence of the Feature}

We first prove some properties at the beginning of phase I.

\begin{lemma}[Initialization range for features]\label{lem:init-range-feature-symmetry}
    Under random initialization and \cref{assump:structure-2-restated}, with probability at least \(1- o(1)\), the following holds
    \begin{enumerate}[(a)]
        \item For any \(v \in \cV\), we have \(|\vbrack{\Wb_{5,j,r,p}^{(0)}, e_x} | \leq O(\mu/\sqrt{\log d})\) for any \(j,r, p \).
        \item Let \(\psi \in \Sigma^\star, \psi' \in \Psi\) and \(\psi \prec \psi'\), then \(V_\psi^{(0)} \geq V_{\psi'}^{(0)} + \Delta_0\) for a gap \(\Delta_0=\Omega(\sigma_0/\polylog d)\).
        \item We have \(\Lambda_{5,j,r}^{(0)}(\Zb) = \Omega(\mu)\) for all \(j\in[d],r\in[m]\).
    \end{enumerate}
\end{lemma}
\begin{proof}
    The proof is basically the same as that of \cref{fact:range-parameter-simply}. The first one uses basic concentration properties of Gaussian distribution and the second uses a union bound with the anti-concentration property of the Gaussian distribution. (c) is based on the fact that 
    \[
        \Lambda_{5,j,r}^{(0)}(\Zb) = b_{5,j,r} + \frac{1}{2}\sum_{\kk\in\cI^{1,0}}\sum_{p\in[5]}\vbrack{\Wb_{5,j,r,p}^{(0)}, \Zb_{\kk,p}} = \mu \pm O(\mu/\sqrt{\log d})
    \]
    which concludes the proof.
\end{proof}

\begin{corollary}\label{coro:grad-approx-symmetry}
    For any \(\psi = (j,r,\varphi) \in \Sigma^\star\) and \(y \in \varphi^2\), we have with probability \(\geq 1 - o(1)\):
    \begin{align*}
        &\left|\frac{1}{|\varphi|}\sum_{g\in\varphi^1}\E_x[\ReLU'(\Lambda_{5,j,r}^{(0)})\mid \cH_{g,y}] - \E_x[\ReLU'(V_{j,r}^{(0)}(y) + V_{j,r}(x) + b_{5,j,r})]\right| \\
        &\leq  O(\frac{1}{\sqrt{n_G/n_Y}}) \E_x[\ReLU'(V_{j,r}^{(0)}(y) + V_{j,r}(x) + b_{5,j,r})]
    \end{align*}
\end{corollary}

\begin{proof}
    This is true because of \cref{lem:init-range-feature-symmetry}a and Hoeffding's inequality.
\end{proof}

Now we present an induction for phase I.

\begin{induction}\label{induction:phase-I-symmetry}
    Let \(\psi = (j,r,\varphi)\in \Sigma^\star\), for \(t \leq T_{\psi}^1\), the gradients satisfy \(\Gamma_{j,r,y}^{+,(t)} \geq \Omega(\frac{n_G}{n_Y})\Gamma_{i,j,g}^{+,(t)}\) for any \((g,y)\in \varphi\).
\end{induction}

Using \Cref{fact:grad-computation-simply} with a representative \(\phi=(g,y)\in\varphi_{j,y}\) and the definition \(V_{j,r}(\phi)=\tfrac12(V_{j,r}(g)+V_{j,r}(y))\), a single gradient step with step size \(\eta\) yields
\begin{align}
    V_{j,r}^{(t+1)}(\phi) - V_{j,r}^{(t)}(\phi)
    &= \frac{\eta}{2}\Big( \vbrack{-\nabla_{\Wb_{5,j,r,5}}\Loss, e_y} + \vbrack{-\nabla_{\Wb_{5,j,r,2}}\Loss, e_g} \Big)\nonumber\\
    &= \frac{\eta}{2}\Big( \Gamma_{\psi}^{+,(t)} + \Gamma^{+,(t)}_{j,r,\phi} - \Gamma^{-,(t)}_{j,r,g} - \Gamma^{-,(t)}_{j,r,y}\Big)\label{eq:phase-I-one-step}
\end{align}

Similar to the proof of the simply transitive actions, we define a proxy term and argue that the trajectory of \(V_\psi\) is similar to the proxy.

\begin{lemma}[Approximating the gradient with proxy]\label{lem:grad-proxy-symmetry}
Assuming \cref{induction:symmetric-group-actions,induction:phase-I-symmetry} holds at \(t\). Let \(\psi = (j,r,\varphi) \in \Psi\), for every \(\phi = (g,y) \in \varphi\), we define two proxy sequences \(\wt V^{(t)}_{j,r}(v), v \in \{g,y\}\) by the update rule:
\[
    \wt V^{(t+1)}_{j,r}(y) = \wt V^{(t)}_{j,r}(y)+\tfrac{\eta}{2}\wt\Gamma^{+,(t)}_{\psi},\quad \wt V^{(t+1)}_{j,r}(g) = \wt V^{(t)}_{j,r}(g)+\tfrac{\eta}{2}\wt\Gamma^{+,(t)}_{j,r,g};
\]
where both \(\wt V_{j,r}(g)\) and \(\wt V_{j,r}(y)\) start from \(\wt V^{(0)}_{j,r}(y):=V^{(0)}_{j,r}(y),\ \widetilde V^{(0)}_{j,r}(g):=V^{(0)}_{j,r}(g)\). And the \(\wt\Gamma\) are defined by replacing the features of \(V_{j,r}(v), v\in\varphi^1\cup\{y\}\) with \(\wt V_{j,r}(v)\). Now for all \(t \leq T_{\psi}^1\), we have
\[
    |\wt V^{(t)}_{j,r}(\phi) - V^{(t)}_{j,r}(\phi)| \leq \tO(\mu/d^{\frac{1}{2}-3c_1}), \quad \text{ where } \wt V^{(t)}_{j,r}(\phi) = \wt V_{j,r}(g) + \wt V_{j,r}(y).
\]
\end{lemma}

\begin{proof}
    Fix \(\psi=(j,r,\varphi_{j,y})\). \eqref{eq:phase-I-one-step} gives
    \[ 
        | \wt V^{(t)}_{j,r}(y)-V^{(t)}_{j,r}(y)| = \tfrac{\eta}{2}\Gamma^{-,(t)}_{j,r,y},\qquad | \wt V^{(t)}_{j,r}(g)-V^{(t)}_{j,r}(g)| = \tfrac{\eta}{2}\Gamma^{-,(t)}_{j,r,g}
    \]
    Since \cref{induction:phase-I-symmetry} holds at \(t\), we have that 
    \(\wt V^{(t+1)}_{j,r}(y) - \wt V^{(t)}_{j,r}(y) \geq \Omega(\frac{n_G}{n_Y}) (\wt V^{(t+1 )}_{j,r}(g) -\wt V^{(t)}_{j,r}(g))\).
    Moreover, since \(\sum_{g\in \varphi^1}\Gamma^{+,(t)}_{j,r,g} = \Gamma^{+,(t)}_{j,r,y}\) by \cref{def:gamma-notation-symmetry}. This, combined with \cref{induction:symmetric-group-actions}, implies that 
    \[
        \wt V^{(t+1)}_{j,r}(y) - \wt V^{(t)}_{j,r}(y) \geq \sum_{g \in \varphi^1}\Gamma^{+,(t)}_{j,r,g} \geq \Omega(\frac{n_G}{n_Y})(\wt V^{(t+1)}_{j,r}(g) - \wt V^{(t)}_{j,r}(g)), \quad \forall g\in\varphi^1
    \]
    The above update bound imply, for all \(t\leq T_{\psi}^1\), the following property hold: \(\wt V_{j,r}^{(t)}(g) - \wt V_{j,r}^{(0)}(g)\leq O(\frac{1}{n_G/n_Y})(\wt V_{j,r}^{(t)}(y)- \wt V_{j,r}^{(0)}(y))\) for all \(g \in \varphi^1\). Now, by the same induction procedure in the proof of \Cref{lem:grad-proxy-simply}, we know that the sum
    \begin{align*}
        \left| V^{(t)}_{j,r}(y) +\sum_{g\in\varphi^1}V_{j,r}^{(t)}(g) -  \wt V^{(t)}_{j,r}(y) +\sum_{g\in\varphi^1}\wt V_{j,r}^{(t)}(g)\right| \leq O(\mu/d^{\frac{1}{2}-3c_1})
    \end{align*}
    And then by the above relations we know 
    \[
        |\wt V^{(t)}_{j,r}(y) - V^{(t)}_{j,r}(y)|\leq O(\mu/d^{\frac{1}{2}-3c_1})
    \]
    Thus the conclusion holds.
\end{proof}

\begin{lemma}[competition]\label{lem:competition-symmetry}
    Let \(\psi \in \Sigma^\star\) Assuming \cref{induction:symmetric-group-actions,induction:phase-I-symmetry} holds at \(t \leq T_{\psi}^1\). Then at \(t = T_{\psi}^1\), we have 
    for any feature \(\psi'\in \Psi\), \(\psi' \succ \psi\), \(\overline{V}_{\psi'}^{(t)} \leq \tO(\mu)\).
\end{lemma}

\begin{proof}
    From \cref{lem:grad-proxy-symmetry}, we know that 
    \[
        \wt V^{(t+1)}_{j,r}(y) = \wt V^{(t)}_{j,r}(y) \geq \frac{\eta}{2}\sum_{g \in \varphi^1}\E_x[(\wt V^{(t)}_{j,r}(y) + \wt V^{(t)}_{j,r}(g) + b_{5,j,r} + V_{j,r}(x))^{q-1}\1_{\cH_{g,y}}]
    \]
    Suppose there is another \(\psi' = (\varphi', y') \in \Sigma^\star\) such that \(\psi'\succ \psi\), then we can also write a comparison based on \cref{lem:init-range-feature-symmetry}:
    \begin{align*}
        \wt V^{(0)}_{j,r}(y) \geq \wt V^{(0)}_{j,r}(y') + \Delta_0, 
    \end{align*}
    And therefore one can compute by absolutely bounding the average over \(g \in \varphi\) using \cref{coro:grad-approx-symmetry}:
    \begin{align*}
        \Gamma_{\psi}^{+,(0)} = \ & \frac{1}{n_Y|\varphi^1|}\sum_{g \in \varphi^1}\E_x[(\wt V^{(0)}_{j,r}(y) + \wt V^{(0)}_{j,r}(g) + b_{5,j,r} + V_{j,r}(x))^{q-1}] \\
        \geq \ & \frac{1}{n_Y}(1 - \frac{1}{\sqrt{n_G/n_Y}})\E_x[(\wt V^{(0)}_{j,r}(y) + b_{5,j,r} + V_{j,r}(x))^{q-1}]
    \end{align*}
    Similarly, for \(\psi' = (j',r',\varphi')\) and \(\phi'\in\varphi'\), we have
    \begin{align*}
        \Gamma_{\psi'}^{+,(0)} = \ & \frac{1}{|\varphi^{\prime,1}|}\sum_{g \in \varphi^{\prime,1}}\E_x[(\wt V^{(0)}_{j',r'}(y') + \wt V^{(0)}_{j',r'}(g) + b_{5,j',r'} + V_{j',r'}(x))^{q-1} \mid \cH_{g,y'}] \\
        \leq \ & (1 + \frac{1}{\sqrt{n_G/n_Y}}) \E_x[(\wt V^{(0)}_{j',r'}(y') + b_{5,j',r'} + V_{j',r'}(x))^{q-1}]
    \end{align*}
    Now we can compare the proxies \(\wt{V}_{j,r}(y) + b_{5,j,r}\) with \(\wt{V}_{j,r}(y')\) using \cref{lem:tpm-in-expectation} to get that when \(V_{j,r}(y) \geq d^{c_1}\mu\) we have \(V_{j',r'}(y')\leq \tO(\mu)\). Now due to \cref{lem:grad-proxy-symmetry} and \cref{induction:phase-I-symmetry}, we have the desired result.
\end{proof}

\begin{proof}[Proof of \Cref{induction:symmetric-group-actions,induction:phase-I-symmetry} in Phase I]
Fix \(\psi=(j,r,\varphi)\in\Sigma^{\star}\). We check items (a)-(d) in \Cref{induction:symmetric-group-actions}:
\begin{itemize}[leftmargin=2em]
    \item \cref{induction:symmetric-group-actions}a: This is simple by combining \cref{lem:grad-proxy-symmetry,lem:init-range-feature-symmetry}, where each \(V_{j,r}(\phi)\) is well approximated by a monotonically growing proxy \(\wt V_{j,r}(\phi)\).
    \item \cref{induction:symmetric-group-actions}b: This is proved later but here \cref{lem:competition-symmetry} showed that for any pair \(\psi\prec \psi' \in \Sigma^\star\), that \(T_{\psi}^1 < T_{\psi'}^1\), which is a precondition for \cref{induction:symmetric-group-actions}b.
    \item \cref{induction:symmetric-group-actions}c-d: These two are proved later in future phases and are not violated in phase I.
    \item \cref{induction:phase-I-symmetry}: We shall show that this is the case. Consider any feature \(g \in \varphi^1\), its gradient is clearly bounded by \(O(n_Y/n_G)\Gamma_{\psi}^{+,(t)}\) initially before the feature \(V_{j,r}(g) \geq \Omega(\mu)\). In this range the difference of its growth with the update of \(V_{j,r}(y)\) averaged over all \(g \in \varphi^1\) is bounded by \(O(n_Y/n_G)(\wt V_{j,r}^{(t)}(y) + b_{5,j,r})^{q-1}\) because it grows slower than \(O(n_Y/n_G) \wt V_{j,r}(y)\) and thus Hoeffding's inequality with variance parameter \(\mu\) would work to bound the difference between the averaged gradient and the gradient with \(V_{j,r}(g)\) in expectation, and thus we obtain a bound \(\Gamma_{\psi}^{+,(t)}\geq \Omega(n_G/n_Y)\Gamma_{j,r,g}^{+.(t)}\). When \(V_{j,r}(g) \geq \Omega(\mu)\), the feature \(V_{j,r}(y)\) has grown more than \(\Omega(n_G/n_Y)\mu\) and the growth ratio still holds. This proved the induction for \(t\leq T_{\psi}^1\).
\end{itemize}
\end{proof}

\subsection{Phase II: Feature Growth and Cancellations}

Below we present an additional induction hypothesis for phase II.
\begin{induction}[Phase II]\label{induction:phase-II-symmetry}
    For all features \(\psi=(j,r,\varphi) \in \Sigma^\star\) and \(t \in [T^1_{\psi}, T^{2}_{\psi}]\), then for any \(g'\notin\varphi^1, y'\notin \varphi^2\), we have \(V_{j,r}^{(t)}(g'), V_{j,r}^{(t)}(y') \in (-B - \tO(\mu), \tO(\mu))\).
\end{induction}

\subsubsection{Technical Lemmas}

\begin{lemma}[gradient estimation]\label{lem:grad-estimate-phase-2-symmetry}
    Assume \cref{induction:symmetric-group-actions,induction:phase-II-symmetry}. Let \(\psi=(j,r,\varphi_{j,y}) \in \Sigma^\star\), let \(\phi = (g,y)\in\varphi_{j,y}\), then at \(t \in [T^{1}_{\psi}, T^{2}_{\psi}]\), we have
    \begin{enumerate}[(a)]
        \item For \(t \in [T^{1}_{\psi}, T^{2,1}_{\psi}]\), we have \(\Gamma_{\psi}^{(t)} \geq \Omega(\Pr(\cH_\varphi))\ReLU'(\underline{V}_\psi^{(t)})\);
        \item If \(V_{j,r}^{(t)}(\phi) \in (\frac{1}{3}\log d,\frac{2}{3}\log d)\), then \(\Gamma_{j,r,\phi}^{(t)} = \Gamma_{j,r,\phi}^{+,(t)} - \Gamma_{j,r,\phi}^{-,(t)}  \geq \Omega(\Pr(\cH_\phi))\);
        \item For any \(g'\notin \varphi^1\) and any \(y'\neq y\),
        \begin{align*}
            \Gamma_{j,r,g'}^{(t)}  &= - \E\big[\logit_{5,j}^{(t)}\ReLU'(\Lambda_{5,j,r}^{(t)})\1_{\cH_{(g',y)}}\big] \pm O(\varpi n_Y/n_G) , \\
            \Gamma_{j,r,y'}^{(t)}  &= - \sum_{g\in\varphi^1} \E\big[\logit_{5,j}^{(t)}\ReLU'(\Lambda_{5,j,r}^{(t)})\1_{\cH_{(g,y')}}\big]  \pm O(\varpi).
        \end{align*}
        \item Any individual component \(V_{j,r}^{(t)}(v), v \in \varphi^1\cup\varphi^2\) cannot fall below \(O(\mu)\), and \(\Gamma_{j,r,v}^{+,(t)}\geq 0\).
    \end{enumerate}
\end{lemma}

\begin{proof}
    Write \(\psi=(j,r,\varphi_{j,y})\).
    \begin{enumerate}[(a)]
        \item By the definition of \( T_{\psi}^{2,1}\), we know there exist \(\phi \in \varphi\) such that \(V_{j,r}(\phi) \leq \tfrac{1}{2}\log d\). By the same arguments in \cref{lem:grad-proxy-symmetry}, we can get that all \(\phi\in\varphi\) satisfy \(\tfrac{1}{2}\log d + o(1)\). Given that \cref{induction:symmetric-group-actions} holds at \(t\), conditioned on \(\cH_\phi\), all other neuron \((j',r')\neq (j,r)\) has activation smaller than \(\varrho\). Therefore \(F_{5,j'} \leq o(1)\), and \(\logit_{5,j'} = o(1)\) for any \(j'\neq j\) conditioned on \(\cH_\phi\). So conditioned on \(\Zb \in \cH_\phi\), when \(V_{j,r}^{(t)}(\phi) \leq \frac{2}{3}\log d\), we have 
        \begin{align*}
            \logit_{5,j} = \frac{e^{F_{5,j}(\Zb)}}{e^{F_{5,j}(\Zb)} + (1 + o(1))(d - 1)} \leq O(d^{-\Omega(1)})
        \end{align*}
        Now we can simply bound all \(\Gamma_{j,r,v}^{-,(t)}\) for \(v \in \varphi^1\cup\varphi^2\) to be smaller than \(\tO(\frac{1}{\sqrt{d}})\Gamma_{j,r,v}^{+,(t)}\). Thus the proof is obtained by computing \(\Gamma_{j,r,v}^{+,(t)}\) and sum to get \(\Gamma_{\psi}^{+,(t)}\).
        \item Similar to (a), here we also notice that \(\ReLU'(\Lambda_{5,j,r}(\Zb)) = 1 \) because \(\Lambda_{5,j,r}(\Zb) \in [\varrho,B]\) is in the linear regime. Now we can see that the terms \(\Gamma_{j,r,\phi}^{+,(t)} \geq \Omega(\Pr(\cH_\phi))\) and \(\Gamma_{j,r,\phi}^{-,(t)}\leq O(d^{-\Omega(1)})\) and thus we have the result.
        \item For \(g'\notin \fiber_{j,y}\), let \(y'\in\cY\) such that \(g' \in \fiber_{j,y'}\), then we have
        \begin{align*}
            \Gamma^{(t)}_{j,r,g'}&=\E[(1-\logit_{5,j})\ReLU'(\Lambda_{5,j})\1_{\cH_{(g',y')}}]-\sum_{y'\neq y',y}\E[\logit_{5,j} \ReLU'(\Lambda_{5,j})\1_{\cH_{(g',y')}}] \\
            &\quad  -\E[\logit_{5,j} \ReLU'(\Lambda_{5,j})\1_{\cH_{(g',y)}}]
        \end{align*}
        Since \(\cH_{(g',y)}\) is confusing, the negative term dominates; other appearances contribute at most \(O(\varpi n_Y/n_G) \)or \(O(\mu^{q-1})\). Similar arguments hold for \(y'\).
        \item This one is proven similarly to (a) but can adapt to any level of \(V_{j,r}(\phi)\) for the corresponding \(\phi \in \varphi\) for \(v \in \varphi^1\cup\varphi^2\). From \cref{induction:phase-II-symmetry}c we know that all other feature \(V_{j,r}(v') \leq \tO(\mu)\) for \(v' \notin \varphi^1\cup\varphi^2\). So for any \(v \in \varphi^1\cup\varphi^2\) such that \(V_{j,r}(v) \leq \tO(\mu)\), we have that the combination \(\phi = (v, v') \in \Phi \setminus \varphi\) has \(V_{j,r}(\phi) < O(\mu)\) and thus have gradient
        \begin{align*}
            \Gamma_{j,r,v}^{(t)} = \Gamma_{j,r,v}^{+,(t)} - \Gamma_{j,r,v}^{-,(t)} \geq \Gamma_{j,r,v}^{+,(t)} - O(\mu^{q-1})
        \end{align*}
        Note that \(\Gamma_{j,r,v}^{+,(t)}\) is either \(\geq \lambda /d\) or \(0\) at this phase. So we can sum up the negative gradient \(-O(\mu^{q-1})\) for all iterations \(t \in [T_{\psi}^1, T_{\psi}^2]\) (which is fewer than \(\tO(\frac{1}{\eta (d^{c_1}\mu)^{q-2}})\)) to get that the feature \(V_{j,r}(v) \gg \mu\). This means \(\Lambda_{5,j,r}\geq 0\) when \(\cH_\phi\) happens and thus \(\Gamma_{j,r,v}^{+,(t)}\geq 0\) for all \(t \in [T_{\psi}^1,T_{\psi}^2]\).
    \end{enumerate}
\end{proof}

\begin{lemma}[feature magnitude]\label{lem:feature-magnitude-phase-2-symmetry}
    Assume \cref{induction:symmetric-group-actions,induction:phase-II-symmetry} holds. For any feature \(\psi=(j,r,\varphi) \in \Sigma^\star\), we have
    \begin{enumerate}[(a)]
        \item when \(\cC^{+}_{\psi}(\delta)\) holds for some \(\delta \in [n_G\varpi, O(\lambda/d^{c_1})]\) at \(t \ge T^{2,1}_{\psi}\), we have \(V_{j,r}^{(t)}(\phi) \ge B - O(d^{c_1}\mu), \forall \phi \in \varphi\).
        \item when \(\cC^{-}_{\psi}(\delta)\) holds for some \(\delta \in [n_G\varpi,  O(\lambda/d^{1 + c_1})]\) at \(t \geq T^{2,1}_{\psi}\), we have \(V_{j,r}^{(t)}(\phi) \in [-B-O(\mu), \tO((\frac{d\delta}{\lambda})^{\frac{1}{q-1}})]\) for all \(\phi \in \Phi^\dagger_\varphi\).
    \end{enumerate}
\end{lemma}
\begin{proof}
    \begin{enumerate}[(a)]
        \item If \(\cC^{+}_{\psi}(\delta)\) holds with \(\delta\le O(\lambda/d^{c_1})\), then there are two probability: \(V_{j,r}^{(t)}(\phi)\) for all \(\phi \in \varphi\) is as small as \(O(\delta^{\frac{1}{q-1}})\) which is not possible for \(\delta\) so small because of \cref{lem:grad-estimate-phase-2-symmetry}a; or that \(V_{j,r}^{(t)}(\phi) \ge B - O(d^{c_1}\mu)\) such that \(F_{5,j}\geq B\) and thus the gradient has vanished for some events.
        \item The logic from (a) can be applied here as well. Since the positive gradient is at most \(O(\varpi)\) by \cref{lem:grad-estimate-phase-2-symmetry}b, thus whenever \(\delta \gg \varpi\), the negative gradient conditioned on \(\cH_\phi, \phi \in \Phi^\dagger_\varphi\) is bounded, which gives
        \begin{align*}
            \sum_{\phi \in \varphi}\E[\logit_{5,j}\ReLU'(\Lambda_{5,j,r})\1_{\cH_\phi}] \leq O(\delta) & \implies (\Lambda_{5,j,r})^{q-1} \1_{\cH_\phi}\leq \delta\lambda /(d\Pr(\cH_\phi)) \tag{because \cref{fact:logit-lower-bound}} \\
            & \implies V_{j,r}^{(t)}(\phi) \lesssim (\frac{n_Gn_Y\lambda \delta}{d})^{\frac{1}{q-1}}, \forall \phi \in \varphi
        \end{align*}
        which proves the desired result.
    \end{enumerate}
\end{proof}

\begin{lemma}[gradient stationarity]\label{lem:grad-stationarity-symmetry}
    Let \(\psi=(j,r,\varphi)\in\Sigma^\star\), then for any \(\delta > \varpi^{q-2}\), the followings hold:
    \begin{itemize}
        \item[(a)] Define \(\cB^{+}_{\delta} :=\{t \in [T^{2,1}_{\psi}, T^{2,3}_{\psi}] \mid \cC^{+}_{\psi}(\delta) \text{ doesn't hold}\}\) then \( |\cB^{+}_{\delta}|\le O(\frac{\log^2 d}{\eta\delta})\);
        \item[(b)] Define \(\cB^{+}_{\delta} :=\{t \in [T^{2,1}_{\psi}, T^{2,3}_{\psi}] \mid \cC^{-}_{\psi}(\delta) \text{ doesn't hold}\} \) then \( |B^-_\delta| \le O(\frac{\log^3 d}{\eta\delta})\);
        \item[(c)] For some \(t = T_{\psi}^{2,3} \geq T^{2,1}_{\psi} + O(\frac{1}{\eta n_G \varpi})\), we have  \(V_{j,r}^{(t)}(\phi) \geq B - O(d^{c_1}\mu)\) for all \(\phi \in \varphi\) and \(|V_{j,r}^{(t)}(\phi')| \le \tO((\frac{d\varpi}{\lambda})^{\frac{1}{q-1}})\) for all \(\phi' \in \Phi^\dagger_\varphi\), moreover, this will hold until \(t = T_{\psi}^2\)
    \end{itemize}
\end{lemma}
\begin{proof}
    The proof is very similar in spirit to the proof of \Cref{lem:grad-stationarity-simply}. We first define the following quantity:
    \begin{align*}
        \Upsilon_\psi^{(t)} = V_{j,r}^{(t)} (y) - \sum_{g\notin \varphi^1}  V_{j,r}^{(t)} (g)
    \end{align*}
    One can compute that the update of \(\Upsilon_\psi^{(t)} \) is simply 
    \begin{align*}
        \eta \bigg(\Gamma_{\psi}^{+,(t)} - \sum_{\varphi'\neq \varphi \in \Phi^\star_j}\sum_{\phi\in\varphi'}\Gamma_{j,r,\phi}^{+,(t)} + \sum_{\phi'\notin \varphi\cup\Phi_\varphi^\dagger } \Gamma_{j,r,\phi'}^{-,(t)} \bigg)
    \end{align*}
    The middle term is bounded by \(O(\varpi)\) and the last term is bounded below by \(-O(\varpi)\), thus we have that the update satisfy 
    \begin{align*}
        \Upsilon_\psi^{(t+1)} - \Upsilon_\psi^{(t)} \geq \eta \Gamma_{\psi}^{+,(t)} - O(\eta\varpi)
    \end{align*}
    Since \(\Gamma_{\psi}^{+,(t)} \geq 0\) For any \(\delta \gg \varpi\) by \cref{lem:grad-estimate-phase-2-symmetry}c, we know that the above bound accumulate at \(t \geq T_{\psi}^{2,1}\) at 
    \begin{align*}
        \Upsilon_\psi^{(t)} - \Upsilon_\psi^{(T_{\psi}^{2,1})} \geq \sum_{s\in[T_{\psi}^{2,1},t]\cup B^+_\delta} \eta \delta - O(\eta\varpi)
    \end{align*}
    So for no more than \(O(\frac{\log^2 d}{\eta \delta})\) iterations we should have the that \(\Upsilon^{(t)} \geq \Omega(\log^2 d)\), which is impossible in our setting. This contradiction proves the result for (a). The proof of (b) is similar to (a) and the proof of \cref{lem:grad-stationarity-simply}b. (c) is also similar to \cref{lem:grad-stationarity-simply}c by using \cref{lem:feature-magnitude-phase-2-symmetry}a-b with parameter \(\lambda/ d^{c_1}\) and parameter \(n_G\varpi\), alonge with \cref{lem:symmetry-of-features-symmetry}. We do not repeat the details here.
\end{proof}

\begin{lemma}[Symmetry of features]\label{lem:symmetry-of-features-symmetry}
    Let \(\psi = (j,r,\varphi) \in \Sigma^\star\) and \(t \geq T_{\psi}^{2,3}\), we have that if \(\cF_{\psi,2}(\delta)\) hold for some \(\delta > 0\), then
    \begin{align*}
        |V_{j,r}^{(t)}(g) - C_\alpha V_{j,r}^{(t)}(y)| \leq O( \delta)
    \end{align*}
    For some \(C_\alpha = \frac{1 + n_G(n_Y - 1)/n_Y}{(n_Y-1) + n_G/n_Y}\).
\end{lemma}

\begin{proof}
    For a feature \(\psi = (j,r, \varphi) \in \Sigma^\star\), their update can be represented as:
    \begin{align*}
        V_{j,r}^{(t)}(g) -  V_{j,r}^{(0)}(g) & = U_{g,y} - \sum_{y'\neq y} R_{g,y'}\\
        V_{j,r}^{(t)}(y) - V_{j,r}^{(0)}(y) & = \sum_{g\in \fiber_{j,y}} U_{g,y} - \sum_{g'\notin \fiber_{j,y}} R_{g',y} 
    \end{align*}
    where \(U_{g,y} = \sum_{s\leq t} \eta \Gamma_{j,r,\phi}^{+,(s)}\) and \(R_{g,y'}\) is
    \begin{align*}
        R_{g,y'} = \sum_{s\leq t}\eta\E[\logit_{5,j}^{(s)} \ReLU'(\Lambda_{5,j,r}^{(s)} )\1_{\cH_{g',y}}]
    \end{align*}
    Therefore:
    \begin{align*}
        \sum_{g\in \fiber_{j,y}}  V_{j,r}^{(t)}(g) = \sum_{g\in \fiber_{j,y}}U_{g,y} - \sum_{g\in \fiber_{j,y}}\sum_{y'\neq y} R_{g,y'}
    \end{align*}
    Thus it is true for any \(\fiber_{j,y}\) and \(y \in \cY\) that the following holds
    \begin{align*}
        V_{j,r}^{(t)}(y) - V_{j,r}^{(0)}(y) + \sum_{g'\notin \fiber_{j,y}} R_{g',y}  =  \sum_{g\in \fiber_{j,y}} (V_{j,r}^{(t)}(g) -  V_{j,r}^{(0)}(g)) + \sum_{g\in \fiber_{j,y}}\sum_{y'\neq y} R_{g,y'}
    \end{align*}
    By summing both the left and right hand side with all the \((\fiber_{j,y}, y)\) pairs, we have 
    \begin{align*}
        \sum_{y\in\cY}( V_{j,r}^{(t)}(y) - V_{j,r}^{(0)}(y)) + \sum_{y\in\cY}\sum_{g'\notin \fiber_{j,y}} R_{g',y} = \sum_{g\in \cG} (V_{j,r}^{(t)}(g) -  V_{j,r}^{(0)}(g)) + \sum_{g\in \cG}\sum_{y': \tau(g\cdot y') \neq j} R_{g,y'}
    \end{align*}
    Since \(\sum_{y\in\cY}\sum_{g'\notin \fiber_{j,y}} R_{g',y} = \sum_{g\in \cG}\sum_{y': \tau(g\cdot y') \neq j} R_{g,y'}\), we have 
    \begin{align*}
        &\sum_{y\in\cY}( V_{j,r}^{(t)}(y) - V_{j,r}^{(0)}(y)) = \sum_{g\in \cG} (V_{j,r}^{(t)}(g) -  V_{j,r}^{(0)}(g)) \\
        \implies \ &  \sum_{y\in\cY} V_{j,r}^{(t)}(y) = \sum_{g\in \cG} V_{j,r}^{(t)}(g) \pm O(\mu)
    \end{align*}
    Moreover, we have assumed
    \begin{align*}
        &|V_{j,r}^{(t)}(g) + V_{j,r}^{(t)}(y')| = O(\delta),\quad  \forall y'\neq y, g\in\fiber_{j,y}, \\ 
        &|V_{j,r}^{(t)}(y) + V_{j,r}^{(t)}(g')| = O(\delta),\quad \forall g'\notin \fiber_{j,y} 
    \end{align*}
    Thus 
    \begin{align*}
        (1 + \frac{|\cG|}{|\cY|}(|\cY| - 1))  V_{j,r}^{(t)}(y) = (\frac{|\cY|(|\cY| - 1)}{|\cG|} + 1)\sum_{g\in\fiber_{j,y}}V_{j,r}^{(t)}(g) \pm O(n_G \delta)
    \end{align*}
    Therefore choosing whatever \(g \in \fiber_{j,t}\), we have 
    \begin{align*}
        (1 + \frac{|\cG|}{|\cY|}(|\cY| - 1))  V_{j,r}^{(t)}(y) &= (\frac{|\cY|^2}{|\cG|} + 1)\frac{|\cG|}{|\cY|}V_{j,r}^{(t)}(g) \pm O(n_G \delta_2) = (|\cY| - 1 + \frac{|\cG|}{|\cY|})V_{j,r}^{(t)}(g) \pm O(n_G \delta)
    \end{align*}
    Now by dividing the right-hand side by the factor on \(V_{j,r}^{(t)}(y)\) on the LHS, we have the desired result.
\end{proof}

\subsubsection{Proof of Induction}

\begin{lemma}[Arrival times]\label{lem:arrival-time-estimate-symmetry}
    Assuming \cref{induction:symmetric-group-actions,induction:phase-II-symmetry} for \(\psi=(j,r,\varphi) \in \Sigma^\star\), we have
    \begin{enumerate}[(a)]
        \item \(T^{2,1}_{\psi} - T^1_{\psi} \le \tO(\frac{1}{\eta(d^{c_1}\sigma_0)^{q-2}})\);
        \item \(T^{2,2}_{\psi} - T^{2,1}_{\psi} \le O(\frac{d^{\frac{1}{2}+c_1}}{\eta})\);
        \item \(T^{2,3}_{\psi} \le \tO(\frac{1}{\eta\varpi})\);
        \item \(T^{2}_{\psi} \le \tO(\frac{1}{\eta(d^{2c_1}\mu)^{q-2}})\).
    \end{enumerate}
\end{lemma}
\begin{proof}
    \begin{enumerate}[(a)]
        \item By \cref{lem:grad-estimate-phase-2-symmetry}a, after phase I, within \(\tO(1/(\eta(d^{c_1}\sigma_0)^{q-2}))\) we have \(V_{j,r}^{(t)}(\phi) \geq \frac{1}{2}\log d\) for all \(\phi \in \varphi\). 
        \item On \([T^1,T^{2,1}]\), \cref{lem:grad-stationarity-symmetry}a gives the rate \(O(d^{\frac12+c_1}/\eta)\) time for the positive grad to arrive at \(1 - \logit_{5,j}\leq \frac{1}{\sqrt{d}}\) conditioned on \(\cH_{\varphi}\).
        \item This is again provided by \cref{lem:grad-stationarity-symmetry}c by using \(\delta = n_G\varpi\), combined with \cref{lem:feature-magnitude-phase-2-symmetry}b for the feature shape guarantees.
        \item This is simply by the definition of \(T_\psi^2\) and that \(T_\psi^{2,3} \ll \frac{1}{\eta(d^{2c_1}\mu)^{q-2}}\).
    \end{enumerate}
\end{proof}

\begin{proof}[Proof of \cref{induction:symmetric-group-actions,induction:phase-II-symmetry}]
    Fix \(\psi=(j,r,\varphi_{j,y})\) and consider \(t\in[T^1_{\psi},T^{2}_{\psi}]\), then
    \begin{itemize}
        \item \Cref{induction:symmetric-group-actions}a: This is guaranteed by \cref{lem:grad-estimate-phase-2-symmetry}a-b.
        \item \Cref{induction:symmetric-group-actions}b: Because the time used is small between \(T_{\psi}^1\) and \( T_\psi^2\) by \cref{lem:arrival-time-estimate-symmetry}, we have that \(T_{\psi'}^1\) is still behind when this happens.
        \item \Cref{induction:symmetric-group-actions}c: This will be proven in phase III, we do not use any properties that violate this hypothesis.
        \item \Cref{induction:symmetric-group-actions}d: This is due to \cref{induction:phase-II-symmetry}a-b and \cref{lem:grad-stationarity-symmetry}a with parameter \(\delta = 1/\sqrt{d}\).
        \item \Cref{induction:phase-II-symmetry}: For any \(\phi \notin \varphi\), we have that they are either desirable features \(\phi \in \varphi_{j,y'}\) for some different \(y'\geq y\in\cY\), or that they are the wrong features for the class \(j\), i.e., \(\phi \in \varphi_{j', y}\) for all \(y\in\cY\). The feature \(\psi' = (j,r,\varphi_{j,y'})\) in the former case was out-competed by \(\psi = (j,r,\phi)\) so their growth is bounded by \(\tO(\mu)\) in the first two stage and negative in the end due to the feature shape \(\cF_\psi(d^{c_1}\mu, (\frac{d\varpi}{\lambda})^{\frac{1}{q-1}})\) at the end as shown in \cref{lem:grad-stationarity-symmetry}c. Combination \(\phi\) in the latter case will be smaller as well due to \cref{lem:grad-stationarity-symmetry}c.
    \end{itemize}
    The time bounds follow from \cref{lem:arrival-time-estimate-symmetry}. This proves the induction in phase II for any feature \(\psi \in \Sigma^\star\).
\end{proof}

\subsection{Phase III: Convergence}

We present the induction hypothesis in this phase.

\begin{induction}[Phase III, final]\label{induction:phase-III-symmetry}
    Let \(\psi = (j,r,\varphi) \in \Sigma^\star\), for \(t \in [T_{\psi}^{2}, T_1]\), the following holds:
    \begin{enumerate}[(a)]
        \item At \(t \in [T_\psi^2, T_{1,1}]\), we have \(\cF_{\psi}(\delta_1, \delta_2)\) holds with \(\delta_1 = d^{c_1}\mu\), \(\delta_2 = \tO((\frac{d\varpi}{\lambda})^{\frac{1}{q-1}})\);
        \item At \(t \in [T_{1,1}, T_1]\), we have \(\cF_{\psi}(\delta_1, \delta_2)\) holds with \(\delta_1 = d^{c_1}\mu\), \(\delta_2 = \tO((\varpi)^{\frac{1}{q-1}})\).
    \end{enumerate}
\end{induction}

We obtain the shape of logits at convergence.

\begin{lemma}[Logit shape at convergence]\label{lem:logit-shape-at-convergence-symmetry}
    Assuming \cref{,induction:phase-III-symmetry}. Let \(\psi_{n_y^2} = \in \Sigma^\star\) be the last of \(\Sigma^\star\), then at \(t \in [T_{\psi_{n_y^2}}^{2,3}, T_1]\), the followings hold:
    \begin{enumerate}[(a)]
        \item For \(\phi \in \Phi^\star_j\), \(\logit_{5,j}^{(t)} \geq 1 - O(\lambda) \) conditioned on \(\cH_\phi\).
        \item For \(\phi \notin\Phi^\star_j\), \(\logit_{5,j}^{(t)} \leq O(\lambda /d) \) conditioned on \(\cH_\phi\).
    \end{enumerate}
\end{lemma}

\begin{proof}
    The proof is very similar to that of \cref{lem:logit-shape-at-convergence}, we do not repeat here.
\end{proof}

\begin{lemma}[Gradient bounds, phase III]\label{lem:grad-bound-phase-3-symmetry}
    Assuming \cref{induction:symmetric-group-actions,induction:phase-III-symmetry} holds, and let \(\psi = (j,r,\varphi)\in\Sigma^\star\), then at \(t\in [T_{1,1}, T_1]\), the followings hold
    \begin{enumerate}[(a)]
        \item \(\Gamma_{\psi}^{(t)}\geq \Omega(\lambda)\) when \(V_\psi^{(t)} \leq B - d^{c_1}\mu\);
        \item \(|\Gamma_{j,r,v}^{-,(t)}| \leq O(\frac{\lambda \varpi}{d n_y}) \) for any \(v \in \cG\cup\cY\);
        \item For \(\phi \in \cup_{y\in\cY}\varphi_{j,y} \), we have
        \begin{align*}
            \Gamma_{j,r,\phi}^{+,(t)} = 
            \begin{cases}
                -\Theta(\lambda\varpi/n_y^2), & \text{if } V_{j,r,\phi'}^{(t)} \in [-B + 2\mu, -\varpi] \\
                \in [-\Theta(\lambda\varpi/n_y^2), -\Theta(\lambda\varpi/dn_y^2)] \cup \{0\}, & \text{if } V_{j,r,\phi'}^{(t)} \in [-B - 2\mu, -B + 2\mu]
            \end{cases}
        \end{align*}
    \end{enumerate}
\end{lemma}

\begin{proof}
    The proof of (a) is simple as \(\lambda \gg \varpi\) which is the upper bound of all other gradient terms. (b) is because of \cref{lem:logit-shape-at-convergence-symmetry}b and that \cref{induction:phase-III-symmetry}b holds which guaranteed suitable feature shape. (c) is simply a combination of (b) and the \cref{lem:logit-shape-at-convergence-symmetry}a, and \cref{induction:phase-III-symmetry}b.
\end{proof}

\subsubsection{Proof of Induction in Phase III and \cref{thm:learning-symmetric-actions}}

\begin{proof}[Proof of \cref{induction:phase-III-symmetry,thm:learning-symmetric-actions}]
    The proof is similar to the proof of \cref{induction:phase-III-simply}, thus we leave out some details here.
    \begin{itemize}
        \item \cref{induction:phase-III-symmetry}a: We know from \cref{induction:phase-II-symmetry,lem:grad-stationarity-symmetry,lem:arrival-time-estimate-symmetry} that \cref{induction:phase-III-symmetry}a holds at \(t = T_{\psi}^2\). The feature shape \(\cF_{\psi}(d^{c_1}\mu, \tO((\frac{d\varpi}{\lambda})^{\frac{1}{q-1}}))\) guaranteed a starting point for reusing the argument in \cref{lem:grad-stationarity-symmetry}c to guarantee that the feature stays at the same feature shape.
        \item \cref{induction:phase-III-symmetry}b: After \(t = T_{1,1} \), which is \(t = T_{\psi_{n_y^2}}^2\) for the last \(\psi \in \Sigma^\star\). We shall prove that the feature shape \(\cF_{\psi}(d^{c_1}\mu, \tO((\varpi)^{\frac{1}{q-1}}))\) holds for every \(\psi \in \Sigma^\star\). Since (a) holds for all \(\psi \in\Sigma^\star\), we have that at \(T_{\psi_{n_y^2}}^{2,3}\) that the feature shape in (a) holds for all \(\psi = (j,r,\varphi) \in \Sigma^\star\). So by \cref{lem:grad-bound-phase-3-symmetry}b and similar argument in \cref{lem:grad-stationarity-simply}c, we can prove that the feature \(|V_{j,r}^{(t)}(\phi)| \leq \tO(\varpi^{\frac{1}{q-1}})\) for any \(\phi \in \Phi_\varphi^\dagger\), which is \(\cF_{\psi,2}(\tO(\varpi^{\frac{1}{q-1}}))\). \(\cF_{\psi,1}(d^{c_1}\mu)\) is guaranteed by \cref{lem:grad-bound-phase-3-symmetry}a and \(\cF_{\psi,3}(\tO(\varpi^{\frac{1}{q-1}}))\) is guaranteed by \(\cF_{\psi,1}, \cF_{\psi,2}\) and \cref{lem:symmetry-of-features-symmetry}. \(\cF_{\psi,4}\) is simple as any feature \(\psi \in \Sigma^{\dagger,3}_\psi\) lost the competition and are bounded by \(\tO(\mu)\). Moreover, once each \(\varphi_{j,y}, j\in\tau(\cY), y\in\cY\) is learned. the feature in \(\Sigma^{\dagger,3}_\psi\) has too small gradient bounded by \(O(\tfrac{\lambda}{d}\mu^{q-1})\) which will not have sufficient growth before \(t = T_1\).
    \end{itemize}
    Since all feature \(\psi \in \Sigma^\star\) have shape \(\cF_{\psi}(\delta_1,\delta_2)\) for \(\delta_1 = d^{c_1}\mu\) and \(\delta_2 = \tO(\varpi^{\frac{1}{q-1}})\) at \(t = T_1\), we have proven \cref{thm:learning-symmetric-actions}.
\end{proof}

\section{Learning the Attention Layer: Simply Transitive Case}
In this section, we consider the case where the group operations form a simply transitive group. According to \Cref{assump:structure-1}, we assume that for any \( y_1, y_2 \in \cY \), there exists a unique \( g \in \cG \) such that \( g \cdot y_1 = y_2 \). Without loss of generality, we let \( \cY = \{0, 1, \dots, n_y - 1\} \), where \( n_y \in [\Omega(\log\log d), \log d] \).

We focus on updating only \( \Qb \), while keeping \( \Wb \) fixed. Combined with the attention structure specified in \Cref{assump-Q-structure}, it suffices to consider the updates to the blocks \( \Qb_{4,3} \) and \( \Qb_{4,4} \) only. We consider the contribution to the gradient from the position \( i = 5 \) on task \( \cT^2 \); specifically, the relevant loss function is given by \(\sum_{\ell=1}^{2} \Loss_{5}^{2,\ell}.\) As \( \Wb \) remains fixed in this section, we omit the superscript \( (t) \) in \( \Wb \) and in all related notations that depend solely on \( \Wb \) (e.g., \( V_{j, r} \)) for notational simplicity.

\subsection{Gradient Computations}
\paragraph{Notations for gradient expressions.} We firsr introduce some notations for the gradients of the attention layer. For $1 \leq \ell\leq L$, given $\Zb^{L,\ell-1}$ and $\kk\in \mathcal{I}^{L, \ell-1}$, define   
   \begin{align}
    & \Xi^{L}_{\ell, i,\kk}(\Zb^{L,\ell-1})\triangleq \sum_{j \in[d]} \Ecal_{i,j}(\Zb^{L,\ell-1}) \sum_{r\in [m]}\ReLU^{\prime}\big(\Lambda_{i, j,r}(\Zb^{L,\ell-1})\big)\dbrack{\Wb_{i,j,r},\Z_{\kk}}, \quad { i\in [5].}\label{eq-def-xi}
    \end{align}
For simplicity of notation, we will henceforth omit the dependence on $\Zb^{L,\ell-1}$ in the notation of $\Xi^{L}_{\ell, i,\kk}$ 
when it is clear from the context.

\begin{fact}[Gradients of \(\Q\)]\label{fact:gradients-Q}
    For any \(p,q \in [5]\), 
    we have 
    \begin{align*}
        &-\nabla_{\Q_{p,q}}\Loss^{L}  = \sum_{\ell=1}^{L}\sum_{i\in [5]} -\nabla_{\Q_{p,q}}\Loss^{L,\ell}_{i}, \quad  \text{ where }\\
      & -\nabla_{\Q_{p,q}}\Loss^{L,\ell}_{i}=\\
       & ~~~~~~~\E\Bigg[\sum_{\mathbf{k} \in \mathcal{I}^{L, \ell-1}}\attn_{{\ans,\ell-1} \rightarrow \kk} \cdot\left(\Xi^{L}_{\ell, i,\kk} - \sum_{\mathbf{k}^{\prime} \in \mathcal{I}^{L, \ell-1}}\attn_{{\ans,\ell-1} \rightarrow \kk^{\prime}}\Xi^L_{\ell, i,\kk^{\prime}}\right)\Z_{\ans,\ell-1,p}\Z_{\kk,q}^{\top} \Bigg].
    \end{align*}
\end{fact}

\begin{lemma}[Gradients of $\Qb_{4,3}$]\label{lem-gradients-Q43}
    Given $s\in\tau(\X)$, for the diagonal entry \( [\Q_{4,3}]_{s,s} \) of the block \(\Qb_{4,3}\), we have 
    \begin{align*}
    \Big[-\nabla_{\Q_{4,3}}\Loss^{2,1}_{5}\Big]_{s,s}&= \E\Bigg[
    \attn_{{\ans,0} \rightarrow \pred,1} \cdot \bigg(\sum_{j \in[d]} \Ecal_{5,j}(\Zb^{2,0})\sum_{r\in [m]}\ReLU^{\prime}\big(\Lambda_{5, j,r}\big)\cdot  \\
    &~~~~~~~~~~~~~\Big( \dbrack{\Wb_{5,j,r},\Z_{\pred,1}}- \Lambda_{5, j,r}+b_{5,j,r}\Big)\bigg) \1_{s=\tau(x_0)}\Bigg],\\
    \Big[-\nabla_{\Q_{4,3}}\Loss^{2,2}_{5}\Big]_{s,s}&= \E\Bigg[
    \attn_{{\ans,1} \rightarrow \pred,2} \cdot \bigg(\sum_{j \in[d]} \Ecal_{5,j}(\Zb^{2,1})\sum_{r\in [m]}\ReLU^{\prime}\big(\Lambda_{5, j,r}\big)\cdot  \\
    &~~~~~~~~~~~~~\Big( \dbrack{\Wb_{5,j,r},\Z_{\pred,2}}- \Lambda_{5, j,r}+b_{5,j,r}\Big)\bigg) \1_{s=\tau(x_1)}\Bigg].
\end{align*}
Moreover, for the off-diagonal entries \( [\Q_{4,3}]_{s,s'} \) with \( s \neq s' \), we have
    \begin{align*}
    \Big[-\nabla_{\Q_{4,3}}\Loss^{2,1}_{5}\Big]_{s,s'}&= \E\Bigg[
    \attn_{{\ans,0} \rightarrow \pred,2} \cdot \bigg(\sum_{j \in[d]} \Ecal_{5,j}(\Zb^{2,0})\sum_{r\in [m]}\ReLU^{\prime}\big(\Lambda_{5, j,r}\big)\cdot  \\
    &~~~~~\Big( \dbrack{\Wb_{5,j,r},\Z_{\pred,2}}- \Lambda_{5, j,r}+b_{5,j,r}\Big)\bigg) \1_{s=\tau(x_0), s'=\tau(x_1)}\Bigg],\\
    \Big[-\nabla_{\Q_{4,3}}\Loss^{2,2}_{5}\Big]_{s,s'}&= \E\Bigg[
    \attn_{{\ans,1} \rightarrow \pred,1} \cdot \bigg(\sum_{j \in[d]} \Ecal_{5,j}(\Zb^{2,1})\sum_{r\in [m]}\ReLU^{\prime}\big(\Lambda_{5, j,r}\big)\cdot  \\
    &~~~~~\Big( \dbrack{\Wb_{5,j,r},\Z_{\pred,1}}- \Lambda_{5, j,r}+b_{5,j,r}\Big)\bigg) \1_{s=\tau(x_1),s'=\tau(x_0)}\Bigg].
\end{align*}
\end{lemma}
\begin{proof}
    For $\ell=1$, due to \Cref{fact:gradients-Q}, the diagonal entry \( [\Q_{4,3}]_{s,s} \) with \( s \in \tau(\X) \), the expected gradient contribution from \( -\nabla_{\Q_{4,3}} \Loss^{2,1}_{5} \) takes the form
\[
\E\Big[
    \attn_{{\ans,0} \rightarrow \pred,1} \cdot \Big(
        \Xi^2_{\ell, 5,\pred,1}
        - \sum_{\mathbf{k}^{\prime} \in \mathcal{I}^{2, \ell-1}} 
        \attn_{{\ans,\ell-1} \rightarrow \kk^{\prime}} \, \Xi^2_{\ell, 5,\kk^{\prime}}
    \Big)
    \cdot \1_{s = \tau(x_0)}
\Big],
\]
which is nonzero in expectation only when \( s = \tau(x_0) \). Therefore, combined the definition of $\Xi$ in \eqref{eq-def-xi}  we have:
\begin{align*}
    \Big[-\nabla_{\Q_{4,3}}\Loss^{2,1}_{5}\Big]_{s,s}
    &= \E\Bigg[
    \attn_{{\ans,0} \rightarrow \pred,1} \cdot\Big(\Xi^2_{\ell, 5,\pred,1} - \sum_{\mathbf{k}^{\prime} \in \mathcal{I}^{2, \ell-1}}\attn_{{\ans,\ell-1} \rightarrow \kk^{\prime}}\Xi^2_{\ell, 5,\kk^{\prime}}\Big) \1_{s=\tau(x_0)}\Bigg]\\
    &= \E\Bigg[
    \attn_{{\ans,0} \rightarrow \pred,1} \cdot \bigg(\sum_{j \in[d]} \Ecal_{5,j}\sum_{r\in [m]}\ReLU^{\prime}\big(\Lambda_{5, j,r}\big)\cdot  \\
    &~~~~~~~~~~~\Big( \dbrack{\Wb_{5,j,r},\Z_{\pred,1}}- \sum_{\mathbf{k}^{\prime} \in \mathcal{I}^{2, \ell-1}}\attn_{{\ans,\ell-1} \rightarrow \kk^{\prime}}\dbrack{\Wb_{5,j,r},\Z_{\kk'}}\Big)\bigg) \1_{s=\tau(x_0)}\Bigg]\\
    &= \E\Bigg[
    \attn_{{\ans,0} \rightarrow \pred,1} \cdot \bigg(\sum_{j \in[d]} \Ecal_{5,j}\sum_{r\in [m]}\ReLU^{\prime}\big(\Lambda_{5, j,r}\big)\cdot  \\
    &~~~~~~~~~~~~~~~~~\Big( \dbrack{\Wb_{5,j,r},\Z_{\pred,1}}- \Lambda_{5, j,r}+b_{5,j,r}\Big)\bigg) \1_{s=\tau(x_0)}\Bigg].
\end{align*}
Other quantities are computed similarly, and thus is omitted here.
\end{proof}

\begin{lemma}[Gradients of $\Qb_{4,4}$]\label{lem-gradients-Q44}
    Given $s\in\tau(\X)$, for the diagonal entry \( [\Q_{4,4}]_{s,s} \) of the block \(\Qb_{4,4}\), we have 
    \begin{align*}
    \Big[-\nabla_{\Q_{4,4}}\Loss^{2,1}_{5}\Big]_{s,s}&= \E\Bigg[
    \attn_{{\ans,0} \rightarrow \ans,0} \cdot \bigg(\sum_{j \in[d]} \Ecal_{5,j}(\Zb^{2,0})\sum_{r\in [m]}\ReLU^{\prime}\big(\Lambda_{5, j,r}\big)\cdot  \\
    &~~~~~~~~~~~~~\Big( \dbrack{\Wb_{5,j,r},\Z_{\ans,0}}- \Lambda_{5, j,r}+b_{5,j,r}\Big)\bigg) \1_{s=\tau(x_0)}\Bigg],\\
    \Big[-\nabla_{\Q_{4,4}}\Loss^{2,2}_{5}\Big]_{s,s}&= \E\Bigg[
    \attn_{{\ans,1} \rightarrow \ans,1} \cdot \bigg(\sum_{j \in[d]} \Ecal_{5,j}(\Zb^{2,1})\sum_{r\in [m]}\ReLU^{\prime}\big(\Lambda_{5, j,r}\big)\cdot  \\
    &~~~~~~~~~~~~~\Big( \dbrack{\Wb_{5,j,r},\Z_{\ans,1}}- \Lambda_{5, j,r}+b_{5,j,r}\Big)\bigg) \1_{s=\tau(x_1)}\Bigg].
\end{align*}
Moreover, for the off-diagonal entries \( [\Q_{4,4}]_{s,s'} \) with \( s \neq s' \), we have $[-\nabla_{\Q_{4,4}}\Loss^{2,1}_{5}]_{s,s'}=0$, and
    \begin{align*}
    \Big[-\nabla_{\Q_{4,4}}\Loss^{2,2}_{5}\Big]_{s,s'}&= \E\Bigg[
    \attn_{{\ans,1} \rightarrow \ans,0} \cdot \bigg(\sum_{j \in[d]} \Ecal_{5,j}(\Zb^{2,1})\sum_{r\in [m]}\ReLU^{\prime}\big(\Lambda_{5, j,r}\big)\cdot  \\
    &~~~~~\Big( \dbrack{\Wb_{5,j,r},\Z_{\ans,0}}- \Lambda_{5, j,r}+b_{5,j,r}\Big)\bigg) \1_{s=\tau(x_1),s'=\tau(x_0)}\Bigg].
\end{align*}
\end{lemma}
\begin{proof}
    The analysis is similar to that of \Cref{lem-gradients-Q43}, and we omit the details here.
\end{proof}

\subsection{Some Useful Bounds for Gradients}
In this subsection, we establish several useful bounds on the gradients of the attention layer, leveraging the feature structure of the MLP layer learned during stage 1. These bounds will be instrumental for the subsequent analysis.

As established in \Cref{lem:symmetry-of-features-simply} and \Cref{thm:mlp-simply-transitive}, at the end of stage 1, the network exhibits the following activation properties:
\begin{itemize}[leftmargin=2em]
    \item \textbf{Sparse activations:} For each \( j \in \tau(\cY) \), and any feature \((g, y) \in \fF_j\), there exists a unique \emph{activated} neuron \( r \in [m] \) such that, under the event \( g_1 = g \) and \( y_0 = y \), the following holds:
    \begin{align*}
        \Lambda_{5,j,r}^{(T_1)} &\geq B - O(\delta), \quad |V_{j,r}^{(T_1)}(g) - V_{j,r}^{(T_1)}(y)| \leq O(\delta) \\
        \Lambda_{5,j,r'}^{(T_1)} &\leq O(\delta^{q-1}) \quad \text{for all } r' \neq r.
    \end{align*}

    \item \textbf{Cancellation of incorrect features:} let \( r \in [m] \) be the activated neuron associated with \( (g, y) \in \fF_j \). Then for any \( g' \neq g \in \cG \) and any \( y' \in \cY \), we have:
    \begin{align*}
        \left| V_{j,r}^{(T_1)}(g) + V_{j,r}^{(T_1)}(y') \right| &\leq O(\delta), \\
        \left| V_{j,r}^{(T_1)}(g') + V_{j,r}^{(T_1)}(y) \right| &\leq O(\delta),
    \end{align*}
\end{itemize}
In the analysis of FFN layer,  we have a pair \(\delta=(\delta_1,\delta_2)\); in the
expressions above, some terms should use \(\delta_1\) and others \(\delta_2\).
For brevity, we write \(\delta := \max\{\delta_1,\delta_2\}\) and bound all such
terms by \(\delta\). Since both \(\delta_1\) and \(\delta_2\) are small, this
simplification does not affect our analysis.
\paragraph{Notations for activated neurons.} Since in the simply transitive case, the feature sets are disjoint across indices—i.e., \( \fF_j \cap \fF_{j'} = \emptyset \) for all \( j \neq j' \)—we denote the activated neuron corresponding to \( (g, y) \in \fF_j \) by \( r_{g \cdot y} \). Moreover, let 
\begin{align*}
\fA\triangleq\cup_{j\in\tau(\Y)}\fA_{j}, \text{ where } \fA_j\triangleq \{r\mid \exists (g,y)\in\fF_j, r=r_{g\cdot y}\}.
\end{align*}
In other words, \( \fA \) is the set of all activated neurons across all feature sets \( \fF_j \) for \( j \in \tau(\Y) \). Given $\Zb^{L,\ell-1}$, letting $\hat{\cG}(\Zb^{L,\ell-1})=\cup\{g_{\ell'}\}_{\ell'=1}^L$ be the collection of all the chosen group elements in the predicate clauses. Similarly $\hat{\cY}=\cup\{y_{\ell'}\}_{\ell'=0}^{\ell-1}$. Then define $\hat{\fA}_{j}(\Zb^{L,\ell-1})= \{r_{g\cdot y}\mid (g,y)\in\fF_j\wedge \big(g\in \hat{\cG}(\Zb^{L,\ell-1})\vee y\in \hat{\cY}\big) \}$. For simplicity, we omit the dependence on \( \Zb^{L,\ell-1} \) in the notation of \( \hat{\fA}_j \) when it is clear from the context. Equipped with these notations, we can summarize the above properties in the following lemmas.
\begin{lemma}[Properties of target feature magnitude]\label{lem-prop-psi-cyc}
   Given \( (g, y) \in \fF_j \) with $j\in\tau(\Y)$, then, the following properties hold.
    \begin{align}
    &\frac{1}{2} \left( V_{j, r_{g \cdot y}}(g) + V_{j, r_{g \cdot y}}(y) \right) \geq B - O(\delta), \quad
    \left| V_{j, r_{g \cdot y}}(g) - V_{j, r_{g \cdot y}}(y) \right| \leq O(\delta), \label{attn-init-prop-cyc-1} \\
    &\left| V_{j, r_{g \cdot y}}(g) + V_{j, r_{g \cdot y}}(y') \right| \leq O(\delta), V_{j, r_{g \cdot y}}(y')<0 \quad \text{for all } y' \neq y, \label{attn-init-prop-cyc-2}\\
   & \left| V_{j, r_{g \cdot y}}(g') + V_{j, r_{g \cdot y}}(y) \right| \leq O(\delta), V_{j, r_{g \cdot y}}(g')<0 \quad \text{for all } g' \neq g.\label{attn-init-prop-cyc-3}\\
   &\left|V_{j, r}(g)\right|, \left| V_{j, r}(y) \right| \leq O(\delta) \quad \text{for all } r \notin \fA_{j}. \label{attn-init-prop-cyc-4}
\end{align}
\end{lemma}
\begin{lemma}[Properties of irrelevant magnitude] \label{lem-prop-irrelavant-cyc}
    If \( (p,v) \notin \{2\} \times \cG \cup \{5\} \times \cY \), or \( j \notin \tau(\Y) \), then for any \( r \in [m] \), we have
    \begin{align}
       \big| \langle \Wb_{5,j,r,p}, e_v\rangle\big|\leq \tilde{O}(\sigma_0).
    \end{align}
\end{lemma}

The above lemmas give us some direct computations of the inner products between the weight matrices and input embedding vectors.

\begin{lemma}
Let \( j \in \tau(\cY) \) and \( \ell \in [2] \). Then for any \( r \in [m] \), the following holds:
\begin{align}
    \langle \Wb_{5,j,r}, \Zb_{\pred,\ell} \rangle &= V_{j,r}(g_{\ell}) \pm \tilde{O}(\sigma_0), \label{eq-inner-product-1} \\
    \langle \Wb_{5,j,r}, \Zb_{\ans,\ell-1} \rangle &= V_{j,r}(y_{\ell-1}) \pm \tilde{O}(\sigma_0). \label{eq-inner-product-2}
\end{align}
Moreover, for \( j \notin \tau(\cY) \) and any \( \kk \in \cI^{2,1} \) and \( r \in [m] \), we have
\begin{align}
    \big| \langle \Wb_{5,j,r}, \Zb_{\kk} \rangle \big| = \tilde{O}(\sigma_0). \label{eq-inner-product-3}
\end{align}
\end{lemma}

\begin{proof}
    By direct computations and the definition of $\Zb_{\pred,\ell}$, we have
    \begin{align*}
        \langle \Wb_{5,j,r}, \Zb_{\pred,\ell}\rangle = \langle \Wb_{5,j,r,1}, e_{x_{\ell}}\rangle+ \langle \Wb_{5,j,r,2}, e_{g_{\ell}}\rangle+ \langle \Wb_{5,j,r,3}, e_{x_{\ell-1}}\rangle
    \end{align*}
    Plug in \Cref{lem-prop-irrelavant-cyc} and the definition of $V_{j,r}(g_{\ell})$, we obtain \eqref{eq-inner-product-1}. The proof of \eqref{eq-inner-product-2} and \eqref{eq-inner-product-3} is similar, and we omit the details here.
\end{proof}

Furthermore, we can establish some characterizations of the \(\Lambda_{5,j,r}(\Zb^{2,\ell-1})\) quantities, which are crucial for  the following analysis. 
\begin{lemma}[Characterizations of Lambda]\label{lem-lambda-char}
    Given $\Zb^{2,\ell-1}$ with $\ell\in [2]$, with an attention structure \( \{\attn_{\ans,\ell-1 \to \kk} \}_{\kk\in\cI^{2,\ell-1}}\), 
    \begin{enumerate}[(a)]
        \item for \( j \in \tau(\cY) \), for activated neuron $r\in\fA_j$, we have  
        \begin{align*}
            \Lambda_{5,j,r}(\Zb^{2,\ell-1})=\sum_{\ell'=1}^2 \attn_{\ans,\ell-1 \to \pred,\ell'}V_{j,r}(g_{\ell'})+  \sum_{\ell'=1}^{\ell} \attn_{\ans,\ell-1 \to \ans,\ell'-1} V_{j,r}(y_{\ell'-1})\pm  \tilde{O}(\sigma_0). 
        \end{align*}
         \item for \( j \in \tau(\cY) \), for any non-activated neuron $r\notin\fA_j$ 
         we have
      \begin{align*}
            \Big|\Lambda_{5,j,r}(\Zb^{2,\ell-1})\Big| \leq {O}(\delta).
        \end{align*} 
        \item for \( j \notin \tau(\cY) \),  for any $r\in [m]$,  we have 
        \begin{align*}
            \Big|\Lambda_{5,j,r}(\Zb^{2,\ell-1})\Big|
            \leq \tilde{O}(\sigma_0).
        \end{align*} 
    \end{enumerate}
\end{lemma}
\begin{proof}
    Recall the definition of \(\Lambda_{5,j,r}(\Zb^{2,\ell-1})\) in \eqref{eq-def-Lambda-icl}, we have 
\begin{align*}
\Lambda_{5,j,r}(\Zb^{2,\ell-1})=\sum_{\mathbf{k} \in \mathcal{I}^{2, \ell-1}} \attn_{{\ans,\ell-1} \rightarrow \kk}\cdot\big\langle \Wb_{5,j, r}, \mathbf{Z}_{\mathbf{k}}\big\rangle +b_{5,j,r}.
\end{align*}
Thus, the first part follows directly from \eqref{eq-inner-product-1} and \eqref{eq-inner-product-2}; similarly, the second part holds by plugging \eqref{attn-init-prop-cyc-4} into \eqref{eq-inner-product-1} and \eqref{eq-inner-product-2}; the last part is a direct consequence of 
\eqref{eq-inner-product-3} and the fact $b_{i,j,r}=\sigma_0\log d$. 

\end{proof}
A direct consequence of the above lemma is the following finer characterization of the activated neurons.

\begin{lemma}\label{lem-non-activated-neuron}
    Given $j\in\tau(\cY)$, for $r\in \fA_{j}\setminus\hat{\fA}_{j}$, we have $\ReLU^{\prime}\big(\Lambda_{5, j,r}\big)=0$.
\end{lemma}
\begin{proof}
    For $j\in\tau(\Y)$, for $r\in \fA_{j}$, by \Cref{lem-lambda-char}, we have
    \begin{align*}
        \Lambda_{5,j,r}(\Zb^{2,\ell-1})=  \sum_{\ell'=1}^2 \attn_{\ans,\ell-1 \to \pred,\ell'}V_{j,r}(g_{\ell'})+  \sum_{\ell'=1}^{\ell} \attn_{\ans,\ell-1 \to \ans,\ell'-1} V_{j,r}(y_{\ell'-1})\pm  \tilde{O}(\sigma_0).
    \end{align*}
    By \eqref{attn-init-prop-cyc-4},  for for $r\in \fA_{j}\setminus\hat{\fA}_{j}$, we have $V_{j,r}(g_{\ell'}),V_{j,r}(y_{\ell'-1})\leq -B-O(\delta)$. Hence
    \begin{align*}
        \Lambda_{5,j,r}=  -\Omega(B)\pm  \tilde{O}(\sigma_0)\ll -\varrho,
    \end{align*}
    which implies $\ReLU^{\prime}\big(\Lambda_{5, j,r}\big)=0$.
\end{proof}

Now we are ready to further derive the gradients of the attention layer starting from \Cref{lem-gradients-Q43,lem-gradients-Q44} and the properties established above.
\begin{lemma}[Refined expression for the gradient of $\Qb_{4,3}$] \label{lem-refined-grad-Q43} Given $s\in\tau(\X)$, for the diagonal entry \( [\Q_{4,3}]_{s,s} \) of the block \(\Qb_{4,3}\), letting $j_1=\tau(g_1(y_0))$ and $j_2=\tau(g_2(y_1))$, we have 
      \begin{align*}
    &\Big[-\nabla_{\Q_{4,3}}\Loss^{2,1}_{5}\Big]_{s,s}= \E\Bigg[
    \attn_{{\ans,0} \rightarrow \pred,1} \cdot\\
    &~~~~~~~~~~~~~ \bigg( (1-\logit_{5,j_1})\cdot \Big(\sum_{r\in\hat{\fA}_{j_1}}\ReLU^{\prime}(\Lambda_{5,j_1, r
    })
    \cdot \Big( V_{j_1,  r}(g_1)- \Lambda_{5,j_1, r}\pm\tilde{O}(\sigma_0) \Big)\pm \tilde{O}(\delta^{q}) \Big)\\   
    &~~~~~~~~~~-\sum_{j\neq j_1\in\tau(\cY)}\logit_{5,j} \cdot \Big(\sum_{r\in\hat{\fA}_{j}}\ReLU^{\prime}(\Lambda_{5,j,r
    })\cdot  \Big( V_{j, r}(g_1)- \Lambda_{5,j,r}\pm\tilde{O}(\sigma_0) \Big)\pm \tilde{O}(\delta^{q}) \Big) \\
    &~~~~~~~~~~~~\pm\sum_{j\notin\tau(\cY)}\logit_{5,j}\tilde{O}(\sigma^{q}_0)  
    \bigg)\1_{\tau(x_0)=s}\Bigg];
\end{align*}
  \begin{align*}
    &\Big[-\nabla_{\Q_{4,3}}\Loss^{2,2}_{5}\Big]_{s,s}= 
      \E\Bigg[
    \attn_{{\ans,1} \rightarrow \pred,2} \cdot \\
    &~~~~~~~~~~~~~ \bigg( (1-\logit_{5,j_2})\cdot \Big(\sum_{r\in\hat{\fA}_{j_2}}\ReLU^{\prime}(\Lambda_{5,j_2, r
    })\cdot \Big( V_{j_2,  r}(g_2)- \Lambda_{5,j_2, r}\pm\tilde{O}(\sigma_0) \Big)\pm \tilde{O}(\delta^{q}) \Big)\\   
    &~~~~~~~~~~-\sum_{j\neq j_2\in\tau(\cY)}\logit_{5,j} \cdot \Big(\sum_{r\in\hat{\fA}_{j}}\ReLU^{\prime}(\Lambda_{5,j,r
    })\cdot  \Big( V_{j, r}(g_2)- \Lambda_{5,j,r}\pm\tilde{O}(\sigma_0) \Big)\pm \tilde{O}(\delta^{q}) \Big) \\
&~~~~~~~~~~~~\pm\sum_{j\notin\tau(\cY)}\logit_{5,j}\tilde{O}(\sigma^{q}_0)  \bigg)\1_{\tau(x_1)=s}\Bigg].
\end{align*}
Moreover, for the off-diagonal entries \( [\Q_{4,3}]_{s,s'} \) with \( s \neq s' \), we have
  \begin{align*}
    &\Big[-\nabla_{\Q_{4,3}}\Loss^{2,1}_{5}\Big]_{s,s'}= \E\Bigg[
    \attn_{{\ans,0} \rightarrow \pred,2} \cdot\\
    &~~~~~~~~~~~~~ \bigg( (1-\logit_{5,j_1})\cdot \Big(\sum_{r\in\hat{\fA}_{j_1}}\ReLU^{\prime}(\Lambda_{5,j_1, r
    })\cdot \Big( V_{j_1,  r}(g_2)- \Lambda_{5,j_1, r}\pm\tilde{O}(\sigma_0) \Big)\pm \tilde{O}(\delta^{q}) \Big)\\   
    &~~~~~~~~~~-\sum_{j\neq j_1\in\tau(\cY)}\logit_{5,j} \cdot \Big(\sum_{r\in\hat{\fA}_{j}}\ReLU^{\prime}(\Lambda_{5,j,r
    })\cdot  \Big( V_{j, r}(g_2)- \Lambda_{5,j,r}\pm\tilde{O}(\sigma_0) \Big)\pm \tilde{O}(\delta^{q}) \Big) \\
    &~~~~~~~~~~~~\pm\sum_{j\notin\tau(\cY)}\logit_{5,j}\tilde{O}(\sigma^{q}_0)  \bigg)\1_{\tau(x_0)=s,\tau(x_1)=s'}\Bigg];
\end{align*}
  \begin{align*}
    &\Big[-\nabla_{\Q_{4,3}}\Loss^{2,2}_{5}\Big]_{s,s'}=  \E\Bigg[
    \attn_{{\ans,1} \rightarrow \pred,1} \cdot\\
    &~~~~~~~~~~~~~ \bigg( (1-\logit_{5,j_2})\cdot \Big(\sum_{r\in\hat{\fA}_{j_2}}\ReLU^{\prime}(\Lambda_{5,j_2, r
    })\cdot \Big( V_{j_2,  r}(g_1)- \Lambda_{5,j_2, r}\pm\tilde{O}(\sigma_0) \Big)\pm \tilde{O}(\delta^{q}) \Big)\\   
    &~~~~~~~~~~-\sum_{j\neq j_2\in\tau(\cY)}\logit_{5,j} \cdot \Big(\sum_{r\in\hat{\fA}_{j}}\ReLU^{\prime}(\Lambda_{5,j,r
    })\cdot  \Big( V_{j, r}(g_1)- \Lambda_{5,j,r}\pm\tilde{O}(\sigma_0) \Big)\pm \tilde{O}(\delta^{q}) \Big) \\
    &~~~~~~~~~~~~\pm\sum_{j\notin\tau(\cY)}\logit_{5,j}\tilde{O}(\sigma^{q}_0)  \bigg)\1_{\tau(x_1)=s,\tau(x_0)=s'}\Bigg].
\end{align*}
\end{lemma}
\begin{proof}
    For $\ell=1$, 
   \begin{itemize}[left=0pt]
        \item for the diagonal entry \( [\Q_{4,3}]_{s,s} \) with \( s \in \tau(\X) \),
  \begin{align*}
    &\Big[-\nabla_{\Q_{4,3}}\Loss^{2,1}_{5}\Big]_{s,s}= \E\Bigg[
    \attn_{{\ans,0} \rightarrow \pred,1} \cdot \bigg(\sum_{j \in[d]} \Ecal_{5,j}(\Zb^{2,0})\sum_{r\in [m]}\ReLU^{\prime}\big(\Lambda_{5, j,r}\big)\cdot  \\
    &~~~~~~~~~~~~~~~~~~~~~~~~~~~~\Big( \dbrack{\Wb_{5,j,r},\Z_{\pred,1}}- \Lambda_{5, j,r}+b_{5,j,r}\Big)\bigg) \1_{s=\tau(x_0)}\Bigg]\\
     &= \E\Bigg[
    \attn_{{\ans,0} \rightarrow \pred,1} \cdot \bigg( (1-\logit_{5,j_1})\cdot \sum_{r\in[m]}\ReLU^{\prime}(\Lambda_{5,j_1, r
    }) \Big( V_{j_1, r}(g_1)- \Lambda_{5,j_1,r}\pm\tilde{O}(\sigma_0) \Big)\\   
    &~~~~~~~~~~~~-\sum_{j\neq j_1\in\tau(\cY)}\logit_{5,j} \cdot \sum_{r\in[m]}\ReLU^{\prime}(\Lambda_{5,j, r
    }) \Big( V_{j, r}(g_1)- \Lambda_{5,j,r}\pm\tilde{O}(\sigma_0) \Big)\\
    &~~~~~~~~~~~~-\sum_{j\notin\tau(\cY)}\logit_{5,j} \cdot \sum_{r\in[m]}\ReLU^{\prime}(\Lambda_{5,j, r
    }) \Big( \langle\Wb_{5,j,r},\Zb_{\pred,1}\rangle- \Lambda_{5,j,r}\pm\tilde{O}(\sigma_0) \Big)
    \bigg)\1_{\tau(x_0)=s}\Bigg]\\
      &= \E\Bigg[
    \attn_{{\ans,0} \rightarrow \pred,1} \cdot\\
    &~~~~~~~~~~~~~ \bigg( (1-\logit_{5,j_1})\cdot \Big(\sum_{r\in\hat{\fA}_{j_1}}\ReLU^{\prime}(\Lambda_{5,j_1, r
    })\cdot \Big( V_{j_1,  r}(g_1)- \Lambda_{5,j_1, r}\pm\tilde{O}(\sigma_0) \Big)\pm \tilde{O}(\delta^{q}) \Big)\\   
    &~~~~~~~~~~-\sum_{j\neq j_1\in\tau(\cY)}\logit_{5,j} \cdot \Big(\sum_{r\in\hat{\fA}_{j}}\ReLU^{\prime}(\Lambda_{5,j,r
    })\cdot  \Big( V_{j, r}(g_1)- \Lambda_{5,j,r}\pm\tilde{O}(\sigma_0) \Big)\pm \tilde{O}(\delta^{q}) \Big) \\
    &~~~~~~~~~~~~\pm\sum_{j\notin\tau(\cY)}\logit_{5,j}\tilde{O}(\sigma^{q}_0)  \bigg)\1_{\tau(x_0)=s}\Bigg],
\end{align*}
where the last equality follows from \Cref{lem-lambda-char} and \Cref{lem-non-activated-neuron}.
   \item for the non-diagonal entry \( [\Q_{4,3}]_{s,s'} \) with \( s\not=s' \in \tau(\X) \), the analysis is similar unless the condition that the gradient is non-zero only when \( s = \tau(x_0) \) and \( s' = \tau(x_1) \). Thus, we have
  \begin{align*}
    &\Big[-\nabla_{\Q_{4,3}}\Loss^{2,1}_{5}\Big]_{s,s'}= \E\Bigg[
    \attn_{{\ans,0} \rightarrow \pred,2} \cdot \bigg(\sum_{j \in[d]} \Ecal_{5,j}(\Zb^{2,0})\sum_{r\in [m]}\ReLU^{\prime}\big(\Lambda_{5, j,r}\big)\cdot  \\
    &~~~~~~~~~~~~~~~~~~~~~~~~~~~~\Big( \dbrack{\Wb_{5,j,r},\Z_{\pred,2}}- \Lambda_{5, j,r}+b_{5,j,r}\Big)\bigg) \1_{s=\tau(x_0), s'=\tau(x_1)}\Bigg]\\
      &= \E\Bigg[
    \attn_{{\ans,0} \rightarrow \pred,2} \cdot\\
    &~~~~~~~~~~~~~ \bigg( (1-\logit_{5,j_1})\cdot \Big(\sum_{r\in\hat{\fA}_{j_1}}\ReLU^{\prime}(\Lambda_{5,j_1, r
    })\cdot \Big( V_{j_1,  r}(g_2)- \Lambda_{5,j_1, r}\pm\tilde{O}(\sigma_0) \Big)\pm \tilde{O}(\delta^{q}) \Big)\\   
    &~~~~~~~~~~-\sum_{j\neq j_1\in\tau(\cY)}\logit_{5,j} \cdot \Big(\sum_{r\in\hat{\fA}_{j}}\ReLU^{\prime}(\Lambda_{5,j,r
    })\cdot  \Big( V_{j, r}(g_2)- \Lambda_{5,j,r}\pm\tilde{O}(\sigma_0) \Big)\pm \tilde{O}(\delta^{q}) \Big) \\
    &~~~~~~~~~~~~\pm\sum_{j\notin\tau(\cY)}\logit_{5,j}\tilde{O}(\sigma^{q}_0)  \bigg)\1_{\tau(x_0)=s,\tau(x_1)=s'}\Bigg].
\end{align*}
\end{itemize}
 For $\ell=2$, 
   \begin{itemize}[left=0pt]
        \item for the diagonal entry \( [\Q_{4,3}]_{s,s} \) with \( s \in \tau(\X) \),
  \begin{align*}
    &\Big[-\nabla_{\Q_{4,3}}\Loss^{2,2}_{5}\Big]_{s,s}= \E\Bigg[
    \attn_{{\ans,1} \rightarrow \pred,2} \cdot \bigg(\sum_{j \in[d]} \Ecal_{5,j}(\Zb^{2,1})\sum_{r\in [m]}\ReLU^{\prime}\big(\Lambda_{5, j,r}\big)\cdot  \\
    &~~~~~~~~~~~~~~~~~~~~~~~~~~~~\Big( \dbrack{\Wb_{5,j,r},\Z_{\pred,2}}- \Lambda_{5, j,r}+b_{5,j,r}\Big)\bigg) \1_{s=\tau(x_1)}\Bigg]\\
     &= \E\Bigg[
    \attn_{{\ans,1} \rightarrow \pred,2} \cdot \bigg( (1-\logit_{5,j_2})\cdot \sum_{r\in[m]}\ReLU^{\prime}(\Lambda_{5,j_2, r
    }) \Big( V_{j_2, r}(g_2)- \Lambda_{5,j_2,r}\pm\tilde{O}(\sigma_0) \Big)\\   
    &~~~~~~~~~~~~-\sum_{j\neq j_2\in\tau(\cY)}\logit_{5,j} \cdot \sum_{r\in[m]}\ReLU^{\prime}(\Lambda_{5,j, r
    }) \Big( V_{j, r}(g_2)- \Lambda_{5,j,r}\pm\tilde{O}(\sigma_0) \Big)\\
    &~~~~~~~~~~~~-\sum_{j\notin\tau(\cY)}\logit_{5,j} \cdot \sum_{r\in[m]}\ReLU^{\prime}(\Lambda_{5,j, r
    }) \Big( \langle\Wb_{5,j,r},\Zb_{\pred,2}\rangle- \Lambda_{5,j,r}\pm\tilde{O}(\sigma_0) \Big)
    \bigg)\1_{\tau(x_1)=s}\Bigg]\\
      &= \E\Bigg[
    \attn_{{\ans,1} \rightarrow \pred,2} \cdot\\
    &~~~~~~~~~~~~~ \bigg( (1-\logit_{5,j_2})\cdot \Big(\sum_{r\in\hat{\fA}_{j_2}}\ReLU^{\prime}(\Lambda_{5,j_2, r
    })\cdot \Big( V_{j_2,  r}(g_2)- \Lambda_{5,j_2, r}\pm\tilde{O}(\sigma_0) \Big)\pm \tilde{O}(\delta^{q}) \Big)\\   
    &~~~~~~~~~~-\sum_{j\neq j_2\in\tau(\cY)}\logit_{5,j} \cdot \Big(\sum_{r\in\hat{\fA}_{j}}\ReLU^{\prime}(\Lambda_{5,j,r
    })\cdot  \Big( V_{j, r}(g_2)- \Lambda_{5,j,r}\pm\tilde{O}(\sigma_0) \Big)\pm \tilde{O}(\delta^{q}) \Big) \\
    &~~~~~~~~~~~~\pm\sum_{j\notin\tau(\cY)}\logit_{5,j}\tilde{O}(\sigma^{q}_0)  \bigg)\1_{\tau(x_1)=s}\Bigg],
\end{align*}
where the last equality follows from \Cref{lem-lambda-char} and \Cref{lem-non-activated-neuron}.
   \item for the non-diagonal entry \( [\Q_{4,3}]_{s,s'} \) with \( s\not=s' \in \tau(\X) \), the analysis is similar unless the condition that the gradient is non-zero only when \( s = \tau(x_0) \) and \( s' = \tau(x_1) \). Thus, we have
  \begin{align*}
    &\Big[-\nabla_{\Q_{4,3}}\Loss^{2,2}_{5}\Big]_{s,s'}= \E\Bigg[
    \attn_{{\ans,1} \rightarrow \pred,1} \cdot \bigg(\sum_{j \in[d]} \Ecal_{5,j}(\Zb^{2,1})\sum_{r\in [m]}\ReLU^{\prime}\big(\Lambda_{5, j,r}\big)\cdot  \\
    &~~~~~~~~~~~~~~~~~~~~~~~~~~~~\Big( \dbrack{\Wb_{5,j,r},\Z_{\pred,1}}- \Lambda_{5, j,r}+b_{5,j,r}\Big)\bigg) \1_{s=\tau(x_1), s'=\tau(x_0)}\Bigg]\\
      &= \E\Bigg[
    \attn_{{\ans,1} \rightarrow \pred,1} \cdot\\
    &~~~~~~~~~~~~~ \bigg( (1-\logit_{5,j_2})\cdot \Big(\sum_{r\in\hat{\fA}_{j_2}}\ReLU^{\prime}(\Lambda_{5,j_2, r
    })\cdot \Big( V_{j_2,  r}(g_1)- \Lambda_{5,j_2, r}\pm\tilde{O}(\sigma_0) \Big)\pm \tilde{O}(\delta^{q}) \Big)\\   
    &~~~~~~~~~~-\sum_{j\neq j_2\in\tau(\cY)}\logit_{5,j} \cdot \Big(\sum_{r\in\hat{\fA}_{j}}\ReLU^{\prime}(\Lambda_{5,j,r
    })\cdot  \Big( V_{j, r}(g_1)- \Lambda_{5,j,r}\pm\tilde{O}(\sigma_0) \Big)\pm \tilde{O}(\delta^{q}) \Big) \\
    &~~~~~~~~~~~~\pm\sum_{j\notin\tau(\cY)}\logit_{5,j}\tilde{O}(\sigma^{q}_0)  \bigg)\1_{\tau(x_1)=s,\tau(x_0)=s'}\Bigg].
\end{align*}
\end{itemize}
Therefore, we complete the proof.
\end{proof}
\begin{lemma}[Refined expression for the gradient of $\Qb_{4,4}$] \label{lem-refined-grad-Q44} Given $s\in\tau(\X)$, for the diagonal entry \( [\Q_{4,4}]_{s,s} \) of the block \(\Qb_{4,4}\), letting $j_1=\tau(g_1(y_0))$ and $j_2=\tau(g_2(y_1))$, we have 
      \begin{align*}
    &\Big[-\nabla_{\Q_{4,4}}\Loss^{2,1}_{5}\Big]_{s,s}= \E\Bigg[
    \attn_{{\ans,0} \rightarrow \ans,0} \cdot\\
    &~~~~~~~~~~~~~ \bigg( (1-\logit_{5,j_1})\cdot \Big(\sum_{r\in\hat{\fA}_{j_1}}\ReLU^{\prime}(\Lambda_{5,j_1, r
    })
    \cdot \Big( V_{j_1,  r}(y_0)- \Lambda_{5,j_1, r}\pm\tilde{O}(\sigma_0) \Big)\pm \tilde{O}(\delta^{q}) \Big)\\   
    &~~~~~~~~~~-\sum_{j\neq j_1\in\tau(\cY)}\logit_{5,j} \cdot \Big(\sum_{r\in\hat{\fA}_{j}}\ReLU^{\prime}(\Lambda_{5,j,r
    })\cdot  \Big( V_{j, r}(y_0)- \Lambda_{5,j,r}\pm\tilde{O}(\sigma_0) \Big)\pm \tilde{O}(\delta^{q}) \Big) \\
    &~~~~~~~~~~~~\pm\sum_{j\notin\tau(\cY)}\logit_{5,j}\tilde{O}(\sigma^{q}_0)  
    \bigg)\1_{\tau(x_0)=s}\Bigg].
\end{align*}
  \begin{align*}
    &\Big[-\nabla_{\Q_{4,4}}\Loss^{2,2}_{5}\Big]_{s,s}= 
      \E\Bigg[
    \attn_{{\ans,1} \rightarrow \ans,1} \cdot \\
    &~~~~~~~~~~~~~ \bigg( (1-\logit_{5,j_2})\cdot \Big(\sum_{r\in\hat{\fA}_{j_2}}\ReLU^{\prime}(\Lambda_{5,j_2, r
    })\cdot \Big( V_{j_2,  r}(y_1)- \Lambda_{5,j_2, r}\pm\tilde{O}(\sigma_0) \Big)\pm \tilde{O}(\delta^{q}) \Big)\\   
    &~~~~~~~~~~-\sum_{j\neq j_2\in\tau(\cY)}\logit_{5,j} \cdot \Big(\sum_{r\in\hat{\fA}_{j}}\ReLU^{\prime}(\Lambda_{5,j,r
    })\cdot  \Big( V_{j, r}(y_1)- \Lambda_{5,j,r}\pm\tilde{O}(\sigma_0) \Big)\pm \tilde{O}(\delta^{q}) \Big) \\
    &~~~~~~~~~~~~\pm\sum_{j\notin\tau(\cY)}\logit_{5,j}\tilde{O}(\sigma^{q}_0)  \bigg)\1_{\tau(x_1)=s}\Bigg].
\end{align*}
Moreover, for the off-diagonal entries \( [\Q_{4,4}]_{s,s'} \) with \( s \neq s' \), we have $[-\nabla_{\Q_{4,4}}\Loss^{2,1}_{5}]_{s,s'}=0$, and
  \begin{align*}
    &\Big[-\nabla_{\Q_{4,4}}\Loss^{2,2}_{5}\Big]_{s,s'}=  \E\Bigg[
    \attn_{{\ans,1} \rightarrow \ans,0} \cdot\\
    &~~~~~~~~~~~~~ \bigg( (1-\logit_{5,j_2})\cdot \Big(\sum_{r\in\hat{\fA}_{j_2}}\ReLU^{\prime}(\Lambda_{5,j_2, r
    })\cdot \Big( V_{j_2,  r}(y_0)- \Lambda_{5,j_2, r}\pm\tilde{O}(\sigma_0) \Big)\pm \tilde{O}(\delta^{q}) \Big)\\   
    &~~~~~~~~~~-\sum_{j\neq j_2\in\tau(\cY)}\logit_{5,j} \cdot \Big(\sum_{r\in\hat{\fA}_{j}}\ReLU^{\prime}(\Lambda_{5,j,r
    })\cdot  \Big( V_{j, r}(y_0)- \Lambda_{5,j,r}\pm\tilde{O}(\sigma_0) \Big)\pm \tilde{O}(\delta^{q}) \Big) \\
    &~~~~~~~~~~~~\pm\sum_{j\notin\tau(\cY)}\logit_{5,j}\tilde{O}(\sigma^{q}_0)  \bigg)\1_{\tau(x_1)=s,\tau(x_0)=s'}\Bigg].
\end{align*}
\end{lemma}
The proof is similar to \Cref{lem-refined-grad-Q43} and we omit it here.

\paragraph{Notations for gradient decompositions.} We shall define some useful notations to further simplify the expressions of the gradient.
  \begin{itemize}[left=0pt]
    \item for $\ell=1$,  
    \begin{itemize}[left=0pt]
        \item for $[\Q_{4,3}]_{s,s}$ with $s\in\tau(\X)$, we have $[-\nabla_{\Qb_{4,3}}\Loss_5^{2,1}]_{s,s}=\cN_{s,3,1,\rom1}+\cN_{s,3,1,\rom2}+\cN_{s,3,1,\rom3}$, where 
      \begin{align}\label{eq-def-N-s-3-1-1}
       \cN_{s,3,1,\rom1}&= \E\Bigg[
    \attn_{{\ans,0} \rightarrow \pred,1} \cdot (1-\logit_{5,j_1})\cdot \\
    &~~~~~~~~~  \Big(\sum_{r\in\hat{\fA}_{j_1}}\ReLU^{\prime}(\Lambda_{5,j_1, r
    })
    \cdot \Big( V_{j_1,  r}(g_1)- \Lambda_{5,j_1, r}\pm\tilde{O}(\sigma_0) \Big)\pm \tilde{O}(\delta^{q}) \Big) 
    \1_{\tau(x_0)=s}\Bigg],\notag
\end{align}
      \begin{align}\label{eq-def-N-s-3-1-2}
       \cN_{s,3,1,\rom2}&= -\E\Bigg[
    \attn_{{\ans,0} \rightarrow \pred,1} \cdot\sum_{j\neq j_1\in\tau(\cY)}\logit_{5,j} \cdot \\
    &~~~~~~~~~~\Big(\sum_{r\in\hat{\fA}_{j}}\ReLU^{\prime}(\Lambda_{5,j,r
    })\cdot  \Big( V_{j, r}(g_1)- \Lambda_{5,j,r}\pm\tilde{O}(\sigma_0) \Big)\pm \tilde{O}(\delta^{q}) \Big) 
\1_{\tau(x_0)=s}\Bigg],\notag
\end{align}
      \begin{align}\label{eq-def-N-s-3-1-3}
       \cN_{s,3,1,\rom3}&= \pm\E\Bigg[
    \attn_{{\ans,0} \rightarrow \pred,1} \cdot\sum_{j\notin\tau(\cY)}\logit_{5,j}\tilde{O}(\sigma^{q}_0)\1_{\tau(x_0)=s}\Bigg].~~~~~~~~~~~~~~~~~~~~~~~~~~~~~~
\end{align}
\item for $[\Q_{4,4}]_{s,s}$ with $s\in\tau(\X)$, we have $[-\nabla_{\Qb_{4,4}}\Loss_5^{2,1}]_{s,s}=\cN_{s,4,1,\rom1}+\cN_{s,4,1,\rom2}+\cN_{s,4,1,\rom3}$, where 
      \begin{align}\label{eq-def-N-s-4-1-1}
       \cN_{s,4,1,\rom1}&= \E\Bigg[
    \attn_{{\ans,0} \rightarrow \ans,0} \cdot (1-\logit_{5,j_1})\cdot \\
    &~~~~~~~~~  \Big(\sum_{r\in\hat{\fA}_{j_1}}\ReLU^{\prime}(\Lambda_{5,j_1, r
    })
    \cdot \Big( V_{j_1,  r}(y_0)- \Lambda_{5,j_1, r}\pm\tilde{O}(\sigma_0) \Big)\pm \tilde{O}(\delta^{q}) \Big) 
    \1_{\tau(x_0)=s}\Bigg],\notag
\end{align}
      \begin{align}\label{eq-def-N-s-4-1-2}
       \cN_{s,4,1,\rom2}&= -\E\Bigg[
    \attn_{{\ans,0} \rightarrow \ans,0} \cdot\sum_{j\neq j_1\in\tau(\cY)}\logit_{5,j} \cdot \\
    &~~~~~~~~~~\Big(\sum_{r\in\hat{\fA}_{j}}\ReLU^{\prime}(\Lambda_{5,j,r
    })\cdot  \Big( V_{j, r}(y_0)- \Lambda_{5,j,r}\pm\tilde{O}(\sigma_0) \Big)\pm \tilde{O}(\delta^{q}) \Big) 
\1_{\tau(x_0)=s}\Bigg],\notag
\end{align}
      \begin{align}\label{eq-def-N-s-4-1-3}
       \cN_{s,4,1,\rom3}&= \pm\E\Bigg[
    \attn_{{\ans,0} \rightarrow \ans,0} \cdot\sum_{j\notin\tau(\cY)}\logit_{5,j}\tilde{O}(\sigma^{q}_0)\1_{\tau(x_0)=s}\Bigg].~~~~~~~~~~~~~~~~~~~~~~~~~~~~~~
\end{align}
    \end{itemize}
    \item for $\ell=2$,  
    \begin{itemize}[left=0pt]
        \item for $[\Q_{4,3}]_{s,s}$ with $s\in\tau(\X)$, we have $[-\nabla_{\Qb_{4,3}}\Loss_5^{2,2}]_{s,s}=\cN_{s,3,2,\rom1}+\cN_{s,3,2,\rom2}+\cN_{s,3,2,\rom3}$, where 
      \begin{align}
       \cN_{s,3,2,\rom1}&= \E\Bigg[
    \attn_{{\ans,1} \rightarrow \pred,2} \cdot (1-\logit_{5,j_2})\cdot \label{eq-def-N-s-3-2-1} \\
    &~~~~~~~~~  \Big(\sum_{r\in\hat{\fA}_{j_2}}\ReLU^{\prime}(\Lambda_{5,j_2, r
    })
    \cdot \Big( V_{j_2,  r}(g_2)- \Lambda_{5,j_2, r}\pm\tilde{O}(\sigma_0) \Big)\pm \tilde{O}(\delta^{q}) \Big) 
    \1_{\tau(x_1)=s}\Bigg],\notag
\end{align}
      \begin{align}\label{eq-def-N-s-3-2-2}
       \cN_{s,3,2,\rom2}&= -\E\Bigg[
    \attn_{{\ans,1} \rightarrow \pred,2} \cdot\sum_{j\neq j_2\in\tau(\cY)}\logit_{5,j} \cdot \\
    &~~~~~~~~~~\Big(\sum_{r\in\hat{\fA}_{j}}\ReLU^{\prime}(\Lambda_{5,j,r
    })\cdot  \Big( V_{j, r}(g_2)- \Lambda_{5,j,r}\pm\tilde{O}(\sigma_0) \Big)\pm \tilde{O}(\delta^{q}) \Big) 
\1_{\tau(x_1)=s}\Bigg],\notag
\end{align}
      \begin{align}\label{eq-def-N-s-3-2-3}
       \cN_{s,3,2,\rom3}&= \pm\E\Bigg[
    \attn_{{\ans,1} \rightarrow \pred,2} \cdot\sum_{j\notin\tau(\cY)}\logit_{5,j}\tilde{O}(\sigma^{q}_0)\1_{\tau(x_1)=s}\Bigg].~~~~~~~~~~~~~~~~~~~~~~~~~~~~~~
\end{align}
    \item for $[\Q_{4,4}]_{s,s}$ with $s\in\tau(\X)$, we have $[-\nabla_{\Qb_{4,4}}\Loss_5^{2,2}]_{s,s}=\cN_{s,4,2,\rom1}+\cN_{s,4,2,\rom2}+\cN_{s,4,2,\rom3}$, where 
      \begin{align}\label{eq-def-N-s-4-2-1}
       \cN_{s,4,2,\rom1}&= \E\Bigg[
    \attn_{{\ans,1} \rightarrow \ans,1} \cdot (1-\logit_{5,j_2})\cdot \\
    &~~~~~~~~~  \Big(\sum_{r\in\hat{\fA}_{j_2}}\ReLU^{\prime}(\Lambda_{5,j_2, r
    })
    \cdot \Big( V_{j_2,  r}(y_1)- \Lambda_{5,j_2, r}\pm\tilde{O}(\sigma_0) \Big)\pm \tilde{O}(\delta^{q}) \Big) 
    \1_{\tau(x_1)=s}\Bigg],\notag
\end{align}
      \begin{align}\label{eq-def-N-s-4-2-2}
       \cN_{s,4,2,\rom2}&= -\E\Bigg[
    \attn_{{\ans,1} \rightarrow \ans,1} \cdot\sum_{j\neq j_2\in\tau(\cY)}\logit_{5,j} \cdot \\
    &~~~~~~~~~~\Big(\sum_{r\in\hat{\fA}_{j}}\ReLU^{\prime}(\Lambda_{5,j,r
    })\cdot  \Big( V_{j, r}(y_1)- \Lambda_{5,j,r}\pm\tilde{O}(\sigma_0) \Big)\pm \tilde{O}(\delta^{q}) \Big) 
\1_{\tau(x_1)=s}\Bigg],\notag
\end{align}
      \begin{align}\label{eq-def-N-s-4-2-3}
       \cN_{s,4,2,\rom3}&= \pm\E\Bigg[
    \attn_{{\ans,1} \rightarrow \ans,1} \cdot\sum_{j\notin\tau(\cY)}\logit_{5,j}\tilde{O}(\sigma^{q}_0)\1_{\tau(x_1)=s}\Bigg].~~~~~~~~~~~~~~~~~~~~~~~~~~~~~~
\end{align}
    \end{itemize}
\end{itemize}

\paragraph{Probabilistic Events.}  
We conclude this subsection by introducing several probabilistic events that will be used to simplify the characterization of activated neurons in the subsequent analysis.
\begin{align}
    \cE_{1} &\triangleq \Big\{g_1\not=g_2\Big\}, \label{eq-cyc-event-1}\\
    \cE_{2} &\triangleq \Big\{g_{\ell_1}(y_{\ell'_{1}})\not=g_{\ell_2}(y_{\ell'_{2}}), \text{ for any } (\ell_1,\ell'_1)\not=(\ell_2,\ell'_2), \text{ where } \ell_{k}\in [2], \ell'_k\in\{0,1\} \Big\}. \label{eq-cyc-event-2}
\end{align}
It is easy to see that $\cE_{1}$ and $\cE_{2}$ hold with high probability $1-O\big(\frac{1}{\log d}\big)$.

\subsection{Stage 2.1: Initial Growth of Gap}\label{sec-s22-cyc}

At the beginning of stage 2, since $\Qb$ has not been trained for long, we have the attention score is still close to the uniform structure. Therefore, for $\ell=1$, we have 
        \begin{align*}
            \Lambda_{5,j,r}(\Zb^{2,\ell-1})\approx \frac{1}{3}V_{j,r}(g_{1})+ \frac{1}{3}V_{j,r}(g_{2})+ \frac{1}{3}V_{j,r}(y_{0})\pm  \tilde{O}(\sigma_0). 
        \end{align*}
        If $\Zb^{2,\ell-1}\in \cE_1$,
        \begin{itemize}
            \item for $j=j_1$, for $r\in \hat{\fA}_{j_1}$, only $r_{g_1\cdot y_0}$ is activated since $\Lambda_{5,j_1,r_{g_1\cdot y_0}}\approx \frac{1}{3}B$ and $\Lambda_{5,j_1,r_{g_2\cdot g_2^{-1}(j_1)}}\approx -\frac{1}{3}B\ll -\varrho$; 
            \item for $j=j'_1\triangleq \tau\big(g_2(y_0)\big)$, only $r_{g_2\cdot y_0}$ is activated since $\Lambda_{5,j'_1,r_{g_2\cdot y_0}}\approx \frac{1}{3}B$ and $\Lambda_{5,j'_1,r_{g_1\cdot g_2^{-1}(j'_1)}}\approx -\frac{1}{3}B\ll -\varrho$; 
            \item for other $j\in\tau(\Y)$, we have $\Lambda_{5,j,r}\leq -\frac{1}{3}B$ for all $r\in \hat{\fA}_{j}$, thus no activation.
        \end{itemize}
        Moreover, for $\ell=2$, we have  
             \begin{align*}
            \Lambda_{5,j,r}(\Zb^{2,\ell-1})\approx \frac{1}{4}V_{j,r}(g_{1})+ \frac{1}{4}V_{j,r}(g_{2})+ \frac{1}{4}V_{j,r}(y_{0})+\frac{1}{4}V_{j,r}(y_{1})\pm  \tilde{O}(\sigma_0). 
        \end{align*}
        If $\Zb^{2,\ell-1}\in \cE_2$,
                \begin{itemize}
            \item for $j\in \{\tau(g_{\ell}(y_{\ell'}))\}_{\ell\in[2],\ell'\in\{0,1\}}$, only the corresponding $r_{g_\ell\cdot y_{\ell'}}$ is activated in the smoothed regime since $|\Lambda_{5,j,r_{g_\ell\cdot y_{\ell'}}}|=O(\delta)$ and $\Lambda_{5,j,r}\approx -\frac{1}{2}B\ll -\varrho$ for $r\in \hat{\fA}_{j}\setminus\{r_{g_\ell\cdot y_{\ell'}}\}$;
            \item for other $j\in\tau(\Y)$, we have $\Lambda_{5,j,r}\leq -\frac{1}{2}B$ for all $r\in \hat{\fA}_{j}$, thus no activation.
        \end{itemize}

Here, activation means that the corresponding $\ReLU^{\prime}(\Lambda_{5,j,r})$ is non-zero, which is crucial for the gradient computation.
        Based on the above observations, we can see that the gradient from $\ell=2$ is relatively small since $\Lambda$ is only activated in the smoothed regime. Thus, initially, the learning process is dominated by $\nabla_{\Qb}\Loss^{2,1}_5$. Moreover, if we take a closer look at the gradient from $\ell=1$, we have
   \begin{align*}
       \cN_{s,3,1,\rom2}&\approx -\E\Bigg[
    \attn_{{\ans,0} \rightarrow \pred,1} \cdot\logit_{5,j'_1} \cdot \\
    &\Big(\ReLU^{\prime}(\Lambda_{5,j'_1,r_{g_2\cdot y_0}
    })\cdot  \Big( V_{j'_1, r_{g_2\cdot y_0}}(g_1)- \Lambda_{5,j'_1,r_{g_2\cdot y_0}}\pm\tilde{O}(\sigma_0) \Big)\pm \tilde{O}(\delta^{q}) \Big) 
\1_{\tau(x_0)=s}\Bigg]\geq 0,\\
 \cN_{s,4,1,\rom2}&\approx -\E\Bigg[
    \attn_{{\ans,0} \rightarrow \ans,0} \cdot\logit_{5,j'_1} \cdot \\
    &\Big(\ReLU^{\prime}(\Lambda_{5,j'_1,r_{g_2\cdot y_0}
    })\cdot  \Big( V_{j'_1, r_{g_2\cdot y_0}}(y_0)- \Lambda_{5,j'_1,r_{g_2\cdot y_0}}\pm\tilde{O}(\sigma_0) \Big)\pm \tilde{O}(\delta^{q}) \Big) \leq 0,
\end{align*}
since $V_{j'_1, r_{g_2\cdot y_0}}(y_0)\geq \Omega(B)$ while $V_{j'_1, r_{g_2\cdot y_0}}(g_1)\leq -\Omega(B)$. Thus, $[\Qb_{4,3}]_{s,s}$ will have a significant positive gradient while $[\Qb_{4,4}]_{s,s}$ will have a negative counterpart. This will lead to the growth of the gap between $[\Qb_{4,3}]_{s,s}$ and $[\Qb_{4,4}]_{s,s}$.

We formally characterize this growth behavior within this substage. At the beginning of each sub-stage, we establish an induction hypothesis that we expect to hold throughout. Subsequently, we analyze the dynamics under this hypothesis within the substage, aiming to prove its validity by the end of sub-stage. Due to the symmetry of $[\Qb_{4,3}]_{s,s}$ and $[\Qb_{4,4}]_{s,s}$ across $s\in\tau(\X)$, we may, without loss of generality, focus on a particular $s \in \tau(\X)$.

\begin{induction}\label{induction-s21-cyc}
  Given $s\in\tau(\X)$,  let $T_{2,1,s}$ denote the first time that $[\Qb^{(t)}_{4,3}]_{s,s}$ reaches $\Omega\Big(\frac{1}{\log d}\Big)$.
     For all iterations $t< T_{2,1,s}$, we have the following holds
     \begin{enumerate}[(a)]
        \item $[\Qb^{(t)}_{4,3}]_{s,s}$ monotonically increases; 
        \item $\big|[\Qb^{(t)}_{4,4}]_{s,s}\big|\leq [\Qb^{(t)}_{4,3}]_{s,s}$ and $[\Qb^{(t)}_{4,3}]_{s,s}-[\Qb^{(t)}_{4,4}]_{s,s}= \Theta\Big([\Qb^{(t)}_{4,3}]_{s,s}\Big)$;
        \item for $p\in\{3,4\}$, for $s'\in\tau(\X)\not=s$, $|[\Qb^{(t)}_{4,p}]_{s,s'}|\leq O\Big(\frac{[\Qb^{(t)}_{4,p}]_{s,s}}{d}\Big)$. 
     \end{enumerate}
  \end{induction}
  \subsubsection{Attention and Logit Preliminaries}\label{sec-s21-attn-cyc}
  We first  introduce several properties of the attention scores and logits if \Cref{induction-s21-cyc} holds.
\begin{lemma}\label{lem-s21-attn-cyc}
    If \Cref{induction-s21-cyc} holds for all iterations $<t$, given input $\Zb^{2,\ell-1}$, then we have 
    \begin{enumerate}
        \item for $\ell=1$, 
        \begin{enumerate}
        \item $\attn^{(t)}_{\ans,0\to \pred, 1}\in \Big[\frac{1}{3}, \frac{1}{3}+O\big(\frac{1}{\log d}\big)\Big]$; 
         \item  $\attn^{(t)}_{\ans,0\to \pred,2}, \attn^{(t)}_{\ans,0\to \ans,0}\leq \attn^{(t)}_{\ans,0\to \pred,1}$;
         \item   $\big|\attn^{(t)}_{\ans,0\to \kk} -\attn^{(t)}_{\ans,0\to \kk^{\prime}}\big|\leq O(\frac{1}{\log d})$ for $\kk\not=\kk^{\prime}\in\cI^{(2,0)}$.
        \end{enumerate} 
         \item for $\ell=2$,  
         \begin{enumerate}
            \item $\attn^{(t)}_{\ans,1\to \pred, 2}\in \Big[\frac{1}{4}, \frac{1}{4}+O\big(\frac{1}{\log 
             d}\big)\Big]$; 
             \item  $\attn^{(t)}_{\ans,1\to \kk} \leq \attn^{(t)}_{\ans,1\to \pred,2}$ for $\kk\not= (\pred,2)$;
             \item   for $\kk\in \big\{(\ans,0), (\pred,1)\big\}$, $\kk^{\prime}\in \big\{(\ans,1), (\pred,2)\big\}$, 
             $$\big|\attn^{(t)}_{\ans,1\to \kk} -\attn^{(t)}_{\ans,1\to \kk^{\prime}}\big|\leq O\Big(\frac{1}{\log d}\Big);$$
                      \item $\big|\attn^{(t)}_{\ans,1\to \pred,1} -\attn^{(t)}_{\ans,1\to \ans, 0}\big|\leq O\big(\frac{1}{d}\big)$.  
            \end{enumerate} 
    \end{enumerate}
\end{lemma}
\begin{proof}
    For $\ell=1$, given $\Zb^{2,\ell-1}$, according to \Cref{assump-Q-structure}, 
we have 
    \begin{align*}
        \attn^{(t)}_{\ans,0\to \pred,1}&=\frac{e^{[\Qb^{(t)}_{4,3}]_{\tau(x_0),\tau(x_0)}}}{e^{[\Qb^{(t)}_{4,3}]_{\tau(x_0),\tau(x_0)}}+e^{[\Qb^{(t)}_{4,3}]_{\tau(x_0),\tau(x_1)}}+e^{[\Qb^{(t)}_{4,4}]_{\tau(x_0),\tau(x_0)}}}\\
        &=\frac{1}{1+e^{[\Qb^{(t)}_{4,3}]_{\tau(x_0),\tau(x_1)}-[\Qb^{(t)}_{4,3}]_{\tau(x_0),\tau(x_0)}}+e^{[\Qb^{(t)}_{4,4}]_{\tau(x_0),\tau(x_0)}-[\Qb^{(t)}_{4,3}]_{\tau(x_0),\tau(x_0)}}}.
    \end{align*}
    Thus, by \Cref{induction-s21-cyc}, $$-O\Big(\frac{1}{\log d}\Big)\leq [\Qb^{(t)}_{4,3}]_{\tau(x_0),\tau(x_1)}-[\Qb^{(t)}_{4,3}]_{\tau(x_0),\tau(x_0)}, [\Qb^{(t)}_{4,4}]_{\tau(x_0),\tau(x_0)}-[\Qb^{(t)}_{4,3}]_{\tau(x_0),\tau(x_0)}\leq 0, $$
    which implies that  $0 \leq \attn^{(t)}_{\ans,0\to \pred,1}\leq \frac{1}{3}+O\big(\frac{1}{\log d}\big)$. (b) is straightforward since  $$[\Qb^{(t)}_{4,3}]_{\tau(x_0),\tau(x_1)}, [\Qb^{(t)}_{4,4}]_{\tau(x_0),\tau(x_0)}\leq [\Qb^{(t)}_{4,3}]_{\tau(x_0),\tau(x_0)}.$$
(c) is a direct consequence of (a) and (b).

For $\ell=2$, given $\Zb^{2,\ell-1}$, (a)- (c) are very similar to the above analysis, and then for (d),  we have
    \begin{align*}
        &\attn^{(t)}_{\ans,1\to \pred,1}-\attn^{(t)}_{\ans,1\to \ans,0}\\&~~~=\frac{e^{[\Qb^{(t)}_{4,3}]_{\tau(x_1),\tau(x_0)}}-e^{[\Qb^{(t)}_{4,4}]_{\tau(x_1),\tau(x_0)}}}{e^{[\Qb^{(t)}_{4,3}]_{\tau(x_1),\tau(x_0)}}+e^{[\Qb^{(t)}_{4,3}]_{\tau(x_1),\tau(x_1)}}+e^{[\Qb^{(t)}_{4,4}]_{\tau(x_1),\tau(x_0)}}+e^{[\Qb^{(t)}_{4,4}]_{\tau(x_1),\tau(x_1)}}}\\
        &~~~\stackrel{(i)}{\leq} O\Big([\Qb^{(t)}_{4,3}]_{\tau(x_1),\tau(x_0)}-[\Qb^{(t)}_{4,4}]_{\tau(x_1),\tau(x_0)}\Big)\stackrel{(ii)}{\leq} 
  O(\frac{1}{d}),
    \end{align*}
    where (i) is due to the fact that $|e^x-e^y|\leq O(|x-y|)$ when $x,y$ are small, and (ii) is due to \Cref{induction-s21-cyc} (b).
\end{proof}
\begin{lemma}\label{lem-s21-act-cyc}
    If \Cref{induction-s21-cyc} holds for all iterations $<t$, given input $\Zb^{2,\ell-1}$, then we have
    \begin{enumerate}
        \item for $\ell=1$, if $\Zb^{2,\ell-1}\in \cE_1$, then
  \begin{enumerate}[(a)]
            \item for $j=j_1$, 
       $\Lambda^{(t)}_{5,j_1,r}\ll -\varrho$  for  $r\in \hat{\fA}_{j_1}\setminus\{r_{g_1\cdot y_0}\}$; 
            \item for $j=j'_1\triangleq \tau\big(g_2(y_0)\big)$, 
            $\Lambda^{(t)}_{5,j'_1,r}\ll -\varrho$ for  $r\in \hat{\fA}_{j'_1}\setminus\{r_{g_2\cdot y_0}\}$;
            \item for other $j\in\tau(\Y)$, $r$ is not activated for all $r\in \hat{\fA}_{j}$, i.e., $\Lambda^{(t)}_{5,j,r}\ll -\varrho$.
        \end{enumerate}

       \item $\ell=2$, if  $\Zb^{2,\ell-1}\in \cE_2$, then
                \begin{enumerate}[(a)]
            \item for $j\in \{\tau(g_{\ell}(y_{\ell'}))\}_{\ell\in[2],\ell'\in\{0,1\}}$, only the corresponding $r_{g_\ell\cdot y_{\ell'}}$ may be activated,  with $|\Lambda^{(t)}_{5,j,r_{g_\ell\cdot y_{\ell'}}}|\leq O(1)$, while  all other $\Lambda^{(t)}_{5,j,r}\ll -\varrho$ for $r\in \hat{\fA}_{j}\setminus\{r_{g_\ell\cdot y_{\ell'}}\}$;
            \item for other $j\in\tau(\Y)$, $r$ is not activated for all $r\in \hat{\fA}_{j}$, i.e., $\Lambda^{(t)}_{5,j,r}\ll -\varrho$.
        \end{enumerate}
    \end{enumerate}
        
\end{lemma}
\begin{proof}
    We only prove (a) for $\ell=2$ since the other cases are straightforward. 
$j=\tau(g_{\ell}(y_{\ell'}))$, by  \Cref{lem-lambda-char}
    we have
    \begin{align*}
            \Lambda_{5,j,r_{g_{\ell}\cdot y_{\ell'}}}&=  \attn^{(t)}_{\ans,1\to\pred,1}V_{j,r_{g_{\ell}\cdot y_{\ell'}}}(g_{1})+ \attn^{(t)}_{\ans,1\to\pred,2}V_{j,r_{g_{\ell}\cdot y_{\ell'}}}(g_{2})\\
            &+ \attn^{(t)}_{\ans,1\to\ans,0}V_{j,r_{g_{\ell}\cdot y_{\ell'}}}(y_{0})+ \attn^{(t)}_{\ans,1\to\ans,1}V_{j,r_{g_{\ell}\cdot y_{\ell'}}}(y_{1})\pm  \tilde{O}(\sigma_0). 
        \end{align*}
        Notice that since $\Zb^{2,\ell-1}\in \cE_2$, we have two $V$ terms are positive and two are negative with magnitude $B\pm O(\delta)$. Therefore, $|\Lambda_{5,j,r_{g_{\ell}\cdot y_{\ell'}}}|\leq O\Big(\frac{1}{\log d}\Big)\cdot B=O(1)$. 
\end{proof}
\begin{lemma}\label{lem-s21-lambda-cyc}
    If \Cref{induction-s21-cyc} holds for all iterations $<t$, given input $\Zb^{2,0}\in\cE_1$, then we have
      \begin{enumerate}[(a)]
            \item $\Lambda^{(t)}_{5,j_1,r_{g_1\cdot y_0}}\geq \frac{1}{3}B-O(1)$; 
            \item  $-O(\delta)\leq\Lambda^{(t)}_{5,j_1,r_{g_1\cdot y_0}}-\Lambda^{(t)}_{5,j'_1,r_{g_2\cdot y_0}}\leq O(1)$; 
        \end{enumerate}
\end{lemma}
\begin{proof}
    By \Cref{lem-lambda-char}
    we have
    \begin{align*}
            &\Lambda_{5,j_1,r_{g_1\cdot y_0}}=  \\
            &\attn^{(t)}_{\ans,0\to\pred,1}V_{j_1,r_{g_1\cdot y_0}}(g_{1})+ \attn^{(t)}_{\ans,0\to\pred,2}V_{j_1,r_{g_1\cdot y_0}}(g_{2})+ \attn^{(t)}_{\ans,0\to\ans,0}V_{j,r_{g_1\cdot y_0}}(y_{0})\pm  \tilde{O}(\sigma_0). 
        \end{align*}
        By \Cref{lem-s21-attn-cyc} and the cancellation in \eqref{attn-init-prop-cyc-3}, we have 
        \begin{align*}
            &\attn^{(t)}_{\ans,0\to\pred,2}V_{j_1,r_{g_1\cdot y_0}}(g_{2})+ \attn^{(t)}_{\ans,0\to\ans,0}V_{j_1,r_{g_1\cdot y_0}}(y_{0})\\
            &~~~\geq \Big(\frac{1}{3}- O\Big(\frac{1}{\log d}\Big)\Big)\cdot O(\delta)-O\Big(\frac{1}{\log d}\Big)\cdot B.
        \end{align*}
        Putting it back, and using the fact that $\attn^{(t)}_{\ans,0\to\pred,1}V_{j_1,r_{g_1\cdot y_0}}(g_{1})\geq \frac{1}{3}\cdot(B-O(\delta))$, we have 
        \begin{align*}
            \Lambda_{5,j_1,r_{g_1\cdot y_0}}&\geq   
            \frac{1}{3}\cdot (B-O(\delta))+ \Big(\frac{1}{3}- O\Big(\frac{1}{\log d}\Big)\Big)\cdot O(\delta)-O\Big(\frac{1}{\log d}\Big)\cdot B \pm  \tilde{O}(\sigma_0)\\&\geq \frac{1}{3}B-O(1).
        \end{align*}
        Moving on to (b), we have
        \begin{align*}
            &\Lambda_{5,j_1,r_{g_1\cdot y_0}}-\Lambda_{5,j'_1,r_{g_2\cdot y_0}}
           =\attn^{(t)}_{\ans,0\to\pred,1}(V_{j_1,r_{g_1\cdot y_0}}(g_{1})-V_{j'_1,r_{g_2\cdot y_0}}(g_{1}))\pm  \tilde{O}(\sigma_0)\\&+\attn^{(t)}_{\ans,0\to\pred,2}(V_{j_1,r_{g_1\cdot y_0}}(g_{2})-V_{j'_1,r_{g_2\cdot y_0}}(g_{2}))+\attn^{(t)}_{\ans,0\to\ans,0}(V_{j_1,r_{g_1\cdot y_0}}(y_{0})-V_{j'_1,r_{g_2\cdot y_0}}(y_{0}))\\
           &\leq \attn^{(t)}_{\ans,0\to\pred,1}\cdot (2B+O(\delta))- \attn^{(t)}_{\ans,0\to\pred,2}\cdot (2B-O(\delta))+ \attn^{(t)}_{\ans,0\to\ans,0}\cdot O(\delta)\\
           &\leq (2B-O(\delta))\cdot O\Big(\frac{1}{\log d}\Big)+O(\delta)\leq O(1).
        \end{align*}
        Similarly, we have
        \begin{align*}
            &\Lambda_{5,j_1,r_{g_1\cdot y_0}}-\Lambda_{5,j'_1,r_{g_2\cdot y_0}}\\
           &\geq \attn^{(t)}_{\ans,0\to\pred,1}\cdot (2B-O(\delta))- \attn^{(t)}_{\ans,0\to\pred,2}\cdot (2B+O(\delta))- \attn^{(t)}_{\ans,0\to\ans,0}\cdot O(\delta)\\
           &\geq- \attn^{(t)}_{\ans,0\to\pred,2}\cdot O(\delta)- \attn^{(t)}_{\ans,0\to\ans,0}\cdot O(\delta)\geq -O(\delta).
        \end{align*}
\end{proof}
\begin{lemma}
\label{lem-s21-logit-cyc}
    If \Cref{induction-s21-cyc} holds for all iterations $<t$, given input $\Zb^{2,\ell-1}$, then we have
\begin{enumerate}   
    \item for $\ell=1$, if $\Zb^{2,\ell-1}\in \cE_1$, $\logit^{(t)}_{5,j}=\Omega(1)$ for $j\in\{j_1,j_1^{\prime}\}$, $1-\logit^{(t)}_{5,j_1}-\logit^{(t)}_{5,j_1^{\prime}}=\frac{1}{\poly d}$. 
    \item for $\ell=2$,  if $\Zb^{2,\ell-1}\in \cE_2$, $\logit^{(t)}_{5,j}=O(\frac{1}{d})$ for all $j$. 
\end{enumerate}
\end{lemma}
\begin{proof}
    \begin{itemize}[left=0pt]
        \item For $\ell=1$, by \Cref{lem-s21-act-cyc} and \Cref{lem-lambda-char}, we have 
        \begin{align*}
           & F^{(t)}_{5, j_1}(\Zb^{2,\ell-1})=\sum_{r\in[m]} \ReLU\big(\Lambda^{(t)}_{5,j_1,r}\big)=\Lambda^{(t)}_{5,j_1,r_{g_1\cdot y_0}}+ \varrho(m/q-1)=\Lambda^{(t)}_{5,j_1,r_{g_1\cdot y_0}}+O\Big(\frac{1}{\polylog d}\Big),\\
            & F^{(t)}_{5, j'_1}(\Zb^{2,\ell-1})=\sum_{r\in[m]} \ReLU\big(\Lambda^{(t)}_{5,j'_1,r}\big)=\Lambda^{(t)}_{5,j'_1,r_{g_2\cdot y_0}}+ \varrho(m/q-1),\\
            & F^{(t)}_{5, j}(\Zb^{2,\ell-1})=\sum_{r\in[m]} \ReLU\big(\Lambda^{(t)}_{5,j,r}\big)\leq O\Big(\frac{1}{\polylog d}\Big) \text{ for } j\not=j_1,j'_1\in\tau(\Y),\\
          & F^{(t)}_{5, j}(\Zb^{2,\ell-1})\leq m\cdot \tilde{O}(\sigma_0^{q})  \text{ for } j\notin\tau(\Y).
        \end{align*}
        Putting it together, we obtain
        \begin{align*}
            \logit^{(t)}_{5,j_1}&=\frac{1}{1+e^{F^{(t)}_{5, j'_1}-F^{(t)}_{5, j_1}}+\Big(\sum_{j\not= j_1, j'_1 \in \tau(\Y)}e^{F^{(t)}_{5, j}}+\sum_{j\notin \tau(\Y)}e^{F^{(t)}_{5, j}}\Big)\cdot e^{-F^{(t)}_{5, j_1}}}\\
            &=\frac{1}{1+e^{\Lambda^{(t)}_{5,j'_1,r_{g_2\cdot y_0}}-\Lambda^{(t)}_{5,j_1,r_{g_1\cdot y_0}}}+\Big(O(\log d) \cdot e^{O(\frac{1}{\polylog d})}+O(d)\cdot e^{\tilde{O}(m\sigma_0^q)}\Big)\cdot e^{-F^{(t)}_{5, j_1}}}.
        \end{align*}
        Thus, by \Cref{lem-s21-lambda-cyc}, and the fact that $B=C_B\log d$ for some sufficiently large constant $C_B>0$,  we have $\logit^{(t)}_{5,j_1}=\frac{1}{1+e^{-O(\delta)}+O(1)\cdot e^{-(C_B/3-1)\log d}}=\Omega(1)$. Similarly, we have
        \begin{align*}
            \logit^{(t)}_{5,j'_1}
            &=\frac{1}{1+e^{-\Lambda^{(t)}_{5,j'_1,r_{g_2\cdot y_0}}+\Lambda^{(t)}_{5,j_1,r_{g_1\cdot y_0}}}+\Big(O(\log d) \cdot e^{O(\frac{1}{\polylog d})}+O(d)\cdot e^{\tilde{O}(m\sigma_0^q)}\Big)\cdot e^{-F^{(t)}_{5, j'_1}}}\\
            &=\frac{1}{1+e^{O(1)}+O(1)\cdot e^{-(C_B/3-1)\log d}}=\Omega(1).
        \end{align*}
     From the above analysis,   it is easy to see that $1-\logit^{(t)}_{5,j_1}-\logit^{(t)}_{5,j'_1}\leq O\Big(\frac{1}{e^{(C_B/3-1)\log d}}\Big)=O(\frac{1}{\poly d})$.
   \item For $\ell=2$, by \Cref{lem-s21-act-cyc} and \Cref{lem-lambda-char}, we have 
        \begin{align*}
           & F^{(t)}_{5, j}(\Zb^{2,\ell-1})=\sum_{r\in[m]} \ReLU\big(\Lambda^{(t)}_{5,j_1,r}\big)\in [\varrho m/q, O(1)+ \varrho(m/q-1)] \text{ for }  j\in \{\tau(g_{\ell}(y_{\ell'}))\}_{\ell\in[2],\ell'\in\{0,1\}} \\
            & F^{(t)}_{5, j}(\Zb^{2,\ell-1})=\sum_{r\in[m]} \ReLU\big(\Lambda^{(t)}_{5,j,r}\big)\leq O\Big(\frac{1}{\polylog d}\Big) \text{ for } j\in\tau(\Y)\setminus \{\tau(g_{\ell}(y_{\ell'}))\}_{\ell\in[2],\ell'\in\{0,1\}} \\
          & F^{(t)}_{5, j}(\Zb^{2,\ell-1})\leq m\cdot \tilde{O}(\sigma_0^{q})  \text{ for } j\notin\tau(\Y).
        \end{align*}
        Therefore, for any $j$, we have $\logit^{(t)}_{j}=O\big(\frac{1}{d}\big)$ since  $F^{(t)}_{5, j'}\leq O(1)$ for all $j'$. 
    \end{itemize}
\end{proof}
In the following, we illustrate the activations on the non-high probability event. 

\begin{lemma}\label{lem-s21-act-other-cyc}
     If \Cref{induction-s21-cyc} holds for all iterations $<t$, given input $\Zb^{2,\ell-1}$, then we have
         \begin{enumerate}
        \item for $\ell=1$, if $\Zb^{2,\ell-1}\notin \cE_1$, then
  \begin{enumerate}[(a)]
            \item for $j=j_1$, $\hat{\fA}_{j_1}=\{r_{g_1\cdot y_0}\}$, $\Lambda^{(t)}_{5,j_1,r_{g_1\cdot y_0}}= B\pm O(\delta)$;
            \item for $j\neq j_1\in\tau(\Y)$, assuming $j=\tau(g_1(y))$, then $\Lambda^{(t)}_{5,j,r_{g_1\cdot y}}=\frac{1}{3}B\pm O(1)$ and    $\Lambda^{(t)}_{5,j,r}\ll -\varrho$ for $r\in \hat{\fA}_{j}\setminus\{r_{g_1\cdot y}\}$.
        \end{enumerate}

       \item $\ell=2$, if  $\Zb^{2,\ell-1}\notin \cE_2$, then
                \begin{enumerate}[(a)]
                    \item if $g_1=g_2\wedge y_0\neq y_1$,
                    \begin{enumerate}
                        \item for $j=j_2$, $\Lambda^{(t)}_{5,j_2,r_{g_2\cdot y_1}}= \frac{1}{2}B\pm O(1)$ and $\Lambda^{(t)}_{5,j_2,r}\ll -\varrho$ for $r\in \hat{\fA}_{j_2}\setminus\{r_{g_2\cdot y_1}\}$;
                        \item for $j=\tau(g_2(y_0))$, $\Lambda^{(t)}_{5,j_2,r_{g_2\cdot y_0}}= \frac{1}{2}B\pm O(1)$ and $\Lambda^{(t)}_{5,j,r}\ll -\varrho$ for $r\in \hat{\fA}_{j}\setminus\{r_{g_2\cdot y_0}\}$;
                        \item for other $j\in\tau(\cY)$, assuming $j=\tau(g_2(y))$,   $|\Lambda^{(t)}_{5,j,r_{g_2\cdot y}}|\leq O(1)$ and $\Lambda^{(t)}_{5,j,r}\ll -\varrho$ for $r\in \hat{\fA}_{j}\setminus\{r_{g_2\cdot y}\}$.
                    \end{enumerate}
                  \item if $g_1\neq g_2\wedge y_0= y_1$,
                    \begin{enumerate}
                        \item for $j=j_2$, $\Lambda^{(t)}_{5,j_2,r_{g_2\cdot y_1}}= \frac{1}{2}B\pm O(1)$ and $\Lambda^{(t)}_{5,j_2,r}\ll -\varrho$ for $r\in \hat{\fA}_{j_2}\setminus\{r_{g_2\cdot y_1}\}$;
                        \item for $j=\tau(g_1(y_1))$, $\Lambda^{(t)}_{5,j,r_{g_1\cdot y_1}}= \frac{1}{2}B\pm O(1)$ and $\Lambda^{(t)}_{5,j,r}\ll -\varrho$ for $r\in \hat{\fA}_{j}\setminus\{r_{g_1\cdot y_1}\}$;
                        \item for other $j\in\tau(\cY)$, assuming $j=\tau(g(y_1))$,   $|\Lambda^{(t)}_{5,j,r_{g\cdot y_1}}|\leq O(1)$ and $\Lambda^{(t)}_{5,j,r}\ll -\varrho$ for $r\in \hat{\fA}_{j}\setminus\{r_{g\cdot y_1}\}$.
                    \end{enumerate}
                \item if $g_1=g_2\wedge y_0= y_1$,
                    \begin{enumerate}
                        \item for $j=j_2$, $ \hat{\fA}_{j_2}=\{r_{g_2\cdot y_1}\}$, $\Lambda^{(t)}_{5,j_2,r_{g_2\cdot y_1}}= B\pm O(\delta)$ ;
                        \item for $j\neq j_2\in\tau(\Y)$,     $|\Lambda^{(t)}_{5,j,r}|\leq O(1)$ for all  $r\in \hat{\fA}_{j}$.
                    \end{enumerate}
          \item $g_1\neq g_2\wedge y_0\neq y_1\wedge \Big(g_1(y_0)=g_2(y_1)\vee g_{2}(y_0)=g_1(y_1)\Big)$,  $|\Lambda^{(t)}_{5,j,r}|\leq O(1)$ for all  $r\in \hat{\fA}_{j}$.
        \end{enumerate}
    \end{enumerate}

\end{lemma}
With the characterization of activated neurons, we can derive the following logits for the non-high probability event.
\begin{lemma}\label{lem-s21-logit-other-cyc}
     If \Cref{induction-s21-cyc} holds for all iterations $<t$, given input $\Zb^{2,\ell-1}$, then we have
         \begin{enumerate}
        \item for $\ell=1$, if $\Zb^{2,\ell-1}\notin \cE_1$, then $1-\logit^{(t)}_{5,j_1}=O\Big(\frac{1}{\poly d}\Big)$.
       \item $\ell=2$, if  $\Zb^{2,\ell-1}\notin \cE_2$, then
                \begin{enumerate}[(a)]
                    \item if $g_1=g_2\wedge y_0\neq y_1$, $\logit^{(t)}_{5,j}=\Omega(1)$ for $j\in\{j_2,\tau(g_2(y_0))\}$, $1-\logit^{(t)}_{5,j_2}-\logit^{(t)}_{5,\tau(g_2(y_0))}=\frac{1}{\poly d}$. 
                  \item if $g_1\neq g_2\wedge y_0= y_1$, $\logit^{(t)}_{5,j}=\Omega(1)$ for $j\in\{j_2,\tau(g_1(y_1))\}$, $1-\logit^{(t)}_{5,j_2}-\logit^{(t)}_{5,\tau(g_1(y_1))}=\frac{1}{\poly d}$. 
                \item if $g_1=g_2\wedge y_0= y_1$, $1-\logit^{(t)}_{5,j_2}=O\Big(\frac{1}{\poly d}\Big)$.
          \item $g_1\neq g_2\wedge y_0\neq y_1\wedge \Big(g_1(y_0)=g_2(y_1)\vee g_{2}(y_0)=g_1(y_1)\Big)$,  $\logit^{(t)}_{5,j_2}=O(\frac{1}{d})$ for all $j$. 
        \end{enumerate}
    \end{enumerate}
\end{lemma}
\subsubsection{Gradient Lemma}
 
\begin{lemma}\label{lem-s21-gd1-cyc}
    If \Cref{induction-s21-cyc} holds for all iterations $<t$, given $s\in\tau(\X)$,  for $[\Qb^{(t)}_{4,3}]_{s,s}$, we have
    \begin{align*}
       \sum_{\ell=1}^2 \Big[-\nabla_{\Q^{(t)}_{4,3}}\Loss^{2,\ell}_{5}\Big]_{s,s}= \Theta\Big(\frac{\log d}{d}\Big).
    \end{align*}

\end{lemma}
\begin{proof}
        By gradient decompositions, we have
    \begin{align*}
           & \sum_{\ell=1}^2 \Big[-\nabla_{\Q^{(t)}_{4,3}}\Loss^{2,\ell}_{5}\Big]_{s,s}= \sum_{\ell\in [2]}\sum_{\kappa\in \{\rom1,\rom2,\rom3\}}\cN^{(t)}_{s,3,\ell,\kappa}.
    \end{align*}
  By \Cref{lem-s21-logit-cyc} and \Cref{lem-s21-logit-other-cyc}, it is straightforward to see that $|\cN^{(t)}_{s,3,1,\rom3}|$, $|\cN^{(t)}_{s,3,2,\rom3}|=O(\frac{1}{\poly d})$, and thus we can focus on other terms. 

  By \Cref{lem-s21-act-cyc} and \Cref{lem-s21-act-other-cyc}, we have
    \begin{align*}
       \cN^{(t)}_{s,3,1,\rom1}&= \E\Bigg[
    \attn^{(t)}_{{\ans,0} \rightarrow \pred,1} \cdot (1-\logit^{(t)}_{5,j_1})\cdot \\
    &~~~~~~~~~  \Big(
 \Big( V_{j_1,  r_{g_1\cdot y_0}}(g_1)- \Lambda^{(t)}_{5,j_1, r_{g_1\cdot y_0}}\pm\tilde{O}(\sigma_0) \Big)\pm \tilde{O}(\delta^{q}) \Big) 
    \1_{\{\tau(x_0)=s\}\cap \cE_1}\Bigg]\\
    &+ \E\Bigg[
    \attn^{(t)}_{{\ans,0} \rightarrow \pred,1} \cdot (1-\logit^{(t)}_{5,j_1})\cdot \\
    &~~~~~~~~~  \Big(
 \Big( V_{j_1,  r_{g_1\cdot y_0}}(g_1)- \Lambda^{(t)}_{5,j_1, r_{g_1\cdot y_0}}\pm\tilde{O}(\sigma_0) \Big)\pm \tilde{O}(\delta^{q}) \Big) 
    \1_{\{\tau(x_0)=s\}\cap \cE^c_1}\Bigg]\\
    &\stackrel{(a)}{=}\Theta(\frac{1}{d})\cdot \Omega(1)\cdot \Theta(B)+ \Theta(\frac{1}{d})\cdot O(\frac{1}{\poly d})\cdot  \Theta(B)\cdot O(\frac{1}{\log d})\\
    &=\Theta\Big(\frac{\log d}{d}\Big),
\end{align*}
where (a) follows from  \Cref{lem-s21-lambda-cyc}, \Cref{lem-s21-logit-cyc}, \Cref{lem-s21-act-other-cyc} and \Cref{lem-s21-logit-other-cyc}, and the fact that $\tau(x_0)=s$ holds with probability $\frac{1}{d}$. $\cN^{(t)}_{s,3,2,\rom1}$ can be upper bounded similarly. 

Moving to $\cN^{(t)}_{s,3,1,\rom2}$, noticing that $V_{j'_1,r_{g_2\cdot y_0}}(g_1)=-B+O(\delta)$ on $\cE_1$,  we have
    \begin{align*}
       \cN^{(t)}_{s,3,1,\rom2}&= -\E\Bigg[
    \attn^{(t)}_{{\ans,0} \rightarrow \pred,1} \cdot \logit^{(t)}_{5,j'_1} \cdot \\
    &~~~~~~~~~~\Big(  \Big( V_{j'_1,r_{g_2\cdot y_0}}(g_1)- \Lambda^{(t)}_{5,j'_1,r_{g_2\cdot y_0}}\pm\tilde{O}(\sigma_0) \Big)\pm \tilde{O}(\delta^{q}) \Big) 
\1_{\{\tau(x_0)=s\}\cap \cE_1}\Bigg]\\
& -\E\Bigg[
    \attn^{(t)}_{{\ans,0} \rightarrow \pred,1} \cdot\sum_{y\neq y_0\in\cY}\logit^{(t)}_{5,\tau(g_1(y))} \cdot \\
    &~~~~~~~~~~\Big( \Big( V_{\tau(g_1(y)), r_{g_1\cdot y}}(g_1)- \Lambda^{(t)}_{5,\tau(g_1(y)), r_{g_1\cdot y}}\pm\tilde{O}(\sigma_0) \Big)\pm \tilde{O}(\delta^{q}) \Big) 
\1_{\{\tau(x_0)=s\}\cap \cE^c_1}\Bigg]\\
&= \Theta(\frac{1}{d})\cdot \Omega(1)\cdot \Theta(B)- \Theta(\frac{1}{d})\cdot O(\frac{1}{\poly d})\cdot  \Theta(B)\cdot O(\frac{1}{\log d})\\
    &=\Theta\Big(\frac{\log d}{d}\Big).
\end{align*}
For $\cN^{(t)}_{s,3,2,\rom2}$, we only need to control the negative gradient, since the positive part can be easily upper bounded by 
 $O\Big(\frac{\log d}{d}\Big).$   \begin{align*}
       \cN^{(t)}_{s,3,2,\rom2}&\geq -\Theta(\frac{1}{d})\cdot O(\frac{1}{d})\cdot \Theta(B) 
 -\E\Bigg[
    \attn^{(t)}_{{\ans,1} \rightarrow \pred,2} \cdot\sum_{y\neq y_1\in \cY}\logit^{(t)}_{5,\tau(g_2(y))} \cdot \\
    &~~\Big( \Big( V_{\tau(g_2(y)), r_{g_2\cdot y}}(g_2)- \Lambda_{5,j,r}\pm\tilde{O}(\sigma_0) \Big)\pm \tilde{O}(\delta^{q}) \Big) 
\1_{\{\tau(x_0)=s, g_1=g_2, y_0\neq y_1\}}\Bigg]\\
&\stackrel{(a)}{\geq} -\Theta(\frac{1}{d})\cdot O(\frac{1}{d})\cdot \Theta(B)- \Theta(\frac{1}{d})\cdot O(1)\cdot  \Theta(B)\cdot O(\frac{1}{\log d})\\
    &\geq -O\Big(\frac{1}{d}\Big),
\end{align*}
where (a) is due to the fact that $\logit^{(t)}_{5,\tau(g_2(y_0))}=\Omega(1)$ and $\logit^{(t)}_{5,\tau(g_2(y))}=O(\frac{1}{d})$ for other $y\neq y_1, y_0$. Putting everything together, we complete the proof.
\end{proof}
\begin{lemma}[Negative gradient]\label{lem-s21-gd2-cyc}
    If \Cref{induction-s21-cyc} holds for all iterations $<t$, given $s\in\tau(\X)$,   we have
        \begin{align*}
        \sum_{\ell=1}^2  \Big[-\nabla_{\Q^{(t)}_{4,4}}\Loss^{2,\ell}_{5}\Big]_{s,s} \geq -O\Big(\frac{1}{\log d}\Big) \sum_{\ell=1}^2 \Big[-\nabla_{\Q^{(t)}_{4,3}}\Loss^{2,\ell}_{5}\Big]_{s,s}.
    \end{align*}

\end{lemma}
\begin{proof}
        By gradient decompositions, we have
    \begin{align*}
           & \sum_{\ell=1}^2 \Big[-\nabla_{\Q^{(t)}_{4,4}}\Loss^{2,\ell}_{5}\Big]_{s,s}= \sum_{\ell\in [2]}\sum_{\kappa\in \{\rom1,\rom2,\rom3\}}\cN^{(t)}_{s,4,\ell,\kappa}.
    \end{align*}
   Similarly as \Cref{lem-s21-gd1-cyc}, $|\cN^{(t)}_{s,4,1,\rom3}|$, $|\cN^{(t)}_{s,4,2,\rom3}|=O(\frac{1}{\poly d})$, and thus we can focus on other terms. By \Cref{lem-s21-act-cyc}, \Cref{lem-s21-lambda-cyc}, and \Cref{lem-s21-act-other-cyc}, we have
   \begin{align}
    \cN^{(t)}_{s,4,1,\rom1}+ \cN^{(t)}_{s,4,1,\rom2}&= \E\Bigg[
    \attn^{(t)}_{{\ans,0} \rightarrow \ans,0} \cdot (1-\logit^{(t)}_{5,j_1})\cdot \notag\\
    &~~~~~~~~~  \Big(
 \Big( V_{j_1,  r_{g_1\cdot y_0}}(y_0)- \Lambda^{(t)}_{5,j_1, r_{g_1\cdot y_0}}\pm\tilde{O}(\sigma_0) \Big)\pm \tilde{O}(\delta^{q}) \Big) 
    \1_{\{\tau(x_0)=s\}\cap \cE_1}\Bigg] \label{eq-s21-gd2-1}\\
    &+ \E\Bigg[
    \attn^{(t)}_{{\ans,0} \rightarrow \ans,0} \cdot (1-\logit^{(t)}_{5,j_1})\cdot\notag \\
    &~~~~~~~~~  \Big(
 \Big( V_{j_1,  r_{g_1\cdot y_0}}(y_0)- \Lambda^{(t)}_{5,j_1, r_{g_1\cdot y_0}}\pm\tilde{O}(\sigma_0) \Big)\pm \tilde{O}(\delta^{q}) \Big) 
    \1_{\{\tau(x_0)=s\}\cap \cE^c_1}\Bigg]\label{eq-s21-gd2-2}\\
   & -\E\Bigg[
    \attn^{(t)}_{{\ans,0} \rightarrow \ans,0} \cdot \logit^{(t)}_{5,j'_1} \cdot \notag\\
    &~~~~~~~~~~\Big(  \Big( V_{j'_1,r_{g_2\cdot y_0}}(y_0)- \Lambda^{(t)}_{5,j'_1,r_{g_2\cdot y_0}}\pm\tilde{O}(\sigma_0) \Big)\pm \tilde{O}(\delta^{q}) \Big) 
\1_{\{\tau(x_0)=s\}\cap \cE_1}\Bigg]\label{eq-s21-gd2-3}\\
& -\E\Bigg[
    \attn^{(t)}_{{\ans,0} \rightarrow \ans,0} \cdot\sum_{g\neq g_1\in\cG}\logit^{(t)}_{5,\tau(g(y_0))} \cdot \notag\\
    &~~~~~~~~~~\Big( \Big( V_{\tau(g(y_0)), r_{g\cdot y_0}}(y_0)- \Lambda^{(t)}_{5,\tau(g(y_0)), r_{g\cdot y_0}}\pm\tilde{O}(\sigma_0) \Big)\pm \tilde{O}(\delta^{q}) \Big) 
\1_{\{\tau(x_0)=s\}\cap \cE^c_1}\Bigg].\label{eq-s21-gd2-4}
   \end{align}
   First, considering the event $\{\tau(x_0)=s\}\cap \cE_1$, we have
       \begin{align*}
        &\eqref{eq-s21-gd2-1}+\eqref{eq-s21-gd2-3}\\
            &=  \E\Bigg[
            \attn^{(t)}_{{\ans,0} \rightarrow \ans,0} \cdot \bigg( (1-\logit^{(t)}_{5,j_1})\cdot \Big( V_{j_1,  r_{g_1\cdot y_1}}(y_0)- \Lambda^{(t)}_{5,j_1,r_{g_1\cdot y_0}} \pm \tilde{O}(\sigma_0)\pm \tilde{O}(\delta^q)\Big)\\
            &~~~~~~~~~~~~~~~~~~~~~~-\logit^{(t)}_{5,j_1^{\prime}}\cdot \Big(V_{j'_1,  r_{g_2\cdot y_0}}(y_0)- \Lambda^{(t)}_{5,j_1^{\prime},  r_{g_2\cdot y_0}}\pm \tilde{O}(\sigma_0)\pm \tilde{O}(\delta^q)\Big)
    \bigg)\1_{\{\tau(x_0)=s\}\cap \cE_1}\Bigg]\\
    &\stackrel{(a)}{=}  \E\Bigg[
            \attn^{(t)}_{{\ans,0} \rightarrow \ans,0} \cdot \bigg( (1-\logit^{(t)}_{5,j_1})\cdot   \Big( V_{j_1,  r_{g_1\cdot y_1}}(y_0)- \Lambda^{(t)}_{5,j_1,r_{g_1\cdot y_0}} \pm \tilde{O}(\sigma_0)\pm \tilde{O}(\delta^q)\Big)\\
            &~~~~~-(1-\logit^{(t)}_{5,j_1}-\frac{1}{\poly d})\cdot  \Big(V_{j'_1,  r_{g_2\cdot y_0}}(y_0)- \Lambda^{(t)}_{5,j_1^{\prime},  r_{g_2\cdot y_0}}\pm \tilde{O}(\sigma_0)\pm \tilde{O}(\delta^q)\Big)
    \bigg)\1_{\{\tau(x_0)=s\}\cap \cE_1}\Bigg]\\
    &\geq -\E\Bigg[
            \attn^{(t)}_{{\ans,0} \rightarrow \ans,0} \cdot \bigg( (1-\logit^{(t)}_{5,j_1})\cdot \\
            &~~~~~~~~~  \Big( V_{j_1,  r_{g_1\cdot y_1}}(y_0)-V_{j'_1,  r_{g_2\cdot y_0}}(y_0)+ \Lambda^{(t)}_{5,j_1^{\prime},  r_{g_2\cdot y_0}}- \Lambda^{(t)}_{5,j_1,r_{g_1\cdot y_0}} \pm \tilde{O}(\sigma_0)\pm \tilde{O}(\delta^q)\Big)
    \bigg)\1_{\{\tau(x_0)=s\}\cap \cE_1}\Bigg]\\
    &\stackrel{(b)}{\geq} -\E\Bigg[
            \attn^{(t)}_{{\ans,0} \rightarrow \ans,0} \cdot \bigg( (1-\logit^{(t)}_{5,j_1})\cdot \Big( O(\delta)+ O(1)\pm \tilde{O}(\sigma_0)\pm \tilde{O}(\delta^q)\Big)
    \bigg)\1_{\{\tau(x_0)=s\}\cap \cE_1}\Bigg]\\
    &\geq -O\Big(\frac{\cN^{(t)}_{s,3,1,\rom1}}{\log d}\Big),
        \end{align*}
        where (a) follows from \Cref{lem-s21-logit-cyc}; (b) follows from \Cref{lem-s21-lambda-cyc} and \Cref{lem-prop-psi-cyc}. 
  
Notice that $\eqref{eq-s21-gd2-2}\geq 0$ since $V_{j_1,r_{g_1\cdot y_0}}(y_0)- \Lambda^{(t)}_{5,j_1, r_{g_1\cdot y_0}}\geq \Omega(B)$, thus we just need to consider the possible negative gradient from \eqref{eq-s21-gd2-4}.  By \Cref{lem-s21-logit-other-cyc}, we have $\sum_{g\neq g_1\in\cG}\logit^{(t)}_{5,\tau(g(y_0))}\leq O(\frac{1}{\poly d})$ on $\{\tau(x_0)=s\}\cap \cE^c_1$, and hence $\eqref{eq-s21-gd2-4}\ll \cN^{(t)}_{s,3,1,\rom1}$. 

Moving to the gradient from $\ell=2$, it is straightforward to see that $\cN^{(t)}_{s,4,2,\rom1}$ is non-negative, and thus we can focus on the possible negative gradient from $\cN^{(t)}_{s,4,2,\rom2}$. 
\begin{align*}
       \cN^{(t)}_{s,4,2,\rom2}&= -\E\Bigg[
    \attn^{(t)}_{{\ans,1} \rightarrow \ans,1} \cdot\sum_{j\neq j_2\in\tau(\cY)}\logit^{(t)}_{5,j} \cdot \\
    &~~~~~~~~~~\Big(\sum_{r\in\hat{\fA}_{j}}\ReLU^{\prime}(\Lambda^{(t)}_{5,j,r
    })\cdot  \Big( V_{j, r}(y_1)- \Lambda^{(t)}_{5,j,r}\pm\tilde{O}(\sigma_0) \Big)\pm \tilde{O}(\delta^{q}) \Big) 
\1_{\{\tau(x_1)=s\}\cap \cE_2}\Bigg]\\
&-\E\Bigg[
    \attn^{(t)}_{{\ans,1} \rightarrow \ans,1} \cdot\sum_{j\neq j_2\in\tau(\cY)}\logit^{(t)}_{5,j} \cdot \\
    &~~~~~~~~~~\Big(\sum_{r\in\hat{\fA}_{j}}\ReLU^{\prime}(\Lambda^{(t)}_{5,j,r
    })\cdot  \Big( V_{j, r}(y_1)- \Lambda^{(t)}_{5,j,r}\pm\tilde{O}(\sigma_0) \Big)\pm \tilde{O}(\delta^{q}) \Big) 
\1_{\{\tau(x_1)=s\}\cap \cE^c_2}\Bigg]\\
&\geq - \Theta(\frac{1}{d})\cdot O(\frac{1}{d})\cdot \Theta(B)- \Theta(\frac{1}{d})\cdot O(1)\cdot  \Theta(B)\cdot O(\frac{1}{\log d})\geq -O\Big(\frac{\cN^{(t)}_{s,3,1,\rom1}}{\log d}\Big).
\end{align*}
Putting everything together, and combining with the fact that $\sum_{\ell=1}^2 \Big[-\nabla_{\Q^{(t)}_{4,3}}\Loss^{2,\ell}_{5}\Big]_{s,s}= \Theta(\cN^{(t)}_{s,3,1,\rom1})$ from \Cref{lem-s21-gd1-cyc}, we complete the proof.
\end{proof}
\begin{lemma}[Growth
    of gap]\label{lem-s21-gd3-cyc}
    If \Cref{induction-s21-cyc} holds for all iterations $<t$, given $s\in\tau(\X)$,  we have
    \begin{align*}
        \sum_{\ell=1}^2 \Big[-\nabla_{\Q^{(t)}_{4,3}}\Loss^{2,\ell}_{5}\Big]_{s,s}+ \sum_{\ell=1}^2 \Big[\nabla_{\Q^{(t)}_{4,4}}\Loss^{2,\ell}_{5}\Big]_{s,s} \geq \Omega\bigg(\sum_{\ell=1}^2 \Big[-\nabla_{\Q^{(t)}_{4,3}}\Loss^{2,\ell}_{5}\Big]_{s,s} \bigg).
    \end{align*}

\end{lemma}
\begin{proof}
    By gradient decompositions in \eqref{eq-def-N-s-3-1-1}-\eqref{eq-def-N-s-4-2-3}, we have
    \begin{align*}
           & \sum_{\ell=1}^2 \Big[-\nabla_{\Q^{(t)}_{4,3}}\Loss^{2,\ell}_{5}\Big]_{s,s}+ \sum_{\ell=1}^2 \Big[\nabla_{\Q^{(t)}_{4,4}}\Loss^{2,\ell}_{5}\Big]_{s,s}= \sum_{\ell\in [2]}\sum_{\kappa\in \{\rom1,\rom2,\rom3\}}\cN^{(t)}_{s,3,\ell,\kappa}-\cN^{(t)}_{s,4,\ell,\kappa}.
    \end{align*}
   Due to  \Cref{lem-s21-logit-cyc} and \Cref{lem-s21-logit-other-cyc}, $|\cN^{(t)}_{s,p,\ell,\rom3}|=O(\frac{1}{\poly d})$ for $p\in\{3,4\}$ and $\ell\in[2]$, we can focus on the gradient difference between $\Q^{(t)}_{4,3}$ and $\Q^{(t)}_{4,4}$ contributed by other terms. 

   By \Cref{lem-s21-attn-cyc}, we have 
   \begin{align*}
  \attn^{(t)}_{{\ans,0} \rightarrow \ans,0}\leq  \attn^{(t)}_{{\ans,0} \rightarrow \pred,1}, \quad   \attn^{(t)}_{{\ans,1} \rightarrow \ans,1}\leq  \attn^{(t)}_{{\ans,1} \rightarrow \pred,2}.
   \end{align*}
   Hence, for $\ell\in [2]$, we have
   \begin{align*}
    \cN^{(t)}_{s,3,\ell,\rom1}- \cN^{(t)}_{s,4,\ell,\rom1}\geq -O(\delta)\cdot O(\frac{1}{d})\cdot \Theta(1)\geq -O(\delta/d).
   \end{align*}
    \begin{align*}
     & \cN^{(t)}_{s,3,1,\rom2}-\cN^{(t)}_{s,4,1,\rom2}\\
     &= -\E\Bigg[
    \attn^{(t)}_{{\ans,0} \rightarrow \pred,1} \cdot \logit^{(t)}_{5,j'_1} \cdot \\
    &~~~~~~~~~~\Big(  \Big( V_{j'_1,r_{g_2\cdot y_0}}(g_1)- \Lambda^{(t)}_{5,j'_1,r_{g_2\cdot y_0}}\pm\tilde{O}(\sigma_0) \Big)\pm \tilde{O}(\delta^{q}) \Big) 
\1_{\{\tau(x_0)=s\}\cap \cE_1}\Bigg]\\
&+\E\Bigg[
    \attn^{(t)}_{{\ans,0} \rightarrow \ans,0} \cdot \logit^{(t)}_{5,j'_1} \cdot \\
    &~~~~~~~~~~\Big(  \Big( V_{j'_1,r_{g_2\cdot y_0}}(y_0)- \Lambda^{(t)}_{5,j'_1,r_{g_2\cdot y_0}}\pm\tilde{O}(\sigma_0) \Big)\pm \tilde{O}(\delta^{q}) \Big) 
\1_{\{\tau(x_0)=s\}\cap \cE_1}\Bigg]\\
& -\E\Bigg[
    \attn^{(t)}_{{\ans,0} \rightarrow \pred,1} \cdot\sum_{y\neq y_0\in\cY}\logit^{(t)}_{5,\tau(g_1(y))} \cdot \\
    &~~~~~~~~~~\Big( \Big( V_{\tau(g_1(y)), r_{g_1\cdot y}}(g_1)- \Lambda^{(t)}_{5,\tau(g_1(y)), r_{g_1\cdot y}}\pm\tilde{O}(\sigma_0) \Big)\pm \tilde{O}(\delta^{q}) \Big) 
\1_{\{\tau(x_0)=s\}\cap \cE^c_1}\Bigg]\\
& +\E\Bigg[
    \attn^{(t)}_{{\ans,0} \rightarrow \ans,0} \cdot\sum_{y\neq y_0\in\cY}\logit^{(t)}_{5,\tau(g_1(y))} \cdot \\
    &~~~~~~~~~~\Big( \Big( V_{\tau(g_1(y)), r_{g_1\cdot y}}(y_0)- \Lambda^{(t)}_{5,\tau(g_1(y)), r_{g_1\cdot y}}\pm\tilde{O}(\sigma_0) \Big)\pm \tilde{O}(\delta^{q}) \Big) 
\1_{\{\tau(x_0)=s\}\cap \cE^c_1}\Bigg]\\
&\stackrel{(a)}{\geq} \Theta(\frac{1}{d})\cdot \Omega(1)\cdot \Theta(B)- \Theta(\frac{1}{d})\cdot O(\frac{1}{\poly d})\cdot  O(B)\cdot O(\frac{1}{\log d})\geq \Omega\Big(\frac{\log d}{d}\Big),
\end{align*}
where (a) follows from \Cref{lem-prop-psi-cyc} that $V_{j'_1,r_{g_2\cdot y_0}}(g_1)- \Lambda^{(t)}_{5,j'_1,r_{g_2\cdot y_0}}\leq -\Omega(B)$, and $V_{j'_1,r_{g_2\cdot y_0}}(y_0)- \Lambda^{(t)}_{5,j'_1,r_{g_2\cdot y_0}}\geq \Omega(B)$.
\begin{align*}
       &\cN^{(t)}_{s,3,2,\rom2}-\cN^{(t)}_{s,4,2,\rom2}\\
       &= -\E\Bigg[
    \attn^{(t)}_{{\ans,1} \rightarrow \pred,2} \cdot\sum_{j\neq j_2\in\tau(\cY)}\logit^{(t)}_{5,j} \cdot \\
    &~~~~~~~~~~\Big(\sum_{r\in\hat{\fA}_{j}}\ReLU^{\prime}(\Lambda^{(t)}_{5,j,r
    })\cdot  \Big( V_{j, r}(g_2)- \Lambda^{(t)}_{5,j,r}\pm\tilde{O}(\sigma_0) \Big)\pm \tilde{O}(\delta^{q}) \Big) 
\1_{\{\tau(x_1)=s\}\cap \cE_2}\Bigg]\\
&+\E\Bigg[
    \attn^{(t)}_{{\ans,1} \rightarrow \ans,1} \cdot\sum_{j\neq j_2\in\tau(\cY)}\logit^{(t)}_{5,j} \cdot \\
    &~~~~~~~~~~\Big(\sum_{r\in\hat{\fA}_{j}}\ReLU^{\prime}(\Lambda^{(t)}_{5,j,r
    })\cdot  \Big( V_{j, r}(y_1)- \Lambda^{(t)}_{5,j,r}\pm\tilde{O}(\sigma_0) \Big)\pm \tilde{O}(\delta^{q}) \Big) 
\1_{\{\tau(x_1)=s\}\cap \cE_2}\Bigg]\\
&-\E\Bigg[
    \attn^{(t)}_{{\ans,1} \rightarrow \pred,2} \cdot\sum_{j\neq j_2\in\tau(\cY)}\logit^{(t)}_{5,j} \cdot \\
    &~~~~~~~~~~\Big(\sum_{r\in\hat{\fA}_{j}}\ReLU^{\prime}(\Lambda^{(t)}_{5,j,r
    })\cdot  \Big( V_{j, r}(g_2)- \Lambda^{(t)}_{5,j,r}\pm\tilde{O}(\sigma_0) \Big)\pm \tilde{O}(\delta^{q}) \Big) 
\1_{\{\tau(x_1)=s\}\cap \cE^c_2}\Bigg]\\
&+\E\Bigg[
    \attn^{(t)}_{{\ans,1} \rightarrow \ans,1} \cdot\sum_{j\neq j_2\in\tau(\cY)}\logit^{(t)}_{5,j} \cdot \\
    &~~~~~~~~~~\Big(\sum_{r\in\hat{\fA}_{j}}\ReLU^{\prime}(\Lambda^{(t)}_{5,j,r
    })\cdot  \Big( V_{j, r}(y_1)- \Lambda^{(t)}_{5,j,r}\pm\tilde{O}(\sigma_0) \Big)\pm \tilde{O}(\delta^{q}) \Big) 
\1_{\{\tau(x_1)=s\}\cap \cE^c_2}\Bigg]\\
&\stackrel{(a)}{\geq} - \Theta\Big(\frac{1}{d}\Big)\cdot O\Big(\frac{1}{d}\Big)\cdot \Theta(B)- \Theta\Big(\frac{1}{d}\Big)\cdot O(1)\cdot  \Theta(B)\cdot O\Big(\frac{1}{\log d}\Big)\geq -O\Big(\frac{1}{d}\Big),
\end{align*}
where (a) follows from \Cref{lem-s21-act-cyc} and \Cref{lem-s21-logit-cyc}, which together imply that on the event \( \cE_2 \), only a constant number of neurons are activated and \( \logit^{(t)}_{5,j} \leq O\left( \frac{1}{d} \right) \) for all \( j \neq j_2 \); and from \Cref{lem-s21-logit-other-cyc}, which implies that on the complement event \( \cE_2^c \), occurring with probability at most \( O\left( \frac{1}{\log d} \right) \), there exists at most one \( j \neq j_2 \) such that \( \logit^{(t)}_{5,j} \geq \Omega(1) \) while \( \logit^{(t)}_{5,j}\leq O\Big(\frac{1}{\poly d}\Big) \) for other $j\in\tau(\Y)$.

    Putting it all together, we finish the proof.
\end{proof}
\begin{lemma}\label{lem-s21-gd4-cyc}
    If \Cref{induction-s21-cyc} holds for all iterations $<t$, given $s\in\tau(\X)$,  for $[\Qb^{(t)}_{4,3}]_{s,s}\geq \Omega(\frac{\varrho}{\log d})$, we have
    \begin{align*}
\sum_{\ell=1}^2 \Big[-\nabla_{\Q^{(t)}_{4,4}}\Loss^{2,\ell}_{5}\Big]_{s,s} \geq \Omega\bigg(\sum_{\ell=1}^2 \Big[-\nabla_{\Q^{(t)}_{4,3}}\Loss^{2,\ell}_{5}\Big]_{s,s} \bigg).
    \end{align*}

\end{lemma}
\begin{proof}
    Notice that by \Cref{lem-s21-gd2-cyc}, when $[\Qb^{(t)}_{4,3}]_{s,s}\geq \Omega(\frac{\varrho}{\log d})$, we have $[\Qb^{(t)}_{4,4}]_{s,s}\geq -O(\frac{\varrho}{\log^2 d})$. Hence
    \begin{align*}
        \attn^{(t)}_{\ans,1\to\pred,2}+\attn^{(t)}_{\ans,1\to\ans,1}-\attn^{(t)}_{\ans,1\to\pred,1}-\attn^{(t)}_{\ans,1\to\ans,0} \geq \Omega\Big(\frac{\varrho}{\log d}\Big), 
    \end{align*}
    which implies $\Lambda^{(t)}_{5,j_2,r_{g_2\cdot y_1}}(\Zb^{2,1})\geq \Omega(\frac{\varrho}{\log d})\cdot B \pm O(\delta)\geq \varrho$ already lies in the linear regime for $\Zb^{2,1}\in \cE_2$. Then we have 
    \begin{align*}
        \cN^{(t)}_{s,4,2,\rom1} &\geq \E\Bigg[
            \attn^{(t)}_{{\ans,1} \rightarrow \ans,1} \cdot  \bigg( (1-\logit^{(t)}_{5,j_2})\cdot\\
            &~~~~~~~~~~~~~~~~~~~~~ \Big( V_{j_2,r_{g_2\cdot y_1}}- \Lambda^{(t)}_{5,j_2,r_{g_2\cdot y_1}} \pm \tilde{O}(\sigma_0)\Big)\pm \tilde{O}(\delta^q)\bigg)\1_{\{\tau(x_1)=s\}\cap\cE_2}\Bigg]\\
            &\geq \Omega(1)\cdot \Theta(B) \cdot \Theta\Big(\frac{1}{d}\Big)\geq \Omega\Big(\frac{\log d}{d}\Big).
    \end{align*}
    Moreover, from \Cref{lem-s21-gd2-cyc},   the magnitude of negative gradient from other $\cN$ terms can be upper bounded by $O\Big(\frac{\cN^{(t)}_{s,3,1,\rom1}}{\log d}\Big)$. Therefore, combining with the fact that $\sum_{\ell=1}^2 \Big[-\nabla_{\Q^{(t)}_{4,3}}\Loss^{2,\ell}_{5}\Big]_{s,s}= \Theta(\cN^{(t)}_{s,3,1,\rom1})$ from \Cref{lem-s21-gd1-cyc}, we complete the proof.
\end{proof}
\begin{lemma}\label{lem-s21-gd5}
    If \Cref{induction-s21-cyc} holds for all iterations $<t$, given $s\neq s' \in\tau(\X)$,  for $p\in\{3,4\}$, we have
    \begin{align*}
        \bigg|\sum_{\ell=1}^2 \Big[-\nabla_{\Q^{(t)}_{4,p}}\Loss^{2,\ell}_{5}\Big]_{s,s'}\bigg|\leq  O\Big(\frac{\log d}{d^2}\Big)=O(\frac{1}{d})\cdot \bigg|\sum_{\ell=1}^2 \Big[-\nabla_{\Q^{(t)}_{4,p}}\Loss^{2,\ell}_{5}\Big]_{s,s}\bigg|.
    \end{align*}

\end{lemma}
\begin{proof}
The proof follows directly by combining the expressions from \Cref{lem-refined-grad-Q43}, \Cref{lem-refined-grad-Q44}, and \Cref{lem-s21-gd1-cyc}, along with the fact that the event \( \{\tau(x_0) = s, \tau(x_1) = s'\} \) occurs with probability \( O\big(\tfrac{1}{d^2}\big) \).
\end{proof}
\subsubsection{At the End of Stage 2.1}
Putting gradient lemmas together, we can directly prove that \Cref{induction-s21-cyc} holds for all iterations $t$ until the end of stage 2.1, where we can conclude the following:
\begin{lemma}[End of stage 2.1]\label{lem-end-s21-cyc}
  Given $s\in\tau(\X)$, \Cref{induction-s21-cyc} holds for all iterations $t<T_{2,1,s}=O\Big(\frac{d}{\eta \log^2 d}\Big)$, then at the end of stage 2.1, we have
 \begin{enumerate}[(a)]
    \item $[\Qb^{(t)}_{4,p}]_{s,s}=\Omega(\frac{1}{\log d})$ for $p\in\{3,4\}$;
    \item $[\Qb^{(t)}_{4,3}]_{s,s}-[\Qb^{(t)}_{4,4}]_{s,s}\geq \Omega\Big(\frac{1}{\log d}\Big)$;
    \item $\Big|[\Qb^{(t)}_{4,p}]_{s,s'}\Big|\leq O\Big(\frac{[\Qb^{(t)}_{4,p}]_{s,s}}{d}\Big)$ for $s'\in\tau(\X)\not=s$ for $p\in\{3,4\}$; otherwise,  $[\Qb^{(t)}_{4,p}]_{s,s'}=0$.
 \end{enumerate}
    
\end{lemma}

\subsection{Stage 2.2: Continual Growth of Diagonal Entries}
In Stage 2.2, the diagonal entries \( [\Q_{4,3}]_{s,s} \) and \( [\Q_{4,4}]_{s,s} \) continue to grow until they reach a certain threshold. The analysis in this stage parallels that of Stage 2.1, but our focus now shifts to the gradients contributed by \( \ell = 2 \), as the logit at \( \ell = 1 \) is already near-optimal and thus contributes negligibly to the growth of \( \Q_{4,3} \) and \( \Q_{4,4} \).

\begin{induction}\label{induction-s22-cyc}
  Given $s\in\tau(\X)$,  let $T_{2,2,s}$ denote the first time that $[\Qb_{4,3}]_{s,s}$ reaches $0.0001$. For all iterations $T_{2,1,s}\leq t<  T_{2,2,s}$, we have the following holds
     \begin{enumerate}[(a)]
        \item $[\Qb^{(t)}_{4,3}]_{s,s}, [\Qb^{(t)}_{4,4}]_{s,s}\leq O(1)$ monotonically increases;
        \item $[\Qb^{(t)}_{4,3}]_{s,s}-[\Qb^{(t)}_{4,4}]_{s,s}\in \Big[\Omega\big(\frac{1}{\log d}\big), O(1)\Big]$;
        \item for $(p,q)\in\{(4,3), (4,4)\}$, $\Big|[\Qb^{(t)}_{p,q}]_{s,s'}\Big|\leq O\Big(\frac{[\Qb^{(t)}_{p,q}]_{s,s}}{d}\Big)$ for $s'\in\tau(\X)\not=s$; other $[\Qb^{(t)}_{p,q}]_{s,s'}=0$.
     \end{enumerate}
  \end{induction}
  Throughout the following analysis, instead of $\cE_2$ defined in \eqref{eq-cyc-event-2}, we consider a renewed event $\tilde{\cE}_2$ for $\ell=2$:
  \begin{align}\label{eq-cyc-event-3}
    \tilde{\cE}_2\triangleq\Big\{g_1\neq g_2\wedge y_0\neq y_1\}.
  \end{align}
  \subsubsection{Attention and Logit Preliminaries}
  The proof in this part proceeds analogously to the arguments in \Cref{sec-s21-attn-cyc}, with the induction hypothesis from \Cref{induction-s22-cyc} incorporated. Hence, we omit the details here.

\begin{lemma}\label{lem-s22-attn-cyc}
    If \Cref{induction-s22-cyc} holds for all iterations $\in [T_{2,1,s}, t)$, given input $\Zb^{2,\ell-1}$, then we have 
\begin{enumerate}
    \item for $\ell=1$,  
    \begin{enumerate}[(a)]
    \item $\attn^{(t)}_{\ans,0\to \pred, 1}\in \Big[\frac{1}{3}+\Omega\big(\frac{1}{\log d}\big), \frac{1}{3}+c_1\Big]$, where $c_1>0$ is a small constant; 
    \item  $\attn^{(t)}_{\ans,0\to \pred, 2} \in \Big[\frac{1}{3}-c_2, \frac{1}{3}-\Omega\big(\frac{1}{\log d}\big)\Big]$, where $c_2>0$ is a small constant;
    \item $\attn^{(t)}_{\ans,0\to \pred, 2}+\Omega\big(\frac{1}{\log d}\big)\leq \attn^{(t)}_{\ans,0\to \ans, 0}$;
    \item $ \attn^{(t)}_{\ans,0\to \pred, 1} -\attn^{(t)}_{\ans,0\to \ans, 0}\in \Big[\Omega\big(\frac{1}{\log d}\big), c_3\Big]$.
    \end{enumerate}
    
      \item for $\ell=2$, 
      \begin{enumerate}[(a)]
        \item $\attn^{(t)}_{\ans,1\to \pred, 2}\in \Big[\frac{1}{4}+\Omega\big(\frac{1}{\log d}\big), \frac{1}{4}+c_4\Big]$, where $c_4>0$ is a small constant; 
        \item  $\attn^{(t)}_{\ans,1\to \pred, 1}, \attn^{(t)}_{\ans,1\to \ans, 0} \in \Big[\frac{1}{4}-c_5, \frac{1}{4}-\Omega\big(\frac{1}{\log d}\big)\Big]$, moreover, $\big|\attn^{(t)}_{\ans,1\to \ans, 0}- \attn^{(t)}_{\ans,1\to \pred, 1}\big|\leq O\Big(\frac{1}{d}\Big)$, where $c_5>0$ is a small constant;
        \item $\attn^{(t)}_{\ans,1\to \pred, 1},\attn^{(t)}_{\ans,1\to \ans, 0}\leq \attn^{(t)}_{\ans,1\to \ans, 1}-\Omega\Big(\frac{1}{\log d}\Big) $;
        \item $\attn^{(t)}_{\ans,1\to \pred, 2}- \attn^{(t)}_{\ans,1\to \ans, 1} \in \Big[\Omega\big(\frac{1}{\log d}\big), c_6\Big]$, where $c_6>0$ is a small constant. 
        \end{enumerate}
\end{enumerate}
Notice that the constant $c_1-c_6$ depends on the threshold $0.0001$ in \Cref{induction-s22-cyc}. We choose the  threshold $0.0001$ small enough to ensure $2C_B(c_1+c_2)< 1-c_6C_B$ and $1-4c_5C_B>0$.
\end{lemma}

\begin{lemma}\label{lem-s22-act-cyc}
    If \Cref{induction-s22-cyc} holds for all iterations $\in [T_{2,1,s}, t)$, given input $\Zb^{2,\ell-1}$, then we have
    \begin{enumerate}
        \item for $\ell=1$, if $\Zb^{2,\ell-1}\in \cE_1$, then
  \begin{enumerate}[(a)]
            \item for $j=j_1$, 
       $\Lambda^{(t)}_{5,j_1,r}\ll -\varrho$  for  $r\in \hat{\fA}_{j_1}\setminus\{r_{g_1\cdot y_0}\}$; 
            \item for $j=j'_1\triangleq \tau\big(g_2(y_0)\big)$, 
            $\Lambda^{(t)}_{5,j'_1,r}\ll -\varrho$ for  $r\in \hat{\fA}_{j'_1}\setminus\{r_{g_2\cdot y_0}\}$;
            \item for other $j\in\tau(\Y)$, $r$ is not activated for all $r\in \hat{\fA}_{j}$, i.e., $\Lambda^{(t)}_{5,j,r}\ll -\varrho$.
        \end{enumerate}

       \item $\ell=2$, if  $\Zb^{2,\ell-1}\in \tilde{\cE}_2$, then
                \begin{enumerate}[(a)]
                    \item for $j=j_2$,  $\Lambda^{(t)}_{5,j_2,r}\ll -\varrho$  for  $r\in \hat{\fA}_{j_2}\setminus\{r_{g_2\cdot y_1}\}$;
  \item for $j=j'_2\triangleq \tau\big(g_2(y_0)\big)$, 
            $\Lambda^{(t)}_{5,j'_2,r}\ll -\varrho$ for  $r\in \hat{\fA}_{j'_2}\setminus\{r_{g_2\cdot y_0}\}$;
            \item for other $j\in\tau(\Y)$, $r$ is not activated for all $r\in \hat{\fA}_{j}$, i.e., $\Lambda^{(t)}_{5,j,r}\ll -\varrho$.
        \end{enumerate}
    \end{enumerate}
        
\end{lemma}

\begin{lemma}\label{lem-s22-lambda-cyc}
    If \Cref{induction-s22-cyc} holds for all iterations $\in [T_{2,1,s}, t)$, given input $\Zb^{2,\ell}$, then we have
    \begin{enumerate}
        \item $\ell=1$, for $\Zb^{2,\ell-1}\in\cE_1$, 
  \begin{enumerate}[(a)]
            \item $\Lambda^{(t)}_{5,j_1,r_{g_1\cdot y_0}}\geq \Big(\frac{1}{3}+\Omega\big(\frac{1}{\log d}\big)\Big)B$;
            \item  $\Omega(1)\leq\Lambda^{(t)}_{5,j_1,r_{g_1\cdot y_0}}-\Lambda^{(t)}_{5,j'_1,r_{g_2\cdot y_0}}\leq 2(c_1+c_2)B$.
        \end{enumerate}
        \item $\ell=2$, for $\Zb^{2,\ell-1}\in\tilde{\cE}_2$, 
   \begin{enumerate}[(a)]
            \item $\Lambda^{(t)}_{5,j_2,r_{g_2\cdot y_1}}\in \Big[\Omega(1), 4c_5B \Big]$; 
            \item$\Lambda^{(t)}_{5,j'_2,r_{g_2\cdot y_0}}\in \big[\Omega(1), c_6B \big]$ and  $\Lambda^{(t)}_{5,j_2,r_{g_2\cdot y_1}}-\Lambda^{(t)}_{5,j'_2,r_{g_2\cdot y_0}}\geq \Omega(1)$.

        \end{enumerate}
    \end{enumerate}
  
\end{lemma}
\begin{lemma}
\label{lem-s22-logit-cyc}
    If \Cref{induction-s22-cyc} holds for all iterations $\in [T_{2,1,s}, t)$, given input $\Zb^{2,\ell-1}$, then we have
\begin{enumerate}   
    \item for $\ell=1$, if $\Zb^{2,\ell-1}\in \cE_1$, 
 $\logit^{(t)}_{5,j_1^{\prime}}\geq \Omega\Big(\frac{1}{d^{2C_B(c_1+c_2)}}\Big)$;
    \item for $\ell=2$,  if $\Zb^{2,\ell-1}\in \tilde{\cE}_2$, $1-\logit^{(t)}_{5,j_2}=\Omega(1)$, $\logit^{(t)}_{5,j_2^{\prime}}=O\big(\frac{1}{d^{1-c_6C_B }}\big)$. 
\end{enumerate}
\end{lemma}

\begin{proof}
    \begin{itemize}[left=0pt]
        \item For $\ell=1$, we have 
   \begin{align*}
    \logit^{(t)}_{5, j^{\prime}_1} 
    &= \frac{1}{1+e^{\Lambda^{(t)}_{5,j_1,r_{g_1\cdot y_0}}-\Lambda^{(t)}_{5,j'_1,r_{g_2\cdot y_0}}}+O(d)\cdot e^{-\Lambda^{(t)}_{5,j'_1,r_{g_2\cdot y_0}}}}\\
    &\stackrel{(a)}{\geq}  \frac{1}{1+e^{2(c_1+c_2)B}+O(d)\cdot e^{-\big(\frac{1}{3}-2c_1-2c_2\big)B}} \geq \Omega\Big(\frac{1}{d^{2C_B(c_1+c_2)}}\Big),
   \end{align*}
   where the inequality (a) follows from \Cref{lem-s22-lambda-cyc} and the last inequality is due to the fact that $(c_1+c_2)$ is some sufficiently small constant s.t., $ e^{-\big(\frac{1}{3}-2c_1-2c_2\big)B}=1/\poly d.$
   \item For $\ell=2$, we have 
   \begin{align*}
    \logit^{(t)}_{j_2}&= \frac{1}{1+e^{-\Lambda^{(t)}_{5,j_2,r_{g_2\cdot y_1}}+\Lambda^{(t)}_{5,j'_2,r_{g_2\cdot y_0}}}+O(d)\cdot e^{-\Lambda^{(t)}_{5,j_2,r_{g_2\cdot y_1}}}}\\
    &\stackrel{(a)}{\leq}  \frac{1}{1+O(d)\cdot e^{-4c_5B}} = O\bigg(\frac{1}{d^{1-4c_5C_B}}\bigg),
   \end{align*}
   where the inequality (a) follows from \Cref{lem-s22-lambda-cyc}. Similarly, we have
   \begin{align*}
    \logit^{(t)}_{j'_2}&= \frac{1}{1+e^{\Lambda^{(t)}_{5,j_2,r_{g_2\cdot y_1}}-\Lambda^{(t)}_{5,j'_2,r_{g_2\cdot y_0}}}+O(d)\cdot e^{-\Lambda^{(t)}_{5,j'_2,r_{g_2\cdot y_0}}}}\\
    & \leq \frac{1}{1+e^{\Omega(1)}+O(d)\cdot e^{-c_6B}} = O\bigg(\frac{1}{d^{1-c_6C_B}}\bigg).
   \end{align*}
    \end{itemize}
\end{proof}
\begin{lemma}\label{lem-s22-act-other-cyc}
     If \Cref{induction-s22-cyc} holds for all iterations $\in [T_{2,1,s}, t)$, given input $\Zb^{2,\ell-1}$, then we have
         \begin{enumerate}
        \item for $\ell=1$, if $\Zb^{2,\ell-1}\notin \cE_1$, then
  \begin{enumerate}[(a)]
            \item for $j=j_1$, $\hat{\fA}_{j_1}=\{r_{g_1\cdot y_0}\}$, $\Lambda^{(t)}_{5,j_1,r_{g_1\cdot y_0}}= B\pm O(\delta)$;
            \item for $j\neq j_1\in\tau(\Y)$, assuming $j=\tau(g_1(y))$, then $\Lambda^{(t)}_{5,j,r_{g_1\cdot y}}\in \bigg[\Big(\frac{1}{3}-c_2\Big)B,\Big(\frac{1}{3}+c_1\Big)B\bigg]$ and    $\Lambda^{(t)}_{5,j,r}\ll -\varrho$ for $r\in \hat{\fA}_{j}\setminus\{r_{g_1\cdot y}\}$.
        \end{enumerate}

       \item $\ell=2$, if  $\Zb^{2,\ell-1}\notin \tilde{\cE}_2$, then
                \begin{enumerate}[(a)]
                    \item if $g_1=g_2\wedge y_0\neq y_1$,
                    \begin{enumerate}
                        \item for $j=j_2$, $\Lambda^{(t)}_{5,j_2,r_{g_2\cdot y_1}}\in \Big[ \frac{1}{2}B+\Omega(1), \big(\frac{1}{2}+2c_5\big)B\Big]$ and $\Lambda^{(t)}_{5,j_2,r}\ll -\varrho$ for $r\in \hat{\fA}_{j_2}\setminus\{r_{g_2\cdot y_1}\}$;
                        \item for $j=\tau(g_2(y_0))$, $\Lambda^{(t)}_{5,j_2,r_{g_2\cdot y_1}}-\Lambda^{(t)}_{5,j_2,r_{g_2\cdot y_0}}\geq \Omega(1)$ and $\Lambda^{(t)}_{5,j,r}\ll -\varrho$ for $r\in \hat{\fA}_{j}\setminus\{r_{g_2\cdot y_0}\}$;
                        \item for other $j\in\tau(\cY)$, assuming $j=\tau(g_2(y))$,   $\Lambda^{(t)}_{5,j,r_{g_2\cdot y}}\in [\Omega(1),c_6 B]$ and $\Lambda^{(t)}_{5,j,r}\ll -\varrho$ for $r\in \hat{\fA}_{j}\setminus\{r_{g_2\cdot y}\}$.
                    \end{enumerate}
                  \item if $g_1\neq g_2\wedge y_0= y_1$,
                    \begin{enumerate}
                        \item for $j=j_2$, $\Lambda^{(t)}_{5,j_2,r_{g_2\cdot y_1}}\in \Big[ \frac{1}{2}B+\Omega(1), \big(\frac{1}{2}+2c_5\big)B\Big]$ and $\Lambda^{(t)}_{5,j_2,r}\ll -\varrho$ for $r\in \hat{\fA}_{j_2}\setminus\{r_{g_2\cdot y_1}\}$;
                        \item for $j=\tau(g_1(y_1))$, $\Lambda^{(t)}_{5,j_2,r_{g_2\cdot y_1}}-\Lambda^{(t)}_{5,j_2,r_{g_1\cdot y_1}}\geq \Omega(1)$ and $\Lambda^{(t)}_{5,j,r}\ll -\varrho$ for $r\in \hat{\fA}_{j}\setminus\{r_{g_1\cdot y_1}\}$;
                        \item for other $j\in\tau(\Y)$, $r$ is not activted for all $r\in \hat{\fA}_{j}$, i.e., $\Lambda^{(t)}_{5,j,r}\ll -\varrho$.
                    \end{enumerate}
                \item if $g_1=g_2\wedge y_0= y_1$,
                    \begin{enumerate}
                        \item for $j=j_2$, $ \hat{\fA}_{j_2}=\{r_{g_2\cdot y_1}\}$, $\Lambda^{(t)}_{5,j_2,r_{g_2\cdot y_1}}= B\pm O(\delta)$ ;
                        \item for other $j\in\tau(\cY)$, assuming $j=\tau(g_2(y))$,   $\Lambda^{(t)}_{5,j,r_{g_2\cdot y}}\in [\Omega(1),c_6 B]$ and $\Lambda^{(t)}_{5,j,r}\ll -\varrho$ for $r\in \hat{\fA}_{j}\setminus\{r_{g_2\cdot y}\}$.
                    \end{enumerate}
        \end{enumerate}
    \end{enumerate}
\end{lemma}

\begin{lemma}\label{lem-s22-logit-other-cyc}
     If \Cref{induction-s22-cyc} holds for all iterations $\in [T_{2,1,s}, t)$, given input $\Zb^{2,\ell-1}$, then we have
         \begin{enumerate}
        \item for $\ell=1$, if $\Zb^{2,\ell-1}\notin \cE_1$, then $1-\logit^{(t)}_{5,j_1}=O\Big(\frac{1}{\poly d}\Big)$.

       \item $\ell=2$, if  $\Zb^{2,\ell-1}\notin \tilde{\cE}_2$, then
                \begin{enumerate}[(a)]
                    \item if $g_1=g_2\wedge y_0\neq y_1$, $\logit^{(t)}_{5,j_2}=\Omega(1)$, $1-\logit^{(t)}_{5,j_2}-\logit^{(t)}_{5,\tau(g_2(y_0))}=\frac{1}{\poly d}$. 
                  \item if $g_1\neq g_2\wedge y_0= y_1$, $\logit^{(t)}_{5,j_2}=\Omega(1)$, $1-\logit^{(t)}_{5,j_2}-\logit^{(t)}_{5,\tau(g_1(y_1))}=\frac{1}{\poly d}$. 
                \item if $g_1=g_2\wedge y_0= y_1$, $1-\logit^{(t)}_{5,j_2}=O\Big(\frac{1}{\poly d}\Big)$.
        \end{enumerate}
    \end{enumerate}

\end{lemma}

\subsubsection{Gradient Lemma}
\begin{lemma}\label{lem-s22-gd1-cyc}
    If \Cref{induction-s22-cyc} holds for all iterations $\in [T_{2,1,s}, t)$, given $s\in\tau(\X)$,  for $[\Qb^{(t)}_{4,3}]_{s,s}$, we have
    \begin{align*}
       \sum_{\ell=1}^2 \Big[-\nabla_{\Q^{(t)}_{4,3}}\Loss^{2,\ell}_{5}\Big]_{s,s}= \Theta\Big(\frac{\log d}{d}\Big).
    \end{align*}

\end{lemma}
\begin{proof}
The proof is similar to \Cref{lem-s21-gd1-cyc}, but we need to shift our focus to $\Big[-\nabla_{\Q^{(t)}_{4,3}}\Loss^{2,2}_{5}\Big]_{s,s}$. By \Cref{lem-s22-act-cyc} and \Cref{lem-s22-act-other-cyc}, we have
    \begin{align*}
       \cN^{(t)}_{s,3,2,\rom1}&= \E\Bigg[
    \attn^{(t)}_{{\ans,1} \rightarrow \pred,2} \cdot (1-\logit^{(t)}_{5,j_2})\cdot \\
    &~~~~~~~~~  \Big(
 \Big( V_{j_2,  r_{g_2\cdot y_1}}(g_2)- \Lambda^{(t)}_{5,j_2, r_{g_2\cdot y_1}}\pm\tilde{O}(\sigma_0) \Big)\pm \tilde{O}(\delta^{q}) \Big) 
    \1_{\{\tau(x_1)=s\}\cap \tilde{\cE}_2}\Bigg]\\
    &+ \E\Bigg[
    \attn^{(t)}_{{\ans,1} \rightarrow \pred,2} \cdot (1-\logit^{(t)}_{5,j_2})\cdot \\
    &~~~~~~~~~  \Big(
 \Big( V_{j_2,  r_{g_2\cdot y_1}}(g_2)- \Lambda^{(t)}_{5,j_2, r_{g_2\cdot y_1}}\pm\tilde{O}(\sigma_0) \Big)\pm \tilde{O}(\delta^{q}) \Big) 
    \1_{\{\tau(x_1)=s\}\cap \tilde{\cE}^c_2}\Bigg]\\
    &\stackrel{(a)}{=}\Theta\big(\frac{1}{d}\big)\cdot \Omega(1)\cdot \Theta(B)+ \Theta(\frac{1}{d})\cdot O(1)\cdot  \Theta(B)\cdot O(\frac{1}{\log d})\\
    &=\Theta\Big(\frac{\log d}{d}\Big),
\end{align*}
where the inequality (a) follows from \Cref{lem-s22-logit-cyc} and \Cref{lem-s22-act-other-cyc}.

Furthermore, for $\cN^{(t)}_{s,3,2,\rom2}$, we have
\begin{align*}
       &\Big|\cN^{(t)}_{s,3,2,\rom2}\Big|\\
       &\leq  \Bigg|\E\bigg[
    \attn^{(t)}_{{\ans,1} \rightarrow \pred,2} \cdot\logit^{(t)}_{5,j'_2} \cdot \\
    &~~~~~~~~~~\Big(\Big( V_{j'_2, r_{g_2\cdot y_0}}(g_2)- \Lambda^{(t)}_{5,j,r}\pm\tilde{O}(\sigma_0) \Big)\pm \tilde{O}(\delta^{q}) \Big) 
\1_{\{\tau(x_1)=s\}\cap \tilde{\cE}_2}\bigg]\Bigg|\\
&+\Bigg|\E\bigg[
    \attn^{(t)}_{{\ans,1} \rightarrow \pred,2} \cdot\sum_{j\neq j_2\in\tau(\cY)}\logit^{(t)}_{5,j} \cdot \\
    &~~~~~~~~~~\Big(\sum_{r\in\hat{\fA}_{j}}\ReLU^{\prime}(\Lambda^{(t)}_{5,j,r
    })\cdot  \Big( V_{j, r}(g_2)- \Lambda^{(t)}_{5,j,r}\pm\tilde{O}(\sigma_0) \Big)\pm \tilde{O}(\delta^{q}) \Big) 
\1_{\{\tau(x_1)=s\}\cap \tilde{\cE}^c_2}\bigg]\Bigg|\\
&\leq \Theta\big(\frac{1}{d}\big)\cdot O\Big(\frac{1}{d^{1-c_6C_B }}\Big)\cdot \Theta(B)+ \Theta\big(\frac{1}{d}\big)\cdot O(1)\cdot  \Theta(B)\cdot O\Big(\frac{1}{\log d}\Big)\leq O(\frac{1}{d}). 
\end{align*}
 $|\cN^{(t)}_{s,3,1,\rom1}|$ and $|\cN^{(t)}_{s,3,2,\rom1}|$ can be upper bounded by $O\big(\frac{\log d}{d}\big)$ as \Cref{lem-s21-gd1-cyc}. Thus, we complete the proof.
\end{proof}
\begin{lemma}\label{lem-s22-gd2-cyc}
    If \Cref{induction-s22-cyc} holds for all iterations $\in [T_{2,1,s}, t)$, given $s\in\tau(\X)$, we have
    \begin{align*}
\sum_{\ell=1}^2 \Big[-\nabla_{\Q^{(t)}_{4,4}}\Loss^{2,\ell}_{5}\Big]_{s,s} =\Theta\Big(\frac{\log d}{d}\Big).
    \end{align*}

\end{lemma}
\begin{proof}
    The proof follows a similar analysis to \Cref{lem-s21-gd4-cyc}, and we thus omit the details here.
\end{proof}

\begin{lemma}\label{lem-s22-gd3-cyc}
    If \Cref{induction-s22-cyc} holds for all iterations $\in [T_{2,1,s}, t)$, given $s\in\tau(\X)$,  we have
    \begin{align*}
     &\sum^{t}_{t'=T_{2,1,s}}   \bigg(\sum_{\ell=1}^2 \Big[-\nabla_{\Q^{(t')}{4,3}}\Loss^{2,\ell}_{5}\Big]_{s,s}- \sum_{\ell=1}^2 \Big[-\nabla_{\Q^{(t')}_{4,4}}\Loss^{2,\ell}_{5}\Big]_{s,s}\bigg) \\
     &~~~~~~~~\geq  -O\big(\frac{1}{\log d}\big)\cdot \bigg(\sum^{t}_{t'=T_{2,1,s}}\sum_{\ell=1}^2 \Big[-\nabla_{\Q^{(t')}_{4,4}}\Loss^{2,\ell}_{5}\Big]_{s,s}\bigg).
    \end{align*}

\end{lemma}
\begin{proof}
   Following the analogous analysis as \Cref{lem-s21-gd3-cyc}, we have
    \begin{align*}
           & \sum_{\ell=1}^2 \Big[-\nabla_{\Q^{(t)}_{4,3}}\Loss^{2,\ell}_{5}\Big]_{s,s}+ \sum_{\ell=1}^2 \Big[\nabla_{\Q^{(t)}_{4,4}}\Loss^{2,\ell}_{5}\Big]_{s,s}= \sum_{\ell\in [2]}\sum_{\kappa\in \{\rom1,\rom2,\rom3\}}\cN^{(t)}_{s,3,\ell,\kappa}-\cN^{(t)}_{s,4,\ell,\kappa}.
    \end{align*}
Meanwhile $|\cN^{(t)}_{s,p,\ell,\rom3}|=O(\frac{1}{\poly d})$ for $p\in\{3,4\}$ and $\ell\in[2]$, we can focus on the gradient difference between $\Q^{(t)}_{4,3}$ and $\Q^{(t)}_{4,4}$ contributed by other terms.

For $\ell=1$, since $\attn^{(t)}_{{\ans,0} \rightarrow \pred,1} \geq \attn^{(t)}_{{\ans,0} \rightarrow \ans,0}$, and thus it is straightforward to see that $\cN^{(t)}_{s,3,1,\rom1}-\cN^{(t)}_{s,4,1,\rom1}\geq 0$. Furthermore, we have
   \begin{align*}
     & \cN^{(t)}_{s,3,1,\rom2}-\cN^{(t)}_{s,4,1,\rom2}\\
     &= -\E\Bigg[
    \attn^{(t)}_{{\ans,0} \rightarrow \pred,1} \cdot \logit^{(t)}_{5,j'_1} \cdot \\
    &~~~~~~~~~~\Big(  \Big( V_{j'_1,r_{g_2\cdot y_0}}(g_1)- \Lambda^{(t)}_{5,j'_1,r_{g_2\cdot y_0}}\pm\tilde{O}(\sigma_0) \Big)\pm \tilde{O}(\delta^{q}) \Big) 
\1_{\{\tau(x_0)=s\}\cap \cE_1}\Bigg]\\
&+\E\Bigg[
    \attn^{(t)}_{{\ans,0} \rightarrow \ans,0} \cdot \logit^{(t)}_{5,j'_1} \cdot \\
    &~~~~~~~~~~\Big(  \Big( V_{j'_1,r_{g_2\cdot y_0}}(y_0)- \Lambda^{(t)}_{5,j'_1,r_{g_2\cdot y_0}}\pm\tilde{O}(\sigma_0) \Big)\pm \tilde{O}(\delta^{q}) \Big) 
\1_{\{\tau(x_0)=s\}\cap \cE_1}\Bigg]\\
& -\E\Bigg[
    \attn^{(t)}_{{\ans,0} \rightarrow \pred,1} \cdot\sum_{y\neq y_0\in\cY}\logit^{(t)}_{5,\tau(g_1(y))} \cdot \\
    &~~~~~~~~~~\Big( \Big( V_{\tau(g_1(y)), r_{g_1\cdot y}}(g_1)- \Lambda^{(t)}_{5,\tau(g_1(y)), r_{g_1\cdot y}}\pm\tilde{O}(\sigma_0) \Big)\pm \tilde{O}(\delta^{q}) \Big) 
\1_{\{\tau(x_0)=s\}\cap \cE^c_1}\Bigg]\\
& +\E\Bigg[
    \attn^{(t)}_{{\ans,0} \rightarrow \ans,0} \cdot\sum_{y\neq y_0\in\cY}\logit^{(t)}_{5,\tau(g_1(y))} \cdot \\
    &~~~~~~~~~~\Big( \Big( V_{\tau(g_1(y)), r_{g_1\cdot y}}(y_0)- \Lambda^{(t)}_{5,\tau(g_1(y)), r_{g_1\cdot y}}\pm\tilde{O}(\sigma_0) \Big)\pm \tilde{O}(\delta^{q}) \Big) 
\1_{\{\tau(x_0)=s\}\cap \cE^c_1}\Bigg]\\
&\stackrel{(a)}{\geq} \Theta(\frac{1}{d})\cdot \Omega\Big(\frac{1}{d^{2C_B(c_1+c_2)}}\Big)\cdot \Theta(B)- \Theta(\frac{1}{d})\cdot O(\frac{1}{\poly d})\cdot  O(B)\cdot O(\frac{1}{\log d})\geq \Omega\Big(\frac{\log d}{d^{1+2C_B(c_1+c_2)}}\Big).
\end{align*}
    where the  inequality (a) is due to \Cref{lem-s22-logit-cyc} and \Cref{lem-s22-logit-other-cyc}.

For $\ell=2$, since $\attn^{(t)}_{{\ans,\pred, 2} \rightarrow \pred,1} \geq \attn^{(t)}_{{\ans,1} \rightarrow \ans,1}$, and thus it is straightforward to see that $\cN^{(t)}_{s,3,2,\rom1}-\cN^{(t)}_{s,4,2,\rom1}\geq 0$. Moreover, we have
\begin{align}
       &\cN^{(t)}_{s,3,2,\rom2}-\cN^{(t)}_{s,4,2,\rom2}\notag\\
       &= -\E\Bigg[
    \attn^{(t)}_{{\ans,1} \rightarrow \pred,2} \cdot \logit^{(t)}_{5,j'_2} \cdot  \notag \\
    &~~~~~~~~~~\Big( \Big( V_{j'_2, r_{g_2\cdot y_0}}(g_2)- \Lambda^{(t)}_{5,j'_2,r_{g_2\cdot y_0}}\pm\tilde{O}(\sigma_0) \Big)\pm \tilde{O}(\delta^{q}) \Big) 
\1_{\{\tau(x_1)=s\}\cap \tilde{\cE}_2}\Bigg] \label{eq-s22-cyc-gd1}\\
&+\E\Bigg[
    \attn^{(t)}_{{\ans,1} \rightarrow \ans,1} \cdot\logit^{(t)}_{5,j'_2} \cdot \notag \\
    &~~~~~~~~~~\Big(\Big( V_{j'_2, r_{g_2\cdot y_0}}(y_1)- \Lambda^{(t)}_{5,j'_2,r_{g_2\cdot y_0}}\pm\tilde{O}(\sigma_0) \Big)\pm \tilde{O}(\delta^{q}) \Big) 
\1_{\{\tau(x_1)=s\}\cap \tilde{\cE}_2}\Bigg] \label{eq-s22-cyc-gd2}\\
&-\E\Bigg[
    \attn^{(t)}_{{\ans,1} \rightarrow \pred,2} \cdot\sum_{j\neq j_2\in\tau(\cY)}\logit^{(t)}_{5,j} \cdot \notag\\
    &~~~~~~~~~~\Big(\sum_{r\in\hat{\fA}_{j}}\ReLU^{\prime}(\Lambda^{(t)}_{5,j,r
    })\cdot  \Big( V_{j, r}(g_2)- \Lambda^{(t)}_{5,j,r}\pm\tilde{O}(\sigma_0) \Big)\pm \tilde{O}(\delta^{q}) \Big) 
\1_{\{\tau(x_1)=s\}\cap \tilde{\cE}^c_2}\Bigg] \label{eq-s22-cyc-gd3}\\
&+\E\Bigg[
    \attn^{(t)}_{{\ans,1} \rightarrow \ans,1} \cdot\sum_{j\neq j_2\in\tau(\cY)}\logit^{(t)}_{5,j} \cdot \notag\\
    &~~~~~~~~~~\Big(\sum_{r\in\hat{\fA}_{j}}\ReLU^{\prime}(\Lambda^{(t)}_{5,j,r
    })\cdot  \Big( V_{j, r}(y_1)- \Lambda^{(t)}_{5,j,r}\pm\tilde{O}(\sigma_0) \Big)\pm \tilde{O}(\delta^{q}) \Big) 
\1_{\{\tau(x_1)=s\}\cap\tilde{\cE}^c_2}\Bigg]. \label{eq-s22-cyc-gd4}
\end{align}
By \Cref{lem-s22-logit-cyc}, we obtain that
\begin{align*}
    \eqref{eq-s22-cyc-gd1}+\eqref{eq-s22-cyc-gd2}\geq  -O\big(\frac{1}{d}\big)\cdot O\Big(\frac{1}{d^{1-c_6C_B }}\Big)\cdot \Theta(B)\geq-O\Big(\frac{\log d}{d^{2-c_6C_B }}\Big).
\end{align*}
By \Cref{lem-s22-act-other-cyc}, for $\tilde{\cE}^c_2$, we only need to consider the case that $g_1=g_2\wedge y_0\neq y_1$, and we have
\begin{align*}
   &\eqref{eq-s22-cyc-gd3}+\eqref{eq-s22-cyc-gd4} \\&\geq -\E\Bigg[
    \attn^{(t)}_{{\ans,1} \rightarrow \pred,2} \cdot \sum_{y\neq y_1\in\cY}\logit^{(t)}_{5,\tau(g_2(y))}  \cdot \notag \\
    &~~~\Big(\Big( V_{\tau(g_2(y)), r_{g_2\cdot y}}(g_2)- \Lambda^{(t)}_{5,\tau(g_2(y)),r_{g_2\cdot y}}\pm\tilde{O}(\sigma_0) \Big)\pm \tilde{O}(\delta^{q}) \Big) 
\1_{\{\tau(x_1)=s\}\cap \{g_1=g_2\wedge y_0\neq y_1\}}\Bigg] \\
&+\E\Bigg[
    \attn^{(t)}_{{\ans,1} \rightarrow \ans,1} \cdot\sum_{y\neq y_1\in\cY}\logit^{(t)}_{5,\tau(g_2(y))} \cdot \notag\\
    &~~~\Big( \Big( V_{\tau(g_2(y)), r_{g_2\cdot y}}(y_1)- \Lambda^{(t)}_{5,\tau(g_2(y)),r_{g_2\cdot y}}\pm\tilde{O}(\sigma_0) \Big)\pm \tilde{O}(\delta^{q}) \Big) 
\1_{\{\tau(x_1)=s\}\cap \{g_1=g_2\wedge y_0\neq y_1\}}\Bigg]\\
&\geq -\Theta\Big(\frac{1}{d}\Big)\cdot O(1)\cdot \Theta(B)\cdot O\Big(\frac{1}{\log d}\Big)\geq - O\Big(\frac{1}{d}\Big).
\end{align*}
    Putting it all together, combining with the fact that $c_6$ and $(c_1+c_2)$ are sufficiently small,  we finish the proof.
\end{proof}
\begin{lemma}[Lower bound of gap]\label{lem-s22-gd4-cyc}
    If \Cref{induction-s22-cyc} holds for all iterations $\in [T_{2,1,s}, t)$, given $s \in\tau(\X)$, we have $[\Qb^{(t)}_{4,3}]_{s,s}-[\Qb^{(t)}_{4,4}]_{s,s}
    \geq \Omega\big(\frac{1}{\log d}\big)$.
\end{lemma}
\begin{proof}
    Letting $\tilde{T}$ denote the first time that $[\Qb^{(t)}_{4,3}]_{s,s}-[\Qb^{(t)}_{4,4}]_{s,s}
    \leq \frac{1}{2}\Big( \Qb^{(T_{2,1,s})}_{4,3}]_{s,s}-[\Qb^{(T_{2,1,s})}_{4,4}]_{s,s}\Big)$, which implies that 
    $$\sum^{\tilde{T}}_{t'=T_{2,1,s}}   \bigg(\sum_{\ell=1}^2 \Big[-\nabla_{\Q^{(t')}{4,3}}\Loss^{2,\ell}_{5}\Big]_{s,s}- \sum_{\ell=1}^2 \Big[-\nabla_{\Q^{(t')}_{4,4}}\Loss^{2,\ell}_{5}\Big]_{s,s}\bigg).$$
    Hence, by \Cref{lem-s22-gd3-cyc}, we have $[\Qb^{(\tilde{T})}_{4,3}]_{s,s}$ and $[\Qb^{(\tilde{T})}_{4,3}]_{s,s}$ reaches $\Omega(1)$.  Thus, we can have a refined lower bound  for $\eqref{eq-s22-cyc-gd3}+\eqref{eq-s22-cyc-gd4}$ in \Cref{lem-s22-gd3-cyc}, and obtain:
    \begin{align*}
   \eqref{eq-s22-cyc-gd3}&+\eqref{eq-s22-cyc-gd4} \geq -O(\frac{1}{\poly d})\\&-\E\Bigg[
    \attn^{(t)}_{{\ans,1} \rightarrow \pred,2} \cdot \logit^{(t)}_{5,\tau(g_2(y_0))}  \cdot \notag \\
    &~~~\Big(\Big( V_{\tau(g_2(y_0)), r_{g_2\cdot y_0}}(g_2)- \Lambda^{(t)}_{5,\tau(g_2(y_0)),r_{g_2\cdot y_0}}\pm\tilde{O}(\sigma_0) \Big)\pm \tilde{O}(\delta^{q}) \Big) 
\1_{\{\tau(x_1)=s\}\cap \{g_1=g_2\wedge y_0\neq y_1\}}\Bigg] \\
&+\E\Bigg[
    \attn^{(t)}_{{\ans,1} \rightarrow \ans,1} \cdot \logit^{(t)}_{5,\tau(g_2(y_0))} \cdot \notag\\
    &~~~\Big( \Big( V_{\tau(g_2(y_0)), r_{g_2\cdot y_0}}(y_1)- \Lambda^{(t)}_{5,\tau(g_2(y_0)),r_{g_2\cdot y_0}}\pm\tilde{O}(\sigma_0) \Big)\pm \tilde{O}(\delta^{q}) \Big) 
\1_{\{\tau(x_1)=s\}\cap \{g_1=g_2\wedge y_0\neq y_1\}}\Bigg]\\
&\stackrel{(a)}{\geq} -O\Big(\frac{1}{\poly d}\Big) -\Theta\Big(\frac{1}{d}\Big)\cdot O\Big(\frac{1}{d^{\Omega(1)}}\Big)\cdot \Theta(B)\cdot O\Big(\frac{1}{\log d}\Big)\geq - O\Big(\frac{1}{d^{1+\Omega(1)}}\Big),
\end{align*}
where (a) follows from the fact that on the event \( \{g_1 = g_2 \wedge y_0 \neq y_1\} \), we have $\Lambda^{(t)}_{5,j_2,r_{g_2 \cdot y_1}} - \Lambda^{(t)}_{5,j_2,r_{g_2 \cdot y_0}} \geq \Omega(B)$
once \( [\Qb^{(t)}_{4,3}]_{s,s} \) and \( [\Qb^{(t)}_{4,4}]_{s,s} \) reach constant magnitude, and  consequently, the logit satisfies $\logit^{(t)}_{5,\tau(g_2(y_0))} \leq O\left( \frac{1}{d^{\Omega(1)}} \right).$

Therefore, 
$$\sum^{t}_{t'=\tilde{T}+1}   \bigg(\sum_{\ell=1}^2 \Big[-\nabla_{\Q^{(t')}{4,3}}\Loss^{2,\ell}_{5}\Big]_{s,s}- \sum_{\ell=1}^2 \Big[-\nabla_{\Q^{(t')}_{4,4}}\Loss^{2,\ell}_{5}\Big]_{s,s}\bigg)\geq -O\Big(\frac{1}{\log d\cdot  d^{\Omega(1)}}\Big)\cdot O(1),$$
which means that $[\Qb^{(t)}_{4,3}]_{s,s}-[\Qb^{(t)}_{4,4}]_{s,s}\geq \Omega\Big(\frac{1}{\log d}\Big)-O\Big(\frac{1}{\log d\cdot  d^{\Omega(1)}}\Big)\geq \Omega\Big(\frac{1}{\log d}\Big).$ 
\end{proof}
\begin{lemma}\label{lem-s22-gd5-cyc}
    If \Cref{induction-s22-cyc} holds for all iterations $\in [T_{2,1,s}, t)$, given $s'\not=s \in\tau(\X)$,  for $p\in\{3,4\}$, we have
    \begin{align*}
        \bigg|\sum_{\ell=1}^2 \Big[-\nabla_{\Q^{(t)}_{4,p}}\Loss^{2,\ell}_{5}\Big]_{s,s'}\bigg|\leq  O\Big(\frac{\log d}{d^2}\Big)=O\Big(\frac{1}{d}\Big)\cdot \bigg|\sum_{\ell=1}^2 \Big[-\nabla_{\Q^{(t)}_{4,p}}\Loss^{2,\ell}_{5}\Big]_{s,s}\bigg|.
    \end{align*}

\end{lemma}

\subsubsection{At the End of Stage 2.2}
Putting gradient lemmas together, we can directly prove that \Cref{induction-s22-cyc} holds for all iterations $t$ until the end of stage 2.2, where we can conclude the following:
\begin{lemma}[End of Stage 2.2]\label{lem-end-s22-cyc}
  Given $s\in\tau(\X)$, \Cref{induction-s22-cyc} holds for all iterations $T_{2,1,s}\leq t<T_{2,2,s}=O\Big(\frac{d}{\eta \log d}\Big)$, then at the end of stage 2.2, we have
 \begin{enumerate}[(a)]
    \item $[\Qb^{(t)}_{4,p}]_{s,s}=\Omega(1)$ for $p\in\{3,4\}$;
    \item $[\Qb^{(t)}_{4,3}]_{s,s}-[\Qb^{(t)}_{4,4}]_{s,s}\in \Big[ \Omega\big(\frac{1}{\log d}\big), O(1)\Big]$;
    \item $|[\Qb^{(t)}_{4,p}]_{s,s'}|\leq O\Big(\frac{[\Qb^{(t)}_{4,p}]_{s,s}}{d}\Big)$ for $s'\in\tau(\X)\not=s$, and other $[\Qb_{4,p}]_{s,s'}=0$.
 \end{enumerate}
    
\end{lemma}

\subsection{Stage 2.3: Decrease of Gap and Convergence}
After rapid growth of diagonal entries in stage 2.2, we now focus on the convergence of the attention and logit matrices, and the decrease of the gap between $[\Qb^{(t)}_{4,3}]_{s,s}$ and $[\Qb^{(t)}_{4,4}]_{s,s}$. Recall that 
\begin{align*}
        \epsilon^{L,\ell}_{\mathsf{attn}}\big(\Zb^{L,\ell-1}\big) &= 1 - \attn_{\ans,\ell-1 \to \pred,\ell}\big(\Zb^{L,\ell-1}\big) - \attn_{\ans,\ell-1 \to \ans,\ell-1}\big(\Zb^{L,\ell-1}\big),\\
    \Delta^{L,\ell}\big(\Zb^{L,\ell-1}\big)& =  \attn_{\ans,\ell-1 \to \pred,\ell}\big(\Zb^{L,\ell-1}\big) - \attn_{\ans,\ell-1 \to \ans,\ell-1}\big(\Zb^{L,\ell-1}\big). 
\end{align*}
Throught stage 2.3, we will focus on the attention gap $ \Delta^{L,\ell}\big(\Zb^{L,\ell-1}\big)$ instead of the gap of the attention matrices. We abbreviate. $\epsilon^{L,\ell}_{\mathsf{attn}}\big(\Zb^{L,\ell-1}\big)$ and $ \Delta^{L,\ell}\big(\Zb^{L,\ell-1}\big)$ as $\epsilon^{L,\ell}_{\mathsf{attn}}$ and $\Delta^{L, \ell}$ for simplicity.
\begin{induction}\label{induction-s23-cyc}
   Given $\epsilon\geq \tilde{\Omega}(\sigma_0)$,  for $s\in\tau(\X)$,  let $T_{2,3,s}$ denote the first time that  $\E\big[ \epsilon^{2,2}_{\mathsf{attn}}\mid \tau(x_1)=s\big]\leq \epsilon$. 
       For all iterations $T_{2,2,s}\leq t < T_{2,3,s}$, we have the following holds:
       \begin{enumerate}[(a)]
          \item $[\Qb^{(t)}_{4,3}]_{s,s}$ and $[\Qb^{(t)}_{4,4}]_{s,s}$ monotonically increases $\leq \tilde{O}(1)$;
          \item ${\Delta^{2,\ell}}\geq 0$ for any $\Zb^{2,\ell}$ with $\ell\in\{1,2\}$;
          \item  $\attn^{(t)}_{\ans,1\to\pred,2}\leq 0.5+\tilde{c}_1$ for some small constant $\tilde{c}_1>0$;
          \item for $(p,q)\in\{(4,3), (4,4)\}$, $\Big|[\Qb^{(t)}_{p,q}]_{s,s'}\Big|\leq O\Big(\frac{[\Qb^{(t)}_{p,q}]_{s,s}}{d}\Big)$ for $s'\in\tau(\X)\not=s$; other $[\Qb^{(t)}_{p,q}]_{s,s'}=0$.
       \end{enumerate}
    \end{induction}
    \subsubsection{Attention and Logit Preliminaries}
  \begin{lemma}\label{lem-s23-attn-cyc}
      If \Cref{induction-s23-cyc} holds for all iterations $\in [T_{2,2,s}, t)$, given $\Zb^{2,\ell-1}$ then we have 
  \begin{enumerate}
      \item for $\ell=1$,  
      \begin{enumerate}[(a)]
      \item $ \attn^{(t)}_{\ans,0\to \ans, 0}\geq \Omega(1)$, and  $ \attn^{(t)}_{\ans,0\to \ans, 0} >\attn^{(t)}_{\ans,0\to \pred, 2}$;
      \end{enumerate}
        \item for $\ell=2$, 
        \begin{enumerate}
          \item $\attn^{(t)}_{\ans,1\to \ans, 1}\geq \Omega(1)$, $\attn^{(t)}_{\ans,1\to \pred, 1},\attn^{(t)}_{\ans,1\to \ans, 0}\leq \attn^{(t)}_{\ans,1\to \ans, 1}$;
          \item $\Big|\attn^{(t)}_{\ans,1\to \ans, 0}- \attn^{(t)}_{\ans,1\to \pred, 1}\Big|\leq \tilde{O}\big(\frac{1}{d}\big)$. 
          \end{enumerate}
  \end{enumerate}
   Moreover, given $\Z^{2,1}$ and corresponding $\Z^{2,0}$,  $ \attn^{(t)}_{\ans,0\to \pred, 1}\big(\Z^{2,0}\big)\geq \attn^{(t)}_{\ans,1\to \pred, 2}\big(\Z^{2,1}\big)$.
  \end{lemma}

\begin{lemma}\label{lem-s23-act-cyc}
    If \Cref{induction-s23-cyc} holds for all iterations $\in [T_{2,2,s}, t)$, given input $\Zb^{2,\ell-1}$, then we have
    \begin{enumerate}
        \item for $\ell=1$, if $\Zb^{2,\ell-1}\in \cE_1$, then
  \begin{enumerate}[(a)]
            \item for $j=j_1$, 
       $\Lambda^{(t)}_{5,j_1,r}\ll -\varrho$  for  $r\in \hat{\fA}_{j_1}\setminus\{r_{g_1\cdot y_0}\}$; 
            \item for $j=j'_1\triangleq \tau\big(g_2(y_0)\big)=\tau(g_1(\tilde{y}))$, 
            $\Lambda^{(t)}_{5,j'_1,r}\ll -\varrho$ for  $r\in \hat{\fA}_{j'_1}\setminus\{r_{g_2\cdot y_0}, r_{g_1\cdot \tilde{y}}\}$;
            \item for other $j\in\tau(\Y)$, assuming  $j=\tau(g_1(y))$ for some $y\neq y_0$, then 
            $\Lambda^{(t)}_{5,j,r}\ll -\varrho$ for  $r\in \hat{\fA}_{j}\setminus\{r_{g_1\cdot y}\}$.
        \end{enumerate}

       \item $\ell=2$, if  $\Zb^{2,\ell-1}\in \tilde{\cE}_2$, then
                \begin{enumerate}[(a)]
                    \item for $j=j_2$,  $\Lambda^{(t)}_{5,j_2,r}\ll -\varrho$  for  $r\in \hat{\fA}_{j_2}\setminus\{r_{g_2\cdot y_1}\}$;
  \item for $j=j'_2\triangleq \tau\big(g_2(y_0)\big)$, 
            $\Lambda^{(t)}_{5,j'_2,r}\ll -\varrho$ for  $r\in \hat{\fA}_{j'_2}\setminus\{r_{g_2\cdot y_0}\}$;
            \item for other $j\in\tau(\Y)$, if $j=\tau(g_2(y))$ for some $y\in \cY$, then 
$\Lambda^{(t)}_{5,j,r}\ll -\varrho$ for  $r\in \hat{\fA}_{j}\setminus\{r_{g_2\cdot y}\}$.
        \end{enumerate}
    \end{enumerate}
        
\end{lemma}

\begin{lemma}\label{lem-s23-lambda-cyc}
    If \Cref{induction-s23-cyc} holds for all iterations $\in [T_{2,2,s}, t)$, given input $\Zb^{2,\ell-1}$, then we have
    \begin{enumerate}
        \item $\ell=1$, for $\Zb^{2,\ell-1}\in\cE_1$, 
  \begin{enumerate}[(a)]
            \item $\Lambda^{(t)}_{5,j_1,r_{g_1\cdot y_0}}=\big(1-2\ate^{2,1}\big)B \pm O(\delta) \geq \Big(\frac{1}{3}+c_1\Big)B$;
            \item  $\Lambda^{(t)}_{5,j_1,r_{g_1\cdot y_0}}-\Lambda^{(t)}_{5,j'_1,r_{g_2\cdot y_0}}=2\big(\attn^{(t)}_{\ans,0\to\pred,1}-\attn^{(t)}_{\ans,0\to\pred,2}\big)B\pm O(\delta)\geq 2(c_1+c_2)B$;
            \item for $y\neq y_0$, $\Lambda^{(t)}_{5,\tau(g_1(y)),r_{g_1\cdot y}}= \Big(2\attn^{(t)}_{\ans,0\to\pred,1}-1\Big)B\pm O(\delta)$, which is only activated if $\attn^{(t)}_{\ans,0\to\pred,1}>\frac{1}{2}$
        \end{enumerate}
        \item $\ell=2$, for $\Zb^{2,\ell-1}\in\tilde{\cE}_2$, 
   \begin{enumerate}[(a)]
            \item $\Lambda^{(t)}_{5,j_2,r_{g_2\cdot y_1}}=\big(1-2\ate^{2,2}\big)B \pm O(\delta) \geq 4c_5B $; 
            \item$\Lambda^{(t)}_{5,j'_2,r_{g_2\cdot y_0}}= \Delta^{2,2}\cdot B
            \pm O(\delta)$, and $$\Lambda^{(t)}_{5,j_2,r_{g_2\cdot y_1}}-\Lambda^{(t)}_{5,j'_2,r_{g_2\cdot y_0}}=2\big(\attn^{(t)}_{\ans,1\to\ans,1}-\attn^{(t)}_{\ans,1\to\ans,0}\big)B\pm O(\delta);$$
            \item for $y\neq y_0, y_1$, $\Lambda^{(t)}_{5,\tau(g_2(y)),r_{g_2\cdot y}}= \Big(2\attn^{(t)}_{\ans,1\to\pred,2}-1\Big)B\pm O(\delta)$, which is only activated if $\attn^{(t)}_{\ans,1\to\pred,2}>\frac{1}{2}$.

        \end{enumerate}
    \end{enumerate}
  
\end{lemma}
\begin{lemma}
\label{lem-s23-logit-cyc}
    If \Cref{induction-s23-cyc} holds for all iterations $\in [T_{2,2,s}, t)$, given input $\Zb^{2,\ell-1}$, then we have
\begin{enumerate}   
    \item for $\ell=1$, if $\Zb^{2,\ell-1}\in \cE_1$, 
 $$1-\logit^{(t)}_{5, j_1}= \Theta(1)\cdot \logit^{(t)}_{5,j_1^{\prime}}=\Theta\Bigg(\frac{1}{d^{2\big(\attn^{(t)}_{\ans,0\to\pred,1}-\attn^{(t)}_{\ans,0\to\pred,2}\big)C_B}
 }\Bigg).$$
    \item for $\ell=2$,  if $\Zb^{2,\ell-1}\in \tilde{\cE}_2$,
    $$
    \logit^{(t)}_{5,j'_2}=\Theta\Bigg(\frac{1}{d^{2\big(\attn^{(t)}_{\ans,1\to\ans,1}-\attn^{(t)}_{\ans,1\to\ans,0}\big)C_B}
 +d^{1-\Delta^{2,2}C_B}}\Bigg),
    $$
    moreover,
    $$ 1-\logit^{(t)}_{5,j_2}\geq\min\bigg\{\Omega(1), \Omega\Big(\frac{1}{d^{C_B\cdot(1-2\ate^{2,2})-
    1}}\Big)\bigg\}.$$
\end{enumerate}
\end{lemma}

\begin{lemma}\label{lem-s23-act-other-cyc}
     If \Cref{induction-s23-cyc} holds for all iterations $\in [T_{2,2,s}, t)$, given input $\Zb^{2,\ell-1}$, then we have
         \begin{enumerate}
        \item for $\ell=1$, if $\Zb^{2,\ell-1}\notin \cE_1$, then
  \begin{enumerate}[(a)]
            \item for $j=j_1$, $\hat{\fA}_{j_1}=\{r_{g_1\cdot y_0}\}$, $\Lambda^{(t)}_{5,j_1,r_{g_1\cdot y_0}}= B\pm O(\delta)$;
            \item for $j\neq j_1\in\tau(\Y)$, assuming $j=\tau(g_1(y))$ for $y\neq y_0$, then $\Lambda^{(t)}_{5,j_1,r_{g_1\cdot y_0}}-\Lambda^{(t)}_{5,j,r_{g_1\cdot y}}\geq 2\attn^{(t)}_{\ans,0\to\ans,0}\cdot B\geq \Omega(B)$ and    $\Lambda^{(t)}_{5,j,r}\ll -\varrho$ for $r\in \hat{\fA}_{j}\setminus\{r_{g_1\cdot y}\}$.
        \end{enumerate}

       \item $\ell=2$, if  $\Zb^{2,\ell-1}\notin \tilde{\cE}_2$, then
                \begin{enumerate}[(a)]
                    \item if $g_1=g_2\wedge y_0\neq y_1$,
                    \begin{enumerate}
                        \item for $j=j_2$, $\Lambda^{(t)}_{5,j_2,r_{g_2\cdot y_1}}=(1-\ate^{2,2})\cdot B\pm O(\delta)$ and $\Lambda^{(t)}_{5,j_2,r}\ll -\varrho$ for $r\in \hat{\fA}_{j_2}\setminus\{r_{g_2\cdot y_1}\}$;
                        \item for $j=\tau(g_2(y_0))$, $\Lambda^{(t)}_{5,j_2,r_{g_2\cdot y_1}}-\Lambda^{(t)}_{5,j,r_{g_2\cdot y_0}}=2(\attn^{(t)}_{\ans,1\to\ans,1}-\attn^{(t)}_{\ans,1\to\ans,0})B\pm O(\delta)$ and $\Lambda^{(t)}_{5,j,r}\ll -\varrho$ for $r\in \hat{\fA}_{j}\setminus\{r_{g_2\cdot y_0}\}$;
                        \item for other $j\in\tau(\cY)$, assuming $j=\tau(g_2(y))$ for some $y\neq y_0,y_1$,   $\Lambda^{(t)}_{5,j_2,r_{g_2\cdot y_1}}-\Lambda^{(t)}_{5,j,r_{g_2\cdot y}}=2\attn^{(t)}_{\ans,1\to\ans,1}B\pm O(\delta)$ and $\Lambda^{(t)}_{5,j,r}\ll -\varrho$ for $r\in \hat{\fA}_{j}\setminus\{r_{g_2\cdot y}\}$.
                    \end{enumerate}
                  \item if $g_1\neq g_2\wedge y_0= y_1$,
                    \begin{enumerate}
\item for $j=j_2$, $\Lambda^{(t)}_{5,j_2,r_{g_2\cdot y_1}}=(1-\ate^{2,2})\cdot B\pm O(\delta)$ and $\Lambda^{(t)}_{5,j_2,r}\ll -\varrho$ for $r\in \hat{\fA}_{j_2}\setminus\{r_{g_2\cdot y_1}\}$;

\item for $j=\tau(g_1(y_1))=\tau(g_2(\tilde{y}))$, 
\begin{align*}
    &\Lambda^{(t)}_{5,j_2,r_{g_2\cdot y_1}}-\Lambda^{(t)}_{5,j,r_{g_1\cdot y_1}}=2(\attn^{(t)}_{\ans,1\to\pred,2}-\attn^{(t)}_{\ans,1\to\pred,1})B\pm O(\delta)\\
&\Lambda^{(t)}_{5,j_2,r_{g_2\cdot y_1}}-\Lambda^{(t)}_{5,j,r_{g_2\cdot \tilde{y}}}=2(\attn^{(t)}_{\ans,1\to\ans,0}+\attn^{(t)}_{\ans,1\to\ans,1})B\pm O(\delta), 
\end{align*}
where $r_{g_2\cdot \tilde{y}}$ is only activated if $\attn^{(t)}_{\ans,1\to\pred,2}>\frac{1}{2}$.
$\Lambda^{(t)}_{5,j,r}\ll -\varrho$ for $r\in \hat{\fA}_{j}\setminus\{r_{g_1\cdot y_1}, r_{g_2\cdot \tilde{y}}\}$;
\item for other $j\in\tau(\cY)$, assuming $j=\tau(g_2(y))$ for $y\neq y_1, \tilde{y}$, then    $r_{g_2\cdot {y}}$ is only activated if $\attn^{(t)}_{\ans,1\to\pred,2}>\frac{1}{2}$ and $$\Lambda^{(t)}_{5,j_2,r_{g_2\cdot y_1}}-\Lambda^{(t)}_{5,j_2,r_{g_2\cdot y}}=2(\attn^{(t)}_{\ans,1\to\ans,0}+\attn^{(t)}_{\ans,1\to\ans,1})B\pm O(\delta).$$
$\Lambda^{(t)}_{5,j,r}\ll -\varrho$ for $r\in \hat{\fA}_{j}\setminus\{r_{g_2\cdot y}\}$.
                    \end{enumerate}
                \item if $g_1=g_2\wedge y_0= y_1$,
                    \begin{enumerate}
                        \item for $j=j_2$, $ \hat{\fA}_{j_2}=\{r_{g_2\cdot y_1}\}$, $\Lambda^{(t)}_{5,j_2,r_{g_2\cdot y_1}}= B\pm O(\delta)$ ;
                        \item for other $j\in\tau(\cY)$, assuming $j=\tau(g_2(y))$, 
                        $$\Lambda^{(t)}_{5,j_2,r_{g_2\cdot y_1}}-\Lambda^{(t)}_{5,j_2,r_{g_2\cdot y}}=2(\attn^{(t)}_{\ans,1\to\ans,0}+\attn^{(t)}_{\ans,1\to\ans,1})B\pm O(\delta).$$  
                        and $\Lambda^{(t)}_{5,j,r}\ll -\varrho$ for $r\in \hat{\fA}_{j}\setminus\{r_{g_2\cdot y}\}$.
                    \end{enumerate}
        \end{enumerate}
    \end{enumerate}

\end{lemma}

\begin{lemma}\label{lem-s23-logit-other-cyc}
     If \Cref{induction-s23-cyc} holds for all iterations $\in [T_{2,2,s}, t)$, given input $\Zb^{2,\ell-1}$, then we have
         \begin{enumerate}
        \item for $\ell=1$, if $\Zb^{2,\ell-1}\notin \cE_1$, then $1-\logit^{(t)}_{5,j_1}=O\bigg(\frac{1}{d^{2\attn^{(t)}_{\ans,0\to\ans,0}C_B}}\bigg)=O\Big(\frac{1}{\poly d}\Big)$.
       \item $\ell=2$, if  $\Zb^{2,\ell-1}\notin \tilde{\cE}_2$, then
                \begin{enumerate}[(a)]
                    \item if $g_1=g_2\wedge y_0\neq y_1$, $\logit^{(t)}_{5,\tau(g_2(y_0))}=O\bigg(\frac{1}{d^{2\big(\attn^{(t)}_{\ans,1\to\ans,1}-\attn^{(t)}_{\ans,1\to\ans,0}\big)C_B}}\bigg)$.
                  \item if $g_1\neq g_2\wedge y_0= y_1$, $\logit^{(t)}_{5,\tau(g_1(y_1))}=O\bigg(\frac{1}{d^{2\big(\attn^{(t)}_{\ans,1\to\pred,2}-\attn^{(t)}_{\ans,1\to\pred,1}\big)C_B}}\bigg)$.
                \item if $g_1=g_2\wedge y_0= y_1$, $1-\logit^{(t)}_{5,j_2}=O\bigg(\frac{1}{d^{2\big(\attn^{(t)}_{\ans,1\to\ans,0}+\attn^{(t)}_{\ans,1\to\ans,1}\big)C_B}}\bigg)=\frac{1}{\poly d}$.
        \end{enumerate}
    \end{enumerate}

\end{lemma}
\subsubsection{Gradient Lemma}
  \begin{lemma}\label{lem-s23-gd1-cyc}
      If \Cref{induction-s23-cyc} holds for all iterations $\in [T_{2,2,s}, t)$, given $s\in\tau(\X)$,  for $[\Qb_{4,4}]_{s,s}$, we have
      \begin{align*}
         \sum_{\ell=1}^2 \Big[-\nabla_{\Q^{(t)}_{4,4}}{\Loss^{2,\ell}_{5}}\Big]_{s,s} \geq \Omega\Big(\frac{\epsilon\log d}{d^{(1-2\epsilon)C_B}}\Big).
      \end{align*}
  
  \end{lemma}
  \begin{proof}
       By gradient decomposition, we have 
    \begin{align*}
           & \sum_{\ell=1}^2 \Big[-\nabla_{\Q^{(t)}_{4,4}}\Loss^{2,\ell}_{5}\Big]_{s,s}= \sum_{\ell\in [2]}\sum_{\kappa\in \{\rom1,\rom2,\rom3\}}\cN^{(t)}_{s,4,\ell,\kappa}.
    \end{align*}
Firstly, for $\cN^{(t)}_{s,4,2,\rom1}$, by \Cref{lem-s23-act-cyc} and \Cref{lem-s23-act-other-cyc}, we have
    \begin{align}
       \cN^{(t)}_{s,4,2,\rom1}&= \E\Bigg[
    \attn^{(t)}_{{\ans,1} \rightarrow \ans,1} \cdot (1-\logit^{(t)}_{5,j_2})\cdot \notag\\
    &~~~~~~~~~  \Big(
 \Big( V_{j_2,  r_{g_2\cdot y_1}}(y_1)- \Lambda^{(t)}_{5,j_2, r_{g_2\cdot y_1}}\pm\tilde{O}(\sigma_0) \Big)\pm \tilde{O}(\delta^{q}) \Big) 
    \1_{\{\tau(x_1)=s\}\cap \tilde{\cE}_2}\Bigg] \notag\\
    &+ \E\Bigg[
    \attn^{(t)}_{{\ans,1} \rightarrow \ans,1} \cdot (1-\logit^{(t)}_{5,j_2})\cdot \notag\\
    &~~~~~~~~~  \Big(
 \Big( V_{j_2,  r_{g_2\cdot y_1}}(y_1)- \Lambda^{(t)}_{5,j_2, r_{g_2\cdot y_1}}\pm\tilde{O}(\sigma_0) \Big)\pm \tilde{O}(\delta^{q}) \Big) 
    \1_{\{\tau(x_1)=s\}\cap \tilde{\cE}^c_2}\Bigg] \notag\\
    &\stackrel{(a)}{\geq }\Theta\Big(\frac{1}{d}\Big)\cdot \E\bigg[\min\Big\{\Omega(1), \Omega\Big(\frac{1}{d^{C_B\cdot(1-2\ate^{2,2})-1}}\Big)\Big\} \cdot 2\ate^{2,2}\cdot  B \mid \tau(x_1)=s \bigg] \notag
    \\
    &\geq \Omega\Big(\frac{\epsilon\log d}{d^{(1-2\epsilon)C_B}}\Big),  \label{eq-s23-gd1-cyc-1}
\end{align}
where (a) follows from \Cref{lem-s23-logit-cyc}. Noticing that $\cN^{(t)}_{s,4,1,\rom1}>0$ and $|\cN^{(t)}_{s,4,\ell,\rom3}|$ for $\ell\in[2]$ is sufficiently small, to provide a lower bound for $\sum_{\ell=1}^2 \Big[-\nabla_{\Q^{(t)}_{4,4}}\Loss^{2,\ell}_{5}\Big]_{s,s}$, we only need to focus on the negative gradient from $\cN^{(t)}_{s,4,1,\rom2}$ and $\cN^{(t)}_{s,4,2,\rom2}$.
\begin{align}
&\cN^{(t)}_{s,4,2,\rom2}\notag=\\
&-\E\Bigg[
    \attn^{(t)}_{{\ans,1} \rightarrow \ans,1} \cdot\logit^{(t)}_{5,j'_2} \cdot \notag \\
    &~~~~~~~~~~\Big(\Big( V_{j'_2, r_{g_2\cdot y_0}}(y_1)- \Lambda^{(t)}_{5,j'_2,r_{g_2\cdot y_0}}\pm\tilde{O}(\sigma_0) \Big)\pm \tilde{O}(\delta^{q}) \Big) 
\1_{\{\tau(x_1)=s\}\cap \tilde{\cE}_2}\Bigg] \notag \\
&-\E\Bigg[
    \attn^{(t)}_{{\ans,1} \rightarrow \ans,1} \cdot\sum_{j\neq j_2\in\tau(\cY)}\logit^{(t)}_{5,j} \cdot \notag\\
    &~~~~~~~~~~\Big(\sum_{r\in\hat{\fA}_{j}}\ReLU^{\prime}(\Lambda^{(t)}_{5,j,r
    })\cdot  \Big( V_{j, r}(y_1)- \Lambda^{(t)}_{5,j,r}\pm\tilde{O}(\sigma_0) \Big)\pm \tilde{O}(\delta^{q}) \Big) 
\1_{\{\tau(x_1)=s\}\cap\tilde{\cE}^c_2}\Bigg] \notag\\
&\geq -\E\Bigg[
    \attn^{(t)}_{{\ans,1} \rightarrow \ans,1} \cdot \logit^{(t)}_{5,\tau(g_1(y_1))} \cdot \Big(\ReLU^{\prime}(\Lambda^{(t)}_{5,\tau(g_1(y_1)),r_{g_1\cdot y_1}
    })\cdot\notag\\
    &~~~~~~~~~~ \Big( V_{\tau(g_1(y_1)), r_{g_1\cdot y_1}}(y_1)- \Lambda^{(t)}_{5,\tau(g_1(y_1)),r_{g_1\cdot y_1}}\pm\tilde{O}(\sigma_0) \Big)\pm \tilde{O}(\delta^{q}) \Big) 
\1_{\{\tau(x_1)=s\}\cap \{g_1\neq g_2, y_0=y_1\}}\Bigg].\label{eq-s23-gd1-cyc-2}
\end{align}
    When $\E\Big[C_B\cdot(1-2\ate^{2,2})\mid \tau(x_1)=s\Big]<1$, by \Cref{lem-s23-logit-cyc} and \Cref{lem-s23-logit-other-cyc}, we have
    \begin{align*}
&\logit^{(t)}_{5,\tau(g_1(y_1))}\1_{\{\tau(x_1)=s\}\cap \{g_1\neq g_2, y_0=y_1\}}\\
&\leq O\bigg(\frac{1}{d^{2\big(\attn^{(t)}_{\ans,1\to\pred,2}-\attn^{(t)}_{\ans,1\to\pred,1}\big)C_B}}\bigg)\big(1-\logit^{(t)}_{5,j_2}\big)\1_{\{\tau(x_1)=s\}\cap \tilde{\cE}_2}.
    \end{align*}
    During this time,  $\ate^{2,2}\geq \Omega(1)$, thus 
    we can lower bound $\cN^{(t)}_{s,4,2,\rom2}$ by $\cN^{(t)}_{s,4,2,\rom2}\notag\geq - O\bigg(\frac{1}{d^{\Omega(1)}\cdot \log d}\bigg)\cdot \cN^{(t)}_{s,4,2,\rom1}$,
which implies that the negative gradient from $\cN^{(t)}_{s,4,2,\rom2}$ is dominated by the positive gradient from $\cN^{(t)}_{s,4,2,\rom1}$.
 
When $\E\Big[C_B\cdot(1-2\ate^{2,2})\mid \tau(x_1)=s\Big]\geq 1$, since $2\big(\attn^{(t)}_{\ans,1\to\pred,2}-\attn^{(t)}_{\ans,1\to\pred,1}\big)C_B\geq(1-2\eta^{2,2})C_B$, we obtain
\begin{align}
\logit^{(t)}_{\tau(g_1(y_1))}\1_{\{\tau(x_1)=s\}\cap \{g_1\neq g_2, y_0=y_1\}}\leq O(\frac{1}{d})\cdot (1-\logit^{(t)}_{j_2})\1_{\{\tau(x_1)=s\}\cap \tilde{\cE}_2}. \label{eq-s23-gd1-cyc-3}
\end{align}  
For the event $\{g_1\neq g_2, y_0=y_1\}$, if $\Lambda^{(t)}_{5,\tau(g_1(y_1)), r_{g_1\cdot y_1}}$ is still in the linear regime, we have
            \begin{align*}             \Lambda^{(t)}_{5,\tau(g_1(y_1)), r_{g_1\cdot y_1}}= \big(1-2\attn^{(t)}_{\ans,1\to\pred,2}\big)\cdot B \pm O(\delta)  \geq \varrho, 
            \end{align*}
            which implies 
            \begin{align}
                \ate^{2,2}\geq 1-2\attn_{\ans,1\to\pred,2}\geq \Omega\Big(\frac{\varrho}{B}\Big). \label{eq-s23-gd1-cyc-4}
            \end{align}
Hence, putting   \eqref{eq-s23-gd1-cyc-4} back to \eqref{eq-s23-gd1-cyc-2}, and putting \eqref{eq-s23-gd1-cyc-3} back to \eqref{eq-s23-gd1-cyc-2}, we can  lower bound $\cN^{(t)}_{s,4,2,\rom2}$ as follows $\cN^{(t)}_{s,4,2,\rom2}\notag\geq - O\big(\frac{1}{d}\big)\cdot \cN^{(t)}_{s,4,2,\rom1}$. If  $\Lambda^{(t)}_{5,\tau(g_1(y_1)), r_{g_1\cdot y_1}}$ falls into the smoothed regime, we can upper bound $\ReLU^{\prime}(\Lambda^{(t)}_{5,\tau(g_1(y_1)), r_{g_1\cdot y_1}})$ by $O(\frac{\ate^{2,2}B}{\varrho})$, and then similarly, we can obatian $\cN^{(t)}_{s,4,2,\rom2}\notag\geq - O\big(\frac{1}{d}\big)\cdot \cN^{(t)}_{s,4,2,\rom1}$.

Following the analogous analysis, the negative gradient from $\cN^{(t)}_{s,4,1,\rom2}$ can also be dominated by the positive gradient from $\cN^{(t)}_{s,4,2,\rom1}$. Hence, we complete the proof.
  \end{proof}
  \begin{lemma}
   If \Cref{induction-s23-cyc} holds for all iterations $\in [T_{2,2,s}, t)$,  given $s\in\tau(\X)$,  for $[\Qb_{4,p}]_{s,s'}$, $p\in\{3,4\}$, $s'\not=s\in\tau(\X)$, $\ell\in [2]$,  we have
    \begin{align*}
        \bigg|\Big[-\nabla_{\Q^{(t)}_{4,p}}{\Loss^{2,\ell}_{5}}\Big]_{s,s'}\bigg| \leq O\Big(\frac{1}{d}\Big) \bigg|\Big[-\nabla_{\Q^{(t)}_{4,p}}{\Loss^{2,\ell}_{5}}\Big]_{s,s}\bigg| .
    \end{align*}
  \end{lemma}
  \subsubsection{Non-negative Gap}
    \begin{lemma}
   If \Cref{induction-s23-cyc} holds for all iterations $\in [T_{2,2,s}, t)$, then at time $t$, we have  ${\Delta^{2,\ell}}\geq 0$ for any $\Zb^{2,\ell}$ with $\ell\in\{1,2\}$.
  \end{lemma}
  \begin{proof}
    Let $\tilde{T}$ denote the first time that $\E[\Delta^{2,2}\mid \tau(x_1)=s]<\alpha$, where $\alpha=\frac{\epsilon^{\frac{1}{\alpha-2}}}{\poly d}$. 

Following the analogous analysis as \Cref{lem-s22-gd3-cyc}, we have
    \begin{align*}
           & \sum_{\ell=1}^2 \Big[-\nabla_{\Q^{(\tilde{T})}_{4,3}}\Loss^{2,\ell}_{5}\Big]_{s,s}+ \sum_{\ell=1}^2 \Big[\nabla_{\Q^{(\tilde{T})}_{4,4}}\Loss^{2,\ell}_{5}\Big]_{s,s}= \sum_{\ell\in [2]}\sum_{\kappa\in \{\rom1,\rom2,\rom3\}}\cN^{(\tilde{T})}_{s,3,\ell,\kappa}-\cN^{(\tilde{T})}_{s,4,\ell,\kappa}.
    \end{align*}
We can ignore the negligible difference introduced by $\cN^{(\tilde{T})}_{s,p,\ell,\rom3}$ for $p\in\{3,4\}$ and $\ell\in[2]$. 

For $\ell=1$, since $\attn^{(\tilde{T})}_{{\ans,0} \rightarrow \pred,1} \geq \attn^{(\tilde{T})}_{{\ans,0} \rightarrow \ans,0}$, and thus it is straightforward to see that $\cN^{(\tilde{T})}_{s,3,1,\rom1}-\cN^{(\tilde{T})}_{s,4,1,\rom1}\geq 0$. Furthermore, we have
   \begin{align}
     & \cN^{(\tilde{T})}_{s,3,1,\rom2}-\cN^{(\tilde{T})}_{s,4,1,\rom2} \notag\\
     &= -\E\Bigg[
    \attn^{(\tilde{T})}_{{\ans,0} \rightarrow \pred,1} \cdot \logit^{(\tilde{T})}_{5,j'_1} \cdot \ReLU^{\prime}\big(\Lambda^{(t)}_{5,j'_1,r_{g_2\cdot y_0}}\big) \notag\\
    &~~~~~~~~~~\Big(  \Big( V_{j'_1,r_{g_2\cdot y_0}}(g_1)- \Lambda^{(\tilde{T})}_{5,j'_1,r_{g_2\cdot y_0}}\pm\tilde{O}(\sigma_0) \Big)\pm \tilde{O}(\delta^{q}) \Big) 
\1_{\{\tau(x_0)=s\}\cap \cE_1}\Bigg] \label{eq-s23-gd-gap-1}\\
&+\E\Bigg[
    \attn^{(\tilde{T})}_{{\ans,0} \rightarrow \ans,0} \cdot \logit^{(\tilde{T})}_{5,j'_1} \cdot \ReLU^{\prime}\big(\Lambda^{(\tilde{T})}_{5,j'_1,r_{g_2\cdot y_0}}\big) \notag\\
    &~~~~~~~~~~\Big(  \Big( V_{j'_1,r_{g_2\cdot y_0}}(y_0)- \Lambda^{(\tilde{T})}_{5,j'_1,r_{g_2\cdot y_0}}\pm\tilde{O}(\sigma_0) \Big)\pm \tilde{O}(\delta^{q}) \Big) 
\1_{\{\tau(x_0)=s\}\cap \cE_1}\Bigg] \label{eq-s23-gd-gap-2}\\
& -\E\Bigg[
    \attn^{(\tilde{T})}_{{\ans,0} \rightarrow \pred,1} \cdot\sum_{y\neq y_0\in\cY}\logit^{(\tilde{T})}_{5,\tau(g_1(y))} \cdot \ReLU^{\prime}\big(\Lambda^{(\tilde{T})}_{5,\tau(g_1(y)),r_{g_1\cdot y}}\big)\notag\\
    &~~~~~~~~~~\Big( \Big( V_{\tau(g_1(y)), r_{g_1\cdot y}}(g_1)- \Lambda^{(\tilde{T})}_{5,\tau(g_1(y)), r_{g_1\cdot y}}\pm\tilde{O}(\sigma_0) \Big)\pm \tilde{O}(\delta^{q}) \Big) 
\1_{\{\tau(x_0)=s\}\cap \cE^c_1}\Bigg] \notag\\
& +\E\Bigg[
    \attn^{(\tilde{T})}_{{\ans,0} \rightarrow \ans,0} \cdot\sum_{y\neq y_0\in\cY}\logit^{(\tilde{T})}_{5,\tau(g_1(y))} \cdot \ReLU^{\prime}\big(\Lambda^{(\tilde{T})}_{5,\tau(g_1(y)),r_{g_1\cdot y}}\big)\notag\\
    &~~~~~~~~~~\Big( \Big( V_{\tau(g_1(y)), r_{g_1\cdot y}}(y_0)- \Lambda^{(\tilde{T})}_{5,\tau(g_1(y)), r_{g_1\cdot y}}\pm\tilde{O}(\sigma_0) \Big)\pm \tilde{O}(\delta^{q}) \Big) 
\1_{\{\tau(x_0)=s\}\cap \cE^c_1}\Bigg]. \notag
\end{align}

    By \Cref{lem-s23-logit-other-cyc}, we obtain that $\cN^{(\tilde{T})}_{s,3,1,\rom2}-\cN^{(\tilde{T})}_{s,4,1,\rom2} $ is dominated by $\eqref{eq-s23-gd-gap-1}+\eqref{eq-s23-gd-gap-2}$.  Moreover, by \Cref{lem-s23-logit-cyc}, we have
\begin{align}  \logit^{(\tilde{T})}_{5,\tau(g_1(y))} (\Z^{2,0}) &\geq \Omega\Bigg(\frac{1}{d^{2\big(\attn^{(\tilde{T})}_{\ans,0\to\pred,1}-\attn^{(\tilde{T})}_{\ans,0\to\pred,2}\big)C_B}}\Bigg) \notag \\
    &\geq \Omega\Bigg(\frac{1}{d^{2\big(\attn^{(\tilde{T})}_{\ans,1\to\pred,2}-\attn^{(\tilde{T})}_{\ans,1\to\pred,1}\big)C_B}}\Bigg).
\label{eq-s23-logit-cyc-1}
\end{align}

For $\ell=2$, we have 
\begin{align}
    &\cN^{(\tilde{T})}_{s,3,2,\rom1}-\cN^{(\tilde{T})}_{s,4,2,\rom1}\notag\\
    &\geq  \E\Bigg[
    \attn^{(\tilde{T})}_{{\ans,1} \rightarrow \pred,2} \cdot (1-\logit^{(\tilde{T})}_{5,j_2})\cdot \notag\\
    &~~~~~~~~~  \Big(
 \Big( V_{j_2,  r_{g_2\cdot y_1}}(g_2)- \Lambda^{(\tilde{T})}_{5,j_2, r_{g_2\cdot y_1}}\pm\tilde{O}(\sigma_0) \Big)\pm \tilde{O}(\delta^{q}) \Big) 
    \1_{\{\tau(x_1)=s\}\cap \tilde{\cE}_2}\Bigg] \notag\\
    &- \E\Bigg[
    \attn^{(\tilde{T})}_{{\ans,1} \rightarrow \ans,1} \cdot (1-\logit^{(\tilde{T})}_{5,j_2})\cdot \notag\\
    &~~~~~~~~~  \Big(
 \Big( V_{j_2,  r_{g_2\cdot y_1}}(y_1)- \Lambda^{(\tilde{T})}_{5,j_2, r_{g_2\cdot y_1}}\pm\tilde{O}(\sigma_0) \Big)\pm \tilde{O}(\delta^{q}) \Big) 
    \1_{\{\tau(x_1)=s\}\cap \tilde{\cE}_2}\Bigg] \notag\\
    &\geq \Omega \big(\alpha\epsilon \log d\big) \E\bigg[(1-\logit^{(\tilde{T})}_{5,j_2})\cdot \1_{\tau(x_1)=s}\bigg],\label{eq-s23-gd-gap-3}
\end{align}
where the last inequality follows from the definition of $\tilde{T}$. 
\begin{align}
       &\cN^{(\tilde{T})}_{s,3,2,\rom2}-\cN^{(\tilde{T})}_{s,4,2,\rom2}\notag\\
       &= -\E\Bigg[
    \attn^{(\tilde{T})}_{{\ans,1} \rightarrow \pred,2} \cdot \logit^{(\tilde{T})}_{5,j'_2} \cdot \ReLU^{\prime}\big(\Lambda^{(\tilde{T})}_{5,j'_2,r_{g_2\cdot y_0}}\big)  \notag \\
    &~~~~~~~~~~\Big( \Big( V_{j'_2, r_{g_2\cdot y_0}}(g_2)- \Lambda^{(\tilde{T})}_{5,j'_2,r_{g_2\cdot y_0}}\pm\tilde{O}(\sigma_0) \Big)\pm \tilde{O}(\delta^{q}) \Big) 
\1_{\{\tau(x_1)=s\}\cap \tilde{\cE}_2}\Bigg] \label{eq-s23-cyc-gd1}\\
&+\E\Bigg[
    \attn^{(\tilde{T})}_{{\ans,1} \rightarrow \ans,1} \cdot\logit^{(\tilde{T})}_{5,j'_2} \cdot \ReLU^{\prime}\big(\Lambda^{(\tilde{T})}_{5,j'_2,r_{g_2\cdot y_0}}\big) \notag \\
    &~~~~~~~~~~\Big(\Big( V_{j'_2, r_{g_2\cdot y_0}}(y_1)- \Lambda^{(\tilde{T})}_{5,j'_2,r_{g_2\cdot y_0}}\pm\tilde{O}(\sigma_0) \Big)\pm \tilde{O}(\delta^{q}) \Big) 
\1_{\{\tau(x_1)=s\}\cap \tilde{\cE}_2}\Bigg] \label{eq-s23-cyc-gd2}\\
&-\E\Bigg[
    \attn^{(\tilde{T})}_{{\ans,1} \rightarrow \pred,2} \cdot\sum_{j\neq j_2\in\tau(\cY)}\logit^{(\tilde{T})}_{5,j} \cdot \notag\\
    &~~~~~~~~~~\Big(\sum_{r\in\hat{\fA}_{j}}\ReLU^{\prime}(\Lambda^{(\tilde{T})}_{5,j,r
    })\cdot  \Big( V_{j, r}(g_2)- \Lambda^{(\tilde{T})}_{5,j,r}\pm\tilde{O}(\sigma_0) \Big)\pm \tilde{O}(\delta^{q}) \Big) 
\1_{\{\tau(x_1)=s\}\cap \tilde{\cE}^c_2}\Bigg] \label{eq-s23-cyc-gd3}\\
&+\E\Bigg[
    \attn^{(\tilde{T})}_{{\ans,1} \rightarrow \ans,1} \cdot\sum_{j\neq j_2\in\tau(\cY)}\logit^{(\tilde{T})}_{5,j} \cdot \notag\\
    &~~~~~~~~~~\Big(\sum_{r\in\hat{\fA}_{j}}\ReLU^{\prime}(\Lambda^{(\tilde{T})}_{5,j,r
    })\cdot  \Big( V_{j, r}(y_1)- \Lambda^{(\tilde{T})}_{5,j,r}\pm\tilde{O}(\sigma_0) \Big)\pm \tilde{O}(\delta^{q}) \Big) 
\1_{\{\tau(x_1)=s\}\cap\tilde{\cE}^c_2}\Bigg]. \label{eq-s23-cyc-gd4}
\end{align}
Notice that by \Cref{lem-s23-lambda-cyc} and the definition of $\tilde{T}$, we have 
\begin{align*}
    |\eqref{eq-s23-cyc-gd1}+\eqref{eq-s23-cyc-gd2}|\leq O\big((\alpha \log d)^{q-1}\big)\cdot \E\bigg[\Big(1-\logit^{(\tilde{T})}_{5,j_2}\Big)\cdot \1_{\tau(x_1)=s}\bigg],
\end{align*}
which is dominated by \eqref{eq-s23-gd-gap-3} due to the choice of $\alpha$.

Furthermore, by \Cref{lem-s23-act-other-cyc}, for $\tilde{\cE}^c_2$, we only need to focus on the output $\logit^{(t)}_{5,\tau(g_2(y_0))}$  from the case  $g_1=g_2\wedge y_0\neq y_1$, and obtain
\begin{align*}
   &\eqref{eq-s23-cyc-gd3}+\eqref{eq-s23-cyc-gd4} \\&\geq -\E\Bigg[
    \attn^{(\tilde{T})}_{{\ans,1} \rightarrow \pred,2} \cdot  \logit^{(\tilde{T})}_{5,\tau(g_2(y_0))}  \cdot \ReLU^{\prime}\big(\Lambda^{(\tilde{T})}_{5,\tau(g_2(y_0)),r_{g_2\cdot y_0}}\big) \notag \\
    &~~~\Big(\Big( V_{\tau(g_2(y_0)), r_{g_2\cdot y_0}}(g_2)- \Lambda^{(\tilde{T})}_{5,\tau(g_2(y_0)),r_{g_2\cdot y_0}}\pm\tilde{O}(\sigma_0) \Big)\pm \tilde{O}(\delta^{q}) \Big) 
\1_{\{\tau(x_1)=s\}\cap \{g_1=g_2\wedge y_0\neq y_1\}}\Bigg] \\
&+\E\Bigg[
    \attn^{(\tilde{T})}_{{\ans,1} \rightarrow \ans,1} \cdot \logit^{(\tilde{T})}_{5,\tau(g_2(y_0))} \cdot \ReLU^{\prime}\big(\Lambda^{(\tilde{T})}_{5,\tau(g_2(y_0)),r_{g_2\cdot y_0}}\big) \notag\\
    &~~~\Big( \Big( V_{\tau(g_2(y_0)), r_{g_2\cdot y}}(y_1)- \Lambda^{(\tilde{T})}_{5,\tau(g_2(y_0)),r_{g_2\cdot y_0}}\pm\tilde{O}(\sigma_0) \Big)\pm \tilde{O}(\delta^{q}) \Big) 
\1_{\{\tau(x_1)=s\}\cap \{g_1=g_2\wedge y_0\neq y_1\}}\Bigg].
\end{align*}

By \Cref{lem-s23-logit-other-cyc}, we have 
\begin{align*}
    \logit^{(\tilde{T})}_{5,\tau(g_2(y_0))}(\Z^{2,1})=O\bigg(\frac{1}{d^{2\big(\attn^{(\tilde{T})}_{\ans,1\to\ans,1}-\attn^{(\tilde{T})}_{\ans,1\to\ans,0}\big)C_B}}\bigg).
\end{align*}

Since $\E[\Delta^{2,2}\mid \tau(x_1)=s]<\alpha$,  
where $\alpha$ is sufficiently small,  then combining with \eqref{eq-s23-logit-cyc-1}, we have $$\logit^{(\tilde{T})}_{5,\tau(g_2(y_0))}(\Z^{2,1})=O(1)\cdot \logit^{(\tilde{T})}_{5,\tau(g_2(y_0))}(\Z^{2,0}).$$
Therefore, since $\tilde{\cE}_2^c$ happens with probability $O(\frac{1}{\log d})$, we obtain that $$\eqref{eq-s23-cyc-gd3}+\eqref{eq-s23-cyc-gd4} \geq - O(\frac{1}{\log d})\Big(\cN^{(t)}_{s,3,1,\rom2}-\cN^{(t)}_{s,4,1,\rom2} \Big).$$ Putting it all together, we can conclude that when $\E[\Delta^{2,2}\mid \tau(x_1)=s]$ reaches below $\alpha$, the gap in non-decreasing direction is guaranteed, i.e., $$\sum_{\ell=1}^2 \Big[-\nabla_{\Q^{(\tilde{T})}_{4,3}}\Loss^{2,\ell}_{5}\Big]_{s,s}+ \sum_{\ell=1}^2 \Big[\nabla_{\Q^{(\tilde{T})}_{4,4}}\Loss^{2,\ell}_{5}\Big]_{s,s}>0,$$ which completes the proof.
  \end{proof}

\subsubsection{Upper Bound for $\Qb$}
  \begin{lemma}
   If \Cref{induction-s23-cyc} holds for all iterations $\in [T_{2,2,s}, t)$, given $s\in\tau(\X)$,  then at time $t$, we have  $[\Qb^{(t)}_{4,p}]_{s,s}\leq \tilde{O}(1)$ for $p\in\{3,4\}$.
  \end{lemma}
  \begin{proof}
    Denote the first time that $[\Qb^{(t)}_{4,3}]_{s,s}$ reaches $\Omega(\log^{1+c} d)$ for some small constant $c>0$ as $\tilde{T}$. Then by direct calculations,  we have 
\begin{align*}    &\attn^{(t)}_{\ans,0\to\pred,2}\leq O\bigg(\frac{1}{e^{\Omega(\log^{1+c} d)}}\bigg),\\
    &\attn^{(t)}_{\ans,1\to\pred,1}+\attn^{(t)}_{\ans,1\to\ans,0}\leq O\bigg(\frac{1}{e^{\Omega(\log^{1+c} d)}}\bigg).
\end{align*}
Moreover, by \Cref{lem-s23-logit-cyc}, we can simply bound the logits as follows: 
\begin{align*}
    1-\logit^{(t)}_{j_1} &\leq  O\Bigg(\frac{1}{d^{C_B\big(1-2\attn^{(t   )}_{\ans,0\to\pred,2} \big)-1}}\Bigg)\leq O\Big(\frac{1}{d^{C_B/2-1}}\Big) \text{ for } \Zb^{2,0}\in \cE_1;\\
    1-\logit^{(t    )}_{j_2}& \leq  O\Bigg(\frac{1}{d^{C_B\big(1-2\attn^{(t)}_{\ans,1\to\pred,1}-2\attn^{(t    )}_{\ans,1\to\ans,0} \big)-1}}\Bigg)\\
 &~~~+O\Bigg(\frac{1}{d^{2\big(\attn^{(t)}_{\ans,1\to\ans,1}-\attn^{(t)}_{\ans,1\to\ans,0}\big)C_B}
 }\Bigg)\leq O\Big(\frac{1}{d^{C_B/2-1}}\Big) \text{ for } \Zb^{2,1}\in \tilde{\cE}_2.\\
\end{align*}
Thus, by focusing on $\cN_{s, 3,1,\rom1}$ and  $\cN_{s, 4,1,\rom1}$, we have 
\begin{align*}
    &\Bigg|\sum_{\ell=1}^2\Big[-\nabla_{\Q^{(t)}_{4,3}}\Loss^{2,\ell}_{5}\Big]_{s,s}\Bigg|\leq \frac{1}{d}\cdot O\Big(\frac{1}{d^{C_B/2-1}}\Big)\cdot  O\bigg(\frac{1}{e^{\Omega(\log^{1+c} d)}}\bigg)\cdot B\leq O\bigg(\frac{\log d}{e^{\Omega(\log^{1+c}d)}d^{C_B/2}}\bigg), 
\end{align*}
which implies 
\begin{align*}
    \Q^{(T)}_{4,3}\leq \Q^{(\tilde{T})}_{4,3}+O\bigg(\frac{T\log d}{e^{\Omega(\log^{1+c}d)}d^{C_B/2}}\bigg)\leq \Q^{(\tilde{T})}_{4,3}+O\bigg(\frac{\poly d\cdot \log d}{e^{\Omega(\log^{1+c}d)}d^{C_B/2}}\bigg)\leq O(\log^{1+c} d).
\end{align*}

  \end{proof}

\subsubsection{Attention Upper Bound}
  \begin{lemma}
   If \Cref{induction-s23-cyc} holds for all iterations $\in [T_{2,2,s}, t)$, given $s\in\tau(\X)$,  then at time $t$, for any $\Zb^{2,1}$,  we have  $\attn^{(t)}_{\ans,1\to \pred,2}\leq 0.5+\tilde{c}_1$, where $\tilde{c}_1>0$ is some small constant. 
  \end{lemma} 
  \begin{proof}
    Let $\tilde{T}$ denote the first time that $\E[\attn^{(t)}_{\ans,1\to \pred,2}\mid \tau(x_1)=s]$ reaches $0.5+\tilde{c}$. where $\tilde{c}>0$ is some small constant s.t., $2\tilde{c}\cdot C_B>1$. At 
$\tilde{T}$, we have $\attn^{(\tilde{T})}_{\ans,0\to \pred,1}\1_{\tau(x_0)=s}\geq \attn^{(\tilde{T})}_{\ans,1\to \pred,2}\1_{\tau(x_1)=s}$ 
 ; moreover, $\attn^{(\tilde{T})}_{\ans,1\to \ans,1}$ and $\attn^{(\tilde{T})}_{\ans,0\to \ans,0}$  is still at the constant level. 
\begin{itemize}
    \item For $\ell=1$, by \Cref{lem-s23-act-cyc} and \Cref{lem-s23-act-other-cyc}, 
    for $y\in\cY\setminus\{y_0\}$, we have only $r_{g_1\cdot y}$ has been activted to the linear regime for the prediction $\tau(g_1(y))$. Furthermore, we obtain 
    \begin{align*}       \sum_{y\in \cY\setminus\{y_0\}}\logit^{(\tilde{T})}_{5,\tau(g_1(y))}&\geq  \frac{\log d\cdot e^{2\tilde{c}C_B\log d}}{\log d\cdot e^{2\tilde{c}C_B\log d}+O(d)}\Big(1-\logit^{(\tilde{T})}_{5,j_1}\Big)\\
        &=(1-o(1))\cdot \Big(1-\logit^{(\tilde{T})}_{5,j_1}\Big).
    \end{align*}
    Thus,
  \begin{align*}
        &\Big[-\nabla_{\Q^{(\tilde{T})}_{4,3}}\Loss_{5}^{2,1}\Big]_{s,s}\\
        &= \E\Bigg[
        \attn^{(\tilde{T})}_{{\ans,0} \rightarrow \pred,1} \cdot \bigg( \Big(1-\logit^{(\tilde{T})}_{5,j_1}\Big)\cdot \Big(
 \Big( V_{j_1,  r_{g_1\cdot y_0}}(g_1)- \Lambda^{(\tilde{T})}_{5,j_1, r_{g_1\cdot y_0}}\pm\tilde{O}(\sigma_0) \Big)\pm \tilde{O}(\delta^{q}) \Big) \\
        &~~-\sum_{y\in\cY\setminus\{y_0\}}\logit^{(\tilde{T})}_{5,\tau(g_1(y))}\cdot  \Big(
 \Big( V_{\tau(g_1(y)),  r_{g_1\cdot y}}(g_1)- \Lambda^{(\tilde{T})}_{5,\tau(g_1(y)), r_{g_1\cdot y}}\pm\tilde{O}(\sigma_0) \Big)\pm \tilde{O}(\delta^{q}) \Big)\\
 &~~~\pm\frac{1}{\poly d} \Big(1-\logit^{(\tilde{T})}_{5,j_1}\Big)\cdot \tilde{O}(\sigma_0^q)
        \bigg)\1_{\tau(x_0)=s}\Bigg]\\
         &\leq  \E\Bigg[
        \attn^{(\tilde{T})}_{{\ans,0} \rightarrow \pred,1} \cdot \bigg( \Big(1-\logit^{(\tilde{T})}_{5,j_1}\Big)\cdot \Big(
 \Big( 2 \attn^{(\tilde{T})}_{{\ans,0} \rightarrow \pred,2} \cdot B\pm\tilde{O}(\sigma_0) \Big)\pm \tilde{O}(\delta^{q}) \Big) \\
        &~~-\sum_{y\in\cY\setminus\{y_0\}}\logit^{(\tilde{T})}_{5,\tau(g_1(y))}\cdot  \Big(
 \Big( 2 \attn^{(\tilde{T})}_{{\ans,0} \rightarrow \ans,0}\cdot B\pm\tilde{O}(\sigma_0) \Big)\pm \tilde{O}(\delta^{q}) \Big)\\
 &~~~\pm\frac{1}{\poly d} \Big(1-\logit^{(\tilde{T})}_{5,j_1}\Big)\cdot \tilde{O}(\sigma_0^q)
        \bigg)\1_{\tau(x_0)=s}\Bigg]<0.
    \end{align*}
    \item For $\ell=2$, for $\Zb^{2,1}\in\tilde{\cE}_2\cup \{g_1=g_2, y_0\neq y_1x\}$,  we have  
    \begin{align*}       \sum_{y\in \cY\setminus\{y_0,y_1\}}\logit^{(\tilde{T})}_{5,\tau(g_2(y))}&\geq \frac{\log d\cdot e^{2\tilde{c}C_B\log d}}{\log d\cdot e^{2\tilde{c}C_B\log d}+O(d)}\Big(1-\logit^{(\tilde{T})}_{5,j_2}-\logit^{(\tilde{T})}_{5,j_2^{\prime}}\Big)\\
        &=(1-o(1))\cdot \Big(1-\logit^{(\tilde{T})}_{5,j_2}-\logit^{(\tilde{T})}_{5,j_2^{\prime}}\Big).
    \end{align*}
    Otherwise, we  have 
       \begin{align*}       \sum_{y\in \cY\setminus\{y_1\}}\logit^{(\tilde{T})}_{5,\tau(g_2(y))}
        &=(1-o(1))\cdot \Big(1-\logit^{(\tilde{T})}_{5,j_2}-\logit^{(\tilde{T})}_{5,j_2^{\prime}}\Big).
    \end{align*}
    Therefore, similar to $\ell=1$, we obtain $\Big[-\nabla_{\Q^{(\tilde{T})}_{4,3}}\Loss_{5}^{2,2}\Big]_{s,s}<0$. 
\end{itemize}
Combing the above two cases, we have $\sum_{\ell=1}^2\Big[-\nabla_{\Q^{(\tilde{T})}_{4,3}}\Loss^{2,\ell}_{5}\Big]_{s,s}<0$. It is also clear from \Cref{lem-s23-gd1-cyc} that $\sum_{\ell=1}^2\Big[-\nabla_{\Q^{(\tilde{T})}_{4,4}}\Loss^{2,\ell}_{5}\Big]_{s,s}\geq 0$. Hence $\Attn^{(\tilde{T})}_{\ans,1\to\pred,2}$ cannot further grow once it reaches $0.5+\tilde{c}$.
  \end{proof}

\subsubsection{Decreasing Gap at the End of Convergence}
Let $\tilde{T}$ denote the first time that $\E[\ate^{2,2}\mid \tau(x_1)=s]\leq 3\epsilon$, if $\E[\Delta^{2,2}\mid \tau(x_1)=s]
\leq O\big(\frac{(d^{1.01}\epsilon)^{\frac{1}{q-1}}}{\log d}\big)$, then we can let $T^{\star}=\tilde{T}$ and stop the training. Otherwise, we have $\E[\Delta^{2,2}\mid \tau(x_1)=s]\geq \Omega\big(\frac{(d^{1.01}\epsilon)^{\frac{1}{q-1}}}{\log d}\big)$. Following the similar argument as in \Cref{lem-s23-gd1-cyc}, we have the gradient contribution from $\ell=1$ is dominated by the gradient contribution from $\ell=2$. Thus, we focus on $\ell=2$, and obtain
  \begin{align*}
        &\Big[-\nabla_{\Q^{(\tilde{T})}_{4,3}}\Loss_{5}^{2,2}\Big]_{s,s}\\
        &= \E\Bigg[
\attn^{(\tilde{T})}_{{\ans,1} \rightarrow \pred,2} \cdot \bigg( \Big(1-\logit^{(\tilde{T})}_{5,j_2}\Big)\cdot \Big(
 \Big( V_{j_2,  r_{g_2\cdot y_1}}(g_2)- \Lambda^{(\tilde{T})}_{5,j_2, r_{g_2\cdot y_1}}\pm\tilde{O}(\sigma_0) \Big)\pm \tilde{O}(\delta^{q}) \Big) \\
       &~~~~~~~~\pm\frac{1}{\poly d} \Big(1-\logit^{(\tilde{T})}_{5,j_2}\Big)\cdot \tilde{O}(\sigma_0^q)
        \bigg)\1_{\{\tau(x_0)=s\}}\Bigg]\\
         & -\E\Bigg[
\attn^{(\tilde{T})}_{{\ans,1} \rightarrow \pred,2} \cdot \bigg(\sum_{y\in\cY\setminus\{y_1\}}\logit^{(\tilde{T})}_{5,\tau(g_2(y))}\cdot  \ReLU^{\prime}\big(\Lambda^{(\tilde{T})}_{5,\tau(g_2(y)), r_{g_2\cdot y}}\big)\\
        &~~~~~~~~~~~\cdot \Big(
 \Big( V_{\tau(g_2(y)),  r_{g_2\cdot y}}(g_2)- \Lambda^{(\tilde{T})}_{5,\tau(g_2(y)), r_{g_2\cdot y}}\pm\tilde{O}(\sigma_0) \Big)\pm \tilde{O}(\delta^{q}) \Big) \bigg)\1_{\{\tau(x_0)=s\}\cap \tilde{\cE}_{2}}\Bigg]\\
        &-\E\Bigg[
    \attn^{(t)}_{{\ans,1} \rightarrow \pred,2} \cdot\sum_{j\neq j_2\in\tau(\cY)}\logit^{(t)}_{5,j} \cdot \notag\\
    &~~~~~~~~~~\Big(\sum_{r\in\hat{\fA}_{j}}\ReLU^{\prime}(\Lambda^{(t)}_{5,j,r
    })\cdot  \Big( V_{j, r}(g_2)- \Lambda^{(t)}_{5,j,r}\pm\tilde{O}(\sigma_0) \Big)\pm \tilde{O}(\delta^{q}) \Big) 
\1_{\{\tau(x_1)=s\}\cap \tilde{\cE}^c_2}\Bigg]\\
         &\leq  \E\Bigg[\Big(1-\logit^{(t)}_{5,j_2}\Big)\cdot  O(\epsilon\log d)\1_{\tau(x_1)=s}\Bigg]. 
    \end{align*}
    Turning to $\Qb_{4,4}$, we have 
 \begin{align}
        &\Big[-\nabla_{\Q^{(\tilde{T})}_{4,4}}\Loss_{5}^{2,2}\Big]_{s,s} \notag\\
        &= \E\Bigg[
\attn^{(\tilde{T})}_{{\ans,1} \rightarrow \ans,1} \cdot \bigg( \Big(1-\logit^{(\tilde{T})}_{5,j_2}\Big)\cdot \Big(
 \Big( V_{j_2,  r_{g_2\cdot y_1}}(y_1)- \Lambda^{(\tilde{T})}_{5,j_2, r_{g_2\cdot y_1}}\pm\tilde{O}(\sigma_0) \Big)\pm \tilde{O}(\delta^{q}) \Big) \notag\\
       &~~~~~~~~\pm\frac{1}{\poly d} \Big(1-\logit^{(\tilde{T})}_{5,j_2}\Big)\cdot \tilde{O}(\sigma_0^q)
        \bigg)\1_{\{\tau(x_0)=s\}}\Bigg] \notag \\
         & -\E\Bigg[
\attn^{(\tilde{T})}_{{\ans,1} \rightarrow \ans,1} \cdot \bigg(\sum_{y\in\cY\setminus\{y_1\}}\logit^{(\tilde{T})}_{5,\tau(g_2(y))}\cdot  \ReLU^{\prime}\big(\Lambda^{(\tilde{T})}_{5,\tau(g_2(y)), r_{g_2\cdot y}}\big) \notag\\
        &~~~~~~~~~~~\cdot \Big(
 \Big( V_{\tau(g_2(y)),  r_{g_2\cdot y}}(y_1)- \Lambda^{(\tilde{T})}_{5,\tau(g_2(y)), r_{g_2\cdot y}}\pm\tilde{O}(\sigma_0) \Big)\pm \tilde{O}(\delta^{q}) \Big) \bigg)\1_{\{\tau(x_0)=s\}\cap \tilde{\cE}_{2}}\Bigg] \label{eq-s23-de-gap-1}\\
        &-\E\Bigg[
    \attn^{(t)}_{{\ans,1} \rightarrow \ans,1} \cdot\sum_{j\neq j_2\in\tau(\cY)}\logit^{(t)}_{5,j} \cdot \notag\\
    &~~~~~~~~~~\Big(\sum_{r\in\hat{\fA}_{j}}\ReLU^{\prime}(\Lambda^{(t)}_{5,j,r
    })\cdot  \Big( V_{j, r}(y_1)- \Lambda^{(t)}_{5,j,r}\pm\tilde{O}(\sigma_0) \Big)\pm \tilde{O}(\delta^{q}) \Big) 
\1_{\{\tau(x_1)=s\}\cap \tilde{\cE}^c_2}\Bigg]. \label{eq-s23-de-gap-2}
    \end{align}
For \eqref{eq-s23-de-gap-1}, we have
\begin{align*}
    \eqref{eq-s23-de-gap-1} &\geq \E\Bigg[
        \attn^{(\tilde{T})}_{{\ans,1} \rightarrow \ans,1} \cdot  \bigg( -\logit^{(\tilde{T})}_{5,j^{\prime}_2}\cdot  \\
&~~~~~~\min \bigg\{\Omega(\log d),\Omega\Big(\big((\attn^{(\tilde{T})}_{\ans,1\to\pred,2}-\attn^{(\tilde{T})}_{\ans,1\to\ans,1})\log d\big)^{q-1} \log d\Big)\bigg\}\bigg)\1_{\{\tau(x_0)=s\}\cap \tilde{\cE}_{2}}\Bigg]\\
&\geq \E\Bigg[
      \Big( 1-\logit^{(\tilde{T})}_{5,j_2}\Big)  \cdot \Omega\Big(\frac{1}{d}\Big)  \cdot  \\
&~~~~~~\min \bigg\{\Omega(\log d),\Omega\Big(\big((\attn^{(\tilde{T})}_{\ans,1\to\pred,2}-\attn^{(\tilde{T})}_{\ans,1\to\ans,1})\log d\big)^{q-1} \log d\Big)\bigg\}\1_{\{\tau(x_0)=s\}\cap \tilde{\cE}_{2}}\Bigg].
\end{align*}
Moreover, for \eqref{eq-s23-de-gap-2}, we have
\begin{align*}
  \eqref{eq-s23-de-gap-2}  &\geq -\E\Bigg[
    \attn^{(\tilde{T})}_{{\ans,1} \rightarrow \ans,1} \cdot  \logit^{(\tilde{T})}_{5,\tau(g_1(y_1))}  \cdot \ReLU^{\prime}\big(\Lambda^{(\tilde{T})}_{5,\tau(g_1(y_1)),r_{g_1\cdot y_1}}\big) \notag \\
    &~~~\Big(\Big( V_{\tau(g_1(y_1)), r_{g_1\cdot y_1}}(y_1)- \Lambda^{(\tilde{T})}_{5,\tau(g_1(y_1)),r_{g_1\cdot y_1}}\pm\tilde{O}(\sigma_0) \Big)\pm \tilde{O}(\delta^{q}) \Big) 
\1_{\{\tau(x_1)=s\}\cap \{g_1\neq g_2\wedge y_0=y_1\}}\Bigg] \\
&\geq -\E\Bigg[
      \Big( 1-\logit^{(\tilde{T})}_{5,j_2}\Big)  \cdot O\Big(\frac{1}{d}\Big)  \cdot O\Big(\frac{1}{\log d}\Big) \cdot O\Big(\big(\epsilon \log d\big)^{q-1} \Big)\1_{\{\tau(x_0)=s\}\cap \tilde{\cE}_{2}}\Bigg],
\end{align*}
where the last inequality holds since $\Lambda^{(\tilde{T})}_{5,\tau(g_1(y_1)),r_{g_1\cdot y_1}}\leq \Big(\attn^{(\tilde{T})}_{\ans,1\to\pred,2}-\attn^{(\tilde{T})}_{\ans,1\to\ans,1}\Big)B\leq O(\epsilon \log d)$. 

Since $\E\Big[\Delta^{2,2}\mid \tau(x_1)=s\Big]\geq \Omega\big(\frac{(d^{1.01}\epsilon)^{\frac{1}{q-1}}}{\log d}\big)$ 
we have
 $$\eqref{eq-s23-de-gap-1}\geq d^{0.01} \cdot  \bigg| \Big[-\nabla_{\Q^{(\tilde{T})}_{4,3}}\Loss_{5}^{2,2}\Big]_{s,s}\bigg|, 
$$
and thus $\eqref{eq-s23-de-gap-1}\gg \eqref{eq-s23-de-gap-2}$. 
Thus, for $\epsilon \ll \frac{1}{d}$, 
if the attention gap does not decrease to the level of $O\bigg(\frac{(d^{1.01}\epsilon)^{\frac{1}{q-1}}}{\log d}\bigg)$, $[\Q_{4,4}]_{s,s}$ will start to dominantly increase while $[\Q_{4,3}]_{s,s}$ will not change too much. On the other hand, if the gap of attention holds, then $[\Q_{4,3}]_{s,s}\geq [\Q_{4,4}]_{s,s}$,  we have 
\begin{align*}
   \ate^{2,2}= 1-\attn^{(t)}_{\ans,1\to\pred,2}-\attn^{(t)}_{\ans,1\to\ans,1}&=\frac{O(1)}{O(1)+e^{[\Qb^{(t)}_{4,4}]_{s,s}}+e^{[\Qb^{(t)}_{4,3}]_{s,s}}}\\
    &\geq \frac{O(1)}{O(1)+2e^{[\Qb^{(\tilde{T})}_{4,3}]_{s,s}}}\geq \frac{1}{2}\cdot 3\epsilon>\epsilon. 
\end{align*}
This implies that, we can find some time between $\tilde{T}$ and $T_{2,3,s}$, s.t., the gap will decrease to $O\bigg(\frac{(d^{1.01}\epsilon)^{\frac{1}{q-1}}}{\log d}\bigg)$. We denote this time as $T^{\star}$ and stop the training.


\subsubsection{At the End of the Training}
Putting everything together, we have that at the end of the training, we have 
\begin{lemma}\label{lem-s23-end}
   Given $s\in\tau(\X)$, at $T^{\star}=\tilde{O}\Big(\frac{d^{(1-2\epsilon)C_B}}{\eta \epsilon }\Big)$, if $\tilde{\Omega}(\sigma_0)\ll \epsilon\leq O(\frac{1}{d^{1.01}})$, we have
    \begin{enumerate}[(a)]
        \item Attention convergence: $\ate^{2,\ell}\leq O(\epsilon)$ for $\ell\in[2]$, and $\E\big[\Delta^{2,2}\mid \tau(x_1)=s\big]\leq O(\frac{(d^{1.01}\epsilon)^{\frac{1}{q-1}}}{\log d})$;
        \item $\Big[\Qb^{(T_{2,3,s})}_{4,p}\Big]_{s,s'}\geq \Omega(\log d)$ for $p\in \{3,4\}$ if $s=s'\in\tau(\X)$, otherwise, $\Big[\Qb^{(T_{2,3,s})}_{4,p}\Big]_{s,s'}\leq \tilde{O}(\frac{1}{d})$;
        \item Loss convergence: $\sum_{\ell=1}^2\Loss^{2,\ell}_{5}\leq \frac{1}{\poly d}$.
    \end{enumerate}
\end{lemma}

\subsection{Proof of Main Theorem}

\begin{theorem}[Restatement of \Cref{thm:length-generalization}]\label{thm:length-generalization-re}
    Under Assumptions~\ref{assump:asymptotic-regime}, \ref{assump:lego-data-distribution}, \ref{assump:init}, \ref{assumption:output-bound},  \ref{assump-Q-structure} and \ref{assump:structure-1}, for some constant $0<c^{*}<1$, $n_y<m\ll \log^2 d$, the transformer model \(F^{(T_1 + T_2)}\) obtained by Algorithm~\ref{alg:cot-transitive-training} with learning rate \(\eta = \frac{1}{\poly(d)}\),  and stage 1 and 2 iteration $T_1=\tilde{O}\Big(\frac{d}{\eta (\sigma_0)^{q-2}}\Big)$, $T_2=\tilde{O}\Big(\frac{\poly (d)}{\eta \sigma_0 }\Big)$ satisfies 
        \begin{align}
        \mathrm{Acc}_L\!\Bigl(F^{(T_1 + T_2)}\Bigr) \;\geq\; 1 - \frac{1}{\poly(d)}, \text{ for every \(L \leq O(d^{c^{*}})\), }
    \end{align}
i.e., \(F^{(T_1 + T_2)}\), which is trained for task \(\cT^1\) and \(\cT^2\), generalizes to solve the tasks \(\cT^{\ell}, \ell \leq L\). 
    \end{theorem}
\begin{proof}
    By \Cref{lem-s23-end}, at the end of Stage 2 training, we have $\Big[\Qb^{(T_{2,3,s})}_{4,p}\Big]_{s,s} \geq \Omega(\log d)$ for all $p \in \{3,4\}$ and $s \in \tau(\X)$. 

Therefore, for task $\cT^L$ with $L \leq O(d^{c^*})$, where $0<c^{*}<1$ is a constant depending on the value of $\epsilon$ in \Cref{lem-s23-end},  we obtain
\[
\ea^{L,\ell} \leq \frac{O(1) \cdot L}{O(1) \cdot L + e^{\Big[\Qb^{(T_{2,3,s})}_{4,3}\Big]_{s,s}} + e^{\Big[\Qb^{(T_{2,3,s})}_{4,4}\Big]_{s,s}}} = o(1).
\]
Moreover, we have 
\[
\Delta^{L,\ell} \leq \Delta^{2,1} = \attn^{(T_{2,3,s})}_{\ans,1 \to \pred,2} - \attn^{(T_{2,3,s})}_{\ans,1 \to \ans,1} \leq o(1).
\]

These together guarantee that
\[
1 - \logit_{5, \tau(g_{\ell+1}(y_\ell))}(F^{(T^\star)}, \Zb^{(L,\ell)}) 
\leq \frac{O(1) \cdot d + e^{o(1)}}{O(1) \cdot d + e^{o(1)} + e^{\Omega(\log d)}} 
\leq \frac{1}{\poly(d)},
\]
which implies 
\[
\mathrm{Acc}_L\big(F^{(T^\star)}\big) \geq 1 - \frac{1}{\poly(d)}.
\]
\end{proof}
\section{Learning the Attention Layer: Symmetry Case for Short-Length}
In this part, we consider the scenario where the group operations form a symmetry group. Specifically, following \Cref{assump:structure-2}, we assume that $\cY = \{0, 1, \dots, n_y - 1\}$, and let $\cG$ be the symmetry group for $\cY$, with order $|\cG| = n_y! =\Theta(\polylog d)\gg \frac{1}{\varrho}$, where $n_y = \Theta\left(\frac{\log \log d}{\log\log  \log d}\right)$. Similar to the simply transitive case, we restrict our analysis to updating only $\Qb$, specifically the blocks $\Qb_{4,3}$ and $\Qb_{4,4}$.

Throughout this section, we focus on the simple task $\cT^2$ and analyze gradient descent updates with respect to the per-token loss $\Loss_{5}^{2,2}$. Given $s \in \tau(\X)$,  let $\Loss^{2,2}_{5,s} = -\E\Big[\log p_{F}(\Z_{\ans,2,5} \mid \Z^{2,1}) \,\big|\, \tau(x_1) = s\Big]$. Due to the symmetry of $[\Qb_{4,3}]_{s,s}$ and $[\Qb_{4,4}]_{s,s}$ across $s\in\tau(\X)$, we may, without loss of generality, focus on a particular $s \in \tau(\X)$ and analyze the corresponding loss $\Loss^{2,2}_{5,s}$ in what follows.

\subsection{Gradient Computations}
We start with the gradient computations for the attention layer. 

\paragraph{Notations for gradient expressions.} We first introduce some notations for the gradients of the attention layer. For $1 \leq \ell\leq L$, given $\Zb^{L,\ell-1}$ and $\kk\in \mathcal{I}^{L, \ell-1}$, define   
   \begin{align}
    & \Xi^{L}_{\ell, i,\kk}(\Zb^{L,\ell-1})\triangleq \sum_{j \in[d]} \Ecal_{i,j}(\Zb^{L,\ell-1}) \sum_{r\in [m]}\ReLU^{\prime}\big(\Lambda_{i, j,r}(\Zb^{L,\ell-1})\big)\dbrack{\Wb_{i,j,r},\Z_{\kk}}, \quad { i\in [5].}\label{eq-def-xi-sym}
    \end{align}
For simplicity of notation, we will henceforth omit the dependence on $\Zb^{L,\ell-1}$ in the notation of $\Xi^{L}_{\ell, i,\kk}$ 
when it is clear from the context.

\begin{fact}[Gradients of \(\Q\)]\label{fact:gradients-Q-sym}
    For any \(p,q \in [5]\), 
    we have 
    \begin{align*}
        &-\nabla_{\Q_{p,q}}\Loss^{L}  = \sum_{\ell=1}^{L}\sum_{i\in [5]} -\nabla_{\Q_{p,q}}\Loss^{L,\ell}_{i}, \quad  \text{ where }\\
      & -\nabla_{\Q_{p,q}}\Loss^{L,\ell}_{i}=\\
       & ~~~~~~~\E\Bigg[\sum_{\mathbf{k} \in \mathcal{I}^{L, \ell-1}}\attn_{{\ans,\ell-1} \rightarrow \kk} \cdot\left(\Xi^{L}_{\ell, i,\kk} - \sum_{\mathbf{k}^{\prime} \in \mathcal{I}^{L, \ell-1}}\attn_{{\ans,\ell-1} \rightarrow \kk^{\prime}}\Xi^L_{\ell, i,\kk^{\prime}}\right)\Z_{\ans,\ell-1,p}\Z_{\kk,q}^{\top} \Bigg].
    \end{align*}
\end{fact}

\begin{lemma}[Gradients of $\Qb_{4,3}$]\label{lem-gradients-Q43-sym}
    Given $s\in\tau(\X)$, for the diagonal entry \( [\Q_{4,3}]_{s,s} \) of the block \(\Qb_{4,3}\), we have 
    \begin{align*}
    \Big[-\nabla_{\Q_{4,3}}\Loss^{2,2}_{5}\Big]_{s,s}&= \E\Bigg[
    \attn_{{\ans,1} \rightarrow \pred,2} \cdot \bigg(\sum_{j \in[d]} \Ecal_{5,j}(\Zb^{2,1})\sum_{r\in [m]}\ReLU^{\prime}\big(\Lambda_{5, j,r}\big)\cdot  \\
    &~~~~~~~~~~~~~\Big( \dbrack{\Wb_{5,j,r},\Z_{\pred,2}}- \Lambda_{5, j,r}+b_{5,j,r}\Big)\bigg) \1_{s=\tau(x_1)}\Bigg].
\end{align*}
Moreover, for the off-diagonal entries \( [\Q_{4,3}]_{s,s'} \) with \( s \neq s' \), we have
    \begin{align*}
    \Big[-\nabla_{\Q_{4,3}}\Loss^{2,2}_{5}\Big]_{s,s'}&= \E\Bigg[
    \attn_{{\ans,1} \rightarrow \pred,1} \cdot \bigg(\sum_{j \in[d]} \Ecal_{5,j}(\Zb^{2,1})\sum_{r\in [m]}\ReLU^{\prime}\big(\Lambda_{5, j,r}\big)\cdot  \\
    &~~~~~\Big( \dbrack{\Wb_{5,j,r},\Z_{\pred,1}}- \Lambda_{5, j,r}+b_{5,j,r}\Big)\bigg) \1_{s=\tau(x_1),s'=\tau(x_0)}\Bigg].
\end{align*}
\end{lemma}

\begin{lemma}[Gradients of $\Qb_{4,4}$]\label{lem-gradients-Q44-sym}
    Given $s\in\tau(\X)$, for the diagonal entry \( [\Q_{4,4}]_{s,s} \) of the block \(\Qb_{4,4}\), we have 
    \begin{align*}
    \Big[-\nabla_{\Q_{4,4}}\Loss^{2,2}_{5}\Big]_{s,s}&= \E\Bigg[
    \attn_{{\ans,1} \rightarrow \ans,1} \cdot \bigg(\sum_{j \in[d]} \Ecal_{5,j}(\Zb^{2,1})\sum_{r\in [m]}\ReLU^{\prime}\big(\Lambda_{5, j,r}\big)\cdot  \\
    &~~~~~~~~~~~~~\Big( \dbrack{\Wb_{5,j,r},\Z_{\ans,1}}- \Lambda_{5, j,r}+b_{5,j,r}\Big)\bigg) \1_{s=\tau(x_1)}\Bigg].
\end{align*}
Moreover, for the off-diagonal entries \( [\Q_{4,4}]_{s,s'} \) with \( s \neq s' \), we have 
    \begin{align*}
    \Big[-\nabla_{\Q_{4,4}}\Loss^{2,2}_{5}\Big]_{s,s'}&= \E\Bigg[
    \attn_{{\ans,1} \rightarrow \ans,0} \cdot \bigg(\sum_{j \in[d]} \Ecal_{5,j}(\Zb^{2,1})\sum_{r\in [m]}\ReLU^{\prime}\big(\Lambda_{5, j,r}\big)\cdot  \\
    &~~~~~\Big( \dbrack{\Wb_{5,j,r},\Z_{\ans,0}}- \Lambda_{5, j,r}+b_{5,j,r}\Big)\bigg) \1_{s=\tau(x_1),s'=\tau(x_0)}\Bigg].
\end{align*}
\end{lemma}

\subsection{Some Useful Bounds for Gradients}\label{sec:useful-bounds-for-gradients-sym}
In this subsection, we establish several useful bounds on the gradients of the attention layer, leveraging
the feature structure of the MLP layer learned during stage 1.1. These bounds will be instrumental
for the subsequent analysis.

Recall that for \(j\in\tau(\cY)\) and \(y \in \cY\),  the fiber \(\fiber_{j,y}\), and  the set of feature combinations for predicting \(y = \tau^{-1}(j)\), are defined as
   \begin{align*}
        \fiber_{j,y} = \{g \in \cG \mid \tau\big(g(y)\big) = j\},\quad 
        \fF_{j} = \Big\{(\fiber_{j,y}, y) \in 2^{\cG} \times \cY\Big\}.
    \end{align*}

As established in \Cref{thm:learning-symmetric-actions}, at the end of stage 1.1,  we have 
    \begin{itemize}
        \item \textbf{Sparse activations}: For \(j\in\tau(\cY)\), let \(\phi = (\fiber_{j,y},y)\in\fF_j\), then there exists exactly one \emph{activated} neuron \(r\in[m]\) such that when \(g_1 =g \in \fiber_{j,y}, y_0 = y\) happens:
        \begin{align*}
            &\Lambda_{5,j,r}^{(T_{1.1})} \geq B - O(\delta), \qquad \Big|C_\alpha V_{j,r}^{(t)}(y) - V_{j,r}^{(t)}(g)\Big| \leq O(\delta)  
             \\
            &\Lambda_{5,j,r'}^{(T_{1.1})}  \leq O(d^{-\Omega(1)}) \quad \forall r'\neq r.
        \end{align*}
        for some \(C_\alpha = \Theta(n_y)\);
        \item \textbf{Cancellation of incorrect features}: For \(j\in\tau(\cY)\), let \(\phi = (\fiber_{j,y},y)\in\fF_j\), and let the \(r \in [m]\) be the activated neuron, then for any \(g' \notin \fiber_{j,y} \), and any \(y'\neq y\in \cY\), we have
        \begin{align*}
            \Big|V_{j,r}^{(T_{1.1})} (g) + V_{j,r}^{(T_{1.1})} (y')\Big| \leq O(\delta) \quad \text{and} \quad \Big|V_{j,r}^{(T_{1.1})} (g') + V_{j,r}^{(T_{1.1})} (y)\Big| \leq O(\delta).
        \end{align*}
       
    \end{itemize}
In the analysis of FFN layer,  we have a pair \(\delta=(\delta_1,\delta_2)\); in the
expressions above, some terms should use \(\delta_1\) and others \(\delta_2\).
\paragraph{Notations for activated neurons.} 
We denote by $r_{j,y}$ the unique activated neuron corresponding to $\phi = (\fiber_{j,y},y) \in \fF_j$. For any $g \in \fiber_{j,y}$, we also write $r_{g\cdot y}$ for the same neuron $r_{j,y}$. Note that $r_{g_1\cdot y} = r_{g_2\cdot y}$ for distinct $g_1, g_2 \in \fiber_{j,y}$. Moreover, define
\begin{align*}
\fA\triangleq\cup_{j\in\tau(\Y)}\fA_{j}, \text{ where } \fA_j\triangleq \{r_{j,y}\mid y\in\Y \}.
\end{align*}
In other words, \( \fA \) is the set of all activated neurons across all feature sets \( \fF_j \) for \( j \in \tau(\Y) \). Given $\Zb^{L,\ell-1}$, letting $\hat{\cG}(\Zb^{L,\ell-1})=\cup\{g_{\ell'}\}_{\ell'=1}^L$ be the collection of all the chosen group elements in the predicate clauses. Similarly $\hat{\cY}=\cup\{y_{\ell'}\}_{\ell'=0}^{\ell-1}$. Then define $\hat{\fA}_{j}(\Zb^{L,\ell-1})= \bigg\{r_{g\cdot y}\mid g\in\fiber_{j,y}\wedge \Big(g\in \hat{\cG}(\Zb^{L,\ell-1})\vee y\in \hat{\cY}\Big) \bigg\}$. For simplicity, we omit the dependence on \( \Zb^{L,\ell-1} \) in the notation of \( \hat{\fA}_j \) when it is clear from the context. Equipped with these notations, we can summarize the above properties in the following lemmas.

\begin{lemma}[Properties of target feature magnitude]\label{lem-prop-psi-sym}
   Given $j\in\tau(\Y)$ and $y\in\cY$ then for  \(g\in\fiber_{j,y}\), the following properties hold.
    \begin{align}
    & V_{j, r_{g \cdot y}}(g) + V_{j, r_{g \cdot y}}(y)  \geq 2B - O(\delta),  \quad \Big|C_\alpha V_{j,r}^{(t)}(y) - V_{j,r}^{(t)}(g)\Big| \leq O(\delta), 
    \label{attn-init-prop-sym-1} \\
    &\left| V_{j, r_{g \cdot y}}(g) + V_{j, r_{g \cdot y}}(y') \right| \leq O(\delta), V_{j, r_{g \cdot y}}(y')<0 \quad \text{for all } y' \neq y, \label{attn-init-prop-sym-2}\\
   & \left| V_{j, r_{g \cdot y}}(g') + V_{j, r_{g \cdot y}}(y) \right| \leq O(\delta), V_{j, r_{g \cdot y}}(g')<0 \quad \text{for all } g' \notin \fiber_{j,y}.\label{attn-init-prop-sym-3}\\
   &\left|V_{j, r}(g)\right|, \left| V_{j, r}(y) \right| \leq O(\delta) \quad \text{for all } r \notin\fA_{j}. \label{attn-init-prop-sym-4}
\end{align}
\end{lemma}

\begin{lemma}[Properties of irrelevant magnitude] \label{lem-prop-irrelavant-sym}
    If \( (p,v) \notin \{2\} \times \cG \cup \{5\} \times \cY \), or \( j \notin \tau(\Y) \), then for any \( r \in [m] \), we have
    \begin{align}
       \big| \langle \Wb_{5,j,r,p}, e_v\rangle\big|\leq \tilde{O}(\sigma_0).
    \end{align}
\end{lemma}
The above lemmas give us some direct computations of the inner products between the weight matrices and input embedding vectors.

\begin{lemma}
Let \( j \in \tau(\cY) \) and \( \ell \in [2] \). Then for any \( r \in [m] \), the following holds:
\begin{align}
    \langle \Wb_{5,j,r}, \Zb_{\pred,\ell} \rangle &= V_{j,r}(g_{\ell}) \pm \tilde{O}(\sigma_0), \label{eq-inner-product-1-sym} \\
    \langle \Wb_{5,j,r}, \Zb_{\ans,\ell-1} \rangle &= V_{j,r}(y_{\ell-1}) \pm \tilde{O}(\sigma_0). \label{eq-inner-product-2-sym}
\end{align}
Moreover, for \( j \notin \tau(\cY) \) and any \( \kk \in \cI^{2,1} \) and \( r \in [m] \), we have
\begin{align}
    \big| \langle \Wb_{5,j,r}, \Zb_{\kk} \rangle \big| = \tilde{O}(\sigma_0). \label{eq-inner-product-3-sym}
\end{align}
\end{lemma}
Furthermore, we can establish some characterizations of the \(\Lambda_{5,j,r}(\Zb^{2,\ell-1})\) quantities, which are crucial for  the following analysis.  
\begin{lemma}[Characterizations of Lambda]\label{lem-lambda-char-sym}
    Given $\Zb^{2,1}$ with  \( \{\attn_{\ans,1 \to \kk} \}_{\kk\in\cI^{2,1}}\), then 
    \begin{enumerate}[(a)]
        \item for \( j \in \tau(\cY) \), for activated neuron $r\in\fA_j$, we have  
        \begin{align*}
            \Lambda_{5,j,r}=\sum_{\ell'=1}^2 \attn_{\ans,1 \to \pred,\ell'}V_{j,r}(g_{\ell'})+  \sum_{\ell'=1}^{2} \attn_{\ans,1 \to \ans,\ell'-1} V_{j,r}(y_{\ell'-1})\pm  \tilde{O}(\sigma_0). 
        \end{align*}
         \item for \( j \in \tau(\cY) \), for any non-activated neuron $r\notin\fA_j$ 
         we have
      \begin{align*}
            \Big|\Lambda_{5,j,r}\Big| \leq {O}(\delta).
        \end{align*} 
        \item for \( j \notin \tau(\cY) \),  for any $r\in [m]$,  we have 
        \begin{align*}
            \Big|\Lambda_{5,j,r}\Big|%
            \leq \tilde{O}(\sigma_0).
        \end{align*} 
    \end{enumerate}
\end{lemma}

A direct consequence of the above lemma is the following finer characterization of the activated neurons.

\begin{lemma}\label{lem-non-activated-neuron-sym}
    Given $j\in\tau(\cY)$, for $r\in \fA_{j}\setminus\hat{\fA}_{j}$, we have $\ReLU^{\prime}\big(\Lambda_{5, j,r}\big)=0$.
\end{lemma}
Now we are ready to further derive the gradients of the attention layer starting from \Cref{lem-gradients-Q43-sym,lem-gradients-Q44-sym} and the properties established above.
\begin{lemma}[Refined expression for the gradient of $\Qb_{4,3}$] \label{lem-refined-grad-Q43-sym} Given $s\in\tau(\X)$, for the diagonal entry \( [\Q_{4,3}]_{s,s} \) of the block \(\Qb_{4,3}\), letting $j_2=\tau(g_2(y_1))$, we have 
  \begin{align*}
    &\Big[-\nabla_{\Q_{4,3}}\Loss^{2,2}_{5}\Big]_{s,s}= 
      \E\Bigg[
    \attn_{{\ans,1} \rightarrow \pred,2} \cdot \\
    &~~~~~~~~~~~~~ \bigg( (1-\logit_{5,j_2})\cdot \Big(\sum_{r\in\hat{\fA}_{j_2}}\ReLU^{\prime}(\Lambda_{5,j_2, r
    })\cdot \Big( V_{j_2,  r}(g_2)- \Lambda_{5,j_2, r}\pm\tilde{O}(\sigma_0) \Big)\pm \tilde{O}(\delta^{q}) \Big)\\   
    &~~~~~~~~~~-\sum_{j\neq j_2\in\tau(\cY)}\logit_{5,j} \cdot \Big(\sum_{r\in\hat{\fA}_{j}}\ReLU^{\prime}(\Lambda_{5,j,r
    })\cdot  \Big( V_{j, r}(g_2)- \Lambda_{5,j,r}\pm\tilde{O}(\sigma_0) \Big)\pm \tilde{O}(\delta^{q}) \Big) \\
    &~~~~~~~~~~~~\pm\sum_{j\notin\tau(\cY)}\logit_{5,j}\tilde{O}(\sigma^{q}_0)  \bigg)\1_{\tau(x_1)=s}\Bigg].
\end{align*}
Moreover, for the off-diagonal entries \( [\Q_{4,3}]_{s,s'} \) with \( s \neq s' \), we have
  \begin{align*}
    &\Big[-\nabla_{\Q_{4,3}}\Loss^{2,2}_{5}\Big]_{s,s'}=  \E\Bigg[
    \attn_{{\ans,1} \rightarrow \pred,1} \cdot\\
    &~~~~~~~~~~~~~ \bigg( (1-\logit_{5,j_2})\cdot \Big(\sum_{r\in\hat{\fA}_{j_2}}\ReLU^{\prime}(\Lambda_{5,j_2, r
    })\cdot \Big( V_{j_2,  r}(g_1)- \Lambda_{5,j_2, r}\pm\tilde{O}(\sigma_0) \Big)\pm \tilde{O}(\delta^{q}) \Big)\\   
    &~~~~~~~~~~-\sum_{j\neq j_2\in\tau(\cY)}\logit_{5,j} \cdot \Big(\sum_{r\in\hat{\fA}_{j}}\ReLU^{\prime}(\Lambda_{5,j,r
    })\cdot  \Big( V_{j, r}(g_1)- \Lambda_{5,j,r}\pm\tilde{O}(\sigma_0) \Big)\pm \tilde{O}(\delta^{q}) \Big) \\
    &~~~~~~~~~~~~\pm\sum_{j\notin\tau(\cY)}\logit_{5,j}\tilde{O}(\sigma^{q}_0)  \bigg)\1_{\tau(x_1)=s,\tau(x_0)=s'}\Bigg].
\end{align*}
\end{lemma}
\begin{lemma}[Refined expression for the gradient of $\Qb_{4,4}$] \label{lem-refined-grad-Q44-sym} Given $s\in\tau(\X)$, for the diagonal entry \( [\Q_{4,4}]_{s,s} \) of the block \(\Qb_{4,4}\), letting $j_2=\tau(g_2(y_1))$, we have 
  \begin{align*}
    &\Big[-\nabla_{\Q_{4,4}}\Loss^{2,2}_{5}\Big]_{s,s}= 
      \E\Bigg[
    \attn_{{\ans,1} \rightarrow \ans,1} \cdot \\
    &~~~~~~~~~~~~~ \bigg( (1-\logit_{5,j_2})\cdot \Big(\sum_{r\in\hat{\fA}_{j_2}}\ReLU^{\prime}(\Lambda_{5,j_2, r
    })\cdot \Big( V_{j_2,  r}(y_1)- \Lambda_{5,j_2, r}\pm\tilde{O}(\sigma_0) \Big)\pm \tilde{O}(\delta^{q}) \Big)\\   
    &~~~~~~~~~~-\sum_{j\neq j_2\in\tau(\cY)}\logit_{5,j} \cdot \Big(\sum_{r\in\hat{\fA}_{j}}\ReLU^{\prime}(\Lambda_{5,j,r
    })\cdot  \Big( V_{j, r}(y_1)- \Lambda_{5,j,r}\pm\tilde{O}(\sigma_0) \Big)\pm \tilde{O}(\delta^{q}) \Big) \\
    &~~~~~~~~~~~~\pm\sum_{j\notin\tau(\cY)}\logit_{5,j}\tilde{O}(\sigma^{q}_0)  \bigg)\1_{\tau(x_1)=s}\Bigg].
\end{align*}
Moreover, for the off-diagonal entries \( [\Q_{4,3}]_{s,s'} \) with \( s \neq s' \), we have 
  \begin{align*}
    &\Big[-\nabla_{\Q_{4,4}}\Loss^{2,2}_{5}\Big]_{s,s'}=  \E\Bigg[
    \attn_{{\ans,1} \rightarrow \ans,0} \cdot\\
    &~~~~~~~~~~~~~ \bigg( (1-\logit_{5,j_2})\cdot \Big(\sum_{r\in\hat{\fA}_{j_2}}\ReLU^{\prime}(\Lambda_{5,j_2, r
    })\cdot \Big( V_{j_2,  r}(y_0)- \Lambda_{5,j_2, r}\pm\tilde{O}(\sigma_0) \Big)\pm \tilde{O}(\delta^{q}) \Big)\\   
    &~~~~~~~~~~-\sum_{j\neq j_2\in\tau(\cY)}\logit_{5,j} \cdot \Big(\sum_{r\in\hat{\fA}_{j}}\ReLU^{\prime}(\Lambda_{5,j,r
    })\cdot  \Big( V_{j, r}(y_0)- \Lambda_{5,j,r}\pm\tilde{O}(\sigma_0) \Big)\pm \tilde{O}(\delta^{q}) \Big) \\
    &~~~~~~~~~~~~\pm\sum_{j\notin\tau(\cY)}\logit_{5,j}\tilde{O}(\sigma^{q}_0)  \bigg)\1_{\tau(x_1)=s,\tau(x_0)=s'}\Bigg].
\end{align*}
\end{lemma}

Following the above calculations, we can further  obtain the gradient summation of $\Qb_{4,3}$ and $\Qb_{4,4}$ as follows:
\begin{lemma}[Gradient sum of $\Qb_{4,3}$ and $\Qb_{4,4}$]
    \label{lem-grad-sum-sym} Given $s\in\tau(\X)$, letting $j_2=\tau(g_2(y_1))$, we have
    \begin{align*}
        &\Big[-\nabla_{\Qb_{4,3}}\Loss_5^{2,2}\Big]_{s,s}+ \Big[-\nabla_{\Qb_{4,3}}\Loss_5^{2,2}\Big]_{s,s}\\
          &=\E\Bigg[
     \bigg( (1-\logit_{5,j_2})\cdot \Big(\sum_{r\in\hat{\fA}_{j_2}}\ReLU^{\prime}(\Lambda_{5,j_2, r
         })\cdot \\
         &~~~~~~~~~~\Big( -\attn_{{\ans,1} \rightarrow \ans,0} \cdot V_{j_2,  r}(y_0)-\attn_{{\ans,1} \rightarrow \pred,1} \cdot V_{j_2,  r}(g_1)\\
         &~~~~~~~~~~~~~~~~~~+\big(1-\attn_{{\ans,1} \rightarrow \ans,1}-\attn_{{\ans,1} \rightarrow \pred,2}\big) \Lambda_{5,j_2, r}\pm\tilde{O}(\sigma_0) \Big)\pm \tilde{O}(\delta^{q}) \Big)\\   
         &~~~~~+\sum_{j\neq j_2\in\tau(\cY)}\logit_{5,j} \cdot \Big(\sum_{r\in\hat{\fA}_{j}}\ReLU^{\prime}(\Lambda_{5,j,r
         })\cdot  \\
        &~~~~~~~~~~\Big( \attn_{{\ans,1} \rightarrow \ans,0} \cdot V_{j,  r}(y_0)+\attn_{{\ans,1} \rightarrow \pred,1} \cdot V_{j,  r}(g_1)\\
         &~~~~~~~~~~~~~~~~~~-\big(1-\attn_{{\ans,1} \rightarrow \ans,1}-\attn_{{\ans,1} \rightarrow \pred,2}\big) \Lambda_{5,j, r}\pm\tilde{O}(\sigma_0) \Big)\pm \tilde{O}(\delta^{q}) \Big)\\      &~~~~~~~~~~~~\pm\sum_{j\notin\tau(\cY)}\logit_{5,j}\cdot \big(\attn_{{\ans,1} \rightarrow \ans,1}+\attn_{{\ans,1} \rightarrow \pred,2}\big)\tilde{O}(\sigma^{q}_0)  \bigg)\1_{\tau(x_1)=s}\Bigg].
     \end{align*}
\end{lemma}

\paragraph{Notations for gradient decompositions.} We shall define some useful notations to further simplify the expressions of gradient.
\begin{lemma}\label{lem-grad-decompositions-sym}
    For any \(s\in\tau(\X)\), we define the following notations for the gradient decompositions:
    \begin{enumerate}
        \item for $[\Q_{4,3}]_{s,s}$  we have $\big[-\nabla_{\Qb_{4,3}}\Loss_5^{2,2}\big]_{s,s}=\E\big[\cN_{s,3,2,\rom1}+\cN_{s,3,2,\rom2}+\cN_{s,3,2,\rom3}\big]$, where 
      \begin{align}
       \cN_{s,3,2,\rom1}&= 
    \attn_{{\ans,1} \rightarrow \pred,2} \cdot (1-\logit_{5,j_2})\cdot \label{eq-def-N-s-3-2-1-sym} \\
    &~~~~~~~~~  \Big(\sum_{r\in\hat{\fA}_{j_2}}\ReLU^{\prime}(\Lambda_{5,j_2, r
    })
    \cdot \Big( V_{j_2,  r}(g_2)- \Lambda_{5,j_2, r}\pm\tilde{O}(\sigma_0) \Big)\pm \tilde{O}(\delta^{q}) \Big) 
    \1_{\tau(x_1)=s},\notag
\end{align}
      \begin{align}\label{eq-def-N-s-3-2-2-sym}
       \cN_{s,3,2,\rom2}&= -
    \attn_{{\ans,1} \rightarrow \pred,2} \cdot\sum_{j\neq j_2\in\tau(\cY)}\logit_{5,j} \cdot \\
    &~~~~~~~~~~\Big(\sum_{r\in\hat{\fA}_{j}}\ReLU^{\prime}(\Lambda_{5,j,r
    })\cdot  \Big( V_{j, r}(g_2)- \Lambda_{5,j,r}\pm\tilde{O}(\sigma_0) \Big)\pm \tilde{O}(\delta^{q}) \Big) 
\1_{\tau(x_1)=s},\notag
\end{align}
      \begin{align}\label{eq-def-N-s-3-2-3-sym}
       \cN_{s,3,2,\rom3}&= \pm
    \attn_{{\ans,1} \rightarrow \pred,2} \cdot\sum_{j\notin\tau(\cY)}\logit_{5,j}\tilde{O}(\sigma^{q}_0)\1_{\tau(x_1)=s}.~~~~~~~~~~~~~~~~~~~~~~~~~~~~~~
\end{align}
    \item for $[\Q_{4,4}]_{s,s}$, we have $\big[-\nabla_{\Qb_{4,4}}\Loss_5^{2,2}\big]_{s,s}=\E\big[\cN_{s,4,2,\rom1}+\cN_{s,4,2,\rom2}+\cN_{s,4,2,\rom3}\big]$, where 
      \begin{align}\label{eq-def-N-s-4-2-1-sym}
       \cN_{s,4,2,\rom1}&=
    \attn_{{\ans,1} \rightarrow \ans,1} \cdot (1-\logit_{5,j_2})\cdot \\
    &~~~~~~~~~  \Big(\sum_{r\in\hat{\fA}_{j_2}}\ReLU^{\prime}(\Lambda_{5,j_2, r
    })
    \cdot \Big( V_{j_2,  r}(y_1)- \Lambda_{5,j_2, r}\pm\tilde{O}(\sigma_0) \Big)\pm \tilde{O}(\delta^{q}) \Big) 
    \1_{\tau(x_1)=s},\notag
\end{align}
      \begin{align}\label{eq-def-N-s-4-2-2-sym}
       \cN_{s,4,2,\rom2}&= -
    \attn_{{\ans,1} \rightarrow \ans,1} \cdot\sum_{j\neq j_2\in\tau(\cY)}\logit_{5,j} \cdot \\
    &~~~~~~~~~~\Big(\sum_{r\in\hat{\fA}_{j}}\ReLU^{\prime}(\Lambda_{5,j,r
    })\cdot  \Big( V_{j, r}(y_1)- \Lambda_{5,j,r}\pm\tilde{O}(\sigma_0) \Big)\pm \tilde{O}(\delta^{q}) \Big) 
\1_{\tau(x_1)=s},\notag
\end{align}
      \begin{align}\label{eq-def-N-s-4-2-3-sym}
       \cN_{s,4,2,\rom3}&= \pm
    \attn_{{\ans,1} \rightarrow \ans,1} \cdot\sum_{j\notin\tau(\cY)}\logit_{5,j}\tilde{O}(\sigma^{q}_0)\1_{\tau(x_1)=s}.~~~~~~~~~~~~~~~~~~~~~~~~~~~~~~
\end{align}
\item for the summation of $[\Q_{4,3}]_{s,s}$ and $[\Q_{4,4}]_{s,s}$, we have $\big[-\nabla_{\Qb_{4,3}}\Loss_5^{2,2}\big]_{s,s}+\big[-\nabla_{\Qb_{4,4}}\Loss_5^{2,2}\big]_{s,s}=\E\big[\cN_{s,2,\rom1}+\cN_{s,2,\rom2}+\cN_{s,2,\rom3}\big]$, where 
 \begin{align*}
    \cN_{s,2,\rom1}&=
 (1-\logit_{5,j_2})\cdot \Big(\sum_{r\in\hat{\fA}_{j_2}}\ReLU^{\prime}(\Lambda_{5,j_2, r
     })\cdot \\
     &~~~~~~~~~~\Big( -\attn_{{\ans,1} \rightarrow \ans,0} \cdot V_{j_2,  r}(y_0)-\attn_{{\ans,1} \rightarrow \pred,1} \cdot V_{j_2,  r}(g_1)\\
     &~~~~~~+\big(1-\attn_{{\ans,1} \rightarrow \ans,1}-\attn_{{\ans,1} \rightarrow \pred,2}\big) \Lambda_{5,j_2, r}\pm\tilde{O}(\sigma_0) \Big)\pm \tilde{O}(\delta^{q}) \Big)\1_{\tau(x_1)=s}\\   
     \cN_{s,2,\rom2}&=   \sum_{j\neq j_2\in\tau(\cY)}\logit_{5,j} \cdot \Big(\sum_{r\in\hat{\fA}_{j}}\ReLU^{\prime}(\Lambda_{5,j,r
     })\cdot  \\
    &~~~~~~~~~~\Big( \attn_{{\ans,1} \rightarrow \ans,0} \cdot V_{j,  r}(y_0)+\attn_{{\ans,1} \rightarrow \pred,1} \cdot V_{j,  r}(g_1)\\
     &~~~~~~~-\big(1-\attn_{{\ans,1} \rightarrow \ans,1}-\attn_{{\ans,1} \rightarrow \pred,2}\big) \Lambda_{5,j, r}\pm\tilde{O}(\sigma_0) \Big)\pm \tilde{O}(\delta^{q}) \Big)\1_{\tau(x_1)=s}\\      \cN_{s,2,\rom3}&=\pm\sum_{j\notin\tau(\cY)}\logit_{5,j}\cdot \big(\attn_{{\ans,1} \rightarrow \ans,1}+\attn_{{\ans,1} \rightarrow \pred,2}\big)\tilde{O}(\sigma^{q}_0)  \1_{\tau(x_1)=s}.
 \end{align*}
    \end{enumerate}
\end{lemma}

\paragraph{Probabilistic events.} We conclude this subsection by introducing several probabilistic events
that will be used to simplify the characterization of activated neurons in the subsequent analysis. We first define some events that may contribute non-trivially to the gradient of $\Qb_{4,3}$ and $\Qb_{4,4}$:
\begin{align}
    \cE_{1}&=\Big\{y_0\neq y_1, g_{1}(y_{\ell-1})\not=g_{2}(y_{\ell-1}), \text{ for all } \ell\in [2]\Big\} ,
\label{eq-event1-non}\\
    \cE_{2}&=\Big\{ g_1, g_2\in \fiber_{\tau(g_2(y_0)),y_0}, g_1(y_1)\neq g_2(y_1)\Big\}, \label{eq-event2-non}\\
    \cE_{3}&=\Big\{y_0=y_1, g_1(y_1)\neq g_2(y_1)\Big\}, \label{eq-event3-non}\\
    \cE_{4}&=\Big\{y_0\neq y_1, g_{1}(y_{\ell-1})=g_{2}(y_{\ell-1}), \text{ for all } \ell\in [2] \Big\} .
\label{eq-event4-non}
\end{align}

We first observe that event $\cE_1$ occurs with high probability $1 - O\big(\frac{1}{n_c}\big)$, and serves as the primary regime of interest.
In contrast, events $\cE_2$ and $\cE_3$ each occur with probability $\Theta(1/n_y)$ and correspond to instances of initial prediction ambiguity, where certain incorrect classes may exhibit disproportionately large logits. Similarly, $\cE_4$ corresponds to instances of initial prediction ambiguity, while occurs with probability $\Theta(1/n^2_y)$.
Furthermore, we define the following events, which yield negligible gradient contributions to $\Qb_{4,3}$ and $\Qb_{4,4}$, as they do not lead to significant confusion among the incorrect predictions:
\begin{align}
    \cE_{5}&=\Big\{y_0\neq y_1, g_{1}(y_{1})=g_{2}(y_{1}), g_{1}(y_{0})\neq g_{2}(y_{0})\Big\} ,
\label{eq-event5-non}\\
    \cE_{6}&=\Big\{y_0=y_1, g_1(y_1)= g_2(y_1)\Big\}. \label{eq-event6-non}
\end{align}
Here, $\cE_5$ occurs with probability $\Theta(1/{n_y})$, and $\cE_6$ occurs with probability $\Theta(1/{n^2_c})$. Together, the events $\cup_{i\in[6]}\cE_i$ form a partition of the entire sample space.

\subsection{Stage 1.2.1: Initial Growth of $\mathbf{Q}$}
We define the following notations:
\begin{align*}
    \epsilon^{L,\ell}_{\mathsf{attn}}\big(\Zb^{L,\ell-1}\big) &= 1 - \attn_{\ans,\ell-1 \to \pred,\ell}\big(\Zb^{L,\ell-1}\big) - \attn_{\ans,\ell-1 \to \ans,\ell-1}\big(\Zb^{L,\ell-1}\big),\\
\Delta^{L,\ell}\big(\Zb^{L,\ell-1}\big)& =  \Big|\attn_{\ans,\ell-1 \to \pred,\ell}\big(\Zb^{L,\ell-1}\big) - \attn_{\ans,\ell-1 \to \ans,\ell-1}\big(\Zb^{L,\ell-1}\big)\Big|. 
\end{align*}

We abbreviate $\epsilon^{L,\ell}_{\mathsf{attn}}\big(\Zb^{L,\ell-1}\big)$ and $ \Delta^{L,\ell}\big(\Zb^{L,\ell-1}\big)$ as $\epsilon^{L,\ell}_{\mathsf{attn}}$ and $\Delta^{L, \ell}$ for simplicity. Since we only focus on the input $\Zb^{2,1}$ in the following analysis, we will omit the notation related to the length, i.e., we use $\epsilon_{\mathsf{attn}}$ and $\Delta$ when the context is clear.

\begin{induction}\label{induction-s21-non}
  Given $s\in\tau(\X)$,  let $T_{1,2,1,s}$ denote the first time that $\E[\ate\mid \tau(x_1)=s]\leq 0.4$
     For all iterations $t\leq T_{1, 2,1,s}$, we have the following holds
     \begin{enumerate}[(a)]
        \item $\big[\Qb^{(t)}_{4,3}\big]_{s,s}+\big[\Qb^{(t)}_{4,4}\big]_{s,s}\leq O(1)$ monotonically increases; 
        \item for $p\in\{3,4\}$, for $s'\in\tau(\X)\not=s$, $\Big|\big[\Qb^{(t)}_{4,p}\big]_{s,s'}\Big|\leq O\bigg(\frac{\big[\Qb^{(t)}_{4,p}\big]_{s,s'}}{d}\bigg)$ ; otherwise $\big[\Qb^{(t)}_{4,p}\big]_{s,s'}=0$;
        \item for any sample $\Zb^{2,1}$, we have $\attn^{(t)}_{\ans,1\to \pred,2}-\attn^{(t)}_{\ans,1\to \pred,1}\geq -O\big(\frac{\log\log d}{\log d}\big)$; 
        \item for any sample $\Zb^{2,1}$, we have $\attn^{(t)}_{\ans,1\to \pred,2}-\attn^{(t)}_{\ans,1\to \ans,1}\leq c_1$ for some small constant $c_1>0$.
     \end{enumerate}
  \end{induction}
  \subsubsection{Attention and Lambda Preliminaries}
\begin{lemma}\label{lem-s21-attn-non}
    If \Cref{induction-s21-non} holds for all iterations $<t$,then we have   
         \begin{enumerate}
            \item $\attn^{(t)}_{\ans,1\to \pred, 1}+\attn^{(t)}_{\ans,1\to \ans, 1}\in [0.4\pm \tilde{O}\big(\frac{1}{d}\big), 0.5]$; 
   \item  $\big|\attn^{(t)}_{\ans,1\to \pred,1}- \attn^{(t)}_{\ans,1\to \ans,0}\big|\leq \tilde{O}\big(\frac{1}{d}\big)$. 
            \end{enumerate} 

\end{lemma}

\begin{lemma}\label{lem-s21-lambda-1-sym}
    If \Cref{induction-s21-non} holds for all iterations $<t$, then given $\Zb^{2,1}\in \cE_{1}$, 
\begin{enumerate}
    \item for the prediction $j_2$, 
    we have 
    \begin{align*}
       & \Lambda^{(t)}_{5,j_2,r_{g_2\cdot y_1}}=\\ & \Big(\attn^{(t)}_{\ans,1\to \pred,2}-\attn^{(t)}_{\ans,1\to \ans,0}\Big) \cdot 2B +\Big(\attn^{(t)}_{\ans,1\to \ans,1}-\attn^{(t)}_{\ans,1\to \pred,1}\Big) \cdot \frac{2B}{n_y}+ O(\delta).
    \end{align*}
    \item for the prediction $j'_2=\tau\big(g_2(y_0)\big)$, 
    we have 
\begin{align*}
    &\Lambda^{(t)}_{5,j'_2,r_{g_2\cdot y_0}}
     = \Big(\attn^{(t)}_{\ans,1\to \pred,2}-\attn^{(t)}_{\ans,1\to \ans,1}\Big) \cdot 2B +\tilde{O}\Big(\frac{B}{d\cdot n_y}\Big)+ O(\delta).
\end{align*}
\item for the prediction $\tau(g_1(y_0))$, 
we have
\begin{align*}
   & \Lambda^{(t)}_{5,\tau(g_1(y_0)),r_{g_1\cdot y_0}} =\\ &\Big(\attn^{(t)}_{\ans,1\to \pred,1}-\attn^{(t)}_{\ans,1\to \ans,1}\Big) \cdot 2B + \Big(\attn^{(t)}_{\ans,1\to \ans,0}-\attn^{(t)}_{\ans,1\to \pred,2}\Big) \cdot \frac{2B}{n_y}+ O(\delta).
\end{align*}
furthermore, we have $\Lambda^{(t)}_{5,\tau(g_1(y_0)),r_{g_1\cdot y_0}}\leq \Lambda^{(t)}_{5,\tau(g_2(y_1)),r_{g_2\cdot y_1}}$.
\item for the prediction $\tau(g_1(y_1))$, 
we have
\begin{align*}
    \Lambda^{(t)}_{5,\tau(g_1(y_1)),r_{g_1\cdot y_1}} =  \Big(\attn^{(t)}_{\ans,1\to \ans,1}-\attn^{(t)}_{\ans,1\to \pred,2}\Big) \cdot \frac{2B}{n_y}+\tilde{O}\Big(\frac{B}{d}\Big)+O(\delta).
\end{align*}

\item for other $j\in\tau(\cY)$, if there are $j$  and $y$, s.t.,  $g_1, g_2\in \fiber_{j,y}$ (notice that $y\neq y_0, y_1$ for $\cE_{1}$),  then for such a $j$, $r\in\hat{\fA}_{j}\setminus \{r_{g_2\cdot y}\}$ cannot be activated, moreover, we have 
\begin{align*}
    &\Lambda^{(t)}_{5,j,r_{g_2\cdot y}}
     = \Big(\attn^{(t)}_{\ans,1\to \pred,2}-\attn^{(t)}_{\ans,1\to \ans,1}\Big) \cdot 2B +\tilde{O}\Big(\frac{B}{d\cdot n_y}\Big)+ O(\delta);
\end{align*}
else,  none of $r\in\hat{\fA}_{j}$ can be activated.
\end{enumerate}
Notice that for $\cE_{1}$, the neurons mentioned in 1-4 should be different neurons, while the predictions in 1 and 3, or 2 and 4 can be the same in some cases, e.g., $g_1(y_1)=g_2(y_0)$. In all cases, except for the neurons mentioned above, i.e., $\cup_{\ell,\ell'\in [2]} \{r_{g_{\ell}\cdot y_{\ell'-1}}\}$ (which may not be activated), all other neurons $r\in  \cup_{\ell,\ell'\in [2]} \Big(\hat{\fA}_{\tau(g_{\ell}\cdot y_{\ell'-1})}\setminus  \{r_{g_{\ell}\cdot y_{\ell'-1}}\}\Big)$ cannot be activated.
\end{lemma}

\begin{lemma}\label{lem-s21-lambda-2-sym}
    If \Cref{induction-s21-non} holds for all iterations $<t$, then given $\Zb^{2,1}\in \cE_{2}$, we have
\begin{enumerate}
    \item for the prediction $j_2$, $r\in\hat{\fA}_{j_2}\setminus \{r_{g_2\cdot y_1}\}$  cannot be activated, moreover, we have 
    \begin{align*}
        & \Lambda^{(t)}_{5,j_2,r_{g_2\cdot y_1}}=\Big(\attn^{(t)}_{\ans,1\to \pred,2}-\attn^{(t)}_{\ans,1\to \ans,0}\Big) \cdot 2B \\
        &~~~~~~~~+\Big(\attn^{(t)}_{\ans,1\to \ans,1}-\attn^{(t)}_{\ans,1\to \pred,1}\Big) \cdot \frac{2B}{n_y}+ O(\delta).
     \end{align*}
    \item for the prediction $j'_2=\tau\big(g_2(y_0)\big)$, $r\in\hat{\fA}_{j'_2}\setminus \{r_{g_2\cdot y_0}\}$ (notice that in this case $r_{g_2\cdot y_0}=r_{g_1\cdot y_0}$) cannot be activated, moreover, we have 
\begin{align*}
    \Lambda^{(t)}_{5,j'_2,r_{g_2\cdot y_0}} = \Big(\attn^{(t)}_{\ans,1\to \pred,2}+\attn^{(t)}_{\ans,1\to \pred,1}-\attn^{(t)}_{\ans,1\to \ans,1}\Big) \cdot 2B + O(\delta).
\end{align*}

\item for the prediction $\tau(g_1(y_1))$, $r\in\hat{\fA}_{\tau(g_1(y_1))}\setminus \{r_{g_1\cdot y_1}\}$ cannot be activated, moreover, we have
\begin{align*}
    \Lambda^{(t)}_{5,\tau(g_1(y_1)),r_{g_1\cdot y_1}} =  \Big(\attn^{(t)}_{\ans,1\to \ans,1}-\attn^{(t)}_{\ans,1\to \pred,2}\Big) \cdot \frac{2B}{n_y}+O(\delta).
\end{align*}
\item for other $j\in\tau(\cY)$, if there are  $j$  and $y$, s.t.,  $g_1, g_2\in \fiber_{j,y}$ (notice that $y\neq y_0, y_1$ for $\cE_{2}$),  then for such a $j$, $r\in\hat{\fA}_{j}\setminus \{r_{g_2\cdot y}\}$ cannot be activated, moreover, we have 
\begin{align*}
    &\Lambda^{(t)}_{5,j,r_{g_2\cdot y}}
     = \Big(\attn^{(t)}_{\ans,1\to \pred,2}-\attn^{(t)}_{\ans,1\to \ans,1}\Big) \cdot 2B +\tilde{O}\Big(\frac{B}{d\cdot n_y}\Big)+ O(\delta);
\end{align*}
else,  none of $r\in\hat{\fA}_{j}$ can be activated.
\end{enumerate}

\end{lemma}

\begin{lemma}\label{lem-s21-lambda-3-sym}
    If \Cref{induction-s21-non} holds for all iterations $<t$, then given $\Zb^{2,1}\in \cE_{3}$, we have
\begin{enumerate}
    \item for the prediction $j_2$, $r\in\hat{\fA}_{j_2}\setminus \{r_{g_2\cdot y_1}\}$  cannot be activated, moreover, we have 
    \begin{align*}
        \Lambda^{(t)}_{5,j_2,r_{g_2\cdot y_1}} = \attn^{(t)}_{\ans,1\to \pred,2}\cdot 2B \pm  O(\delta).
    \end{align*}

\item for the prediction $\tau(g_1(y_1))$, $r\in\hat{\fA}_{\tau(g_1(y_1))}\setminus \{r_{g_1\cdot y_1}\}$ cannot be activated, moreover, we have
\begin{align*}
    \Lambda^{(t)}_{5,\tau(g_1(y_1)),r_{g_1\cdot y_1}} = \attn^{(t)}_{\ans,1\to \pred,1}\cdot 2B \pm  O(\delta).
\end{align*}
\item for other  $j=g(y_1)$, where $g_1,g_2\notin\fiber_{j,y_1}$, we have 
\begin{align*}
    &\Lambda^{(t)}_{5,j,r_{g\cdot y_1}}
     = \Big(\attn^{(t)}_{\ans,1\to \ans,1}-\attn^{(t)}_{\ans,1\to \pred,2}\Big) \cdot \frac{2B}{n_y} +\tilde{O}\Big(\frac{B}{d\cdot n_y}\Big)+ O(\delta);
\end{align*}
moreover, if there exists  $y$, s.t.,  $g_1, g_2\in \fiber_{j,y}$ (notice that $y\neq y_1$ for $\cE_{3}$),  we have 
\begin{align*}
    &\Lambda^{(t)}_{5,j,r_{g_2\cdot y}}
     = \Big(\attn^{(t)}_{\ans,1\to \pred,2}-\attn^{(t)}_{\ans,1\to \ans,1}\Big) \cdot 2B +\tilde{O}\Big(\frac{B}{d\cdot n_y}\Big)+ O(\delta);
\end{align*}
besides,  other $r\in\hat{\fA}_{j}$ cannot be activated.
\end{enumerate}
\end{lemma}

\begin{lemma}\label{lem-s21-lambda-4-sym}
    If \Cref{induction-s21-non} holds for all iterations $<t$, then given $\Zb^{2,1}\in \cE_{4}$, we have
\begin{enumerate}
    \item for the prediction $j_2$, $r\in\hat{\fA}_{j_2}\setminus \{r_{g_2\cdot y_1}\}$  cannot be activated, moreover, we have 
    \begin{align*}
        \Lambda^{(t)}_{5,j_2,r_{g_2\cdot y_1}} = \attn^{(t)}_{\ans,1\to \pred,2} \cdot 2B + O(\delta). 
    \end{align*}

\item for the prediction $\tau(g_2(y_0))$, $r\in\hat{\fA}_{\tau(g_2(y_0))}\setminus \{r_{g_2\cdot y_0}\}$ cannot be activated, moreover, we have
\begin{align*}
    \Lambda^{(t)}_{5,j'_2,r_{g_2\cdot y_0}} = \Big(\attn^{(t)}_{\ans,1\to \pred,2}+\attn^{(t)}_{\ans,1\to \pred,1}-\attn^{(t)}_{\ans,1\to \ans,1}\Big) \cdot 2B + O(\delta).
\end{align*}

\item for other $j\in\tau(\cY)$, if there exist  $j$  and $y$, s.t.,  $g_1, g_2\in \fiber_{j,y}$,  then for such a $j$, $r\in\hat{\fA}_{j}\setminus \{r_{g_2\cdot y}\}$ cannot be activated, moreover, we have 
\begin{align*}
    &\Lambda^{(t)}_{5,j,r_{g_2\cdot y}}
     = \Big(\attn^{(t)}_{\ans,1\to \pred,2}-\attn^{(t)}_{\ans,1\to \ans,1}\Big) \cdot 2B +\tilde{O}\Big(\frac{B}{d\cdot n_y}\Big)+ O(\delta);
\end{align*}
else,  none of $r\in\hat{\fA}_{j}$ can be activated.
\end{enumerate}
\end{lemma}

\begin{lemma}\label{lem-s21-lambda-5-sym}
    If \Cref{induction-s21-non} holds for all iterations $<t$, then given $\Zb^{2,1}\in \cE_{5}$, we have
\begin{enumerate}
    \item for the prediction $j_2$, $r\in\hat{\fA}_{j_2}\setminus \{r_{g_2\cdot y_1}\}$  cannot be activated, moreover, we have 
    \begin{align*}
        \Lambda^{(t)}_{5,j_2,r_{g_2\cdot y_1}} = \attn^{(t)}_{\ans,1\to \pred,2}\cdot 2B \pm  O(\delta).
    \end{align*}
\item for the prediction $\tau(g_2(y_0))$, $r\in\hat{\fA}_{\tau(g_2(y_0))}\setminus \{r_{g_2\cdot y_0}\}$ cannot be activated, moreover, we have
\begin{align*}
    \Lambda^{(t)}_{5,j'_2,r_{g_2\cdot y_0}} = \Big(\attn^{(t)}_{\ans,1\to \pred,2}-\attn^{(t)}_{\ans,1\to \ans,1}\Big) \cdot 2B + \tilde{O}\Big(\frac{B}{d\cdot n_y}\Big)+ O(\delta).
\end{align*}
\item for the prediction $\tau(g_1(y_0))$, $r\in\hat{\fA}_{\tau(g_2(y_0))}\setminus \{r_{g_1\cdot y_0}\}$ cannot be activated, moreover, we have
\begin{align*}
   & \Lambda^{(t)}_{5,\tau(g_1(y_0)),r_{g_1\cdot y_0}} =\Big(\attn^{(t)}_{\ans,1\to \pred,1}-\attn^{(t)}_{\ans,1\to \ans,1}\Big) \cdot 2B \\ &+ \Big(\attn^{(t)}_{\ans,1\to \ans,0}-\attn^{(t)}_{\ans,1\to \pred,2}\Big) \cdot \frac{2B}{n_y}+ O(\delta).
\end{align*}
\item for other $j\in\tau(\cY)$, if there are  $j$  and $y$, s.t.,  $g_1, g_2\in \fiber_{j,y}$,  then for such a $j$, $r\in\hat{\fA}_{j}\setminus \{r_{g_2\cdot y}\}$ cannot be activated, moreover, we have 
\begin{align*}
    &\Lambda^{(t)}_{5,j,r_{g_2\cdot y}}
     = \Big(\attn^{(t)}_{\ans,1\to \pred,2}-\attn^{(t)}_{\ans,1\to \ans,1}\Big) \cdot 2B +\tilde{O}\Big(\frac{B}{d\cdot n_y}\Big)+ O(\delta);
\end{align*}
else,  none of $r\in\hat{\fA}_{j}$ can be activated.
\end{enumerate}
\end{lemma}

\begin{lemma}\label{lem-s21-lambda-6-sym}
    If \Cref{induction-s21-non} holds for all iterations $<t$, then given $\Zb^{2,1}\in \cE_{6}$, we have
\begin{enumerate}
    \item for the prediction $j_2$, $r\in\hat{\fA}_{j_2}\setminus \{r_{g_2\cdot y_1}\}$  cannot be activated, moreover, we have 
    \begin{align*}
        \Lambda^{(t)}_{5,j_2,r_{g_2\cdot y_1}} = \Big(\attn^{(t)}_{\ans,1\to \pred,2}+\attn^{(t)}_{\ans,1\to \pred,1}\Big)\cdot 2B \pm  O(\delta).
    \end{align*}

\item  for other  $j=g(y_1)$, where $g_2\notin\fiber_{j,y_1}$, we have 
\begin{align*}
    &\Lambda^{(t)}_{5,j,r_{g\cdot y_1}}
     = \Big(\attn^{(t)}_{\ans,1\to \ans,1}-\attn^{(t)}_{\ans,1\to \pred,2}\Big) \cdot \frac{2B}{n_y} +\tilde{O}\Big(\frac{B}{d\cdot n_y}\Big)+ O(\delta);
\end{align*}
moreover, if there exists  $y\neq y_1$, s.t.,  $g_1, g_2\in \fiber_{j,y}$,  we have 
\begin{align*}
    &\Lambda^{(t)}_{5,j,r_{g_2\cdot y}}
     = \Big(\attn^{(t)}_{\ans,1\to \pred,2}-\attn^{(t)}_{\ans,1\to \ans,1}\Big) \cdot 2B +\tilde{O}\Big(\frac{B}{d\cdot n_y}\Big)+ O(\delta);
\end{align*}
besides,  other $r\in\hat{\fA}_{j}$ cannot be activated.
\end{enumerate}

\end{lemma}

\subsubsection{Gradient Lemma}
\begin{lemma}\label{lem-s21-gd1-non}
    If \Cref{induction-s21-non} holds for all iterations $<t$, given $s\in\tau(\X)$, we have
    \begin{align*}
       & \Big[-\nabla_{\Q^{(t)}_{4,3}}{\Loss^{2,2}_{5}}\Big]_{s,s}+ \Big[-\nabla_{\Q^{(t)}_{4,4}}{\Loss^{2,2}_{5}}\Big]_{s,s} 
        \\ &~~~~~\geq 
        \min\bigg\{ \Omega\Big(\frac{1}{d\cdot n_y}\Big), \Omega(\frac{B}{d})\cdot \E\bigg[(1-\logit^{(t)}_{5,j_2})\big| \tau(x_1)=s, \cE_1\bigg]\bigg\}>0.
    \end{align*}
\end{lemma}
\begin{proof}  
    By  \Cref{lem-grad-sum-sym} and \Cref{lem-grad-decompositions-sym}, we have  $\big[-\nabla_{\Qb_{4,3}}\Loss_5^{2,2}\big]_{s,s}+\big[-\nabla_{\Qb_{4,4}}\Loss_5^{2,2}\big]_{s,s}=\E\big[\cN_{s,2,\rom1}+\cN_{s,2,\rom2}+\cN_{s,2,\rom3}\big]$. Based on\Cref{lem-s21-lambda-1-sym}-\Cref{lem-s21-lambda-6-sym}, we can first directly bound the term $\cN^{(t)}_{s,2,\rom3}$ as follows:
\begin{align*}
    \E\bigg[ \cN^{(t)}_{s,2,\rom3}\bigg] \leq \tilde{O}(\sigma_0^q)=\frac{1}{\poly d}.
\end{align*}
In the following discussion, we focus on  $\cN^{(t)}_{s,2,\rom1}$ and $\cN^{(t)}_{s,2,\rom2}$, and consider two regimes for the gradient lower bound: (i) $\E\big[\attn^{(t)}_{{\ans,1} \rightarrow \pred,2}-\attn^{(t)}_{{\ans,1} \rightarrow \ans,0}\mid \tau(x_1)=s\big]\leq \frac{\varrho}{B}$, and  (ii) $\E\big[\attn^{(t)}_{{\ans,1} \rightarrow \pred,2}-\attn^{(t)}_{{\ans,1} \rightarrow \ans,0}\mid \tau(x_1)=s\big]\geq \frac{\varrho}{B}$. The proof strategy is to analyze $\cN^{(t)}_{s,2,\rom1}$ and $\cN^{(t)}_{s,2,\rom2}$ under different event conditions in two regimes.

\begin{enumerate}
    \item For regime (i), 
  \begin{enumerate}
        \item For $\Zb^{2,1}\in \cE_{1}$, by \Cref{induction-s21-non} and \Cref{lem-s21-lambda-1-sym}, we have an immediate logit upper bound for $j\in\tau(\cY)$: $\logit^{(t)}_{5,j}\leq O\big(\frac{1}{d}\big)$. Then  by \Cref{lem-grad-sum-sym}, we have
        \begin{align*}
           &\E\bigg[ \cN^{(t)}_{s,2,\rom1} \1_{\cE_1}\bigg]\\
           &=
           \E\bigg[  (1-\logit^{(t)}_{5,j_2})\cdot \Big(\sum_{r\in\hat{\fA}_{j_2}}\ReLU^{\prime}(\Lambda^{(t)}_{5,j_2, r
             })\cdot \\
             &~~~~~\Big( -\attn^{(t)}_{{\ans,1} \rightarrow \ans,0} \cdot V_{j_2,  r}(y_0)-\attn^{(t)}_{{\ans,1} \rightarrow \pred,1} \cdot V_{j_2,  r}(g_1)\\
             &~~~~~+\big(1-\attn^{(t)}_{{\ans,1} \rightarrow \ans,1}-\attn^{(t)}_{{\ans,1} \rightarrow \pred,2}\big) \Lambda^{(t)}_{5,j_2, r}\pm\tilde{O}(\sigma_0) \Big)\pm \tilde{O}(\delta^{q}) \Big)\1_{\tau(x_1)=s}\1_{\cE_1}\bigg].
        \end{align*}
       By \Cref{lem-s21-lambda-1-sym}, we can obtain 
        \begin{align}
            &\E\bigg[  (1-\logit^{(t)}_{5,j_2})\cdot \Big(\ReLU^{\prime}(\Lambda^{(t)}_{5,j_2, r_{g_2\cdot y_1}
               })\cdot \notag \\
               &~~~\Big( -\attn^{(t)}_{{\ans,1} \rightarrow \ans,0} \cdot V_{j_2,  r_{g_2\cdot y_1}}(y_0)-\attn^{(t)}_{{\ans,1} \rightarrow \pred,1} \cdot V_{j_2,  r_{g_2\cdot y_1}}(g_1)\notag\\
               &~~~+\big(1-\attn^{(t)}_{{\ans,1} \rightarrow \ans,1}-\attn^{(t)}_{{\ans,1} \rightarrow \pred,2}\big) \Lambda^{(t)}_{5,j_2, r_{g_2\cdot y_1}}\pm\tilde{O}(\sigma_0) \Big)\pm \tilde{O}(\delta^{q}) \Big)\1_{\tau(x_1)=s}\1_{\cE_1}\bigg] \notag\\ 
             &  \geq   \Omega\Big(\frac{B}{d}\Big)\cdot  \E\bigg[\ReLU^{\prime}(\Lambda^{(t)}_{5,j_2, r_{g_2\cdot y_1}
             })\big| \tau(x_1)=s, \cE_1\bigg]. \label{eq-s21-gd1-non-1}
        \end{align}
        On the other hand, 
        \begin{align*}
          & - \E\bigg[ (1-\logit^{(t)}_{5,j_2})\cdot \ReLU^{\prime}(\Lambda^{(t)}_{5,j_2, r_{g_1\cdot y_0}
              })\cdot \\
              &~~~~~~~~~~~~\Big( \attn^{(t)}_{{\ans,1} \rightarrow \ans,0} \cdot V_{j_2,  r_{g_1\cdot y_0}}(y_0)+\attn^{(t)}_{{\ans,1} \rightarrow \pred,1} \cdot V_{j_2,  r_{g_1\cdot y_0}}(g_1)\Big)\\
              &~~~~~~~~~~~~~~~\1_{\tau(x_1)=s}\1_{\cE_1}\1_{g_1(y_0)=g_2(y_1)}\bigg]\\
              &\geq -O\big(\frac{B}{d\cdot n_y}\big)\cdot \E\bigg[\ReLU^{\prime}(\Lambda^{(t)}_{5,j_2, r_{g_2\cdot y_1}
              })\big| \tau(x_1)=s, \cE_1\bigg],
        \end{align*}
        which implies that $\E\bigg[ \cN^{(t)}_{s,2,\rom1} \1_{\cE_1}\bigg]\geq  \Omega\Big(\frac{B}{d}\Big)\cdot  \E\bigg[\ReLU^{\prime}(\Lambda^{(t)}_{5,j_2, r_{g_2\cdot y_1}
             })\big| \tau(x_1)=s, \cE_1\bigg]\geq 0$.
             Moving to $\cN^{(t)}_{s,2,\rom2}$, by \Cref{lem-grad-sum-sym} and \Cref{lem-s21-lambda-1-sym}, we have
        \begin{align*}
           & \E\bigg[ \cN^{(t)}_{s,2,\rom2} \1_{\cE_1}\bigg]\\
           &= \E\bigg[  \sum_{j\neq j_2\in\tau(\cY)}\logit^{(t)}_{5,j} \cdot \Big(\sum_{r\in\hat{\fA}_{j}}\ReLU^{\prime}(\Lambda_{5,j,r
             })\cdot  \\
            &~~~~~~~~~~\Big( \attn^{(t)}_{{\ans,1} \rightarrow \ans,0} \cdot V_{j,  r}(y_0)+\attn^{(t)}_{{\ans,1} \rightarrow \pred,1} \cdot V_{j,  r}(g_1)\\
             &~~~~~~~-\big(1-\attn^{(t)}_{{\ans,1} \rightarrow \ans,1}-\attn^{(t)}_{{\ans,1} \rightarrow \pred,2}\big) \Lambda^{(t)}_{5,j, r}\pm\tilde{O}(\sigma_0) \Big)\pm \tilde{O}(\delta^{q}) \Big)\1_{\tau(x_1)=s}\1_{\cE_1}\bigg]\\
             &\geq -{O}\Big(\frac{n_y\cdot B}{d^2}\Big).
         \end{align*}
         \item For $\Zb^{2,1}\in \cE_{2}$, by \Cref{induction-s21-non} and \Cref{lem-s21-lambda-2-sym}, we can also obtain a crude bound of logit: $\logit^{(t)}_{5,j}\leq O\big(\frac{1}{d}\big)$ for $j\neq j'_2$. Then 
          $\cN^{(t)}_{s,2,\rom1}$ can be bounded in the same way as \eqref{eq-s21-gd1-non-1}, and we have 
        \begin{align*}
            \E\bigg[ \cN^{(t)}_{s,2,\rom1} \1_{\cE_2}\bigg]\geq  \Omega\Big(\frac{B}{d\cdot n_y}\Big)\cdot  \E\bigg[\ReLU^{\prime}(\Lambda^{(t)}_{5,j_2, r_{g_2\cdot y_1}
             })\big| \tau(x_1)=s, \cE_2\bigg]\geq 0.
        \end{align*}
        Moving to $\cN^{(t)}_{s,2,\rom2}$,
        \begin{itemize}
            \item for $j=\tau(g_1(y_1))$, we have 
            \begin{align*}
                &= \E\bigg[  \logit_{5,j}^{(t)}\Big(\ReLU^{\prime}(\Lambda_{5,j,r_{g_1\cdot y_1}
                  })\cdot  \\
                 &~~~\Big( \attn^{(t)}_{{\ans,1} \rightarrow \ans,0} \cdot V_{j,  r_{g_1\cdot y_1}}(y_0)+\attn^{(t)}_{{\ans,1} \rightarrow \pred,1} \cdot V_{j,  r_{g_1\cdot y_1}}(g_1)\\
                  &~~-\big(1-\attn^{(t)}_{{\ans,1} \rightarrow \ans,1}-\attn^{(t)}_{{\ans,1} \rightarrow \pred,2}\big) \Lambda^{(t)}_{5,j, r}\pm\tilde{O}(\sigma_0) \Big)\pm \tilde{O}(\delta^{q}) \Big)\1_{\tau(x_1)=s}\1_{\cE_2}\bigg]\\
                  &\geq -{O}\Big(\frac{B}{n^2_y\cdot d^2}\Big).
              \end{align*}
              where the last inequality is due to the cancellation of the term $\attn^{(t)}_{{\ans,1} \rightarrow \ans,0} \cdot V_{j,  r_{g_1\cdot y_1}}(y_0)+\attn^{(t)}_{{\ans,1} \rightarrow \pred,1} \cdot V_{j,  r_{g_1\cdot y_1}}(g_1)$ and the fact that
            \begin{align*}
                \Lambda^{(t)}_{5,\tau(g_1(y_1)),r_{g_1\cdot y_1}} =  \Big(\attn^{(t)}_{\ans,1\to \ans,1}-\attn^{(t)}_{\ans,1\to \pred,2}\Big) \cdot \frac{2B}{n_y}+O(\delta)\leq O\big(\frac{B}{n_y}\big).
            \end{align*}
            \item for $j=\tau(g_2(y))$ if there exists $y\not=y_0, y_1$ s.t., $g_1(y)=g_2(y)$
         \begin{itemize}
            \item    when $\E\big[\attn^{(t)}_{{\ans,1} \rightarrow \ans,1}-\attn^{(t)}_{{\ans,1} \rightarrow \pred,2}\mid \tau(x_1)=s\big]\leq 0.01$, by \Cref{induction-s21-non} and \Cref{lem-s21-lambda-2-sym}, we have $\logit^{(t)}_{5,\tau(g_2(y))}\ll O\big(\frac{1}{d}\big)\cdot \logit^{(t)}_{5,j_2'}$ and $\logit^{(t)}_{5,j_2'}=\Omega(1)$.  Then 
            \begin{align*}
                &\bigg| \E\bigg[  \logit^{(t)}_{5,j} \cdot \Big(\ReLU^{\prime}(\Lambda^{(t)}_{5,j,r_{g_1\cdot y}
                  })\cdot  \\
                 &~~~\Big( \attn^{(t)}_{{\ans,1} \rightarrow \ans,0} \cdot V_{j,  r_{g_1\cdot y}}(y_0)+\attn^{(t)}_{{\ans,1} \rightarrow \pred,1} \cdot V_{j,  r_{g_1\cdot y}}(g_1)\\
                  &-\big(1-\attn^{(t)}_{{\ans,1} \rightarrow \ans,1}-\attn^{(t)}_{{\ans,1} \rightarrow \pred,2}\big) \Lambda^{(t)}_{5,j, r_{g_1\cdot y}}\pm\tilde{O}(\sigma_0) \Big)\pm \tilde{O}(\delta^{q}) \Big)\1_{\tau(x_1)=s}\1_{\cE_2}\bigg]\bigg|\\
                  &\leq  {O}\Big(\frac{B}{n_y\cdot d^2}\Big).
              \end{align*}
              On the other hand,
              \begin{align*}
                & \E\bigg[  \logit^{(t)}_{5,j'_2} \cdot \Big(\ReLU^{\prime}(\Lambda^{(t)}_{5,j'_2,r_{g_1\cdot y_0}
                  })\cdot  \\
                 &~~~\Big( \attn^{(t)}_{{\ans,1} \rightarrow \ans,0} \cdot V_{j'_2,  r_{g_1\cdot y_0}}(y_0)+\attn^{(t)}_{{\ans,1} \rightarrow \pred,1} \cdot V_{j'_2,  r_{g_1\cdot y_0}}(g_1)\\
                  &-\big(1-\attn^{(t)}_{{\ans,1} \rightarrow \ans,1}-\attn^{(t)}_{{\ans,1} \rightarrow \pred,2}\big) \Lambda^{(t)}_{5,j'_2, r_{g_1\cdot y_0}}\pm\tilde{O}(\sigma_0) \Big)\\
                  &~~~~~~~~~~\pm \tilde{O}(\delta^{q}) \Big)\1_{\tau(x_1)=s}\1_{\cE_1}\bigg]\\
                  &\geq  {\Omega}\Big(\frac{B}{n_y\cdot d}\Big),
              \end{align*}
              where the last inequality follows from the fact that 
              \begin{align*}
                &\attn^{(t)}_{{\ans,1} \rightarrow \ans,0} \cdot V_{j'_2,  r_{g_1\cdot y_0}}(y_0)+\attn^{(t)}_{{\ans,1} \rightarrow \pred,1} \cdot V_{j'_2,  r_{g_1\cdot y_0}}(g_1)\\
                  &-\big(1-\attn^{(t)}_{{\ans,1} \rightarrow \ans,1}-\attn^{(t)}_{{\ans,1} \rightarrow \pred,2}\big) \Lambda^{(t)}_{5,j'_2, r_{g_1\cdot y_0}} \\
                  &= \attn^{(t)}_{{\ans,1} \rightarrow \pred,1} \cdot \Big(2B-2\Lambda^{(t)}_{5,j'_2, r_{g_1\cdot y_0}} \Big) \pm O(\delta)\geq \Omega(B)
.              \end{align*}
              \item when $\E\big[\attn^{(t)}_{{\ans,1} \rightarrow \ans,1}-\attn^{(t)}_{{\ans,1} \rightarrow \pred,2}\mid \tau(x_1)=s\big]\geq 0.01$, by  \Cref{lem-s21-lambda-2-sym}, we have $\Lambda^{(t)}_{5,j, r_{g_1\cdot y}}$ cannot be activated. Furthermore,
          \begin{align*}
                & \E\bigg[  \logit^{(t)}_{5,j'_2} \cdot \Big(\ReLU^{\prime}(\Lambda^{(t)}_{5,j'_2,r_{g_1\cdot y_0}
                  })\cdot  \\
                 &~~~\Big( \attn^{(t)}_{{\ans,1} \rightarrow \ans,0} \cdot V_{j'_2,  r_{g_1\cdot y_0}}(y_0)+\attn^{(t)}_{{\ans,1} \rightarrow \pred,1} \cdot V_{j'_2,  r_{g_1\cdot y_0}}(g_1)\\
                  &-\big(1-\attn^{(t)}_{{\ans,1} \rightarrow \ans,1}-\attn^{(t)}_{{\ans,1} \rightarrow \pred,2}\big) \Lambda^{(t)}_{5,j'_2, r_{g_1\cdot y_0}}\pm\tilde{O}(\sigma_0) \Big)\pm \tilde{O}(\delta^{q}) \Big)\1_{\tau(x_1)=s}\1_{\cE_1}\bigg]\\
                  &\geq  0.
              \end{align*}
         \end{itemize}
      
        \end{itemize}
        
       Putting the above discussion together, we have
       \begin{align*}
        \E\bigg[ \cN^{(t)}_{s,2,\rom2} \1_{\cE_2}\bigg]\geq  -O\Big(\frac{B}{n^2_y\cdot d^2}\Big).
       \end{align*}
       \item For $\Zb^{2,1}\in \cE_3$, by \Cref{induction-s21-non} and \Cref{lem-s21-lambda-3-sym}, we can first have some facts of logit:
       \begin{align*}
        1-\logit^{(t)}_{5,j_2} =\Omega(1),\quad  &\logit^{(t)}_{5,\tau(g_1(y_1))}=\Omega(1)\\
        \logit_{5,j}^{(t)}\leq \frac{1}{\poly d} & \text{ for other } j. 
       \end{align*}
       Therefore, we have 
       \begin{align*}
        &\E\bigg[ \cN^{(t)}_{s,2,\rom1} \1_{\cE_3}\bigg]\\
        &=
        \E\bigg[  (1-\logit^{(t)}_{5,j_2})\cdot \Big( -\attn^{(t)}_{{\ans,1} \rightarrow \ans,0} \cdot V_{j_2,  r_{g_2\cdot y_1}}(y_1)-\attn^{(t)}_{{\ans,1} \rightarrow \pred,1} \cdot V_{j_2,  r_{g_2\cdot y_1}}(g_1)\\
          &~~~~~+\big(1-\attn^{(t)}_{{\ans,1} \rightarrow \ans,1}-\attn^{(t)}_{{\ans,1} \rightarrow \pred,2}\big) \Lambda^{(t)}_{5,j_2, r_{g_2\cdot y_1}}\pm\tilde{O}(\sigma_0) \pm \tilde{O}(\delta^{q}) \Big)\1_{\tau(x_1)=s}\1_{\cE_3}\bigg]\\
          &\geq \Omega\Big(\frac{B}{n_y\cdot d}\Big).
     \end{align*}
     Moving to $\cN^{(t)}_{s,2,\rom2}$,  we have
     \begin{align*}
        & \E\bigg[ \cN^{(t)}_{s,2,\rom2} \1_{\cE_3}\bigg]\\
        &\geq \E\bigg[ \logit^{(t)}_{5,\tau(g_1(y_1))} \cdot\\
        &~~~~ \Big( \attn^{(t)}_{{\ans,1} \rightarrow \ans,0} \cdot V_{\tau(g_1(y_1)),  r_{g_1\cdot y_1}}(y_1)+\attn^{(t)}_{{\ans,1} \rightarrow \pred,1} \cdot V_{\tau(g_1(y_1)),  r}(g_1)\\
        &-\big(1-\attn^{(t)}_{{\ans,1} \rightarrow \ans,1}-\attn^{(t)}_{{\ans,1} \rightarrow \pred,2}\big) \Lambda^{(t)}_{5,\tau(g_1(y_1)), r_{g_1\cdot y_1}}\pm\tilde{O}(\sigma_0) \\
        &\pm \tilde{O}(\delta^{q}) \Big)\1_{\tau(x_1)=s}\1_{\cE_3}\bigg]-{O}\Big(\frac{n_y\cdot B}{d\cdot \poly d}\Big)\\
        &\geq  \Omega\Big(\frac{B}{n_y\cdot d}\Big) -{O}\Big(\frac{B}{d\cdot \poly d}\Big)=\Omega\Big(\frac{B}{n_y\cdot d}\Big).
      \end{align*}
    \item For $\Zb^{2,1}\in \cE_{4}$, by comparing \Cref{lem-s21-lambda-4-sym} with \Cref{lem-s21-lambda-2-sym}, we can directly bound $\cN^{(t)}_{s,2,\rom2}$ in the same way as $\cE_{2}$, where the only difference is that we do not need to consider $\tau(g_1(y_1))$ for $\cE_4$. Thus $\E\bigg[ \cN^{(t)}_{s,2,\rom2} \1_{\cE_4}\bigg]\geq 0$. Moreover, it is clear that
    \begin{align*}
       & \E\bigg[ \cN^{(t)}_{s,2,\rom1} \1_{\cE_4}\bigg]\\
        &= \E\bigg[  (1-\logit^{(t)}_{5,j_2})\cdot \Big(\ReLU^{\prime}(\Lambda^{(t)}_{5,j_2, r_{g_2\cdot y_1}
           })\cdot \notag \\
           &~~~\Big( -\attn^{(t)}_{{\ans,1} \rightarrow \ans,0} \cdot V_{j_2,  r_{g_2\cdot y_1}}(y_0)-\attn^{(t)}_{{\ans,1} \rightarrow \pred,1} \cdot V_{j_2,  r_{g_2\cdot y_1}}(g_1)\notag\\
           &~~~+\big(1-\attn^{(t)}_{{\ans,1} \rightarrow \ans,1}-\attn^{(t)}_{{\ans,1} \rightarrow \pred,2}\big) \Lambda^{(t)}_{5,j_2, r_{g_2\cdot y_1}}\pm\tilde{O}(\sigma_0) \Big)\pm \tilde{O}(\delta^{q}) \Big)\1_{\tau(x_1)=s}\1_{\cE_4}\bigg] \\
          & \geq 0.
    \end{align*}
    where the last inequality is due to the cancellation of $\attn^{(t)}_{{\ans,1} \rightarrow \ans,0} \cdot V_{j_2,  r_{g_2\cdot y_1}}(y_0)+\attn^{(t)}_{{\ans,1} \rightarrow \pred,1} \cdot V_{j_2,  r_{g_2\cdot y_1}}(g_1)$, and the fact that
    \begin{align*}
        \Lambda^{(t)}_{5,j_2,r_{g_2\cdot y_1}} = \attn^{(t)}_{\ans,1\to \pred,2} \cdot 2B + O(\delta)\geq \Omega(B). 
    \end{align*} 
      \item For $\Zb^{2,1}\in \cE_5\cup\cE_6$,  by \Cref{induction-s21-non}, \Cref{lem-s21-lambda-5-sym} and \Cref{lem-s21-lambda-6-sym}, we can derive the following logit condition: $1-\logit^{(t)}_{5, j_2}, \logit^{(t)}_{5,j}\leq \frac{1}{\poly d}$ for $j\neq j_2$.  Hence, we can simply bound $\cN^{(t)}_{s,2,\rom1}$ and $\cN^{(t)}_{s,2,\rom2}$ as follows:
      \begin{align*}
        \bigg| \E\bigg[ \cN^{(t)}_{s,2,\rom1} \1_{\cE_m}\bigg]\bigg|,  \bigg| \E\bigg[ \cN^{(t)}_{s,2,\rom2} \1_{\cE_m}\bigg]\bigg|\leq O\Big(\frac{B}{d n_y}\cdot\frac{1}{\poly d}\Big) \text{ for }m\in\{5,6\}.
      \end{align*}
    \end{enumerate}   
    Putting it all together, we have for the regmie (1), 
    \begin{align*}
        \Big[-\nabla_{\Q^{(t)}_{4,3}}{\Loss^{2,2}_{5}}\Big]_{s,s}+ \Big[-\nabla_{\Q^{(t)}_{4,4}}{\Loss^{2,2}_{5}}\Big]_{s,s} =\sum_{\kappa\in\{\rom1,\rom2,\rom3\}}\sum_{i\in[6]}\E\bigg[ \cN^{(t)}_{s,2,\kappa} \1_{\cE_i}\bigg]\geq \Omega(\frac{B}{d\cdot n_y}).
    \end{align*}
    \item In regime (ii), the analysis follows a structure analogous to that of regime (i). The primary distinction lies in the fact that, under the main event $\cE_1$, the term $\Lambda^{(t)}_{5,j_2, r{g_2\cdot y_1}}$ is guaranteed to remain within the linear regime, thereby serving as the dominant component driving the overall gradient.
    \begin{enumerate}
        \item For $\Zb^{2,1}\in \cE_{1}$, by \Cref{induction-s21-non} and \Cref{lem-s21-lambda-1-sym}, we have 
        \begin{align*}
           &\E\bigg[ \cN^{(t)}_{s,2,\rom1} \1_{\cE_1}\bigg]\\
           &=
           \E\bigg[  (1-\logit^{(t)}_{5,j_2})\cdot \Big(\sum_{r\in\hat{\fA}_{j_2}}\ReLU^{\prime}(\Lambda^{(t)}_{5,j_2, r
             })\cdot \\
             &~~~~~\Big( -\attn^{(t)}_{{\ans,1} \rightarrow \ans,0} \cdot V_{j_2,  r}(y_0)-\attn^{(t)}_{{\ans,1} \rightarrow \pred,1} \cdot V_{j_2,  r}(g_1)\\
             &~~~~~+\big(1-\attn^{(t)}_{{\ans,1} \rightarrow \ans,1}-\attn^{(t)}_{{\ans,1} \rightarrow \pred,2}\big) \Lambda^{(t)}_{5,j_2, r}\pm\tilde{O}(\sigma_0) \Big)\pm \tilde{O}(\delta^{q}) \Big)\1_{\tau(x_1)=s}\1_{\cE_1}\bigg].
        \end{align*}
       By \Cref{lem-s21-lambda-1-sym}, we can obtain 
        \begin{align}
            &\E\bigg[  (1-\logit^{(t)}_{5,j_2})\cdot \Big(  
\Big( -\attn^{(t)}_{{\ans,1} \rightarrow \ans,0} \cdot V_{j_2,  r_{g_2\cdot y_1}}(y_0)-\attn^{(t)}_{{\ans,1} \rightarrow \pred,1} \cdot V_{j_2,  r_{g_2\cdot y_1}}(g_1)\notag\\
               &~~~+\big(1-\attn^{(t)}_{{\ans,1} \rightarrow \ans,1}-\attn^{(t)}_{{\ans,1} \rightarrow \pred,2}\big) \Lambda^{(t)}_{5,j_2, r_{g_2\cdot y_1}}\pm\tilde{O}(\sigma_0) \Big)\pm \tilde{O}(\delta^{q}) \Big)\1_{\tau(x_1)=s}\1_{\cE_1}\bigg]\label{eq-s21-gd1-non-2} \\ 
             &  \geq   \Omega\Big(\frac{B}{d}\Big)\cdot  \E\bigg[ (1-\logit^{(t)}_{5,j_2})\big| \tau(x_1)=s, \cE_1\bigg]. \notag
        \end{align}
        On the other hand, 
        \begin{align*}
          & - \E\bigg[ (1-\logit^{(t)}_{5,j_2})\cdot \ReLU^{\prime}(\Lambda^{(t)}_{5,j_2, r_{g_1\cdot y_0}
              })\cdot \\
              &~~~~~~~~~~~~\Big( \attn^{(t)}_{{\ans,1} \rightarrow \ans,0} \cdot V_{j_2,  r_{g_1\cdot y_0}}(y_0)+\attn^{(t)}_{{\ans,1} \rightarrow \pred,1} \cdot V_{j_2,  r_{g_1\cdot y_0}}(g_1)\Big)\\
              &~~~~~~~~~~~~~~~\1_{\tau(x_1)=s}\1_{\cE_1}\1_{g_1(y_0)=g_2(y_1)}\bigg]\\
              &\geq -O\big(\frac{B}{d\cdot n_y}\big)\cdot \E\bigg[ (1-\logit^{(t)}_{5,j_2})\big| \tau(x_1)=s, \cE_1\bigg],
        \end{align*}
        which implies that $\E\bigg[ \cN^{(t)}_{s,2,\rom1} \1_{\cE_1}\bigg]\geq  \Omega\Big(\frac{B}{d}\Big)\cdot  \E\bigg[ (1-\logit^{(t)}_{5,j_2})\big| \tau(x_1)=s, \cE_1\bigg]\geq 0$.
             Moving to $\cN^{(t)}_{s,2,\rom2}$, by \Cref{lem-s21-lambda-1-sym}, we have
             \begin{itemize}
                \item for $j=\tau(g_1(y_0))$, $r=r_{g_1\cdot y_0}$
                \begin{align*}
&\Big( \attn^{(t)}_{{\ans,1} \rightarrow \ans,0} \cdot V_{j,  r}(y_0)+\attn^{(t)}_{{\ans,1} \rightarrow \pred,1} \cdot V_{j,  r}(g_1)\\
                      &~~~~~~~-\big(1-\attn^{(t)}_{{\ans,1} \rightarrow \ans,1}-\attn^{(t)}_{{\ans,1} \rightarrow \pred,2}\big) \Lambda^{(t)}_{5,j, r}\pm\tilde{O}(\sigma_0) \Big) \geq 0.
                  \end{align*}
                  \item for $j=\tau(g_1(y_1))$, $r=r_{g_1\cdot y_1}$, due to the cancellation of $\attn^{(t)}_{{\ans,1} \rightarrow \ans,0} \cdot V_{j,  r}(y_0)+\attn^{(t)}_{{\ans,1} \rightarrow \pred,1} \cdot V_{j,  r}(g_1)$, we have
    \begin{align*}
                    &\Big( \attn^{(t)}_{{\ans,1} \rightarrow \ans,0} \cdot V_{j,  r}(y_0)+\attn^{(t)}_{{\ans,1} \rightarrow \pred,1} \cdot V_{j,  r}(g_1)\\
                                          &~~~~~~~-\big(1-\attn^{(t)}_{{\ans,1} \rightarrow \ans,1}-\attn^{(t)}_{{\ans,1} \rightarrow \pred,2}\big) \Lambda^{(t)}_{5,j, r}\pm\tilde{O}(\sigma_0) \Big)\\
                                          & \geq -\big(1-\attn^{(t)}_{{\ans,1} \rightarrow \ans,1}-\attn^{(t)}_{{\ans,1} \rightarrow \pred,2}\big) \Lambda^{(t)}_{5,j, r}.
                                      \end{align*}
                                      \item for $j=\tau(g_2(y_0))$, $r=r_{g_2\cdot y_0}$
                                      \begin{align*}
                      &\Big( \attn^{(t)}_{{\ans,1} \rightarrow \ans,0} \cdot V_{j,  r}(y_0)+\attn^{(t)}_{{\ans,1} \rightarrow \pred,1} \cdot V_{j,  r}(g_1)\\
                                            &~~~~~~~-\big(1-\attn^{(t)}_{{\ans,1} \rightarrow \ans,1}-\attn^{(t)}_{{\ans,1} \rightarrow \pred,2}\big) \Lambda^{(t)}_{5,j, r}\pm\tilde{O}(\sigma_0) \Big) \\
                                            & 
\geq                            \big(1-\attn^{(t)}_{{\ans,1} \rightarrow \ans,1}-\attn^{(t)}_{{\ans,1} \rightarrow \pred,2}\big) \Lambda^{(t)}_{5,j, r}.
                                        \end{align*}
                                        \item for $j=g_1(y)$, where $\exists y\neq y_0, y_1$, s.t., $g_{2}(y)=g_1(y)$, we have 
                \begin{align*}
                                            &\Big( \attn^{(t)}_{{\ans,1} \rightarrow \ans,0} \cdot V_{j,  r}(y_0)+\attn^{(t)}_{{\ans,1} \rightarrow \pred,1} \cdot V_{j,  r}(g_1)\\
                                                                  &~~~~~~~-\big(1-\attn^{(t)}_{{\ans,1} \rightarrow \ans,1}-\attn^{(t)}_{{\ans,1} \rightarrow \pred,2}\big) \Lambda^{(t)}_{5,j, r}\pm\tilde{O}(\sigma_0) \Big)\\
                                                                  & \geq -\big(1-\attn^{(t)}_{{\ans,1} \rightarrow \ans,1}-\attn^{(t)}_{{\ans,1} \rightarrow \pred,2}\big) \Lambda^{(t)}_{5,j, r}.
                                                              \end{align*}
             \end{itemize}
             Putting them together, and upper bound $\sum_{j\neq j_2\in\tau(\cY)}\logit^{(t)}_{5,j} $ by $1-\logit^{(t)}_{5,j_2}$, we have 
        \begin{align*}
           & \E\bigg[ \cN^{(t)}_{s,2,\rom2} \1_{\cE_1}\bigg]\\
           &= \E\bigg[  \sum_{j\neq j_2\in\tau(\cY)}\logit^{(t)}_{5,j} \cdot \Big(\sum_{r\in\hat{\fA}_{j}}\ReLU^{\prime}(\Lambda_{5,j,r
             })\cdot  \\
            &~~~~~~~~~~\Big( \attn^{(t)}_{{\ans,1} \rightarrow \ans,0} \cdot V_{j,  r}(y_0)+\attn^{(t)}_{{\ans,1} \rightarrow \pred,1} \cdot V_{j,  r}(g_1)\\
             &~~~~~~~-\big(1-\attn^{(t)}_{{\ans,1} \rightarrow \ans,1}-\attn^{(t)}_{{\ans,1} \rightarrow \pred,2}\big) \Lambda^{(t)}_{5,j, r}\pm\tilde{O}(\sigma_0) \Big)\pm \tilde{O}(\delta^{q}) \Big)\1_{\tau(x_1)=s}\1_{\cE_1}\bigg]\\
             &\geq -\E\bigg[\Big(1-\logit^{(t)}_{5,j_2} \Big)\cdot \big(1-\attn^{(t)}_{{\ans,1} \rightarrow \ans,1}-\attn^{(t)}_{{\ans,1} \rightarrow \pred,2}\big)\\
             &~~~\cdot \max_{y\neq y_0,y_1} \Big\{ \Lambda^{(t)}_{5,\tau(g_1(y_1)), r_{g_1\cdot y_1}}, \Lambda^{(t)}_{5,\tau(g_2(y_0)), r_{g_2\cdot y_0}}\Big\}\1_{\tau(x_1)=s}\1_{\cE_1}\bigg],
         \end{align*}
         where the last inequality is due to the fact that $\Lambda^{(t)}_{5,\tau(g_2(y_0)), r_{g_2\cdot y_0}}=\Lambda^{(t)}_{5,\tau(g_2(y)), r_{g_2\cdot y}}\pm \tilde{O}(\frac{B}{d\cdot n_y})$ for $y\neq y_0, y_1$ s.t., $g_1(y)=g_2(y)$.
         Notice that 
         \begin{align*}
            \max_{y\neq y_0,y_1} \Big\{ \Lambda^{(t)}_{5,\tau(g_1(y_1)), r_{g_1\cdot y_1}}, \Lambda^{(t)}_{5,\tau(g_2(y_0)), r_{g_2\cdot y_0}}\Big\}\leq 
        \\
        \max\bigg\{\attn^{(t)}_{{\ans,1} \rightarrow \pred,2}-\attn^{(t)}_{{\ans,1} \rightarrow \ans,1}, \Theta(\frac{1}{n_y})\bigg\}2B,
         \end{align*}
while 
\begin{align*}
    \eqref{eq-s21-gd1-non-2}&\geq \E\bigg[\Big(1-\logit^{(t)}_{5,j_2} \Big)\cdot \attn^{(t)}_{{\ans,1} \rightarrow \ans,0}\cdot 2B\1_{\tau(x_1)=s}\1_{\cE_1}\bigg]\\
    &= \frac{1}{2} \E\bigg[\Big(1-\logit^{(t)}_{5,j_2} \Big)\cdot \big(1-\attn^{(t)}_{{\ans,1} \rightarrow \ans,1}-\attn^{(t)}_{{\ans,1} \rightarrow \pred,2}\big)2B\1_{\tau(x_1)=s}\1_{\cE_1}\bigg].
\end{align*}

Since by \Cref{induction-s21-non}, $\attn^{(t)}_{{\ans,1} \rightarrow \pred,2}-\attn^{(t)}_{{\ans,1} \rightarrow \ans,1}\leq c_1\ll \frac{1}{2}$, thus we have 
\begin{align*}
    \E\bigg[ \cN^{(t)}_{s,2,\rom1} \1_{\cE_1}\bigg]+\E\bigg[ \cN^{(t)}_{s,2,\rom2} \1_{\cE_1}\bigg]\geq  \Omega\Big(\frac{B}{d}\Big)\cdot  \E\bigg[ (1-\logit^{(t)}_{5,j_2})\big| \tau(x_1)=s, \cE_1\bigg]. 
\end{align*}
         \item For $\Zb^{2,1}\in \cE_{2}$, 
          $\cN^{(t)}_{s,2,\rom1}$ can be bounded in the same way as \eqref{eq-s21-gd1-non-2}, and we have 
        \begin{align*}
            \E\bigg[ \cN^{(t)}_{s,2,\rom1} \1_{\cE_2}\bigg]\geq  \Omega\Big(\frac{B}{d\cdot n_y}\Big)\cdot  \E\bigg[(1-\logit^{(t)}_{5,j_2})\big| \tau(x_1)=s, \cE_2\bigg]\geq 0.
        \end{align*}
        Moving to $\cN^{(t)}_{s,2,\rom2}$,
        \begin{itemize}
            \item for $j=\tau(g_1(y_1))$, we have 
            \begin{align}
                &= \E\bigg[  \logit_{5,j}^{(t)}\Big(\ReLU^{\prime}(\Lambda_{5,j,r_{g_1\cdot y_1}
                  })\cdot \notag \\
                 &~~~\Big( \attn^{(t)}_{{\ans,1} \rightarrow \ans,0} \cdot V_{j,  r_{g_1\cdot y_1}}(y_0)+\attn^{(t)}_{{\ans,1} \rightarrow \pred,1} \cdot V_{j,  r_{g_1\cdot y_1}}(g_1)\label{eq-s21-gd1-non-3}\\
                  &~~-\big(1-\attn^{(t)}_{{\ans,1} \rightarrow \ans,1}-\attn^{(t)}_{{\ans,1} \rightarrow \pred,2}\big) \Lambda^{(t)}_{5,j, r}\pm\tilde{O}(\sigma_0) \Big)\pm \tilde{O}(\delta^{q}) \Big)\1_{\tau(x_1)=s}\1_{\cE_2}\bigg] \notag\\
                  &\geq -{O}\Big(\frac{B}{n^2_y\cdot d}\Big)\cdot  \E\bigg[(1-\logit^{(t)}_{5,j_2})\big| \tau(x_1)=s, \cE_2\bigg] \notag.
              \end{align}
              where the last inequality is due to the cancellation of the term $\attn^{(t)}_{{\ans,1} \rightarrow \ans,0} \cdot V_{j,  r_{g_1\cdot y_1}}(y_0)+\attn^{(t)}_{{\ans,1} \rightarrow \pred,1} \cdot V_{j,  r_{g_1\cdot y_1}}(g_1)$ and the fact that
            \begin{align*}
                \Lambda^{(t)}_{5,\tau(g_1(y_1)),r_{g_1\cdot y_1}} =  \Big(\attn^{(t)}_{\ans,1\to \ans,1}-\attn^{(t)}_{\ans,1\to \pred,2}\Big) \cdot \frac{2B}{n_y}+O(\delta)\leq O\big(\frac{B}{n_y}\big).
            \end{align*}
            \item  for $j'_2=\tau(g_2(y_0))=\tau(g_1(y_0))$, clearly, we have  
            \begin{align*}
                  & \E\bigg[  \logit^{(t)}_{5,j'_2} \cdot \Big(\ReLU^{\prime}(\Lambda^{(t)}_{5,j'_2,r_{g_1\cdot y_0}
                    })\cdot  \\
                   &~~~\Big( \attn^{(t)}_{{\ans,1} \rightarrow \ans,0} \cdot V_{j'_2,  r_{g_1\cdot y_0}}(y_0)+\attn^{(t)}_{{\ans,1} \rightarrow \pred,1} \cdot V_{j'_2,  r_{g_1\cdot y_0}}(g_1)\\
                    &-\big(1-\attn^{(t)}_{{\ans,1} \rightarrow \ans,1}-\attn^{(t)}_{{\ans,1} \rightarrow \pred,2}\big) \Lambda^{(t)}_{5,j'_2, r_{g_1\cdot y_0}}\pm\tilde{O}(\sigma_0) \Big)\pm \tilde{O}(\delta^{q}) \Big)\\
                    &~~~~~\1_{\tau(x_1)=s}\1_{\cE_2}\bigg]
                    \geq  0,
                \end{align*}
                where the last inequality is due to the fact that 
        \begin{align*}
           & \attn^{(t)}_{{\ans,1}\rightarrow \pred,1}\cdot V_{j'_2,  r_{g_1\cdot y_0}}(g_1)-\big(1-\attn^{(t)}_{{\ans,1}\rightarrow \ans,1}
                  \\
                  &~~~~~~~~  -\attn^{(t)}_{{\ans,1}\rightarrow \pred,2}\big)\Lambda^{(t)}_{5,j'_2, r_{g_1\cdot y_0}}\\
&\geq \attn^{(t)}_{{\ans,1}\rightarrow \pred,1}\\
&~~~~~~~\cdot\bigg(1-2\Big(\attn^{(t)}_{{\ans,1}\rightarrow \pred,2}+\attn^{(t)}_{{\ans,1}\rightarrow \pred,1}-\attn^{(t)}_{{\ans,1}\rightarrow \ans,1}\Big)\bigg)2B\\
&\geq \attn^{(t)}_{{\ans,1}\rightarrow \pred,1}\bigg(1-2\Big(c_1+0.25\Big)\bigg)2B\geq 0.
        \end{align*}
            \item for $j=\tau(g_2(y))$ if there exists $y\not=y_0, y_1$ s.t., $g_1(y)=g_2(y)$
         \begin{itemize}
            \item    when $\E\big[\attn^{(t)}_{{\ans,1} \rightarrow \ans,1}-\attn^{(t)}_{{\ans,1} \rightarrow \pred,2}\mid \tau(x_1)=s\big]\leq 0.01$, by \Cref{induction-s21-non} and \Cref{lem-s21-lambda-2-sym}, we have $\logit^{(t)}_{5,\tau(g_2(y))}\ll O\big(\frac{1}{d}\big)\cdot \logit^{(t)}_{5,j_2'}$. Then 
            \begin{align*}
                &\bigg| \E\bigg[  \logit^{(t)}_{5,j} \cdot \Big(\ReLU^{\prime}(\Lambda^{(t)}_{5,j,r_{g_1\cdot y}
                  })\cdot  \\
                 &~~~\Big( \attn^{(t)}_{{\ans,1} \rightarrow \ans,0} \cdot V_{j,  r_{g_1\cdot y}}(y_0)+\attn^{(t)}_{{\ans,1} \rightarrow \pred,1} \cdot V_{j,  r_{g_1\cdot y}}(g_1)\\
                  &-\big(1-\attn^{(t)}_{{\ans,1} \rightarrow \ans,1}-\attn^{(t)}_{{\ans,1} \rightarrow \pred,2}\big) \Lambda^{(t)}_{5,j, r_{g_1\cdot y}}\pm\tilde{O}(\sigma_0) \Big)\pm \tilde{O}(\delta^{q}) \Big)\1_{\tau(x_1)=s}\1_{\cE_2}\bigg]\bigg|\\
                  &\leq  {O}\Big(\frac{1}{d}\Big)\cdot \eqref{eq-s21-gd1-non-3}.
              \end{align*}
              \item when $\E\big[\attn^{(t)}_{{\ans,1} \rightarrow \ans,1}-\attn^{(t)}_{{\ans,1} \rightarrow \pred,2}\mid \tau(x_1)=s\big]\geq 0.01$, by  \Cref{lem-s21-lambda-2-sym}, we have $\Lambda^{(t)}_{5,j, r_{g_1\cdot y}}$ cannot be activated. 
         \end{itemize}
      
        \end{itemize}
        
       Putting the above discussion together, we have
       \begin{align*}
        \E\bigg[ \cN^{(t)}_{s,2,\rom2} \1_{\cE_1}\bigg]+  \E\bigg[ \cN^{(t)}_{s,2,\rom2} \1_{\cE_2}\bigg]\geq   \Omega\Big(\frac{B}{d\cdot n_y}\Big)\cdot  \E\bigg[(1-\logit^{(t)}_{5,j_2})\big| \tau(x_1)=s, \cE_2\bigg].
       \end{align*}
       \item For $\Zb^{2,1}\in \cE_3$, by \Cref{induction-s21-non} and \Cref{lem-s21-lambda-3-sym}, we can first have the following logit bound:
       \begin{align*}
        \logit_{5,j}^{(t)}\leq \frac{1}{\poly d}\cdot \Big(1-\logit^{(t)}_{5,j_2}\Big) & \text{ for } j\neq j_2, \tau\big(g_1(y_1)\big). 
       \end{align*}
       Therefore, we have 
       \begin{align*}
        &\E\bigg[ \cN^{(t)}_{s,2,\rom1} \1_{\cE_3}\bigg]\\
        &=
        \E\bigg[  (1-\logit^{(t)}_{5,j_2})\cdot \Big( -\attn^{(t)}_{{\ans,1} \rightarrow \ans,0} \cdot V_{j_2,  r_{g_2\cdot y_1}}(y_1)-\attn^{(t)}_{{\ans,1} \rightarrow \pred,1} \cdot V_{j_2,  r_{g_2\cdot y_1}}(g_1)\\
          &~~~~~+\big(1-\attn^{(t)}_{{\ans,1} \rightarrow \ans,1}-\attn^{(t)}_{{\ans,1} \rightarrow \pred,2}\big) \Lambda^{(t)}_{5,j_2, r_{g_2\cdot y_1}}\pm\tilde{O}(\sigma_0) \pm \tilde{O}(\delta^{q}) \Big)\1_{\tau(x_1)=s}\1_{\cE_3}\bigg]\\
          &\geq \Omega\Big(\frac{B}{n_y\cdot d}\Big)\cdot  \E\bigg[(1-\logit^{(t)}_{5,j_2})\big| \tau(x_1)=s, \cE_3\bigg].
     \end{align*}
     Moving to $\cN^{(t)}_{s,2,\rom2}$,  we have
     \begin{align*}
        & \E\bigg[ \cN^{(t)}_{s,2,\rom2} \1_{\cE_3}\bigg]\\
        &\geq \E\bigg[ \logit^{(t)}_{5,\tau(g_1(y_1))} \cdot \\
        &~~\Big( \attn^{(t)}_{{\ans,1} \rightarrow \ans,0} \cdot V_{\tau(g_1(y_1)),  r_{g_1\cdot y_1}}(y_1)+\attn^{(t)}_{{\ans,1} \rightarrow \pred,1} \cdot V_{\tau(g_1(y_1)),  r}(g_1)\\
        &-\big(1-\attn^{(t)}_{{\ans,1} \rightarrow \ans,1}-\attn^{(t)}_{{\ans,1} \rightarrow \pred,2}\big) \Lambda^{(t)}_{5,\tau(g_1(y_1)), r_{g_1\cdot y_1}}\pm\tilde{O}(\sigma_0) \pm \tilde{O}(\delta^{q}) \Big)\\
        &~~~~~\1_{\tau(x_1)=s}\1_{\cE_3}\bigg]\\
        &-{O}\Big(\frac{n_y\cdot B}{n_y\cdot d\cdot \poly d}\Big)\cdot \E\bigg[(1-\logit^{(t)}_{5,j_2})\big| \tau(x_1)=s, \cE_3\bigg].
      \end{align*}
      Thus,
      \begin{align*}
        \E\bigg[ \cN^{(t)}_{s,2,\rom2} \1_{\cE_3}\bigg]+ \E\bigg[ \cN^{(t)}_{s,2,\rom2} \1_{\cE_3}\bigg]\geq \Omega\Big(\frac{B}{n_y\cdot d}\Big)\cdot  \E\bigg[(1-\logit^{(t)}_{5,j_2})\big| \tau(x_1)=s, \cE_3\bigg].
      \end{align*}
    \item For $\Zb^{2,1}\in \cE_{4}$, we can bound the gradient in a manner similar to regime (i), and obtain 
    \begin{align*}
        \E\bigg[ \cN^{(t)}_{s,2,\rom2} \1_{\cE_4}\bigg]+ \E\bigg[ \cN^{(t)}_{s,2,\rom2} \1_{\cE_4}\bigg]\geq 0.
      \end{align*}
      \item For $\Zb^{2,1}\in \cE_5\cup\cE_6$,  by \Cref{induction-s21-non}, \Cref{lem-s21-lambda-5-sym} and \Cref{lem-s21-lambda-6-sym}, we can derive the following logit condition: $1-\logit^{(t)}_{5, j_2}, \logit^{(t)}_{5,j}\leq \frac{1}{\poly d}\cdot  \E\bigg[(1-\logit^{(t)}_{5,j_2})\big| \tau(x_1)=s, \cE_1\bigg]$ for $j\neq j_2$.  Hence, we can simply bound $\cN^{(t)}_{s,2,\rom1}$ and $\cN^{(t)}_{s,2,\rom2}$ as follows:
      \begin{align*}
        &\bigg| \E\bigg[ \cN^{(t)}_{s,2,\rom1} \1_{\cE_m}\bigg]\bigg|,  \bigg| \E\bigg[ \cN^{(t)}_{s,2,\rom2} \1_{\cE_m}\bigg]\bigg|\\& \leq O\Big(\frac{B}{d n_y}\cdot\frac{1}{\poly d}\Big)\cdot  \E\bigg[(1-\logit^{(t)}_{5,j_2})\big| \tau(x_1)=s, \cE_1\bigg], \text{ for } m\in\{5,6\}.
      \end{align*}
    \end{enumerate}  
    Putting everything together, we have for the regmie (ii), 
    \begin{align*}
        \Big[-\nabla_{\Q^{(t)}_{4,3}}{\Loss^{2,2}_{5}}\Big]_{s,s}+ \Big[-\nabla_{\Q^{(t)}_{4,4}}{\Loss^{2,2}_{5}}\Big]_{s,s} =\sum_{\kappa\in\{\rom1,\rom2,\rom3\}}\sum_{i\in[6]}\E\bigg[ \cN^{(t)}_{s,2,\kappa} \1_{\cE_i}\bigg]\\
        \geq \Omega(\frac{B}{d})\cdot \E\bigg[(1-\logit^{(t)}_{5,j_2})\big| \tau(x_1)=s, \cE_1\bigg].
    \end{align*}
\end{enumerate}


\end{proof}

\begin{lemma}\label{lem-s21-grad-2-non}
    If \Cref{induction-s21-non} holds for all iterations $<t$, given $s\in\tau(\X)$,  for $[\Qb_{4,p}]_{s,s'}$, $p\in\{3,4\}$, $s'\not=s\in\tau(\X)$,  we have
    \begin{align*}
        \bigg|\Big[-\nabla_{\Q^{(t)}_{4,p}}{\Loss^{2,2}_{5}}\Big]_{s,s'}\bigg| \leq O\Big(\frac{1}{d}\Big) \bigg|\Big[-\nabla_{\Q^{(t)}_{4,p}}{\Loss^{2,2}_{5}}\Big]_{s,s}\bigg| .
    \end{align*}
  \end{lemma}
  \subsubsection{Bounded Decrease of Attention to Related Context Clause}
  \begin{lemma}\label{lem-s21-attention-decrease-non}
     If \Cref{induction-s21-non} holds for all iterations $<t$, then for any sample $\Zb^{2,1}$, we have 
    \begin{align*}
      \attn^{(t)}_{{\ans,1} \rightarrow \pred,1} - \attn^{(t)}_{{\ans,1} \rightarrow \pred,2}\geq -O\big(\frac{\log d\log d
      }{\log d}\big).
    \end{align*}
  \end{lemma}
  \begin{proof}
    Denote the first time that $\E\Big[\attn^{(t)}_{{\ans,1} \rightarrow \pred,1} - \attn^{(t)}_{{\ans,1} \rightarrow \pred,2}\big| \tau(x_1)=s\Big]\leq -\Omega\big(\frac{\log d\log d
    }{\log d}\big)$ as $\tilde{T}$. Notice that $\left|\left[\mathbf{Q}_{4, p}^{(t)}\right]_{s, s^{\prime}}\right| \leq O\left(\frac{\left[\mathbf{Q}_{4, p}^{(t)}\right]_{s, s}}{d}\right)$ for $p\in \{3,4\}$, thus for any sample $\Zb^{2,1}$ satisfying $\tau(x_1)=s$, at time $\tilde{T}$, we have 
    \begin{align*}
        \attn^{(\tilde{T})}_{{\ans,1} \rightarrow \pred,1} - \attn^{(\tilde{T})}_{{\ans,1} \rightarrow \pred,2}\leq -\Omega\Big(\frac{\log d\log d
        }{\log d}\Big).
      \end{align*}
      Based on the the gradient compositions from \Cref{lem-grad-decompositions-sym}, we have $\Big[-\nabla_{\mathbf{Q}_{4,3}} \operatorname{Loss}_5^{2,2}\Big]_{s, s}=\mathbb{E}\Big[\mathcal{N}_{s, 3,2, i}+\mathcal{N}_{s, 3,2, i i}+\mathcal{N}_{s, 3,2, i i i}\Big]$, and we will discuss  $\cN_{s,3,2,\kappa}$ for $\kappa\in\{\rom1,\rom2, \rom3\}$ on different samples $\Zb^{2,1}$. 
    Following the similar argument as \Cref{lem-s21-gd1-non},  we can first directly bound the term $\cN^{(\tilde{T})}_{s,3,2,\rom3}$ as follows:
\begin{align*}
  \E\bigg[ \cN^{(\tilde{T})}_{s,3,2,\rom3}\bigg] \leq \tilde{O}(\sigma_0^q)=\frac{1}{\poly d}.
\end{align*}
      \begin{itemize}
        \item for $\Zb^{2,1}\in\cE_1$, at time $\tilde{T}$, since the neurons for predicting $j_2$ cannot be activated, and $\logit^{(\tilde{T})}_{5,j}\leq O(\frac{1}{d})$ for $j\neq j_2$, thus we can naively bound the gradient on the event $\cE_1$ as follows: 
        \begin{align*}
            \Big| \sum_{\kappa\in\{\rom1,\rom2\}}\E\big[\cN_{s,3,2,\kappa}^{(\tilde{T})}\1_{\cE_1}\big]\Big|\leq O\Big(\frac{n_y B}{d^2}\Big)+\tilde{O}(\delta^q). 
        \end{align*}
        \item for $\Z^{2,1}\in\cE_2$, similarly, $\Lambda_{5,j_2,r}^{(t)}$ is not activated, and thus we can focus on the term $\cN_{s,3,2,\rom2}$, and specifically, the prediction of $j_2^{'}=\tau(g_2(y_0))$. By \eqref{eq-def-N-s-3-2-2-sym} and \Cref{lem-s21-lambda-2-sym}, we have 
        \begin{align}
            &\E\bigg[ \cN^{(\tilde{T})}_{s,3,2,\rom2} \1_{\cE_2}\bigg] \geq-\E\bigg[ \attn^{(\tilde{T})}_{\ans,1\to \pred,2}\notag\\
            &\cdot \Big(1+\attn^{(\tilde{T})}_{\ans,1\to \ans,1}-\attn^{(\tilde{T})}_{\ans,1\to \pred,2}-\attn^{(\tilde{T})}_{\ans,1\to \pred,1}\Big)\mid \tau(x_1)=s\bigg] \notag\\
            &\cdot \E\bigg[\logit^{(\tilde{T})}_{5,j'_2}\mid \tau(x_1)=s, \cE_2\bigg]\frac{2B}{n_y\cdot d} \pm 
            \tilde{O}(\sigma_0^q). \label{eq-bound-neg}
        \end{align}
        \item for $\Zb^{2,1}\in\cE_3$, by \Cref{lem-s21-lambda-3-sym}, we can mainly focus on the term $\cN^{(\tilde{T})}_{s,3,2,\rom1}$, and the prediction of $\tau(g_1(y_1))$ in $\cN^{(\tilde{T})}_{s,3,2,\rom2}$ since $\logit^{(\tilde{T})}_{5,\tau(g_1(y_1))}=1-O\big(\frac{1}{\log d}\big)$. Hence, we have 
     \begin{align}
            &\E\bigg[\Big(\cN^{(\tilde{T})}_{s,3,2,\rom2} + \cN^{(\tilde{T})}_{s,3,2,\rom2}\Big) \1_{\cE_3}\bigg] \geq \E\bigg[ \attn^{(\tilde{T})}_{\ans,1\to \pred,2} \label{eq-bound-pos}\\
            &\cdot \Big(1+\attn^{(\tilde{T})}_{\ans,1\to \pred,1}-\attn^{(\tilde{T})}_{\ans,1\to \pred,2}\Big)\mid \tau(x_1)=s\bigg]\cdot\Big(1-O\big(\frac{1}{\log d}\big)\Big)\frac{2B}{n_y\cdot d} \pm 
            \tilde{O}(\sigma_0^q).\notag
        \end{align}
        \item for $\Zb^{2,1}\in\cE_4$, the negative gradient can be bounded in the same way as \eqref{eq-bound-neg}, however, the probability of $\cE_4$ is order-wise smaller than $\cE_2$ and $\cE_3$, which can be neglected. Moreover, for for $\Zb^{2,1}\in\cE_5\cup\cE_6$, the overall gradient is also negeligble since $\logit^{(t)}_{5,j_2}$ is very close to $1$.
      \end{itemize}
      Putting it all together, we have 
      \begin{align*}
     &\Big[-\nabla_{\Q^{(\tilde{T})}_{4,3}}{\Loss^{2,2}_{5}}\Big]_{s,s}\geq -\E\bigg[ \attn^{(\tilde{T})}_{\ans,1\to \pred,2}\notag\\
     &\cdot \Big(1+\attn^{(\tilde{T})}_{\ans,1\to \ans,1}-\attn^{(\tilde{T})}_{\ans,1\to \pred,2}-\attn^{(\tilde{T})}_{\ans,1\to \pred,1}\Big)\mid \tau(x_1)=s\bigg] \notag\\
     &\cdot \E\bigg[\logit^{(\tilde{T})}_{5,j'_2}\mid \tau(x_1)=s, \cE_2\bigg]\frac{B}{n_y\cdot d} \\
   &+  \E\bigg[ \attn^{(\tilde{T})}_{\ans,1\to \pred,2} \\
            &\cdot \Big(1+\attn^{(\tilde{T})}_{\ans,1\to \pred,1}-\attn^{(\tilde{T})}_{\ans,1\to \pred,2}\Big)\mid \tau(x_1)=s\bigg]\cdot\Big(1-O\big(\frac{1}{\log d}\big)\Big)\frac{2B}{n_y\cdot d} \pm 
            \tilde{O}(\sigma_0^q).
    \end{align*}
    Notice that 
    \begin{align*}
        &\Big(1+\attn^{(\tilde{T})}_{\ans,1\to \pred,1}-\attn^{(\tilde{T})}_{\ans,1\to \pred,2}\Big)\\
        &-\Big(1+\attn^{(\tilde{T})}_{\ans,1\to \ans,1}-\attn^{(\tilde{T})}_{\ans,1\to \pred,2}-\attn^{(\tilde{T})}_{\ans,1\to \pred,1}\Big)\\
        &= 2\attn^{(\tilde{T})}_{\ans,1\to \pred,1}-\attn^{(\tilde{T})}_{\ans,1\to \ans,1}.
    \end{align*}
If $2\attn^{(\tilde{T})}_{\ans,1\to \pred,1}-\attn^{(\tilde{T})}_{\ans,1\to \ans,1}<c$ for some small constant $c>0$, s.t., $\frac{cB}{\log d}<1$,   then $\Lambda_{5,j'_2,r_{g_2\cdot y_0}}^{(\tilde{T})}\leq cB$. Hence $\logit^{(\tilde{T})}_{5,j'_2}\leq \frac{1}{d^{1-\frac{cB}{\log d}}}=o(1)$, and \eqref{eq-bound-neg} is dominated by the positive term \eqref{eq-bound-pos}. Else, clearly, \eqref{eq-bound-neg} can be cancelled out by the positive term \eqref{eq-bound-pos}. Therefore, 
\begin{align*}
    \Big[-\nabla_{\Q^{(\tilde{T})}_{4,3}}{\Loss^{2,2}_{5}}\Big]_{s,s}\geq \Omega\Big(\frac{B}{n_y\cdot d}\Big).
\end{align*}
As a consequence, together with the growth of $\Q^{(\tilde{T})}_{4,3}+\Q^{(\tilde{T})}_{4,4}$, and nearly no change of $[\Qb^{(t)}_{4,p}]_{s,s'}$ in \Cref{lem-s21-grad-2-non}, we have $\attn^{(t)}_{{\ans,1} \rightarrow \pred,1} - \attn^{(t)}_{{\ans,1} \rightarrow \pred,2}$  must start to decrease and cannot be larger than $O\big(\frac{\log\log d}{\log d}\big)$.
  \end{proof}

  \subsubsection{Attention gap is small}
\begin{lemma}\label{lem-s21-attention-gap-non}
    If \Cref{induction-s21-non} holds for all iterations $<t$, then for any sample $\Zb^{2,1}$, we have 
   \begin{align*}
     \attn^{(t)}_{{\ans,1} \rightarrow \pred,2} - \attn^{(t)}_{{\ans,1} \rightarrow \ans,1}\leq c_1,
   \end{align*}
   where $c_1>0$ is a small constant.
 \end{lemma}
 \begin{proof}
    Let $\tilde{T}$ denote the first time $\E\bigg[\attn^{(t)}_{\ans,1\to\pred,2}-\attn^{(t)}_{\ans,1\to\ans,1}\big|\tau(x_1)=s\bigg]\geq \frac{1.0005\log d}{2B}$. Notice that $\left|\left[\mathbf{Q}_{4, p}^{(t)}\right]_{s, s^{\prime}}\right| \leq O\left(\frac{\left[\mathbf{Q}_{4, p}^{(t)}\right]_{s, s}}{d}\right)$ for $p\in \{3,4\}$, thus for any sample $\Zb^{2,1}$ satisfying $\tau(x_1)=s$, at time $\tilde{T}$, we have  
    \begin{align*}
        \attn^{(\tilde{T})}_{\ans,1\to\pred,2}-\attn^{(\tilde{T})}_{\ans,1\to\ans,1}\geq \frac{1.0005\log d}{2B}.
    \end{align*}
    Based on \Cref{lem-grad-decompositions-sym}, we will discuss the gradients of $\Qb_{4,3}$ and $\Qb_{4,4}$ on different samples $\Zb^{2,1}$.  Since $\Lambda^{(\tilde{T})}_{5,j_2,r_{g_2\cdot y_1}}$, $\Lambda^{(\tilde{T})}_{5,j'_2,r_{g_2\cdot y_0}}$ is guaranteed to be activated to the linear regime for $\Zb^{2,1}\in\cE_1$, we only need to focus on the main event $\cE_1$. 
    
    From \Cref{lem-s21-lambda-1-sym}, at time $\tilde{T}$, we have 
    \begin{align*}
        \logit^{(\tilde{T})}_{5,j_2^{\prime}}= \frac{e^{\frac{1.0005\log d}{2B}\cdot 2B}}{e^{\frac{1.0005\log d}{2B}\cdot 2B}+O(d)}\Big(1-\logit^{(\tilde{T})}_{5,j_2}\Big)=\Big(1-O\big({1}/{d^{0.0005}}\big)\Big)\Big(1-\logit^{(\tilde{T})}_{5,j_2}\Big).
    \end{align*}
Then by \Cref{lem-grad-decompositions-sym}, we can obtain the gradient on the event $\cE_1$ as follows:
    \begin{align*}
       &\E\bigg[ \Big(\cN^{(\tilde{T})}_{s,3,2,\rom 1}+\cN^{(\tilde{T})}_{s,3,2,\rom 2}+\cN^{(\tilde{T})}_{s,3,2,\rom 3}\Big)\1_{\cE_1}\bigg]\\
        &= \E\Bigg[
        \attn^{(\tilde{T})}_{{\ans,1} \rightarrow \pred,2} \cdot  \bigg( \big(1-\logit^{(\tilde{T})}_{5,j_2}\big)\cdot  \Big(  V_{j_2,  r_{g_2\cdot y_1}}(g_2)- \Lambda^{(\tilde{T})}_{5,j_2,r_{g_2\cdot y_1}} \pm \tilde{O}(\sigma_0)\Big)\\
            &~~~~~~~~-\bigg(1-O\Big(\frac{1}{d^{0.0005}}\Big)\bigg)(1-\logit^{(t)}_{5,j_2})\cdot  \Big( V_{j'_2,  r_{g_2\cdot y_0}}(g_2)- \Lambda^{(\tilde{T})}_{5,j_2,r_{g_2\cdot y_0}} \pm \tilde{O}(\sigma_0)\Big)\\
            &~~~~~~~~~~\pm\tilde{O}(\sigma^q_0)\bigg)\1_{\tau(x_1)=s}\1_{\cE_1}\Bigg]\\
            &= \E\Bigg[
            \attn^{(\tilde{T})}_{{\ans,1} \rightarrow \pred,2} \cdot  \bigg( \Big(1-O\Big(\frac{1}{d^{0.0005}}\Big)\Big)(1-\logit^{(\tilde{T})}_{5,j_2})\cdot  \Big( \Lambda^{(\tilde{T})}_{5,j^{\prime}_2,r_{g_2\cdot y_0}} - \Lambda^{(\tilde{T})}_{5,j_2,r_{g_2\cdot y_1}} \Big)\\
            &~~~~~~~~+O\Big(\frac{1}{d^{0.0005}}\Big)(1-\logit^{(\tilde{T})}_{5,j_2})\cdot  \Big(  V_{5,j_2,  r_{g_2\cdot y_1}, 2}(g_2)- \Lambda^{(\tilde{T})}_{5,j_2,r_{g_2\cdot y_1}} \Big)\bigg)\1_{\tau(x_1)=s}\1_{\cE_1}\Bigg],
    \end{align*}
    on the other hand, we have  
    \begin{align*}
        &\E\bigg[ \Big(\cN^{(\tilde{T})}_{s,4,2,\rom 1}+\cN^{(\tilde{T})}_{s,4,2,\rom 2}+\cN^{(\tilde{T})}_{s,4,2,\rom 3}\Big)\1_{\cE_1}\bigg]\\
         &= \E\Bigg[
         \attn^{(\tilde{T})}_{{\ans,1} \rightarrow \pred,2} \cdot  \bigg( \big(1-\logit^{(\tilde{T})}_{5,j_2}\big)\cdot  \Big(  V_{j_2,  r_{g_2\cdot y_1}}(y_1)- \Lambda^{(\tilde{T})}_{5,j_2,r_{g_2\cdot y_1}} \pm \tilde{O}(\sigma_0)\Big)\\
             &~~~~~~~~-\bigg(1-O\Big(\frac{1}{d^{0.0005}}\Big)\bigg)(1-\logit^{(t)}_{5,j_2})\cdot  \Big( V_{j'_2,  r_{g_2\cdot y_0}}(y_1)- \Lambda^{(\tilde{T})}_{5,j_2,r_{g_2\cdot y_0}} \pm \tilde{O}(\sigma_0)\Big)\\
             &~~~~~~~~~~\pm\tilde{O}(\sigma^q_0)\bigg)\1_{\tau(x_1)=s}\1_{\cE_1}\Bigg]\\
             &= \E\Bigg[
             \attn^{(\tilde{T})}_{{\ans,1} \rightarrow \pred,2} \cdot  \bigg( \Big(1-O\Big(\frac{1}{d^{0.0005}}\Big)\Big)(1-\logit^{(\tilde{T})}_{5,j_2})\\
             &~~~~~\cdot  \Big( -V_{j'_2,  r_{g_2\cdot y_0}}(y_1)+\Lambda^{(\tilde{T})}_{5,j^{\prime}_2,r_{g_2\cdot y_0}} - \Lambda^{(\tilde{T})}_{5,j_2,r_{g_2\cdot y_1}} \Big)\\
             &~~~~~~~~-O\Big(\frac{1}{d^{0.0005}}\Big)(1-\logit^{(\tilde{T})}_{5,j_2})\cdot  \Lambda^{(\tilde{T})}_{5,j_2,r_{g_2\cdot y_1}} \bigg)\1_{\tau(x_1)=s}\1_{\cE_1}\Bigg].
     \end{align*}
    Since at time $\tilde{T}$, we have $\attn^{(\tilde{T})}_{\ans,1\to\pred,1}-\attn^{(\tilde{T})}_{\ans,1\to\ans,1}\leq \attn^{(\tilde{T})}_{\ans,1\to\pred,2}-\attn^{(\tilde{T})}_{\ans,1\to\ans,1}\leq \frac{1.005\log d}{2B}$, we have $\Lambda^{(\tilde{T})}_{5,j^{\prime}_2,r_{g_2\cdot y_0}} - \Lambda^{(\tilde{T})}_{5,j_2,r_{g_2\cdot y_1}}\ll  -V_{j'_2,  r_{g_2\cdot y_0}}(y_1)$. Therefore,
    \begin{align*}
   \Big[-\nabla_{\Q^{(\tilde{T})}_{4,4}}\Loss_{5}^{2,2}\Big]_{s,s}+\Big[\nabla_{\Q^{(\tilde{T})}_{4,3}}\Loss_{5}^{2,2}\Big]_{s,s}\geq \Omega\Big(\frac{B}{d}\Big)\cdot \E\bigg[1-\logit^{(\tilde{T})}_{5,j_2}\big|\tau(x_1)=s\bigg],
    \end{align*}
    which implies that  
    $[\Qb_{4,4}]_{s,s}$ will grow faster than $[\Qb_{4,3}]_{s,s}$, and thus the attention gap cannot be further increased. 
\end{proof}

\subsubsection{At the End of Stage 1.2.1}
\begin{lemma}\label{lem-s21-end-non}
   For all iterations $t\leq T_{1,2,1,s}=O\Big(\frac{n_y d}{\eta\log d}\Big)+O\bigg(\frac{ d^{\frac{(0.1+c_1/2)B}{\log d}}}{\eta\log d }\bigg)$, we have \Cref{induction-s21-non} holds, and at time  $T_{1,2,1,s}$, we have 
   
    \begin{enumerate}[(a)]
        \item $[\Qb^{(T_{1,2,1,s})}_{4,3}]_{s,s},[\Qb^{(T_{1,2,1,s})}_{4,4}]_{s,s}\geq \Omega(1)$;
        \item other $|[\Qb^{(T_{2,1,s})}_{4,p}]_{s,s'}|$ for $p\in\{3,4\}$,  $s'\in\tau(\X)\not=s$ are at most  $\tilde{O}(\frac{1}{d})$.
    \end{enumerate}
\end{lemma}
\begin{proof}
    The existence of $T_{1,2,1,s}$ can be directly obtained by using \Cref{lem-s21-gd1-non,lem-s21-grad-2-non,lem-s21-attention-decrease-non,lem-s21-attention-gap-non}. Furthermore, $[\Qb^{(T_{1,2,1,s})}_{4,4}]_{s,s}\geq \Omega(1)$ can be guaranteed since \Cref{lem-s21-attention-gap-non} implies that $\attn^{(T_{1,2,1,s})}_{\ans,1\to\ans,1}\geq 0.3-\frac{c_1}{2}\pm\tilde{O}(1/d)>0.2$, which means that $[\Qb^{(T_{1,2,1,s})}_{4,4}]_{s,s}$ should at least grow to a constant level compared to $|[\Qb^{(T_{1,2,1,s})}_{4,p}]_{s,s'}|=\tilde{O}(1/d)$ with $p\in\{3, 4\}$.

    We will handle $[\Qb^{(T_{1,2,1,s})}_{4,3}]_{s,s}$ by means of a proof by contradiction. Suppose that
 $\attn^{(T_{1,2,1,s})}_{\ans,1\to\ans,1}\geq 0.4-\tilde{c}$, where $\tilde{c}=\frac{\log d}{8 B}$ is a sufficiently small constant. Then denote $\tilde{T}$ the first  time that 
$\E\bigg[\attn^{(t)}_{\ans,1\to\ans,1}\big|\tau(x_1)=s\bigg]\geq 0.4-2\tilde{c}$. Notice that $\left|\left[\mathbf{Q}_{4, p}^{(t)}\right]_{s, s^{\prime}}\right| \leq O\left(\frac{\left[\mathbf{Q}_{4, p}^{(t)}\right]_{s, s}}{d}\right)$ for $p\in \{3,4\}$, thus for any sample $\Zb^{2,1}$ satisfying $\tau(x_1)=s$, at time $\tilde{T}$, we have  
    \begin{align*}    \attn^{(\tilde{T})}_{\ans,1\to\ans,1}\geq 0.4-2\tilde{c}\pm \tilde{O}(1/d).
    \end{align*}
\begin{itemize}
    \item If $\attn^{(\tilde{T})}_{\ans,1\to\pred,2}-\attn^{(\tilde{T})}_{\ans,1\to\ans,0}\geq 2\varrho$, then 
    \begin{itemize}
    \item for $\Zb^{2,1}\in \cE_1$, $\Lambda^{(\tilde{T})}_{5,j_2,r_{g_2\cdot y_1}}$  has already been activated to the linear regime. Furthermore, $\attn^{(\tilde{T})}_{\ans,1\to\pred,2}-\attn^{(\tilde{T})}_{\ans,1\to\ans,0}\leq 0.2+2\tilde{c}-0.2\leq 2\tilde{c}$, which implies $1-\logit^{(1)}_{5,j_2}=1-o(1)$. Thus, by \Cref{lem-grad-decompositions-sym}, we have $$\E\bigg[ \Big(\cN^{(\tilde{T})}_{s,3,2,\rom 1}+\cN^{(\tilde{T})}_{s,3,2,\rom 2}+\cN^{(\tilde{T})}_{s,3,2,\rom 3}\Big)\1_{\cE_1}\bigg]\geq \Omega\Big(\frac{B}{d}\Big),$$ while $$\E\bigg[ \Big(\cN^{(\tilde{T})}_{s,4,2,\rom 1}+\cN^{(\tilde{T})}_{s,4,2,\rom 2}+\cN^{(\tilde{T})}_{s,4,2,\rom 3}\Big)\1_{\cE_1}\bigg]\leq- \Omega\Big(\frac{B}{d}\Big).$$
    \item for $\Zb^{2,1}\in \cE_2$, $\Lambda^{(\tilde{T})}_{5,j'_2,r_{g_2\cdot y_0}}=(\attn^{(\tilde{T})}_{\ans,1\to\pred,2}+\attn^{(\tilde{T})}_{\ans,1\to\pred,1}-\attn^{(\tilde{T})}_{\ans,1\to\ans,0})2B\leq (0.4+\frac{4}{3}\tilde{c}-0.4+2\tilde{c})2B=\frac{5\log d}{6}$. Thus $\logit^{(1)}_{5,j_2}=O(\frac{1}{d^{1/6}})$. Hence by \Cref{lem-grad-decompositions-sym}, we have 
    $$\E\bigg[ \Big(\cN^{(\tilde{T})}_{s,3,2,\rom 1}+\cN^{(\tilde{T})}_{s,3,2,\rom 2}+\cN^{(\tilde{T})}_{s,3,2,\rom 3}\Big)\1_{\cE_1}\bigg]\geq -O\Big(\frac{B}{d^{7/6}n_y}\Big),$$ while
   $$\E\bigg[ \Big(\cN^{(\tilde{T})}_{s,4,2,\rom 1}+\cN^{(\tilde{T})}_{s,4,2,\rom 2}+\cN^{(\tilde{T})}_{s,4,2,\rom 3}\Big)\1_{\cE_1}\bigg]\leq -\Omega\Big(\frac{B}{n_y d}\Big).$$ 
    \item $\Zb^{2,1}\in \cE_3$,  we can use the following naive bounds for $p\in\{3,4\}$:
     $$|\E\bigg[ \Big(\cN^{(\tilde{T})}_{s,p,2,\rom 1}+\cN^{(\tilde{T})}_{s,p,2,\rom 2}+\cN^{(\tilde{T})}_{s,p,2,\rom 3}\Big)\1_{\cE_1}\bigg]|\leq O\Big(\frac{B}{d n_y}\Big).$$
    \end{itemize}
    Putting them together, combining with the fact that the gradient contributed by $\cE_{4}\cup\cE_5\cup\cE_6$ is negligible, we can conclude that  
    \begin{align*}
        \Big[-\nabla_{\Q^{(\tilde{T})}_{4,3}}\Loss_{5}^{2,2}\Big]_{s,s} \geq \Omega(\frac{B}{d}), \Big[-\nabla_{\Q^{(\tilde{T})}_{4,4}}\Loss_{5}^{2,2}\Big]_{s,s}\leq -\Omega(\frac{B}{d}),
         \end{align*}
         which implies $\attn^{(\tilde{T})}_{\ans,1\to\ans,1}$ cannot further increase above $0.4-2\tilde{c}$. 
         \item If $\attn^{(\tilde{T})}_{\ans,1\to\pred,2}-\attn^{(\tilde{T})}_{\ans,1\to\ans,0}< 2\varrho$, we shift our focus to the comparison between  event $\cE_2$ $\cE_3$,
    \begin{itemize}
    \item for $\Zb^{2,1}\in \cE_1$, we have  
    $$\E\bigg[ \Big(\cN^{(\tilde{T})}_{s,3,2,\rom 1}+\cN^{(\tilde{T})}_{s,3,2,\rom 2}+\cN^{(\tilde{T})}_{s,3,2,\rom 3}\Big)\1_{\cE_1}\bigg]\geq -\tilde{O}(\sigma^q_0),$$ while $$\E\bigg[ \Big(\cN^{(\tilde{T})}_{s,4,2,\rom 1}+\cN^{(\tilde{T})}_{s,4,2,\rom 2}+\cN^{(\tilde{T})}_{s,4,2,\rom 3}\Big)\1_{\cE_1}\bigg]\leq \tilde{O}(\sigma^q_0).$$

    \item for $\Zb^{2,1}\in \cE_2$, similarly as previous case, we have 
    $$\E\bigg[ \Big(\cN^{(\tilde{T})}_{s,3,2,\rom 1}+\cN^{(\tilde{T})}_{s,3,2,\rom 2}+\cN^{(\tilde{T})}_{s,3,2,\rom 3}\Big)\1_{\cE_2}\bigg]\geq -O\Big(\frac{B}{d^{7/6}n_y}\Big),$$ 
    while
   $$\E\bigg[ \Big(\cN^{(\tilde{T})}_{s,4,2,\rom 1}+\cN^{(\tilde{T})}_{s,4,2,\rom 2}+\cN^{(\tilde{T})}_{s,4,2,\rom 3}\Big)\1_{\cE_2}\bigg]\leq O\Big(\frac{B}{d^{7/6}n_y}\Big).$$ 
    \item $\Zb^{2,1}\in \cE_3$, $\Big|\Lambda^{(\tilde{T})}_{5,j'_2,r_{g_1\cdot y_1}}-\Lambda^{(\tilde{T})}_{5,j_2,r_{g_2\cdot y_1}}\Big|\leq 4\rho B= o(1)$, hence, we have 
    $$\E\bigg[ \Big(\cN^{(\tilde{T})}_{s,3,2,\rom 1}+\cN^{(\tilde{T})}_{s,3,2,\rom 2}+\cN^{(\tilde{T})}_{s,3,2,\rom 3}\Big)\1_{\cE_3}\bigg]\geq \Omega\Big(\frac{B}{d n_y}\Big),$$ 
    while
   $$\E\bigg[ \Big(\cN^{(\tilde{T})}_{s,4,2,\rom 1}+\cN^{(\tilde{T})}_{s,4,2,\rom 2}+\cN^{(\tilde{T})}_{s,4,2,\rom 3}\Big)\1_{\cE_3}\bigg]\leq O\Big(\frac{B\log\log d}{d n_y\cdot \log d}\Big).$$ 
    \end{itemize}
    Putting them together, again we can conclude that  
    \begin{align*}
        \Big[-\nabla_{\Q^{(\tilde{T})}_{4,3}}\Loss_{5}^{2,2}\Big]_{s,s} \geq \Omega(\frac{B}{n_y d}), \Big[-\nabla_{\Q^{(\tilde{T})}_{4,4}}\Loss_{5}^{2,2}\Big]_{s,s}\leq O\Big(\frac{B\log\log d}{d n_y\cdot \log d}\Big),
         \end{align*}
         which implies $\attn^{(\tilde{T})}_{\ans,1\to\ans,1}$ cannot further increase above $0.4-2\tilde{c}$. 
\end{itemize}
Consequently, this leads to a contradiction, and $\attn^{(T_{1,2,1,s})}_{\ans,1\to\ans,1}< 0.4-\tilde{c}$, where $\tilde{c}=\frac{\log d}{8 B}$. Then it would follow that $\attn^{(T_{1,2,1,s})}_{\ans,1\to\pred,2}\geq \tilde{c}$, and thus $[\Q^{(T_{1,2,1,s})}_{4,3}]_{s,s}\geq\Omega(1)$. 
\end{proof}

\subsection{Stage 1.2.2: Convergence with Small Wrong Attention}
Recall that $\Loss^{2,2}_{5,s}=-\E\Big[\log p_{F}(\Z_{\ans,2,5}|\Z^{(2,1)})\big|\tau(x_1)=s\Big]$ for $s\in\tau(\X)$. 

\begin{induction}\label{induction-s22-non}
    Given $s\in\tau(\X)$,  let $T_{1,2,2, s}$ denote the first time that $\Loss^{2,2}_{5,s}$ decreases below  $\Theta\Big(e^{(-\frac{1}{2}+3.01c_1)\cdot 2B}\Big)$.
       For all iterations $T_{1,2,1,s}\leq t< T_{1,2,2,s}$, we have the following holds
       \begin{enumerate}[(a)]
          \item $\Big[\Qb^{(t)}_{4,3}\Big]_{s,s}+\Big[\Qb^{(t)}_{4,4}\Big]_{s,s}$ monotonically increases; 
          \item for $p\in\{3,4\}$, for $j\in\tau(\X)\not=s$, $|[\Qb^{(t)}_{4,p}]_{s,j}|\leq O(\frac{[\Qb^{(t)}_{4,p}]_{s,j}}{d})$ ; 
        \item for any sample $\Zb^{2,1}$, we have $\attn^{(t)}_{\ans,1\to \pred,2}-\attn^{(t)}_{\ans,1\to \ans,1}\leq c_1$ for some sufficiently small constant $c_1=\frac{1.005\log d}{B}>0$; 
        \item for any $\Zb^{2,1}$, we have $\attn^{(t)}_{\ans,1\to \ans,1}- \attn^{(t)}_{\ans,1\to \pred,2}\leq \min\Big\{\attn^{(t)}_{\ans,1\to \pred,1}-c_2,0\Big\}$, where $c_2=\frac{\log d}{4B}>0$ is some sufficiently small constant. 
       \end{enumerate}
    \end{induction}

\subsubsection{Attention and Lambda Preliminaries}
\begin{lemma}\label{lem-s22-attn-non}
    If \Cref{induction-s22-non} holds for all iterations $[T_{1,2,1,s},t)$,then for any sample $\Zb^{2,1}$, we have  
         \begin{enumerate}
            \item $\attn^{(t)}_{\ans,1\to \pred, 1}+\attn^{(t)}_{\ans,1\to \ans, 1}\in \Big[\frac{4}{3}c_1 ,0.4\pm \tilde{O}\big(\frac{1}{d}\big) \Big]$; 
            \item   $\attn^{(t)}_{\ans,1\to \pred,2}-\attn^{(t)}_{\ans,1\to \ans,0}\geq \Omega(1)$; 
   \item  $\big|\attn^{(t)}_{\ans,1\to \pred,1}- \attn^{(t)}_{\ans,1\to \ans,0}\big|\leq \tilde{O}\big(\frac{1}{d}\big)$.
            \end{enumerate} 

\end{lemma}

\begin{proof}
    In the following, we focus on the main events $\cE_1$, $\cE_2$, and $\cE_3$, which correspond to cases where some confused wrong predictions occur.  
    We denote 
    \[
    \attn^{(t)}_{\ans,1\to \pred,1}, \quad \attn^{(t)}_{\ans,1\to \ans,1} \;=\; c+\tilde{O}\Big(\tfrac{1}{d}\Big).
    \]
    
    \begin{itemize}
        \item If $\attn^{(t)}_{\ans,1\to \pred,2}\geq \attn^{(t)}_{\ans,1\to \ans,1}$, then
        \begin{itemize}
            \item If $\Zb^{2,1}\in \cE_{1}$, we have
            \begin{align*}
              &  \log p_{F}(\Z_{\ans,2,5}|\Z^{(2,1)})\\
                &\leq \Theta(1)\cdot \max\Bigg\{
                    e^{ \big(\attn^{(t)}_{\ans,1\to \pred,2}-\attn^{(t)}_{\ans,1\to \ans,1}\big)2B
                     - \big(\attn^{(t)}_{\ans,1\to \pred,2}-\attn^{(t)}_{\ans,1\to \ans,0}\big)2B}, \\
                &\qquad\qquad e^{\log d - \big(\attn^{(t)}_{\ans,1\to \pred,2}-\attn^{(t)}_{\ans,1\to \ans,0}\big)2B}
                \Bigg\}\\
                &\leq \Theta\!\left( e^{ 2c_1B - \big(\attn^{(t)}_{\ans,1\to \pred,2}-\attn^{(t)}_{\ans,1\to \ans,0}\big)2B} \right) \\
                &\leq \Theta\!\left( e^{-2\left(\tfrac{1}{2}-2c-c_1\right)B} \right),
            \end{align*}
            where the last inequality follows from the fact that $\attn^{(t)}_{\ans,1\to \pred,2}\geq \tfrac{1}{2}(1-2c)=\tfrac{1}{2}-c$.
    
            \item If $\Zb^{2,1}\in \cE_{2}$, we have
            \begin{align*}
                &\log p_{F}(\Z_{\ans,2,5}|\Z^{(2,1)})\leq \Theta(1)\\
                &\cdot \max\Bigg\{
                    e^{ \big(\attn^{(t)}_{\ans,1\to \pred,2}+\attn^{(t)}_{\ans,1\to \pred,1}-\attn^{(t)}_{\ans,1\to \ans,1}\big)B
                     - \big(\attn^{(t)}_{\ans,1\to \pred,2}-\attn^{(t)}_{\ans,1\to \ans,0}\big)2B}, \\
                &\qquad\qquad e^{\log d - \big(\attn^{(t)}_{\ans,1\to \pred,2}-\attn^{(t)}_{\ans,1\to \ans,0}\big)2B}
                \Bigg\}\\
                &\leq \Theta\!\left( e^{2(c+c_1)B - \big(\attn^{(t)}_{\ans,1\to \pred,2}-\attn^{(t)}_{\ans,1\to \ans,0}\big)2B} \right) \\
                &\leq \Theta\!\left( e^{-2\left(\tfrac{1}{2}-3c-c_1\right)B} \right).
            \end{align*}
    
            \item If $\Zb^{2,1}\in \cE_{3}$, we have
            \begin{align*}
                &\log p_{F}(\Z_{\ans,2,5}|\Z^{(2,1)})\\
                &\leq \Theta(1)\cdot \max\Bigg\{
                    e^{ \big(\attn^{(t)}_{\ans,1\to \pred,1}-\attn^{(t)}_{\ans,1\to \pred,2}\big)2B},
                    e^{\log d - \attn^{(t)}_{\ans,1\to \pred,2}2B}
                \Bigg\} \\
                &\leq \Theta(1)\cdot \max\Big\{ e^{-2(\tfrac{1}{2}-c_1-c)B}, \; e^{-2(\tfrac{1}{2}-2c)B} \Big\}.
            \end{align*}
        \end{itemize}
    
        \item If $\attn^{(t)}_{\ans,1\to \pred,2}\leq \attn^{(t)}_{\ans,1\to \ans,1}$, then
        \begin{itemize}
            \item If $\Zb^{2,1}\in \cE_{1}$, we have
            \begin{align*}
                \log p_{F}(\Z_{\ans,2,5}|\Z^{(2,1)})
                &\leq \Theta(1)\cdot e^{\log d - \big(\attn^{(t)}_{\ans,1\to \pred,2}-\attn^{(t)}_{\ans,1\to \ans,0}\big)2B} \\
                &\leq \Theta\!\left( e^{2c_1 B - \big(\attn^{(t)}_{\ans,1\to \pred,2}-\attn^{(t)}_{\ans,1\to \ans,0}\big)2B} \right) \\
                &\leq \Theta\!\left( e^{-2\left(\tfrac{1}{2}-\tfrac{5}{2}c-\tfrac{c_2}{2}-c_1\right)B} \right),
            \end{align*}
            where the last inequality uses the fact that 
            $\attn^{(t)}_{\ans,1\to \ans,1}-\attn^{(t)}_{\ans,1\to \pred,2}\leq c-c_2$, implying 
            $\attn^{(t)}_{\ans,1\to \pred,2}\geq \tfrac{1}{2}-\tfrac{3}{2}c+\tfrac{c_2}{2}$.
    
            \item If $\Zb^{2,1}\in \cE_{2}$, we have
            \begin{align*}
                &\log p_{F}(\Z_{\ans,2,5}|\Z^{(2,1)})\leq \Theta(1)\cdot \\
                &\max\Bigg\{
                    e^{ \big(\attn^{(t)}_{\ans,1\to \pred,2}+\attn^{(t)}_{\ans,1\to \pred,1}-\attn^{(t)}_{\ans,1\to \ans,1}\big)2B
                     - \big(\attn^{(t)}_{\ans,1\to \pred,2}-\attn^{(t)}_{\ans,1\to \ans,0}\big)2B}, \\
                &\qquad\qquad e^{\log d - \big(\attn^{(t)}_{\ans,1\to \pred,2}-\attn^{(t)}_{\ans,1\to \ans,0}\big)2B}
                \Bigg\}\\
                &\leq \Theta\!\left( e^{\max\{c,c_1\} 2B - \big(\attn^{(t)}_{\ans,1\to \pred,2}-\attn^{(t)}_{\ans,1\to \ans,0}\big)2B} \right) \\
                &\leq \Theta\!\left( e^{-2\left(\tfrac{1}{2}-2c-\max\{c,c_1\}\right)B} \right).
            \end{align*}
    
            \item If $\Zb^{2,1}\in \cE_{3}$, we have
            \begin{align*}
                &\log p_{F}(\Z_{\ans,2,5}|\Z^{(2,1)})\\
                &\leq \Theta(1)\cdot \max\Bigg\{
                    e^{ \big(\attn^{(t)}_{\ans,1\to \pred,1}-\attn^{(t)}_{\ans,1\to \pred,2}\big)2B},
                    e^{\log d - \attn^{(t)}_{\ans,1\to \pred,2}2B}
                \Bigg\}\\
                &\leq \Theta(1)\cdot \max\Big\{ e^{-2(\tfrac{1}{2}-c_1-c)B}, \; e^{-2(\tfrac{1}{2}-2c)B} \Big\}.
            \end{align*}
        \end{itemize}
    \end{itemize}
    
    Putting all cases together, if $c\leq \tfrac{2}{3}c_1$, we must have
    \[
    \Loss_{5,s}^{2,2}\;\leq\; \Theta\!\left(e^{-2(\tfrac{1}{2} + 3\cdot\tfrac{2c_1}{3} + c_1)B}\right)
    = \Theta\!\left(e^{(-\tfrac{1}{2} + 3c_1)2B}\right),
    \]
    which contradicts the definition of $T_{1,2,2,s}$.  
    Therefore, it must hold that $c \geq \tfrac{2}{3}c_1$ for all $t \leq T_{1,2,2,s}$.
    
\end{proof}

\begin{lemma}\label{lem-s22-lambda-1-sym}
    If \Cref{induction-s22-non} holds for all iterations $[T_{1,2,1,s},t)$, then given $\Zb^{2,1}\in \cE_{1}$, 
\begin{enumerate}
    \item for the prediction $j_2$, 
    we have 
    \begin{align*}
       & \Lambda^{(t)}_{5,j_2,r_{g_2\cdot y_1}}= \Big(\attn^{(t)}_{\ans,1\to \pred,2}-\attn^{(t)}_{\ans,1\to \ans,0}\Big) \cdot 2B + O\Big(\frac{B}{n_y}\Big)+ O(\delta).
    \end{align*}
    \item for the prediction $j'_2=\tau\big(g_2(y_0)\big)$, 
    we have 
\begin{align*}
    &\Lambda^{(t)}_{5,j'_2,r_{g_2\cdot y_0}}
     = \Big(\attn^{(t)}_{\ans,1\to \pred,2}-\attn^{(t)}_{\ans,1\to \ans,1}\Big) \cdot 2B +\tilde{O}\Big(\frac{B}{d\cdot n_y}\Big)+ O(\delta).
\end{align*}
\item for the prediction $\tau(g_1(y_0))$, 
$r_{g_1\cdot y_0}$ cannot be activated.

\item for the prediction $\tau(g_1(y_1))$, 
we have
\begin{align*}
    \Lambda^{(t)}_{5,\tau(g_1(y_1)),r_{g_1\cdot y_1}} =  \Big(\attn^{(t)}_{\ans,1\to \ans,1}-\attn^{(t)}_{\ans,1\to \pred,2}\Big) \cdot \frac{2B}{n_y}+\tilde{O}\Big(\frac{B}{d}\Big)+O(\delta).
\end{align*}

\item for other $j\in\tau(\cY)$, if there are $j$  and $y$, s.t.,  $g_1, g_2\in \fiber_{j,y}$ (notice that $y\neq y_0, y_1$ for $\cE_{1}$),  then for such a $j$, $r\in\hat{\fA}_{j}\setminus \{r_{g_2\cdot y}\}$ cannot be activated, moreover, we have 
\begin{align*}
    &\Lambda^{(t)}_{5,j,r_{g_2\cdot y}}
     = \Big(\attn^{(t)}_{\ans,1\to \pred,2}-\attn^{(t)}_{\ans,1\to \ans,1}\Big) \cdot 2B +\tilde{O}\Big(\frac{B}{d\cdot n_y}\Big)+ O(\delta);
\end{align*}
else,  none of $r\in\hat{\fA}_{j}$ can be activated.
\end{enumerate}
Notice that for $\cE_{1}$, the neurons mentioned in 1-4 should be different neurons, while the predictions in 1 and 3, or 2 and 4 can be the same in some cases, e.g., $g_1(y_1)=g_2(y_0)$. In all cases, except for the neurons mentioned above, i.e., $\cup_{\ell,\ell'\in [2]} \{r_{g_{\ell}\cdot y_{\ell'-1}}\}$ (which may not be activated), all other neurons $r\in  \cup_{\ell,\ell'\in [2]} \Big(\hat{\fA}_{\tau(g_{\ell}\cdot y_{\ell'-1})}\setminus  \{r_{g_{\ell}\cdot y_{\ell'-1}}\}\Big)$ cannot be activated.
\end{lemma}

\begin{lemma}\label{lem-s22-lambda-2-sym}
    If \Cref{induction-s22-non} holds for all iterations $t\in [T_{1,2,1,s}, T_{1,2,2,s})$, then given $\Zb^{2,1}\in \cE_{2}$, we have
\begin{enumerate}
    \item for the prediction $j_2$, $r\in\hat{\fA}_{j_2}\setminus \{r_{g_2\cdot y_1}\}$  cannot be activated, moreover, we have 
    \begin{align*}
        & \Lambda^{(t)}_{5,j_2,r_{g_2\cdot y_1}}= \Big(\attn^{(t)}_{\ans,1\to \pred,2}-\attn^{(t)}_{\ans,1\to \ans,0}\Big) \cdot 2B +\tilde{O} \Big(\frac{B}{n_y}\Big)+ O(\delta).
     \end{align*}
    \item for the prediction $j'_2=\tau\big(g_2(y_0)\big)$, $r\in\hat{\fA}_{j'_2}\setminus \{r_{g_2\cdot y_0}\}$ (notice that in this case $r_{g_2\cdot y_0}=r_{g_1\cdot y_0}$) cannot be activated, moreover, we have 
\begin{align*}
    \Lambda^{(t)}_{5,j'_2,r_{g_2\cdot y_0}} = \Big(\attn^{(t)}_{\ans,1\to \pred,2}+\attn^{(t)}_{\ans,1\to \pred,1}-\attn^{(t)}_{\ans,1\to \ans,1}\Big) \cdot 2B + O(\delta).
\end{align*}

\item for the prediction $\tau(g_1(y_1))$, $r\in\hat{\fA}_{\tau(g_1(y_1))}\setminus \{r_{g_1\cdot y_1}\}$ cannot be activated, moreover, we have
\begin{align*}
    \Lambda^{(t)}_{5,\tau(g_1(y_1)),r_{g_1\cdot y_1}} =  \Big(\attn^{(t)}_{\ans,1\to \ans,1}-\attn^{(t)}_{\ans,1\to \pred,2}\Big) \cdot \frac{2B}{n_y}+O(\delta).
\end{align*}
\item for other $j\in\tau(\cY)$, if there are  $j$  and $y$, s.t.,  $g_1, g_2\in \fiber_{j,y}$ (notice that $y\neq y_0, y_1$ for $\cE_{2}$),  then for such a $j$, $r\in\hat{\fA}_{j}\setminus \{r_{g_2\cdot y}\}$ cannot be activated, moreover, we have 
\begin{align*}
    &\Lambda^{(t)}_{5,j,r_{g_2\cdot y}}
     = \Big(\attn^{(t)}_{\ans,1\to \pred,2}-\attn^{(t)}_{\ans,1\to \ans,1}\Big) \cdot 2B +\tilde{O}\Big(\frac{B}{d\cdot n_y}\Big)+ O(\delta);
\end{align*}
else,  none of $r\in\hat{\fA}_{j}$ can be activated.
\end{enumerate}
\end{lemma}

\begin{lemma}\label{lem-s22-lambda-3-sym}
    If \Cref{induction-s22-non} holds for all iterations $t\in [T_{1,2,1,s}, T_{1,2,2,s})$, then given $\Zb^{2,1}\in \cE_{3}$, we have
\begin{enumerate}
    \item for the prediction $j_2$, $r\in\hat{\fA}_{j_2}\setminus \{r_{g_2\cdot y_1}\}$  cannot be activated, moreover, we have 
    \begin{align*}
        \Lambda^{(t)}_{5,j_2,r_{g_2\cdot y_1}} = \attn^{(t)}_{\ans,1\to \pred,2}\cdot 2B \pm  O(\delta).
    \end{align*}

\item for the prediction $\tau(g_1(y_1))$, $r\in\hat{\fA}_{\tau(g_1(y_1))}\setminus \{r_{g_1\cdot y_1}\}$ cannot be activated, moreover, we have
\begin{align*}
    \Lambda^{(t)}_{5,\tau(g_1(y_1)),r_{g_1\cdot y_1}} = \attn^{(t)}_{\ans,1\to \pred,1}\cdot 2B \pm  O(\delta).
\end{align*}
\item for other  $j=g(y_1)$, where $g_1,g_2\notin\fiber_{j,y_1}$, we have 
\begin{align*}
    &\Lambda^{(t)}_{5,j,r_{g\cdot y_1}}
     = \Big(\attn^{(t)}_{\ans,1\to \ans,1}-\attn^{(t)}_{\ans,1\to \pred,2}\Big) \cdot \frac{2B}{n_y} +\tilde{O}\Big(\frac{B}{d\cdot n_y}\Big)+ O(\delta);
\end{align*}
moreover, if there exists  $y$, s.t.,  $g_1, g_2\in \fiber_{j,y}$ (notice that $y\neq y_1$ for $\cE_{3}$),  we have 
\begin{align*}
    &\Lambda^{(t)}_{5,j,r_{g_2\cdot y}}
     = \Big(\attn^{(t)}_{\ans,1\to \pred,2}-\attn^{(t)}_{\ans,1\to \ans,1}\Big) \cdot 2B +\tilde{O}\Big(\frac{B}{d\cdot n_y}\Big)+ O(\delta);
\end{align*}
besides,  other $r\in\hat{\fA}_{j}$ cannot be activated.
\end{enumerate}
\end{lemma}

\begin{lemma}\label{lem-s22-lambda-4-sym}
    If \Cref{induction-s22-non} holds for all iterations $t\in [T_{1,2,1,s}, T_{1,2,2,s})$, then given $\Zb^{2,1}\in \cE_{4}$, we have
\begin{enumerate}
    \item for the prediction $j_2$, $r\in\hat{\fA}_{j_2}\setminus \{r_{g_2\cdot y_1}\}$  cannot be activated, moreover, we have 
    \begin{align*}
        \Lambda^{(t)}_{5,j_2,r_{g_2\cdot y_1}} = \attn^{(t)}_{\ans,1\to \pred,2} \cdot 2B + O(\delta). 
    \end{align*}

\item for the prediction $\tau(g_2(y_0))$, $r\in\hat{\fA}_{\tau(g_2(y_0))}\setminus \{r_{g_2\cdot y_0}\}$ cannot be activated, moreover, we have
\begin{align*}
    \Lambda^{(t)}_{5,j'_2,r_{g_2\cdot y_0}} = \Big(\attn^{(t)}_{\ans,1\to \pred,2}+\attn^{(t)}_{\ans,1\to \pred,1}-\attn^{(t)}_{\ans,1\to \ans,1}\Big) \cdot 2B + O(\delta).
\end{align*}


\item for other $j\in\tau(\cY)$, if there are  $j$  and $y$, s.t.,  $g_1, g_2\in \fiber_{j,y}$,  then for such a $j$, $r\in\hat{\fA}_{j}\setminus \{r_{g_2\cdot y}\}$ cannot be activated, moreover, we have 
\begin{align*}
    &\Lambda^{(t)}_{5,j,r_{g_2\cdot y}}
     = \Big(\attn^{(t)}_{\ans,1\to \pred,2}-\attn^{(t)}_{\ans,1\to \ans,1}\Big) \cdot 2B +\tilde{O}\Big(\frac{B}{d\cdot n_y}\Big)+ O(\delta);
\end{align*}
else,  none of $r\in\hat{\fA}_{j}$ can be activated.
\end{enumerate}
\end{lemma}

\begin{lemma}\label{lem-s22-lambda-5-sym}
    If \Cref{induction-s22-non} holds for all iterations $t\in [T_{1,2,1,s}, T_{1,2,2,s})$,  then given $\Zb^{2,1}\in \cE_{5}$, we have
\begin{enumerate}
    \item for the prediction $j_2$, $r\in\hat{\fA}_{j_2}\setminus \{r_{g_2\cdot y_1}\}$  cannot be activated, moreover, we have 
    \begin{align*}
        \Lambda^{(t)}_{5,j_2,r_{g_2\cdot y_1}} = \attn^{(t)}_{\ans,1\to \pred,2}\cdot 2B \pm  O(\delta).
    \end{align*}

\item for the prediction $\tau(g_2(y_0))$, $r\in\hat{\fA}_{\tau(g_2(y_0))}\setminus \{r_{g_2\cdot y_0}\}$ cannot be activated, moreover, we have
\begin{align*}
    \Lambda^{(t)}_{5,j'_2,r_{g_2\cdot y_0}} = \Big(\attn^{(t)}_{\ans,1\to \pred,2}-\attn^{(t)}_{\ans,1\to \ans,1}\Big) \cdot 2B + \tilde{O}\Big(\frac{B}{d\cdot n_y}\Big)+ O(\delta).
\end{align*}
\item for the prediction $\tau(g_1(y_0))$, none of $r\in\hat{\fA}_{\tau(g_2(y_0))}$ can be activated.

\item for other $j\in\tau(\cY)$, if there are $j$  and $y$, s.t.,  $g_1, g_2\in \fiber_{j,y}$,  then for such a $j$, $r\in\hat{\fA}_{j}\setminus \{r_{g_2\cdot y}\}$ cannot be activated, moreover, we have 
\begin{align*}
    &\Lambda^{(t)}_{5,j,r_{g_2\cdot y}}
     = \Big(\attn^{(t)}_{\ans,1\to \pred,2}-\attn^{(t)}_{\ans,1\to \ans,1}\Big) \cdot 2B +\tilde{O}\Big(\frac{B}{d\cdot n_y}\Big)+ O(\delta);
\end{align*}
else,  none of $r\in\hat{\fA}_{j}$ can be activated.
\end{enumerate}
\end{lemma}

\begin{lemma}\label{lem-s22-lambda-6-sym}
    If \Cref{induction-s22-non} holds for all iterations $t\in [T_{1,2,1,s}, T_{1,2,2,s})$, then given $\Zb^{2,1}\in \cE_{6}$, we have
\begin{enumerate}
    \item for the prediction $j_2$, $r\in\hat{\fA}_{j_2}\setminus \{r_{g_2\cdot y_1}\}$  cannot be activated, moreover, we have 
    \begin{align*}
        \Lambda^{(t)}_{5,j_2,r_{g_2\cdot y_1}} = \Big(\attn^{(t)}_{\ans,1\to \pred,2}+\attn^{(t)}_{\ans,1\to \pred,1}\Big)\cdot 2B \pm  O(\delta).
    \end{align*}
\item  for other  $j=g(y_1)$, where $g_2\notin\fiber_{j,y_1}$, we have  $r_{g\cdot y_1}$ cannot be activated, 
moreover, if there exists  $y\neq y_1$, s.t.,  $g_1, g_2\in \fiber_{j,y}$,  we have 
\begin{align*}
    &\Lambda^{(t)}_{5,j,r_{g_2\cdot y}}
     = \Big(\attn^{(t)}_{\ans,1\to \pred,2}-\attn^{(t)}_{\ans,1\to \ans,1}\Big) \cdot 2B +\tilde{O}\Big(\frac{B}{d\cdot n_y}\Big)+ O(\delta);
\end{align*}
besides,  other $r\in\hat{\fA}_{j}$ cannot be activated.
\end{enumerate}

\end{lemma}

\subsubsection{Gradient Lemma}
\begin{lemma}\label{lem-s22-gd1-non}
    If \Cref{induction-s22-non} holds for all iterations $t\in[T_{1,2,1,s}, T_{1,2,2,s})$,  given $s\in\tau(\X)$, if  $\E\big[\attn^{(t)}_{{\ans,1} \rightarrow \pred,2} \mid \tau(x_1)=s\big]\leq \frac{1}{2}$  we have
    \begin{align*}
        &\Big[-\nabla_{\Q^{(t)}_{4,3}}{\Loss^{2,2}_{5}}\Big]_{s,s}+ \Big[-\nabla_{\Q^{(t)}_{4,4}}{\Loss^{2,2}_{5}}\Big]_{s,s} \\
        &\geq \Omega\Big(\frac{B}{d}\Big)\cdot \E\bigg[(1-\logit^{(t)}_{5,j_2})\big| \tau(x_1)=s, \cE_1\bigg]+ \Omega\Big(\frac{B}{n_y d}\Big)\cdot\sum_{m=2}^3 \E\bigg[(1-\logit^{(t)}_{5,j_2})\big| \tau(x_1)=s, \cE_m\bigg].
    \end{align*}
    

\end{lemma}
\begin{proof}  
    The analysis follows the similar idea as \Cref{lem-s21-gd1-non}, while \Cref{induction-s22-non} ensures that $\E\big[\attn^{(t)}_{{\ans,1} \rightarrow \pred,2}-\attn^{(t)}_{{\ans,1} \rightarrow \ans,0}\mid \tau(x_1)=s\big]\geq \Omega(1)>\frac{\varrho}{B}$, which falls into the regime (ii).  Thus
    for the main event $\cE_1$, the term $\Lambda^{(t)}{5,j_2, r{g_2\cdot y_1}}$ is guaranteed to remain within the linear regime. 
  We can first directly upper bound the term $\cN^{(t)}_{s,2,\rom3}$ by $\tilde{O}(\sigma_0^q)$. 

    \begin{enumerate}[(a)]
        \item 
        For $\Zb^{2,1}\in \cE_{1}$, 
       by \Cref{induction-s22-non} and \Cref{lem-s22-lambda-1-sym}, we can obtain 
        \begin{align}
            &\E\bigg[ \cN^{(t)}_{s,2,\rom1} \1_{\cE_1}\bigg]\notag\\
            &~~=\E\bigg[  (1-\logit^{(t)}_{5,j_2})\cdot \Big(  
\Big( -\attn^{(t)}_{{\ans,1} \rightarrow \ans,0} \cdot V_{j_2,  r_{g_2\cdot y_1}}(y_0)-\attn^{(t)}_{{\ans,1} \rightarrow \pred,1} \cdot V_{j_2,  r_{g_2\cdot y_1}}(g_1)\notag\\
               &~~~+\big(1-\attn^{(t)}_{{\ans,1} \rightarrow \ans,1}-\attn^{(t)}_{{\ans,1} \rightarrow \pred,2}\big) \Lambda^{(t)}_{5,j_2, r_{g_2\cdot y_1}}\pm\tilde{O}(\sigma_0) \Big)\pm \tilde{O}(\delta^{q}) \Big)\1_{\tau(x_1)=s}\1_{\cE_1}\bigg]\label{eq-s22-gd1-non-2} \\ 
             &  \geq   \Omega\Big(\frac{B}{d}\Big)\cdot  \E\bigg[ (1-\logit^{(t)}_{5,j_2})\big| \tau(x_1)=s, \cE_1\bigg]. \notag
        \end{align}
             Moving to $\cN^{(t)}_{s,2,\rom2}$, by \Cref{lem-s22-lambda-1-sym}, we have
             \begin{itemize}
                  \item for $j=\tau(g_1(y_1))$, $r=r_{g_1\cdot y_1}$, due to the cancellation of $\attn^{(t)}_{{\ans,1} \rightarrow \ans,0} \cdot V_{j,  r}(y_0)+\attn^{(t)}_{{\ans,1} \rightarrow \pred,1} \cdot V_{j,  r}(g_1)$, we have
    \begin{align*}
                    &\Big( \attn^{(t)}_{{\ans,1} \rightarrow \ans,0} \cdot V_{j,  r}(y_0)+\attn^{(t)}_{{\ans,1} \rightarrow \pred,1} \cdot V_{j,  r}(g_1)\\
                                          &~~~~~~~-\big(1-\attn^{(t)}_{{\ans,1} \rightarrow \ans,1}-\attn^{(t)}_{{\ans,1} \rightarrow \pred,2}\big) \Lambda^{(t)}_{5,j, r}\pm\tilde{O}(\sigma_0) \Big)\\
                                          & \geq -\big(1-\attn^{(t)}_{{\ans,1} \rightarrow \ans,1}-\attn^{(t)}_{{\ans,1} \rightarrow \pred,2}\big) \Lambda^{(t)}_{5,j, r}.
                                      \end{align*}
                                      \item for $j=\tau(g_2(y_0))$, $r=r_{g_2\cdot y_0}$
                                      \begin{align*}
                      &\Big( \attn^{(t)}_{{\ans,1} \rightarrow \ans,0} \cdot V_{j,  r}(y_0)+\attn^{(t)}_{{\ans,1} \rightarrow \pred,1} \cdot V_{j,  r}(g_1)\\
                                            &~~~~~~~-\big(1-\attn^{(t)}_{{\ans,1} \rightarrow \ans,1}-\attn^{(t)}_{{\ans,1} \rightarrow \pred,2}\big) \Lambda^{(t)}_{5,j, r}\pm\tilde{O}(\sigma_0) \Big) \\
                                         &   \geq -\big(1-\attn^{(t)}_{{\ans,1} \rightarrow \ans,1}-\attn^{(t)}_{{\ans,1} \rightarrow \pred,2}\big) \Lambda^{(t)}_{5,j, r}.
                                       \end{align*}
                                        \item for $j=g_1(y)$, where $\exists y\neq y_0, y_1$, s.t., $g_{2}(y)=g_1(y)$, we have 
                \begin{align*}
                                            &\Big( \attn^{(t)}_{{\ans,1} \rightarrow \ans,0} \cdot V_{j,  r}(y_0)+\attn^{(t)}_{{\ans,1} \rightarrow \pred,1} \cdot V_{j,  r}(g_1)\\
                                                                  &~~~~~~~-\big(1-\attn^{(t)}_{{\ans,1} \rightarrow \ans,1}-\attn^{(t)}_{{\ans,1} \rightarrow \pred,2}\big) \Lambda^{(t)}_{5,j, r}\pm\tilde{O}(\sigma_0) \Big)\\
                                                                  & \geq -\big(1-\attn^{(t)}_{{\ans,1} \rightarrow \ans,1}-\attn^{(t)}_{{\ans,1} \rightarrow \pred,2}\big) \Lambda^{(t)}_{5,j, r}.
                                                              \end{align*}
             \end{itemize}
             Putting them together, and upper bound $\sum_{j\neq j_2\in\tau(\cY)}\logit^{(t)}_{5,j} $ by $1-\logit^{(t)}_{5,j_2}$, we have 
             \begin{align*}
                & \E\bigg[ \cN^{(t)}_{s,2,\rom2} \1_{\cE_1}\bigg]\\
                &= \E\bigg[  \sum_{j\neq j_2\in\tau(\cY)}\logit^{(t)}_{5,j} \cdot \Big(\sum_{r\in\hat{\fA}_{j}}\ReLU^{\prime}(\Lambda_{5,j,r
                  })\cdot  \\
                 &~~~~~~~~~~\Big( \attn^{(t)}_{{\ans,1} \rightarrow \ans,0} \cdot V_{j,  r}(y_0)+\attn^{(t)}_{{\ans,1} \rightarrow \pred,1} \cdot V_{j,  r}(g_1)\\
                  &~~~~~~~-\big(1-\attn^{(t)}_{{\ans,1} \rightarrow \ans,1}-\attn^{(t)}_{{\ans,1} \rightarrow \pred,2}\big) \Lambda^{(t)}_{5,j, r}\pm\tilde{O}(\sigma_0) \Big)\pm \tilde{O}(\delta^{q}) \Big)\1_{\tau(x_1)=s}\1_{\cE_1}\bigg]\\
                  &\geq -\E\bigg[\Big(1-\logit^{(t)}_{5,j_2} \Big)\cdot \big(1-\attn^{(t)}_{{\ans,1} \rightarrow \ans,1}-\attn^{(t)}_{{\ans,1} \rightarrow \pred,2}\big)\\
                  &~~~\cdot \max_{y\neq y_0,y_1} \Big\{ \Lambda^{(t)}_{5,\tau(g_1(y_1)), r_{g_1\cdot y_1}}, \Lambda^{(t)}_{5,\tau(g_2(y_0)), r_{g_2\cdot y_0}}\Big\}\1_{\tau(x_1)=s}\1_{\cE_1}\bigg],
              \end{align*}
              where the last inequality is due to the fact that $\Lambda^{(t)}_{5,\tau(g_2(y_0)), r_{g_2\cdot y_0}}=\Lambda^{(t)}_{5,\tau(g_2(y)), r_{g_2\cdot y}}\pm \tilde{O}(\frac{B}{d\cdot n_y})$ for $y\neq y_0, y_1$ s.t., $g_1(y)=g_2(y)$.
              Notice that 
              \begin{align*}
                 \max_{y\neq y_0,y_1} \Big\{ \Lambda^{(t)}_{5,\tau(g_1(y_1)), r_{g_1\cdot y_1}}, \Lambda^{(t)}_{5,\tau(g_2(y_0)), r_{g_2\cdot y_0}}\Big\}\leq 
             \\
             \max\bigg\{\attn^{(t)}_{{\ans,1} \rightarrow \pred,2}-\attn^{(t)}_{{\ans,1} \rightarrow \ans,1}, \Theta(\frac{1}{n_y})\bigg\}2B,
              \end{align*}
     while 
     \begin{align*}
         \eqref{eq-s22-gd1-non-2}&\geq \E\bigg[\Big(1-\logit^{(t)}_{5,j_2} \Big)\cdot \attn^{(t)}_{{\ans,1} \rightarrow \ans,0}\cdot 2B\1_{\tau(x_1)=s}\1_{\cE_1}\bigg]\\
         &=  \E\bigg[\Big(1-\logit^{(t)}_{5,j_2} \Big)\cdot \big(1-\attn^{(t)}_{{\ans,1} \rightarrow \ans,1}-\attn^{(t)}_{{\ans,1} \rightarrow \pred,2}\big)B\1_{\tau(x_1)=s}\1_{\cE_1}\bigg].
     \end{align*}
     
     Since by \Cref{induction-s22-non}, $\attn^{(t)}_{{\ans,1} \rightarrow \pred,2}-\attn^{(t)}_{{\ans,1} \rightarrow \ans,1}\leq c_1\ll \frac{1}{2}$, thus we have 
     \begin{align*}
         \E\bigg[ \cN^{(t)}_{s,2,\rom1} \1_{\cE_1}\bigg]+\E\bigg[ \cN^{(t)}_{s,2,\rom2} \1_{\cE_1}\bigg]\geq  \Omega\Big(\frac{B}{d}\Big)\cdot  \E\bigg[ (1-\logit^{(t)}_{5,j_2})\big| \tau(x_1)=s, \cE_1\bigg]. 
     \end{align*}
         \item For $\Zb^{2,1}\in \cE_{2}$, 
          $\cN^{(t)}_{s,2,\rom1}$ can be bounded analogously to \eqref{eq-s22-gd1-non-2}, yielding
        \begin{align*}
            \E\bigg[ \cN^{(t)}_{s,2,\rom1} \1_{\cE_2}\bigg]\geq  \Omega\Big(\frac{B}{d\cdot n_y}\Big)\cdot  \E\bigg[(1-\logit^{(t)}_{5,j_2})\big| \tau(x_1)=s, \cE_2\bigg]\geq 0.
        \end{align*}
        Moving to $\cN^{(t)}_{s,2,\rom2}$,
        \begin{itemize}
            \item For $j=\tau(g_1(y_1))$, we obtain
            \begin{align}
                &= \E\Bigg[  \logit_{5,j}^{(t)}\Big(\ReLU^{\prime}(\Lambda_{5,j,r_{g_1\cdot y_1}})
                  \cdot \notag \\
                 &~~~\Big( \attn^{(t)}_{{\ans,1} \rightarrow \ans,0}\cdot V_{j,  r_{g_1\cdot y_1}}(y_0)
                 +\attn^{(t)}_{{\ans,1} \rightarrow \pred,1}\cdot V_{j,  r_{g_1\cdot y_1}}(g_1)\label{eq-s22-gd1-non-3}\\
                  &~~-\big(1-\attn^{(t)}_{{\ans,1}\rightarrow \ans,1}
                  -\attn^{(t)}_{{\ans,1}\rightarrow \pred,2}\big)\Lambda^{(t)}_{5,j, r}
                  \pm\tilde{O}(\sigma_0)\Big)\pm \tilde{O}(\delta^{q}) \Big)
                  \1_{\tau(x_1)=s}\1_{\cE_2}\Bigg] \notag\\
                  &\geq -O\!\left(\frac{B}{n^2_y\cdot d}\right)
                  \cdot \E\Big[(1-\logit^{(t)}_{5,j_2})\,\big|\, \tau(x_1)=s, \cE_2\Big] \notag,
            \end{align}
            where the inequality follows from the cancellation of the term 
            \[
              \attn^{(t)}_{{\ans,1} \rightarrow \ans,0}\cdot V_{j,  r_{g_1\cdot y_1}}(y_0)
              +\attn^{(t)}_{{\ans,1} \rightarrow \pred,1}\cdot V_{j,  r_{g_1\cdot y_1}}(g_1),
            \]
            together with the fact that
            \begin{align*}
                \Lambda^{(t)}_{5,\tau(g_1(y_1)),r_{g_1\cdot y_1}}
                = \Big(\attn^{(t)}_{\ans,1\to \ans,1}-\attn^{(t)}_{\ans,1\to \pred,2}\Big)\cdot \frac{2B}{n_y}
                +O(\delta)\;\leq\; O\!\left(\frac{B}{n_y}\right).
            \end{align*}
        
            \item For $j'_2=\tau(g_2(y_0))=\tau(g_1(y_0))$, it is clear that
            \begin{align*}
                  & \E\Bigg[  \logit^{(t)}_{5,j'_2} \Big(\ReLU^{\prime}(\Lambda^{(t)}_{5,j'_2,r_{g_1\cdot y_0}})
                    \cdot  \\
                   &~~~\Big( \attn^{(t)}_{{\ans,1}\rightarrow \ans,0}\cdot V_{j'_2,  r_{g_1\cdot y_0}}(y_0)
                   +\attn^{(t)}_{{\ans,1}\rightarrow \pred,1}\cdot V_{j'_2,  r_{g_1\cdot y_0}}(g_1)\\
                    &-\big(1-\attn^{(t)}_{{\ans,1}\rightarrow \ans,1}
                    -\attn^{(t)}_{{\ans,1}\rightarrow \pred,2}\big)\Lambda^{(t)}_{5,j'_2, r_{g_1\cdot y_0}}
                    \pm\tilde{O}(\sigma_0)\Big)\pm \tilde{O}(\delta^{q}) \Big)
                    \1_{\tau(x_1)=s}\1_{\cE_2}\Bigg]\\
                    &\geq  0,
                \end{align*}
        where the last inequality is due to the fact that 
        \begin{align*}
           & \attn^{(t)}_{{\ans,1}\rightarrow \pred,1}\cdot V_{j'_2,  r_{g_1\cdot y_0}}(g_1)-\big(1-\attn^{(t)}_{{\ans,1}\rightarrow \ans,1}
                    -\attn^{(t)}_{{\ans,1}\rightarrow \pred,2}\big)\Lambda^{(t)}_{5,j'_2, r_{g_1\cdot y_0}}\\
&\geq \attn^{(t)}_{{\ans,1}\rightarrow \pred,1}\bigg(1-2\Big(\attn^{(t)}_{{\ans,1}\rightarrow \pred,2}+\attn^{(t)}_{{\ans,1}\rightarrow \pred,1}-\attn^{(t)}_{{\ans,1}\rightarrow \ans,1}\Big)\bigg)2B\\
&\geq \attn^{(t)}_{{\ans,1}\rightarrow \pred,1}\bigg(1-2\Big(c_1+0.2\Big)\bigg)2B\geq 0.
        \end{align*}
            \item For $j=\tau(g_2(y))$ with some $y\not=y_0,y_1$ such that $g_1(y)=g_2(y)$, we distinguish two cases:
                \begin{itemize}
                    \item If 
                    \[
                      \E\big[\attn^{(t)}_{{\ans,1}\rightarrow \ans,1}-\attn^{(t)}_{{\ans,1}\rightarrow \pred,2}\mid \tau(x_1)=s\big]\leq \frac{c_1}{4},
                    \]
                    then by \Cref{induction-s22-non} and \Cref{lem-s22-lambda-2-sym}, we have   we have  $\Lambda^{(t)}_{5,j'_2,r_{g_2\cdot y_0}}-\Lambda^{(t)}_{5,j,r_{g_2\cdot y}}\geq \big(-\frac{c_1}{6}+\attn^{(t)}_{{\ans,1}\rightarrow \pred,1}\big)2B\geq (-\frac{c_1}{4}+\frac{3}{4}c_1)2B=c_1B$, which implies     
                    \[
                      \sum_{y\neq y_0,y_1}\logit^{(t)}_{5,\tau(g_2(y))}\1_{g_1(y)=g_2(y)}\ll \tilde{O}\!\left(\tfrac{1}{d^{1/2}}\right)\cdot \logit^{(t)}_{5,j_2'}.
                    \]
                    Consequently,
                    \begin{align*}
                        &\sum_{y\neq y_0,y_1}\Bigg| \E\Bigg[  \logit^{(t)}_{5,j}\Big(\ReLU^{\prime}(\Lambda^{(t)}_{5,j,r_{g_1\cdot y}})
                          \cdot  \\
                         &~~~\Big( \attn^{(t)}_{{\ans,1}\rightarrow \ans,0}\cdot V_{j,  r_{g_1\cdot y}}(y_0)
                         +\attn^{(t)}_{{\ans,1}\rightarrow \pred,1}\cdot V_{j,  r_{g_1\cdot y}}(g_1)\\
                          &-\big(1-\attn^{(t)}_{{\ans,1}\rightarrow \ans,1}
                          -\attn^{(t)}_{{\ans,1}\rightarrow \pred,2}\big)\Lambda^{(t)}_{5,j, r_{g_1\cdot y}}
                          \pm\tilde{O}(\sigma_0)\Big)\pm \tilde{O}(\delta^{q}) \Big)
                          \\            &~~~~\1_{\tau(x_1)=s}\1_{\cE_2}\1_{g_1(y)=g_2(y)}\Bigg]\Bigg|
                          \leq O\!\left(\tfrac{1}{d^{1/2}}\right)\cdot \eqref{eq-s22-gd1-non-3}.
                      \end{align*}
        
                    \item If 
                    \[
                      \E\big[\attn^{(t)}_{{\ans,1}\rightarrow \ans,1}-\attn^{(t)}_{{\ans,1}\rightarrow \pred,2}\mid \tau(x_1)=s\big]\geq \frac{c_1}{4},
                    \]
                    then by \Cref{lem-s22-lambda-2-sym}, the activation $\Lambda^{(t)}_{5,j, r_{g_1\cdot y}}$ cannot be triggered.
                \end{itemize}
        \end{itemize}

       Putting the above discussion together, we have
       \begin{align*}
        \E\bigg[ \cN^{(t)}_{s,2,\rom2} \1_{\cE_1}\bigg]+  \E\bigg[ \cN^{(t)}_{s,2,\rom2} \1_{\cE_2}\bigg]\geq   \Omega\Big(\frac{B}{d\cdot n_y}\Big)\cdot  \E\bigg[(1-\logit^{(t)}_{5,j_2})\big| \tau(x_1)=s, \cE_2\bigg].
       \end{align*}
       \item For $\Zb^{2,1}\in \cE_3$, the gradient admits an analysis analogous to that in \Cref{lem-s21-gd1-non}. 
       In particular, by \Cref{induction-s22-non} and \Cref{lem-s22-lambda-3-sym}, we first obtain the following logit bound:
       \begin{align*}
           \logit_{5,j}^{(t)} \;\leq\; \frac{1}{\poly d}\cdot \Big(1-\logit^{(t)}_{5,j_2}\Big),
           \qquad j\neq j_2,\, \tau\!\big(g_1(y_1)\big).
       \end{align*}
       Therefore,
       \begin{align*}
           &\E\Big[ \cN^{(t)}_{s,2,\rom1}\,\1_{\cE_3}\Big]\\
           &= \E\Bigg[ (1-\logit^{(t)}_{5,j_2}) \cdot 
             \Big( -\attn^{(t)}_{{\ans,1}\rightarrow \ans,0}\cdot V_{j_2, r_{g_2\cdot y_1}}(y_1)
             -\attn^{(t)}_{{\ans,1}\rightarrow \pred,1}\cdot V_{j_2, r_{g_2\cdot y_1}}(g_1)\\
             &\hspace{1cm}+\big(1-\attn^{(t)}_{{\ans,1}\rightarrow \ans,1}
             -\attn^{(t)}_{{\ans,1}\rightarrow \pred,2}\big)\Lambda^{(t)}_{5,j_2, r_{g_2\cdot y_1}}
             \pm\tilde{O}(\sigma_0)\pm\tilde{O}(\delta^{q})\Big)
             \1_{\tau(x_1)=s}\1_{\cE_3}\Bigg]\\
           &\geq \;\Omega\!\left(\frac{B}{n_y\cdot d}\right)\cdot  
           \E\Big[(1-\logit^{(t)}_{5,j_2})\;\big|\;\tau(x_1)=s,\, \cE_3\Big].
       \end{align*}
       
       Moving to $\cN^{(t)}_{s,2,\rom2}$, we have
       \begin{align*}
           &\E\Big[\cN^{(t)}_{s,2,\rom2}\,\1_{\cE_3}\Big]\\
           &\geq \E\Bigg[ \logit^{(t)}_{5,\tau(g_1(y_1))}\cdot 
               \Big( \attn^{(t)}_{{\ans,1}\rightarrow \ans,0}\cdot V_{\tau(g_1(y_1)), r_{g_1\cdot y_1}}(y_1)
               +\attn^{(t)}_{{\ans,1}\rightarrow \pred,1}\cdot V_{\tau(g_1(y_1)), r}(g_1)\\
           &
               -\big(1-\attn^{(t)}_{{\ans,1}\rightarrow \ans,1}
               -\attn^{(t)}_{{\ans,1}\rightarrow \pred,2}\big)\Lambda^{(t)}_{5,\tau(g_1(y_1)), r_{g_1\cdot y_1}}
               \pm\tilde{O}(\sigma_0)\pm\tilde{O}(\delta^{q})\Big)\,
               \1_{\tau(x_1)=s}\1_{\cE_3}\Bigg]\\
           &\quad - O\!\left(\frac{B}{d\cdot \poly d}\right)\cdot
           \E\Big[(1-\logit^{(t)}_{5,j_2})\;\big|\;\tau(x_1)=s,\, \cE_3\Big].
       \end{align*}
       
       Combining the two bounds, we conclude
       \begin{align*}
           \E\Big[\cN^{(t)}_{s,2,\rom1}\,\1_{\cE_3}\Big]
           +\E\Big[\cN^{(t)}_{s,2,\rom2}\,\1_{\cE_3}\Big]
           \;\;\geq\;\; \Omega\!\left(\frac{B}{n_y\cdot d}\right)\cdot  
           \E\Big[(1-\logit^{(t)}_{5,j_2})\;\big|\;\tau(x_1)=s,\, \cE_3\Big].
       \end{align*}
       
    \item For $\Zb^{2,1}\in \cE_{4}$, in analogy with \Cref{lem-s21-gd1-non}, we obtain the following crude bound:
    \begin{align*}
        \E\Big[\cN^{(t)}_{s,2,\rom1}\,\1_{\cE_4}\Big]
        +\E\Big[\cN^{(t)}_{s,2,\rom2}\,\1_{\cE_4}\Big]
        \;\;\geq\; 0.
    \end{align*}

      \item For $\Zb^{2,1}\in \cE_5\cup\cE_6$, by \Cref{induction-s22-non}, \Cref{lem-s22-lambda-5-sym} and \Cref{lem-s22-lambda-6-sym}, 
      we obtain the following logit condition:
      \[
          1-\logit^{(t)}_{5, j_2}, \;\;\logit^{(t)}_{5,j}
          \;\;\leq\;\; \frac{1}{\poly d}\cdot  
          \E\Big[(1-\logit^{(t)}_{5,j_2})\;\big|\;\tau(x_1)=s,\, \cE_1\Big],
          \qquad j\neq j_2.
      \]
      Hence, $\cN^{(t)}_{s,2,\rom1}$ and $\cN^{(t)}_{s,2,\rom2}$ can be bounded as
      \begin{align*}
          &\Big|\E\big[\cN^{(t)}_{s,2,\rom1}\,\1_{\cE_m}\big]\Big|, \quad 
           \Big|\E\big[\cN^{(t)}_{s,2,\rom2}\,\1_{\cE_m}\big]\Big| \\
          &\hspace{2cm}\leq O\!\left(\frac{B}{d n_y}\cdot\frac{1}{\poly d}\right)\cdot  
          \E\Big[(1-\logit^{(t)}_{5,j_2})\;\big|\;\tau(x_1)=s,\, \cE_1\Big],
          \qquad m\in\{5,6\}.
      \end{align*}
      
      Moreover, for $\Zb^{2,1}\in \cE_6$, if 
      \[
         \attn^{(t)}_{{\ans,1}\rightarrow \pred,1}+\attn^{(t)}_{{\ans,1}\rightarrow \pred,2}>\tfrac{1}{2},
      \]
      then $\cN^{(t)}_{s,2,\rom1}=0$. 
      Nevertheless, due to the above order-wise bound, this observation does not affect the overall analysis.
      
    \end{enumerate}  
    Putting everything together, we have 
    \begin{align*}
        &\Big[-\nabla_{\Q^{(t)}_{4,3}}{\Loss^{2,2}_{5}}\Big]_{s,s}+ \Big[-\nabla_{\Q^{(t)}_{4,4}}{\Loss^{2,2}_{5}}\Big]_{s,s} =\sum_{\kappa\in\{\rom1,\rom2,\rom3\}}\sum_{m\in[6]}\E\bigg[ \cN^{(t)}_{s,2,\kappa} \1_{\cE_m}\bigg]\\
        &\geq \Omega\Big(\frac{B}{d}\Big)\cdot \E\bigg[(1-\logit^{(t)}_{5,j_2})\big| \tau(x_1)=s, \cE_1\bigg]+ \Omega\Big(\frac{B}{n_y d}\Big)\cdot\sum_{m=2}^3 \E\bigg[(1-\logit^{(t)}_{5,j_2})\big| \tau(x_1)=s, \cE_m\bigg].
    \end{align*}


\end{proof}

\begin{lemma}\label{lem-s22-grad-2-non}
    If \Cref{induction-s22-non} holds for all iterations $t\in [T_{1,2,1,s}, T_{1,2,2,s})$, given $s\in\tau(\X)$,  for $[\Qb_{4,p}]_{s,s'}$, $p\in\{3,4\}$, $s'\not=s\in\tau(\X)$,  we have
    \begin{align*}
        \bigg|\Big[-\nabla_{\Q^{(t)}_{4,p}}{\Loss^{2,2}_{5}}\Big]_{s,s'}\bigg| \leq O\Big(\frac{1}{d}\Big) \bigg|\Big[-\nabla_{\Q^{(t)}_{4,p}}{\Loss^{2,2}_{5}}\Big]_{s,s}\bigg| .
    \end{align*}
  \end{lemma}
  \subsubsection{Attention gap is small}
\begin{lemma}\label{lem-s22-attention-gap-1-non}
    If \Cref{induction-s22-non} holds for all iterations $t\in [T_{1,2,1,s}, T_{1,2,2,s})$, then for any sample $\Zb^{2,1}$, we have 
   \begin{align*}
     \attn^{(t)}_{{\ans,1} \rightarrow \pred,2} - \attn^{(t)}_{{\ans,1} \rightarrow \ans,1}\leq c_1,
   \end{align*}
   where $c_1>0$ is a small constant.
 \end{lemma}
 The proof follows along the same lines as \Cref{lem-s21-attention-gap-non}, and hence the details are omitted for brevity.

 \begin{lemma}\label{lem-s22-attention-gap-2-non}
    If \Cref{induction-s22-non} holds for all iterations $t\in [T_{1,2,1,s}, T_{1,2,2,s})$, then for any sample $\Zb^{2,1}$, we have 
   \begin{align*}
    \attn^{(t)}_{{\ans,1} \rightarrow \ans,1}-  \attn^{(t)}_{{\ans,1} \rightarrow \pred,2} \leq \attn^{(t)}_{{\ans,1} \rightarrow \pred,1}-c_2,
   \end{align*}
   where $c_2=\frac{\log d }{4B}>0$ is a small constant.
 \end{lemma}
 \begin{proof}
    Let $\tilde{T}$ denote the first time such that
    \[
       \E\Big[\,\attn^{(t)}_{{\ans,1} \rightarrow \pred,2}  
         +\attn^{(t)}_{{\ans,1} \rightarrow \pred,1} 
         -\attn^{(t)}_{{\ans,1} \rightarrow \ans,1}
         \;\big|\;\tau(x_1)=s\Big]\;\leq\; \frac{\log d}{4.02B}.
    \]
    Since $\big|[\Q_{4,p}^{(t)}]_{s,s'}\big| \leq O\!\left(\tfrac{[\Q_{4,p}^{(t)}]_{s,s}}{d}\right)$ for $p\in\{3,4\}$, it follows that for any sample $\Zb^{2,1}$ with $\tau(x_1)=s$, at time $\tilde{T}$ we have  
    \begin{align*}   
        \attn^{(t)}_{{\ans,1} \rightarrow \pred,2}  
        +\attn^{(t)}_{{\ans,1} \rightarrow \pred,1} 
        -\attn^{(t)}_{{\ans,1} \rightarrow \ans,1}
        \;\geq\; \frac{\log d}{4.02B}\;\pm\;\tilde{O}(1/d).
    \end{align*}
    By \Cref{lem-s22-attn-non}, we also have 
    $\attn^{(\tilde{T})}_{\ans,1\to\pred,2}-\attn^{(\tilde{T})}_{\ans,1\to\ans,0}\geq \Omega(1)\gg \varrho$. 
    We then consider the following cases:
    \begin{itemize}
        \item If $\Zb^{2,1}\in \cE_1$, then $\Lambda^{(\tilde{T})}_{5,j_2,r_{g_2\cdot y_1}}$ is already activated into the linear regime. 
        Moreover, since $\attn^{(\tilde{T})}_{\ans,1\to\pred,2}-\attn^{(\tilde{T})}_{\ans,1\to\ans,1}\leq \frac{\log d}{4.02B}-\attn^{(t)}_{{\ans,1}\to\pred,1}<0$, 
        it follows that $\Lambda^{(\tilde{T})}_{5,j'_2,r_{g_2\cdot y_0}}$ cannot be activated. 
        Thus, by \Cref{lem-grad-decompositions-sym},
        \[
          \E\Big[\big(\cN^{(\tilde{T})}_{s,3,2,\rom1}+\cN^{(\tilde{T})}_{s,3,2,\rom2}+\cN^{(\tilde{T})}_{s,3,2,\rom3}\big)\1_{\cE_1}\Big]
          \;\geq\; \Omega\!\left(\tfrac{B}{d}\right)\cdot \E\big[1-\logit_{5,j_2}^{(\tilde{T})}\;\big|\;\tau(x_1)=s,\, \cE_1\big],
        \]
        while
        \[
          \E\Big[\big(\cN^{(\tilde{T})}_{s,4,2,\rom1}+\cN^{(\tilde{T})}_{s,4,2,\rom2}+\cN^{(\tilde{T})}_{s,4,2,\rom3}\big)\1_{\cE_1}\Big]
          \;\leq\; -\Omega\!\left(\tfrac{B}{d n_y}\right)\cdot \E\big[1-\logit_{5,j_2}^{(\tilde{T})}\;\big|\;\tau(x_1)=s,\, \cE_1\big].
        \]
 
        \item If $\Zb^{2,1}\in \cE_2$, then
        \[
          \Lambda^{(\tilde{T})}_{5,j'_2,r_{g_2\cdot y_0}}
          =(\attn^{(\tilde{T})}_{\ans,1\to\pred,2}+\attn^{(\tilde{T})}_{\ans,1\to\pred,1}
            -\attn^{(\tilde{T})}_{\ans,1\to\ans,0})2B
          \;\leq\;\frac{\log d}{2.01}.
        \]
        Hence $\logit^{(1)}_{5,j_2}\leq O(d^{-1.01/2.01})(1-\logit_{5,j_2}^{(t)})$. 
        By \Cref{lem-grad-decompositions-sym},
        \[
          \E\Big[\big(\cN^{(\tilde{T})}_{s,3,2,\rom1}+\cN^{(\tilde{T})}_{s,3,2,\rom2}+\cN^{(\tilde{T})}_{s,3,2,\rom3}\big)\1_{\cE_2}\Big]
          \;\geq\;\Omega\!\left(\tfrac{B}{d n_y}\right)\cdot \E\big[1-\logit_{5,j_2}^{(\tilde{T})}\;\big|\;\tau(x_1)=s,\, \cE_2\big],
        \]
        and
        \[
          \E\Big[\big(\cN^{(\tilde{T})}_{s,4,2,\rom1}+\cN^{(\tilde{T})}_{s,4,2,\rom2}+\cN^{(\tilde{T})}_{s,4,2,\rom3}\big)\1_{\cE_2}\Big]
          \;\leq\; -\Omega\!\left(\tfrac{B}{d n_y}\right)\cdot \E\big[1-\logit_{5,j_2}^{(\tilde{T})}\;\big|\;\tau(x_1)=s,\, \cE_2\big].
        \]
 
        \item If $\Zb^{2,1}\in \cE_3$, we can apply the crude bounds
        \[
          \E\Big[\big(\cN^{(\tilde{T})}_{s,3,2,\rom1}+\cN^{(\tilde{T})}_{s,3,2,\rom2}+\cN^{(\tilde{T})}_{s,3,2,\rom3}\big)\1_{\cE_3}\Big]\;\geq\; 0,
        \]
        and
        \[
          \E\Big[\big(\cN^{(\tilde{T})}_{s,4,2,\rom1}+\cN^{(\tilde{T})}_{s,4,2,\rom2}+\cN^{(\tilde{T})}_{s,4,2,\rom3}\big)\1_{\cE_3}\Big]\;\leq\; 0.
        \]
    \end{itemize}
 
    Combining the above, and noting that the gradient contribution from $\cE_{4}\cup\cE_5\cup\cE_6$ is negligible, we conclude
    \begin{align}
        \big[-\nabla_{\Q^{(\tilde{T})}_{4,3}}\Loss_{5}^{2,2}\big]_{s,s} 
        &\;\geq\; \Omega\!\left(\tfrac{B}{d}\right)\cdot \E\big[1-\logit_{5,j_2}^{(\tilde{T})}\;\big|\;\tau(x_1)=s,\, \cE_1\big] \label{s22-eq-grad-gap-3-non}\\
        &\quad +\Omega\!\left(\tfrac{B}{d n_y}\right)\cdot \E\big[1-\logit_{5,j_2}^{(\tilde{T})}\;\big|\;\tau(x_1)=s,\, \cE_2\big], \notag\\
        \big[-\nabla_{\Q^{(\tilde{T})}_{4,4}}\Loss_{5}^{2,2}\big]_{s,s} 
        &\;\leq\; -\Omega\!\left(\tfrac{B}{d}\right)\cdot \E\big[1-\logit_{5,j_2}^{(\tilde{T})}\;\big|\;\tau(x_1)=s,\, \cE_1\big] \label{s22-eq-grad-gap-4-non}\\
        &\quad -\Omega\!\left(\tfrac{B}{d n_y}\right)\cdot \E\big[1-\logit_{5,j_2}^{(\tilde{T})}\;\big|\;\tau(x_1)=s,\, \cE_2\big]. \notag
    \end{align}
 
    Finally, observe that
    \[
       \attn^{(t)}_{{\ans,1} \rightarrow \pred,2}  
       +\attn^{(t)}_{{\ans,1} \rightarrow \pred,1} 
       -\attn^{(t)}_{{\ans,1} \rightarrow \ans,1}
       \;=\; \frac{e^{[\Q^{(t)}_{4,3}]_{s,s}}-e^{[\Q^{(t)}_{4,4}]_{s,s}}-e^{\tilde{O}(1/d)}}
         {e^{[\Q^{(t)}_{4,3}]_{s,s}}+e^{[\Q^{(t)}_{4,4}]_{s,s}}+e^{\tilde{O}(1/d)}}.
    \]
    Hence, \eqref{s22-eq-grad-gap-3-non} and \eqref{s22-eq-grad-gap-4-non} together imply that 
    \[
      \attn^{(t)}_{{\ans,1} \rightarrow \pred,2}  
      +\attn^{(t)}_{{\ans,1} \rightarrow \pred,1} 
      -\attn^{(t)}_{{\ans,1} \rightarrow \ans,1}
    \]
    must decrease at time $\tilde{T}+1$.
 \end{proof}

\subsubsection{At the End of Training}
\begin{lemma}[At the end of stage 1.2.2]\label{lem-end-sym} \Cref{induction-s22-non} holds for all iterations $T_{1,2,1,s}<t\leq T_{1,2,2,s}={O}\Big(\frac{\poly(d)}{\eta B}\Big)$. At the end of training, we have
    \begin{enumerate}[(a)]
        \item Attention concentration: $\E[\ea\mid \tau(x_1)=s]\leq 2.52c_1$ for some small constant $0<c_1=\frac{1.005\log d }{2B}$;
        \item Loss convergence: $\Loss^{2,2}_{5,s}\leq e^{(-\frac{1}{2}+{3.01 c_1})\cdot 2B}=\frac{1}{\poly d}$.
    \end{enumerate}
\end{lemma}
\begin{proof}
    \Cref{lem-s22-attn-non} and \Cref{lem-s22-gd1-non} guarantee the continual growth of $[\Qb^{(t)}_{4,3}]$ and $[\Qb^{(t)}_{4,4}]$ until the 
    attention weight $\ate$ reaches $\tfrac{4c_1}{3}$.  
    Combining \Cref{lem-s22-gd1-non}, \Cref{lem-s22-grad-2-non}, \Cref{lem-s22-attention-gap-1-non}, and \Cref{lem-s22-attention-gap-2-non}, we conclude that there exists a stopping time 
    \[
    T_{1,2,2,s}=O\!\left(\frac{\poly(d)}{\eta B}\right),
    \]
    and that \Cref{induction-s22-non} holds for all $t\in [T_{1,2,1,s},\, T_{1,2,2,s})$.
    
    Next,  we establish an upper bound for $\ate$ at the end of training.  
    In this stage, it suffices to focus on the main event $\cE_1$.  
    We denote
    \[
    \attn^{(t)}_{\ans,1\to \pred,1}, \quad \attn^{(t)}_{\ans,1\to \ans,1} \;=\; c+\tilde{O}\!\left(\tfrac{1}{d}\right).
    \]
    
    \begin{itemize}
        \item If $\attn^{(t)}_{\ans,1\to \pred,2}\geq \attn^{(t)}_{\ans,1\to \ans,1}$, then for $\Zb^{2,1}\in \cE_{1}$ we have
        \begin{align*}
            \log p_{F}(\Z_{\ans,2,5}|\Z^{(2,1)})
            &\geq \Theta(1)\cdot e^{\log d - \big(\attn^{(t)}_{\ans,1\to \pred,2}-\attn^{(t)}_{\ans,1\to \ans,0}\big)2B}\\
            &\geq \Theta\!\left( e^{-\left(\tfrac{1}{2}+\tfrac{c_1}{2}-2c-\tfrac{\log d}{B}\right)2B} \right),
        \end{align*}
        where the last inequality follows from 
        \[
        \attn^{(t)}_{\ans,1\to \pred,2}\leq \tfrac{1}{2}(1-2c+c_1)=\tfrac{1+c_1}{2}-c.
        \]
    
        \item If $\attn^{(t)}_{\ans,1\to \pred,2}\leq \attn^{(t)}_{\ans,1\to \ans,1}$, then for $\Zb^{2,1}\in \cE_{1}$ we have
        \begin{align*}
            \log p_{F}(\Z_{\ans,2,5}|\Z^{(2,1)})
            &\geq \Theta(1)\cdot e^{\log d - \big(\attn^{(t)}_{\ans,1\to \pred,2}-\attn^{(t)}_{\ans,1\to \ans,0}\big)2B}\\
            &\geq \Theta\!\left( e^{-\left(\tfrac{1}{2}-2c-\tfrac{\log d}{B}\right)2B} \right),
        \end{align*}
        where the last inequality follows from the fact that 
        $\attn^{(t)}_{\ans,1\to \pred,2}\leq \tfrac{1}{2}-c$.
    \end{itemize}
    
    Putting the above cases together, we obtain
    \[
    \tfrac{1}{2}+\frac{c_1}{2}-2c-\tfrac{\log d}{B}\;\geq\;\tfrac{1}{2}-3.01c_1,
    \]
    which implies 
    \[
    c \;\leq\; \tfrac{1}{2}(3.02c_1+0.5c_1-c_1)\;=\;1.26c_1.
    \]
    \end{proof}

\section{Recursive Learning the Attention Layer: Symmetry Case}
\label{sec:recursive-learning}
In this section, we analyze the recursive learning dynamics of the attention layer in the symmetry case. 
To this end, we recall the definitions of the greedy data annotator, the bootstrapped LEGO data distribution, and the associated self-training loss.

\begin{definition}[Greedy data annotator]\label{def:greedy-annotator-appendix}
    The greedy language model \(\widehat{p}_F\) induced by a network \(F\) is defined as
    \begin{equation}
        \widehat{p}_{F}(Z_{\ans,L'+1}\mid Z^{L,L'}) = 
        \begin{cases}
            1, & \text{if } Z_{\ans,L'+1} = \argmax_{Z} p_F(Z\mid Z^{L,L'}), \\
            0, & \text{otherwise.}
        \end{cases}
    \end{equation} 
\end{definition}

\begin{definition}[Bootstrapped LEGO distribution]\label{def:bootstrap-lego-distribution-appendix}
    We define \(\cD_{F}^{L,L'}\) as the LEGO distribution in \Cref{assump:lego-data-distribution}, 
    except that the answers \(Z_{\ans,\ell}, 1\leq \ell \leq L'\) are generated recursively by sampling 
    \(Z_{\ans,\ell} \sim \widehat{p}_{F}(\cdot\mid Z^{L,\ell-1})\) from the greedy language model \(\widehat{p}_F\).
\end{definition}

\begin{definition}[Self-training loss]\label{def:self-training-loss-appendix}
    Given a (fixed) model $\tilde{F}$ and length $L$, 
    the self-training next-clause-prediction loss is defined by replacing \(\cD^{L,L'}\) in \Cref{def:next-clause-loss} 
    with the bootstrapped distribution \(\cD_{F}^{L,L'}\) from \Cref{def:bootstrap-lego-distribution-appendix}:
    \begin{align}
        \Loss^{L,L'}_{\tilde{F}}(F)
        \triangleq \E_{Z^{L,L'}\sim \cD_{\tilde{F}}^{L,L'}}
        \Big[ -\log p_{F}(Z_{\ans,L',i} \mid Z^{L,L'-1}) \Big]. 
        \label{eq-self-loss-appendix}    
    \end{align}
    Similarly, the per-token loss is defined by 
    \[
        \Loss^{L,L'}_{\tilde{F}, i} 
        = \E_{Z^{L,L'}\sim \cD_{\tilde{F}}^{L,L'}}\big[-\log p_{F_i}(Z_{\ans,L',i} \mid Z^{L,L'-1})\big], 
        \quad i\in[5].
    \]
\end{definition}

At stage $k \geq 2$, let $F^{(T_{k-1})}$ denote the model obtained from the previous stage, 
trained on the task $\cT^{L_{k-1}}$ with $L_{k-1} = 2^{k-1}$. 
Using the greedy annotator $\widehat{p}_{F^{(T_{k-1})}}$, 
we construct the bootstrapped LEGO distribution $\cD_{F^{(T_{k-1})}}^{L_k,L_k}$ 
with total length $L_k = 2^k$. 
The corresponding self-training loss $\Loss^{L_k,L'}_{F^{(T_{k-1})}}$ is then defined 
as in \Cref{def:self-training-loss-appendix}. 
In this stage, we focus on training via gradient descent on $\Loss^{L_k,L'}_{F^{(T_{k-1})},i}$ 
with sequence length $L'=2$ and target token $i=5$, 
initialized from the previous-stage model $F^{(T_{k-1})}$.

For notational simplicity, throughout the discussion we drop the subscript of the expectation operator $\E$ when $\Zb$ is sampled from the ground-truth LEGO distribution; the subscript will be explicitly included only when the expectation is taken over the bootstrapped distribution. Moreover, we abbreviate the wrong attention error $\ate^{L_k,2}$ as $\ate^{L_k}$ and the attention gap $\Delta^{L_k,2}$ as $\Delta^{L_k}$ when the context is clear.

\subsection{Preliminaries}
We first present preliminaries on the recursive learning attention layer, including its gradient computations and some useful probability lemmas.

\subsubsection{Gradient Computations}
\begin{fact}[Gradients of \(\Q\)]\label{fact:gradients-Q-sym-recursive}
Given $F^{(T_{k-1})}$ 
from the previous stage $\cT^{L_{k-1}}$ with $L_{k-1}=2^{k-1}$ for $k\geq 2$,   for $(p,q)\in \Big\{(4,3), (4,4)\Big\}$, we have 
    \begin{align*}
        -\nabla_{\Q_{p,q}}\Loss^{L_k,2}_{F^{(T_{k-1})},5}= -\nabla_{\Q_{p,q}}\Loss^{L_k,2}_{5}.
    \end{align*}
\end{fact}
\begin{remark}\label{remark-fact-1-rec}
    This fact is straightforward to verify. 
Since the attention error $\ate^{L_{k-1},2}$ has already been controlled at a small constant level in the previous stage for the model $F^{(T_{k-1})}$, 
when transitioning from length $2^{k-1}$ to $2^k$, the incorrect attention produced by $F^{(T_{k-1})}$ cannot increase significantly. 
In particular, we have
\[
   \ate^{L_k} \;\leq\; 2 \ate^{L_{k-1}},
\]
which remains small. 
Moreover, the attention gap 
$\Delta^{L_k} = \lvert \attn_{\ans,1 \to \pred,2} - \attn_{\ans,1 \to \ans,1} \rvert$ 
is non-increasing due to the doubling of irrelevant clauses. 
As a result, the probability assigned by $F^{(T_{k-1})}$ to the correct prediction of $Z_{\ans,2,5}$ given $Z^{L_k,1}$ 
remains significantly larger than that of any incorrect prediction. 
Consequently, under the greedy data annotator, 
the self-training loss $\Loss^{L_k,2}_{F^{(T_{k-1})},5}$ coincides with the original loss $\Loss^{L_k,2}_{5}$ 
on the ground-truth LEGO distribution.
\end{remark}

\begin{lemma}[Gradients of $\Qb_{4,3}$]\label{lem-gradients-Q43-rec}
    Given $F^{(T_{k-1})}$ 
    from the previous stage and  $s\in\tau(\X)$, for the diagonal entry \( [\Q_{4,3}]_{s,s} \) of the block \(\Qb_{4,3}\), we have 
    \begin{align*}
    \Big[-\nabla_{\Q_{4,3}}\Loss^{L_k,2}_{F^{(T_{k-1})},5}\Big]_{s,s}&= \E\Bigg[
    \attn_{{\ans,1} \rightarrow \pred,2} \cdot \bigg(\sum_{j \in[d]} \Ecal_{5,j}(\Zb^{2,1})\sum_{r\in [m]}\ReLU^{\prime}\big(\Lambda_{5, j,r}\big)\cdot  \\
    &~~~~~~~~~~~~~\Big( \dbrack{\Wb_{5,j,r},\Z_{\pred,2}}- \Lambda_{5, j,r}+b_{5,j,r}\Big)\bigg) \1_{s=\tau(x_1)}\Bigg].
\end{align*}
Moreover, for the off-diagonal entries \( [\Q_{4,3}]_{s,s'} \) with \( s \neq s' \), we have
    \begin{align*}
    \Big[-\nabla_{\Q_{4,3}}\Loss^{L_k,2}_{F^{(T_{k-1})},5}\Big]_{s,s'}&= \E\Bigg[
    \attn_{{\ans,1} \rightarrow \pred,1} \cdot \bigg(\sum_{j \in[d]} \Ecal_{5,j}(\Zb^{2,1})\sum_{r\in [m]}\ReLU^{\prime}\big(\Lambda_{5, j,r}\big)\cdot  \\
    &~~~~~\Big( \dbrack{\Wb_{5,j,r},\Z_{\pred,1}}- \Lambda_{5, j,r}+b_{5,j,r}\Big)\bigg) \1_{s=\tau(x_1),s'=\tau(x_0)}\Bigg].
\end{align*}
\end{lemma}

\begin{lemma}[Gradients of $\Qb_{4,4}$]\label{lem-gradients-Q44-rec}
    Given $F^{(T_{k-1})}$ 
    from the previous stage and  $s\in\tau(\X)$, for the diagonal entry \( [\Q_{4,4}]_{s,s} \) of the block \(\Qb_{4,4}\), we have 
    \begin{align*}
    \Big[-\nabla_{\Q_{4,4}}\Loss^{L_k,2}_{F^{(T_{k-1})},5}\Big]_{s,s}&= \E\Bigg[
    \attn_{{\ans,1} \rightarrow \ans,1} \cdot \bigg(\sum_{j \in[d]} \Ecal_{5,j}(\Zb^{2,1})\sum_{r\in [m]}\ReLU^{\prime}\big(\Lambda_{5, j,r}\big)\cdot  \\
    &~~~~~~~~~~~~~\Big( \dbrack{\Wb_{5,j,r},\Z_{\ans,1}}- \Lambda_{5, j,r}+b_{5,j,r}\Big)\bigg) \1_{s=\tau(x_1)}\Bigg].
\end{align*}
Moreover, for the off-diagonal entries \( [\Q_{4,4}]_{s,s'} \) with \( s \neq s' \), we have 
    \begin{align*}
        \Big[-\nabla_{\Q_{4,4}}\Loss^{L_k,2}_{F^{(T_{k-1})},5}\Big]_{s,s'}&= \E\Bigg[
    \attn_{{\ans,1} \rightarrow \ans,0} \cdot \bigg(\sum_{j \in[d]} \Ecal_{5,j}(\Zb^{2,1})\sum_{r\in [m]}\ReLU^{\prime}\big(\Lambda_{5, j,r}\big)\cdot  \\
    &~~~~~\Big( \dbrack{\Wb_{5,j,r},\Z_{\ans,0}}- \Lambda_{5, j,r}+b_{5,j,r}\Big)\bigg) \1_{s=\tau(x_1),s'=\tau(x_0)}\Bigg].
\end{align*}
\end{lemma}

\paragraph{Notations for activated neurons.} 
Recall that 
\begin{align*}
    \fA = \bigcup_{j \in \tau(\Y)} \fA_{j}, 
    \quad \text{where } \fA_j = \{\, r_{j,y} \mid y \in \Y \,\}.
\end{align*}
Given $\Zb^{L,\ell-1}$, let 
\[
    \hat{\cG}(\Zb^{L,\ell-1}) = \bigcup_{\ell'=1}^L \{g_{\ell'}\}
\]
be the collection of all group elements chosen in the predicate clauses, 
and similarly define 
\[
    \hat{\cY} = \bigcup_{\ell'=0}^{\ell-1} \{y_{\ell'}\}.
\]
We then introduce
\[
    \hat{\fA}_{j}(\Zb^{L,\ell-1})
    = \Big\{\, r_{g\cdot y}\;\Big|\; g \in \fiber_{j,y},\;
       g \in \hat{\cG}(\Zb^{L,\ell-1}) \ \vee\  y \in \hat{\cY} \Big\}.
\]
For simplicity, we omit the dependence on \( \Zb^{L,\ell-1} \) in the notation of \( \hat{\fA}_j \) when it is clear from context.  

The above recalls the relevant notations from \Cref{sec:useful-bounds-for-gradients-sym}. 
With these preliminaries in place, we now state the following lemma, which also relies on the feature-magnitude bounds established therein.

\begin{lemma}[Characterizations of Lambda]\label{lem-lambda-char-rec}
    Given $\Zb^{L_k,1}$ with  \( \{\attn_{\ans,1 \to \kk} \}_{\kk\in\cI^{L_k,1}}\), then 
    \begin{enumerate}[(a)]
        \item for \( j \in \tau(\cY) \), for activated neuron $r\in\fA_j$, we have  
        \begin{align*}
            \Lambda_{5,j,r}=\sum_{\ell'=1}^2 \attn_{\ans,1 \to \pred,\ell'}V_{j,r}(g_{\ell'})+  \sum_{\ell'=1}^{L_k} \attn_{\ans,1 \to \ans,\ell'-1} V_{j,r}(y_{\ell'-1})\pm  \tilde{O}(\sigma_0). 
        \end{align*}
         \item for \( j \in \tau(\cY) \), for any non-activated neuron $r\notin\fA_j$ 
         we have
      \begin{align*}
            \Big|\Lambda_{5,j,r}\Big| \leq {O}(\delta).
        \end{align*} 
        \item for \( j \notin \tau(\cY) \),  for any $r\in [m]$,  we have 
        \begin{align*}
            \Big|\Lambda_{5,j,r}\Big|%
            \leq \tilde{O}(\sigma_0).
        \end{align*} 
    \end{enumerate}
\end{lemma}

\begin{lemma}\label{lem-non-activated-neuron-rec}
    Given $j\in\tau(\cY)$, for $r\in \fA_{j}\setminus\hat{\fA}_{j}$, we have $\ReLU^{\prime}\big(\Lambda_{5, j,r}\big)=0$.
\end{lemma}
We are now ready to derive the gradients of the attention layer, building on \Cref{lem-gradients-Q43-rec,lem-gradients-Q44-rec} and the properties established above.

\begin{lemma}[Refined expression for the gradient of $\Qb_{4,3}$] \label{lem-refined-grad-Q43-rec} Given $F^{(T_{k-1})}$ 
    from the previous stage and  $s\in\tau(\X)$, for the diagonal entry \( [\Q_{4,3}]_{s,s} \) of the block \(\Qb_{4,3}\), letting $j_2=\tau(g_2(y_1))$, we have 
 
  \begin{align*}
    &\Big[-\nabla_{\Q_{4,3}}\Loss^{L_k,2}_{F^{(T_{k-1})},5}\Big]_{s,s}= 
      \E\Bigg[
    \attn_{{\ans,1} \rightarrow \pred,2} \cdot \\
    &~~~~~~~~~~~~~ \bigg( (1-\logit_{5,j_2})\cdot \Big(\sum_{r\in\hat{\fA}_{j_2}}\ReLU^{\prime}(\Lambda_{5,j_2, r
    })\cdot \Big( V_{j_2,  r}(g_2)- \Lambda_{5,j_2, r}\pm\tilde{O}(\sigma_0) \Big)\pm \tilde{O}(\delta^{q}) \Big)\\   
    &~~~~~~~~~~-\sum_{j\neq j_2\in\tau(\cY)}\logit_{5,j} \cdot \Big(\sum_{r\in\hat{\fA}_{j}}\ReLU^{\prime}(\Lambda_{5,j,r
    })\cdot  \Big( V_{j, r}(g_2)- \Lambda_{5,j,r}\pm\tilde{O}(\sigma_0) \Big)\pm \tilde{O}(\delta^{q}) \Big) \\
    &~~~~~~~~~~~~\pm\sum_{j\notin\tau(\cY)}\logit_{5,j}\tilde{O}(\sigma^{q}_0)  \bigg)\1_{\tau(x_1)=s}\Bigg].
\end{align*}
Moreover, for the off-diagonal entries \( [\Q_{4,3}]_{s,s'} \) with \( s \neq s' \), we have
  \begin{align*}
    &\Big[-\nabla_{\Q_{4,3}}\Loss^{L_k,2}_{F^{(T_{k-1})},5}\Big]_{s,s'}=  \E\Bigg[
    \attn_{{\ans,1} \rightarrow \pred,1} \cdot\\
    &~~~~~~~~~~~~~ \bigg( (1-\logit_{5,j_2})\cdot \Big(\sum_{r\in\hat{\fA}_{j_2}}\ReLU^{\prime}(\Lambda_{5,j_2, r
    })\cdot \Big( V_{j_2,  r}(g_1)- \Lambda_{5,j_2, r}\pm\tilde{O}(\sigma_0) \Big)\pm \tilde{O}(\delta^{q}) \Big)\\   
    &~~~~~~~~~~-\sum_{j\neq j_2\in\tau(\cY)}\logit_{5,j} \cdot \Big(\sum_{r\in\hat{\fA}_{j}}\ReLU^{\prime}(\Lambda_{5,j,r
    })\cdot  \Big( V_{j, r}(g_1)- \Lambda_{5,j,r}\pm\tilde{O}(\sigma_0) \Big)\pm \tilde{O}(\delta^{q}) \Big) \\
    &~~~~~~~~~~~~\pm\sum_{j\notin\tau(\cY)}\logit_{5,j}\tilde{O}(\sigma^{q}_0)  \bigg)\1_{\tau(x_1)=s,\tau(x_0)=s'}\Bigg].
\end{align*}
\end{lemma}
\begin{lemma}[Refined expression for the gradient of $\Qb_{4,4}$] \label{lem-refined-grad-Q44-rec} Given $F^{(T_{k-1})}$ 
    from the previous stage and  $s\in\tau(\X)$, for the diagonal entry \( [\Q_{4,4}]_{s,s} \) of the block \(\Qb_{4,4}\), letting $j_2=\tau(g_2(y_1))$, we have 
  \begin{align*}
    &\Big[-\nabla_{\Q_{4,4}}\Loss^{L_k,2}_{F^{(T_{k-1})}, 5}\Big]_{s,s}= 
      \E\Bigg[
    \attn_{{\ans,1} \rightarrow \ans,1} \cdot \\
    &~~~~~~~~~~~~~ \bigg( (1-\logit_{5,j_2})\cdot \Big(\sum_{r\in\hat{\fA}_{j_2}}\ReLU^{\prime}(\Lambda_{5,j_2, r
    })\cdot \Big( V_{j_2,  r}(y_1)- \Lambda_{5,j_2, r}\pm\tilde{O}(\sigma_0) \Big)\pm \tilde{O}(\delta^{q}) \Big)\\   
    &~~~~~~~~~~-\sum_{j\neq j_2\in\tau(\cY)}\logit_{5,j} \cdot \Big(\sum_{r\in\hat{\fA}_{j}}\ReLU^{\prime}(\Lambda_{5,j,r
    })\cdot  \Big( V_{j, r}(y_1)- \Lambda_{5,j,r}\pm\tilde{O}(\sigma_0) \Big)\pm \tilde{O}(\delta^{q}) \Big) \\
    &~~~~~~~~~~~~\pm\sum_{j\notin\tau(\cY)}\logit_{5,j}\tilde{O}(\sigma^{q}_0)  \bigg)\1_{\tau(x_1)=s}\Bigg].
\end{align*}
Moreover, for the off-diagonal entries \( [\Q_{4,3}]_{s,s'} \) with \( s \neq s' \), we have 
  \begin{align*}
    &\Big[-\nabla_{\Q_{4,4}}\Loss^{L_k,2}_{F^{(T_{k-1})}, 5}\Big]_{s,s'}=  \E\Bigg[
    \attn_{{\ans,1} \rightarrow \ans,0} \cdot\\
    &~~~~~~~~~~~~~ \bigg( (1-\logit_{5,j_2})\cdot \Big(\sum_{r\in\hat{\fA}_{j_2}}\ReLU^{\prime}(\Lambda_{5,j_2, r
    })\cdot \Big( V_{j_2,  r}(y_0)- \Lambda_{5,j_2, r}\pm\tilde{O}(\sigma_0) \Big)\pm \tilde{O}(\delta^{q}) \Big)\\   
    &~~~~~~~~~~-\sum_{j\neq j_2\in\tau(\cY)}\logit_{5,j} \cdot \Big(\sum_{r\in\hat{\fA}_{j}}\ReLU^{\prime}(\Lambda_{5,j,r
    })\cdot  \Big( V_{j, r}(y_0)- \Lambda_{5,j,r}\pm\tilde{O}(\sigma_0) \Big)\pm \tilde{O}(\delta^{q}) \Big) \\
    &~~~~~~~~~~~~\pm\sum_{j\notin\tau(\cY)}\logit_{5,j}\tilde{O}(\sigma^{q}_0)  \bigg)\1_{\tau(x_1)=s,\tau(x_0)=s'}\Bigg].
\end{align*}
\end{lemma}

Following the above calculations, we can further  obtain the gradient summation of $\Qb_{4,3}$ and $\Qb_{4,4}$ as follows:
\begin{lemma}[Gradient sum of $\Qb_{4,3}$ and $\Qb_{4,4}$]
    \label{lem-grad-sum-rec} Given $F^{(T_{k-1})}$ 
    from the previous stage and  $s\in\tau(\X)$,  letting $j_2=\tau(g_2(y_1))$, we have
    \begin{align*}
        &\Big[-\nabla_{\Qb_{4,3}}\Loss^{L_k,2}_{F^{(T_{k-1})}, 5}\Big]_{s,s}+ \Big[-\nabla_{\Qb_{4,3}}\Loss^{L_k,2}_{F^{(T_{k-1})}, 5}\Big]_{s,s}\\
          &=\E\Bigg[
     \bigg( (1-\logit_{5,j_2})\cdot \Big(\sum_{r\in\hat{\fA}_{j_2}}\ReLU^{\prime}(\Lambda_{5,j_2, r
         })\cdot \\
         &~~~~~~~~~~\Big( -\attn_{{\ans,1} \rightarrow \ans,0} \cdot V_{j_2,  r}(y_0)-\sum_{\ell\neq 2} \attn_{{\ans,1} \rightarrow \pred,\ell} \cdot V_{j_2,  r}(g_{\ell})\\
         &~~~~~~~~~~~~~~~~~~+\big(1-\attn_{{\ans,1} \rightarrow \ans,1}-\attn_{{\ans,1} \rightarrow \pred,2}\big) \Lambda_{5,j_2, r}\pm\tilde{O}(\sigma_0) \Big)\pm \tilde{O}(\delta^{q}) \Big)\\   
         &~~~~~+\sum_{j\neq j_2\in\tau(\cY)}\logit_{5,j} \cdot \Big(\sum_{r\in\hat{\fA}_{j}}\ReLU^{\prime}(\Lambda_{5,j,r
         })\cdot  \\
        &~~~~~~~~~~\Big( \attn_{{\ans,1} \rightarrow \ans,0} \cdot V_{j,  r}(y_0)+\sum_{\ell\neq 2} \attn_{{\ans,1} \rightarrow \pred,\ell} \cdot V_{j_2,  r}(g_{\ell})\\
         &~~~~~~~~~~~~~~~~~~-\big(1-\attn_{{\ans,1} \rightarrow \ans,1}-\attn_{{\ans,1} \rightarrow \pred,2}\big) \Lambda_{5,j, r}\pm\tilde{O}(\sigma_0) \Big)\pm \tilde{O}(\delta^{q}) \Big)\\      &~~~~~~~~~~~~\pm\sum_{j\notin\tau(\cY)}\logit_{5,j}\cdot \big(\attn_{{\ans,1} \rightarrow \ans,1}+\attn_{{\ans,1} \rightarrow \pred,2}\big)\tilde{O}(\sigma^{q}_0)  \bigg)\1_{\tau(x_1)=s}\Bigg].
     \end{align*}
\end{lemma}

\paragraph{Notations for gradient decompositions.} Next we define some useful notations to further
simplify the expressions of gradient.

\begin{lemma}\label{lem-grad-decompositions-rec}
    Given $F^{(T_{k-1})}$ 
    from the previous stage and  $s\in\tau(\X)$, we define the following notations for the gradient decompositions:
    \begin{enumerate}
        \item for $[\Q_{4,3}]_{s,s}$  we have $\big[-\nabla_{\Qb_{4,3}}\Loss^{L_k,2}_{F^{(T_{k-1})}, 5}\big]_{s,s}=\E\big[\cN_{s,3,L_k,\rom1}+\cN_{s,3,L_k,\rom2}+\cN_{s,3,L_k,\rom3}\big]$, where 
      \begin{align}
       \cN_{s,3,L_k,\rom1}&= 
    \attn_{{\ans,1} \rightarrow \pred,2} \cdot (1-\logit_{5,j_2})\cdot \label{eq-def-N-s-3-2-1-rec} \\
    &~~~~~~~~~  \Big(\sum_{r\in\hat{\fA}_{j_2}}\ReLU^{\prime}(\Lambda_{5,j_2, r
    })
    \cdot \Big( V_{j_2,  r}(g_2)- \Lambda_{5,j_2, r}\pm\tilde{O}(\sigma_0) \Big)\pm \tilde{O}(\delta^{q}) \Big) 
    \1_{\tau(x_1)=s},\notag
\end{align}
      \begin{align}\label{eq-def-N-s-3-2-2-rec}
       \cN_{s,3,L_k,\rom2}&= -
    \attn_{{\ans,1} \rightarrow \pred,2} \cdot\sum_{j\neq j_2\in\tau(\cY)}\logit_{5,j} \cdot \\
    &~~~~~~~~~~\Big(\sum_{r\in\hat{\fA}_{j}}\ReLU^{\prime}(\Lambda_{5,j,r
    })\cdot  \Big( V_{j, r}(g_2)- \Lambda_{5,j,r}\pm\tilde{O}(\sigma_0) \Big)\pm \tilde{O}(\delta^{q}) \Big) 
\1_{\tau(x_1)=s},\notag
\end{align}
      \begin{align}\label{eq-def-N-s-3-2-3-rec}
       \cN_{s,3,L_k,\rom3}&= \pm
    \attn_{{\ans,1} \rightarrow \pred,2} \cdot\sum_{j\notin\tau(\cY)}\logit_{5,j}\tilde{O}(\sigma^{q}_0)\1_{\tau(x_1)=s}.~~~~~~~~~~~~~~~~~~~~~~~~~~~~~~
\end{align}
    \item for $[\Q_{4,4}]_{s,s}$, we have $\big[-\nabla_{\Qb_{4,4}}\Loss^{L_k,2}_{F^{(T_{k-1})}, 5}\big]_{s,s}=\E\big[\cN_{s,4,L_k,\rom1}+\cN_{s,4,L_k,\rom2}+\cN_{s,4,L_k,\rom3}\big]$, where 
      \begin{align}\label{eq-def-N-s-4-2-1-rec}
       \cN_{s,4,L_k,\rom1}&=
    \attn_{{\ans,1} \rightarrow \ans,1} \cdot (1-\logit_{5,j_2})\cdot \\
    &~~~~~~~~~  \Big(\sum_{r\in\hat{\fA}_{j_2}}\ReLU^{\prime}(\Lambda_{5,j_2, r
    })
    \cdot \Big( V_{j_2,  r}(y_1)- \Lambda_{5,j_2, r}\pm\tilde{O}(\sigma_0) \Big)\pm \tilde{O}(\delta^{q}) \Big) 
    \1_{\tau(x_1)=s},\notag
\end{align}
      \begin{align}\label{eq-def-N-s-4-2-2-rec}
       \cN_{s,4,L_k,\rom2}&= -
    \attn_{{\ans,1} \rightarrow \ans,1} \cdot\sum_{j\neq j_2\in\tau(\cY)}\logit_{5,j} \cdot \\
    &~~~~~~~~~~\Big(\sum_{r\in\hat{\fA}_{j}}\ReLU^{\prime}(\Lambda_{5,j,r
    })\cdot  \Big( V_{j, r}(y_1)- \Lambda_{5,j,r}\pm\tilde{O}(\sigma_0) \Big)\pm \tilde{O}(\delta^{q}) \Big) 
\1_{\tau(x_1)=s},\notag
\end{align}
      \begin{align}\label{eq-def-N-s-4-2-3-rec}
       \cN_{s,4,L_k,\rom3}&= \pm
    \attn_{{\ans,1} \rightarrow \ans,1} \cdot\sum_{j\notin\tau(\cY)}\logit_{5,j}\tilde{O}(\sigma^{q}_0)\1_{\tau(x_1)=s}.~~~~~~~~~~~~~~~~~~~~~~~~~~~~~~
\end{align}
\item for the summation of $[\Q_{4,3}]_{s,s}$ and $[\Q_{4,4}]_{s,s}$, we have $\big[-\nabla_{\Qb_{4,3}}\Loss_5^{2,2}\big]_{s,s}+\big[-\nabla_{\Qb_{4,4}}\Loss_5^{2,2}\big]_{s,s}=\E\big[\cN_{s,2,\rom1}+\cN_{s,2,\rom2}+\cN_{s,2,\rom3}\big]$, where 
 \begin{align*}
    \cN_{s,L_k,\rom1}&=
 (1-\logit_{5,j_2})\cdot \Big(\sum_{r\in\hat{\fA}_{j_2}}\ReLU^{\prime}(\Lambda_{5,j_2, r
     })\cdot \\
     &~~~~~~~~~~\Big( -\attn_{{\ans,1} \rightarrow \ans,0} \cdot V_{j_2,  r}(y_0)-\sum_{\ell\neq 2}\attn_{{\ans,1} \rightarrow \pred,\ell} \cdot V_{j_2,  r}(g_{\ell})\\
     &~~~~~~+\big(1-\attn_{{\ans,1} \rightarrow \ans,1}-\attn_{{\ans,1} \rightarrow \pred,2}\big) \Lambda_{5,j_2, r}\pm\tilde{O}(\sigma_0) \Big)\pm \tilde{O}(\delta^{q}) \Big)\1_{\tau(x_1)=s},\\   
     \cN_{s,L_k,\rom2}&=   \sum_{j\neq j_2\in\tau(\cY)}\logit_{5,j} \cdot \Big(\sum_{r\in\hat{\fA}_{j}}\ReLU^{\prime}(\Lambda_{5,j,r
     })\cdot  \\
    &~~~~~~~~~~\Big( \attn_{{\ans,1} \rightarrow \ans,0} \cdot V_{j,  r}(y_0)+\sum_{\ell\neq 2}\attn_{{\ans,1} \rightarrow \pred,\ell} \cdot V_{j,  r}(g_{\ell})\\
     &~~~~~~~-\big(1-\attn_{{\ans,1} \rightarrow \ans,1}-\attn_{{\ans,1} \rightarrow \pred,2}\big) \Lambda_{5,j, r}\pm\tilde{O}(\sigma_0) \Big)\pm \tilde{O}(\delta^{q}) \Big)\1_{\tau(x_1)=s},\\      \cN_{s,L_k,\rom3}&=\pm\sum_{j\notin\tau(\cY)}\logit_{5,j}\cdot \big(\attn_{{\ans,1} \rightarrow \ans,1}+\attn_{{\ans,1} \rightarrow \pred,2}\big)\tilde{O}(\sigma^{q}_0)  \1_{\tau(x_1)=s}.
 \end{align*}
    \end{enumerate}
\end{lemma}

\subsubsection{Probabilistic Events}

We introduce the following events:
\begin{align*}
  \cE_{L_k,1}
  &= \Bigg\{ \frac{1}{L_k}\sum_{\ell\in [L_k]\setminus\{2\}}
      \1\!\big\{ g_\ell(y_1)=g_2(y_1) \big\} \;\le\; U_k \Bigg\},\\
  \cE_{L_k,2}
  &= \big\{ y_0 = y_1 \big\},\\
  \cE_{L_k,3}
  &= \Bigg\{ \max_{y\in\cY}
      \sum_{\ell\in [L_k]\setminus\{2\}}
      \1\!\big\{ g_\ell(y)=g_2(y) \big\} \;\le\; W_k \Bigg\}.
\end{align*}
We set
\[
  U_k \;=\; \frac{1}{L_k}\,\Big\lceil 1+\frac{L_k}{8}\Big\rceil \;=\; \Theta(1),
\qquad
  W_k \;=\;
  \begin{cases}
    \Theta\!\Big(\dfrac{\log n_y}{\log\!\big(\frac{4n_y}{L_k}\,\log n_y\big)}\Big),
      & L_k \le n_y\log n_y,\\[0.75em]
    \Theta\!\big(\dfrac{L_k}{n_y}\big),
      & L_k \ge n_y\log n_y.
  \end{cases}
\]

By standard balls-into-bins tail bounds and maximum-load estimates
(e.g., \cite{raab1998balls}), we obtain
\[
  \Pr\!\big(\Zb^{L_k,1}\notin \cE_{L_k,1}\big)
  \;=\;
  \begin{cases}
    O\!\big(n_y^{-\,U_kL_k+1}\big), & L_k \ll n_y,\\[0.25em]
    \exp\!\big(-\Theta(n_y\log n_y)\big), & L_k = \Theta(n_y),\\[0.25em]
    \exp\!\big(-\Theta(L_k\log n_y)\big), & L_k \gg n_y,
  \end{cases}
\]
and
\[
  \Pr\!\big(\Zb^{L_k,1}\in \cE_{L_k,3}\big)\;=\;1 - n_y^{-\Omega(1)}.
\]

    \subsection{Reducing the Wrong Attention}

As discussed in \Cref{remark-fact-1-rec}, both $\ea^{L_k}$ and the attention gap remain bounded by a small constant at the beginning of stage~$k$. Hence, we can directly proceed to stage~1.2.2 of the convergence analysis as $\cT^2$. Moreover, by the symmetry between $[\Qb_{4,3}]_{s,s}$ and $[\Qb_{4,4}]_{s,s}$, we may, without loss of generality, restrict attention to a particular $s \in \tau(\X)$ and analyze the corresponding loss $\Loss^{L_k,2}_{5,s}$ in what follows.

\begin{induction}\label{induction-s22-rec}
    Given $F^{(T_{k-1})}$ 
    from the previous stage $\cT^{L_{k-1}}$ with $L_{k-1}=2^{k-1}$ for $k\geq 2$, let $T_{k}$ denote the first time that $\Loss^{L_k,2}_{5,s}$ decreases below  $e^{(-\frac{1}{2}+{3.01c_1})\cdot 2B}$.
       For all iterations $t\leq T_{k}$, we have the following holds
       \begin{enumerate}
          \item $\big[\Qb^{(t)}_{4,3}\big]_{s,s}+\big[\Qb^{(t)}_{4,4}\big]_{s,s}$ monotonically increases; 
          \item for any sample $\Zb^{2,1}$, we have $\attn^{(t)}_{\ans,1\to \pred,2}-\attn^{(t)}_{\ans,1\to \ans,1}\leq c_1$ for some sufficiently small constant $c_1=\frac{1.005\log d}{2B}>0$; 
          \item for any $\Zb^{2,1}$, we have $\attn^{(t)}_{\ans,1\to \ans,1}- \attn^{(t)}_{\ans,1\to \pred,2}\leq \min\Big\{\frac{L_{k}-1}{L_k}\ate^{L_k,2}-c_2,0\Big\}$, where $c_2=\frac{\log d}{4B}>0$ is some sufficiently small constant.
          \item for $p\in\{3,4\}$, for $s'\in\tau(\X)\not=s$, $|\big[\Qb^{(t)}_{4,p}\big]_{s,s'}|\leq O\bigg(\frac{[\Qb^{(t)}_{4,p}]_{s,s}}{d}\bigg)$ ; otherwise $[\Qb^{(t)}_{4,p}]_{s,s'}=0$.
       \end{enumerate}
\end{induction}

\subsubsection{Attention and Lambda Preliminaries}
\begin{lemma}\label{lem-s22-attn-rec}
    If \Cref{induction-s22-rec} holds for all iterations $[T_{k-1},t)$,then for any sample $\Zb^{L_k,1}$, we have  
         \begin{enumerate}
            \item $\ate^{L_k}\in \Big[\frac{4}{3}c_1 , 2\ate^{L_{k-1}} \Big]$; 
            \item   $\attn^{(t)}_{\ans,1\to \kk}-\attn^{(t)}_{\ans,1\to \kk'}\geq \Omega(1)$ for any $\kk,\kk'\in \cI^{L_k,1}\setminus \{(\ans,1),(\pred,2)\}$; 
   \item  $\big|\attn^{(t)}_{\ans,1\to \kk}-\attn^{(t)}_{\ans,1\to \kk'}\big|\leq \tilde{O}\big(\frac{1}{d}\big)$ for any $\kk,\kk'\in \cI^{L_k,1}\setminus \{(\ans,1)\}$.
            \end{enumerate} 

\end{lemma}

\begin{proof}
    The highest loss occurs when the sample $\Zb^{L_k}$ makes the wrong answer highly confused with the correct answer. Thus, we consider the case that $g_{\ell} \cdot y_0=g_2\cdot y_0$ for all $\ell\in [L_k]\setminus \{2\}$ and $y_0\neq y_1$.
    \begin{align*}
        &\log p_{F}(\Z_{\ans,2,5}|\Z^{(L_k,1)})\\
        &\leq \Theta(1)\cdot \max\Bigg\{
            e^{ \big(\attn^{(t)}_{\ans,1\to \pred,2}+\frac{L_k-1}{L_k}\ate^{L_k,2}-\attn^{(t)}_{\ans,1\to \ans,1}\big)2B
             - \big(\attn^{(t)}_{\ans,1\to \pred,2}-\frac{1}{L_k}\ate^{L_k,2}\big)2B}, \\
        &\qquad\qquad e^{\log d - \big(\attn^{(t)}_{\ans,1\to \pred,2}-\frac{1}{L_k}\ate^{L_k,2}\big)2B}
        \Bigg\}\\
        &\leq \Theta\!\left( e^{(\frac{L_k-1}{L_k}\ate^{L_k,2}+c_1)2B - \big(\attn^{(t)}_{\ans,1\to \pred,2}-\frac{1}{L_k}\ate^{L_k,2}\big)2B} \right) \\
        &\leq \Theta\!\left( e^{-\left(\tfrac{1}{2}-\frac{3}{2}\ate^{L_k,2}-c_1\right)2B} \right).
    \end{align*}

  Thus, if $\ate^{L_k}\leq \tfrac{4}{3}c_1$, we must have
    \[
    \Loss_{5,s}^{2,2}\;\leq\; \Theta\!\left(e^{(-\tfrac{1}{2} + 3\cdot\tfrac{2c_1}{3} + c_1)2B}\right)
    = \Theta\!\left(e^{-(\tfrac{1}{2} + 3c_1)2B}\right),
    \]
    which contradicts the definition of $T_{1,2,2,s}$.  
    Therefore, it must hold that $\ate^{L_k} \geq \tfrac{4}{3}c_1$ for all $t \leq T_{k}$.
    
\end{proof}

\begin{lemma}\label{lem-s22-lambda-1-rec}
    If \Cref{induction-s22-rec} holds for all iterations $[T_{k-1},t)$, then given $\Zb^{L_k,1}\notin \cE_{L_k,2}$, we have 
\begin{enumerate}
    \item for the prediction $j_2$, 
    we have 
    \begin{align*}
       & \Lambda^{(t)}_{5,j_2,r_{g_2\cdot y_1}}= \Big(\sum_{\ell\in [L_k]}\1_{g_{\ell}(y_1)=g_2(y_1)}\attn^{(t)}_{\ans,1\to \pred,\ell}-\attn^{(t)}_{\ans,1\to \ans,0}\Big) \cdot 2B + O\Big(\frac{B}{n_y}\Big)+ O(\delta).
    \end{align*}
    \item for the prediction $j'_2=\tau\big(g_2(y_0)\big)$, 
    we have 
\begin{align*}
    &\Lambda^{(t)}_{5,j'_2,r_{g_2\cdot y_0}}
     = \Big(\sum_{\ell\in [L_k]}\1_{g_{\ell}(y_0)=g_2(y_0)}\attn^{(t)}_{\ans,1\to \pred,\ell}-\attn^{(t)}_{\ans,1\to \ans,1}\Big) \cdot 2B + O\Big(\frac{B}{n_y}\Big)+ O(\delta).
\end{align*}

\item for other $j=g_2(y)\in\tau(\cY)$, noticing that for this case $y\neq y_0, y_1$, then we have 
\begin{align*}
    &\Lambda^{(t)}_{5,j,r_{g_2\cdot y}}
     = \Big(\sum_{\ell\in [L_k]}\1_{g_{\ell}(y)=g_2(y)}\attn^{(t)}_{\ans,1\to \pred,\ell}-\attn^{(t)}_{\ans,1\to \ans,1}-\attn^{(t)}_{\ans,1\to \ans,0}\Big) \cdot 2B \\
     &~~~~~~~~+ O\Big(\frac{B}{n_y}\Big)+ O(\delta).
\end{align*}
\end{enumerate}

\end{lemma}
\begin{lemma}\label{lem-s22-lambda-2-rec}
    If \Cref{induction-s22-rec} holds for all iterations $[T_{k-1},t)$, then given $\Zb^{L_k,1}\in \cE_{L_k,2}$, we have 
\begin{enumerate}
    \item for the prediction $j_2$, 
    we have 
    \begin{align*}
       & \Lambda^{(t)}_{5,j_2,r_{g_2\cdot y_1}}= \sum_{\ell\in [L_k]}\1_{g_{\ell}(y_1)=g_2(y_1)}\attn^{(t)}_{\ans,1\to \pred,\ell}\cdot 2B + O\Big(\frac{B}{n_y}\Big)+ O(\delta).
    \end{align*}

\item for other $j=\tau(g_2(y))\in\tau(\cY)$ with $y\neq y_1$,  then  we have 
\begin{align*}
    &\Lambda^{(t)}_{5,j,r_{g_2\cdot y}}
     = \Big(\sum_{\ell\in [L_k]}\1_{g_{\ell}(y)=g_2(y)}\attn^{(t)}_{\ans,1\to \pred,\ell}-\attn^{(t)}_{\ans,1\to \ans,1}-\attn^{(t)}_{\ans,1\to \ans,0}\Big) \cdot 2B \\
     &~~~~~~~~+ O\Big(\frac{B}{n_y}\Big)+ O(\delta).
\end{align*}
\end{enumerate}

\end{lemma}

\subsubsection{Gradient Lemma}
\begin{lemma}\label{lem-s22-gd1-rec}
    If \Cref{induction-s22-rec} holds for all iterations $t\in[T_{k-1}, T_{k})$,  given $s\in\tau(\X)$,   we have
    \begin{align*}
        &\Big[-\nabla_{\Q^{(t)}_{4,3}}{\Loss^{2,2}_{5}}\Big]_{s,s}+ \Big[-\nabla_{\Q^{(t)}_{4,4}}{\Loss^{2,2}_{5}}\Big]_{s,s} \\& \geq \Omega(B/d)\E\bigg[\ate^{L_k}\cdot  (1-\logit^{(t)}_{5,j_2}) \big| \tau(x_1)=s, \cE_{L_k,1}\cap\cE_{L_k,2}^c\bigg].
    \end{align*}
    

\end{lemma}

\begin{proof}  
    By \Cref{induction-s22-rec}, \Cref{lem-s22-lambda-1-rec}, and \Cref{lem-s22-lambda-2-rec}, we have 
    \begin{align}
        &\E\Big[ \cN^{(t)}_{s,L_k,\rom1} \Big]\notag\\
        &~~=\E\Big[ (1-\logit^{(t)}_{5,j_2})\cdot 
        \Big( -\sum_{\ell\neq 2}\attn^{(t)}_{{\ans,1} \rightarrow \pred,\ell}\cdot V_{j_2, r_{g_2\cdot y_1}}(g_{\ell})
               -\attn^{(t)}_{{\ans,1} \rightarrow \ans,0}\cdot V_{j_2, r_{g_2\cdot y_1}}(y_0)\notag\\
        &\quad
               +\big(1-\attn^{(t)}_{{\ans,1} \rightarrow \ans,1}-\attn^{(t)}_{{\ans,1} \rightarrow \pred,2}\big)\Lambda^{(t)}_{5,j_2, r_{g_2\cdot y_1}}
               \pm\tilde{O}(\sigma_0)\Big)\pm \tilde{O}(\delta^{q}) \Big]\\
               &~~~~\1_{\tau(x_1)=s}\1_{\Lambda^{(t)}_{5,j_2,r_{g_2 \cdot y_1}}\leq  B}.\notag
    \end{align}
    Similarly,
    \begin{align}
        &\E\Big[ \cN^{(t)}_{s,L_k,\rom2}\Big]\notag\\
        &~~=\E\Big[ \sum_{y\neq y_1}\logit^{(t)}_{5,\tau(g_2(y))}\cdot \ReLU^{\prime}(\Lambda^{(t)}_{5,\tau(g_2(y)), r_{g_2\cdot y}})\notag\\
        &\qquad\qquad \Big( \sum_{\ell\neq 2}\attn^{(t)}_{{\ans,1} \rightarrow \pred,\ell}\cdot V_{\tau(g_2(y)), r_{g_2\cdot y}}(g_{\ell})
               +\attn^{(t)}_{{\ans,1} \rightarrow \ans,0}\cdot V_{\tau(g_2(y)), r_{g_2\cdot y}}(y_0)\notag\\
        &\qquad\qquad\qquad
               -\big(1-\attn^{(t)}_{{\ans,1} \rightarrow \ans,1}-\attn^{(t)}_{{\ans,1} \rightarrow \pred,2}\big)\Lambda^{(t)}_{5,\tau(g_2(y)), r_{g_2\cdot y}}
               \pm\tilde{O}(\sigma_0)\Big)\pm \tilde{O}(\delta^{q}) \Big]\1_{\tau(x_1)=s}.\notag
    \end{align}
    
    Consider the event $\cE_{L_{k},1}$. 
    For $\Zb^{L_k,1}\in  \cE_{L_k,1} \cap \cE^c_{L_k,2}$, we have $\Lambda^{(t)}_{5,j_2,r_{g_2 \cdot y_1}}\leq  B$, since 
    \begin{align*}
        \Lambda^{(t)}_{5,j_2,r_{g_2 \cdot y_1}}
        &=\sum_{\ell\neq 2}\attn^{(t)}_{{\ans,1} \rightarrow \pred,\ell}\cdot V_{j_2, r_{g_2\cdot y_1}}(g_{\ell})\1_{g_{\ell}(y_1)=j_2}
          -\attn^{(t)}_{{\ans,1} \rightarrow \ans,0}\cdot V_{j_2, r_{g_2\cdot y_1}}(y_0)\\
        &\leq 2\Big(\attn^{(t)}_{{\ans,1} \rightarrow \pred,2} +(U_k-1/L_k)\cdot \ate^{L_k}\Big)B\\
        &\leq 2\Big(\tfrac{1-\ate^{L_k}+c_1}{2} +(U_k-1/L_k)\cdot \ate^{L_k}\Big)B\\
        &=2\Big(\tfrac{1+c_1}{2} -(\tfrac{1}{2}+1/L_k-U_k)\cdot \ate^{L_k}\Big)B.
    \end{align*}
    Thus, provided $\ate^{L_k}\geq 4c_1/3$ and choosing $U_k=  \left \lfloor{1+\frac{L_k}{8}}\right \rfloor /L_k$, we indeed have $\Lambda^{(t)}_{5,j_2,r_{g_2 \cdot y_1}}\le B$.
    
    \begin{itemize}
    \item For $j=\tau(g_2(y))$ in $\cN^{(t)}_{s,L_k,\rom2}$:
        \begin{itemize}
            \item $y=y_0\neq y_1$,
            \begin{align}
                &\sum_{\ell\neq 2}\attn^{(t)}_{{\ans,1} \rightarrow \pred,\ell}\cdot V_{\tau(g_2(y)), r_{g_2\cdot y}}(g_{\ell})
                  +\attn^{(t)}_{{\ans,1} \rightarrow \ans,0}\cdot V_{\tau(g_2(y)), r_{g_2\cdot y}}(y_0)\notag\\
                &\geq \sum_{\ell\neq 2}\attn^{(t)}_{{\ans,1} \rightarrow \pred,\ell}\cdot V_{\tau(g_2(y)), r_{g_2\cdot y}}(g_{\ell})\1_{g_{\ell}(y)\neq g_2(y)} \notag\\
                &\geq -O\Big(\tfrac{B}{n_y}\Big)\cdot \big(1-\attn^{(t)}_{{\ans,1} \rightarrow \ans,1}-\attn^{(t)}_{{\ans,1} \rightarrow \pred,2}\big). 
                \label{eq-gd-neg-1-rec}
            \end{align}
            \item  $y\neq y_0, y_1$. If there exists at least one ${\ell}\neq 2$ such that $g_{\ell}(y)=g_{2}(y)$, then by cancellation,
            \begin{align}
                &\sum_{\ell\neq 2}\attn^{(t)}_{{\ans,1} \rightarrow \pred,\ell}\cdot V_{\tau(g_2(y)), r_{g_2\cdot y}}(g_{\ell})
                  +\attn^{(t)}_{{\ans,1} \rightarrow \ans,0}\cdot V_{\tau(g_2(y)), r_{g_2\cdot y}}(y_0) \notag\\
                &\geq -O\Big(\tfrac{B}{n_y}\Big)\cdot \big(1-\attn^{(t)}_{{\ans,1} \rightarrow \ans,1}-\attn^{(t)}_{{\ans,1} \rightarrow \pred,2}\big).
                \label{eq-gd-neg-2-rec} 
            \end{align}
            Else, pick $y'\neq y,y_1$ so that 
            $\Lambda^{(t)}_{5,\tau(g_2(y')), r_{g_2\cdot y'}}\geq  \Lambda^{(t)}_{5,\tau(g_2(y)), r_{g_2\cdot y}}$. 
            Then 
            \[
                \logit_{5,\tau(g_2(y))}\leq  O\Big(\tfrac{1}{n_y}\Big)\big(1-\logit_{5,j_2}\big),
            \]
            which implies 
            \begin{align}
                &\logit^{(t)}_{5,j}\Big(\sum_{\ell\neq 2}\attn^{(t)}_{{\ans,1} \rightarrow \pred,\ell}\cdot V_{\tau(g_2(y)), r_{g_2\cdot y}}(g_{\ell})
                  +\attn^{(t)}_{{\ans,1} \rightarrow \ans,0}\cdot V_{\tau(g_2(y)), r_{g_2\cdot y}}(y_0)\Big)\notag\\
                &\geq -O\Big(\tfrac{B}{L_k n_y}\Big)\cdot \big(1-\logit_{5,j_2}\big)\cdot \big(1-\attn^{(t)}_{{\ans,1} \rightarrow \ans,1}-\attn^{(t)}_{{\ans,1} \rightarrow \pred,2}\big).
                \label{eq-gd-neg-3-rec} 
            \end{align}
        \end{itemize} 
    \item For $\cN^{(t)}_{s, L_k,\rom1}$:
        \begin{itemize}
            \item If $\Zb^{L_k,1}\in \cE_{L_k,1}\cap\cE^c_{L_k,2}$, then
            \begin{align}
                & -\sum_{\ell\neq 2}\attn^{(t)}_{{\ans,1} \rightarrow \pred,\ell}\cdot V_{j_2, r_{g_2\cdot y_1}}(g_{\ell})
                  -\attn^{(t)}_{{\ans,1} \rightarrow \ans,0}\cdot V_{j_2, r_{g_2\cdot y_1}}(y_0)\notag\\
                &\qquad +\big(1-\attn^{(t)}_{{\ans,1} \rightarrow \ans,1}-\attn^{(t)}_{{\ans,1} \rightarrow \pred,2}\big)\Lambda^{(t)}_{5,j_2, r_{g_2\cdot y_1}}\notag\\
                &\geq -\sum_{\ell\neq 2}\attn^{(t)}_{{\ans,1} \rightarrow \pred,\ell}\1_{g_{\ell}=j_2}\cdot V_{j_2, r_{g_2\cdot y_1}}(g_{\ell})
                     -\attn^{(t)}_{{\ans,1} \rightarrow \ans,0}\cdot V_{j_2, r_{g_2\cdot y_1}}(y_0)\notag\\
                &\qquad +\big(1-\attn^{(t)}_{{\ans,1} \rightarrow \ans,1}-\attn^{(t)}_{{\ans,1} \rightarrow \pred,2}\big)\Lambda^{(t)}_{5,j_2, r_{g_2\cdot y_1}}\notag\\
                &\geq \Big(-(U_k-\tfrac{1}{L_k})\cdot 2B+\Lambda^{(t)}_{5,j_2, r_{g_2\cdot y_1}}\Big)\cdot \ate^{L_k}
                \;\;\geq\;\; \big(\Lambda^{(t)}_{5,j_2, r_{g_2\cdot y_1}}-\tfrac{B}{4}\big)\cdot \ate^{L_k}.
                \label{eq-gd-pos-1-rec}
            \end{align}
            \item Else, for $\Zb^{L_k,1}\in \cE^c_{L_k,1}\cup\cE_{L_k,2}$, note that having more identical predictions $g_{\ell}(y_1)=j_2$ or less distraction (i.e., $y_0=y_1$) increases $\Lambda^{(t)}_{5,j_2, r_{g_2\cdot y_1}}$. Hence
            \[
                \E\big[(1-\logit_{5,j_2})\mid \cE^c_{L_k,1}\cup\cE_{L_k,2}\big]
                \;\le\; O(1)\cdot \E\big[(1-\logit_{5,j_2})\mid \cE_{L_k,1}\cap\cE^c_{L_k,2}\big],
            \]
            and
            \begin{align}
                &\Big|-\sum_{\ell\neq 2}\attn^{(t)}_{{\ans,1} \rightarrow \pred,\ell}\cdot V_{j_2, r_{g_2\cdot y_1}}(g_{\ell})
                       -\attn^{(t)}_{{\ans,1} \rightarrow \ans,0}\cdot V_{j_2, r_{g_2\cdot y_1}}(y_0)\notag\\
                &\qquad +\big(1-\attn^{(t)}_{{\ans,1} \rightarrow \ans,1}-\attn^{(t)}_{{\ans,1} \rightarrow \pred,2}\big)\Lambda^{(t)}_{5,j_2, r_{g_2\cdot y_1}}\Big|
                \;\leq\; O(1)\cdot \ate^{L_k}\cdot B.
                \label{eq-gd-pos-2-rec}
            \end{align}
        \end{itemize}
    \end{itemize}
    
    Putting the above together, we have 
    \begin{align*}
        &\Big[-\nabla_{\Q^{(t)}_{4,3}}{\Loss^{2,2}_{5}}\Big]_{s,s}+ \Big[-\nabla_{\Q^{(t)}_{4,4}}{\Loss^{2,2}_{5}}\Big]_{s,s} \\
        &=\E\Big[ \cN^{(t)}_{s,L_k,1} \1_{\cE_{L_k,1}\cap\cE_{L_k,2}^c}\Big]
          +\E\Big[ \cN^{(t)}_{s,L_k,1} \1_{\Lambda^{(t)}_{5,j_2, r_{g_2\cdot y_1}}\leq B}\1_{\cE^c_{L_k,1}\cup\cE_{L_k,2}}\Big]\\
        &\qquad +\E\Big[ \cN^{(t)}_{s,L_k,2} \Big]\pm \tilde{O}(\sigma_0^q)\\
        &=\E\Big[ (\cN^{(t)}_{s,L_k,1} +\cN^{(t)}_{s,L_k,2})\1_{\cE_{L_k,1}\cap\cE_{L_k,2}^c}\Big]\\
        &\qquad +\E\Big[ \cN^{(t)}_{s,L_k,1} \1_{\Lambda^{(t)}_{5,j_2, r_{g_2\cdot y_1}}\leq B}\1_{\cE^c_{L_k,1}\cup\cE_{L_k,2}}\Big]
                   +\E\Big[ \cN^{(t)}_{s,L_k,2} \1_{\cE^c_{L_k,1}\cup\cE_{L_k,2}}\Big]\pm \tilde{O}(\sigma_0^q).
    \end{align*}
    First, by \eqref{eq-gd-neg-1-rec}, \eqref{eq-gd-neg-2-rec}, \eqref{eq-gd-pos-1-rec},
    \begin{align*}
        &\E\Big[ (\cN^{(t)}_{s,L_k,1} +\cN^{(t)}_{s,L_k,2})\1_{\cE_{L_k,1}\cap\cE_{L_k,2}^c}\Big]\\
        &\geq  \E\Big[ (1-\logit^{(t)}_{5,j_2})\Big(\big(\Lambda^{(t)}_{5,j_2, r_{g_2\cdot y_1}}-\tfrac{B}{4}\big)\ate^{L_k}\Big)
               -(1-\logit^{(t)}_{5,j_2})\Big(O(\tfrac{B}{n_y})\ate^{L_k}+ \max_{y\neq y_1}\Lambda^{(t)}_{5,\tau(g_2(y)), r_{g_2\cdot y}} \Big)\\
        &\hspace{7.5cm}\cdot\1_{\tau(x_1)=s} \1_{\cE_{L_k,1}\cap\cE_{L_k,2}^c}\Big]\\
        &\geq \Omega(B/d)\cdot \E\Big[\ate^{L_k}\cdot  (1-\logit^{(t)}_{5,j_2}) \mid \tau(x_1)=s, \cE_{L_k,1}\cap\cE_{L_k,2}^c\Big]. 
    \end{align*}
    Moreover, by \eqref{eq-gd-pos-2-rec},
    \begin{align*}
        &\E\Big[ \cN^{(t)}_{s,L_k,1} \1_{\Lambda^{(t)}_{5,j_2, r_{g_2\cdot y_1}}\leq B/2}\1_{\cE^c_{L_k,1}\cup\cE_{L_k,2}}\Big]\\
        &\geq - \Pr(\Zb^{L_k,1}\notin \cE^c_{L_k,1})\cdot O(B/d)\cdot 
              \E\Big[\ate^{L_k}\cdot  (1-\logit^{(t)}_{5,j_2}) \mid \tau(x_1)=s, \cE^c_{L_k,1}\cup\cE_{L_k,2}\Big]\\
        &\geq - \Big(O(1/n^{\Omega(1)})+1/n_y\Big)\cdot O(B/d)\cdot 
              \E\Big[\ate^{L_k}\cdot  (1-\logit^{(t)}_{5,j_2}) \mid \tau(x_1)=s, \cE_{L_k,1}\cap\cE_{L_k,2}^c\Big],
    \end{align*}
    and, by \eqref{eq-gd-neg-1-rec}, \eqref{eq-gd-neg-2-rec}, \eqref{eq-gd-neg-3-rec},
    \begin{align*}
       & \E\Big[ \cN^{(t)}_{s,L_k,2} \1_{\cE^c_{L_k,1}\cup\cE_{L_k,2}}\Big]
        \;\geq\; \tfrac{1}{n_y}\cdot O(B/d)\cdot 
                  \E\Big[\ate^{L_k}\cdot  (1-\logit^{(t)}_{5,j_2}) \mid \tau(x_1)=s, \cE^c_{L_k,1}\cup\cE_{L_k,2}\Big].
    \end{align*}
    Therefore,
    \begin{align*}
        &\Big[-\nabla_{\Q^{(t)}_{4,3}}{\Loss^{2,2}_{5}}\Big]_{s,s}+ \Big[-\nabla_{\Q^{(t)}_{4,4}}{\Loss^{2,2}_{5}}\Big]_{s,s} \\
        &\geq \Omega(B/d)\cdot \E\Big[\ate^{L_k}\cdot  (1-\logit^{(t)}_{5,j_2}) \mid \tau(x_1)=s, \cE_{L_k,1}\cap\cE_{L_k,2}^c\Big]. 
    \end{align*}
    \end{proof}

\begin{lemma}\label{lem-s22-gd2-rec}
    If \Cref{induction-s22-rec} holds for all iterations $t\in[T_{k-1}, T_{k})$,  given $s\in\tau(\X)$,   for $[\Qb_{4,p}]_{s,s'}$, $p\in\{3,4\}$, $s'\not=s\in\tau(\X)$,  we have
    \begin{align*}
        \Big|\Big[-\nabla_{\Q^{(t)}_{4,p}}{\Loss^{2,2}_{5}}\Big]_{s,s'}\Big| \leq O(\frac{1}{d}) \Big|\Big[-\nabla_{\Q^{(t)}_{4,p}}{\Loss^{2,2}_{5}}\Big]_{s,s}\Big| .
    \end{align*}
  \end{lemma}
  \subsubsection{Attention gap is small}
  \begin{lemma}\label{lem-s22-attention-gap-1-rec}
      If \Cref{induction-s22-rec} holds for all iterations $t\in [T_{k-1}, T_{k})$, then for any sample $\Zb^{L_k,1}$, we have 
     \begin{align*}
       \attn^{(t)}_{{\ans,1} \rightarrow \pred,2} - \attn^{(t)}_{{\ans,1} \rightarrow \ans,1}\leq c_1,
     \end{align*}
     where $c_1>0$ is a small constant.
   \end{lemma}
   \begin{proof}
    The proof follows the same high-level idea as \Cref{lem-s21-attention-gap-non}. 
    The key observation is that for each $\Zb^{L_k,1}$, there always exists at least one $y \neq y_1$ such that, 
    for some index $j$ appearing in the gradient term $\cN^{(t)}_{s,L_k,2}$, we have
    \[
       \Lambda^{(t)}_{5,j, r_{g_2\cdot y}}
       \;\geq\;
       2\big(\attn^{(t)}_{{\ans,1} \rightarrow \pred,2}
           - \attn^{(t)}_{{\ans,1} \rightarrow \ans,1}\big)B.
    \]
    Let $\hat{\cJ}$ denote the set of indices $j$ achieving 
    $\arg\max_{y \notin \{y_1\}}\Lambda^{(t)}_{5,j, r_{g_2\cdot y}}$, 
    and pick one such index $j'=g_2(y')$. 
    Define $\tilde{T}$ to be the first time when
    $\attn^{(t)}_{{\ans,1} \rightarrow \pred,2} - \attn^{(t)}_{{\ans,1} \rightarrow \ans,1}\geq c_1$. 
    At this time, we have
    \[
       \sum_{j\in\hat{\cJ}} \logit^{(\tilde{T})}_{5,j}
       \;=\;
       \Big(1 - O(d^{-\Omega(1)})\Big)\big(1-\logit^{(\tilde{T})}_{5,j_2}\big).
    \]
    
    Plugging this into the gradient expressions yields
    \begin{align*}
        & \cN^{(\tilde{T})}_{s,3,L_k,\rom 1}+\cN^{(\tilde{T})}_{s,3,L_k,\rom 2} \\
        &\quad = \attn^{(\tilde{T})}_{{\ans,1} \rightarrow \pred,2} \cdot 
           \bigg( \Big(1-O(d^{-\Omega(1)})\Big)\big(1-\logit^{(\tilde{T})}_{5,j_2}\big)\cdot
                  \Big(\Lambda^{(\tilde{T})}_{5,j',r_{g_2\cdot y'}} - \Lambda^{(\tilde{T})}_{5,j_2,r_{g_2\cdot y_1}}\Big) \\
        &\qquad\quad + O(d^{-\Omega(1)})\big(1-\logit^{(\tilde{T})}_{5,j_2}\big)\cdot
                  \Big(V_{5,j_2,r_{g_2\cdot y_1},2}(g_2)-\Lambda^{(\tilde{T})}_{5,j_2,r_{g_2\cdot y_1}}\Big) \bigg)\1_{\tau(x_1)=s},
    \end{align*}
    and similarly
    \begin{align*}
        & \cN^{(\tilde{T})}_{s,4,L_k,\rom 1}+\cN^{(\tilde{T})}_{s,4,L_k,\rom 2} \\
        &\quad = \attn^{(\tilde{T})}_{{\ans,1} \rightarrow \pred,2} \cdot 
           \bigg( \Big(1-O(d^{-\Omega(1)})\Big)\big(1-\logit^{(\tilde{T})}_{5,j_2}\big)\cdot
                  \Big(-V_{j'_2,r_{g_2\cdot y_0}}(y_1)
                       +\Lambda^{(\tilde{T})}_{5,j',r_{g_2\cdot y'}}
                       -\Lambda^{(\tilde{T})}_{5,j_2,r_{g_2\cdot y_1}}\Big) \\
        &\qquad\quad + O(d^{-\Omega(1)})\big(1-\logit^{(\tilde{T})}_{5,j_2}\big)\cdot
                  \Big(V_{5,j_2,r_{g_2\cdot y_1},2}(y_1)
                       -\Lambda^{(\tilde{T})}_{5,j_2,r_{g_2\cdot y_1}}\Big) \bigg)\1_{\tau(x_1)=s}.
    \end{align*}
    
    Since $\Lambda^{(\tilde{T})}_{5,j',r_{g_2\cdot y'}} - \Lambda^{(\tilde{T})}_{5,j_2,r_{g_2\cdot y_1}} \leq 0$ 
    and 
    \[
       -V_{j_2,r_{g_2\cdot y_1}}(y_1)
       -\Lambda^{(\tilde{T})}_{5,j_2,r_{g_2\cdot y_1}}
       +\Lambda^{(\tilde{T})}_{5,j_2,r_{g_2\cdot y'  }}
       \;\geq\; \Omega(B),
    \]
    we obtain
    \[
       \Big[-\nabla_{\Q^{(\tilde{T})}_{4,4}}\Loss_{5}^{2,2}\Big]_{s,s}
       + \Big[\nabla_{\Q^{(\tilde{T})}_{4,3}}\Loss_{5}^{2,2}\Big]_{s,s}
       \;\;\geq\;\; \Omega\!\left(\tfrac{B}{d}\right)\cdot 
       \E\big[1-\logit^{(\tilde{T})}_{5,j_2}\;\big|\;\tau(x_1)=s\big].
    \]
    
    Therefore, the quantity 
    $\attn^{(t)}_{{\ans,1} \rightarrow \pred,2} - \attn^{(t)}_{{\ans,1} \rightarrow \ans,1}$ 
    cannot further increase, and we conclude that
    \[
       \attn^{(t)}_{{\ans,1} \rightarrow \pred,2}
       - \attn^{(t)}_{{\ans,1} \rightarrow \ans,1}
       \;\leq\; c_1.
    \]
    \end{proof}

   \begin{lemma}\label{lem-s22-attention-gap-2-rec}
      If \Cref{induction-s22-rec} holds for all iterations $t\in [T_{k-1}, T_{k})$, then for any sample $\Zb^{L_k,1}$, we have 
     \begin{align*}
      \attn^{(t)}_{{\ans,1} \rightarrow \ans,1}-  \attn^{(t)}_{{\ans,1} \rightarrow \pred,2} \leq \frac{L_k-1}{L_k}\ate^{L_k}-c_2,
     \end{align*}
     where $c_2=\frac{\log d }{2B}>0$ is a small constant.
   \end{lemma}
   \begin{proof}
    The argument parallels \Cref{lem-s22-attention-gap-2-non}, but in the recursive setting we first note that for every $\Zb^{L_k,1}\in\cE_{L_k,2}^c$,
    \[
      \Lambda^{(t)}_{5,\tau(g_2(y_0)),\,r_{g_2\cdot y_0}}
      \;\le\;
      \attn^{(t)}_{\ans,1\to\pred,2}
      -\attn^{(t)}_{\ans,1\to\ans,1}
      +\frac{L_k-1}{L_k}\,\ate^{L_k}.
    \]
    Let $\tilde T$ be the first time such that
    \[
      \E\Big[
        \attn^{(t)}_{\ans,1\to\pred,2}
        -\attn^{(t)}_{\ans,1\to\ans,1}
        +\frac{L_k-1}{L_k}\,\ate^{L_k}
        \;\Big|\;\tau(x_1)=s
      \Big]
      \;\le\; \frac{\log d}{4.02\,B}.
    \]
    Then, for any $y$,
    \begin{align*}
       & \Lambda^{(\tilde T)}_{5,\tau(g_2(y)),\,r_{g_2\cdot y}}
     \\
     & \;=\;
      2\big(\attn^{(\tilde T)}_{\ans,1\to\pred,2}
          +\!\!\sum_{\ell\neq 2}\!\attn^{(\tilde T)}_{\ans,1\to\pred,\ell}\,\1_{g_\ell(y)=g_2(y)}
          -\attn^{(\tilde T)}_{\ans,1\to\ans,0}\big)\,B
      \;\le\; \frac{\log d}{2.01}.
        \end{align*}

    Hence the corresponding wrong-class logit mass is small:
    \[
      \logit^{(1)}_{5,j_2}
      \;\le\; O\!\big(d^{-1.01/2.01}\big)\,\big(1-\logit^{(t)}_{5,j_2}\big).
    \]
    
    Plugging this into the gradient decomposition, we obtain
    \begin{align*}
      &\cN^{(\tilde T)}_{s,3,L_k,\rom 1}+\cN^{(\tilde T)}_{s,3,L_k,\rom 2} \\
      &\qquad
      = \attn^{(\tilde T)}_{\ans,1\to\pred,2}\,
        \Big(
          (1-\logit^{(\tilde T)}_{5,j_2})
          \big(V_{5,j_2,r_{g_2\cdot y_1},2}(g_2)-\Lambda^{(\tilde T)}_{5,j_2,r_{g_2\cdot y_1}}\big) \\
      &\qquad\qquad\qquad
          - O\!\big(d^{-1.01/2.01}\big)\,(1-\logit^{(\tilde T)}_{5,j_2})
            \sum_{y\neq y_0}\big(V_{5,j_2,r_{g_2\cdot y},2}(g_2)-\Lambda^{(\tilde T)}_{5,j_2,r_{g_2\cdot y}}\big)
        \Big)\,\1_{\tau(x_1)=s} \\
      &\qquad \ge
      \attn^{(\tilde T)}_{\ans,1\to\pred,2}\cdot
      \Omega(1)\cdot(1-\logit^{(\tilde T)}_{5,j_2})\cdot
      \big(V_{5,j_2,r_{g_2\cdot y_1},2}(g_2)-\Lambda^{(\tilde T)}_{5,j_2,r_{g_2\cdot y_1}}\big)\,\1_{\tau(x_1)=s},
    \end{align*}
    and
    \begin{align*}
      &\cN^{(\tilde T)}_{s,4,L_k,\rom 1}+\cN^{(\tilde T)}_{s,4,L_k,\rom 2} \\
      &\qquad
      = \attn^{(\tilde T)}_{\ans,1\to\pred,2}\,
        \Big(
          (1-\logit^{(\tilde T)}_{5,j_2})
          \big(V_{5,j_2,r_{g_2\cdot y_1},2}(y_1)-\Lambda^{(\tilde T)}_{5,j_2,r_{g_2\cdot y_1}}\big) \\
      &\qquad\qquad\qquad
          - O\!\big(d^{-1.01/2.01}\big)\,(1-\logit^{(\tilde T)}_{5,j_2})
            \sum_{y\neq y_0}\big(V_{5,j_2,r_{g_2\cdot y},2}(y_1)-\Lambda^{(\tilde T)}_{5,j_2,r_{g_2\cdot y}}\big)
        \Big)\,\1_{\tau(x_1)=s} \\
      &\qquad \le
      -\,\attn^{(\tilde T)}_{\ans,1\to\pred,2}\cdot
      \Omega(1)\cdot(1-\logit^{(\tilde T)}_{5,j_2})\cdot
      \Lambda^{(\tilde T)}_{5,j_2,r_{g_2\cdot y_1}}\,\1_{\tau(x_1)=s}.
    \end{align*}
    Therefore,
    \begin{align}
      \big[-\nabla_{\Q^{(\tilde T)}_{4,3}}\Loss_{5}^{2,2}\big]_{s,s}
      &\;\ge\;
      \Omega\!\left(\frac{B}{d}\right)\cdot
      \E\big[\,1-\logit^{(\tilde T)}_{5,j_2}\,\big|\,\tau(x_1)=s,\,\cE_{L_k,2}^c\big], \\
      \big[-\nabla_{\Q^{(\tilde T)}_{4,4}}\Loss_{5}^{2,2}\big]_{s,s}
      &\;\le\;
      -\,\Omega\!\left(\frac{B}{d}\right)\cdot
      \E\big[\,1-\logit^{(\tilde T)}_{5,j_2}\,\big|\,\tau(x_1)=s,\,\cE_{L_k,2}^c\big].
    \end{align}
    
    Finally, recall that
    \[
      \attn^{(t)}_{\ans,1\to\pred,2}
      +\attn^{(t)}_{\ans,1\to\pred,1}
      -\attn^{(t)}_{\ans,1\to\ans,1}
      \;=\;
      \frac{e^{[\Q^{(t)}_{4,3}]_{s,s}}-e^{[\Q^{(t)}_{4,4}]_{s,s}}-e^{\tilde{O}(1/d)}}
           {e^{[\Q^{(t)}_{4,3}]_{s,s}}+e^{[\Q^{(t)}_{4,4}]_{s,s}}+e^{\tilde{O}(1/d)}}.
    \]
    Combining the two gradient inequalities shows that the quantity
    \[
      \attn^{(t)}_{\ans,1\to\pred,2}
      +\frac{L_k-1}{L_k}\,\ate^{L_k}
      -\attn^{(t)}_{\ans,1\to\ans,1}
    \]
    must strictly decrease at time $\tilde T+1$.
    \end{proof}

  \subsubsection{At the end of stage}
\begin{lemma}[At the end of stage]\label{lem-end-rec} \Cref{induction-s22-rec} holds for all iterations $T_{k-1}<t\leq T_k={O}(\frac{\poly(d)}{\eta B})$. At the end of stage $k$, we have
    \begin{enumerate}[(a)]
        \item Attention concentration: $\ea^{L_k}\leq 5.04c_1$ for some small constant $c_1=\frac{1.005\log d}{B}>0$;
        \item Loss convergence: $\Loss^{L_k,2}_{F^{(T_{k-1})},5}\leq e^{(-\frac{1}{2}+3.01c_1)2B}=\frac{1}{\poly d}$.
    \end{enumerate}
\end{lemma}

\begin{proof}
    By \Cref{lem-s22-attn-rec,lem-s22-gd1-rec}, the diagonal entries
    $[\Qb^{(t)}_{4,3}]_{s,s}$ and $[\Qb^{(t)}_{4,4}]_{s,s}$ keep increasing
    until the total wrong-attention mass satisfies $\ate \ge \tfrac{4c_1}{3}$.
    Combining \Cref{lem-s22-gd1-rec,lem-s22-gd2-rec,lem-s22-attention-gap-1-rec,lem-s22-attention-gap-2-rec},
    there exists a stopping time
    \[
      T_k \;=\; O\!\left(\frac{\poly(d)}{\eta B}\right),
    \]
    such that \Cref{induction-s22-rec} holds for all $t\in[T_1,\,T_k)$.
    
    \medskip\noindent
   To obtain an upper bound on $\ate^{L_k}$ at the end of stage $k$, it suffices to work on the main event $\cE_{L_k,3}$.
    
    \smallskip
    \emph{Case 1: $L_k=O(1)$.}
    For each $y\in\cY$, at most one index $\ell\neq 2$ can satisfy $g_\ell(y)=g_2(y)$,
    and thus the number of collisions per $y$ is uniformly bounded by $W_k$.
    Hence,
    \begin{align*}
      &\log p_F\!\big(Z_{\ans,2,5}\mid Z^{(2,1)}\big)\\
      &\ge \Theta(1)\cdot
         \exp\!\left(
           \log d - \big(\attn^{(t)}_{\ans,1\to\pred,2}
           + \sum_{\ell\neq 2}\!\attn^{(t)}_{\ans,1\to\pred,\ell}\,\1_{g_\ell(y_1)=g_2(y_1)}
           - \attn^{(t)}_{\ans,1\to\ans,0}\big)\,2B
         \right) \\
      &\ge \Theta(1)\cdot
         \exp\!\left(
           \log d - \big(\attn^{(t)}_{\ans,1\to\pred,2}
           + (1/L_k-1/L_k)\,\ate^{L_k}\big)\,2B
         \right) \\
      &\ge \Theta\!\left(
           \exp\!\big(-(\tfrac{1}{2}+\tfrac{c_1}{2}-\tfrac{\ate^{L_k}}{2}-\tfrac{\log d}{B})\,2B\big)
         \right),
    \end{align*}
    where the last line uses
    $\attn^{(t)}_{\ans,1\to\pred,2}\le \tfrac{1}{2}(1-\ate^{L_k}+c_1)
    = \tfrac{1+c_1}{2}-\tfrac{\ate^{L_k}}{2}$.
    
    \smallskip
    \emph{Case 2: $L_k=\omega(1)$.}
    If $L_k\le n_y\log n_y$ then $W_k/L_k=O(1/L_k)$; otherwise $W_k/L_k=O(1/n_y)$.
    In either regime $W_k/L_k=o(1)$, and we similarly obtain
    \begin{align*}
      \log p_F\!\big(Z_{\ans,2,5}\mid Z^{(2,1)}\big)
      &\ge \Theta(1)\cdot
         \exp\!\left(
           \log d - \big(\attn^{(t)}_{\ans,1\to\pred,2}
           + (o(1)-1/L_k)\,\ate^{L_k}\big)\,2B
         \right) \\
      &\ge \Theta\!\left(
           \exp\!\big(-(\tfrac{1}{2}+\tfrac{c_1}{2}-\tfrac{\ate^{L_k}}{2}-\tfrac{\log d}{B})\,2B\big)
         \right).
    \end{align*}
    
    \medskip
    At convergence we therefore have
    \[
      \tfrac{1}{2}+\tfrac{c_1}{2}-\tfrac{\ate^{L_k}}{2}-\tfrac{\log d}{B}
      \;\;\ge\;\; \tfrac{1}{2}-3.01\,c_1,
    \]
    which rearranges to
    \[
      \ate^{L_k}\leq 2(3.02c_1+0.5c_1-c_1)=5.04c_1.
    \]
    In particular, since $B/\log d$ can be taken as a sufficiently large constant,
    we obtain the clean stage-end bound
    \[
      \ate^{L_k} \;=\; O(c_1).
    \]
    \end{proof}
    
\subsection{Proof of Main Theorem}
\begin{theorem}[Recursive self-training (Restatement)]\label{thm:length-gen-self-training-restate}
Assume the distribution \(\cD^{L}\) induced from \(\lego(\cX,\cG,\cY)\) satisfies \Cref{assump:asymptotic-regime}, \ref{assump:lego-data-distribution} and \ref{assump:structure-2}, and assume the transformer network satisfies Assumption~\ref{assump:init}, \ref{assumption:output-bound}, \ref{assump-Q-structure}, and $n_y<m\ll \log^2 d$. Then for any $1 \leq k < \log_2|\cX|$, the transformer $F^{(T_{k})} $ trained via Algorithm~\ref{alg:cot-symmetry-training} up to length \(L_{k} = 2^{k}\) and \(T_k = O(\frac{\poly(d)}{\eta})\) satisfies: \(F^{(T_k)}\) is able to solve \(\cT^{L_{k+1}}\) with \({L_{k+1}}=2^{k+1}\), i.e., 
    $$
    \mathrm{Acc}_{L_{k+1}}\bigl(F^{(T_{k})}\bigr) = 1 - O(1 /\poly(d)).
    $$
    \end{theorem}

    \begin{proof}
        By \Cref{lem-end-rec}, at the time $T_{k}$, we have $\ea^{L_{k}}\leq 5.04 c_1$,  combining with \Cref{induction-s22-rec}, non-diagonal entry of $\Qb_{p,q}$ remains close to $0$, thus moving to the next stage, we have $\ea^{L_{k+1},\ell}\leq 4\ea^{L_{k}}\leq 20.16 c_1$ (notice that $20.16c_1$ could be sufficiently small e.g., $0.01$ since we can set a reasonably large B) for all $\ell<L_{k+1}$. Moreover, $\Delta^{L_{k+1},\ell}\leq \Delta^{L_{k},1}$. Hence, we have
\begin{align*}
    &\E_{Z^{L_{k+1}}\sim \cD^{L_{k+1}}}\bigg[\E_{ \hat{Z}_{\ans,\ell+1} \sim p_F^{(T_k)}(\cdot |Z^{L_{k+1},\ell }) } [\1_{\{\hat{Z}_{\ans,\ell+1} \neq Z_{\ans,\ell+1}\}}]\bigg]\\
    &\leq O(1)\cdot \E_{Z^{L_{k+1}}\sim \cD^{L_{k+1}}}\bigg[1-\logit_{5,\tau(\Zb_{\ans,\ell+1, 5})}\bigg]\\
    &\leq O(1)\cdot e^{\big(-(\frac{1}{2}-\ate^{L_{k+1}}-(\ate^{L_{k+1}}-c_2)/2-\ate^{L_{k+1}})\big)\cdot 2B}\\
    &\leq O(1)\cdot e^{(-\frac{1}{2}-\frac{\log d}{4B}+\frac{5}{2}\cdot 0.01)2B}=O\Big(\frac{1}{\poly d}\Big).
\end{align*}
Thus $\mathrm{Acc}_{L_{k+1}}(F^{T_k})\geq 1-O\big(\frac{1}{\poly d}\big)$.
    \end{proof}

\stopcontents[sections]

\end{document}